\documentclass{article} 
\usepackage{iclr2022_conference,times}

\usepackage{amsmath,amsfonts,bm}

\def\eqref#1{equation~\ref{#1}}

\def\ceil#1{\lceil #1 \rceil}
\def\floor#1{\lfloor #1 \rfloor}
\def\1{\bm{1}}

\DeclareMathAlphabet{\mathsfit}{\encodingdefault}{\sfdefault}{m}{sl}
\SetMathAlphabet{\mathsfit}{bold}{\encodingdefault}{\sfdefault}{bx}{n}

\newcommand{\E}{\mathbb{E}}

\newcommand{\R}{\mathbb{R}}

\usepackage{hyperref,cleveref}
\usepackage{url}
\usepackage[utf8]{inputenc} 
\usepackage[T1]{fontenc}    
\usepackage{hyperref}       
\usepackage{url}            
\usepackage{booktabs}       
\usepackage{amsfonts}       
\usepackage{nicefrac}       
\usepackage{microtype}      
\usepackage{xcolor}         

\usepackage{amsthm, amsmath, amssymb}
\usepackage{amsfonts}
\usepackage{stackengine}
\usepackage{bbm}
\usepackage{bm}
\usepackage{graphicx}
\usepackage{enumitem}

\def\delequal{\mathrel{\ensurestackMath{\stackon[1pt]{=}{\scriptscriptstyle\Delta}}}}
\def\distequal{\mathrel{\ensurestackMath{\stackon[1pt]{=}{\scriptstyle\mathcal{D}}}}}

\newtheorem{theorem}{Theorem}

\newtheorem{corollary}[theorem]{Corollary}

\newtheorem{definition}{Definition}

\newtheorem{assumption}{Assumption}

\newtheorem{lemma}[theorem]{Lemma}
\newtheorem{proposition}[theorem]{Proposition}
\usepackage{mathtools}
\renewcommand{\P}{\mathbb{P}}
\newcommand{\RV}{\mathcal{RV}}

\usepackage{multirow}

\usepackage{longtable}
\usepackage{array}
\def\delequal{\mathrel{\ensurestackMath{\stackon[1pt]{=}{\scriptscriptstyle\Delta}}}}
\def\distequal{\mathrel{\ensurestackMath{\stackon[1pt]{=}{\scriptstyle\mathcal{D}}}}}
\def\delequal{\mathrel{\ensurestackMath{\stackon[1pt]{=}{\scriptscriptstyle\Delta}}}}
\def\distequal{\mathrel{\ensurestackMath{\stackon[1pt]{=}{\scriptstyle d}}}}
\newcommand{\norm}[1]{\left\lVert#1\right\rVert}
\usepackage{mathtools}

\title{Eliminating Sharp Minima from SGD with Truncated Heavy-tailed Noise}

\author{ Xingyu Wang$^{1}$\ \ \ \ \ \ \ Sewoong Oh$^{2}$\ \ \ \ \ \ \ Chang-Han Rhee$^{1}$\\
${ }^{1}$Northwestern University,\ \ \ \ \ \ \ ${ }^{2}$University of Washington\\
\texttt{xingyuwang2017@u.northwestern.edu}
}

\iclrfinalcopy 
\begin{document}

\maketitle

\begin{abstract}
The empirical success of deep learning is often attributed to SGD’s mysterious ability to avoid sharp local minima in the loss landscape, as sharp minima are  known to lead to poor generalization. 
Recently, evidence of {\em heavy-tailed} gradient noise was reported in many deep learning tasks, and it was shown in \citep{csimcsekli2019heavy, simsekli2019tail} that SGD can {\em escape} sharp local minima under the presence of such heavy-tailed gradient noise, providing a partial solution to the mystery. 
In this work, we analyze a popular variant of SGD where gradients are truncated above a fixed threshold. We show that it achieves a stronger notion of avoiding sharp minima: it can effectively {\em eliminate} sharp local minima entirely from its training trajectory. 
We characterize the dynamics of truncated SGD driven by heavy-tailed noises. First, we show that the truncation threshold and width of the attraction field dictate the order of the first exit time from the associated local minimum. 
Moreover, when the objective function satisfies appropriate structural conditions, we prove that as the learning rate decreases the dynamics of the heavy-tailed truncated SGD closely resemble those of a continuous-time Markov chain that never visits any sharp minima. 
Real data experiments on deep learning confirm our theoretical prediction that heavy-tailed SGD with gradient clipping finds a "flatter" local minima and achieves better generalization.  
\end{abstract}

\section{Introduction}
\begin{figure}[ht]
\vskip 0.2in

\begin{center}
\begin{tabular}{ c c c}
\includegraphics[width=0.28\textwidth]{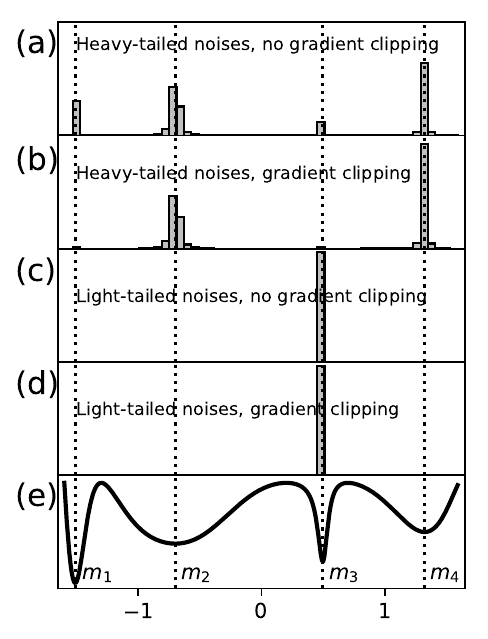}  & 
\includegraphics[width=0.335\textwidth]{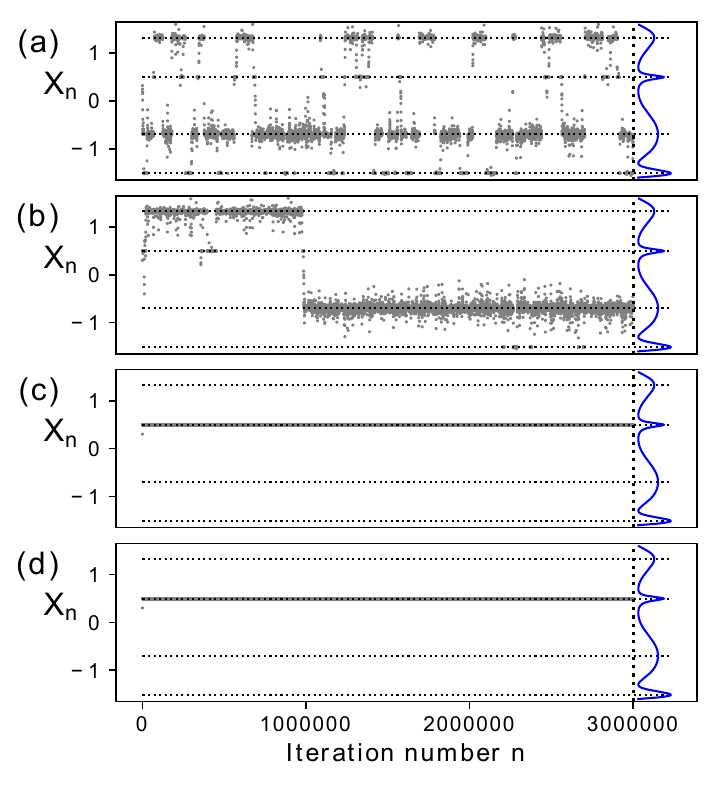} &
\includegraphics[width=0.30\textwidth]{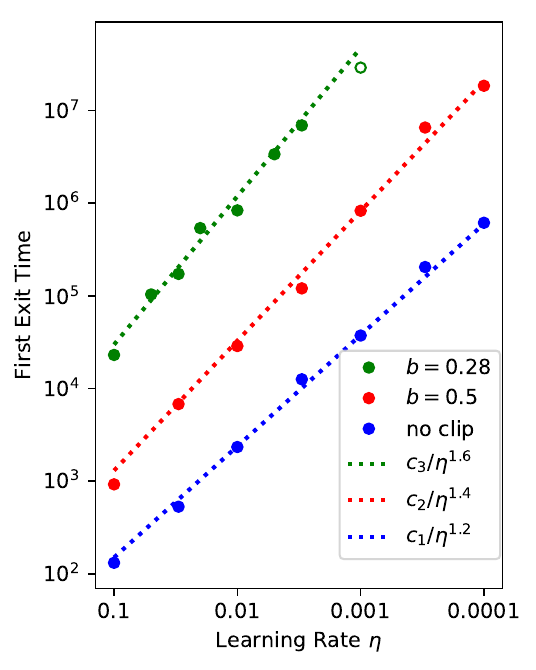}  \\
\end{tabular}
\caption{ \textbf{(Left)} Histograms of the locations visited by SGD. With truncated heavy-tailed noises, SGD hardly ever visits the two sharp minima $m_1$ and $m_3$. 
The objective function $f$ is plotted at the bottom, and dashed lines are added as references for the locations of local minima. \textbf{(Middle)} Typical trajectories of SGD in different cases: (a) Heavy-tailed noises, no gradient clipping; (b) Heavy-tailed noises, gradient clipping at $b = 0.5$; (c) Light-tailed noises, no gradient clipping; (d) Light-tailed noises, gradient clipping at $b = 0.5$. 
The objective function $f$ is plotted at the right of each figure, and dashed lines are added as references for locations of the local minima. \textbf{(Right)} First Exit Time from $\Omega_2 = (-1.3,0.2)$. Each dot represents the average of 20 samples of first exit time. Each dahsed line shows a polynomial function $c_i/\eta^\beta$ where $\beta$ is predicted by Theorem \ref{theorem first exit time} and $c_i$ is chosen to fit the dots. The non-solid green dot indicates that for some of the 20 samples of the termination threshold $5\times 10^7$ was reached, and hence, it is an underestimation. Results in \textbf{(Left)} and \textbf{(Middle)} are obtained under learning rate $\eta = 0.001$ and initial condition $X_0 = 0.3$.}
\label{fig histograms}
\end{center}
\vskip -0.2in
\end{figure}

Stochastic gradient descent (SGD) and its variants have seen unprecedented empirical successes in training deep neural networks. 
The training of deep neural networks is typically posed as a non-convex optimization problem, and even without explicit regularization the solutions obtained by SGD often perform surprisingly well on test data. 
Such an unexpected generalization performance of SGD in deep neural networks are often attributed to SGD’s ability to avoid \emph{sharp local minima}\footnote{We use the terminology sharpness in a broad sense; we refer to Appendix \ref{sec:appendix-sharpness} for a detailed discussion.}
in the loss landscape, which tends to lead to poor generalization \citep{hochreiter1997flat,keskar2016large,li2018visualizing, jiang2019fantastic}; see Appendix~\ref{sec:appendix-sharpness-generalization} for more details.
Despite significant efforts to explain such phenomena theoretically, understanding how SGD manages to avoid sharp local minima and end up with flat local minima within a realistic training time still remains as a central mystery of deep learning.
\footnote{To see a detailed discussion on existing results on selection of local minima from the stability perspective and the novelty of our analysis, see Appendix \ref{sec:appendix-stability-literature}.}
Recently, the heavy-tailed dynamics of SGD received significant attention, and it was suggested that 
the heavy tails in the stochastic gradients may be a key ingredient that facilitates SGD's escape from sharp local minima: 
for example, \cite{csimcsekli2019heavy,simsekli2019tail} report the empirical evidence of heavy-tails in stochastic gradient noise in popular deep learning architectures (see also \citep{hodgkinson2020multiplicative,srinivasan2021efficient,garg2021on}) and show that SGD can escape sharp local minima in polynomial time under the presence of the heavy-tailed gradient noise. 
More specifically, they view heavy-tailed SGDs as discrete approximations of L\'evy driven Langevin equations and argue that the amount of time SGD trajectory spends in each local minimum is proportional to the width of the associated minimum according to the metastability theory \citep{pavlyukevich2007cooling,imkeller2010first,imkeller2010hierarchy} for such heavy-tailed processes. 

In this paper, we study the global dynamics and long-run behavior of heavy-tailed SGD and its practical variant in depth. 
While in full generality the structure of gradient noises in SGD is state-dependent, in this work we focus on the role of noise magnitude and analyze the setting where each SGD update is perturbed by iid heavy-tailed noise.
In particular, we consider an adaptive version of SGD, where the stochastic gradient is truncated above a fixed threshold. Such truncation scheme is often called \emph{gradient clipping} and employed as default in various contexts \citep{Engstrom2020Implementation, merity2018regularizing, graves2013generating, pascanu2013difficulty,Zhang2020Why,NEURIPS2020_abd1c782}. 
We uncover a rich mathematical structure in the global dynamics of SGD under this scheme and prove that the asymptotic behavior of such SGD is fundamentally different from that of the pure form of SGD: in particular, 
under a suitable structural condition on the geometry of the loss landscape, gradient clipping \emph{completely eliminates sharp minima from the trajectory of SGDs}.
This  provides a critical  insight  into  how heavy-tailed dynamics of SGD can be utilized to find a local minimum that generalizes better.

Figure~\ref{fig histograms} (Left, Middle) clearly illustrates these points with the histograms of the sample trajectories of SGDs. 
Note first that SGDs with light-tailed gradient noise---(c) and (d) of Figure~\ref{fig histograms} (Left, Middle)---never manages to escape a (sharp) minimum regardless of gradient clipping.
In contrast, SGDs with heavy-tailed gradient noise---(a) and (b) of Figure~\ref{fig histograms} (Left, Middle)---easily escapes from local minima.
Moreover, there is a clear difference between SGDs with gradient clipping and without gradient clipping.
In (a) of Figure~\ref{fig histograms}~(Left), SGD without gradient clipping spends a significant amount of time at each of all four local minima ($\{m_1, m_2, m_3, m_4\}$), although it spends more time around the wide ones ($\{m_2, m_4\}$) than the sharp ones ($\{m_1, m_3\}$). 
On the other hand, in (b) of Figure~\ref{fig histograms}~(Left), SGD with gradient clipping not only escapes from local minima but also avoids sharp minima ($\{m_1, m_3\}$) almost completely. 
This means that after we run SGD for long enough (more precisely, the required run length $t/\eta^\beta$ is of polynomial order; see Theorem \ref{thm main paper eliminiate sharp minima from trajectory}), it is almost guaranteed that it won't be at a sharp minimum, effectively eliminating sharp minima from its training trajectories.

We also propose a novel computational strategy that takes advantage of our newly discovered global dynamics of the heavy-tailed SGD.
While the evidence of heavy tails were reported in many deep learning tasks \citep{simsekli2019tail,csimcsekli2019heavy, garg2021on,gurbuzbalaban2020heavy,hodgkinson2020multiplicative,nguyen2019non,mahoney2019traditional,srinivasan2021efficient,Zhang2020Why}, there seem to be plenty of deep learning contexts where the stochastic gradient noises are light-tailed \citep{panigrahi2019non} as well.
\footnote{For a detailed comparison to existing works on heavy-tailed phenomena is SGD, see Appendix \ref{sec:appendix heavy tailed phenomena in SGD}.}
Guided by our new theory, we propose an algorithm that injects heavy-tails to SGD by inflating the tail distribution of the gradient noise and facilitating the discovery of a local minimum that generalizes better. 
Our experiments with image classification tasks, reported in Tables~\ref{table ablation study} and \ref{table data augmentation training}, illustrate that the tail-inflation strategy we propose here can indeed improve the generalization performance of the SGD as predicted by our theory. 


The rest of the paper is organized as follows. Section~\ref{sec:theory} formulates the problem setting and characterizes the global dynamics of the SGD driven by heavy-tailed noises. Section~\ref{sec:numerical-experiments} presents numerical experiments that confirm our theory.
Section~\ref{sec:deeplearning} proposes a new algorithm that artificially injects heavy tailed gradient noise in actual deep learning tasks and demonstrate the improved performance. 

{\bf Technical Contributions:} 1) We rigorously characterize the global behavior of the heavy-tailed SGD with gradient clipping. 
We first focus on the case where the loss function is in $\R^1$ with some simplifying assumptions on its geometry. Even with such assumptions, our theorem involves substantial technical challenges since the traditional tools for analyzing SGD fail in our context due to the adaptive nature of its dynamics and non-Gaussian distributional assumptions. 
For example, while the unclipped pure SGD can be analyzed by partitioning its trajectory at arrival times of large noises (as in \cite{pavlyukevich2005metastable} and \cite{imkeller2010first}), such an approach falls short in our context.
Instead,
we developed a set of delicate arguments for dealing with SGD's (near) regeneration structure and the return times to the local minima, as well as controlling the probability of atypical scenarios that would not arise in the unclipped case.
Moreover, as evidenced by our $\R^d$ results in Appendix~\ref{sec: appendix Rd result},
the approach developed here is critical in extending the analysis to general loss landscapes.


2) We propose a novel computational strategy for improving the generalization performance of SGD by carefully injecting heavy-tailed noise. We test the proposed algorithm with deep learning tasks and demonstrate its superiority with an ablation study. This also suggests that the key phenomenon we characterize in our theory--- elimination of sharp local minima---manifests in real-world tasks.

\section{Theoretical results}\label{sec:theory}

This section characterizes the global dynamics of SGD with gradient clipping when applied to a non-convex objective function $f$. 
In Section~\ref{subsec:first-exit-time}~and~\ref{subset:scaling-limit}, we make the following assumptions for the sake of the simplicity of analysis. 
However, as our multidimensional result in Section~\ref{subsec:multidimensional-fet} and the experiments in Section~\ref{sec:numerical-experiments}~and~\ref{sec:deeplearning} suggest, the gist of the phenomena we analyze---elimination of sharp local minima---persists in general contexts where the domain of $f$ is multi-dimensional, and the stationary points are not necessarily strict local optima separated from one another. 
\begin{assumption}\label{assumption: function f in R 1}
Let $f:\R \to \R$ be a $\mathcal{C}^2$ function. There exist a positive real $L > 0$, a positive integer $n_{\text{min}}$ and an ordered sequence of real numbers  $m_1, s_1 , m_2, s_2, \cdots, s_{n_{\text{min}} - 1}, m_{n_{\text{min}}} $ such that (1) $-L< m_1 < s_1 < m_2 <s_2< \cdots < s_{n_{\text{min}} - 1}< m_{n_{\text{min}}} < L$; (2) $f^\prime(x) = 0$ iff $x \in \{m_1,s_1,\cdots,s_{n_{\text{min}} - 1} , m_{n_{\text{min}}}\}$; (3) For any $x \in \{m_1,m_2,\cdots,m_{n_{\text{min}}} \}$, $f^{\prime\prime}(x) > 0$; (4) For any $x \in \{s_1,s_2,\cdots,s_{n_{\text{min}} - 1}\}$, $f^{\prime\prime}(x) < 0$.
\end{assumption}

As illustrated in Figure \ref{fig typical transition graph} (Left), the assumption above requires that $f$ has finitely many local minima (to be specific, the count is $n_{\text{min}}$), all of which contained in some compact domain $[-L,L]$. Moreover, the points $s_1,\cdots,s_{ n_\text{min} - 1 }$ naturally partition the entire real line into different regions $\Omega_i = (s_{i-1},s_i)$ (here we adopt the convention that $s_0 = -\infty, s_{n_\text{min}} = +\infty$). We call each region $\Omega_i$ the \textbf{attraction field} of the local minimum $m_i$, as the gradient flow in $\Omega_i$ always points to $m_i$.
%
%
%

Throughout the optimization procedure, given any location $x \in \mathbb{R}$ we assume that we  have access to the noisy estimator $ f^\prime(x) - Z_n $ of the true gradient $f'(x)$, and $f'(x)$ itself is difficult to evaluate. 
Specifically, in this work we are interested in the case where the iid sequence of noises $(Z_n)_{n \geq 1}$ are heavy-tailed. Typically, the heavy-tailed phenomena are captured by the concept of regular variation:
for a measurable function $\phi:\mathbb{R}_+ \mapsto \mathbb{R}_+$, we say that $\phi$ is regularly varying at $+\infty$ with index $\beta$ (denoted as $\phi \in \RV_\beta$) if $\lim_{x \rightarrow \infty}\phi(tx)/\phi(x) = t^\beta$ for all $t>0$. 
For details on the definition and properties of regularly varying functions, see, for example, chapter 2 of \cite{resnick2007heavy}. 
In this paper, we work with the following distributional assumption on the gradient noise. Let
\begin{align*}
    H_+(x) & \triangleq \mathbb{P}(Z_1 > x),\ \ \ H_-(x) \triangleq \mathbb{P}(Z_1 < -x), \ \ \ H(x)  \triangleq H_+(x) + H_-(x) = \mathbb{P}(|Z_1| > x).
\end{align*}

\begin{assumption}\label{assumption gradient noise heavy-tailed}
$\mathbb{E}Z_1 = 0$. Furthermore, there exists some $\alpha \in (1,\infty)$ such that function $H(x)$ is regularly varying (at $+\infty$) with index $-\alpha$. Besides, regarding the positive and negative tail for distribution of the noises, we have
    \begin{align*}
        \lim_{x \rightarrow \infty}\frac{ H_+(x) }{H(x)} = p_+,\ \lim_{x \rightarrow \infty}\frac{ H_-(x) }{H(x)} = p_- = 1 - p_+
    \end{align*}
    where $p_+$ and $p_-$ are constants in interval $(0,1)$. 
\end{assumption}
Roughly speaking, Assumption \ref{assumption gradient noise heavy-tailed} means that the shape of the tail for the distribution of noises $Z_n$ resembles a polynomial function $x^{-\alpha}$, which is much heavier than the exponential tail of Gaussian distributions. Therefore, 
large values of $Z_n$ are much more likely to be observed under Assumption \ref{assumption gradient noise heavy-tailed} compared to the typical Gaussian assumption.
The index $\alpha$ of regular variation encodes the heaviness of the tail---the smaller the heavier---and we are assuming that the left and right tails share the same index $\alpha$. 
The purpose of this simplifying assumption is clarity of presentation, 
but our $\R^d$ results in Appendix~\ref{sec: appendix Rd result} relax such a condition and allow different regular variation indices in different directions.

Our work concerns a popular variant of SGD where the stochastic gradient is truncated. Specifically, when updating the SGD iterates with a learning rate $\eta>0$, rather than using the original noisy gradient descent step $\eta(f^\prime(X_n) - Z_n)$, we will truncate it at a threshold $b>0$ and use $\varphi_b\big(\eta(f^\prime(X_n) - Z_n)\big)$ instead. Here the truncation operator $\varphi_\cdot(\cdot)$ is defined as 
\begin{align}
    \varphi_c(w) \triangleq w \cdot \min \{1, c/|w|\}\ \ \ \forall w \in \mathbb{R}, c>0. 
\end{align}
Besides truncating the stochastic gradient, we also project the SGD into $[-L, L]$ at each iteration; recall that $L$ is the constant in Assumption \ref{assumption: function f in R 1}.
That is, the main object of our study is the stochastic process $\{X^\eta_j\}_{j \geq 0}$ driven by the following recursion
\begin{align}\label{def SGD step 2}
     X^{\eta}_{j} & \delequal{} \varphi_L\Big( X^{\eta}_{j-1}- \varphi_b\big(\eta (f^\prime(X^{\eta}_{j-1}) - Z_j)\big) \Big).
\end{align}
The projection $\varphi_L$ and truncation $\varphi_b$ here are common practices in many learning tasks for the purpose of ensuring that the SGD does not explode and drift to infinity. Besides, 
the projection also allows us to drop the sophisticated assumptions on the tail behaviors of $f$ that are commonly seen in previous works (see, for instance, the dissipativity conditions in \cite{nguyen2019non}).
For technical reasons, we make the following assumption about the truncation threshold $b>0$. 
Note that this assumption is a very mild one, as it is obviously satisfied by (Lebesgue) almost every $b>0$.
\begin{assumption}\label{assumption clipping threshold b}
For each $i = 1,2,\cdots,n_{\text{min}}$, $ \min\{ |s_i - m_i|, |s_{i-1} - m_i| \}/b $ is \emph{not} an integer.
\end{assumption}


\subsection{First exit times}\label{subsec:first-exit-time}

Denote the SGD's first exit time from the attraction field $\Omega_i$ with
$\sigma_i(\eta) \delequal{} \min\{ n \geq 0: X^\eta_n \notin \Omega_i \}.$
In this section, we prove that $\sigma_i(\eta,x)$ converges to an exponential distribution when scaled properly. 
To characterize such a scaling, we first introduce a few concepts. 
For each attraction field $\Omega_i$, define (note that $\ceil{x} = \min\{ n \in \mathbb{Z}: n \geq x \}$, $\floor{x} = \max\{n \in \mathbb{Z}: n \leq x\}$ )
\begin{align}
     r_i \delequal{}\min\{ |m_i - s_{i-1}|, |s_i - m_i| \}, \ \ \ \ l^*_i & \delequal{} \ceil{ r_i/b }. \label{def main paper minimum jump l star i}
\end{align}
Note that $l^*_i$'s in fact depend on the the value of gradient clipping threshold $b$ even though this dependency is not highlighted by the notation. 
Here $r_i$ can be interpreted as the radius or the effective \emph{width} of the attraction field, and $l^*_i$ is the \emph{minimum number of jumps} required to escape $\Omega_i$ when starting from $m_i$. 
Indeed, the gradient clipping threshold $b$ dictates that no single SGD step can travel more than $b$, and to exit $\Omega_i$ when starting from $m_i$ 
we can see that at least $\ceil{r_i/b}$ steps are required. 
We can interpret $l^*_i$ as the minimum effort required to exit $\Omega_i$.
In this sense, 
$l^*_i$ is an indicator of the width of the attraction field $\Omega_i$. 
Theorem~\ref{theorem first exit time} states that $l^*_i$ dictates the order of magnitude of the first exit time as well as where the iterates $X^\eta_n$ land on at the first exit time. 
For each $\Omega_i$, define a scaling function $\lambda_i(\eta) \delequal{} H(1/\eta)\left( (1/\eta) H(1/\eta) \right)^{l^*_i - 1}.$
To stress the initial condition, we write $\mathbb{P}_x$ for the probability law when conditioning on $X^\eta_0 = x$, or simply write $X_n^\eta(x)$.

\begin{theorem} \label{theorem first exit time}
Under Assumptions \ref{assumption: function f in R 1}-\ref{assumption clipping threshold b},
there exist constants $q_i > 0\ \forall i\in \{1,2,\cdots,n_\text{min}\}$ and $q_{i,j} \geq 0\ \forall j \in \{1,2,\cdots,n_\text{min}\} \setminus\{ i\}$ such that
\begin{enumerate}[label = \arabic*)]
    \item[(i)] Suppose that $x\in \Omega_k$ for some $k\in \{1,2,\cdots,n_\text{min}\}$. Under $\mathbb{P}_x$, the scaled first exit time $q_k\lambda_k(\eta)\sigma_k(\eta)$ converges in distribution to $Exp(1)$ as $\eta \downarrow 0$.
    \item[(ii)] For $k,l \in \{ 1,2,\cdots,n_\text{min}\}$ such that $k \neq l$, we have $ \lim_{\eta \to 0} \mathbb{P}_x( X^\eta_{ \sigma_k(\eta) } \in \Omega_l ) = q_{k,l}/q_{k}.$
\end{enumerate}
\end{theorem}

The proof and discussion are deferred to Appendix \ref{section proof thm 1}. 
The constants $q_i,q_{i,j}$ are explicitly identified in terms of the gradient flows perturbed by Pareto jumps in Section C of Appendix. 
We note here that Theorem \ref{theorem first exit time} implies (i) for $X^\eta_n$ to escape the current attraction field, say $\Omega_i$, it takes $O\big(1/\lambda_i(\eta)\big)$ time, and (ii) the destination is most likely to be reachable within $l^*_i$ jumps from $m_i$.

\subsection{Elimination of Small Attraction Fields} \label{subset:scaling-limit}

\begin{figure*}
\centering
\includegraphics[width=1\textwidth]{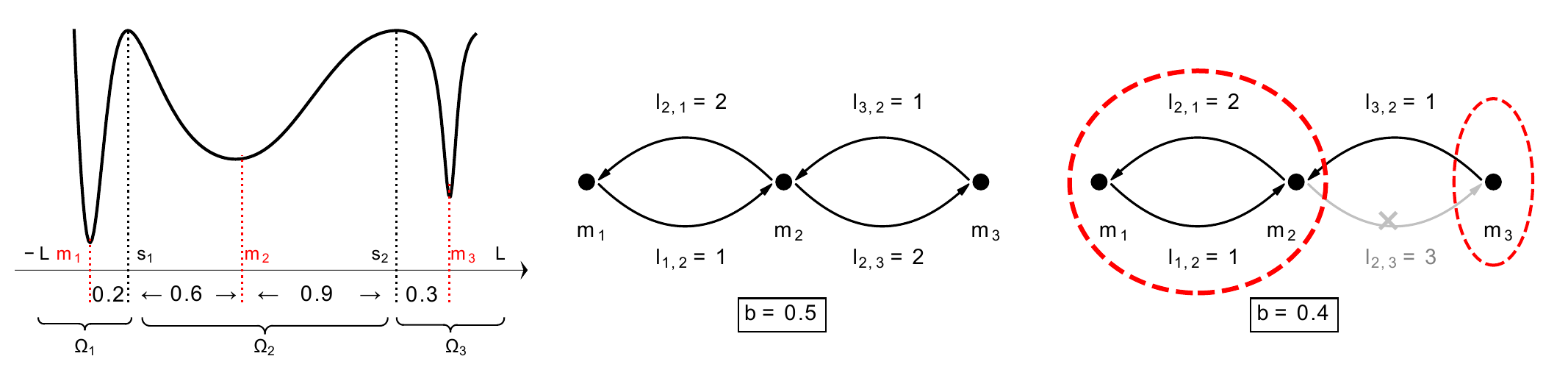}
\caption{\label{fig typical transition graph} Typical transition graphs $\mathcal{G}$ under different gradient clipping thresholds $b$. (Left) The function $f$ illustrated here has 3 attraction fields. For the second one $\Omega_2 = (s_1,s_2)$, we have $s_2 - m_2 = 0.9, m_2 - s_1 = 0.6$. (Middle) The typical transition graph induced by $b = 0.5$. The entire graph $\mathcal{G}$ is irreducible since all nodes communicate with each other. 
(Right) The typical transition graph induced by $b = 0.4$. When $b = 0.4$, since $0.6 < 2b$ and $0.9 > 2b$, the SGD can only exit $\Omega_2$ from the left with only 2 jumps if started from $m_2$. Therefore, on the graph $\mathcal{G}$ there are two communication classes: $G_1 = \{m_1,m_2\}, G_2 = \{m_3\}$; $G_1$ is absorbing while $G_2$ is transient. }
\end{figure*} 

Under proper structural assumptions on the geometry of $f$, the sharp minima of $f$ can be effectively eliminated from the trajectory of heavy-tailed SGD, facilitating the discovery of flat minima.
This is somewhat surprising given that gradient clipping mechanism makes the SGD iterates move \textit{slower}.
The intuition behind this is that for narrow basins, applying gradient clipping has virtually no effect on the order of exit time; whereas for a wide basin that requires multiple jumps to escape under the clipping scheme, the clipping of gradients significantly slows down the escape and makes SGD stay longer. In other words, gradient clipping only makes SGDs stay longer in the wider (better) basins.



Now, we introduce a few new concepts. 
Similar to the the minimum number of jumps $l^*_i$ defined in (\ref{def main paper minimum jump l star i}), we  define the following  as the  \emph{minimum number of jumps to reach $\Omega_j$ from $m_i$} for any $j \neq i$:
\begin{align}
l_{i,j} = \begin{cases} 
\ceil{ (s_{j-1} - m_i)/b  } & \text{if }\ j > i,\\
\ceil{ (m_i - s_j)/b } & \text{if }\ j < i.
\end{cases} \label{def jump number l i j}
\end{align}
Recall that Theorem \ref{theorem first exit time} dictates that $X^\eta_n$ is most likely to move out of the current attraction field, say $\Omega_i$, to somewhere else after $O\big(1/\lambda_i(\eta)\big)$ time steps, and the destination is most likely to be reachable within $l^*_i$ jumps from $m_i$. 
Therefore, the transitions from $\Omega_i$ to $\Omega_j$ can be considered typical if $\Omega_j$ can be reached from $m_i$ with $l^*_i$ jumps---that is, $l_{i,j} = l^*_i$. Now we define the following directed graph that only includes these typical transitions.
\begin{definition}[Typical Transition Graph] \label{def main paper typical transition graph}
Given a function $f$ satisfying Assumption \ref{assumption: function f in R 1} and gradient clipping threshold $b>0$ satisfying Assumption \ref{assumption clipping threshold b}, a directed graph $\mathcal{G} = (V,E)$ is the corresponding \emph{typical transition graph} if (1) $V = \{m_1,\cdots,m_{ n_\text{min} }\}$; (2) An edge $(m_i\rightarrow m_j)$ is in $E$ iff $l_{i,j} = l^*_i$.
\end{definition}

Naturally, the typical transition graph $\mathcal{G}$ can be decomposed into different communication classes $G_1,\cdots,G_K$ that are mutually exclusive by considering the equivalence relation associated with the existence of the (two-way) paths between $i$ and $j$. More specifically, for $i \neq j$, we say that  $i$ and $j$ communicate if and only if there exists a path $(m_i,m_{k_1},\cdots,m_{k_n},m_j)$ as well as a path $(m_j,m_{k^\prime_1},\cdots,m_{k^\prime_{n^\prime}},m_i)$ in $\mathcal{G}$; in other words, by travelling through edges on $\mathcal{G}$, $m_i$ can be reached from $m_j$ and $m_j$ can be reached from $m_i$.
 
We say that a communication class $G$ is \emph{absorbing} if there does not exist any edge $(m_i \rightarrow m_j) \in E$ such that $m_i \in G$ and $m_j \notin G$. 
Otherwise, we say that $G$ is \emph{transient}. 
In the case that all $m_i$'s communicate with each other on graph $\mathcal{G}$, we say $\mathcal{G}$ is \emph{irreducible}. 
See Figure \ref{fig typical transition graph} (Middle) for the illustration of an irreducible case. 
When $\mathcal{G}$ is irreducible, we define the set of \emph{largest} attraction fields  $M^\text{large}\delequal{}\{ m_i:\ i = 1,2,\cdots,n_\text{min},\ l^*_i = l^\text{large} \}$ where $l^\text{large} = \max_j l^*_j$;  
recall that $l^*_i$ characterizes the width of $\Omega_i$.
Define the longest time scale 
$\lambda^\text{large}(\eta) = H(1/\eta) ( (H(1/\eta)/\eta) )^{ l^\text{large} - 1 }.$
Note that this corresponds exactly to the order of the first exit time of the largest attraction fields; see Theorem~\ref{theorem first exit time}. 
The following theorem is the main result of this paper. 

\begin{theorem} \label{thm main paper eliminiate sharp minima from trajectory}
Let Assumptions \ref{assumption: function f in R 1}-\ref{assumption clipping threshold b} hold and assume that the graph $\mathcal{G}$ is \textbf{irreducible}. 
For any $t > 0$, $\beta >1+ (\alpha - 1)l^\text{large}$ and $x \in [-L,L]$,
\begin{align}
    \frac{1}{\floor{t / \eta^\beta }}\int_0^{\floor{t / \eta^\beta } }\mathbbm{1}\Big\{ X^\eta_{ \floor{u} }(x) \in \bigcup_{j: m_j \notin M^\text{large}}\Omega_j \Big\}du \to 0
    \label{def main paper RV proportion of time spent at small basin}
\end{align}
in probability as $\eta\to0$.
\end{theorem}

The proof is deferred to Appendix \ref{section C appendix}. Here we briefly discuss the implication of the result.
Suppose that we terminate the training after a reasonably long time, say, $\floor{t / \eta^\beta }$ iterations. 
Then the random variable that converges to zero in \cref{def main paper RV proportion of time spent at small basin} is exactly the proportion of time that $X^\eta_n$ spent in the attraction fields that are not wide. 
Therefore, by truncating the gradient noise of the heavy-tailed SGD, we can effectively eliminate small attraction fields from its training trajectory.
In other words, it is almost guaranteed that SGD is in one of the widest attraction fields after sufficiently long  training.
Meanwhile, despite the asymptotic nature of Theorem \ref{thm main paper eliminiate sharp minima from trajectory},
it has been confirmed in our simulation and deep learning experiments that the elimination effect can be observed under typical choices of $\eta$.

Theorem \ref{thm main paper eliminiate sharp minima from trajectory} is merely a manifestation of the global dynamics of heavy-tailed SGD. The main messages of the next theorem are: 
(a) under clipped heavy-tailed noises, the dynamics of $X^\eta_n$ for small $\eta$ closely resemble that of a continuous-time Markov chain; (b) this chain \emph{only} visits local minima of the largest attraction fields of $f$, thus minima in small attraction fields are completely avoided.

\begin{theorem} \label{corollary irreducible case}
Let $x \in \Omega_i$ for some $i = 1,2,\cdots,n_{min}$. If Assumptions \ref{assumption: function f in R 1}-\ref{assumption clipping threshold b} hold and $\mathcal{G}$ is irreducible, then there exist a continuous-time Markov chain $Y$ on $M^\text{large}$ and  a random mapping $\pi$ satisfying 
\begin{itemize}
    \item $\pi(m) \equiv m$ if $m \in M^{\text{large}}$;
    \item $\pi(m)$ is a random variable that only takes value in $M^{\text{large}}$ if $m \notin M^{\text{large}}$.
\end{itemize}
such that the scaled process $\{X^\eta_{\floor{ t/\lambda^\text{large}(\eta) }}(x):\ t \geq 0\}$ converges to process $\{Y_t(\pi(m_i)): t \geq 0\}$ 
in the sense of finite-dimensional distributions: for any positive integer $k$ and any $0<t_1<t_2<\cdots<t_k$, the random vector $\Big( X^\eta_{\floor{ t_1/\lambda^\text{large}(\eta) }}(x) , \cdots, X^\eta_{\floor{ t_k/\lambda^\text{large}(\eta) }}(x) \Big)$ converges in distribution to $\Big( Y_{t_1}(\pi(m_i)), \cdots, Y_{t_k}(\pi(m_i)) \Big)$ as $\eta \downarrow 0$.
\end{theorem}
In section D of Appendix, we detail the proof,
the exact parametrization of the generator matrix of process $Y$, and the distribution of random mapping $\pi(\cdot)$. 
Here we add some remarks. Intuitively speaking, this result tells us that,
regardless of where we initialize the SGD iterates, the dynamics of the clipped heavy-tailed SGD converge to a continuous-time Markov chain avoiding any local minima that is not in the largest attraction fields.
Second, under small learning rate $\eta>0$, if $X^\eta_n(x)$ is initialized at $x \in \Omega_i$ where $\Omega_i$ is NOT a largest attraction field, then SGD will quickly escape $\Omega_i$ and arrive at some $\Omega_j$ that is indeed a largest one---i.e., $m_j \in M^{\text{large}}$; 
such a transition is so quick that, under time scaling $\lambda^\text{large}(\eta)$, it is almost instantaneous as if $X^\eta_n(x)$ is actually initialized randomly at some of the largest attraction fields. This randomness is compressed in the random mapping $\pi$.
In Section \ref{sec:discussions}, we discuss how our characterization in Theorem \ref{corollary irreducible case} are more general and applicable in the machine learning context than metastability results cited in \citep{simsekli2019tail};
in Section \ref{section C appendix} we see that the regularization effect of truncated heavy-tailed noises described in Theorem \ref{corollary irreducible case} is of great generality and can still be observed locally when the irreducibility condition is removed. 

\subsection{$\mathbb{R}^d$ Extensions of the Theoretical Results}\label{subsec:multidimensional-fet}

We focused on the $\mathbb{R}^1$ case so far, for the clarity of the exposition.
In this section, we informally reiterate that the same  effect under truncated heavy-tailed noises persists in high-dimensions. 
Rigorous statements are provided in Appendix~\ref{sec: appendix Rd result}.
We consider a setting similar to those in \cite{imkeller2010first} and analyze the first exit time $\sigma(\eta,x) = \min\{k \geq 0: X^\eta_k(x) \notin \mathcal{G}\}$ from an open, bounded domain $\mathcal{G}$ with smooth boundary.
For some $f: \R^d \mapsto \R$,
the SGD iterates $X^\eta_{k}(x) = X^\eta_{k-1}(x) - \varphi_b\big(\nabla f(X^\eta_{k-1}(x)) + \eta Z_k\big)$ are subject to the standard $L_2$ norm clipping with threshold $b > 0$,
iid noises $Z_n$ with heavy tails that resemble $1/x^\alpha$ with $\alpha > 1$,
and initial condition $X^\eta_{0}(x) = x$.
Assume that the origin $\bm{0}$ is the only attractor in $\mathcal{G}$ for the ODE $\dot{\bm{x}}(t) = -\nabla f(\bm{x}(t))$. 
Also, let $l^*_\mathcal{G}$ be the minimum number of jumps required for the ODE $\bm{x}(t)$ to escape from $\mathcal{G}$ provided that $\bm{x}(0) = \bm{0}$ and the $L_2$ norm of all jumps are less than $b$.
The following informal version of Theorem \ref{thm main result Rd simplified 1}   states that, under the presence of heavy-tailed noises and gradient clipping,
the first exit time from a region in $\mathbb{R}^d$ is of order $O(1/\eta^{1 + (\alpha-1) l^*_\mathcal{G}})$.
Therefore, the order of first exit times from different regions in $\mathbb{R}^d$ are still dictated by the geometric characterization $l^*_\mathcal{G}$, i.e. the minimum number of jumps required for escape,
and those with largest $l^*_\mathcal{G}$ may dominate the SGD trajectory as $\eta \downarrow 0$.
Furthermore, we extend the analysis to the generalized case where the distribution of noises $Z_n$ exhibits strong preference of certain directions and has  different heavy-tailed indices $\alpha$ along different directions.
We give the details of the $\mathbb{R}^d$ results in Section~\ref{sec: appendix Rd result} and provide a proof in Section~\ref{sec: proof Rd result}.
\begin{theorem}[Informal]
Under certain regularity conditions, for Lebesgue almost every $b > 0$,  there exist $ q > 0$ and  $\lambda(\eta)$ that is regularly varying w.r.t $\eta$ with index $1 + (\alpha-1) l^*_\mathcal{G}$ such that 
$$\lambda(\eta)\sigma(\eta,x) \Rightarrow Exp(q)\ \text{as }\eta \downarrow 0\ \ \forall x \in \mathcal{G}.$$
\end{theorem}

\section{Simulation Experiments}\label{sec:numerical-experiments}



We empirically demonstrate that, ($a$) as indicated by Theorem \ref{theorem first exit time}, the minimum jump number defined in (\ref{def main paper minimum jump l star i}) accurately characterizes the first exit times of the SGDs with clipped heavy-tailed gradient noises; ($b$) sharp minima can be effectively eliminated from such SGD; and ($c$) these properties are exclusive to heavy-tails. Under light-tailed noises, SGDs are trapped in sharp minima for a long time. 
The test function $f \in \mathcal{C}^2(\mathbb{R})$ is the same as in Fig.~\ref{fig histograms} (Left,e).
$m_1$ and $m_3$ are sharp minima in narrow attraction fields, while $m_2$ and $m_4$ are flatter and located in larger attraction fields. 
Heavy-tailed noises have tail index $\alpha = 1.2$, and \emph{light-tailed} noises are $\mathcal{N}(0,1)$. See Appendix \ref{sec: experiment details appendix} for details. 

First, we compare the first exit time of heavy-tailed SGD (when initialized at -0.7) from $\Omega_2 = (-1.3,0.2)$ under 3 different clipping mechanism: (1) $b = 0.28$, where the minimum jump number required to escape is $l^* = 3$; (2) $b = 0.5$, where $l^* = 2$; (3) no gradient clipping, where $l^* = 1$ obviously. 
According to Theorem \ref{theorem first exit time}, the first exit times for the aforementioned 3 clipping mechanism are of order $(1/\eta)^{1.6},(1/\eta)^{1.4}$ and $(1/\eta)^{1.2}$ respectively.
These theoretical predictions are  accurate 
as demonstrated in Figure \ref{fig histograms} (Right). 
Next, we investigate the global dynamics of heavy-tailed SGD. 
We compared the clipped case (with $b = 0.5$) against the case without clipping. 
Figure \ref{fig histograms} (Left, a, b) show the histograms of the empirical distributions of SGD, and Figure \ref{fig histograms}(Middle, a,b) plots the SGD trajectories. 
 Without gradient clipping, $X_n$ still visits the two sharp minima $m_1,m_3$. Under gradient clipping, the time spent at $m_1,m_3$ is almost completely eliminated and is negligible compared to the time $X_n$ spent at $m_2,m_4$  in larger attraction fields. 
This matches the predictions of Theorems \ref{thm main paper eliminiate sharp minima from trajectory}-\ref{corollary irreducible case}: the elimination of sharp minima with truncated heavy-tailed noises.
We stress that the said properties are exclusive to heavy-tailed SGD. As shown in Figure \ref{fig histograms}(Left,c,d) and Figure \ref{fig histograms}(Middle, c,d), light-tailed SGD are easily trapped at sharp minima for extremely long time.

Figure~\ref{fig Rd simulation} illustrates the same phenomena in $\R^2$, 
where $f$  has several saddle points and infinitely many local minima---the local minima of $\Omega_2$ form a line segment, which is an uncountably infinite set. 
Under clipping threshold $b$, attraction fields $\Omega_1$ and $\Omega_2$ are the \emph{larger} ones since the escape from them requires at least two jumps.
This suggests that the theoretical results from  Section~\ref{sec:theory} hold under more general contexts than Assumptions~\ref{assumption: function f in R 1}-\ref{assumption clipping threshold b}. In the next section, we provide experimental evidence that suggests that  truncated heavy-tailed noise improves the generalization  of SGD in deep learning.

\begin{figure}[ht]
\vskip 0.2in
\begin{center}
\centerline{\includegraphics[width=1.0\textwidth]{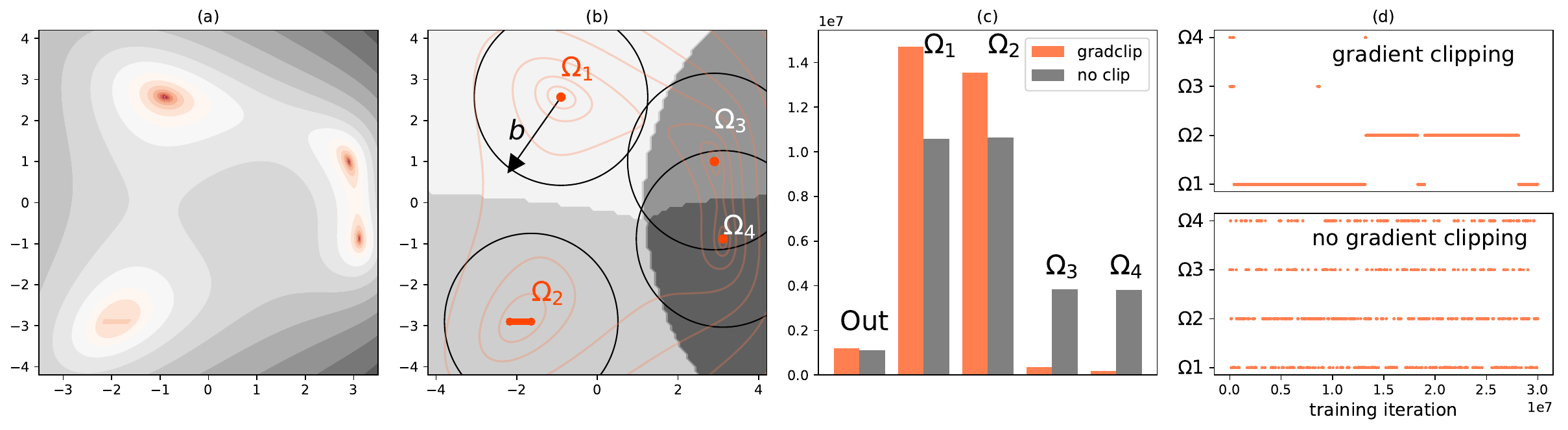} }
\caption{Experiment result of heavy-tailed SGD when optimizing the modified Himmelblau function. \textbf{(a)} Contour plot of the test function $f$. \textbf{(b)} Different shades of gray are used to indicate the area of the four different attraction fields $\Omega_1,\Omega_2,\Omega_3,\Omega_4$ of $f$. We say that a point belongs to an attraction field $\Omega_i$ if, when initializing at this point, the gradient descent iterates converge to the local minima in $\Omega_i$, which are indicated by the colored dots. The circles are added to imply whether the SGD iterates can escape from each $\Omega_i$ with one large jump or not under clipping threshold $b$. \textbf{(c)} The time heavy-tailed SGD spent at different region. An iterate $X_k$ is considered ``visiting'' $\Omega_i$ if its distance to the local minimizer of $\Omega_i$ is less than 0.5; otherwise we label $X_k$ as ``out''. \textbf{(d)} The transition trajectories of heavy-tailed SGD. The dots represent the last ``visited'' attraction field at each iteration.  }
\label{fig Rd simulation}
\end{center}
\vskip -0.2in
\end{figure}

\section{Heavy-tailed SGD in Deep Learning: An Ablation Study} 
\label{sec:deeplearning} 

In this section, we verify our theoretical results and demonstrate the effectiveness of clipped heavy-tailed noise in training deep neural networks. 
Contrary to the report in \citep{csimcsekli2019heavy}, heavy-tailed noise may not be ubiquitous in image classification tasks. For instance, the non-Gaussianity assumption on SGD noise is disputed by experiments in \citep{panigrahi2019non} for ResNet (see \citep{he2016deep}). For tasks considered in this section, the gradient noise is not heavy-tailed when models are randomly initialized (see  Appendix \ref{sec: experiment details appendix}). Motivated by the absence of heavy-tailed noises in image classification, 
 we \textit{make the SGD noise heavy-tailed}. 
 Let $\theta$ be the current model weight during training, $g_{SB}(\theta)$ be the typical small-batch gradient, and $g_{GD}(\theta)$ be the true (deterministic) gradient evaluated on the entire training dataset. Then by evaluating $g_{SB}(\theta) - g_{GD}(\theta)$ we obtain a sample of the gradient noise. Due to the prohibitive cost of evaluating $ g_{GD}(\theta)$,
 we instead use $g_{SB}(\theta) - g_{LB}(\theta)$ as its approximation where $g_{LB}$ denotes the gradient evaluated on a larger batch. 
 This  is justified by the unbiasedness in $\mathbb{E}_{LB}[g_{LB}(\theta)] = g_{GD}(\theta)$.  For some heavy-tailed random variable $Z$, by multiplying $Z$ with SGD noise, we obtain the following perturbed gradient: 
\begin{align}
    g_{heavy}(\theta) \;=\; g_{SB}(\theta) + Z\big( g_{SB*}(\theta) - g_{LB}(\theta) \big) \label{update our heavy tailed method}
\end{align}
where $SB$ and $SB*$ are two mini batches that may or may not be identical. We use the following update recursion under gradient clipping threshold $b$: $X^{\eta}_{k+1} = X^{\eta}_k - \varphi_b(\eta g_{heavy}(X^{\eta}_k))$ where $\varphi_b$ is the truncation operator. 
We consider two different implementations: in \emph{our method 1} (labeled as ``our 1'' in Table \ref{table ablation study}), $SB$ and $SB*$ are chosen independently, while in \emph{our method 2}  (labeled as ``our 2'' in Table \ref{table ablation study}), we use the same batch for $SB$ and $SB*$. In summary, by simply multiplying gradient noise with heavy-tailed random variables, we inject heavy-tailed noise into the optimization procedure.

We conduct an ablation study and benchmark the proposed clipped heavy-tailed methods against the following optimization methods. \emph{LB}: large-batch SGD with $X^{\eta}_{k+1} = X^{\eta}_k - \eta g_{LB}(X^{\eta}_k)$; \emph{SB}: small-batch SGD with $X^{\eta}_{k+1} = X^{\eta}_k - \eta g_{SB}(X^{\eta}_k)$; \emph{SB + Clip}: the update recursion is $X^{\eta}_{k+1} = X^{\eta}_k - \varphi_b(\eta g_{SB}(X^{\eta}_k))$; \emph{SB + Noise}: Our method 2 WITHOUT the gradient clipping mechanism, leading to the update recursion $X^{\eta}_{k+1} = X^{\eta}_k - \eta g_{heavy}(X^{\eta}_k)$. 

The experiment setting and choice of hyperparameters are   adapted from  \citep{pmlr-v97-zhu19e}. We consider three different tasks: (1) LeNet \citep{lecun1990handwritten} on corrupted FashionMNIST \citep{xiao2017/online}, (2) VGG11 \citep{simonyan2014very} on SVHN \citep{netzer2011reading}, (3) VGG11 on CIFAR10 \citep{krizhevsky2009learning} (see Appendix~\ref{sec: experiment details appendix} for details). 
Here we highlight a few points: First, within the same task, for all the 6 candidate methods will use the same $\eta$, batch size, training iteration, and (when needed) the same clipping threshold $b$ and heavy-tailed RV $Z$ for a fair comparison; the training duration is long enough so that \emph{LB} and \emph{SB} have attained 100\% training accuracy and close-to-0 training loss long before the end of training (the exception here is ``\emph{SB + Noise}'' method; see Appendix \ref{sec: experiment details appendix} for the details); Second, to facilitate convergence to local minima for \emph{our methods 1 and 2}, we remove heavy-tailed noise for last final 5,000 iterations and run \emph{LB} instead\footnote{The proposed method can be 
interpreted as a simplified version of GD + annealed heavy-tailed perturbation, where a detailed annealing is substituted by a two-phase training schedule. In the first \emph{exploration} phase the clipped heavy-tailed noises drive the iterates to explore the loss landscape and identify ``wide'' attraction fields. In the second \emph{exploitation} phase, removing the artificial perturbation accelerates  convergence to local minima.}.

\begin{table}
  \caption{Test accuracy and expected sharpness of different methods across different tasks. The reported numbers are the averages over 5 replications. For 95\% CI, see Appendix \ref{sec: experiment details appendix}. }
  \label{table ablation study}
  \centering
  \begin{tabular}{lllllll}
    \hline
    Test accuracy     & LB     & SB   & SB + Clip & SB + Noise & Our 1 & Our 2 \\
    \hline
    
    FashionMNIST, LeNet & 68.66\% & 69.20\% & 68.77\% & 64.43\% & 69.47\% & \textbf{70.06\%} \\ 
    SVHN, VGG11         & 82.87\% & 85.92\% & 85.95\% & 38.85\% & \textbf{88.42\%} & 88.37\% \\
    CIFAR10, VGG11      & 69.39\% & 74.42\% & 74.38\% & 40.50\% & 75.69\% & \textbf{75.87\%} \\

    \hline
    
    Expected Sharpness      & LB     & SB   & SB + Clip & SB + Noise & Our 1 & Our 2 \\
    \hline
    FashionMNIST, LeNet & 0.032 & 0.008 & 0.009 & 0.047 & 0.003 & \textbf{0.002} \\ 
    SVHN, VGG11         & 0.694 & 0.037 & 0.041 & 0.012 & \textbf{0.002} & 0.005 \\
    CIFAR10, VGG11      & 2.043 & 0.050 & 0.039 & 2.046 & \textbf{0.024} & 0.037 \\

    \hline
    
    
    
  \end{tabular}
\end{table}

Table \ref{table ablation study} shows that 
in all 3 tasks both \emph{our method 1} and \emph{our method 2} attain better test accuracy than the other candidate methods. Meanwhile, both methods exhibit similar test performance, implying that the implementation of the heavy-tailed method may not be a the deciding factor. 
We also report  the \emph{expected sharpness} metric $\E_{\nu \sim \mathcal{N}(\textbf{0},\delta^2 \textbf{I})}|L(\theta^* + \nu) - L(\theta^*)|$ used in \cite{pmlr-v97-zhu19e, NIPS2017_10ce03a1} where $\mathcal{N}(\textbf{0},\delta^2 \textbf{I})$ is a Gaussian distribution, $\theta^*$ is the trained model weight and $L$ is training loss. 
In our experiment, we use $\delta = 0.01$ and the expectation is evaluated by averaging over 100 samples. 
We conduct 5 replications for each experiment scenario and report the averaged performance in Table \ref{table ablation study}. 
Smaller sharpness of our methods 1 and 2  confirms that they encourage  minimizers with  a ``flatter'' geometry,  
thus attaining better test performances.

The ablation study in Table \ref{table ablation study}  shows that both heavy-tailed noise and  gradient clipping are necessary to find a flat minima and hence achieve better generalization, which is predicted by our analyses. 
\emph{SB} and \emph{SB + Clip} achieve similar inferior performances, confirming that  clipping does not help when noise is  light-tailed.  
\emph{SB + Noise} injects heavy-tailed noise without gradient clipping, which achieves an  inferior performance. 
This poor performance---even after  extensive parameter tuning and engineering (see Appendix \ref{sec: experiment details appendix} for more details)---demonstrates the difficulty on the optimization front, especially when heavy-tailed noise is present yet little effort is put into controlling the highly volatile gradient noises. 
This is aligned with the observations in  \cite{Zhang2020Why,NEURIPS2020_abd1c782} where adaptive gradient clipping methods are proposed to improve convergence of SGD in the presence of heavy-tailed noises. This confirms that 
gradient clipping is crucial for heavy-tailed SGD.

Lastly, Table \ref{table data augmentation training}
shows that even in the more sophisticated settings with training techniques such as data augmentations and scheduled learning rates,
truncated heavy-tailed SGD still manages to consistently find solutions with better test performance.
For experiment details, see Appendix \ref{sec: experiment details appendix}.
In Table \ref{table appendix data augmentation training} we also report the sharpness of the obtained solutions.
\begin{table}
  \caption{Our method's gain on test accuracy persists even when applied with techniques such as data augmentation and scheduled learning rates. For 95\% CI, see Appendix \ref{sec: experiment details appendix}.}
  \label{table data augmentation training}
  \centering
  \begin{tabular}{lllllll}
    \hline
    CIFAR10-VGG11    & Rep 1     & Rep 2   & Rep 3 & Rep 4 & Rep 5 & Average \\
    \hline
    SB+Clip       & 89.40\%   & 89.41\% & 89.89\% & 89.52\% & 89.47\% & 89.54\% \\
    Our 1 & 90.76\%   & 90.57\% & 90.49\% & 90.85\% & 90.79\% & \textbf{90.67\%} \\
    Our 2 & 90.67\% & 90.23\% & 90.52\% & 90.13\% & 90.70\% & 90.45\% \\
   \hline
    CIFAR100-VGG16    & Rep 1     & Rep 2   & Rep 3 & Rep 4 & Rep 5 & Average \\
   \hline
    SB+Clip       & 55.76\%   & 56.8\% & 56.38\% & 56.35\% & 56.32\% & 56.32\% \\
    Our 1& 67.43\%   & 65.12\% & 65.14\% & 65.96\% & 63.57\% & \textbf{65.44\%} \\
    Our 2& 67.19\% & 61.17\% & 60.97\% & 64.75\% & 60.90\% &   62.99\%            \\
    \hline
    
 \end{tabular}
\end{table}

\subsubsection*{Acknowledgement} 

This work is partially supported by NSF awards DMS-2134012 and CCF-2019844 as a part of NSF Institute for Foundations of Machine Learning (IFML). 

\bibliography{bib_appendix}
\bibliographystyle{iclr2022_conference}

\appendix
\counterwithin{equation}{section}
\counterwithin{theorem}{section}
\counterwithin{definition}{section}
\counterwithin{figure}{section}
\counterwithin{table}{section}
\counterwithin{assumption}{section}
\section{Details of Numerical Experiments}
\label{sec: experiment details appendix}

\subsection{Details of the $\R^1$ simulation experiment}

The function $f$ used in the experiments is
\begin{equation}\label{aeq:f}
\begin{aligned}
    & f(x)= (x+1.6)(x+1.3)^2(x-0.2)^2(x-0.7)^2(x-1.6)\big(0.05|1.65-x|\big)^{0.6} \\
    & \ \ \cdot \Big( 1 + \frac{1}{ 0.01 + 4(x-0.5)^2  } \Big)\Big( 1 + \frac{1}{0.1 + 4(x+1.5)^2} \Big)\Big( 1 - \frac{1}{4}\exp( -5(x + 0.8)(x + 0.8)  ) \Big).
\end{aligned}
\end{equation}

\begin{figure}[ht]
\vskip 0.2in
\begin{center}
\centerline{\includegraphics[width=0.5\textwidth]{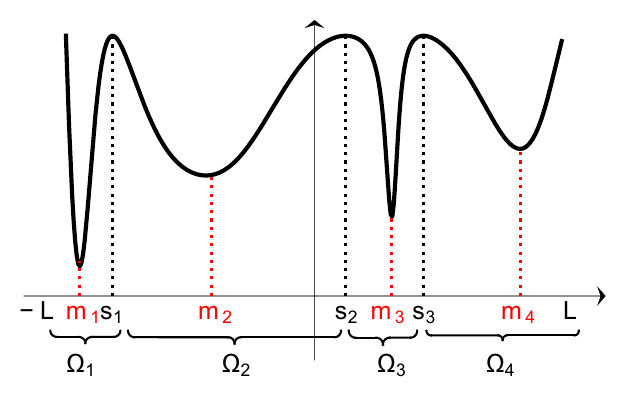} }
\caption{Illustration of the test function $f$ used in the $\R^1$ experiment.}
\label{fig illustration appendix}
\end{center}
\vskip -0.2in
\end{figure}

As shown in Figure \ref{fig illustration appendix},  the four isolated local minimizers of $f$ are $m_1 = -1.51, s_1 = -1.3, m_2 = -0.66, s_2 = 0.2, m_3 = 0.49, s_3 = 0.7, m_4 = 1.32$, and in our experiment we restrict the iterates on $[-L,L]$ with $L = 1.6$. The heavy-tailed noises we used in the experiment were $Z_n = 0.1U_n W_n$ where $W_n$ were sampled from Pareto Type II distribution (aka Lomax distribution) with shape parameter $\alpha = 1.2$, and the signs $U_n$ were iid RVs such that $\P(U_n = 1) = \P(U_n = -1) = 1/2$. 

In the first exit time experiment, we tested three different settings: (a) $b = 0.28$ (so that $l^*=3$); (b) $b = 0.5$ (so that $l^* = 2$); (c) no gradient clipping (so that $l^*_2 = 1$). 
For the first case, we tested learning rates $\{0.1,0.05,0.03,0.02,0.01,0.005,0.003,0.001\}$, while for the other two cases, we tested learning rates $\{0.1,0.03,0.01,0.003,0.001,0.0003,0.0001\}$. For each case, we ran the simulation 20 times and plotted the average of the 20 exit times. Lastly, to prevent excessively long running time of the experiment, the simulation was terminated when the iteration number reached $5\times 10^7$. This threshold was reached only in the setting with $\eta = 0.001, b = 0.28$.

\begin{figure}[ht]
\vskip 0.2in
\begin{center}
\centerline{
\includegraphics[width=0.3\textwidth]{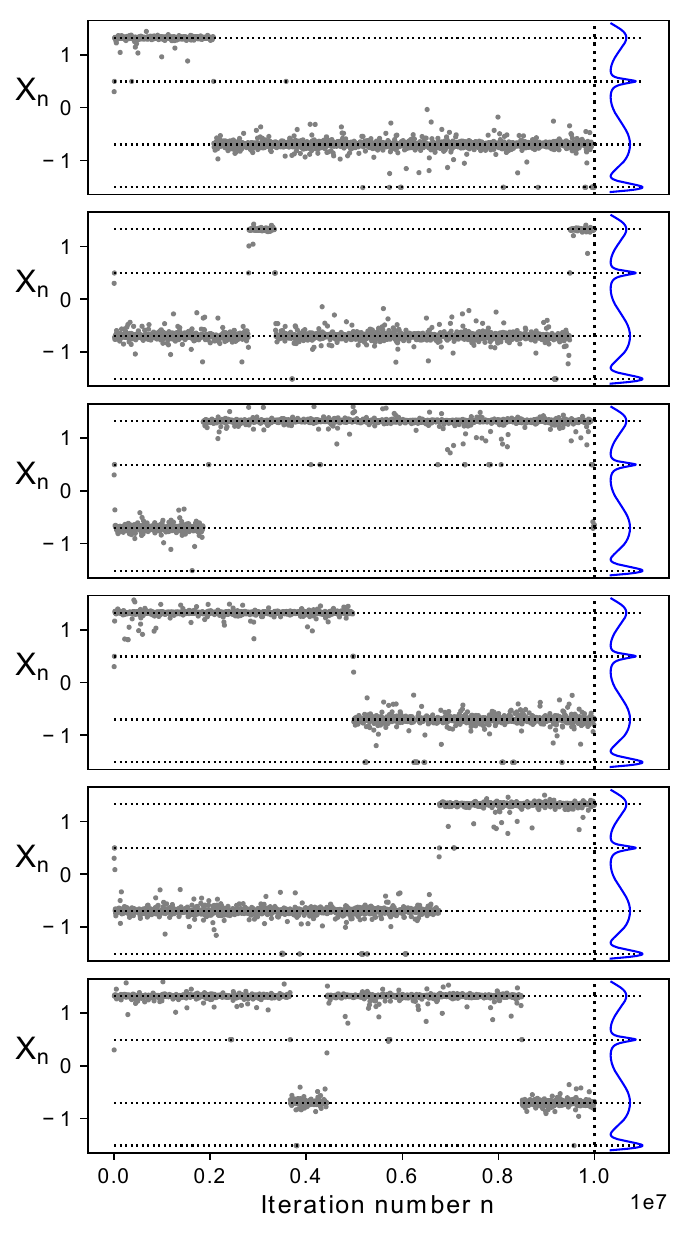} 
\includegraphics[width=0.3\textwidth]{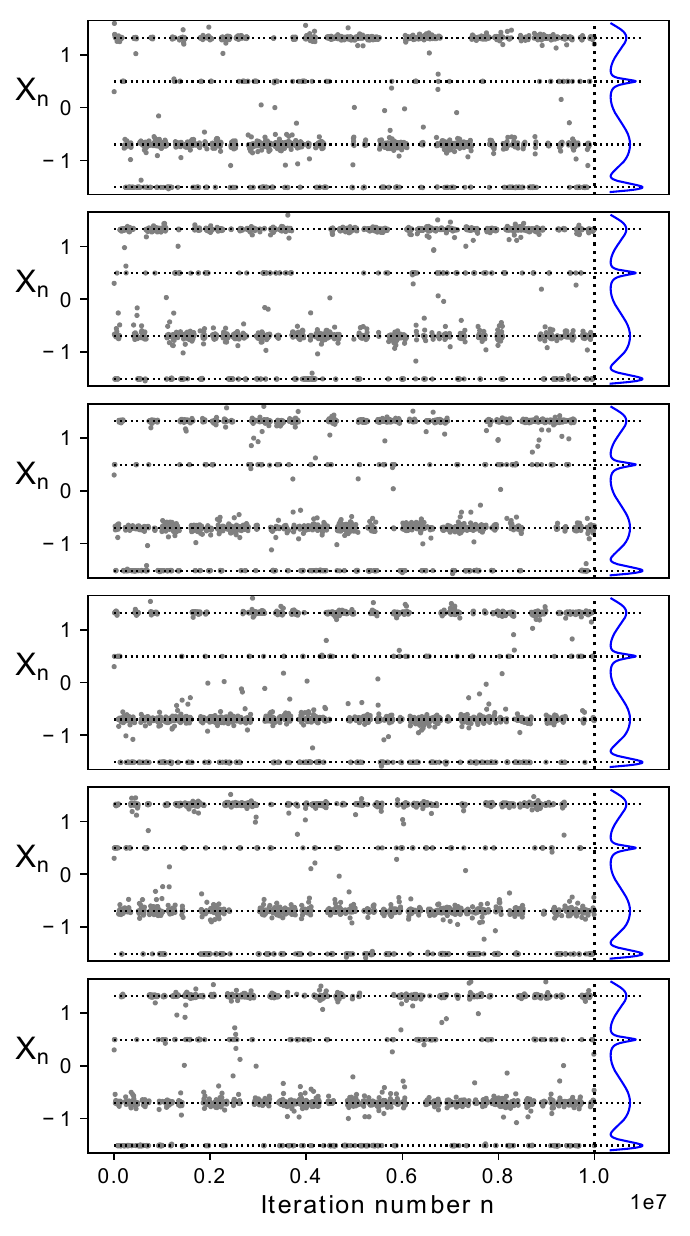}
}
\caption{Five sample paths of SGD under heavy-tailed noises with gradient clipping (left) and without gradient clipping (right). Note that in each case, SGD sample paths exhibit similar patters: SGD almost completely avoided sharp minima with gradient clipping, whereas SGD spent significant amount of time at the sharp minima without gradient clipping.}
\label{fig appendix sample path heavy}
\end{center}
\vskip -0.2in
\end{figure}

\begin{figure}[ht]
\vskip 0.2in
\begin{center}
\centerline{
\includegraphics[width=0.3\textwidth]{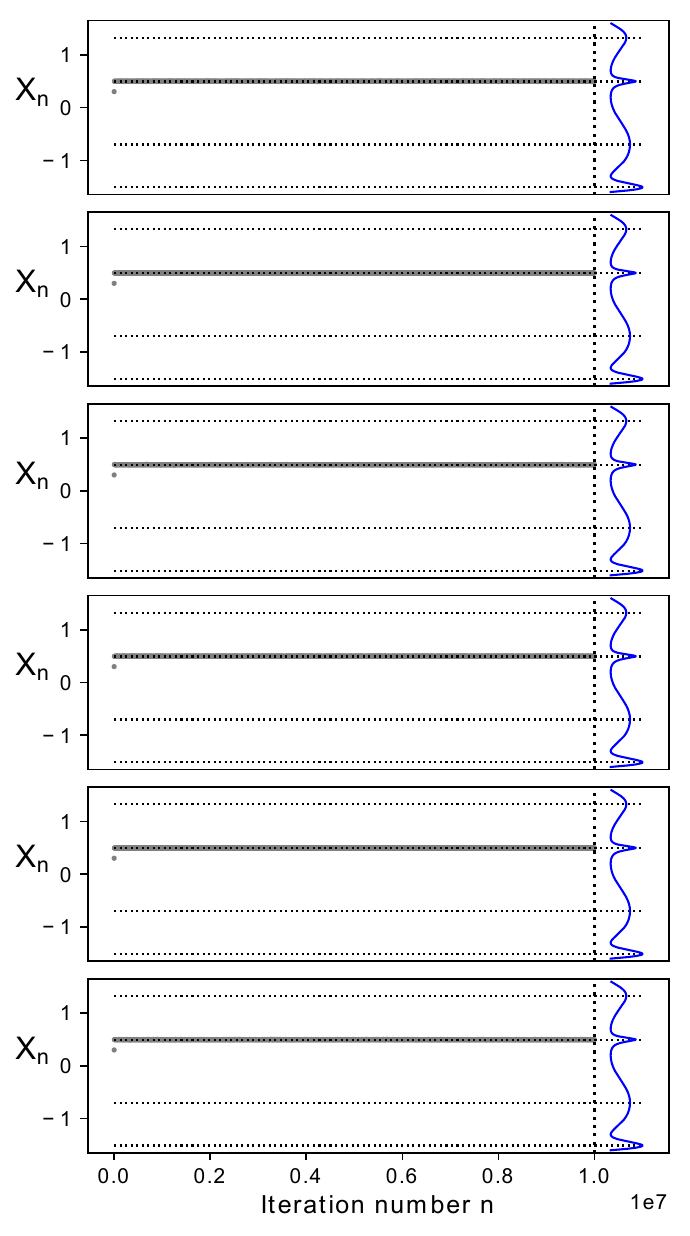} 
\includegraphics[width=0.3\textwidth]{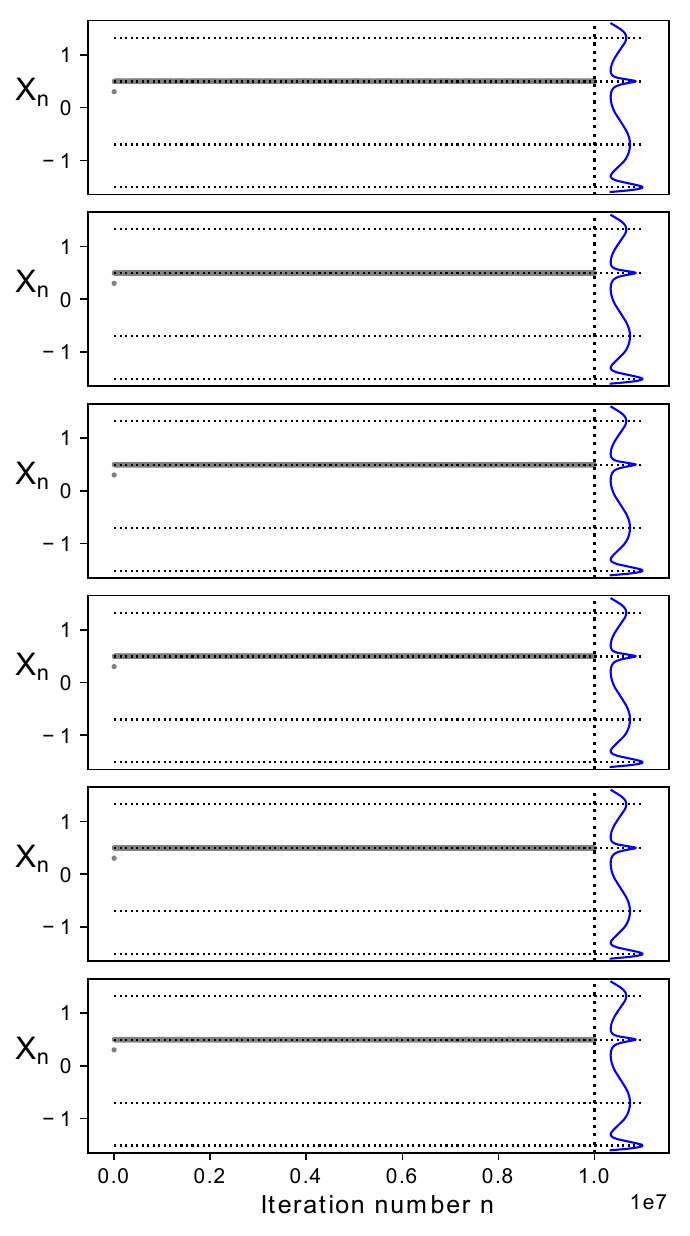}
}
\caption{Five sample paths of SGD under light-tailed noises with gradient clipping (left) and without gradient clipping (right). Note that regardless of the use of gradient clipping, SGD never manages to escape the local minimum that it started from.}
\label{fig appendix sample path light}
\end{center}
\vskip -0.2in
\end{figure}

Next, we present extra sample paths of SGD when applied to function $f$ in \cref{aeq:f} in Figure~\ref{fig appendix sample path heavy} and \ref{fig appendix sample path light}. 
The blue curve on the right side of each plot shows $f$ rotated by 90 degrees, and the dashed lines indicate the locations of local minima. For better readability of the figures, we plotted $X_n$ for every 5,000 iterations.
To generate these plots, we initialized the SGD iterates at $0.3$ (so that it is in $\Omega_3 = (0.2,0.7)$) and fixed the learning rate as $\eta = 0.001$. 
Again, we tested both with gradient clipping (with $b = 0.5$) and without gradient clipping. 
Moreover, we also tested \textbf{light-tailed} noises where we use $N(0,1)$ as the distribution for noises $Z_n$. 
For each sample path of $X_n$, we run $10,000,000$ iterations. 
In the left plots of Figure~\ref{fig appendix sample path heavy}, one can see that with clipped heavy-tailed stochastic gradients, the SGD iterates almost always stay around the wide attraction fields, and the sharp minima are almost completely eliminated from the trajectories of SGD. 
In comparison, in the right plots of Figure~\ref{fig appendix sample path heavy} one can see that without gradient clipping, the heavy-tailed noises will drive SGD to spend substantial amount of time in all the different local minima, including the sharp ones. 
Lastly, from Figures \ref{fig appendix sample path light}, one can see that under light-tailed noises and small learning rates, SGD cannot escape a sharp minima once trapped there. 

\subsection{Details of the $\R^d$ simulation experiment}

As illustrated in the contour plot in Figure \ref{fig Rd simulation} (a), the function $f$ in this experiment is a modified version of Himmelblau function, a commonly used test function for optimization algorithm. The modifications serve two purposes. First, as shown in Figure \ref{fig Rd simulation} (b), for the modified function the four attraction fields $\Omega_1,\Omega_2,\Omega_3,\Omega_4$ have different sizes; in particular, under gradient clipping threshold $b=2.15$, from the local minimizers of $\Omega_1$ and $\Omega_2$ (indicated by red dots in the corresponding area) at least two jumps are required to escape from the attraction field, while from the local minimizer in $\Omega_3$ or $\Omega_4$ it is possible to escape with one jump. Therefore, for the minimum jump number required to escape, we have $l^*_1 = l^*_2 = 2 > l^*_3 = l^*_4 = 1$ in this case. Second, for the modified test function $f$, the local minimizer in $\Omega_2$ is not a single point but a connected line segment, which is indicated by the dark line in bottom-left region in Figure \ref{fig Rd simulation} (a) and the red line segment in in Figure \ref{fig Rd simulation} (b). Therefore, the modification allows us to test the heavy-tailed SGD methods on a more general loss landscape.

Now we describe the construction of the test function $f$. Let $h$ be the Himmelblau function with expression $h(x,y) = (x^2 + y - 11)^2 + (x + y^2 -7 )^2$. Next, define the following transformation for coordinates: $\phi(x,y) = \Big(x(\exp( c_0(x - c_x) + 1 ) ) , y(\exp( c_0(x - c_x) + 1 ) \Big)$. Let the composition be $h_\phi(x,y) = h\Big( \phi(x - a_x, y  ) \Big)$. To create the connected region of local minimizers, define the following locally ``cut'' version of $h_\phi$:
\begin{align*}
i(x,y) & = \mathbbm{1}\{ x \in [b_l,b_r], \ |y-a_y|<b_y  \}, \\
    h^*(x,y) & = (1 - i(x,y))h_\phi(x,y) + i(x,y)\min\{ h_\phi(x,y),\ c_1|y - a_y|^{1.1} \}.
\end{align*}
In other words, by taking minimum of the original $h_\phi$ and a polynomial function w.r.t. $y$ around the original local minimizer of $\Omega_2$, we obtain a function $h^*$ that attains local minimum on an entire line segment with $y = a_y$. Lastly, the test function we use in the experiment is $f = 0.1h^*$, with $a_x = 1.5, a_y = -2.9, b_l = -5.5, b_r = -0.5, b_y = 2.0, c_0 = 0.4, c_1 = 12.$

In the experiment, we initialize the SGD iterates $X_k$ at $X_0 = (2.9, 1.0)$, which is very close to the local minimizer in the small attraction field $\Omega_3$. For both the clipped and unclipped SGD, we perform updates for $3\times 10^7$ steps, under learning rate $5\times 10^{-4}$ and heavy-tailed noise $Z_k = 0.75 W_k$ where the iid samples $W_k$ are isotropic and the law of $\norm{W_k}$, the size of the noise, is Pareto(1.2). For clipped SGD, we use threshold $b = 2.15$. To prevent the iterates from drifting to infinity, after each update $X_k$ is projected back to the $L_2$ ball centered at origin with radius $4.2$ whenever $X_k$ leaves this ball.

\begin{figure}[ht]
\vskip 0.2in
\begin{center}
\centerline{\includegraphics[width=0.8\textwidth]{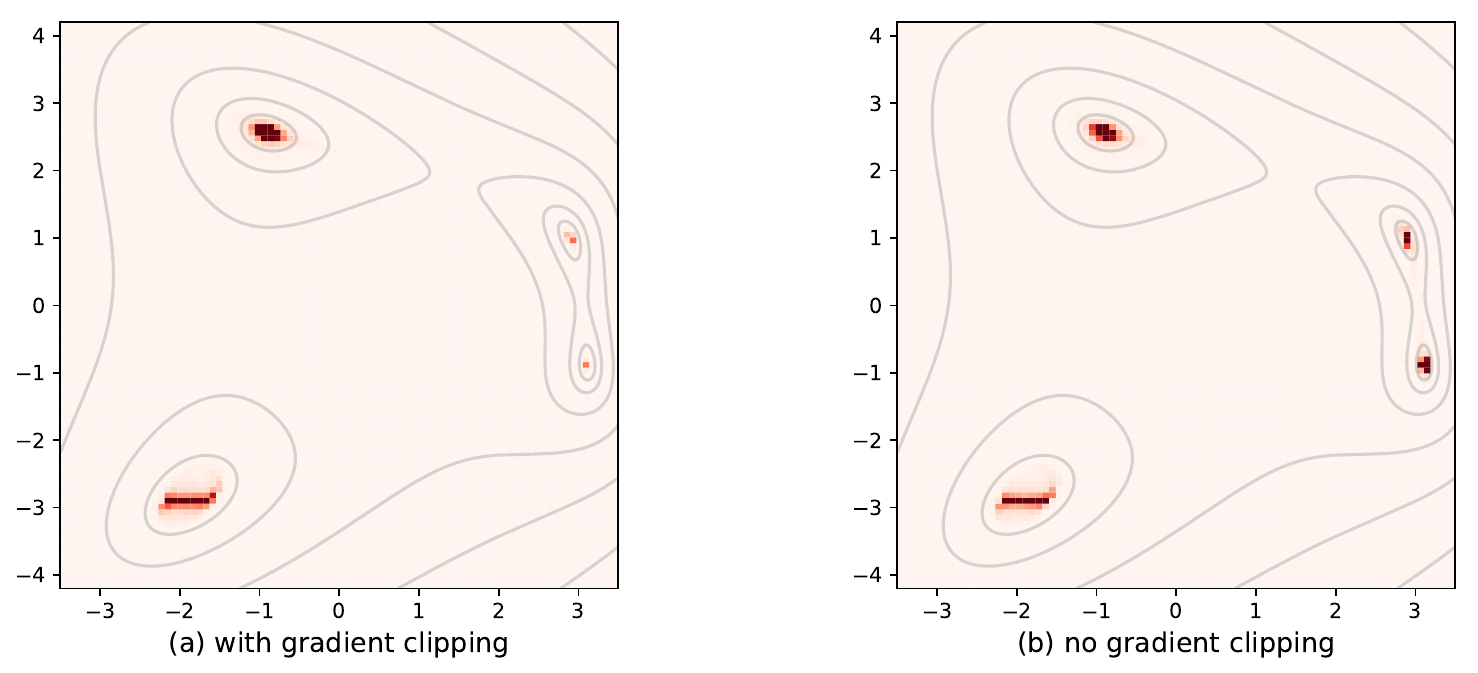} }
\caption{Heat map of SGD iterates when optimizing the modified Himmelblau function.} 
\label{fig Rd simulation heatmap}
\end{center}
\vskip -0.2in
\end{figure}

In Figure \ref{fig Rd simulation heatmap}, we use the $3\times 10^7$ steps of SGD iterates (for both the clipped and unclipped case) to create heat maps showing locations of SGD iterates. From this figure, two points can be made clear: first, the heavy-tailed SGD does spend much less time at the two small attraction fields when gradient clipping is applied; second, in $\Omega_2$ (the bottom-left attraction field) the SGD iterates frequent the entire connected region of local minima instead of a certain point on this line segment.

\subsection{Details of the ablation study}

\begin{figure}[ht]
\vskip 0.2in
\begin{center}
\centerline{\includegraphics[width=0.9\textwidth]{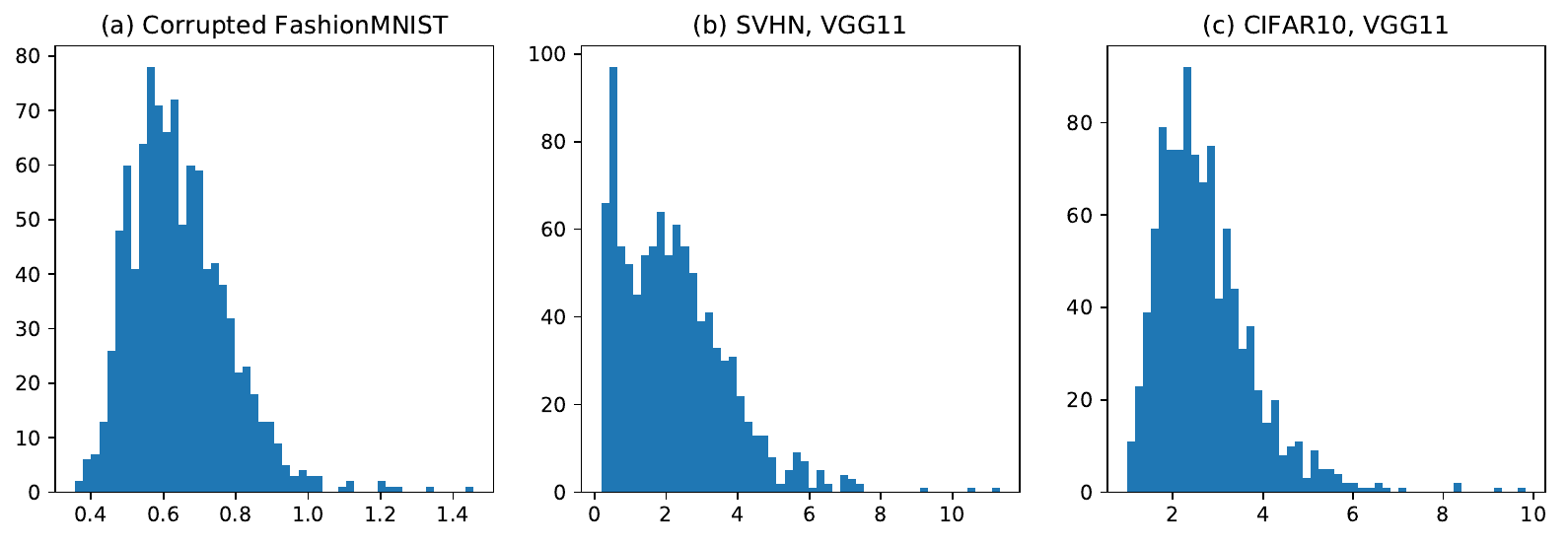} }
\caption{Distribution of gradient noise in different tasks of the ablation study.}
\label{fig grad noise dist ablation study}
\end{center}
\vskip -0.2in
\end{figure}

\begin{figure}[ht]
\vskip 0.2in
\begin{center}
\centerline{\includegraphics[width=0.9\textwidth]{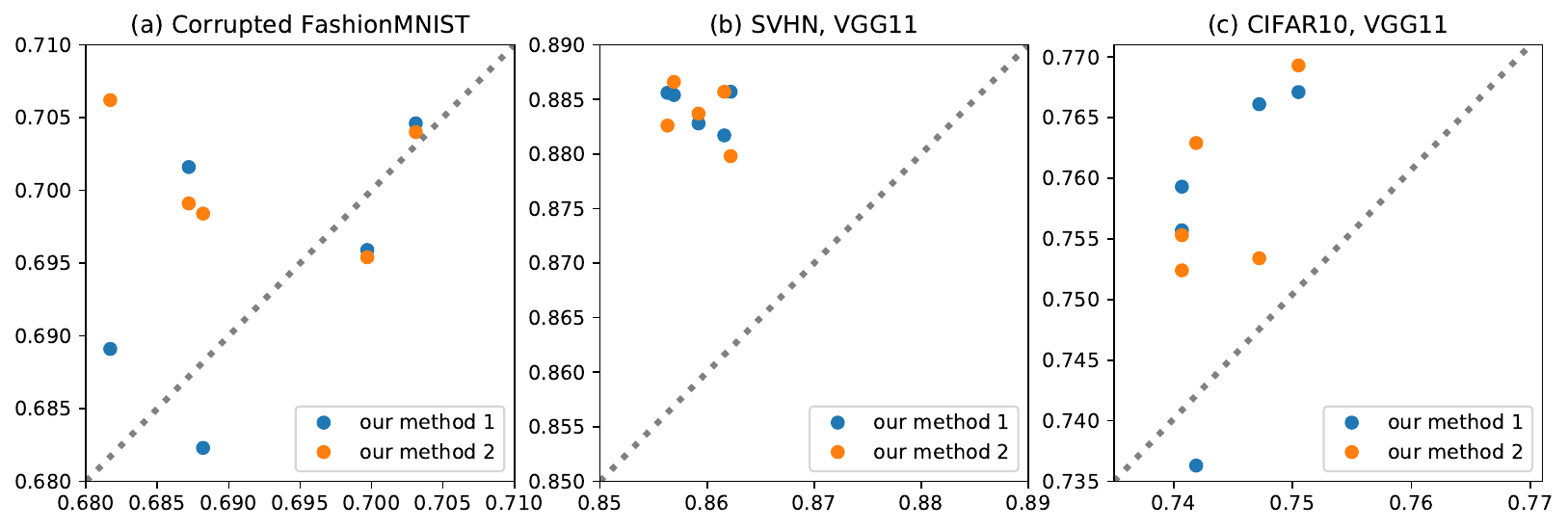} }
\caption{Test accuracy of the proposed clipped heavy-tailed methods vs. test accuracy of vanilla SGD in the ablation study.}
\label{fig accuracy ablation study}
\end{center}
\vskip -0.2in
\end{figure}

We first mention that the all experiments using neural networks are conducted on Nvidia GeForce GTX 1080 Ti. For the ablation study, the experiments and scripts are adapted from the ones in \cite{pmlr-v97-zhu19e}.\footnote{https://github.com/uuujf/SGDNoise}.

In Figure \ref{fig grad noise dist ablation study}, we display the gradient noise distribution in the three tasks of the ablation study after the model is randomly initialized.

The experiment setting and choice of hyperparameters are mostly adapted from the experiment in \cite{pmlr-v97-zhu19e}. We consider three different tasks: (1) training LeNet on corrupted FashionMNIST dataset; specifically, we use a 1200-sample subset of the original FashionMNIST training dataset, and for 200 samples points in the training set we randomly assign a label instead of using the correct ones; (2) VGG11 on SVHN dataset, where we use a 25000-sample subset of the training dataset; (3) VGG11 on CIFAR10, where we use the entire training set. For all tasks we use the entire test dataset when evaluating test accuracy.

\begin{table}
  \caption{Test accuracy (percentage) and expected sharpness of different methods across different tasks. The reported numbers are the averages and 95\%CI over 5 replications. }
  \label{table appendix ablation study with CI}
  \centering
  
  \begin{tabular}{llll}
  \hline 
  Test Accuracy & Corrupted FMNIST, LeNet & SVHN, VGG11  & CIFAR10, VGG11 \\
  \hline 
  LB & 68.7$\pm$0.4 & 82.9$\pm$0.4 & 69.4$\pm$0.5 \\
  SB & 69.2$\pm$0.8 & 85.9$\pm$0.2 & 74.4$\pm$0.4 \\
  SB + Clip & 68.8$\pm$0.6 & 85.9$\pm$0.2 & 74.4$\pm$0.8 \\
  SB + Noise & 64.4$\pm$3.4 & 38.9$\pm$24.1 & 40.5$\pm$25.1 \\
  Our 1 & 69.5$\pm$0.8 & \textbf{88.4}$\pm$\textbf{0.2} & 75.7$\pm$1.1 \\
  Our 2 & \textbf{70.1}$\pm$\textbf{0.4} & \textbf{88.4}$\pm$\textbf{0.2} & \textbf{75.9}$\pm$\textbf{0.7} \\
  \hline
  Expected Sharpness & Corrupted FMNIST, LeNet & SVHN, VGG11  & CIFAR10, VGG11 \\
  \hline 
  LB & 0.032$\pm$0.006 & 0.694$\pm$0.048 & 2.043$\pm$0.083 \\
  SB & 0.008$\pm$0.001 & 0.037$\pm$0.007 & 0.050$\pm$0.013 \\
  SB + Clip & 0.009$\pm$0.001 & 0.041$\pm$0.006 & 0.039$\pm$0.019 \\
  SB + Noise & 0.047$\pm$0.02 & 0.012$\pm$0.009 & 2.046$\pm$2.4 \\
  Our 1 & 0.003$\pm$0.0003 & \textbf{0.002}$\pm$\textbf{0.0007} & \textbf{0.024}$\pm$\textbf{0.005} \\
  Our 2 & \textbf{0.002}$\pm$\textbf{0.0002} & 0.005$\pm$0.004 & 0.037$\pm$0.007 \\
  \hline 
  \end{tabular}
\end{table}

\begin{table}
  \caption{Hyperparameters for training in the ablation study}
  \label{table appendix parameters ablation study}
  \centering
  \begin{tabular}{llll}
    \toprule
    Hyperparameters    &  FashionMNIST, LeNet &  SVHN, VGG11  & CIFAR10, VGG11   \\
    \midrule
    learning rate & 0.05 & 0.05 & 0.05 \\
    batch size for $g_{SB}$ & 100 & 100 & 100 \\
    training iterations & 10,000 & 30,000 & 30,000 \\
    gradient clipping threshold & 5 & 20 & 20 \\
    $c$ & 0.5 & 0.5 & 0.5 \\
    $\alpha$ & 1.4 & 1.4 & 1.4 \\
    \bottomrule
  \end{tabular}
\end{table}

\begin{table}
  \caption{Sharpness of different methods across different tasks. The reported numbers are the averages over 5 replications. }
  \label{table ablation study more sharpness metrics}
  \centering
  \begin{tabular}{llll}
  \hline 
  PAC-Bayes Sharpness & Corrupted FMNIST, LeNet & SVHN, VGG11  & CIFAR10, VGG11 \\
  \hline 
  LB & $5.9\times 10^3$ & $2.97\times 10^4$ & $4.87 \times 10^4$ \\
  SB & $3 \times 10^3$ & $6.9 \times 10^3$ & $7.2 \times 10^3$ \\
  SB + Clip & $3.3 \times 10^3$ & $7.3\times 10^3$ & $6.8 \times 10^3$ \\
  SB + Noise & $3.1 \times 10^3$ & $7.76 \times 10^4$ & $6.74 \times 10^4$ \\
  Our 1 & $1.9 \times 10^3$ & $\bm{2.1 \times 10^3}$ & $\bm{4.8 \times 10^3}$ \\
  Our 2 & $\bm{1.6 \times 10^3}$ & $2.3 \times 10^3$ & $5.8 \times 10^3$ \\
  \hline
  Maximal Sharpness & Corrupted FMNIST, LeNet & SVHN, VGG11  & CIFAR10, VGG11 \\
  \hline 
  LB & $1.01 \times 10^4$ & $3.78 \times 10^4$ & $5.46 \times 10^4$ \\
  SB & $4.9 \times 10^3$ & $9.1 \times 10^3$ & $8.5 \times 10^3$ \\
  SB + Clip & $5.4 \times 10^3$ & $9.3 \times 10^3$ & $8 \times 10^3$ \\
  SB + Noise & $5.4 \times 10^3$ & $1.19 \times 10^5$ & $1.18 \times 10^5$ \\
  Our 1 & $3.2 \times 10^3$ & $\bm{2.5 \times 10^3}$ & $\bm{5.8 \times 10^3}$ \\
  Our 2 & $\bm{2.5 \times 10^3}$ & $2.8 \times 10^3$ & $6.5 \times 10^3$ \\
  \hline 
  \end{tabular}
\end{table}

The heavy-tailed multipliers $Z_n$ used in this experiment, whenever heavy-tailed noise is needed, are $Z_n = cW_n$ where $W_n$ are iid Pareto($\alpha$) RVs. For each task, we first randomly initialize each model, and then run the 6 candidate methods in parallel starting from the same randomly initialized model weights for a fair comparison.

The hyperparameters in training for each task are listed in Table \ref{table appendix parameters ablation study}. The same set of hyperparameters is used for all methods in the same task. Whenever gradient clipping scheme is applied, we clip the gradient if its $L_2$ norm exceeds the threshold given in Table \ref{table appendix parameters ablation study}. The exception here is the ``\emph{SB + Noise}'' method: we use learning rate $\eta = 0.005$; for FashionMNIST task we train for 100,000 iterations and the heavy-tailed noise is removed for the final 50,000 iterations; for SVHN and CIFAR10 tasks, we train for 150,000 iterations and heavy-tailed noise is removed for the last 70,000 iterations. Besides, for this method we always clip the model weights if its $L_\infty$ norm exceeds 1. The reason for the extra tuning and extended training in ``\emph{SB + Noise}'' method is that, without the said modifications, in all three tasks we observed that the model weights quickly drift to infinity and explodes; even with the weight clipping implemented, the model performance stays at random level with no signs of improvements if we do not tune down learning rate. 

In Table \ref{table ablation study more sharpness metrics}, we also report the sharpness of solutions under different shaprness metrics.
First, the \emph{PAC-Bayes Sharpness} metric (see equation (53) in \cite{jiang2019fantastic}) is defined as $1/\sigma^2$ where $\sigma$ is equal to the smallest $\delta$ that induces a 0.1 expected sharpness, and reflects the sharpness/flatness parameter used in studies on generalization gaps under the PAC-Bayes framework (see \cite{neyshabur2017exploring}). 
Besides, the \emph{Maximal Sharpness} metric (see equation (54) in \cite{jiang2019fantastic}) is defined as $1/\sigma^2$ where $\sigma$ is equal to the smallest radius $\delta$ that makes $\max_{ \norm{\nu}_{\infty} \leq \delta }| L(\theta^* + \nu) - L(\theta^*) | \geq 0.1$, and metrics of form $\max_{ \norm{\nu} \leq \delta }| L(\theta^* + \nu) - L(\theta^*) |$ can be considered as a proxy for the spectral norm of the Hessian at the solution (see \cite{dinh2017sharp}). 
It worth noticing that, for all three sharpness metrics, the smaller the value is the "flatter" the loss landscape is around the solution.
Lastly, for evaluation of the PAC-Bayes Sharpness and Maximal Sharpness metrics, we conduct binary search as in Algorithm 2 of \cite{jiang2019fantastic} with $\epsilon_d = 0.01, \epsilon_{\sigma} = 0, M_1 = 10$ and $M_2 = 100$; in our setting we always evaluate the training loss using one sweep of the entire training set, so $M_3$ is a case-specific and is equal to the number of batches of the training set under the batch size for the task at hand.

In Figure \ref{fig accuracy ablation study}, we plot the test accuracy of our method against that of the SGD for all 5 replications and 3 tasks.

\subsection{Details of CIFAR10/100 experiments with data augmentation}

\begin{table}
  \caption{Results and 95\% CI in the experiments with data augmentation.}
  \label{table appendix data augmentation training}
  \centering
  \begin{tabular}{llll}
    \hline
    Test Accuracy    & SB + Clip & Our 1 & Our 2 \\
    \hline
    CIFAR10, VGG11       & 89.5$\pm$0.2 & \textbf{90.7}$\pm$\textbf{0.1} & 90.5$\pm$0.2 \\
    CIFAR100, VGG16      & 56.3$\pm$0.3 & \textbf{65.4}$\pm$\textbf{1.2} & 63.0$\pm$2.5 \\
   \hline
    Expected Sharpness   & SB + Clip & Our 1 & Our 2 \\
   \hline
    CIFAR10, VGG11             & 0.17$\pm$0.005 & \textbf{0.09}$\pm$\textbf{0.004} & 0.10$\pm$0.003 \\ 
    CIFAR100, VGG16     & 0.86$\pm$0.02 & \textbf{0.44}$\pm$\textbf{0.05} & 0.48$\pm$0.07 \\
    \hline
    
 \end{tabular}
\end{table}

\begin{table}
  \caption{Sharpness of solutions obtained by different methods in CIFAR10/100 tasks with data augmentation. Numbers reported here are the average of 5 replications.}
  \label{table appendix data augmentation training}
  \centering
 \begin{tabular}{llll}
    \hline 
    CIFAR10-VGG11 & SB + Clip & Our 1 & Our 2 \\
    \hline
    Expected Sharpness & 0.167 & \textbf{0.085} & 0.096 \\
    PAC-Bayes Sharpness & $1.31 \times 10 ^4$ & $\bm{9 \times 10^3}$ & $10^4$ \\
    Maximal Sharpness & $1.66 \times 10^4$ & $1.29 \times 10^4$ & $\bm{1.22 \times 10^4}$ \\
    \hline
    CIFAR100-VGG16 & SB + Clip & Our 1 & Our 2 \\
    \hline
    Expected Sharpness & 0.857 & \textbf{0.441} & 0.479 \\
    PAC-Bayes Sharpness & $2.49 \times 10 ^4$ & $\bm{1.9 \times 10^4}$ & $1.98 \times 10^4$ \\
    Maximal Sharpness & $2.75 \times 10^4$ & $\bm{2.12 \times 10^4}$ & $2.16 \times 10^4$ \\
    \hline
    
 \end{tabular}
\end{table}

 For both methods, we train the model for 300 epochs and set the initial learning rate as 0.1. In our method, the training can be partitioned into two phases. In the first phase (the first 200 epochs), the learning rate is kept at a constant. In the second phase, for every 30 epoch we reduce the learning rate by half. Also, an $L_2$ weight decaying with coefficient $5\times 10^{-4}$ is enforced. As for parameters for heavy-tailed noises in \cref{update our heavy tailed method}, we use $c = 0.5$ and $\alpha = 1.4$ in the first phase, and in the second phase we remove heavy-tailed noise and use \emph{SB} to update weights. In both methods for the small-batch direction $g_{SB}$ the batch size is 128, while for $g_{LB}$ we evaluate the gradient on a large sample batch of size 1,024. Under the epoch number 300 and batch size 128, the count of total iterations performed during training is $1.17\times 10^5$. To augment the dataset, random horizontal flipping and cropping with padding size 4 is applied for each training batch. Lastly, gradient clipping scheme is applied for both methods, and we fix $b = 0.5$. In other words, when the learning rate is $\eta$ (note that due to the scheduling of learning rates, $\eta$ will be changing throughout the training), the gradient is clipped if its $L_2$ norm is larger than $b/\eta$. The scripts are adapted from the ones in \href{https://github.com/chengyangfu/pytorch-vgg-cifar10}{https://github.com/chengyangfu/pytorch-vgg-cifar10}.
 
These results are presented in Table~\ref{table data augmentation training}.
Furthermore, in Table~\ref{table appendix data augmentation training} we see that our truncated heavy-tailed method also manages to find solutions with a flatter geometry.

\subsection{Discussion on Gradient Noise Distributions in the Experiments}
In this subsection, we (i) present the empirical evidence that supports our characterization of the baseline model—i.e., the absence of heavy tails in the gradient noise—and (ii) clarify that the emergence of heavy-tails in the *stationary distribution* of SGD (argued in \cite{hodgkinson2020multiplicative,gurbuzbalaban2020heavy}) does not contradict the observed absence of heavy tails in the *gradient noise* of SGD. This allows us to study the impact of truncated heavy-tails in the gradient noise separately from the choice of hyper-parameters. 

We start with our observations on the noise distributions.
In view of the well-established wisdom in heavy-tail literature that there is no single perfect tail estimator, we analyzed the stochastic gradient noise with four different methods: QQ plot, empirical mean residual life (EMRL), Hill plot, and PLFIT \citep{clauset2009power}. We applied these estimation/diagnostic tools (i) at the beginning of the training, (ii) halfway through the training, (iii) at the end of the training. Throughout all our experiments, we consistently observe strong evidence that the gradient noises are light-tailed, and even if (against all odds) the noises were from a heavy-tailed distribution, the tail index should be far greater (hence, the resulting tail is much lighter) than the heavy-tails we inject (or the popular alpha stable assumption), and hence, the point we make with our tail-inflation experiment is still valid. We summarize the results as follows.

First, QQ plots below (Fig. \ref{fig: appendix QQ start}-\ref{fig: appendix QQ end}) clearly show that the tails in noise distribution are always much lighter than the Pareto distributions with $\alpha = 2$ or even 10. Therefore, the typical power-law assumption, especially the alpha-stable distributions in \cite{simsekli2019tail} (with $\alpha \in (0,2)$), seems far from the distribution of the actual data we obtained in the image classification tasks. In fact, the tail of the noise distributions seems to be between that of lognormal and normal distributions, implying that it is lighter than any power-law distribution.

\begin{figure}[h!]
\begin{center}
\begin{tabular}{c c c c c}
Pareto, $\alpha = 2$ & Pareto, $\alpha = 5$ & Pareto, $\alpha = 10$ & Lognormal & Normal \\
\includegraphics[width=0.15\textwidth]{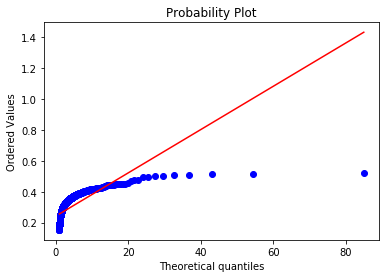}  &
\includegraphics[width=0.15\textwidth]{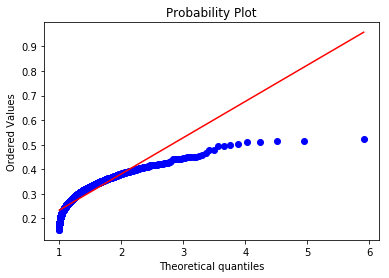}  &
\includegraphics[width=0.15\textwidth]{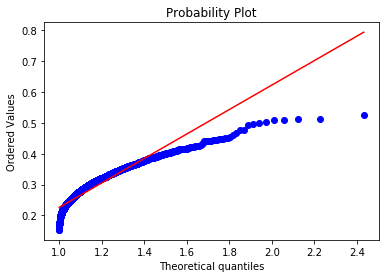}  &
\includegraphics[width=0.15\textwidth]{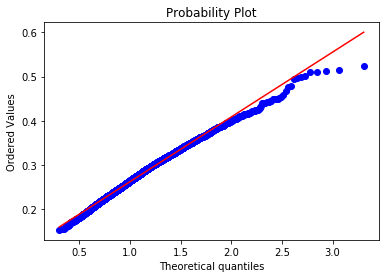}  &
\includegraphics[width=0.15\textwidth]{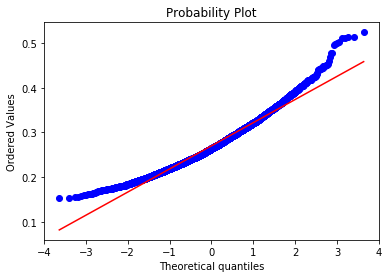}
\\
\end{tabular}
\caption{Ablation Study, Corrupted FMNIST \& LeNet: At the beginning}
\label{fig: appendix QQ start}
\end{center}
\end{figure}

\begin{figure}[h!]
\begin{center}
\begin{tabular}{c c c c c}
Pareto, $\alpha = 2$ & Pareto, $\alpha = 5$ & Pareto, $\alpha = 10$ & Lognormal & Normal \\
\includegraphics[width=0.15\textwidth]{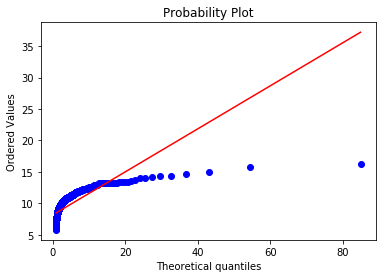}  &
\includegraphics[width=0.15\textwidth]{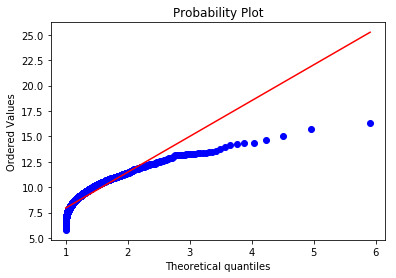}  &
\includegraphics[width=0.15\textwidth]{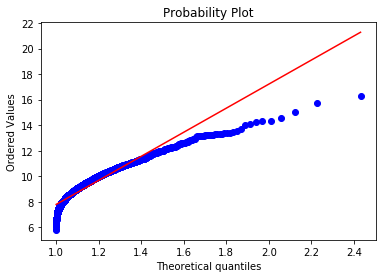}  &
\includegraphics[width=0.15\textwidth]{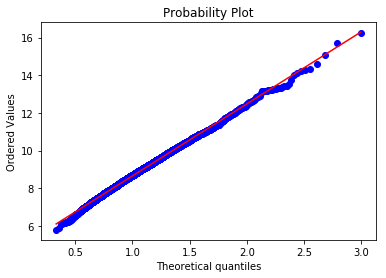}  &
\includegraphics[width=0.15\textwidth]{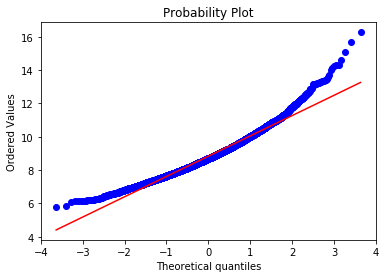}
\\
\end{tabular}
\caption{Ablation Study, Corrupted FMNIST \& LeNet: Half way through the training}
\end{center}
\end{figure}

\begin{figure}[h!]
\begin{center}
\begin{tabular}{c c c c c}
Pareto, $\alpha = 2$ & Pareto, $\alpha = 5$ & Pareto, $\alpha = 10$ & Lognormal & Normal \\
\includegraphics[width=0.15\textwidth]{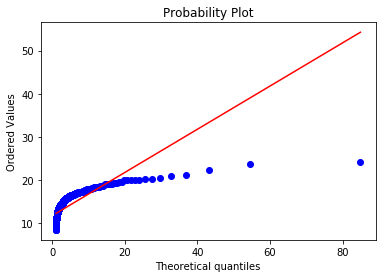}  &
\includegraphics[width=0.15\textwidth]{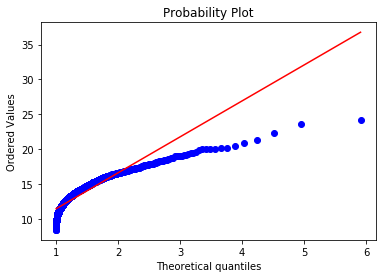}  &
\includegraphics[width=0.15\textwidth]{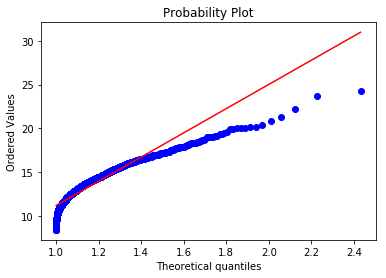}  &
\includegraphics[width=0.15\textwidth]{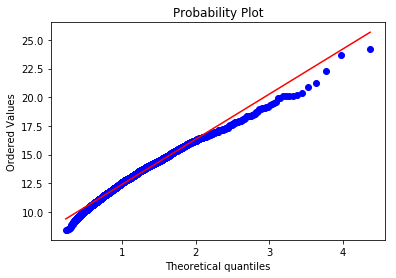}  &
\includegraphics[width=0.15\textwidth]{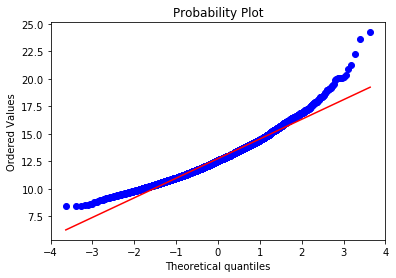}
\\
\end{tabular}
\caption{Ablation Study, Corrupted FMNIST \& LeNet: At the end of training}
\end{center}
\end{figure}


\begin{figure}[h!]
\begin{center}
\begin{tabular}{c c c c c}
Pareto, $\alpha = 2$ & Pareto, $\alpha = 5$ & Pareto, $\alpha = 10$ & Lognormal & Normal \\
\includegraphics[width=0.15\textwidth]{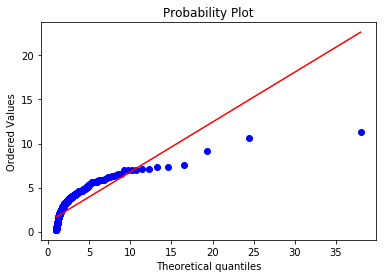}  &
\includegraphics[width=0.15\textwidth]{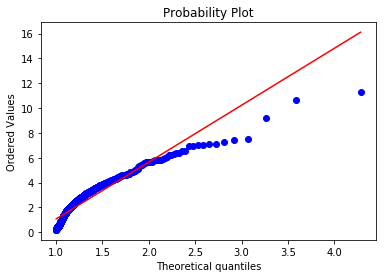}  &
\includegraphics[width=0.15\textwidth]{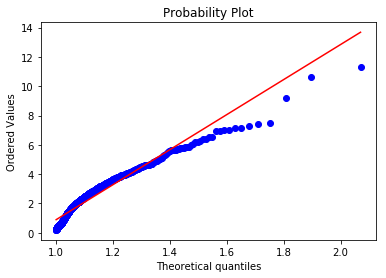}  &
\includegraphics[width=0.15\textwidth]{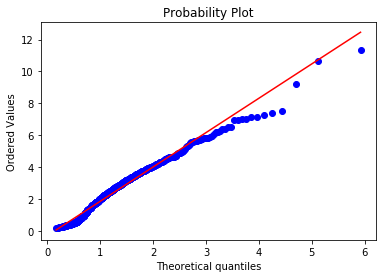}  &
\includegraphics[width=0.15\textwidth]{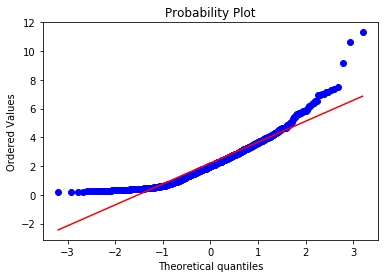}
\\
\end{tabular}
\caption{Ablation Study, SVHN \& VGG11: At the beginning}
\end{center}
\end{figure}

\begin{figure}[h!]
\begin{center}
\begin{tabular}{c c c c c}
Pareto, $\alpha = 2$ & Pareto, $\alpha = 5$ & Pareto, $\alpha = 10$ & Lognormal & Normal \\
\includegraphics[width=0.15\textwidth]{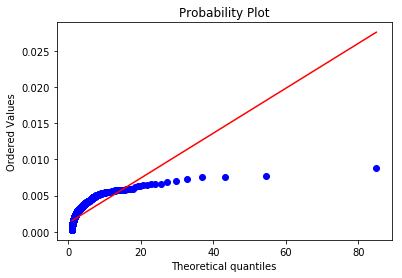}  &
\includegraphics[width=0.15\textwidth]{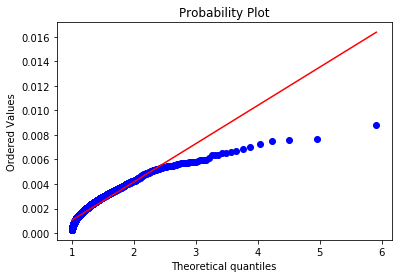}  &
\includegraphics[width=0.15\textwidth]{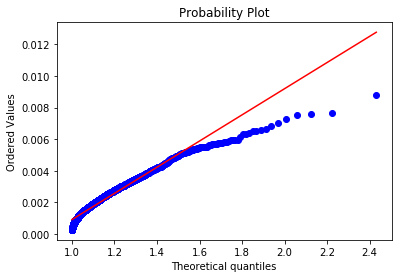}  &
\includegraphics[width=0.15\textwidth]{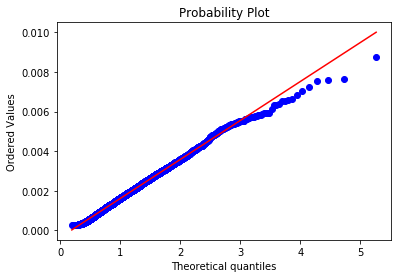}  &
\includegraphics[width=0.15\textwidth]{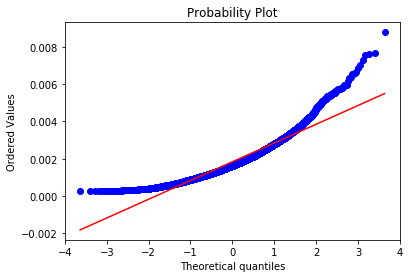}
\\
\end{tabular}
\caption{Ablation Study, SVHN \& VGG11: Half way through the training}
\end{center}
\end{figure}

\begin{figure}[h!]
\begin{center}
\begin{tabular}{c c c c c}
Pareto, $\alpha = 2$ & Pareto, $\alpha = 5$ & Pareto, $\alpha = 10$ & Lognormal & Normal \\
\includegraphics[width=0.15\textwidth]{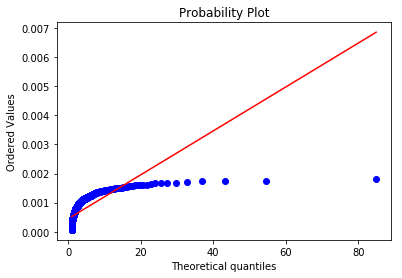}  &
\includegraphics[width=0.15\textwidth]{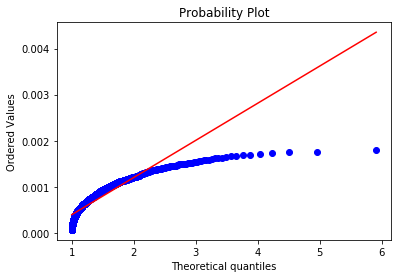}  &
\includegraphics[width=0.15\textwidth]{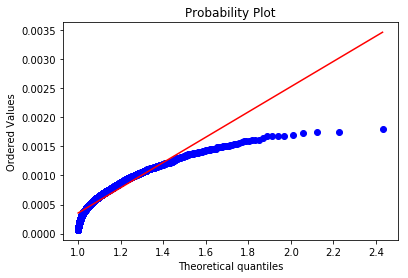}  &
\includegraphics[width=0.15\textwidth]{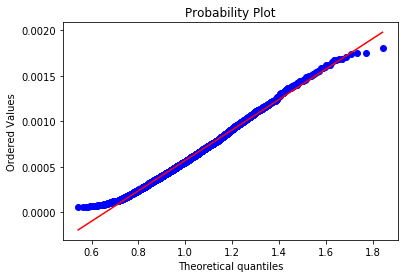}  &
\includegraphics[width=0.15\textwidth]{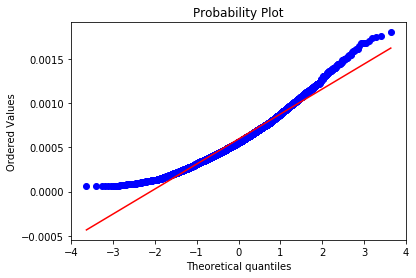}
\\
\end{tabular}
\caption{Ablation Study, SVHN \& VGG11: At the end of training}
\end{center}
\end{figure}


\begin{figure}[h!]
\begin{center}
\begin{tabular}{c c c c c}
Pareto, $\alpha = 2$ & Pareto, $\alpha = 5$ & Pareto, $\alpha = 10$ & Lognormal & Normal \\
\includegraphics[width=0.15\textwidth]{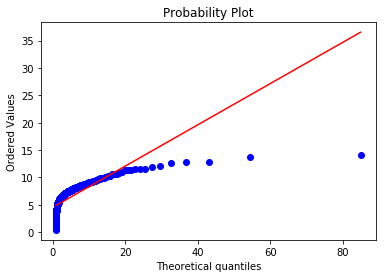}  &
\includegraphics[width=0.15\textwidth]{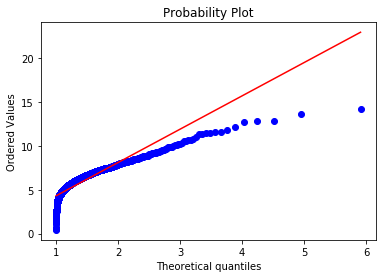}  &
\includegraphics[width=0.15\textwidth]{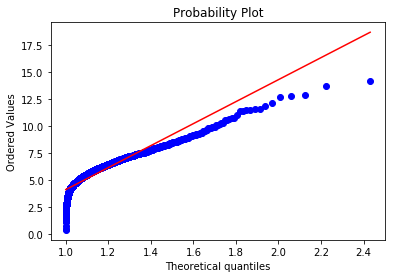}  &
\includegraphics[width=0.15\textwidth]{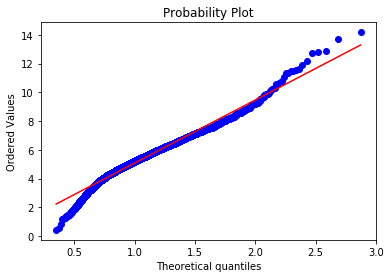}  &
\includegraphics[width=0.15\textwidth]{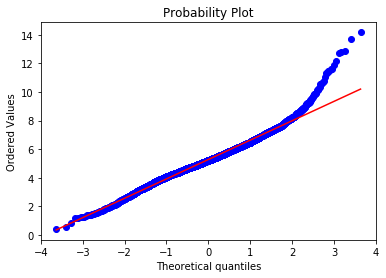}
\\
\end{tabular}
\caption{Ablation Study, CIFAR10 \& VGG11: At the beginning}
\end{center}
\end{figure}

\begin{figure}[h!]
\begin{center}
\begin{tabular}{c c c c c}
Pareto, $\alpha = 2$ & Pareto, $\alpha = 5$ & Pareto, $\alpha = 10$ & Lognormal & Normal \\
\includegraphics[width=0.15\textwidth]{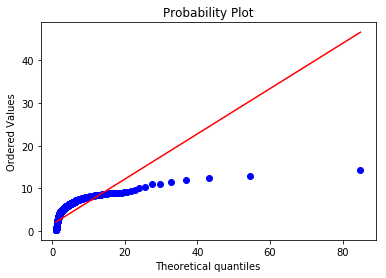}  &
\includegraphics[width=0.15\textwidth]{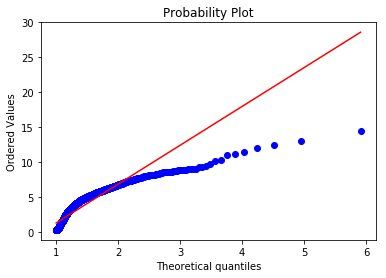}  &
\includegraphics[width=0.15\textwidth]{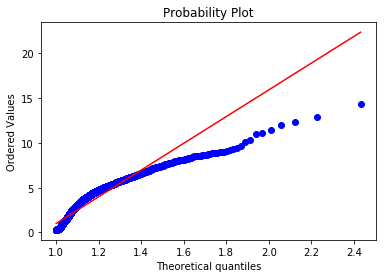}  &
\includegraphics[width=0.15\textwidth]{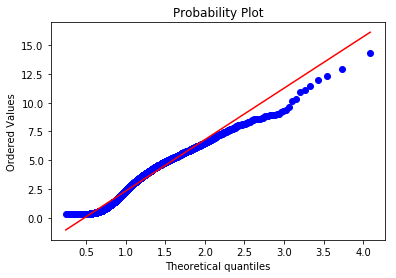}  &
\includegraphics[width=0.15\textwidth]{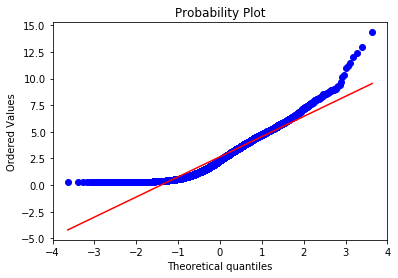}
\\
\end{tabular}
\caption{Ablation Study, CIFAR10 \& VGG11: Half way through the training}
\end{center}
\end{figure}

\begin{figure}[h!]
\begin{center}
\begin{tabular}{c c c c c}
Pareto, $\alpha = 2$ & Pareto, $\alpha = 5$ & Pareto, $\alpha = 10$ & Lognormal & Normal \\
\includegraphics[width=0.15\textwidth]{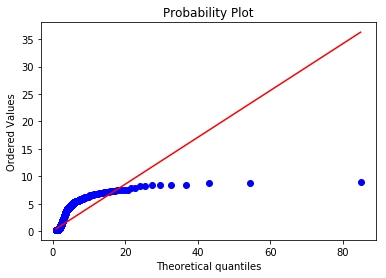}  &
\includegraphics[width=0.15\textwidth]{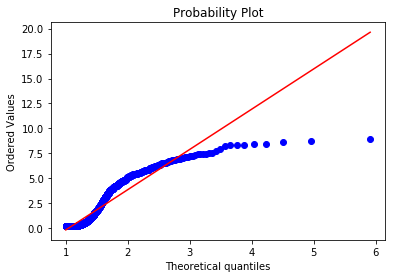}  &
\includegraphics[width=0.15\textwidth]{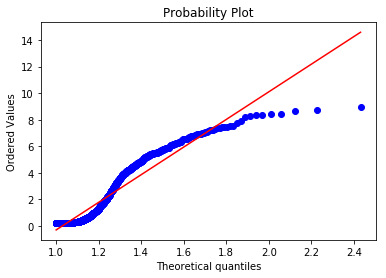}  &
\includegraphics[width=0.15\textwidth]{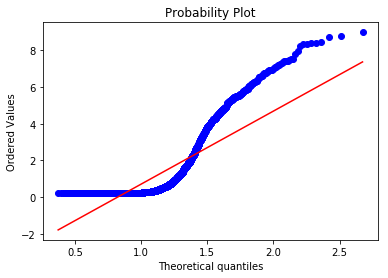}  &
\includegraphics[width=0.15\textwidth]{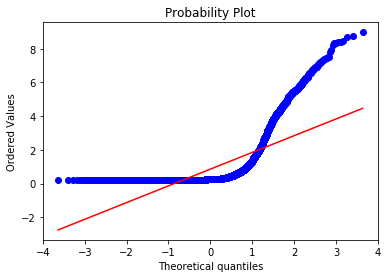}
\\
\end{tabular}
\caption{Ablation Study, CIFAR10 \& VGG11: At the end of training}
\end{center}
\end{figure}


\begin{figure}[h!]
\begin{center}
\begin{tabular}{c c c c c}
Pareto, $\alpha = 2$ & Pareto, $\alpha = 5$ & Pareto, $\alpha = 10$ & Lognormal & Normal \\
\includegraphics[width=0.15\textwidth]{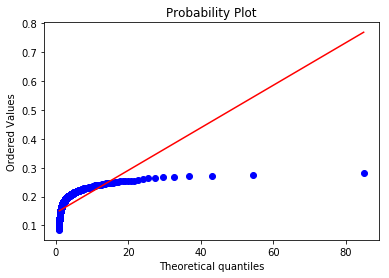}  &
\includegraphics[width=0.15\textwidth]{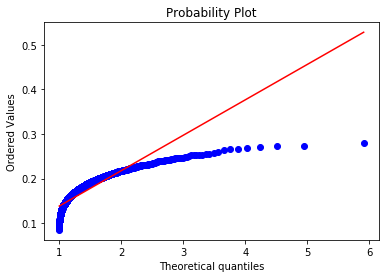}  &
\includegraphics[width=0.15\textwidth]{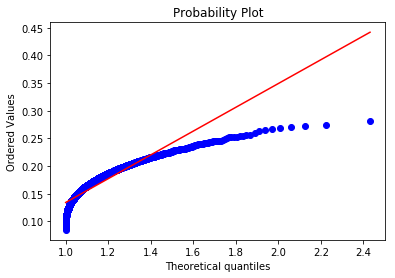}  &
\includegraphics[width=0.15\textwidth]{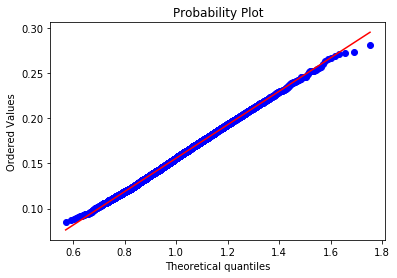}  &
\includegraphics[width=0.15\textwidth]{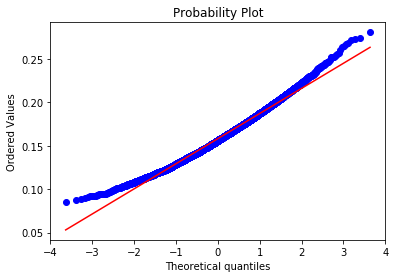}
\\
\end{tabular}
\caption{Data Augmentation, CIFAR10 \& VGG11: At the beginning}
\end{center}
\end{figure}

\begin{figure}[h!]
\begin{center}
\begin{tabular}{c c c c c}
Pareto, $\alpha = 2$ & Pareto, $\alpha = 5$ & Pareto, $\alpha = 10$ & Lognormal & Normal \\
\includegraphics[width=0.15\textwidth]{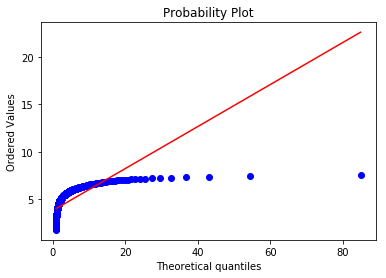}  &
\includegraphics[width=0.15\textwidth]{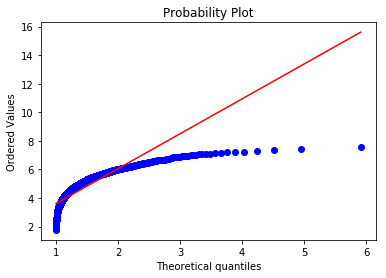}  &
\includegraphics[width=0.15\textwidth]{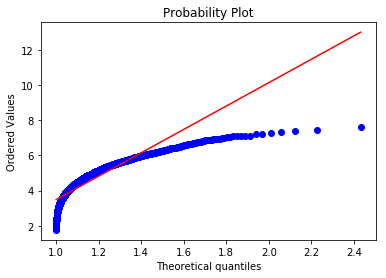}  &
\includegraphics[width=0.15\textwidth]{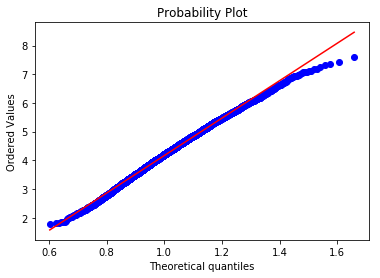}  &
\includegraphics[width=0.15\textwidth]{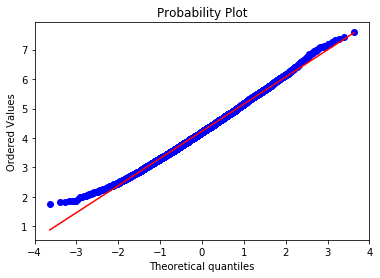}
\\
\end{tabular}
\caption{Data Augmentation, CIFAR10 \& VGG11: Half way through the training}
\end{center}
\end{figure}

\begin{figure}[h!]
\begin{center}
\begin{tabular}{c c c c c}
Pareto, $\alpha = 2$ & Pareto, $\alpha = 5$ & Pareto, $\alpha = 10$ & Lognormal & Normal \\
\includegraphics[width=0.15\textwidth]{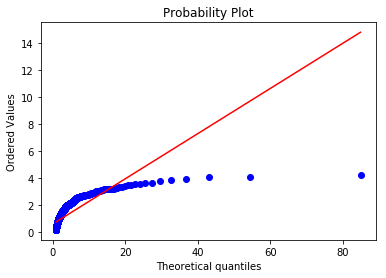}  &
\includegraphics[width=0.15\textwidth]{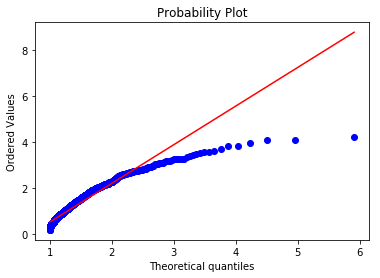}  &
\includegraphics[width=0.15\textwidth]{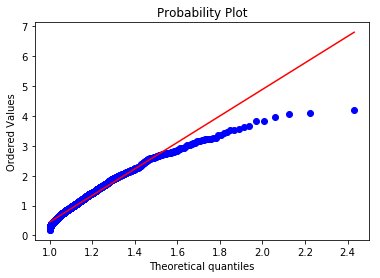}  &
\includegraphics[width=0.15\textwidth]{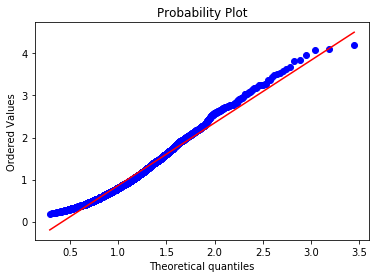}  &
\includegraphics[width=0.15\textwidth]{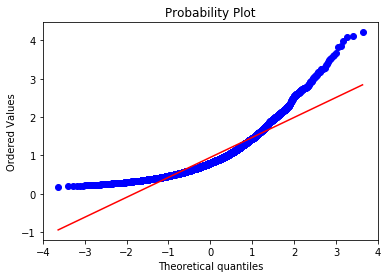}
\\
\end{tabular}
\caption{Data Augmentation, CIFAR10 \& VGG11: At the end of training}
\end{center}
\end{figure}


\begin{figure}[h!]
\begin{center}
\begin{tabular}{c c c c c}
Pareto, $\alpha = 2$ & Pareto, $\alpha = 5$ & Pareto, $\alpha = 10$ & Lognormal & Normal \\
\includegraphics[width=0.15\textwidth]{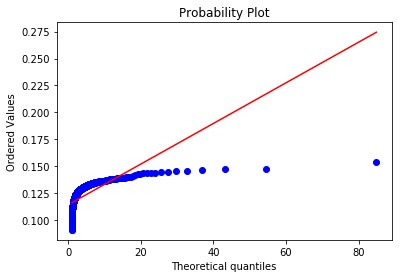}  &
\includegraphics[width=0.15\textwidth]{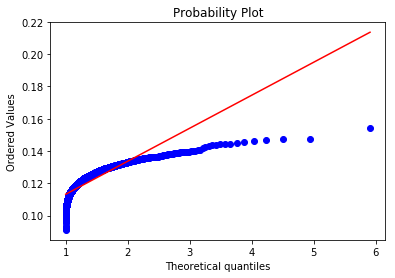}  &
\includegraphics[width=0.15\textwidth]{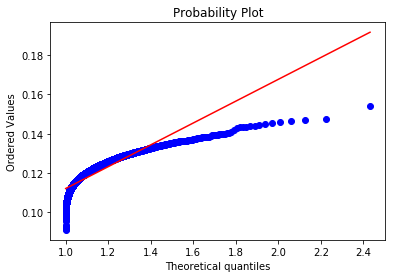}  &
\includegraphics[width=0.15\textwidth]{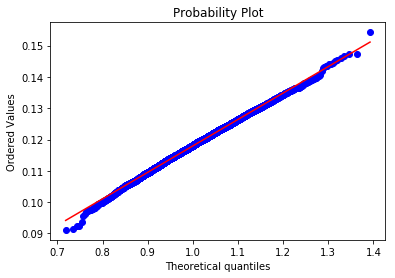}  &
\includegraphics[width=0.15\textwidth]{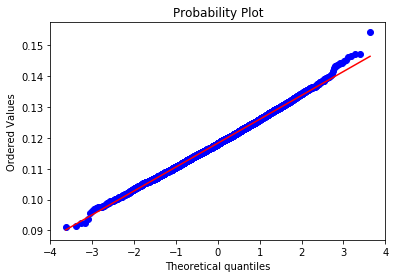}
\\
\end{tabular}
\caption{Data Augmentation, CIFAR100 \& VGG16: At the beginning}
\end{center}
\end{figure}

\begin{figure}[h!]
\begin{center}
\begin{tabular}{c c c c c}
Pareto, $\alpha = 2$ & Pareto, $\alpha = 5$ & Pareto, $\alpha = 10$ & Lognormal & Normal \\
\includegraphics[width=0.15\textwidth]{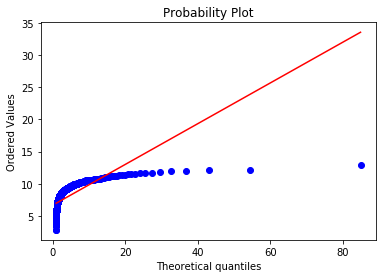}  &
\includegraphics[width=0.15\textwidth]{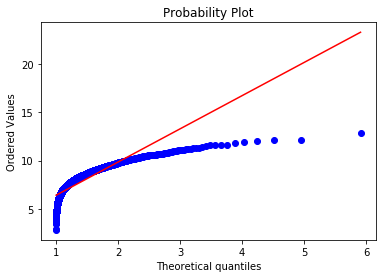}  &
\includegraphics[width=0.15\textwidth]{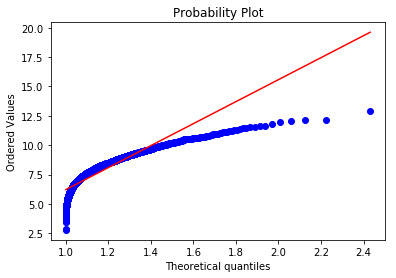}  &
\includegraphics[width=0.15\textwidth]{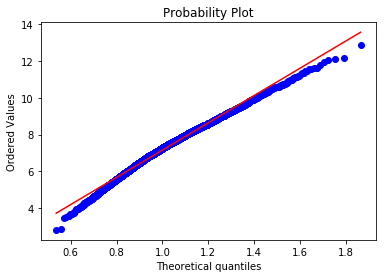}  &
\includegraphics[width=0.15\textwidth]{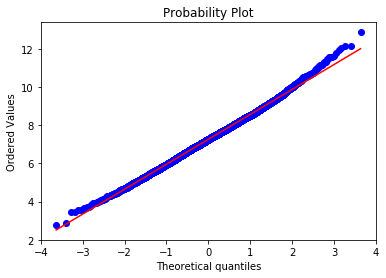}
\\
\end{tabular}
\caption{Data Augmentation, CIFAR100 \& VGG16: Halfway through the training}
\end{center}
\end{figure}

\begin{figure}[h!]
\begin{center}
\begin{tabular}{c c c c c}
Pareto, $\alpha = 2$ & Pareto, $\alpha = 5$ & Pareto, $\alpha = 10$ & Lognormal & Normal \\
\includegraphics[width=0.15\textwidth]{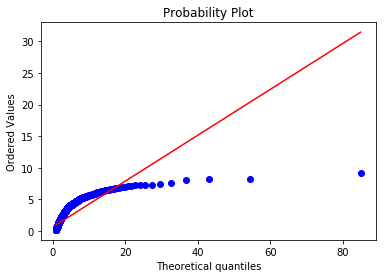}  &
\includegraphics[width=0.15\textwidth]{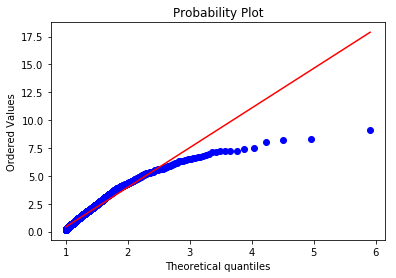}  &
\includegraphics[width=0.15\textwidth]{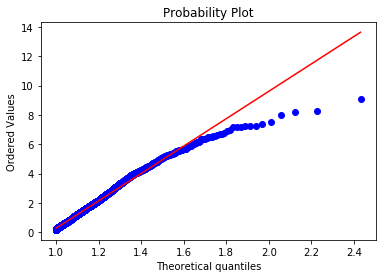}  &
\includegraphics[width=0.15\textwidth]{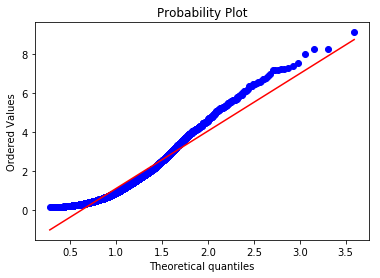}  &
\includegraphics[width=0.15\textwidth]{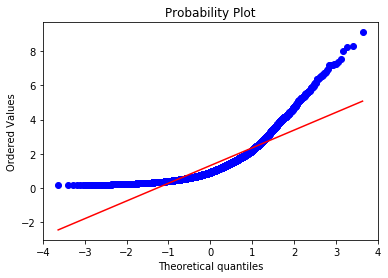}
\\
\end{tabular}
\caption{Data Augmentation, CIFAR100 \& VGG16: At the end of training}
\label{fig: appendix QQ end}
\end{center}
\end{figure}

Next, we plotted the empirical mean residual life (EMRL) of the gradient noise distributions in Fig. \ref{fig: appendix EMRL start}-\ref{fig: appendix EMRL end}. It is well known that the mean residual life blows up to infinity if and only if the distribution is heavy-tailed (more precisely, long-tailed). However, from the figures, one can see that none of the EMRL exhibits such a pattern in any case tested in our experiments. Instead, we see clear downward trends, which strongly suggest light tails, in all the tested cases.

\begin{figure}[h!]
\begin{center}
\begin{tabular}{c c c}
Beginning & Middle & End \\
\includegraphics[width=0.25\textwidth]{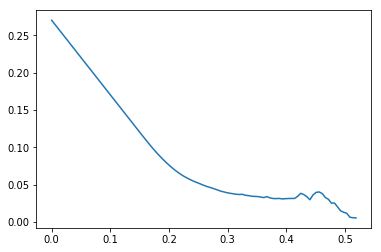}  &
\includegraphics[width=0.25\textwidth]{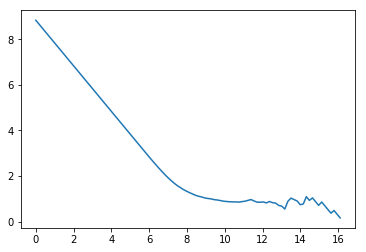}  &
\includegraphics[width=0.25\textwidth]{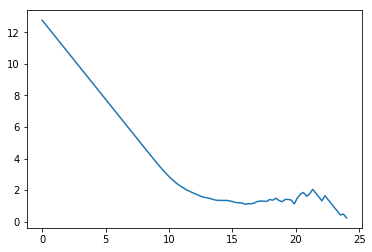}
\\
\end{tabular}
\caption{Plots of empirical mean residual life for noises in FMNIST\&LeNet Task throughout training}
\label{fig: appendix EMRL start}
\end{center}
\end{figure}

\begin{figure}[h!]
\begin{center}
\begin{tabular}{c c c}
Beginning & Middle & End \\
\includegraphics[width=0.25\textwidth]{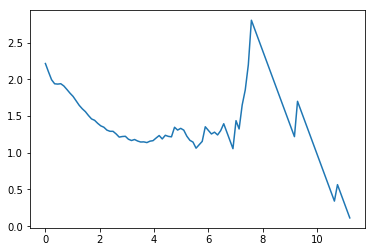}  &
\includegraphics[width=0.25\textwidth]{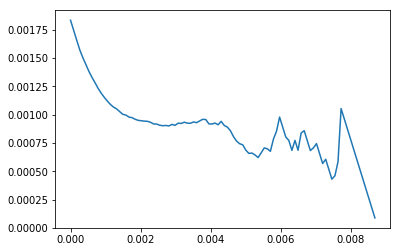}  &
\includegraphics[width=0.25\textwidth]{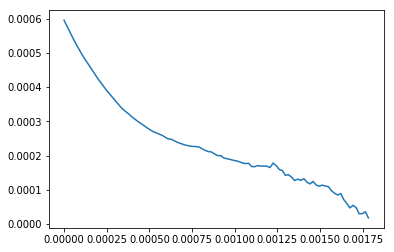}
\\
\end{tabular}
\caption{Plots of empirical mean residual life for noises in SVHN\&VGG 11 Task throughout training}
\end{center}
\end{figure}

\begin{figure}[h!]
\begin{center}
\begin{tabular}{c c c}
Beginning & Middle & End \\
\includegraphics[width=0.25\textwidth]{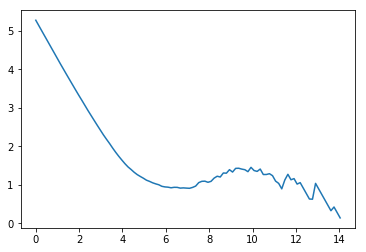}  &
\includegraphics[width=0.25\textwidth]{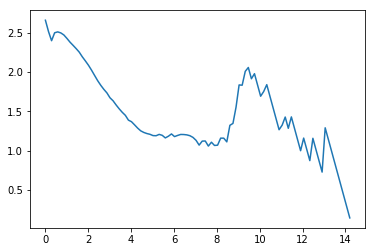}  &
\includegraphics[width=0.25\textwidth]{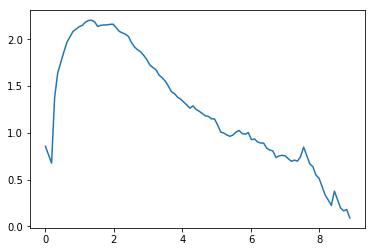}
\\
\end{tabular}
\caption{Plots of empirical mean residual life for noises in CIFAR 10\&VGG 11 Task throughout training}
\end{center}
\end{figure}

\begin{figure}[h!]
\begin{center}
\begin{tabular}{c c c}
Beginning & Middle & End \\
\includegraphics[width=0.25\textwidth]{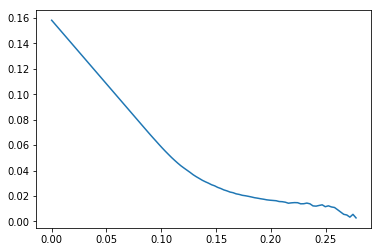}  &
\includegraphics[width=0.25\textwidth]{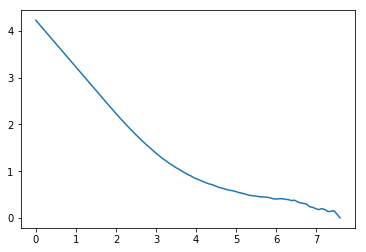}  &
\includegraphics[width=0.25\textwidth]{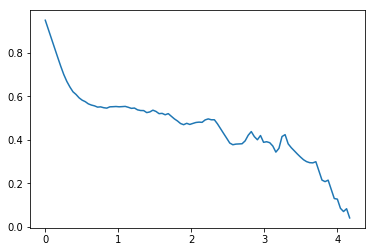}
\\
\end{tabular}
\caption{Plots of empirical mean residual life for noises in dataAug, CIFAR 10\&VGG 11 Task throughout training}
\end{center}
\end{figure}

\begin{figure}[h!]
\begin{center}
\begin{tabular}{c c c}
Beginning & Middle & End \\
\includegraphics[width=0.25\textwidth]{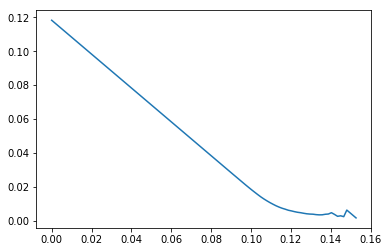}  &
\includegraphics[width=0.25\textwidth]{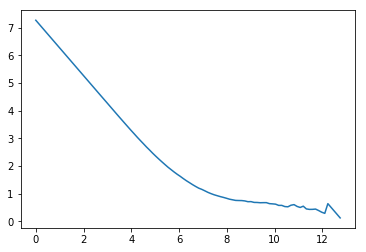}  &
\includegraphics[width=0.25\textwidth]{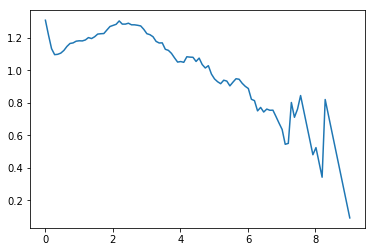}
\\
\end{tabular}
\caption{Plots of empirical mean residual life for noises in dataAug, CIFAR 100\&VGG 11 Task throughout training}
\label{fig: appendix EMRL end}
\end{center}
\end{figure}

If we assume a power-law tail, the Hill estimator is a classical tail index estimator with a long history in extreme value theory literature. A critical algorithmic parameter of the Hill estimator is the number of order statistics used in the estimation, and the Hill plot is a popular exploratory tool for investigating the Hill estimators with different numbers of order statistics. Although it is well known that Hill estimators and Hill plots are fallible if the power-law assumption is not satisfied, (and hence, it is not suited for deciding whether a given set of samples are from light-tailed distribution or heavy-tailed distribution; in particular, the method will return some power-law tail index $\alpha$ whether or not the samples are from a heavy-tailed distribution or a light-tailed one) we present the Hill plot to see what would be the estimated tail indices if the gradient noises hypothetically followed a power law. In the Hill plots shown in Fig. \ref{fig: appendix Hill start}-\ref{fig: appendix Hill end}, we presented the rescaled version (altHill) of the Hill plots (see Chapter 4.4 in \citep{resnick2007heavy}) for the following reason. Hill estimator is a consistent estimator of the power-law index when the proportion of the samples used approaches 0, and altHill plots allow us to scrutinize the estimated indices under a small proportion of samples. In particular, the points around the red dashed lines correspond to estimation using the top 1\% of the samples. We can see that most Hill plots stay well above 10 for the most part and almost never drop below 2. This strongly suggests that even if the gradient noises are from a heavy-tailed distribution, it is likely to have a very high power-law index (implying relatively lighter tails), and hence, we cannot expect to observe a prominent heavy-tailed behavior from them.

\begin{figure}[h!]
\begin{center}
\begin{tabular}{c c c}
Beginning & Middle & End \\
\includegraphics[width=0.25\textwidth]{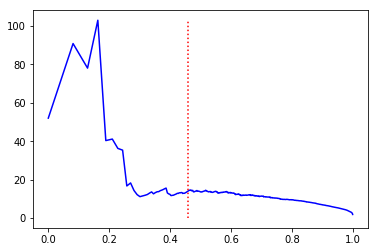}  &
\includegraphics[width=0.25\textwidth]{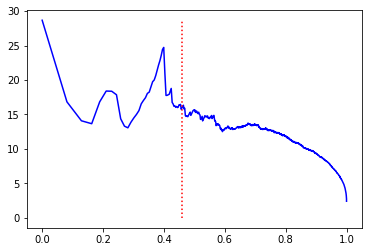}  &
\includegraphics[width=0.25\textwidth]{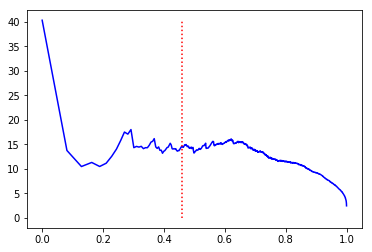}
\\
\end{tabular}
\caption{altHill Plots for noises in FMNIST\&LeNet Task throughout training. Dashed Red Line: Estimation based on the largest $1\%$ data}
\label{fig: appendix Hill start}
\end{center}
\end{figure}

\begin{figure}[h!]
\begin{center}
\begin{tabular}{c c c}
Beginning & Middle & End \\
\includegraphics[width=0.25\textwidth]{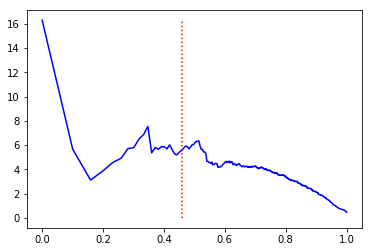}  &
\includegraphics[width=0.25\textwidth]{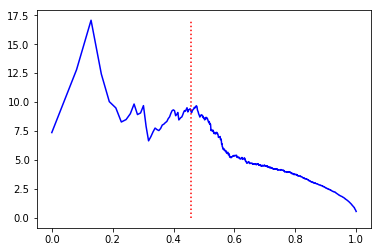}  &
\includegraphics[width=0.25\textwidth]{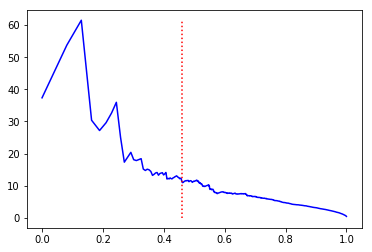}
\\
\end{tabular}
\caption{altHill Plots for noises in SVHN\&VGG 11 Task throughout training. Dashed Red Line: Estimation based on the largest $1\%$ data}
\end{center}
\end{figure}

\begin{figure}[h!]
\begin{center}
\begin{tabular}{c c c}
Beginning & Middle & End \\
\includegraphics[width=0.25\textwidth]{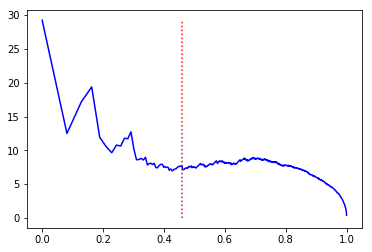}  &
\includegraphics[width=0.25\textwidth]{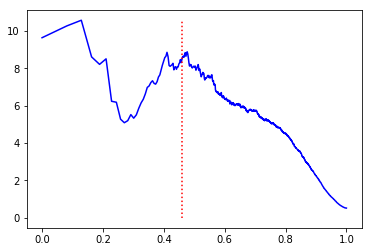}  &
\includegraphics[width=0.25\textwidth]{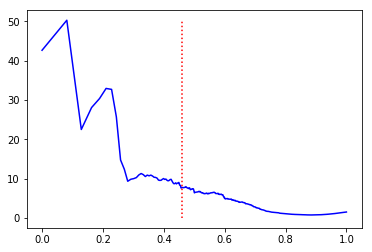}
\\
\end{tabular}
\caption{altHill Plots for noises in CIFAR 10\&VGG 11 Task throughout training. Dashed Red Line: Estimation based on the largest $1\%$ data}
\end{center}
\end{figure}

\begin{figure}[h!]
\begin{center}
\begin{tabular}{c c c}
Beginning & Middle & End \\
\includegraphics[width=0.25\textwidth]{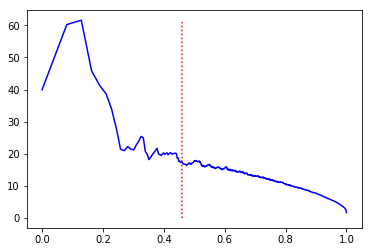}  &
\includegraphics[width=0.25\textwidth]{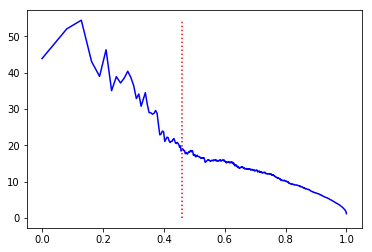}  &
\includegraphics[width=0.25\textwidth]{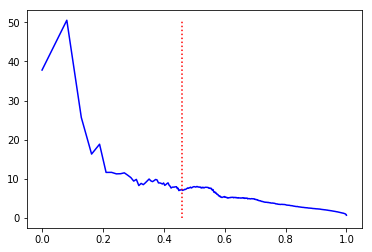}
\\
\end{tabular}
\caption{altHill Plots for noises in Data Augmentation, CIFAR 10\&VGG 11 Task throughout training. Dashed Red Line: Estimation based on the largest $1\%$ data}
\end{center}
\end{figure}

\begin{figure}[h!]
\begin{center}
\begin{tabular}{c c c}
Beginning & Middle & End \\
\includegraphics[width=0.25\textwidth]{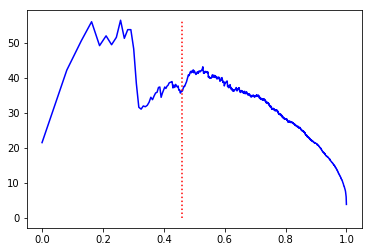}  &
\includegraphics[width=0.25\textwidth]{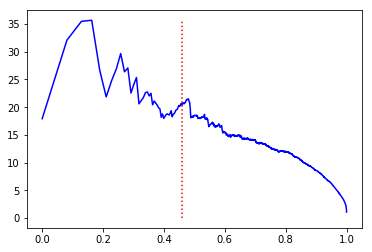}  &
\includegraphics[width=0.25\textwidth]{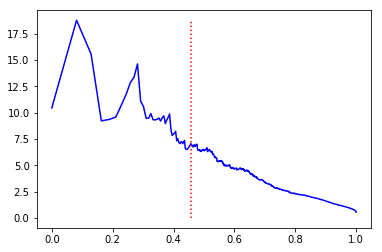}
\\
\end{tabular}
\caption{altHill Plots for noises in Data Augmentation, CIFAR 100\&VGG 16 Task throughout training. Dashed Red Line: Estimation based on the largest $1\%$ data}
\label{fig: appendix Hill end}
\end{center}
\end{figure}

A popular data-driven approach with statistical guarantees (again, under the assumption that the samples are indeed from a heavy-tailed distribution) to select the number of order statistics in the Hill plot is PLFIT \citep{clauset2009power}. We estimated the power-law indices using the python implementation \citep{alstott2014powerlaw} of PLFIT. The numbers are presented in Table \ref{table: appendix PLFIT}. All the estimations are at least 5 for all cases tested in our experiments, and most of the time, the estimation is above 10. Again, this means that even under the hypothetical assumption (against what the QQ plots and EMRLs suggest) that the gradient noises were from a heavy-tailed distribution, the tail indices of the gradient noises should be large, and hence, the gradient noises in our experiments have much lighter tails than any $\alpha$-stable distribution (which requires $\alpha<2$) or the heavy-tailed noises we injected during tail inflation experiments ($\alpha = 1.4$). 
In summary, extensive statistical analyses above suggest the absence of heavy tails in the gradient noises in our experiments. Therefore, the characterization of the vanilla SGD as the light-tailed (or at least lighter than the inflated tails) benchmark in our experiments is valid, and our ablation study is well grounded.

\begin{table}[h!]
  \caption{Power-law Indices Estimation throughout the Training, using PLFIT. All the estimations are at least 5 for all cases tested in our experiments, and most of the times the estimation is above 10. This means that even under the assumption that the gradient noises were from a heavy-tailed distribution, they should have much lighter tails than any $\alpha$-stable distribution (which requires $\alpha<2$) or the heavy-tailed noises we injected during tail inflation experiments ($\alpha = 1.4$). }
  \label{table: appendix PLFIT}
  \centering
  \begin{tabular}{llll}
    \toprule
   Task & Beginning & Middle & End   \\
    \midrule
    FMNIST, LeNet & 14.3 & 14.2 & 16.5 \\
    SVHN, VGG11 & 5.0 & 5.2 & 12.5 \\
    CIFAR10, VGG11 & 9.2 & 6.6 & 7.0 \\
    dataAug, CIFAR10, VGG11 & 16.2 & 16.2 & 8.5 \\
    dataAug, CIFAR100, VGG16 & 35.1 & 14.4 & 5.35 \\
    \bottomrule
  \end{tabular}
\end{table}

Next, we compare our work to recent literature on heavy-tailed phenomena in stationary distribution of SGD; see, for instance, \cite{hodgkinson2020multiplicative,gurbuzbalaban2020heavy}.
To be specific, \cite{hodgkinson2020multiplicative,gurbuzbalaban2020heavy} show that heavy tails can arise in the stationary distribution of SGD through multiplicative dynamics, and the tail index of the resulting stationary distribution can be characterized by the learning rates and the magnitude of noises (through batch size). However, the results in \cite{hodgkinson2020multiplicative,gurbuzbalaban2020heavy} do not imply the existence of heavy tails in the gradient noises. In both papers (as well as other works in the literature), the “heaviness” of the tail of the gradient noise (e.g., power-law index $\alpha$ in heavy-tailed cases) is fixed in the model (same as in our setting) and not entangled with the learning rate or batch size. For example, $B_k$ in (5) of \cite{hodgkinson2020multiplicative} corresponds to the gradient noise, and its tail index doesn’t depend on the learning rate or the batch size. In particular, the change of learning rate does not induce heavy tails in $B_k$.
On the other hand, we focus on the impact of the heavy-tails *in the gradient noise* and the truncation scheme for (any) fixed batch size and small learning rates on the global dynamics of SGD. In view of this, it should be clear that our analysis can be decoupled from the choice of batch size or the impact of the learning rate on SGD’s stationary distribution. Therefore, our observation and the design of the experiments are compatible with the aforementioned references.

 \counterwithin{equation}{section}
\counterwithin{theorem}{section}
\counterwithin{definition}{section}
\counterwithin{figure}{section}
\counterwithin{table}{section}
 \section{Implications of the theoretical results}\label{sec:discussions}

\textbf{Systematic control of the exit times from attraction fields:} In light of the wide minima folklore, one may want to find techniques to modify the sojourn time of SGD at each attraction field. 
Theorem \ref{theorem first exit time} suggests that the order of the first exit time (w.r.t. learning rate $\eta$) is directly controlled by the gradient clipping threshold $b$. 
Recall that for an attraction field with minimum jump number $l^*$, Theorem \ref{theorem first exit time} tells us the exit time from this attraction field is roughly of order $ (1/\eta)^{ 1+(\alpha-1)l^* }$. Given the width of the attraction field, its minimum jump number $l^*$ is dictated by gradient clipping threshold $b$. 
Therefore, gradient clipping provides us with a very systematic method to control the exit time of each attraction field. 
For instance, given clipping threshold $b$, the exit time from an attraction field with width less than $b$ is of order $(1/\eta)^\alpha$, while the exit time from one larger than $b$ is at least $(1/\eta)^{2\alpha - 1}$, which dominates the exit time from smaller ones.


\textbf{The role of structural properties of $\mathcal{G}$ and $f$:} Recall that in order for Theorem~\ref{corollary irreducible case} to apply, the irreducibility of $\mathcal{G}$ is required. 
Along with the choice of $b$, the geometry of function $f$ is a deciding factor of the irreducibility. 
For instance, we say that $\mathcal{G}$ is \emph{symmetric} if for any attraction field $\Omega_i$ such that $i = 2,3,\cdots,n_\text{min} - 1$ (so that $\Omega_i$ is not the leftmost or rightmost one at the boundary), we have $q_{i,i-1}>0, q_{i,i+1} > 0$. 
One can see that $\mathcal{G}$ is symmetric if and only if, for any $i = 2,3,\cdots,n_\text{min} - 1$, $|s_i - m_i| \vee |m_i - s_{i-1}| <l^*_i b$, and symmetry is a sufficient condition for the irreducibility of $\mathcal{G}$. 
The graph illustrated in Figure~\ref{fig typical transition graph}(Middle) is symmetric, while the one in Figure~\ref{fig typical transition graph}(Right) is not. 
As the name suggests, in the $\mathbb{R}^1$ case the symmetry of $\mathcal{G}$ is more likely to hold if the shape of attraction fields in $f$ is also nearly symmetric around its local minimum. 
If not, the symmetry (as well as irreducibility) of $\mathcal{G}$ can be violated as illustrated in Figure~\ref{fig typical transition graph}, especially when a small gradient clipping threshold $b$ is used. 

Generally speaking, our results imply that, even with the truncated heavy-tailed noises, the function $f$ needs to satisfy certain regularity conditions to ensure that SGD iterates avoid undesirable minima. 
This is consistent with the observations in \cite{li2018visualizing}: the deep neural nets that are more trainable with SGD tend to have a much more regular structure in terms of the number and shape of local minima.

\textbf{Heavy-tailed SGD without gradient clipping:}
It is worth mentioning that our results also characterize the dynamics of heavy-tailed SGDs without gradient clipping. For instance, since the reflection operation at $\pm L$ restricts the iterates on the compact set $[-L,L]$, if we use a truncation threshold $b$ that is large than $2L$, then any SGD update that moves larger than $b$ will definitely be reflected at $\pm L$. Therefore, the dynamics are identical to that of the following iterates without gradient clipping:
\begin{align}
    X^{\eta, \text{unclipped} }_n = \varphi_L\Big( X^{\eta, \text{unclipped} }_{n-1} - \eta f^\prime(X^{\eta, \text{unclipped} }_{n-1}) + \eta Z_n\Big). \label{def SGD unclipped}
\end{align}
The next result follows immediately from Theorem \ref{theorem first exit time} and \ref{thm main paper MC GP absorbing case}.

\begin{corollary} \label{corollary unclipped case}
There exist constants $q_i > 0\ \forall i$, $q_{i,j}>0\ \forall j \neq i$ such that the following claims hold for any $i$ and any $x \in \Omega_i$.
\begin{enumerate}[label = \arabic*)]
    \item Under $\mathbb{P}_x$, $q_i H(1/\eta) \sigma_i(\eta)$ converges in distribution to an Exponential random variable with rate 1 as $\eta \downarrow 0$;
    \item For any  $j = 1,2,\cdots,n_\text{min}$ such that $j \neq i$, $$ \lim_{\eta \downarrow 0} \mathbb{P}_x( X^\eta_{ \sigma_i(\eta) } \in \Omega_j ) = q_{i,j}/q_{i}.$$
    \item Let $Y$ be a continuous-time Markov chain on $\{m_1,\cdots,m_{n_\text{min}}\}$ with generator matrix $Q$ parametrized by $Q_{i,i} = -q_i, Q_{i,j} = q_{i,j}$. Then
    \begin{align*}
    X^{\eta, \text{unclipped} }_{\floor{ t/H(1/\eta) }}(x) \rightarrow Y_t(m_i)\ \ \ \text{as }\eta \downarrow 0
\end{align*}
in the sense of finite-dimensional distributions.
\end{enumerate}
\end{corollary}
At first glance, Corollary \ref{corollary unclipped case} may seem similar to the results in \cite{csimcsekli2019heavy} and \cite{pavlyukevich2007cooling}. 
However, the object studied in \cite{csimcsekli2019heavy,pavlyukevich2007cooling} is different: they study the following Langevin-type stochastic differential equation (SDE) driven by an $\alpha$-stable L\'evy process $L_t$ with scaling factor $\eta > 0$:
      $$dY^{\eta}_t = -f^\prime(Y^{\eta}_{t-}) dt + \eta dL_t.$$
In particular, \cite{pavlyukevich2007cooling} studies the metastability of $Y_t^\eta$ and concludes that as $\eta \downarrow 0$, the first exit time and global dynamics of $Y_t^\eta$ admit a similar characterization as described in our Theorem \ref{corollary unclipped case}.
Then Theorem 4 in \cite{csimcsekli2019heavy} argues that when the learning rate $\eta$ is sufficiently small, the distribution of the first exit time of the SGD $X_n^\eta$ and that of the L\'evy-driven Langevin SDE $Y^\eta_t$ are similar. However, the analysis of \cite{csimcsekli2019heavy}  hinges critically on the assumption that $L^\alpha_t$ is symmetric and $\alpha$-stable. 
While such an assumption is convenient for their analysis, it is a strong assumption. 
It implies that the gradient noise distribution belongs to a very specific parametric family and excludes all the other heavy-tailed distributions. In particular, the assumption precludes analysis of any heavy tails with finite variance.
On the contrary, our work directly analyzes the SGD $X_n^\eta$ and reveals the heavy-tailed SGD dynamics at a much greater level of generality. Specifically, we allow the noise to have general regularly varying distributions with arbitrary tail index---which includes $\alpha$-stable distributions as a (very) special case---and extend the characterization of global dynamics of heavy-tailed SGD to the adaptive versions of SGD where gradient clipping is applied.


\counterwithin{equation}{section}

\section{Geometric Characterization $l^*$ and Existing Sharpness Metrics}\label{sec:appendix-sharpness}
Throughout the paper, we have been using the term \textit{sharp minima} when describing the elimination effect of truncated heavy-tailed SGD,
while appealing to $l^*$, the minimum number of jumps requires for escape defined in \cref{def main paper minimum jump l star i},
when rigorously presenting our theoretical results about dynamics of truncated heavy-tailed SGD.
Despite the possible ambiguity of its use in the main paper, the terminology \textit{sharpness} is meant to familiarize our rather technically involved notion of how heavy-tailed SGD behave.
To resolve the potential confusions, in this section we provide a detailed discussion on the relationship between our geometric characterization $l^*$ and the existing sharpness metrics.

In light of recent discovery that all local minima might be global minima for over-parametrized deep neural nets (see \cite{li2018over}),
one popular explanation for the generalization gap between different local minima (for instance, between solutions found by GD and SGD) is that the geometry around the solution is closely related to its performance in the test setting.
After the seminal work by \cite{keskar2016large},
a myriad of attempts have been made to provide empirical or theoretical evidences for the link between the sharpness of a local minimum and its generalization performance;
See, for instance, \cite{xie2020diffusion,jiang2019fantastic,he2019asymmetric,zhou2020towards}.
In summary, there exist at least four different classes of sharpness metrics among the current literature:
\begin{enumerate}[label=(\alph*)]
    \item spectral norm of the Hessian at the local minimum, one surrogates of which is the \textit{maximal sharpness} type of metrics; a variant used in \cite{xie2020diffusion} based on corresponding large deviation theory is the eigenvalue of the Hessian along certain escape directions; 
    \item \textit{expected sharpness} type of metrics that evaluate the general fluctuation of the loss function within a $L_p$ ball of centered at the local minimum (see \cite{pmlr-v97-zhu19e});
    \item complexity metrics based on PAC-Bayes theory (see, for example, \cite{neyshabur2017exploring});
    \item  geometric property of a domain or the entire attraction field instead of the local minimum itself; see the mass of a Radon measure $m(W)$ over the domain $W$ in Theorem 1 in \cite{zhou2020towards}.
\end{enumerate}
Under such taxonomy,
the $l^*$ characterization proposed in this paper falls into category (d).
Indeed, it describes the minimum effort required for escaping the domain rather than the geometry merely around a certain neighborhood of the local minimum.
Here we have two remarks on the definition of $l^*$.
First, for a given loss landscape, this quantity $l^*$ is dictated by the gradient clipping threshold $b > 0$.
In fact, this concept is tailored for the gradient clipping case, as in the unclipped case with heavy-tailed SGD, for any attraction field $\Omega_i$ we always have $l^*_i = 1$.
Second, it is analogous to the term $m(W)$ in \cite{zhou2020towards} in the sense that
it reflects the \textit{volume} of the attraction field (when compared to a given threshold $b$).
Therefore, we stress that a more precise interpretation of $l^*$, as well as other metrics in class (d),
is that they characterize how \emph{wide} or \emph{narrow} each minimum (attraction field) is,
and a more a clear description of our theoretical results Theorem \ref{theorem first exit time}-\ref{corollary irreducible case}
is that truncated heavy-tailed SGD effectively avoids all the \emph{narrow} minima when learning rate is small enough.

While one can intuitively see that wide minima are more likely to be flat ones, we acknowledge the existence of counterexamples where a sharp minimum lies in a wide attraction field.
Nevertheless, as the variety of definitions for sharpness keeps growing in existing literature,
it becomes rather unlikely for one geometric characterization to always agree with (or be equivalent to) other existing approaches.
More importantly, the aim of this paper is not to establish $l^*$ as the orthodox geometric property when studying generalization gap.
Instead, as demonstrated in Table \ref{table ablation study}, \ref{table ablation study more sharpness metrics} and \ref{table appendix data augmentation training},
in modern deep learning tasks the truncated heavy-tailed method (which prefers solutions with high $l^*$ as indicated by our theoretical results) obtains solutions that generalize better and exhibit \textit{flatter} geometry when evaluated under different sharpness metrics.
This observation is well aligned with the recent large-scale empirical study in \cite{jiang2019fantastic},
suggesting that in typical training setting it is beneficial to find wider minima (characterized by $l^*$) in order to achieve a flatter geometry and better generalization performance. 

\section{On the Relationship Between Sharpness and Generalization}\label{sec:appendix-sharpness-generalization}

In principle, a sharp minimum does not necessarily lead to poor generalization in the sense that it is possible to construct pathological counterexamples in theory or in specific experiment settings;
see, for example, \cite{dinh2017sharp,neyshabur2017exploring}.
However, the loss geometry that arises in practice seems to exhibit a strong correlation between the sharpness and the test error.
For instance, sharpness-aware optimization methods \citep{foret2020sharpness,kwon2021asam} improve generalization across various tasks and achieve state-of-the-art performance on the CIFAR dataset.
Also, one of the authors of the aforementioned paper \cite{neyshabur2017exploring} continued investigating different complexity measures and conducted much larger scale experiments in \cite{jiang2019fantastic}.
Among more than 40 complexity measures from theoretical and empirical studies in the literature tested in \cite{jiang2019fantastic},
the sharpness metrics are the top ones for predicting the generalization performance.
In our own experiments, we also show that sharpness and test accuracy are highly correlated. 

\section{Stability-driven Analyses on SGD}\label{sec:appendix-stability-literature}

\cite{wu2018sgd} takes the perspective of linear stability and establishes conditions for SGD to be attracted to or avoid certain solutions based on learning rate, batch size, and the concept of non-uniformity of local minima. Inspired by \citep{wu2018sgd}, \citep{jastrzebski2020break} analyzes the trajectory-wise stability of SGD and found that a “break-even point” partitions the training procedure into two phases: the implicit regularization effects in SGD due to a larger learning rate becomes visible in the second phase after this “break-even point”. Similarly, \citep{cohen2021gradient} reports that typical GD trajectories in standard image classification tasks can be partitioned into two phases: in the first “progressive sharpening” phase the sharpness of the Hessian monotonically increases, while in the second phase we observe the “edge of stability” regime where the sharpness of Hessian hovers slightly above the critical value $2/\eta$ and the training loss slowly decrease in an oscillating fashion.

Compared to the aforementioned stability analyses, one major difference of our heavy-tailed regime is that, even under gradient clipping, the trajectory cannot be partitioned into such “phases” and the traditional sense of stability around certain local minima is nullified by the constant basin hopping and exploration behaviors under heavy-tailed noises. In particular, the polynomial order of the exit time in heavy-tailed SGD dictates that, when compared to exponentially long exit time in vanilla SGD (see \cite{xie2020diffusion}), the iterates won’t become “stabilized” and keep staying around a certain region in a typical training procedure. Instead, our work characterizes the metastability of the truncated heavy-tailed SGD (i.e. constantly transitioning between different wide minima) that cannot be observed in GD or light-tailed SGD.
More importantly, compared to the aforementioned stability related analyses,
our work provides a very tight characterization of the global dynamics and distributions of the entire sample path.

\section{Existing Analyses on Heavy-tailed Phenomena in SGD}
\label{sec:appendix heavy tailed phenomena in SGD}

We start with one clarification: 
what has been established in \cite{hodgkinson2020multiplicative,gurbuzbalaban2021heavy,wojtowytsch2021stochastic,mori2021logarithmic} is that the stationary distribution of the SGD (or the corresponding continuous-time SDE) can be heavy-tailed, and the heavy tailes therein can depend on various training hyperparameters. 
This, however, does not imply that the gradient noise is heavy-tailed. In particular, the heavy-tail index (or heavy-tailedness itself) of the gradient noise will not be affected by hyperparameters such as the learning rate.

To be specific, \cite{hodgkinson2020multiplicative,gurbuzbalaban2021heavy} show that heavy tails can arise in the stationary distribution of SGD through multiplicative dynamics, and the tail index of the resulting stationary distribution can be characterized by the learning rates and the magnitude of noises (through batch size).
Similarly, the SDEs studied in \cite{wojtowytsch2021stochastic,mori2021logarithmic} are both driven by light-tailed (Gaussian) perturbations, and the authors show that even under light-tailed perturbations, heavy tails can arise in the stationary distribution through multiplicative dynamics, and the tail index of the resulting stationary distribution can be characterized by the learning rates. 
In comparison, we focus on the impact of the heavy-tails in the gradient noise and the truncation scheme on the global dynamics of SGD. 
Our results show that, under heavy-tailed noises, the power-law index $\alpha$
 for the tail in noises also characterizes the first exit time and global dynamics of SGD with no dependence on other training hyperparameters (at least in the asymptotic scheme).
In other words, when the driving force of the dynamics is heavy-tailed, 
its ``heavy-tailedness'' (i.e., the same index $\alpha$) characterizes the ``heavy-tailedness'' of the entire SGD without dependency on the other hyperparameters,
and its effect on the entire SGD path (rather than just the stationary distributions) admits a very clear and tight expression, as characterized in theoretical results in this work.

\counterwithin{equation}{section}
\section{Proof of Theorem \ref{theorem first exit time}} \label{section proof thm 1}
This section proves Theorem~\ref{theorem first exit time}. We first start with providing the definitions of $q_i$ and $q_{i,j}$ that appear in the statement of the theorem. 
Let $\textbf{Leb}_+$ denote the Lebesgue measure restricted on $[0,\infty)$, and define a (Borel) measure $\nu_\alpha$ on $\mathbb{R}\symbol{92}\{0\}$ as follows:
$$\nu_\alpha(dx) = \mathbbm{1}\{x > 0\}\frac{\alpha p_+}{x^{\alpha + 1}} + \mathbbm{1}\{x <0\}\frac{\alpha p_-}{|x|^{\alpha + 1}} $$
where $\alpha$, $p_-$, and $p_+$ are constants in Assumption 2. 
Define a Borel measure $\mu_i$ on $\mathbb{R}^{l^*_i}\times \Big(\mathbb{R}_{+}\Big)^{l^*_i - 1}$ as the product measure
\begin{align}
    \mu_i= (\nu_\alpha)^{l^*_i}\times(\textbf{Leb}_+)^{l^*_i - 1}. \label{defMuMeasure i}
\end{align}
We also define mappings $h_i$ as follows. For a real sequence $\textbf{w} = ( w_1,w_2,\cdots,w_{l^*_i})$ and a positive real number sequence $\textbf{t} = (t^\prime_j)_{j = 2}^{l^*_i}$, define $t_1 = t'_1 = 0$ and $t_j = t_1^\prime + t_2^\prime + \cdots + t_j^\prime$ for $j = 2,\cdots,l^*_i$. 
Now we define a path $\hat{\textbf{x}}:[0,\infty)\mapsto \mathbb{R}$ as the solution to the following ODE with jumps:
\begin{align}
    \hat{\textbf{x}}(0) & = \varphi_L\big(m_i +\varphi_b( w_1)\big); \label{defOdeJumpClipping1} \\
    \frac{ d \hat{\textbf{x}}(t) }{ d t } & = -f^\prime(\hat{\textbf{x}}(t)  ), \ \ \forall t\in[t_{j-1},t_j),\ \ \forall j = 2,\cdots,l^*_i;\label{defOdeJumpClipping2} \\
    \hat{\textbf{x}}(t_j) & = \varphi_L\big( \hat{\textbf{x}}(t_j - ) + \varphi_b(w_j)\big), \ \ \forall j = 2,\cdots,l^*_i.\label{defOdeJumpClipping3}
\end{align}
Now we define the mappings $h_i:\mathbb{R}^{l^*_i}\times \Big(\mathbb{R}_+\Big)^{l^*_i - 1}\mapsto \mathbb{R}$ as
$$h_i( \textbf{w},\textbf{t} ) = \hat{\textbf{x}}(t_{l^*_i}).$$
It is easy to see that $h_i$'s are continuous mappings. With these mappings, we define the following sets:
\begin{align}
    E_i \delequal{}  \{(\textbf{w},\textbf{t}) \subseteq \mathbb{R}^{l^*_i}\times \Big(\mathbb{R}_{+}\Big)^{l^*_i - 1}: h_i(\textbf{w},\textbf{t}) \notin \Omega_i\};\label{def set E_i}\\
    E_{i,j} \delequal{}  \{(\textbf{w},\textbf{t}) \subseteq \mathbb{R}^{l^*_i}\times \Big(\mathbb{R}_{+}\Big)^{l^*_i - 1}: h_i(\textbf{w},\textbf{t}) \in \Omega_j\}.\label{def set E_ij}
\end{align}
Lastly, the constant $q_i$ and $q_{i,j}$ are defined as follows:
\begin{align}
    q_i & = \mu_i(E_i),\ \  q_{i,j} = \mu_i(E_{i,j})\ \ \forall i \neq j. \label{def q i q ij}
\end{align}

Before we move on to the proof of Theorem~\ref{theorem first exit time}, we add a few remarks regarding the intuition behind it.
\begin{itemize}
\item 

Suppose that $X_n^\eta$ is started at the $i$\textsuperscript{th} local minimum $m_i$ of $f$, and consider the behavior of $X_n^\eta$ over the time period $H_1 \triangleq \{1,\ldots, \lceil t/\eta \rceil \}$ for a sufficiently large $t$.
The heavy-tailed large deviations theory \cite{rhee2019sample} and a heuristic application of the contraction principle implies that the path of $X_{\lceil n /\eta\rceil}^\eta$ over this period will converge to the gradient flow of $f$, and the event that $X_{\lceil n /\eta\rceil}^\eta$ escapes $\Omega_i$ within this period is a (heavy-tailed) rare event. 
This means that the probability of such an event is of order $(1/\eta)^{(\alpha-1)l_i^*}$. 
Moreover, whenever it happens, it is almost always because $X^\eta_n$ is shaken by \emph{exactly $l^*_i$ large gradient noises of size $\mathcal O(1/\eta)$}, which translates to $l^*_i$ jumps in $X_{\lceil n /\eta\rceil}^\eta$'s path, while the rest of its path closely resemble the deterministic gradient flow. 
Moreover, conditional on the event that $X_n^\eta$ fails to escape from $\Omega_i$ within this period, the endpoint of the path is most likely to be close to the local minima, i.e.,  $X_{\lceil t/\eta \rceil}^\eta \approx m_i$. 
This suggests that over the next time period $H_2 \triangleq \{\lceil t/\eta \rceil+1, \lceil t/\eta \rceil+2,\ldots, 2\lceil t/\eta \rceil \}$ of length $\lceil t/\eta \rceil$, $X_n^\eta$ will behave similarly to its behavior over the first period $H_1$. 
The same argument applies to the subsequent periods $H_3,H_4,\ldots$ as well. 
Therefore, over each time period of length $\lceil t/\eta \rceil$, there is $(1/\eta)^{(\alpha-1)l_i^*}$ probability of exit.  
In view of this, the exit time should be of order $(1/\eta)^{ 1+(\alpha - 1)l^*_i }$ and resemble an exponential distribution when scaled properly.

\item Part (i) of Theorem \ref{theorem first exit time} builds on this intuition and rigorously prove that the first exit time is indeed roughly of order $1/\lambda_i(\eta) \approx (1/\eta)^{ 1+(\alpha - 1)l^*_i }$ and resembles an exponential distribution. 
\item Given this, one would expect that $X^\eta_{\sigma_i(\eta)}$, the location of SGD right at the time of exit, will hardly ever be farther than $l^*_i b$ away from $m_i$: the length of each update is clipped by $b$, and there will most likely be only $l^*_i$ large SGD steps during this successful attempt. 
Indeed, from the definition of $q_{i,j}$'s above, one can see that $q_{i,j} > 0$ if and only if $\inf_{y \in \Omega_j} |y - m_i| < l^*_i b$.

\end{itemize}
Summarizing the three bullet points here, we see that the minimum number of jumps $l^*_i$ dictates \emph{how} heavy-tailed SGD escapes an attraction field, \emph{where} the SGD lands on upon its exit, and \emph{when} the exit occurs.\\

Now we are ready to start proving Theorem~\ref{theorem first exit time}. First, note that Assumption 1 implies the following:
\begin{itemize}
    \item There exist $c_0 > 0, \epsilon_0 \in (0,1)$ such that for any $x \in \{m_1,s_1,\cdots,s_{n_{\text{min}} - 1} , m_{n_{\text{min}}}\}, |y - x| < \epsilon_0$,
    \begin{align}
        |f^\prime(y)|>c_0|y-x|, \label{assumption detailed function f at critical point}
    \end{align}
    and for any $y \in [-L,L]$ such that $|y-x| \geq \epsilon_0$ for all $x \in \{m_1,s_1,\cdots,s_{n_{\text{min}} - 1} , m_{n_{\text{min}}}\}$, we have 
    \begin{align}
        |f^\prime(y)|>c_0; \label{assumption detailed function f at critical point 2} 
    \end{align}
    \item There exist constants $L \in (0,\infty), M \in (0,\infty)$ such that $|m_0| < L, |m_{n_{\text{min}}}|<L$, and (for any $x \in [-L,L]$)
    \begin{align}
        |f^\prime(x)| \leq M,\ |f^{\prime\prime}(x)|\leq M. \label{assumption detailed function f at critical point 3} 
    \end{align}
\end{itemize}
Recall that for $c>0$, the truncation operator was defined as 
\begin{align}
    \varphi_c(w) \triangleq \varphi(w,c) \delequal{} (w\wedge c)\vee(-c),\ \ \ \forall w \in \mathbb{R}, \label{defTruncationClippingOperator}
\end{align}
and the SGD iterates were defined as
\begin{align}
    X^{\eta}_n = \varphi_L\Big(X^{\eta}_n - \varphi_b\big(\eta (f^\prime(X^{\eta}_n) - Z_{n+1})\big)\Big).
    \label{def SGD iterates with gradient clipping}
\end{align}
Here $\eta > 0$ is the learning rate (step length) and $b>0$ is the gradient clipping threshold. 
Also, recall that for any $k \in [n_\text{min}]$, we let
\begin{align*}
    \sigma_k(\eta) \delequal{} \min\{n \geq 0:\ X^\eta_n \notin \Omega_k \}
\end{align*}
to be the time that $X^\eta_n$ exists from the $k-$th attraction field $\Omega_j$. Meanwhile, given the gradient clipping threshold $b$, recall that
\begin{align}
    r_i & \delequal{}\min\{ m_i - s_{i-1}, s_i - m_i \}, \label{def distance attraction field i}\\
    l^*_i & \delequal{} \ceil{ r_i/b }. \label{def minimum jump l star i}
\end{align}
Intuitively speaking, $l^*_i$ tells us the minimum number of jumps with size no larger than $b$ required in order to escape the attraction field if we start from the local minimum of this attraction field $\Omega_i$. Lastly, recall the definition of $H(\cdot) = \mathbb{P}(|Z_1| > \cdot)$ and
\begin{align*}
    \lambda_i(\eta) \delequal{} H(1/\eta)\Big( \frac{H(1/\eta)}{\eta} \Big)^{l^*_i - 1}.
\end{align*}



The proof of Theorem \ref{theorem first exit time} hinges on the following two lemmas that characterize the behavior of $X^\eta_n$ in two different phases respectively. Let $k \in [n_\text{min}]$ and $x \in \Omega_k$. 
We consider the SGD iterates initialized at $X^\eta_0 = x$. In the first phase, the SGD iterates return to $[m_k-2\epsilon,m_k + 2\epsilon]$, a small neighborhood of the local minimizer in attraction field $\Omega_k$; in other words, it ends at
\begin{align}
    T_\text{return}^{(k)}(\eta,\epsilon) \delequal{} \min\{ n \geq 0:\ X^\eta_n \in [m_k - 2\epsilon,m_k + 2\epsilon] \}. \label{def T return k}
\end{align}
During this phase, we show that for all learning rate $\eta$ that is sufficiently small, it is almost always the case that $X^\eta_n$ would quickly return to $[m_k - 2\epsilon,m_k + 2\epsilon]$, and it never leaves $\Omega_k$ before $T_\text{return}^{(k)}$.

\begin{lemma} \label{lemma first exit key lemma 1}
There exists some $c \in (0,\infty)$ such that for any $k \in [n_\text{min}]$, the following claim holds for all $\epsilon > 0$ small enough:
\begin{align*}
    \lim_{\eta \downarrow 0}\inf_{y\in[-L,L]:\ y \in (s_{k-1} + \epsilon,s_k - \epsilon) }\mathbb{P}_y\Big( X^\eta_n \in \Omega_k\ \forall n \in \big[T^{(k)}_\text{return}(\eta,\epsilon)\big],\  T^{(k)}_\text{return}(\eta,\epsilon) \leq \frac{c\log(1/\epsilon)}{\eta} \Big) = 1.
\end{align*}
\end{lemma}

During the second phase, $X^\eta_n$ starts from somewhere in $[m_k - 2\epsilon,m_k + 2\epsilon]$ and tries to escape from $\Omega_k$, meaning that the phase ends at $\sigma_k(\eta)$. During this phase, we show that the distributions of the first exit time $\sigma_k(\eta)$ and the location $X^\eta_{ \sigma_k(\eta) }$ do converge to the ones described in Theorem \ref{theorem first exit time} as learning rate $\eta$ tends to $0$.

\begin{lemma}\label{lemma first exit key lemma 2}
There exist constants $q_i > 0\ \forall i\in [n_\text{min}]$ and $q_{i,j} \geq 0\ \forall j \in [n_\text{min}]$ with $i,j\neq i$ such that the following claim holds: given any $C>0, u>0$ and any $k,l \in [n_\text{min}]$ with $k \neq l$, we have
\begin{align}
    \limsup_{\eta \downarrow 0}\sup_{x \in [-L,L],\ x \in (m_k - 2\epsilon,m_k + 2\epsilon)}\mathbb{P}_x\Big( q_k\lambda_k(\eta)\sigma_k(\eta) > u \Big) & \leq C + \exp\big( -(1-C)u \big) \label{goal 1 key lemma first exit 2}
    \\
    \liminf_{\eta \downarrow 0}\inf_{x \in [-L,L],\ x \in (m_k - 2\epsilon,m_k + 2\epsilon)}\mathbb{P}_x\Big( q_k\lambda_k(\eta)\sigma_k(\eta) > u \Big) & \geq -C + \exp\big( -(1+C)u \big)\label{goal 2 key lemma first exit 2}
    \\
    \limsup_{\eta \downarrow 0}\sup_{x \in [-L,L],\ x \in (m_k - 2\epsilon,m_k + 2\epsilon)}\mathbb{P}_x\Big( X^\eta_{\sigma_k(\eta)} \in \Omega_l \Big) & \leq \frac{q_{k,l} + C }{ q_k}\label{goal 3 key lemma first exit 2}
    \\
    \liminf_{\eta \downarrow 0}\inf_{x \in [-L,L],\ x \in (m_k - 2\epsilon,m_k + 2\epsilon)}\mathbb{P}_x\Big( X^\eta_{\sigma_k(\eta)} \in \Omega_l \Big) & \geq \frac{q_{k,l} - C }{ q_k}\label{goal 4 key lemma first exit 2}
\end{align}
for all $\epsilon>0$ that are sufficiently small.
\end{lemma}

Now we are ready to show Theorem \ref{theorem first exit time}.

\begin{proof}[Proof of Theorem \ref{theorem first exit time}]
Fix some $k \in [n_\text{min}]$ and $x \in \Omega_k \cap [-L,L]$. Let $q_k$ and $q_{k,l}$ be the constants in Lemma \ref{lemma first exit key lemma 2}.

We first prove the weak convergence claim in Theorem \ref{theorem first exit time}(i). Arbitrarily choose some $u > 0$ and $C \in (0,1)$. It suffices to show that
\begin{align*}
    \limsup_{\eta \downarrow 0}\mathbb{P}_x( q_k\lambda_k(\eta)\sigma_k(\eta) > u ) & \leq 2C + \exp\big( -(1-C)u \big),
    \\
    \liminf_{\eta \downarrow 0}\mathbb{P}_x( q_k\lambda_k(\eta)\sigma_k(\eta) > u ) & \geq (1-C)\Big(-C + \exp\big( -(1+C)u \big)\Big).
\end{align*}
Recall the definition of the stopping time $T^{(k)}_\text{return}$ in \cref{def T return k}. Define event
\begin{align*}
    A_k(\eta,\epsilon) \delequal{} \Big\{ X^\eta_n \in \Omega_k\ \forall n \in \big[T^{(k)}_\text{return}(\eta,\epsilon)\big],\  T^{(k)}_\text{return}(\eta,\epsilon) \leq \frac{c\log(1/\epsilon)}{\eta} \Big\}
\end{align*}
where $c < \infty$ is the constant in Lemma \ref{lemma first exit key lemma 1}. First, since $x \in \Omega_k = (s_{k-1},s_k)$, it holds for all $\epsilon>0$ small enough that $x \in (s_{k-1} + \epsilon, s_k - \epsilon)$. Next, one can find some $\epsilon > 0$ such that
\begin{itemize}
    \item (Due to Lemma \ref{lemma first exit key lemma 1})
    \begin{align*}
        \mathbb{P}_x\big( (A_k(\eta,\epsilon))^c \big) \leq C\ \ \forall \eta \text{ sufficiently small};
    \end{align*}
    
    \item (Due to \cref{goal 1 key lemma first exit 2}\cref{goal 2 key lemma first exit 2} and strong Markov property) For all $\eta$ sufficiently small,
    \begin{align*}
        \mathbb{P}_x\Big( q_k\lambda_k(\eta)\big(\sigma(\eta) - T^{(k)}_\text{return}(\eta,\epsilon)\big) > (1-C)u \ \Big|\ A_k(\eta,\epsilon)  \Big) & \leq C + \exp\big( -(1-C)u \big),
        \\
         \mathbb{P}_x\Big( q_k\lambda_k(\eta)\big(\sigma(\eta) - T^{(k)}_\text{return}(\eta,\epsilon)\big) > u \ \Big|\ A_k(\eta,\epsilon) \Big) & \geq -C + \exp\big( -(1+C)u \big).
    \end{align*}
\end{itemize}
Fix such $\epsilon$. Lastly, for this fixed $\epsilon$, due to $\lambda_k \in \RV_{ -1 - l^*_k(\alpha - 1) }(\eta)$ and $\alpha > 1$, we have $q_k\lambda_k(\eta)\cdot \frac{c \log(1/\epsilon)}{\eta} < Cu$ for all $\eta$ sufficiently small. In summary, for all $\eta$ sufficiently small, we have
\begin{align*}
    & \mathbb{P}_x( q_k\lambda_k(\eta)\sigma_k(\eta) > u ) \\
    \leq & \mathbb{P}_x\big( (A_k(\eta,\epsilon))^c\big) + \mathbb{P}_x\Big( \big\{ q_k\lambda_k(\eta)\sigma_k(\eta) > u \big\} \cap  A_k(\eta,\epsilon) \Big)
    \\
    \leq & C + \mathbb{P}_x\Big( \big\{ q_k\lambda_k(\eta)\sigma_k(\eta) > u \big\} \cap  A_k(\eta,\epsilon) \Big)
    \\
    = & C +\mathbb{P}_x\Big( q_k\lambda_k(\eta)\big(\sigma(\eta) - T^{(k)}_\text{return}(\eta,\epsilon)\big) > u - q_k\lambda_k(\eta)T^{(k)}_\text{return}(\eta,\epsilon)\ \Big|\ A_k(\eta,\epsilon)  \Big)\cdot\mathbb{P}_x( A_k(\eta,\epsilon)  )
    \\
    \leq & C + \mathbb{P}_x\Big( q_k\lambda_k(\eta)\big(\sigma(\eta) - T^{(k)}_\text{return}(\eta,\epsilon)\big) > (1-C)u\ \Big|\ A_k(\eta,\epsilon)  \Big)
    \\
    \leq & 2C + \exp\big(-(1-C)u\big)
\end{align*}
and
\begin{align*}
    & \mathbb{P}_x( q_k\lambda_k(\eta)\sigma_k(\eta) > u ) \\
    \geq & \mathbb{P}_x\Big( \big\{ q_k\lambda_k(\eta)\sigma_k(\eta) > u \big\} \cap  A_k(\eta,\epsilon) \Big)
    \\
    = & \mathbb{P}_x\Big( q_k\lambda_k(\eta)\big(\sigma(\eta) - T^{(k)}_\text{return}(\eta,\epsilon)\big) > u - q_k\lambda_k(\eta)T^{(k)}_\text{return}(\eta,\epsilon)\ \Big|\ A_k(\eta,\epsilon)  \Big)\cdot\mathbb{P}_x( A_k(\eta,\epsilon)  )
    \\
    \geq & \mathbb{P}_x\Big( q_k\lambda_k(\eta)\big(\sigma(\eta) - T^{(k)}_\text{return}(\eta,\epsilon)\big) > u - q_k\lambda_k(\eta)T^{(k)}_\text{return}(\eta,\epsilon)\ \Big|\ A_k(\eta,\epsilon)  \Big)\cdot(1-C)
    \\
    \geq & \mathbb{P}_x\Big( q_k\lambda_k(\eta)\big(\sigma(\eta) - T^{(k)}_\text{return}(\eta,\epsilon)\big) > u \ \Big|\ A_k(\eta,\epsilon)  \Big)\cdot(1-C)
    \\
    \geq & (1-C)\Big( -C + \exp\big(-(1+C)u\big) \Big)
\end{align*}
so this concludes the proof for Theorem \ref{theorem first exit time}(i).

In order to prove claims in Theorem \ref{theorem first exit time}(ii), we first observe that on event $A_k(\eta,\epsilon)$ we must have $\sigma(\eta) > T^{(k)}_\text{return}(\eta,\epsilon)$. Next, arbitrarily choose some $C\in (0,1)$, and note that it suffices to show that $ \mathbb{P}_x( X^\eta_{ \sigma(\eta)} \in \Omega_l  ) \in \Big( (1-C)\frac{q_{k,l} -C}{q_k}, C+\frac{q_{k,l} +C}{q_k}\Big)$ holds for all $\eta$ sufficiently small. Again, we can find $\epsilon>0$ such that
\begin{itemize}
    \item (Due to Lemma \ref{lemma first exit key lemma 1})
    \begin{align*}
        \mathbb{P}_x\big( (A_k(\eta,\epsilon))^c \big) \leq C\ \ \forall \eta \text{ sufficiently small};
    \end{align*}
    
    \item (Due to \cref{goal 3 key lemma first exit 2}\cref{goal 4 key lemma first exit 2} and strong Markov property) For all $\eta$ sufficiently small,
    \begin{align*}
        \frac{q_{k,l} - C }{ q_k } \leq \mathbb{P}_x\Big( X^\eta_{\sigma_k(\eta)} \in \Omega_l \ \Big|\ A_k(\eta,\epsilon) \Big) \leq \frac{q_{k,l} + C }{ q_k }.
    \end{align*}
\end{itemize}
In summary, for all $\eta$ sufficiently small, we have
\begin{align*}
    \mathbb{P}_x\Big( X^\eta_{ \sigma_k(\eta) } \in \Omega_l \Big) & \leq \mathbb{P}_x\big( (A_k(\eta,\epsilon))^c \big) + \mathbb{P}_x\Big( \big\{X^\eta_{ \sigma_k(\eta) } \in \Omega_l \big\} \cap A_k(\eta,\epsilon) \Big) 
    \\
    & \leq C + \mathbb{P}_x\Big( X^\eta_{ \sigma_k(\eta) } \in \Omega_l \ \Big|\  A_k(\eta,\epsilon) \Big)\mathbb{P}_x( A_k(\eta,\epsilon)  ) 
    \\
    & \leq C + \frac{q_{k,l} + C }{ q_k }
\end{align*}
and
\begin{align*}
    \mathbb{P}_x\Big( X^\eta_{ \sigma_k(\eta) } \in \Omega_l \Big) & \geq \mathbb{P}_x\Big( \big\{X^\eta_{ \sigma_k(\eta) } \in \Omega_l \big\} \cap A_k(\eta,\epsilon) \Big)
    \\
    & = \mathbb{P}_x\Big( X^\eta_{ \sigma_k(\eta) } \in \Omega_l \ \Big|\  A_k(\eta,\epsilon) \Big)\mathbb{P}_x( A_k(\eta,\epsilon)  ) 
    \\
    & \geq (1-C)\frac{ q_{k,l} - C }{ q_{k} }
\end{align*}
and this concludes the proof.
\end{proof}


The rest of this section is devoted to the proofs of Lemma~\ref{lemma first exit key lemma 1} and Lemma~\ref{lemma first exit key lemma 2}. Specifically, Lemma \ref{lemma first exit key lemma 1} is an immediate Corollary of Lemma \ref{lemma return to local minimum quickly}, the proof of which will be provided below. The proof of Lemma \ref{lemma first exit key lemma 2} can be found at the end of this section.

\subsection{Proofs of Lemma \ref{lemma first exit key lemma 1}, \ref{lemma first exit key lemma 2}}
The following three lemmas will be applied repeatedly throughout this section. The proofs are straightforward but provided in Section~\ref{section: proof of technical lemmas} for the sake of completeness.

\begin{lemma} \label{lemmaGeomDistTail}
Given two real functions $a:\mathbb{R}_+ \mapsto \mathbb{R}_+$, $b:\mathbb{R}_+ \mapsto \mathbb{R}_+$ such that $a(\epsilon) \downarrow 0, b(\epsilon) \downarrow 0$ as $\epsilon \downarrow 0$, and a family of geometric RVs $\{U(\epsilon): \epsilon > 0\}$ with success rate $a(\epsilon)$ (namely, $\mathbb{P}(U(\epsilon) > k) = (1 - a(\epsilon))^k$ for $k \in \mathbb{N}$), for any $c > 1$, there exists $\epsilon_0 > 0$ such that for any $\epsilon \in (0,\epsilon_0)$,
$$\exp\Big(-\frac{c\cdot a(\epsilon)}{b(\epsilon)}\Big) \leq \mathbb{P}\Big( U(\epsilon) > \frac{1}{b(\epsilon)} \Big) \leq \exp\Big(-\frac{a(\epsilon)}{c\cdot b(\epsilon)}\Big).$$
\end{lemma}

\begin{lemma}\label{lemmaGeomFront}
Given two real functions $a:\mathbb{R}_+ \mapsto \mathbb{R}_+$, $b:\mathbb{R}_+ \mapsto \mathbb{R}_+$ such that $a(\epsilon) \downarrow 0, b(\epsilon) \downarrow 0$ and $$a(\epsilon)/b(\epsilon)\rightarrow 0 $$ as $\epsilon \downarrow 0$, and a family of geometric RVs $\{U(\epsilon): \epsilon > 0\}$ with success rate $a(\epsilon)$ (namely, $\mathbb{P}(U(\epsilon) > k) = (1 - a(\epsilon))^k$ for $k \in \mathbb{N}$), for any $c > 1$ there exists some $\epsilon_0 > 0$ such that for any $\epsilon \in (0,\epsilon_0)$,
\begin{align*}
 a(\epsilon)/(c\cdot b(\epsilon))\leq \mathbb{P}( U(\epsilon) \leq 1/b(\epsilon) ) \leq c \cdot a(\epsilon)/b(\epsilon)
\end{align*}
\end{lemma}

\begin{lemma} \label{lemmaBasicGronwall}  Suppose that a function $g:E \mapsto \mathbb{R}$ (where $E$ is an open set in $\mathbb{R}^d$) is $g \in \mathcal{C}^2(E)$ and $\norm{ \nabla^2 g(\cdot) }\leq C$ on its domain $E$ for some constant $C<\infty$. 
For a finite integer $n$, a sequence of vectors $\{z_1,\cdots,z_n\}$ in $\mathbb{R}^d$, and vectors $x,\widetilde{x} \in E, \eta > 0$, consider two sequences $\{x_k\}_{k=0,\ldots, n}$ and $\{\widetilde{x}_k\}_{k=0,\ldots, n}$ constructed as follows: 
\begin{align*}
    x_0 & = x\\ 
    x_k & = x_{k - 1} + \eta \nabla g(x_{k - 1}) + \eta z_k \ \ \ \forall k = 1,2,\cdots,n \\
    \widetilde{x}_0 &= \widetilde{x} \\
    \widetilde{x}_k &= \widetilde{x}_{k - 1} + \eta \nabla g(\widetilde{x}_{k - 1}) \ \ \ \forall k = 1,2,\cdots,n
\end{align*}
If we have that the line segment from $x_k$ to $\widetilde{x}_k$ is contained in $E$ and $\eta\norm{ z_1 + \cdots + z_k} + \norm{x - \widetilde{x}}\leq \widetilde{c}$ for all $k = 1,2,\cdots,n$ for some $\widetilde{c} > 0$, then
\begin{align*}
    \norm{ x_k - \widetilde{x}_k} \leq \widetilde{c}\cdot\exp(\eta C k)\ \ \ \forall k = 1,2,\cdots,n.
\end{align*}

\end{lemma}

To facilitate the analysis below, we introduce some additional notations. 
First, we will group the noises $Z_n$ based on a threshold level $\delta > 0$: let us define
\begin{align}
    Z_{n}^{\leq \delta,\eta} \delequal{} Z_n\mathbbm{1}\{\eta|Z_n|\leq \delta\}, \label{defSmallJump_GradientClipping} \\
     Z_{n}^{> \delta,\eta} \delequal{} Z_n\mathbbm{1}\{\eta|Z_n|> \delta\}.\label{defLargeJump_GradientClipping}
\end{align}
The former are viewed as \textit{small noises} while the latter will be referred to as \textit{large noises} or \textit{large jumps}. 
Furthermore, for any $j \geq 1$, define the $j$\textsuperscript{th} arrival time and size of large jumps as 
\begin{align}
    T^\eta_j(\delta) & \delequal{} \min\{ n > T^\eta_{j-1}(\delta):\ \eta|Z_n| > \delta  \},\quad T_0^\eta(\delta) = 0 \label{defArrivalTime large jump} \\
    W^\eta_j(\delta) & \delequal{} Z_{T^\eta_j(\delta)}. \label{defSize large jump}
\end{align}
Next, for any $\epsilon > 0$, let $\Omega_i(\epsilon) = [m_i - \epsilon,m_i + \epsilon]$ be an $\epsilon-$neighborhood of the local minimum $m_i$, and $S_i(\epsilon) = [s_i-\epsilon,s_i + \epsilon]$ be an $\epsilon-$neighborhood of the local maximum $s_i$. 

For most part of this section, we will zoom in on one of the local minima $m_i$ and its attraction field $\Omega_i = (s_{i-1},s_i)$. Without loss of generality, we assume $m_i = 0$, and denote the attraction field as $\Omega = (s_-,s_+)$. (If $m_i$ happens to be the local minimum at the left or right boundary, then the attraction field is $[-L,s_+)$ or $(s_-,L]$ where the SGD iterates will be reflected at $\pm L$.) Henceforth we will drop the dependency on notation $i$ when referring to this specific attraction field until the very end of this section. Throughout the proof, the following (deterministic) dynamic systems will be used frequently as benchmark processes to indicate the most likely location of the SGD iterates. Specifically, given any $x \in \Omega$, we use $X_n(x)$ to indicate that the starting point is $x$, namely $X_0(x) = x$. Similarly, consider the following ODE $\textbf{x}^\eta(t;x)$ as
\begin{align}
    \textbf{x}^\eta(0;x) & = x;\label{def GD with rate eta 1} \\
    \frac{ d  \textbf{x}^\eta(t;x) }{dt} & = -\eta f^\prime\Big( \textbf{x}^\eta(t;x) \Big).\label{def GD with rate eta 2}
\end{align}
When we use update rate $\eta = 1$, we will drop the dependency of $\eta$ and simply use $\textbf{x}(t;x)$ to denote the process. 

Based on Assumption 3, we know the existence of some constant $\Bar{\epsilon} \in (0,\epsilon_0)$ (note that $\epsilon_0$ is the constant in \cref{assumption detailed function f at critical point}) such that
\begin{align}
    r & \delequal{}\min\{ - s_{-}, s_+ \}, \label{def distance attraction field} \\
    l^* & \delequal{} \ceil{r/b}, \label{def minimum jump l star} 
    \\
    (l^* - 1)b + 100l^* \Bar{\epsilon} &< r - 100l^* \Bar{\epsilon} \label{assumption multiple jump epsilon 0 constant 1} \\
    r + 100l^*\Bar{\epsilon} &< l^* b - 100l^* \Bar{\epsilon}. \label{assumption multiple jump epsilon 0 constant 2}
\end{align}
Here $r$ can be understood as the effective \textit{radius} of the said attraction field. Also, we fix such $\Bar{\epsilon}$ small enough so that (let $c_{-}^L = -f^\prime(-L), c_{+}^L = -f^\prime(-L)$), we have
\begin{align}
    0.9c_{-}^L\leq -f^\prime(x)\leq 1.1c_{-}^L\ \ &\forall x \in [-L,-L + 100\Bar{\epsilon}], \\
    0.9c_{+}^L\geq -f^\prime(x)\geq 1.1c_{+}^L\ \ &\forall x \in [L - 100\Bar{\epsilon},L]. \label{property at boundary neighborhood L}
\end{align}
Similar to the definition of ODE $\textbf{x}^\eta$, let us consider the following construction of ODE $\Tilde{\textbf{x}}^\eta$ that can be understood as $\textbf{x}^\eta$ perturbed by $l^*$ shocks. Specifically, consider a sequence or real numbers $0 = t_1 < t_2 < t_3 < \cdots < t_{l^*} $ and real numbers $w_1,\cdots,w_{l^*}$ where $|w_j| \leq b$ for each $j$. Let $\textbf{t} = (t_1,\cdots,t_{l^*}), \textbf{w} = (w_1,\cdots,w_{l^*})$. Based on these two sequences and rate $\eta > 0$, define $\Tilde{\textbf{x}}^\eta(t;x)$ as 
\begin{align}
    \Tilde{\textbf{x}}^\eta(0,x;\textbf{t},\textbf{w}) & = \varphi_L(x + \varphi_b(w_1) ); \label{def perturbed ODE 1} \\
    \frac{d \Tilde{\textbf{x}}^\eta(t,x;\textbf{t},\textbf{w}) }{ dt } & = -\eta f^\prime\big( \Tilde{\textbf{x}}^\eta(t,x;\textbf{t},\textbf{w}) \big)\ \ \ \forall t \notin \{t_1,t_2,\cdots,t_{l^*}\} \label{def perturbed ODE 2} \\
    \Tilde{\textbf{x}}^\eta(t_j,x;\textbf{t},\textbf{w}) & = \varphi_L\big( \Tilde{\textbf{x}}^\eta(t_j-,x;\textbf{t},\textbf{w}) + \varphi_b(w_j)\big) \ \ \ \forall j = 2,\cdots,l^* \label{def perturbed ODE 3}
\end{align}
Again, when $\eta = 1$ we drop the notational dependency on $\eta$ and use $\Tilde{\textbf{x}}$ to denote the process. Now from \cref{assumption multiple jump epsilon 0 constant 1}\cref{assumption multiple jump epsilon 0 constant 2} one can easily see the following fact: there exist constants $\Bar{t}, \Bar{\delta} > 0$ such that $\Tilde{\textbf{x}}(t_{l^*},0;\textbf{t},\textbf{w}) \notin \Omega$ (note that the starting point is $0$, the local minimum) only if (under the condition that $|w_j| \leq b \ \ \forall j$)
\begin{align}
    t_j - t_{j-1} & \leq \Bar{t}\ \ \forall j = 2,3,\cdots,l^* \label{def bar t} \\
    |w_j| & > \Bar{\delta} \ \ \forall j =1,2,\cdots,l^*. \label{def bar delta}
\end{align}
The intuition is as follows: if the inter-arrival time between any of the $l^*$ jumps is too long, then the path of $\Tilde{\textbf{x}}^\eta(t;x)$ will drift back to the local minimum $m_i$ so that the remaining $l^* - 1$ shocks (whose sizes are bounded by $b$) cannot overcome the radius $r$ which is strictly larger than $(l^* - 1)b$; similarly, if size of any of the shocks is too small, then since all other jumps have sizes bounded by $b$, the shock created by the $l^*_i$ jumps will be smaller than $(l^* - 1)b + 100\Bar{\epsilon}$, which is strictly less than $r$.
We fix these constants $\Bar{t},\Bar{\delta}$ throughout the analysis, and stress again that their values are dictated by the geometry of the function $f$, thus \emph{do not} vary with the accuracy parameters $\epsilon$ and $\delta$ mentioned earlier. In particular, choose $\bar{\delta}$ such that $\bar{\delta} < \bar{\epsilon}$.

In our analysis below, $\epsilon > 0$ will be a variable representing the level of \textit{accuracy} in our analysis. For instance, for small $\epsilon$, the chance that SGD iterates will visit somewhere that is $\epsilon-$close to $s_-$ or $s_+$ (namely, the boundary of the attraction filed) should be small. Consider some $\epsilon \in (0,\epsilon_0)$ where $\epsilon_0$ is the constant in Assumption 1. Due to \cref{assumption detailed function f at critical point}\cref{assumption detailed function f at critical point 2}, one can see the existence of some $g_0 > 0, c_1 < \infty$ such that
\begin{itemize}
    \item $|f^\prime(x)| \geq g_0$ for any $x \in \Omega$ such that $|x - s_-| > \epsilon_0, ||x - s_+| > \epsilon_0$;
    \item Let $\hat{t}_\text{ODE}(x,\eta) \delequal{} \min\{ t \geq 0:\ \textbf{x}^\eta(t,x) \in [-\epsilon,\epsilon] \}$ be the time that the ODE returns to a $\epsilon-$neighborhood of local minimum of $\Omega$ when starting from $x$. As proved in Lemma 3.5 of \cite{pavlyukevich2005metastable}, for any $x \in \Omega$ such that $|x - s_-| > \epsilon, |x - s_+| > \epsilon$, we have
    \begin{align}
        \hat{t}_\text{ODE}(x,\eta)\leq c_1 \frac{\log(1/\epsilon)}{\eta} \label{ineq prior to function hat t}
    \end{align}
    and we define the function
    \begin{align}
        \hat{t}(\epsilon) \delequal{} c_1\log(1/\epsilon). \label{def function hat t}
    \end{align}
\end{itemize}
In short, given any accuracy level $\epsilon$, the results above give us an upper bound for how fast the ODE would return to a neighborhood of the local minimum, if the starting point is not too close to the boundary of this attraction field $\Omega$.

For the first few technical results established below, we show that, without large jumps, the SGD iterates $X^\eta_n(x)$ are unlikely to show significant deviation from the deterministic gradient descent process $\textbf{y}^\eta_n(x)$ defined as
\begin{align}
    \textbf{y}^\eta_0(x) & = x, \label{def GD process 1} \\
    \textbf{y}^\eta_n(x) & = \textbf{y}^\eta_{n-1}(x) - \eta f^\prime\Big( \textbf{y}^\eta_{n-1}(x) \Big).\label{def GD process 2}
\end{align}
We are ready to state the first lemma, where we bound the distance between the gradient descent iterates $\textbf{y}^\eta_n(y)$ and the ODE $\textbf{x}^\eta(t,x)$ when the initial conditions $x,y$ are close enough.

\begin{lemma} \label{lemma Ode Gd Gap}
The following claim holds for all $\eta > 0$: for any $t > 0$, we have
\begin{align*}
    \sup_{s \in [0,t]}| \textbf{x}^\eta(s,x) - \textbf{y}^\eta_{\floor{s}}(y) | \leq (2\eta M + |x-y|)\exp(\eta M t)
\end{align*}
where $M \in (0,\infty)$ is the constant in \cref{assumption detailed function f at critical point 3}.
\end{lemma}

\begin{proof}

Define a continuous-time process $\textbf{y}^\eta(s;y) \delequal{} \textbf{y}^\eta_{\floor{s}}(y)$, and note that
\begin{align*}
    \textbf{x}^\eta(s,x) & = \textbf{x}^\eta(\floor{s},x) - \eta\int_{\floor{s}}^s f^\prime(\textbf{x}^\eta(u,x) )du \\
    \textbf{x}^\eta(\floor{s},x) & = x - \eta\int_0^{\floor{s} }f^{\prime}(\textbf{x}^\eta(u,x) )du \\
    \textbf{y}^\eta_{\floor{s}}(y) = \textbf{y}^\eta(\floor{s} ,y) & = y - \eta \int_0^{ \floor{s} }f^\prime( \textbf{y}^\eta(u,y) )du.
\end{align*}

Therefore, if we define function
$$b(u) = \textbf{x}^\eta(u,x) - \textbf{y}^\eta(u,y),$$
from the fact $|f^{\prime}(\cdot)|\leq M$, one can see that $|b(u)| \leq \eta M + |x - y|$ for any $u \in [0,1)$ and $|b(1)| \leq 2\eta M + |x - y|$. In case that $s > 1$, from the display above and the fact $|f^{\prime\prime}(\cdot)|\leq M$, we now have
\begin{align*}
    |\textbf{y}^\eta_{\floor{s}}(x) - \textbf{x}^\eta(s,x)| & \leq |b({\floor{s}})| + \eta M; \\
    |b({\floor{s}})| & \leq \eta M \int_1^{\floor{s}}|b(u)|du.
\end{align*}
From Gronwall's inequality (see Theorem 68, Chapter V of \cite{protter2005stochastic}, where we let function $\alpha(u)$ be $\alpha(u) = |b(u+1)|$), we have
$$|\textbf{y}^\eta_{\floor{s}}(x) - \textbf{x}^\eta(s,x)|\leq (2\eta M + |x-y|)\exp(\eta M t).$$
This concludes the proof.
\end{proof}

Now we consider an extension of the previous Lemma in the following sense: we add perturbations to the gradient descent process and ODE, and show that, when both perturbed by $l^*$ similar perturbations, the ODE and gradient descent process should still stay close enough. Analogous to the definition of the perturbed ODE $\widetilde{x}^\eta$ in \cref{def perturbed ODE 1}-\cref{def perturbed ODE 3}, we can construct a process $\widetilde{Y}^\eta$ as a perturbed gradient descent process as follows. For a sequence of integers $0 = t_1 < t_2 < \cdots < t_{l^*}$ (let $\textbf{t}=(t_j)_{j \geq 1}$) and a sequence of real numbers $\widetilde{w}_1,\cdots,\widetilde{w}_{l^*}$ (let $\widetilde{\textbf{w}} = (\widetilde{w}_j)_{j \geq 1}$) and $y \in \mathbb{R}$, define (for all $n =1,2,\cdots,t_{l^*}$) the perturbed gradient descent iterates with gradient clipping at $b$ and reflection at $\pm L$ as
\begin{align}
     \widetilde{ \textbf{y} }^{\eta}_n(y;\textbf{t},\widetilde{\textbf{w}}) & = \varphi_L\Big(\widetilde{ \textbf{y} }^{\eta}_{n-1}(y;\textbf{t},\widetilde{\textbf{w}}) + \varphi_b\big(- \eta f^\prime( \widetilde{ \textbf{y} }^{\eta}_{n-1}(y;\textbf{t},\widetilde{\textbf{w}})) + \sum_{j = 2}^{l^*}\mathbbm{1}\{ n = t_j \}\widetilde{w}_j\big) \Big)\label{def perturbed GD}
\end{align}
with initial condition $\widetilde{ \textbf{y} }^{\eta}_0(y;\textbf{t},\widetilde{\textbf{w}}) = \varphi_L\big( y + \varphi_b(\widetilde{w}_1) \big)$.

\begin{corollary} \label{corollary ODE GD gap}
Given any $\epsilon > 0$, the following claim holds for all sufficiently small $\eta > 0$: for any $x,y \in \Omega$, and sequence of integers $\textbf{t} = (t_j)_{j = 1}^{l^*}$ and any two sequence of real numbers $\textbf{w} = (w_j)_{j = 1}^{l^*}, \widetilde{\textbf{w}} = (\widetilde{w}_j)_{j \geq 1}^{l^*}$ such that
\begin{itemize}
    \item $|x-y| < \epsilon$;
    \item $t_1 = 0$, and $t_j - t_{j - 1} \leq 2\bar{t}/\eta$ for all $j \geq 1$ where $\bar{t}$ is the constant in \cref{def bar t};
    \item $|w_j - \widetilde{w}_{j}| < \epsilon$ for all $j \geq 1$;
\end{itemize}
then we have 
\begin{align*}
    \sup_{t \in [0, t_{l^*} ]}| \widetilde{\textbf{x}}^\eta(t,x;\textbf{t},\textbf{w}) - \widetilde{ \textbf{y} }^{\eta}_{ \floor{t} }(y;\textbf{t},\widetilde{\textbf{w}}) | \leq \bar{\rho}\epsilon
\end{align*}
where the constant $\bar{\rho} = (3 \exp(\eta M \bar{t}) + 3)^{l^*}$.
\end{corollary}

\begin{proof} Throughout this proof, fix some $\eta \in (0,\epsilon/2M)$. We will show that for any $\eta$ in the range the claim would hold.

First, on interval $[0,t_2)$, from Lemma \ref{lemma Ode Gd Gap}, one can see that (since $2M\eta < \epsilon$)
\begin{align*}
    \sup_{t \in [0, t_2) }| \widetilde{\textbf{x}}^\eta(t,x;\textbf{t},\textbf{w}) - \widetilde{ \textbf{y} }^{\eta}_{ \floor{t} }(y;\textbf{t},\widetilde{\textbf{w}}) | \leq 3 \exp(\eta M \bar{t})\cdot \epsilon.
\end{align*}
The at $t = t_2$, by considering the difference between $w_2$ and $\widetilde{w}_2$, and the possible change due to one more gradient descent step (which is bounded by $\eta M < \epsilon$), we have
\begin{align*}
    \sup_{t \in [0, t_2] }| \widetilde{\textbf{x}}^\eta(t,x;\textbf{t},\textbf{w}) - \widetilde{ \textbf{y} }^{\eta}_{ \floor{t} }(y;\textbf{t},\widetilde{\textbf{w}}) | \leq (3 \exp(\eta M \bar{t}) + 2)\cdot \epsilon.
\end{align*}
Now we proceed inductively. For any $j = 2,3,\cdots, l^* - 1$, assume that
\begin{align*}
    \sup_{t \in [0, t_j] }| \widetilde{\textbf{x}}^\eta(t,x;\textbf{t},\textbf{w}) - \widetilde{ \textbf{y} }^{\eta}_{ \floor{t} }(y;\textbf{t},\widetilde{\textbf{w}}) | \leq (3 \exp(\eta M \bar{t}) + 3)^{j - 1}\cdot \epsilon.
\end{align*}
Then by focusing on interval $[t_j, t_{j+1}]$ and using Lemma \ref{lemma Ode Gd Gap} again, one can show that
\begin{align*}
    \sup_{t \in [t_j, t_{j-1}] }| \widetilde{\textbf{x}}^\eta(t,x;\textbf{t},\textbf{w}) - \widetilde{ \textbf{y} }^{\eta}_{ \floor{t} }(y;\textbf{t},\widetilde{\textbf{w}}) | & \leq 2\epsilon + \big( (3 \exp(\eta M \bar{t}) + 3)^{j - 1} + 1 \big)\exp(\eta M \bar{t}) \epsilon \\
    & \leq (3 \exp(\eta M \bar{t}) + 3)^{j}\cdot \epsilon.
\end{align*}
This concludes the proof.
\end{proof}

In the next few results, we show that the same can be said for gradient descent iterates $\widetilde{\textbf{y} }_n$ and the SGD iterates $X_n$. Specifically, our first goal is to show that before any \textit{large} jump (see the definition in \ref{defLargeJump_GradientClipping}), it is unlikely that the gradient descent process $\textbf{y}^\eta_n$ would deviate too far from $X^\eta_n$. Define the event
\begin{align}
    A(n,\eta,\epsilon,\delta) & = \Big\{ \max_{k = 1,2,\cdots, n \wedge (T^\eta_1(\delta) - 1)  }\eta|Z_1 + \cdots + Z_k| \leq \epsilon \Big\} \label{def event A small noise large deviation}
\end{align}
and recall that arrival times $T^\eta_j(\delta)$ are defined in \cref{defArrivalTime large jump}. 

As a building block, we first study the case when the starting point $x$ is close to the reflection boundary $-L$. The takeaway from the next result is that the reflection operator hardly comes into play, since the SGD iterates would most likely quickly move to somewhere far enough from $\pm L$; besides, throughout this procedure the SGD iterates would most likely stay pretty close to the corresponding deterministic gradient descent process.
\begin{lemma} \label{lemma no large noise escape the boundary}  Given $\epsilon \in (0,\bar{\epsilon}/9)$, it holds for any sufficiently small $\epsilon, \delta, \eta > 0$ that, if $x \in [-L,-L+\bar{\epsilon}]$ and $\rho_0(|x - y| + 9\epsilon)  < \Bar{\epsilon}$, then on event $A(n,\eta,\epsilon,\delta)$ we have 
\begin{align*}
    | X^\eta_k(x) - \textbf{y}^\eta_k(y)  | \leq \rho_0 \cdot (|x-y| + 9\epsilon) \ \ \forall k = 1,2,\cdots,n\wedge (T^\eta_1(\delta) - 1 )\wedge \widetilde{T}_\text{escape}^\eta(x)
\end{align*}
where $\widetilde{T}_\text{escape}^\eta(x) \delequal{} \min\{n \geq 0:\ X^\eta_n(x) > -L + \bar{\epsilon}\} $ and $\rho_0 \delequal{} \exp\Big( \frac{2M\bar{\epsilon}}{0.9c^L_{-}} \Big)$ is a constant that does not vary with our choice of $\epsilon,\delta,\eta$.
\end{lemma}

\begin{proof}
For any $k < T^\eta_1(\delta)$, we know that $Z_k = Z^{\leq \delta}_k$ (thus $\eta|Z_k| < \delta$). Also, recall that $|f^\prime(x)| \leq M$ for any $x \in \Omega_i$. Therefore, as long as $\eta$ and $\delta$ are small enough, we will have that
\begin{align}
    |\eta(-f^\prime(X^\eta_n(x)) + Z^{\leq \delta}_k)| \leq b \label{property small eta delta no gradient clipping}
\end{align}
so the gradient clipping operator in \cref{def SGD iterates with gradient clipping} has no effect when $k < T^\eta_1(\delta)$, and in fact the only possible time for the gradient clipping trick to work is at $T^\eta_j(\delta)$. Therefore, we can safely rewrite the SGD update as
\begin{align*}
    X^\eta_k(x) = X^\eta_{k-1}(x) - \eta f^{\prime}(X^\eta_{k-1}(x)) + \eta Z_k + R_k\ \ \ \forall k < T^\eta_1(\delta)
\end{align*}
where each $R_k \geq 0$ and it represents the push caused by reflection at $-L$.

First, choose $\epsilon$ small enough so that $9\epsilon < \Bar{\epsilon}$. Next, based on \cref{property at boundary neighborhood L} we have the following lower bound:
\begin{align*}
     X^\eta_k(x) \geq x + 0.9c^L_{-}\eta k  - \epsilon\ \ \ \forall k < T^\eta_1(\delta).
\end{align*}
Let $\widetilde{t}_0(x,\epsilon) \delequal{} \min\{ n \geq 0: X^\eta_n(x) \geq -L + 2\epsilon \}$. Due to the inequality above, we know that
\begin{align}
    \widetilde{t}_0(x,\epsilon) \leq \frac{ 3\epsilon }{ 0.9c^L_{-}\eta }. \label{proof escape boundary bound 1}
\end{align}
One the other hand, given the current choice of $\epsilon$, if we choose $\eta$ and $\delta$ small enough, then using the same argument leading to \cref{property small eta delta no gradient clipping}, we will have
\begin{align*}
    X^\eta_{\widetilde{t}_0(x,\epsilon)}(x) \leq -L + 2.1 \epsilon \leq x + 1.1c^L_{-}\eta k + 2.1 \epsilon
\end{align*}
if $\widetilde{t}_0(x,\epsilon) \geq 1$ (namely $x < -L + 2\epsilon$). 

Let us inspect the two scenarios separately. First, assume $\widetilde{t}_0(x,\epsilon) \geq 1$. For the deterministic gradient descent process $\textbf{y}^\eta_n(y)$, we have the following bounds:
\begin{align*}
    y + 0.9c^L_{-}\eta k \leq \textbf{y}^\eta_k(y) \leq  y + 1.1 c^L_{-}\eta k \ \ \forall k \leq \widetilde{t}_0(x,\epsilon)\wedge( T^\eta_1(\delta) - 1 ).
\end{align*}
This gives us
\begin{align*}
    |X^\eta_k(x) - \textbf{y}^\eta_k(y)| \leq |x-y| + 0.2 c^L_{-}\eta k + 2.1\epsilon\ \ \forall k \leq \widetilde{t}_0(x,\epsilon)\wedge( T^\eta_1(\delta) - 1 ).
\end{align*}
At time $ k = \widetilde{t}_0(x,\epsilon) $, due to previous bound on $\widetilde{t}_0(x,\epsilon)$, we know that $|X^\eta_{\widetilde{t}_0(x,\epsilon)}(x) - \textbf{y}^\eta_{\widetilde{t}_0(x,\epsilon)}(y)| \leq |x-y| + 7\epsilon$. If $n\wedge (T^\eta_1(\delta) - 1 ) \leq \widetilde{t}_0(x,\epsilon)$ then we have already shown the desired claim. Otherwise, starting from time $\widetilde{t}_0(x,\epsilon)$, due to the definition of event $A(n,\eta,\epsilon,\delta)$ in \cref{def event A small noise large deviation}, we know that the SGD iterates $X^\eta_n(x)$ will not touch the boundary $-L$ afterwards. Therefore, by directly applying Lemma \ref{lemmaBasicGronwall}, and notice that $|f^{\prime\prime}(x)| \leq M$ for any $x \in [-L,L]$ and, we have
\begin{align*}
    |X^\eta_k(x) - \textbf{y}^\eta_k(y)| \leq (|x-y| + 9\epsilon)\cdot\exp\Big( \frac{2M\bar{\epsilon}}{0.9c^L_{-}} \Big)\ \ \ \forall k \leq n\wedge (T^\eta_1(\delta) - 1 )\wedge \widetilde{T}_\text{escape}^\eta(x).
\end{align*}
Indeed, it suffices to use Lemma \ref{lemmaBasicGronwall} for the next $\ceil{2\bar{\epsilon}/(0.9\eta c^L_{-})}$ steps to show that $|X^\eta_k(x) - \textbf{y}^\eta_k(y)|$ should be smaller than $\epsilon$ for the next $\ceil{2\bar{\epsilon}/(0.9\eta c^L_{-})}$ steps, while $\textbf{y}^\eta_k(y)$ will reach some where in $(-L+2\bar{\epsilon},-L+3\bar{\epsilon})$ within $\ceil{2\bar{\epsilon}/(0.9\eta c^L_{-})}$ steps so we must have
\begin{align}
    n\wedge (T^\eta_1(\delta) - 1 )\wedge \widetilde{T}_\text{escape}^\eta(x) \wedge \widetilde{t}_0(x,\epsilon) - \widetilde{t}_0(x,\epsilon) \leq 2\bar{\epsilon}/(0.9\eta c^L_{-}) \label{proof escape boundary bound 2}
\end{align}

Lastly, in the case $\widetilde{t}_0(x,\epsilon) = 0$ (which means $x \geq -L + \epsilon$), we can use Lemma \ref{lemmaBasicGronwall} directly as we did above and establish the same bound. This concludes the proof.
\end{proof}

Obviously, a similar result can be shown if $x$ is in the rightmost attraction field $(s_{n_{\text{min} - 1}},L]$ and the approach is identical. We omit the details here. In the next Lemma, we consider the scenario where the starting point $x$ is far enough from the boundaries.

\begin{lemma} \label{lemma SGD GD gap}
Given any $\epsilon > 0$, the following holds for all sufficiently small $\eta > 0$: for any $x,y \in \Omega$ and positive integer $n$ such that $|x - L| >2\epsilon,|x+L|>2\epsilon,|x - s_-| > 2\epsilon, |x - s_+| > 2\epsilon$ and $|x-y|<\frac{  \epsilon}{2\exp(\eta M n)  }$, on event
\begin{align*}
A(n,\eta,\frac{  \epsilon}{2\exp(\eta M n)},\delta  ) \cap \Big\{ |\textbf{y}^\eta_j(y)|\in \Omega,|X^\eta_j(x)|\in \Omega \ \ \forall j = 1,2,\cdots,n \wedge (T^\eta_1(\delta) - 1 ) \Big\}    
\end{align*}
we have 
$$| \widetilde{X}^\eta_m(y) - X^\eta_m(x) | \leq \epsilon\ \ \forall m = 1,2,\cdots,n\wedge(T^\eta_1(\delta) - 1 ) .$$
\end{lemma}

\begin{proof} For sufficiently small $\eta$, we will have that the (deterministic) gradient descent iterates $|\textbf{y}^\eta_n|$ is monotonically decreasing in $n$, which ensures that $\textbf{y}^\eta_n$ always stays in the range that are at least $\epsilon-$away from $\pm L$ or $s_-, s_+$. We now show that the claim holds for any such $\eta$.

On event $\Big\{ |\textbf{y}^\eta_j(y)|\in \Omega,|X^\eta_j(x)|\in \Omega \ \ \forall j = 1,2,\cdots,n \wedge (T^\eta_1(\delta) - 1 ) \Big\} $, we are able to apply Lemma \ref{lemmaBasicGronwall} inductively for any $m \in [n]$ and obtain that
\begin{align*}
    | \textbf{y}^\eta_j(y) - X^\eta_j(x) | \leq (|x - y| + \frac{\epsilon}{2\exp(\eta M n)} )\exp(\eta M j) < \epsilon \ \ \ \forall j = 1,2,\cdots,m
\end{align*}
and conclude the proof. The reason to apply the Lemma inductively for $m = 1,2,\cdots,n$, instead of directly at step $n$, is to ensure that SGD iterates $X^\eta_n$ would not hit the boundary $\pm L$ (so the reflection operator would not come into play on the time interval we are currently interested in), thus ensuring that Lemma \ref{lemmaBasicGronwall} is applicable.
\end{proof}

Similar to the extension from Lemma \ref{lemma Ode Gd Gap} to Corollary \ref{corollary ODE GD gap}, we can extend Lemma \ref{lemma SGD GD gap} to show that, if we consider the a gradient descent process that is only perturbed by large noises, then it should stay pretty close to the SGD iterates $X^\eta_n$. To be specific, let
\begin{align}
    Y^\eta_0(x) & = x \label{def x tilde large 1} \\
     Y^\eta_n(x) & =\varphi_L\Big(Y^\eta_{n - 1}(x) - \varphi_b\big( - \eta f^{\prime}\big( Y^\eta_{n - 1}(x) \big) + \sum_{j \geq 1}\mathbbm{1}\{n = T^\eta_j(\delta)\}\eta Z_{n}\big)\Big). \label{def x tilde large 2} 
\end{align}
be a gradient descent process (with gradient clipping at threshold $b$) that is only shocked by large noises in $(Z_n)_{n \geq 1}$. The next corollary can be shown by an approach that is identical to Corollary \ref{corollary ODE GD gap} (namely, inductively repeating Lemma \ref{lemma SGD GD gap} at each jump time) so we omit the details here.
\begin{corollary} \label{corollary sgd gd gap}
Given any $\epsilon > 0$, the following holds for any sufficiently small $\eta > 0$: For any $|x| < 2\epsilon$, on event $ A_0(\epsilon,\eta,\delta) \cap B_0(\epsilon,\eta,\delta)$, we have
\begin{align*}
    |Y^\eta_n(x) - X^\eta_n(x)| < \widetilde{\rho}\epsilon\ \ \forall n =1,2,\cdots,T^\eta_{l^*}(\delta)
\end{align*}
where
\begin{align*}
    A_0(\epsilon,\eta,\delta) & \delequal{} \Big\{\forall i = 1,\cdots,l^{*},\ \max_{j = T^\eta_{i-1}(\delta) + 1, \cdots, T^\eta_{i}(\delta) - 1} \eta|Z_{ T^\eta_{i-1}(\delta) + 1} + \cdots + Z_j| \leq \frac{ \epsilon }{2\exp(2\bar{t}M) }  \Big\}; \\
    B_0(\epsilon,\eta,\delta) & \delequal{} \Big\{\forall j = 2,\cdots,l^{*}, T^\eta_j(\delta) - T^\eta_{j - 1}(\delta) \leq 2\bar{t}/\eta  \Big\}
\end{align*}
and $\widetilde{\rho} \in (0,\infty)$ is a constant that does not vary with $\eta,\delta,\epsilon$.

\end{corollary}

The next two results shows that the type of events $A(n,\eta,\epsilon,\delta)$ defined in \cref{def event A small noise large deviation} is indeed very likely to occur, especially for small $\epsilon$. For clarity of the presentation, we introduce the following definitions that are slightly more general than the \textit{small} and \textit{large} jumps defined in \cref{defSmallJump_GradientClipping}\cref{defLargeJump_GradientClipping} (for any $c > 0$)
\begin{align*}
    Z^{\leq c}_n & \delequal{} Z_n\mathbbm{1}\{ |Z_n| \leq c \}, \\
    Z^{> c}_n & \delequal{} Z_n\mathbbm{1}\{ |Z_n| > c \}.
\end{align*}

\begin{lemma} \label{lemma bernstein bound small deviation for small jumps}
Define functions $u(\eta) = \delta/\eta^{1 - \Delta},\ v(\eta) = \epsilon\eta^{\widetilde{\Delta}}$ with $\epsilon,\delta > 0$. If real numbers $\Delta,\widetilde{\Delta},\beta,\epsilon,\delta$ and positive integers $j,N$ are such that the following conditions hold:
\begin{align}
    \Delta & \in \big[0,(1-\frac{1}{\alpha})\wedge \frac{1}{2} \big), \label{proof lemma bernstein parameter 1} \\
    \beta & \in \big( 1,(2-2\Delta) \wedge \alpha(1-\Delta) \big), \label{proof lemma bernstein parameter 2} \\
    \widetilde{\Delta} & \in [0,\frac{\Delta}{2}],\
    \widetilde{\Delta} < \alpha(1-\Delta) - \beta , \label{proof lemma bernstein parameter 3} \\
    N & < \big(\alpha(1-\Delta) - \beta\big)j,\label{proof lemma bernstein parameter 4}  \\
    v(\eta) - j\eta u(\eta) & \geq v(\eta)/2\ \ \ \text{for all $\eta > 0$ sufficiently small}, \label{proof lemma bernstein parameter 5}
\end{align}
then
\begin{align*}
    \mathbb{P}\Big( \max_{k = 1,2,\cdots,\ceil{1/\eta^\beta}} \eta \big| Z^{\leq u(\eta)}_1 + \cdots + Z^{\leq u(\eta)}_k \big| > 3v(\eta)  \Big) = o(\eta^N)
\end{align*}
as $\eta \downarrow 0$.
\end{lemma}

\begin{proof}

From the stated range of the parameters, we know that
\begin{align*}
   \alpha(1-\Delta)& > \beta, \\
    \big(\alpha(1 -\Delta) - \beta \big)j & > N,
\end{align*}
so we are able to find $\gamma \in (0,1)$ small enough such that
\begin{align}
    \alpha(1-\Delta)(1 - 2\gamma) & > \beta, \label{proof lemma bernstein choose gamma 1} \\
    \big(\alpha(1 -\Delta)(1-2\gamma) - \beta \big)j & > N.\label{proof lemma bernstein choose gamma 2}
\end{align}

Fix such $\gamma \in (0,1)$ for the rest of the proof, and let $n(\eta) \triangleq \ceil{ (1/\eta)^\beta }$, $I\triangleq \#\Big\{i \in [n(\eta)]: |Z_i^{\leq u(\eta)}| > u(\eta)^{1-\gamma}\Big\}$.  Then 
\begin{align*}
&\P\Big(\big| Z^{\leq u(\eta) }_1 + \cdots + Z^{\leq u(\eta)}_{n(\eta)} \big| > v(\eta)\Big)
\\&= \sum_{i=0}^{j-1}\underbrace{\P\Big(\big| Z^{\leq u(\eta) }_1 + \cdots + Z^{\leq u(\eta)}_{n(\eta)} \big| > v(\eta),\ I = i\Big)}_{\triangleq \text{(I)}} + \underbrace{\P\Big(\big| Z^{\leq u(\eta)}_1 + \cdots + Z^{\leq u(\eta)}_{n(\eta)} \big| > v(\eta),\ I \geq j\Big)}_{\triangleq \text{(II)}}
\end{align*}

Note that since $|Z_i^{\leq u(\eta)}|<u(\eta)$, 
\begin{align}
\text{(I)}
&\leq 
{n(\eta) \choose i} \cdot\P\Big(\big| Z^{\leq u(\eta)}_1 + \cdots + Z^{\leq u(\eta)}_{n(\eta)-i} \big| > \frac{v(\eta) - i\eta u(\eta)}{\eta}, |Z_i^{\leq u(\eta)}| \leq u(\eta)^{1-\gamma} \ \forall i \in [n(\eta)-i]\Big)
\nonumber
\\&\leq 
n(\eta)^i\cdot \P\Big(\big| Z^{\leq u(\eta)^{1 - \gamma}}_1 + \cdots + Z^{\leq u(\eta)^{1 - \gamma}}_{n(\eta)-i} \big| > \frac{v(\eta) - i\eta u(\eta)}{\eta}\Big) \nonumber \\
& \leq n(\eta)^i\cdot \P\Big(\big| Z^{\leq u(\eta)^{1 - \gamma}}_1 + \cdots + Z^{\leq u(\eta)^{1 - \gamma}}_{n(\eta)-i} \big| > \frac{v(\eta) }{2\eta}\Big) \label{proof lemma bernstein prior to bernstein bound}
\end{align}
where the last inequality follows from \cref{proof lemma bernstein parameter 5}. First, since $\mathbb{E}Z_1 = 0$, we have
\begin{align*}
    \big|\mathbb{E}Z^{\leq u(\eta)^{1 - \gamma}}_1\big| & = \big|\mathbb{E}Z^{> u(\eta)^{1 - \gamma}}_1\big| \\
    & = \int_{ u(\eta)^{1 - \gamma} }^\infty \mathbb{P}(|Z_1|>x)dx \in \mathcal{RV}_{ (\alpha-1)(1-\gamma)(1 - \Delta) }(\eta).
\end{align*}
Therefore, for all $\eta > 0$ that are sufficiently small,
\begin{align*}
    & |\mathbb{E}Z^{\leq u(\eta)^{1 - \gamma}}_{1} + \cdots + \mathbb{E}Z^{\leq u(\eta)^{1 - \gamma}}_{n(\eta)-i}  | \\
    \leq & n(\eta)\cdot \eta^{ (\alpha-1)(1-\Delta)(1-2\gamma) } \leq 2 \eta^{ (\alpha-1)(1-\Delta)(1-2\gamma) - \beta } \\
    \leq &  (1/\eta)^{ (1-\Delta)(1-2\gamma)  }\ \ \ \ \text{due to \cref{proof lemma bernstein choose gamma 1}} \\
    \leq & \frac{v(\eta)}{4\eta} \ \ \ \ \text{due to $\widetilde{\Delta}/2 \leq \Delta$ in \cref{proof lemma bernstein parameter 3} and $1 - \gamma < 1$}.
\end{align*}
If we let $Y_n =  Z^{\leq u(\eta)^{1 - \gamma}}_{n} - \mathbb{E}Z^{\leq u(\eta)^{1 - \gamma}}_{n}$ and plug the bound above back into \cref{proof lemma bernstein prior to bernstein bound}, then (for all $\eta > 0$ that are sufficiently small)
\begin{align}
    \text{(I)}
&\leq n(\eta)^i\cdot\mathbb{P}( |Y_1 + \cdots + Y_{ n(\eta) - i }| > \frac{v(\eta)}{4\eta} ) \nonumber \\
& \leq n(\eta)^i\exp\left(-\frac{ \frac{\epsilon^2}{16}\cdot 1/\eta^{ 2 - 2\widetilde{\Delta}}  }{2\big(n(\eta)-i\big)\E |Y_1|^2 + \frac23 \delta^{1-\gamma}\cdot (1/\eta)^{ (1-\Delta)(1-\gamma) }\cdot \frac{\epsilon}{4}/\eta^{1 - \widetilde{\Delta}} }\right) \label{proof lemma bernstein bound}
\end{align}
where the last inequality is obtained from Bernstein's inequality.
Note that from Karamata's theorem, 
\begin{align*}
    \mathbb{E}|Y_1|^2 & = var( Z^{\leq u(\eta)^{1 - \gamma}}_{1}  ) \leq \mathbb{E}|Z^{\leq u(\eta)^{1 - \gamma}}_{1}|^2 \nonumber \\
    & \leq \int_0^{ u(\eta)^{1 - \gamma} } 2x\mathbb{P}(|Z_1| > x)dx  \in \RV_{-(1-\Delta)(1-\gamma)(2-\alpha)}(\eta).
 \end{align*}
Now note that
\begin{itemize}
    \item In case that $\alpha < 2$, for all $\eta>0$ that are sufficiently small, we have (using \cref{proof lemma bernstein parameter 3})
    \begin{align*}
        & 2\big(n(\eta)-i\big)\E |Y_1|^2 \leq (1/\eta)^{\beta +  (2-\alpha)(1-\Delta) } < (1/\eta)^{2(1-\Delta)} \\
        \Rightarrow & \frac{ 1/\eta^{ 2 - 2\widetilde{\Delta} } }{ 2\big(n(\eta)-i\big)\E |Y_1|^2 } \geq 1/\eta^\Delta;
    \end{align*}
    \item In case that $\alpha \geq 2$, for all $\eta>0$ that are sufficiently small,
    \begin{align*}
         & 2\big(n(\eta)-i\big)\E |Y_1|^2 < 1/\eta^{ \beta + \frac{\Delta}{2} }
    \end{align*}
    and we know that $\beta + \frac{\Delta}{2} < 2 - 2\widetilde{\Delta}$ due to $2 - \beta > 2\Delta$ and $2\widetilde{\Delta} \leq \Delta$ (see \cref{proof lemma bernstein parameter 1}-\cref{proof lemma bernstein parameter 3});
    \item Since $\gamma > 0$ and $2\widetilde{\Delta} \leq \Delta$, we know that
    \begin{align*}
        (1 - \Delta)(1-\gamma) + (1 - \widetilde{\Delta}) < 2 - 2\widetilde{\Delta}.
    \end{align*}
\end{itemize}
Therefore, it is easy to see that the R.H.S. of \cref{proof lemma bernstein bound} decays at a geometric rate as $\eta$ tends to zero, hence $o(\eta^N)$. On the other hand,
\begin{align*}
\text{(II)}&
\leq 
\P( I \geq j) 
\leq 
{n(\eta) \choose j}\cdot \P\Big(|Z_i^{ \leq u(\eta) }| > u(\eta)^{1-\gamma} \ \forall i=1,\ldots,j\Big) \\
&\leq 
n(\eta)^j\cdot \P\Big(|Z_1^{ \leq u(\eta) }| > u(\eta)^{1-\gamma}\Big)^j,
\end{align*}
which is regularly varying w.r.t.\ $\eta$ with index $\big(\alpha(1-\gamma)(1-\Delta)-\beta\big)j$.
Therefore, for all $\eta > 0$ sufficiently small,
\begin{align*}
    \text{(II)} \leq \eta^{ \big(\alpha(1-2\gamma)(1-\Delta)-\beta\big)j } < \eta^N\ \ \ \ \text{due to \cref{proof lemma bernstein choose gamma 2}}.
\end{align*}
Collecting results above, we have established that
$$\P\Big(\eta\big| Z^{\leq u(\eta) }_1 + \cdots + Z^{\leq u(\eta)}_{n(\eta)} \big| > v(\eta)\Big) = o(\eta^N).$$
The conclusion of the lemma now follows from Etemadi's theorem.
\end{proof}

Now consider the following setting. Let us fix some positive integer $N$ and $\beta \in (1,2\wedge \alpha)$. Then we can find some positive integer $j$ such that $(\alpha - \beta)j > N$. Meanwhile, given any $\epsilon > 0$, we will have $\epsilon - j\delta \geq \epsilon/2$ for all $\delta > 0$ sufficiently small. Therefore, by applying Lemma \ref{lemma bernstein bound small deviation for small jumps} with $\Delta = \widetilde{\Delta} = 0$ (hence $u(\eta) = \delta/\eta,\ v(\eta) = \epsilon$) and $\beta,j,N,\epsilon,\delta$ as described here, we immediately get the following result.

\begin{lemma} \label{lemma prob event A}
Given any $\beta \in (1,\alpha\wedge 2)$, $\epsilon > 0$, and $N > 0$, the following holds for any sufficiently small $\delta > 0$:
\begin{align*}
    \mathbb{P}\Big( \max_{ j = 1,2,\cdots,\ceil{ (1/\eta)^\beta } } \eta| Z^{\leq\delta/\eta}_1 + \cdots + Z^{\leq\delta/\eta}_j | > \epsilon  \Big) = o(\eta^N)
\end{align*}
as $\eta \downarrow 0$.
\end{lemma}

Using results and arguments above, we are able to illustrate the typical behavior of the SGD iterates $X^\eta_n$ in the following two scenarios. First, we show that, when starting from most parts in the attraction field $\Omega$, the SGD iterates $X^\eta_n$
 will most likely return to the neighborhood of the local minimum within a short period of time without exiting $\Omega$. Given that there are only finitely many attraction fields on $f$, it is easy to see that the key technical tool Lemma \ref{lemma first exit key lemma 1} follows immediately from the next result.
\begin{lemma} \label{lemma return to local minimum quickly}
For sufficiently small $\epsilon>0$, the following claim holds:
\begin{align*}
   \lim_{\eta \downarrow 0} \sup_{ x \in \Omega: |x - s_-|\wedge |x - s_+|>\epsilon } \mathbb{P}_x\Big( X^\eta_n \in \Omega \ \forall n \leq T_\text{return}(\eta,\epsilon), \text{and }T_\mathrm{return}(\eta,\epsilon) \leq \rho(\epsilon)/\eta  \Big) = 1
\end{align*}
where the stopping time involved is defined as
\begin{align*}
    T_\text{return}(\eta,\epsilon) \delequal{} \min\{ n \geq 0: X^\eta_n(x) \in [-2\epsilon,2\epsilon] \}
\end{align*}
the function $\hat{t}(\epsilon)$ is defined in \cref{def function hat t}, and the function $\rho(\cdot)$ is defined as $\rho(\epsilon) = \frac{3\bar{\epsilon}}{0.9c^L_-\wedge c^L_+} + 2\hat{t}(\epsilon)$
\end{lemma} 

\begin{proof}
Throughout, we only consider $\epsilon$ small enough so that Lemma \ref{lemma no large noise escape the boundary} could hold. Also, fix some $N > 0, \Delta_\alpha \in (0,\alpha - 1)$ and $\beta \in (1,\alpha)$. Let $\sigma(x,\eta) \delequal{} \min\{n \geq 0: X^\eta_n \notin \Omega\}$.

Without loss of generality, we assume $\Omega = [-L,s_+)$ and $x < 0$ (so reflection at $-L$) is a possibility. Any other case can be addressed similarly as shown below.

From Lemma \ref{lemmaGeomFront} and the regular varying nature of $H(\cdot)$, we have, for any $\epsilon,\delta > 0$,
\begin{align}
    \mathbb{P}(T^\eta_1(\delta) \leq \rho(\epsilon)/\eta) \leq \eta^{ \alpha - 1 - \Delta_\alpha } \label{proof return to local minimum ineq 1}
\end{align}
for any sufficiently small $\eta$.

Let $\widetilde{T}^\eta_\text{escape}(x)$ be the stopping time defined in Lemma \ref{lemma no large noise escape the boundary}. From \cref{proof escape boundary bound 1},\cref{proof escape boundary bound 2},\cref{proof return to local minimum ineq 1} and Lemma \ref{lemma prob event A}, we know that
\begin{align}
    & \sup_{x \in [-L,-L + \bar{\epsilon}]}\mathbb{P}\Big(\widetilde{T}^\eta_\text{escape}(x) < \sigma(x,\eta),\widetilde{T}^\eta_\text{escape}(x) \leq \frac{3\bar{\epsilon}}{0.9c^L_-\eta} \text{ and } X^\eta_{ \widetilde{T}^\eta_\text{escape} }(x) \in [-L +\bar{\epsilon}, -L + 2\bar{\epsilon}] \Big) \nonumber
    \\ 
    \geq& 1 - \eta^N - \eta^{ \alpha - 1 -\Delta_\alpha } \label{proof return result 1}
\end{align}
for all sufficiently small $\eta$.

Next, we focus on $x \in \Omega$ such that $|x-s_-|\wedge|x-s_+| > \epsilon$ and $x \geq -L + \bar{\epsilon}$. We start by considering the time it took for the (deterministic) gradient descent process $\textbf{y}^\eta_n(x)$ to return to $[-1.5\epsilon,1.5\epsilon]$. From the definition of $\hat{t}(\epsilon)$ in \cref{def function hat t} and Lemma \ref{lemma Ode Gd Gap}, we know that for $\eta$ small enough such that $\eta\exp(2M\hat{t}(\epsilon)) < 0.5 \epsilon$, we have
$$ \min\{ n \geq 0: \textbf{y}^\eta_n(x) \in [-1.5\epsilon,1.5\epsilon] \} \leq 2\hat{t}(\epsilon)/\eta.$$
Now consider event $A(\ceil{ (1/\eta)^\beta },\eta,\frac{ \epsilon }{4\exp( 2M\hat{t}(\epsilon) ) },\delta)$ (see definition in \cref{def event A small noise large deviation}). From Lemma \ref{lemma prob event A}, we know that for any sufficiently small $\delta$, we have
\begin{align}
    \mathbb{P}\Big( \big( A( \ceil{ (1/\eta)^\beta } ,\eta,\frac{ \epsilon }{4\exp( 2M\hat{t}(\epsilon) ) },\delta) \big)^c \Big) = o(\eta^N). \label{proof return to local minimum ineq 2}
\end{align}
Combining this result with \cref{proof return to local minimum ineq 1}\cref{proof return to local minimum ineq 2} and Lemma \ref{lemma SGD GD gap}, we get
\begin{align}
    & \sup_{x \in \Omega: |x-s_-|\wedge|x-s_+| > \epsilon, x \geq -L + \bar{\epsilon} }\mathbb{P}_x\Big( T_\text{return}(\eta,\epsilon) < \sigma(x,\eta), T_\text{return}(\eta,\epsilon) \leq 2\hat{t}(\epsilon)/\eta  \Big)
    \\
    & \geq 1 - \eta^N - \eta^{ \alpha - 1 -\Delta_\alpha } \label{proof return result 2}
\end{align}
for any sufficiently small $\eta$. To conclude the proof, we only to combine strong Markov property (at $\widetilde{T}^\eta_\text{escape}$) with bounds in \cref{proof return result 1}\cref{proof return result 2}.
\end{proof}

In the next result, we show that, once entering a $\epsilon-$small neighborhood of the local minimum, the SGD iterates will most likely stay there until the next large jump.

\begin{lemma} \label{lemma stuck at local minimum before large jump}
Given $N_0 > 0$, the following claim holds for any $\epsilon,\delta>0$ that are sufficiently small:
\begin{align*}
    \sup_{x \in [-2\epsilon,2\epsilon]}\mathbb{P}\Big( \exists n < T^\eta_1(\delta)\ s.t.\ |X^\eta_n(x)| > 3\epsilon \Big) = o(\eta^{N_0})
\end{align*}
as $\eta \downarrow 0$.
\end{lemma}

\begin{proof}

Fix $\epsilon$ small enough such that $3\epsilon < \epsilon_0$ (see Assumption 1 for the constant $\epsilon_0$). Also, fix some $\Delta \alpha \in (0,1), \beta \in (1,\alpha)$, $N > 
\alpha +\Delta \alpha - \beta + N_0$. Due to Lemma \ref{lemma prob event A}, for any $\delta$ sufficiently small, we will have
\begin{align}
    \mathbb{P}\Big( \max_{ j = 1,2,\cdots,\ceil{ (1/\eta)^\beta } } \eta| Z^{\leq,\delta}_1 + \cdots + Z^{\leq,\delta}_j | > \frac{\epsilon}{ \exp(2M) }  \Big) = o(\eta^N). \label{proof stuck at local minimum ineq 1}
\end{align}
Fix such $\delta > 0$. We now show that the desired claim is true for the chosen $\epsilon,\delta$.

First of all, from Lemma \ref{lemmaGeomDistTail}, we know the existence of some $\theta > 0$ such that
\begin{align}
    \mathbb{P}( T^\eta_1(\delta) > 1/\eta^{\alpha + \Delta \alpha} ) = o(\exp(-1/\eta^\theta)). \label{proof stuck at local minimum ineq 2}
\end{align}
Next, let us zoom in on the first $\ceil{(1/\eta)^\beta}$ SGD iterates. For any $\eta$ small enough, we will have $\textbf{y}^\eta_n(x) \in [-2\epsilon,2\epsilon]$ for any $n \geq 1$ and $\textbf{y}^\eta_{ \ceil{ (1/\eta)^\beta } }(x) \in [-\epsilon,\epsilon]$ given $x \in [-2\epsilon,2\epsilon]$. From now on we only consider such $\eta$. Due to Lemma \ref{lemma SGD GD gap}, we know that on event $\big\{ \max_{ j = 1,2,\cdots,\ceil{ (1/\eta) } } \eta| Z^{\leq,\delta}_1 + \cdots + Z^{\leq,\delta}_j | > \frac{\epsilon}{ \exp(2M) }  \big\}$, we have 
\begin{align*}
    |X^\eta_n(x)| \leq 3\epsilon\ \ \forall n \leq \ceil{1/\eta^\beta}\wedge (T^\eta_1(\delta) - 1)
\end{align*}
and on event $\big\{ \max_{ j = 1,2,\cdots,\ceil{ (1/\eta) } } \eta| Z^{\leq,\delta}_1 + \cdots + Z^{\leq,\delta}_j | > \frac{\epsilon}{ \exp(2M) }  \big\} \cap \{ T^\eta_1(\delta) > \ceil{ (1/\eta)^\beta } \}$, we have
$X^\eta_{ T^\eta_1(\delta) }(x) \in [-2\epsilon,2\epsilon]$. Now by repeating the same argument inductively for $\ceil{ 1/\eta^{ \alpha + \Delta \alpha - \beta } }$ times, we can show that on event
\begin{align*}
    \{\forall i = 1,2,\cdots,\ceil{ \frac{1}{\eta^{ \alpha + \Delta \alpha - \beta} } },\ \max_{ j =  1,\cdots,\ceil{ (1/\eta)^\beta } } \eta| Z^{\leq,\delta}_{ i\ceil{ (1/\eta)^\beta } + 1 } + \cdots + Z^{\leq,\delta}_{ i\ceil{ (1/\eta)^\beta } + j} | > \frac{\epsilon}{ \exp(2M) }   \},
\end{align*}
we have $ |X^\eta_n(x)| \leq 3\epsilon \ \forall n \leq 1/\eta^{\alpha + \Delta \alpha}\wedge (T^\eta_1(\delta) - 1) $. To conclude the proof, we only need to combine this fact with \cref{proof stuck at local minimum ineq 1}.
\end{proof}



We introduce a few concepts that will be crucial in the analysis below. Recall the definition of perturbed ODE $\widetilde{\textbf{x}}^\eta$ in \cref{def perturbed ODE 1}-\cref{def perturbed ODE 3} 
(note that we will drop the notational dependency on learning rate $\eta$ when we choose $\eta = 1$). Consider the definition of the following two mappings from where $\textbf{w} = (w_1,\cdots,w_{l^*})$ is a sequence of real numbers and $\textbf{t} = (t_1,t_2,\cdots,t_{l^*})$ with $0 = t_1 < t_2 < t_3 < \cdots$ as
\begin{align*}
    h(\textbf{w},\textbf{t}) & = \widetilde{\textbf{x}}(t_{l^*},0;\textbf{t},\textbf{w}).
\end{align*}
Next, define sets (for any $\epsilon \in (-\bar{\epsilon},\bar{\epsilon})$)
\begin{align}
    E(\epsilon) & = \{(\textbf{w},\textbf{t}) \subseteq \mathbb{R}^{l^*}\times \Big(\mathbb{R}_{+}\Big)^{l^* - 1}: h(\textbf{w},\textbf{t}) \notin [ s_- -\epsilon,s_+ + \epsilon  ] \}. \label{defSetE}
\end{align}
We add a few remarks about the two types of sets defined above.
\begin{itemize}
    \item Intuitively speaking, $E(\epsilon)$ 
    contains all the perturbations (with times and sizes) that can send the ODE 
    out of the current attraction field (allowing for some error with size $\epsilon$);
    \item From the definition of $\bar{t},\bar{\delta}$ in \cref{def bar t}\cref{def bar delta} and Corollary \ref{corollary ODE GD gap}, one can easily see that for a fixed $\epsilon \in (-\bar{\epsilon},\bar{\epsilon})$,
    \begin{align*}
        (\textbf{w},\textbf{t}) \in E(\epsilon) \Rightarrow  |w_j|>\bar{\delta}, t_j - t_{j-1} \leq \bar{t} \ \ \forall j;
    \end{align*}
    \item Lastly, $E(\epsilon)$ 
    are open sets due to $f \in C^2$.
\end{itemize}

Use $\textbf{Leb}_+$ to denote the Lebesgue measure restricted on $[0,\infty)$, and define (Borel) measure $\nu_\alpha$ with density on $\mathbb{R}\symbol{92}\{0\}$:
$$\nu_\alpha(dx) = \mathbbm{1}\{x > 0\}\frac{\alpha p_+}{x^{\alpha + 1}} + \mathbbm{1}\{x <0\}\frac{\alpha p_-}{|x|^{\alpha + 1}} $$
where $\alpha > 1$ is the regular variation index for the distribution of $Z_1$ and $p_-,p_+ \in (0,1)$ are constants in Assumption 2. Now we can define a Borel measure $\mu$ on $\mathbb{R}^{l^*}\times \Big(\mathbb{R}_{+}\Big)^{l^* - 1}$ as product measure
\begin{align}
    \mu= (\nu_\alpha)^{l^*}\times(\textbf{Leb}_+)^{l^* - 1}. \label{defMuMeasure}
\end{align}
Due to remarks above, one can see that for $\epsilon \in (-\bar{\epsilon},\bar{\epsilon})$, we have $\mu( E(\epsilon) ) < \infty$. We are now ready to analyze a specific type of noise $Z_n$.

\begin{definition}\label{definitionOverflow}
For any $n \geq 1$ and any $\epsilon \in (-\bar{\epsilon},\bar{\epsilon}),\delta \in (0,b\wedge \bar{\delta}),\eta> 0$, we say that the jump $Z_n$ has \textbf{$(\epsilon,\delta,\eta)$-overflow} if
\begin{itemize}
    \item $\eta|Z_n| > \delta$;
    \item In the set $\{n + 1, \cdots, n+2\ceil{l^*\bar{t}/\eta}\}$, there are at least $(l^* - 1)$ elements (ordered as $n < t_2 < t_3 < \cdots < t_{l^*}$) such that $\eta|Z_{t_i}| > \delta$ for any $i = 2,\cdots,l^*$;
    \item Let $t_1 = n$ and $t^\prime_i = t_i - t_{i - 1}$ for any $i = 2,\cdots,l^*$, $w_i = \eta Z_i$ for any $i = 1,\cdots,l^*$, for real sequence $\textbf{w} = ( w_1,w_2,\cdots,w_{l^*})$ and a sequence of positive number $\textbf{t} = ( \eta(t_i - n) )_{i = 2}^{l^*}$, we have
    $$(\textbf{w},\textbf{t}) \in E(\epsilon).$$
\end{itemize}
Moreover, if $Z_n$ has $(\epsilon,\delta,\eta)-$overflow, then we call $h( \textbf{w}, \textbf{t} )$ as its $(\epsilon,\delta,\eta)-$\textbf{overflow endpoint}.
\end{definition}

Due to the iid nature of $(Z_j)_{j \geq 1}$, let us consider an iid sequence $(V_j)_{j \geq 0}$ where the sequence has the same law of $Z_1$. Note that for any fixed $n \geq 1$, the probability that $Z_n$ has $(\epsilon,\delta,\eta)$-overflow is equal to the probability that $V_0$ has$(\epsilon,\delta,\eta)$-overflow. More specifically, we know that $\mathbb{P}(\eta|V_0| > \delta) = H(\delta/\eta)$, and now we focus on conditional probability admitting the following form:
\begin{align}
    p(\epsilon,\delta,\eta) = \mathbb{P}\Big(V_0\ \text{has $(\epsilon,\delta,\eta)$-overflow}\ \Big|\ \eta|V_0| > \delta \Big). \label{def overflow conditional probability}
\end{align}
For any open interval $A = (a_1,a_2)$ such that $A \cap [s_- + \bar{\epsilon},s_+ - \bar{\epsilon}] = \emptyset$, we also define
\begin{align}
    p(\epsilon,\delta,\eta; A) = \mathbb{P}\Big(V_0\ \text{has $(\epsilon,\delta,\eta)$-overflow and the endpoint is in }A\ \Big|\ \eta|V_0| > \delta \Big). \label{def overflow conditional probability with endpoint}
\end{align}

\begin{lemma} \label{lemmaOverflowProb}
For any $\epsilon \in (-\bar{\epsilon},\bar{\epsilon}),\delta \in (0,b\wedge \bar{\delta})$, and any open interval $A = (a_1,a_2)$ such that $|a_1|\wedge|a_2|>r - \bar{\epsilon}$ and $|a_1| \neq L, |a_2| \neq L$, we have
\begin{align*}
  \lim_{\eta \downarrow  0} \frac{ p(\epsilon,\delta,\eta;A) }{\delta^\alpha \Big(\frac{H(1/\eta)}{\eta} \Big)^{l^* - 1} } = \mu\big(E(\epsilon) \cap h^{-1}(A) \big)
\end{align*}
where $\mu$ is the measure defined in \cref{defMuMeasure}, and $p(\cdot,\cdot,\cdot; A)$ is the conditional probability defined in \cref{def overflow conditional probability with endpoint}.
\end{lemma}

\begin{proof}
Let us start by fixing some notations. Let $T_1 = 0$, and define stopping times $T_j = \min\{ n > T_{j - 1}: \eta|V_n| > \delta \}$ and inter-arrival times $T^\prime_j = T_j - T_{j - 1}$ for any $j \geq 1$, and large jump $W_j = V_{T_j}$ for any $j \geq 0$. Note that: first of all, the pair $(T^\prime_i,W_i)$ is independent of $(T^\prime_j,W_j)$ whenever $i \neq j$; besides, $W_j$ and $T^\prime_j$ are independent for all $j \geq 1$.

Define the following sequence (of random elements) $\textbf{w} = (w_1,\cdots,w_{l^*})$ and $\textbf{t} = (t_1,\cdots,t_{l^*})$ by
\begin{align*}
    w_j = \eta W_j,\ \ t_j = \eta T_j.
\end{align*}
If $V_0$ has $(\epsilon,\delta,\eta)$-overflow, then the following two events must occur:
\begin{itemize}
    \item $T^\prime_j \leq 2\bar{t}/\eta$ for any $j = 2,\cdots,l^*$;
    \item $\eta|W_j|>\bar{\delta}$ for any $j=1,2,\cdots,l^*$;
    \item $(\textbf{w},\textbf{t})  \in E(\epsilon)$
\end{itemize}
Therefore, for sufficiently small $\eta$, we now have 
\begin{align}
    & p(\epsilon,\delta,\eta) \nonumber \\
    = &\Big( \mathbb{P}(T^\prime_1 \leq 2\bar{t}/\eta ) \Big)^{l^* - 1} \cdot \int \mathbbm{1}\Big\{ ( \textbf{w},\textbf{t} ) \in E(\epsilon) \Big\} \nonumber \\
    &\ \ \ \ \ \ \ \ \ \ \ \ \ \ \ \cdot \mathbb{P}( \eta W_1 = dw_1 | \eta|W_1| > \delta  )\cdots \mathbb{P}( \eta W_{l^*} = dw_{l^*} |\ \eta|W_{l^*}| > \delta )\nonumber \\
    &\ \ \ \ \ \ \ \ \ \ \ \ \ \ \ \cdot \mathbb{P}( \eta T^\prime_2 = dt^\prime_2 |\eta T^\prime_2 \leq 2\bar{t}  )\cdots \mathbb{P}( \eta T^\prime_{l^*} = dt^\prime_{l^*} |\eta T^\prime_{l^*} \leq 2\bar{t} ) \nonumber \\
    = & \Big( \mathbb{P}(T^\prime_1 \leq 2\bar{t}/\eta ) \Big)^{l^* - 1}\cdot  \mathbb{Q}_{\eta,\delta}\big( E(\epsilon)\cap h^{-1}(A) \big)\label{proofCalculatePDeltaEpsilon}
\end{align}
where $\mathbb{Q}_{\eta,\delta}$ is the Borel-measurable probability measure on $\mathbb{R}^{l^*}\times \Big(\mathbb{R}_{+}\Big)^{l^* - 1}$ induced by a sequence of independent random variables $(W^{\uparrow}_1(\eta,\delta),\cdots,W^{\uparrow}_{l^*}(\eta,\delta),T^{\uparrow}_2(\eta,\delta),\cdots,T^{\uparrow}_{l^*}(\eta,\delta) )$ such that
\begin{itemize}
    \item For any $i = 1,\cdots,l^*$, the distribution of $W^{\uparrow}_i(\eta,\delta)$ follows from $\mathbb{P}\Big( \eta W_1 \in \cdot\ \Big|\ \eta|W_{l^*}| > \delta \Big)$;
    \item For any $i = 2,\cdots,l^*$, the distribution of $T^{\uparrow}_i(\eta,\delta)$ follows from $\mathbb{P}\Big( \eta T_1 \in \cdot\ \Big|\ \eta T_1 \leq 2\bar{t} \Big)$;
    \item $ \mathbb{Q}_{\eta,\delta}(\cdot) = \mathbb{P}\Big(  (\eta W^{\uparrow}_1(\eta,\delta),\cdots,\eta W^{\uparrow}_{l^*}(\eta,\delta),\eta T^{\uparrow}_2(\eta,\delta),\cdots,\eta \sum_{j = 2}^{l^*} T^{\uparrow}_{j}(\eta,\delta) )   \in \cdot  \Big)$.
\end{itemize}
Now we study the weak convergence of $W^{\uparrow}_1,T^{\uparrow}_1$:
\begin{itemize}
    \item Due to the regularly varying nature of distribution of $Z_1$ (hence for $W_1$), we know that: for any $x > \delta$,
    $$\lim_{\eta \downarrow 0} \mathbb{P}\Big( \eta W_1 > x \ \Big|\ \eta|W_{l^*}| > \delta \Big) =  p_+\frac{ \delta^\alpha}{x^\alpha}, \ \ \lim_{\eta \downarrow 0} \mathbb{P}\Big( \eta W_1 < -x \ \Big|\ \eta|W_{l^*}| > \delta \Big) = p_-\frac{ \delta^\alpha}{x^\alpha};$$
    therefore, $W^{\uparrow}_1(\eta,\delta)$ weakly converges to a (randomly signed) Pareto RV that admits the density
    $$\nu_{\alpha,\delta}(dx) = \mathbbm{1}\{x > 0\}p_+\frac{\alpha\delta^\alpha}{x^{\alpha + 1}} +  \mathbbm{1}\{x < 0\}p_-\frac{\alpha\delta^\alpha}{|x|^{\alpha + 1}}$$
    as $\eta \downarrow 0$;
    \item For any $x \in [0,2\bar{t}]$, since $\lim_{\eta \downarrow 0}\floor{x/\eta}H(\delta/\eta) = 0$, it is easy to show that
    $$\lim_{\eta \downarrow 0}\frac{1 - (1 - H(\delta/\eta))^{\floor{x/\eta}} }{ \floor{x/\eta}H(\delta/\eta)  } = 1;$$
    therefore, we have (for any $x \in (0,2\bar{t}]$)
    \begin{align*}
        \mathbb{P}(\eta T_1 \leq x\ |\ \eta T_1 \leq 2\bar{t}) & = \frac{ 1 - (1 - H(\delta/\eta))^{\floor{x/\eta}}  }{1 - (1 - H(\delta/\eta))^{\floor{2\bar{t}/\eta}}  } \rightarrow \frac{x}{2\bar{t}}
    \end{align*}
    as $\eta\downarrow 0$, which implies that $T^{\uparrow}_1$ converges weakly to a uniform RV on $[0,2\bar{t}]$.
\end{itemize}
Let us denote the weak limit of measure $\mathbb{Q}_{\eta,\delta}$ as $\mu_{\delta,2\bar{t}}$. In the discussion before the Lemma we have shown that, for any $(\textbf{w},\textbf{t}) \in E(\epsilon)$ (with $\delta \in (0,\bar{\delta})$), we have $|w_i| \geq \bar{\delta}$ and $|t^\prime_i|\leq 2\bar{t}$; since we require $\delta < \bar{\delta}$, by definition of measures $\mu$ and $\mu_{\delta,2\bar{t}}$ we have
$$ \mu_{\delta,2\bar{t}}\big(E(\epsilon)\cap h^{-1}(A) \big) =\frac{ \delta^{\alpha l^*} }{ (2\bar{t})^{l^* - 1} }\cdot\mu\big(E(\epsilon)\cap h^{-1}(A) \big).$$ 

For simplicity of notations, we let $E(\epsilon,A)\delequal{} E(\epsilon)\cap h^{-1}(A)$. By definition of the set $E(\epsilon)$, we have (recall that $A$ is an open interval $(a_1,a_2)$ that does not overlap with $[s_- + \bar{\epsilon},s_+ - \bar{\epsilon}]$)
\begin{align*}
    E(\epsilon,A) & = h^{-1}\big( (-\infty,s_- - \epsilon) \cup (s_+ + \epsilon,\infty) \big) \cap h^{-1}\big( (a_1,a_2) \big)
    \\
    & = h^{-1}\Big( \big( (-\infty,s_- - \epsilon) \cup (s_+ + \epsilon,\infty) \big)\cap(a_1,a_2)  \Big)
    \\
    & = h^{-1}\big( F(\epsilon,a_1,a_2) \big)
\end{align*}
where $F(\epsilon,a_1,a_2) \delequal{} \big( (-\infty,s_- - \epsilon) \cup (s_+ + \epsilon,\infty) \big)\cap(a_1,a_2).$ Meanwhile, it is easy to see that $h$ is a continuous mapping, hence
\begin{align*}
    (\textbf{w},\textbf{t}) \in \partial E(\epsilon,A) \Rightarrow h(\textbf{w},\textbf{t}) \in \{ s_- +\epsilon, s_+ - \epsilon, a_1, a_2\}.
\end{align*}
Fix some $s$ with $s \neq \pm L, |s| > (l^*-1)b + \bar{\epsilon}$. For any fixed real numbers $t_2,\cdots,t_{l^* - 1}$, $w_1,\cdots,w_{l^*}$, if $ h( w_1,\cdots,w_{l^*},t_2,\cdots,t_{l^*-1},t) = s$, then since $\widetilde{\textbf{x}}( t_{l^*}-1,0; w_1,\cdots,w_{l^*},t_2,\cdots,t_{l^*-1},t ) \in [s - b, s + b]$, due to Assumption 1 (in particular, there is no point $x$ on this interval with $|f^\prime(x)| \leq c_0$ ), there exists at most one possible $t$ that makes $ h( w_1,\cdots,w_{l^*},t_2,\cdots,t_{l^*-1},t) = s$. Therefore, let $W^*_j$ be iid RVs from law $\nu_{\alpha,\delta}$ defined above, and $(T^{*,\prime}_j)_{j \geq 2}$ be iid RVs from Unif$[0,\bar{2t}]$, $T^*_0 = 0, T^*_k = \sum_{j = 2}^k T^{*,\prime}_j$. By conditioning on all $W^*_j$ and all $T^{*,\prime}_2,\cdots,T^{*,\prime}_{l^* - 1}$, we must have
\begin{align}
    &\mathbb{P}\Big( h( W^*_1,\cdots,W^*_j,T^*_2,\cdots,T^*_{l^*}  ) = s\ \Big| \ W^*_1 = dw_1,\cdots,W^*_{l^*} = dw_{l^*}, \nonumber
    \\
    &\ \ \ \ \ \ \ \ \ \ \ \ \ \ \ \ \ T^{*,\prime}_2 = dt_2,\cdots, T^{*,\prime}_{l^*-1} = dt_{l^*-1}\Big) = 0 \label{proof overflow conditional prob argument}
\end{align}
which implies
\begin{align*}
     \mathbb{P}\Big( h( W^*_1,\cdots,W^*_j,T^*_2,\cdots,T^*_{l^*}  ) = s\Big) = 0
\end{align*}
hence
\begin{align*}
    \mu\Big( \partial E(\epsilon,A) \Big) = 0.
\end{align*}
By Portmanteau theorem (see Theorem 2.1 of \cite{billingsley2013convergence}) we have
$$\lim_{\eta \downarrow 0} \mathbb{Q}_{\eta,\delta}(E(\epsilon,A)) = \mu_{\delta,2\bar{t}}(E(\epsilon,A)).$$

Collecting the results we have and using \cref{proofCalculatePDeltaEpsilon}, we can see that
\begin{align*}
 & \limsup_{\eta \downarrow 0}\frac{p(\epsilon,\delta,\eta;A)}{ \Big( \frac{H(1/\eta)}{\eta} \Big)^{l^* - 1}\delta^\alpha  } 
 \\
 = & \limsup_{\eta \downarrow 0}\frac{(2\bar{t})^{l^* - 1}\cdot p(\epsilon,\delta,\eta;A)}{\delta^{\alpha l^*} \cdot \Big( \mathbb{P}(T^\prime_1 \leq 2\bar{t}/\eta ) \Big)^{l^* - 1}  }\cdot \Big( \frac{\delta^\alpha}{2\bar{t}}\cdot \frac{\mathbb{P}(T^\prime_1 \leq 2\bar{t}/\eta )}{H(1/\eta)/\eta}  \Big)^{l^* - 1} 
 \\
 \leq & \limsup_{\eta \downarrow 0}\frac{(2\bar{t})^{l^* - 1}\cdot p(\epsilon,\delta,\eta;A)}{\delta^{\alpha l^*} \cdot \Big( \mathbb{P}(T^\prime_1 \leq 2\bar{t}/\eta) \Big)^{l^* - 1}  } \cdot \limsup_{\eta \downarrow 0}\Big( \frac{\delta^\alpha}{2\bar{t}}\cdot \frac{\mathbb{P}(T^\prime_1 \leq 2\bar{t}/\eta )}{H(1/\eta)/\eta}  \Big)^{l^* - 1} 
 \\
 \leq & \mu(E(\epsilon,A))\cdot  \limsup_{\eta \downarrow 0}\Big( \frac{\delta^\alpha}{2\bar{t}}\cdot \frac{\mathbb{P}(T^\prime_1 \leq 2\bar{t}/\eta )}{H(1/\eta)/\eta}  \Big)^{l^* - 1} .
\end{align*}
Fix some $\kappa > 1$. From Lemma \ref{lemmaGeomFront} and the regularly varying nature of function $H$, we get
\begin{align*}
    \limsup_{\eta \downarrow 0}\Big( \frac{\delta^\alpha}{2\bar{t}}\cdot \frac{\mathbb{P}(T^\prime_1 \leq 2\bar{t}/\eta )}{H(1/\eta)/\eta}  \Big)^{l^* - 1} & \leq \kappa^{l^* - 1}\limsup_{\eta \downarrow 0}\Big( \frac{\delta^\alpha}{2\bar{t}}\cdot \frac{2\bar{t}H(\delta/\eta)/\eta}{H(1/\eta)/\eta} \Big)^{l^* - 1} = \kappa^{l^* - 1}.
\end{align*}
Due to the arbitrariness of $\kappa > 1$, we have established that
$$ \limsup_{\eta \downarrow 0}\frac{p(\epsilon,\delta,\eta;A)}{ \Big( \frac{H(1/\eta)}{\eta} \Big)^{l^* - 1}\delta^\alpha  } \leq \mu(E(\epsilon)).$$
The lower bound can be shown by an argument symmetric to the one for upper bound.
\end{proof}

The following result is an immediate corollary of Lemma \ref{lemmaOverflowProb}.

\begin{corollary} \label{corollary overflow transition prob}
For any $\epsilon \in (-\bar{\epsilon},\bar{\epsilon}),\delta \in (0,b\wedge \bar{\delta})$, we have
\begin{align*}
  \lim_{\eta \downarrow  0} \frac{ p(\epsilon,\delta,\eta) }{\delta^\alpha \Big(\frac{H(1/\eta)}{\eta} \Big)^{l^* - 1} } = \mu\big(E(\epsilon)\big)
\end{align*}
where $\mu$ is the measure defined in \cref{defMuMeasure}, and $p(\cdot,\cdot,\cdot)$ is the conditional probability defined in \cref{def overflow conditional probability}.
\end{corollary}

Define the following stopping times:
\begin{align}
    \sigma(\eta) & = \min\{n \geq 0: X^\eta_n \notin \Omega\}; \label{def first exit time sigma} \\
    R(\epsilon,\delta,\eta) & = \min\{n \geq T^\eta_1(\delta): X^\eta_n \in [-2\epsilon,2\epsilon] \}. \label{def first return time R}
\end{align}
$\sigma$ indicate the time that the iterates escape the current attraction field, while $R$ denotes the time the SGD iterates return to a small neighborhood of the local minimum after first exit from this small neighborhood. In the next few results, we study the probability of several atypical scenarios when SGD iterates make attempts to escape $\Omega$ or return to local minimum after the attempt fails. First, we show that, when starting from the local minimum, it is very unlikely to escape with less than $l^*$ big jumps.
\begin{lemma} \label{lemma atypical 1 exit before l star jumps}
Given $\epsilon \in (0,\bar{\epsilon}), N > 0$, the following claim holds for any sufficiently small $\delta > 0$:
\begin{align*}
    \sup_{ x \in [-2\epsilon,2\epsilon]}\mathbb{P}_x\Big( \sigma(\eta) < R(\epsilon,\eta),\ \sigma(\eta)<T^\eta_{l^*}(\delta) \Big) = o(\eta^N)
\end{align*}
as $\eta \downarrow 0$.
\end{lemma}

\begin{proof}
Based on the given $\epsilon > 0$, fix some $\widetilde{\epsilon} = \frac{ \epsilon }{4 \exp(2M\hat{t}(\epsilon)) }$. 
Recall the definition of $\hat{t}(\epsilon)$ in \cref{def function hat t}.

First, using Lemma \ref{lemma stuck at local minimum before large jump}, we know that for sufficiently small $\delta$, we have
\begin{align}
    \sup_{x \in [-2\epsilon,2\epsilon]}\mathbb{P}\Big( A^\times_1(\epsilon,\delta,\eta) \Big) = o(\eta^{N}) \label{proof atypical 1 bound rare event 1}
\end{align}
where
\begin{align*}
    A^\times_1(\epsilon,\delta,\eta) = \Big\{ \exists n < T^\eta_1(\delta)\ s.t.\ |X^\eta_n(x)| > 3\epsilon \Big\}.
\end{align*}

Define event $$A^\times_2(\widetilde{\epsilon},\delta,\eta) \delequal{} \Big\{ \exists j = 2,\cdots,l^*\ s.t.\ \max_{k = 1,2,\cdots,T^\eta_j(\delta) -T^\eta_{j-1}(\delta) - 1 } \eta|Z_{ T^\eta_{j-1}(\delta) + 1 }^{\leq \delta,\eta} + \cdots + Z^{\leq \delta,\eta}_{ T^\eta_{j-1}(\delta) + k }| > \widetilde{\epsilon} \Big\}.$$

From Lemma \ref{lemma prob event A}, we know that for sufficiently small $\delta > 0$,
\begin{align}
    \mathbb{P}\Big( A^\times_2(\widetilde{\epsilon},\delta,\eta)  \Big) = o(\eta^N). \label{proof atypical 1 bound rare event 2}
\end{align}
From now on, we only consider such $\delta$ that \cref{proof atypical 1 bound rare event 1}\cref{proof atypical 1 bound rare event 2} hold.

On event $\Big( A^\times_1 \cup A^\times_2 \Big)^c \cap \{ \sigma(\eta) < R(\epsilon,\eta) \} \cap \{ \sigma(\eta) > T_1^\eta(\delta) \}$, we must have $\sigma(\eta) > T^\eta_1(\delta)$ and 
\begin{align*}
    T^\eta_2(\delta)\wedge \sigma(\eta) - T_1^\eta(\delta) < 2\hat{t}(\epsilon)/\eta.
\end{align*}
Otherwise, due to Lemma \ref{lemma Ode Gd Gap} and \ref{lemma SGD GD gap}, we know that at step $\widetilde{t} = T_1^\eta(\delta) + \floor{\hat{t}(\epsilon)/\eta}$, we have
\begin{align*}
    |X^\eta_{ \widetilde{t} } | < 2\epsilon, \text{ and } |X^\eta_{ n} | \leq \bar{\epsilon} \ \ \forall n \leq \widetilde{t}
\end{align*}
for any sufficiently small $\eta$.
By repeating this argument inductively, we obtain the following result: define
\begin{align*}
    J = \min\{ j = 1,2,\cdots: \sigma(\eta) \in [T^\eta_j(\delta),T^\eta_{j+1}(\delta) ) \},
\end{align*}
then on event $\Big( A^\times_1 \cup A^\times_2 \Big)^c \cap \{ \sigma(\eta) < R(\epsilon,\eta),\ \sigma(\eta)<T^\eta_{l^*}(\delta) \}$, we must have
\begin{align}
    T^\eta_{j}(\delta)\wedge \sigma(\eta) - T_{j-1}^\eta(\delta)\wedge \sigma(\eta) < 2\hat{t}(\epsilon)/\eta \ \ \forall j=2,3,\cdots,J. \label{proof atypical 1 inductive bound}
\end{align}
Furthermore, using this bound and Lemma \ref{lemma SGD GD gap}, we know that on event $\Big( A^\times_1 \cup A^\times_2 \Big)^c \cap \{ \sigma(\eta) < R(\epsilon,\eta),\ \sigma(\eta)<T^\eta_{l^*}(\delta) \}$,
\begin{itemize}
    \item $ |X^\eta_{T^\eta_j(\delta)}| \leq |X^\eta_{T^\eta_{j-1}(\delta)}| + b + \epsilon + \bar{\epsilon} $ for all $j=2,3,J-1$,
    \item $ |X^\eta_{\sigma(\eta)}| \leq |X^\eta_{T^\eta_{J-1}(\delta)}| + \epsilon + \bar{\epsilon} $
\end{itemize}
However, this implies 
\begin{align*}
    |X^\eta_{\sigma(\eta)}| \leq l^*(\bar{\epsilon} + \epsilon) + (l^*-1)b < r
\end{align*}
and contradicts the definition of $\sigma(\eta)$. In summary, 
\begin{align*}
    \sup_{ x \in [-2\epsilon,2\epsilon]}\mathbb{P}_x\Big( \sigma(\eta) < R(\epsilon,\eta),\ \sigma(\eta)<T^\eta_{l^*}(\delta) \Big) \leq \mathbb{P}\Big( A^\times_1(\epsilon,\delta,\eta) \cup A^\times_2(\widetilde{\epsilon},\delta,\eta) \Big) = o(\eta^N).
\end{align*}
\end{proof}

The following two results follow immediately from the proof above, especially the inductive argument leading to bound \cref{proof atypical 1 inductive bound}, and we state them without repeating the deatils of the proof.
\begin{corollary} \label{corollary atypical 1}
Given $\epsilon \in (0,\bar{\epsilon}), N > 0$, the following claim holds for any sufficiently small $\delta > 0$:
\begin{align*}
    & \sup_{ x \in [-2\epsilon,2\epsilon]}\mathbb{P}_x\Big( T^\eta_{l^*}(\delta) \leq \sigma(\eta)\wedge R(\epsilon,\eta),\ \text{and } \exists j = 2,3,\cdots,l^*\ s.t.\ T^\eta_j(\delta) - T^\eta_{j-1}(\delta) > 2\hat{t}(\epsilon)/\eta \Big)
    \\
    = & o(\eta^N)\ \ \ \ \ \text{as }\eta \downarrow 0.
\end{align*}
\end{corollary}

\begin{corollary} \label{corollary atypical 2}
Given $\epsilon \in (0,\bar{\epsilon}), N > 0$, the following claim holds for any sufficiently small $\delta > 0$:
\begin{align*}
    \sup_{ x \in [-2\epsilon,2\epsilon]}\mathbb{P}_x\Big( R(\epsilon,\eta) < T^\eta_{l^*}(\delta)\wedge \sigma(\eta),\ R(\epsilon,\eta) - T^\eta_1(\delta) > 2l^*\hat{t}(\epsilon)/\eta \Big) = o(\eta^N)
\end{align*}
as $\eta \downarrow 0$.
\end{corollary}

In the next result, we show that, if the inter-arrival time between some large jumps are too long, or some large jumps are still not \textit{large enough}, then it is very unlikely that the SGD iterates could escape at the time of $l^*-$th large jump (or even get close enough to the boundary of the attraction field).

\begin{lemma} \label{lemma atypical 2 }
Given $\epsilon \in (0,\bar{\epsilon})$ and any $N>0$, the following holds for all $\delta>0$ that are sufficiently small:
\begin{align*}
    \sup_{ x \in [-2\epsilon,2\epsilon] }\mathbb{P}_x(B^\times_2( \epsilon,\delta,\eta ) ) = o(\eta^N).
\end{align*}
where 
\begin{align*}
    B^\times_2( \epsilon,\delta,\eta ) & = \{ T^\eta_{ l^* }(\delta) \leq \sigma(\eta) \wedge R(\epsilon,\eta)\} \cap \Big\{ \exists j = 2,3,\cdots,l^*\ s.t.\ T^\eta_{j}(\delta) - T^\eta_{j - 1}(\delta) > \bar{t}/\eta \\
    & \ \ \ \ \ \ \ \ \ \ \ \ \ \ \ \ \ \  \text{or }\exists j = 1,2,\cdots,l^*\ s.t.\ \eta|W^\eta_j(\delta)|\leq \bar{\delta} \Big\} \cap \{ |X^\eta_{ T^\eta_{l^*} }| \geq r - \bar{\epsilon} \}.
\end{align*}
\end{lemma}

\begin{proof}

Let $A^\times_1,A^\times_2$ be the events defined in the proof of Lemma \ref{lemma atypical 1 exit before l star jumps}. Based on the given $\epsilon > 0$, fix some $\widetilde{\epsilon} = \frac{ \epsilon }{4 \exp(2M\hat{t}(\epsilon)) }$.

Let $J = \min\{ j = 2,3,\cdots: T^\eta_j(\delta) - T^\eta_{j-1}(\delta) > \bar{t}/\eta\}$. On event $\Big(  A^\times_1(\epsilon,\delta,\eta) \cup A^\times_2(\widetilde{\epsilon},\delta,\eta) \Big)^c \cap \{ J \leq l^*\}$, from Lemma \ref{lemma Ode Gd Gap} and \ref{lemma SGD GD gap} and the definition of constant $\bar{t}$, we know that
\begin{itemize}
    \item $|X^\eta_{T^\eta_j(\delta)}| \leq |X^\eta_{T^\eta_{j-1}(\delta)}| + b + \epsilon + \bar{\epsilon} $ for all $j=2,3,J-1$;
    \item $|X^\eta_{T^\eta_J(\delta)}| \leq 2\bar{\epsilon}$
    \item $|X^\eta_n| < s - \bar{\epsilon}\ \ \forall n \leq T^\eta_J(\delta)$
\end{itemize}
Now starting from step $T^\eta_J(\delta)$, by using Lemma \ref{lemma Ode Gd Gap} and \ref{lemma SGD GD gap} again one can see that
\begin{itemize}
    \item $|X^\eta_{T^\eta_j(\delta)}| \leq |X^\eta_{T^\eta_{j-1}(\delta)}| + b + \epsilon + \bar{\epsilon} $ for all $j= J+1,\cdots,l^*$.
\end{itemize}
Combining these results, we have that $|X^\eta_{ T^\eta_{l^*}  }| < r - \bar{\epsilon}$ on event $\Big(  A^\times_1(\epsilon,\delta,\eta) \cup A^\times_2(\widetilde{\epsilon},\delta,\eta) \Big)^c \cap \{ J \leq l^*\}$.

Next, define $J^\prime = \min\{ j = 1,2,\cdots; \eta|W^\eta_j(\delta)| \leq \bar{\delta} \}$. Similarly, on event $\Big(  A^\times_1(\epsilon,\delta,\eta) \cup A^\times_2(\widetilde{\epsilon},\delta,\eta) \Big)^c \cap \{ J > l^*\} \cap \{ J^\prime \leq l^* \}$, using Lemma \ref{lemma Ode Gd Gap} and \ref{lemma SGD GD gap} again one can see that
\begin{itemize}
    \item $|X^\eta_{T^\eta_j(\delta)}| \leq |X^\eta_{T^\eta_{j-1}(\delta)}| + b + \epsilon + \bar{\epsilon} $ for all $j=1,2,\cdots,l^*, j \neq J^\prime$;
    \item $|X^\eta_{T^\eta_J(\delta)}| \leq |X^\eta_{T^\eta_{J-1}(\delta)}| + \bar{\delta} + \epsilon + \bar{\epsilon} $ for all $j=1,2,\cdots,l^*, j \neq J^\prime$.
\end{itemize}
Since $\bar{\delta}\in (0,\bar{\epsilon})$, we have $ |X^\eta_{ T^\eta_{l^*}  }| < r - \bar{\epsilon}$ on this event.

In summary, the following bound
\begin{align*}
    \sup_{ x\in[-2\epsilon,2\epsilon] } \mathbb{P}_x(B^\times_2) \leq \mathbb{P}( A^\times_1(\epsilon,\delta,\eta) \cup A^\times_2( \widetilde{\epsilon},\delta,\eta ) ) = o(\eta^N)
\end{align*}
holds for any $\delta$ that is sufficiently small, which is established in Lemma \ref{lemma atypical 1 exit before l star jumps}. This conclude the proof.
\end{proof}

In the next lemma, we show that, starting from the local minimum, it is unlikely that the SGD iterates will be right at the boundary of the attraction field after $l^*$ large jumps. Recall that there are $n_\text{min}$ attraction fields on $f$, and excluding $s_0 = -\infty,s_{n_\text{min}} = \infty$ the remaining points $s_1,\cdots,s_{n_\text{min}-1}$ are the boundaries of the attraction fields.

\begin{lemma} \label{lemma atypical 3 }
There exists a function $\Psi(\cdot):\mathbb{R}^+ \mapsto \mathbb{R}^+$ satisfying $\lim_{\epsilon \downarrow 0}\Psi(\epsilon) = 0$ such that the following claim folds. Given $\epsilon \in \big( 0, \bar{\epsilon}/( 3\bar{\rho} + 3\widetilde{\rho} + 9) \big)$, we have
\begin{align*}
    \limsup_{\eta \downarrow 0} \frac{ \sup_{ x \in [-2\epsilon,2\epsilon] }\mathbb{P}_x( B^\times_3(\epsilon,\delta,\eta) ) }{ \Big( H(1/\eta)/\eta \Big)^{l^* - 1}  } \leq \delta^\alpha\Psi(\epsilon) 
\end{align*}
for all $\delta$ sufficiently small, where $\bar{\rho}$ and $\widetilde{\rho}$ are the constants defined in Corollary \ref{corollary ODE GD gap} and \ref{corollary sgd gd gap},  and the event is defined as
\begin{align*}
    & B^\times_3(\epsilon,\delta,\eta)
    \\
    = & \Big\{  T^\eta_{l^*}(\delta) \leq \sigma(\eta)\wedge R(\epsilon,\eta)\Big\}\cap\Big\{\exists k \in [n_\text{min} - 1]\ \text{such that }  X^\eta_{T^\eta_{l^*}(\delta)}\in [s_k - \epsilon,s_k + \epsilon] \Big\}.
\end{align*}
\end{lemma}

\begin{proof}
Let $A^\times_1,A^\times_2$ be the events defined in the proof of Lemma \ref{lemma atypical 1 exit before l star jumps}. Based on the given $\epsilon > 0$, fix some $\widetilde{\epsilon} = \frac{ \epsilon }{4 \exp(2M\hat{t}(\epsilon)) }$. Fix some $N > \alpha l^*$.

Choose $\delta$ small enough so that claim in Lemma \ref{lemma atypical 2 } holds for the $\epsilon$ prescribed. Using the same arguments in Lemma \ref{lemma atypical 2 }, we have the following inclusion of events:
\begin{align*}
    & B^\times_3(\epsilon,\delta,\eta) \cap \Big(  A^\times_1(\epsilon,\delta,\eta) \cup A^\times_2(\widetilde{\epsilon},\delta,\eta) \Big)^c \\
    \subseteq & \Big\{ \forall j = 2,3,\cdots,l^*,\  T^\eta_{j}(\delta) - T^\eta_{j - 1}(\delta) \leq \bar{t}/\eta \Big\} \cap \Big\{\forall j = 1,2,3,\cdots,l^*,\ \eta|W^\eta_1(\delta)|>\bar{\delta} \Big\}.
\end{align*}
Therefore, on event $B^\times_3(\epsilon,\delta,\eta) \cap \Big(  A^\times_1(\epsilon,\delta,\eta) \cup A^\times_2(\widetilde{\epsilon},\delta,\eta) \Big)^c$, we can apply Corollary \ref{corollary ODE GD gap} and \ref{corollary sgd gd gap} and conclude that $Z_{T^\eta_1(\delta)}$ has $(-\bar{\epsilon},\delta,\eta)-$overflow, and its $(-\bar{\epsilon},\delta,\eta)-$overflow endpoint lies $$(s_k - 3( \bar{\rho} + \widetilde{\rho} + 3)\epsilon, s_k + 3( \bar{\rho} + \widetilde{\rho} + 3)\epsilon)$$
for some $k \in [n_\text{min} - 1]$.
Using Lemma \ref{lemmaOverflowProb} and Corollary \ref{corollary overflow transition prob}, we have that (for any sufficiently small $\eta$)
\begin{align*}
    & \mathbb{P}\Big( B^\times_3(\epsilon,\delta,\eta) \cap \Big(  A^\times_1(\epsilon,\delta,\eta) \cup A^\times_2(\widetilde{\epsilon},\delta,\eta) \Big)^c  \Big) \\
    \leq & \delta^\alpha \Big( \frac{H(1/\eta)}{\eta} \Big)^{ l^* - 1 }\cdot\sum_{k = 1}^{ n_\text{min} - 1 }\mu\big( E(-\bar{\epsilon}) \cap h^{-1}\Big(( s_k - \hat{\epsilon},s_k + \hat{\epsilon} )\big)  \Big)
\end{align*}
where $\hat{\epsilon} = 3( \bar{\rho} + \widetilde{\rho} + 3)\epsilon$. Besides, as established in the proof of Lemma \ref{lemma atypical 1 exit before l star jumps}, we have
\begin{align*}
    \mathbb{P}\Big(  A^\times_1(\epsilon,\delta,\eta) \cup A^\times_2(\widetilde{\epsilon},\delta,\eta)  \Big) = o(\eta^N)
\end{align*}
for all sufficiently small $\delta$. In conclusion, we only need to choose
\begin{align*}
    \Psi(\epsilon) & = \sum_{k = 1}^{ n_\text{min} - 1 }\mu\Big( E(-\bar{\epsilon}) \cap h^{-1}\big(( s_k - 3( \bar{\rho} + \widetilde{\rho} + 3)\epsilon,s_k + 3( \bar{\rho} + \widetilde{\rho} + 3)\epsilon )\big)  \Big).
\end{align*}
To conclude the proof, just note that by combining the continuity of measure with the conditional probability argument leading to \cref{proof overflow conditional prob argument}, we can show that $\lim_{\epsilon \downarrow 0}\Psi(\epsilon) = 0$.
\end{proof}

Lastly, we establish the lower bound for the probability of the \textit{most likely} way for SGD iterates to exit the current attraction field: making $l^*$ large jumps in a relatively short period of time. Recall that $\bar{\epsilon}$ is the fixed constant in \cref{def distance attraction field}-\cref{assumption multiple jump epsilon 0 constant 2}.

\begin{lemma}\label{lemma lower bound typical exit} 
Given $\epsilon \in (0,\bar{\epsilon}/3)$, it holds for any sufficiently small $\delta > 0$ such that
\begin{align*}
\liminf_{ \eta \downarrow 0 }\frac{ \inf_{|x| \leq 2\epsilon} \mathbb{P}_x( A^\circ(\epsilon,\delta,\eta) ) }{ \big(H(1/\eta)/\eta\big)^{l^*-1}  } \geq c_*\delta^\alpha
\end{align*}
where the event is defined as
\begin{align*}
    A^\circ(\epsilon,\delta,\eta) & \delequal{}\Big\{ \sigma(\eta) < R(\epsilon,\eta),\ \sigma(\eta) = T^\eta_{l^*}(\delta),\ X^\eta_{T^\eta_{l^*}} \notin [s_- - \epsilon, s_+ +\epsilon]\Big\}
    \\
    &\ \ \ \ \ \ \ \ \ \ \ \ \ \ \ \ \ \ \ \ \ \ \ \cap \Big\{T^\eta_j(\delta) - T^\eta_{j-1}(\delta) \leq \frac{\bar{\epsilon}}{ 2M }\ceil{1/\eta}\ \forall j =2,3,\cdots,l^*   \Big\}
\end{align*}
and the constant
$$c_* = \frac{1}{2}(\frac{1}{2b})^{l^*\alpha}( \frac{\bar{\epsilon}}{4M} )^{l^* - 1}$$
is strictly positive and does not vary with $\epsilon,\delta$.
\end{lemma}

\begin{proof}
Let $A^\times_1,A^\times_2$ be the events defined in the proof of Lemma \ref{lemma atypical 1 exit before l star jumps}. Fix some $N$ such that $N > \alpha l^*$. Based on the given $\epsilon > 0$, fix some $\widetilde{\epsilon} = \frac{ \epsilon }{4 \exp(2M\hat{t}(\epsilon)) }$. We only consider $\delta < M$. Furthermore, choose $\delta$ small enough so that \cref{proof atypical 1 bound rare event 1} and \cref{proof atypical 1 bound rare event 2} hold for the chosen $N$ and $\epsilon$. Also, we only consider $\eta$ small enough so that $\eta M < b\wedge \bar{\epsilon}$.

Due to \cref{def distance attraction field}-\cref{assumption multiple jump epsilon 0 constant 2}, we can,
without loss of generality, assume that $r = s_+$, and in this case we will have
\begin{align*}
    l^*b - 100l^*\bar{\epsilon} > s_+ + 100l^*\bar{\epsilon}.
\end{align*}
Under this assumption, we will now focus on providing a lower bound for the following event that describes the exit from the right side of $\Omega$ (in other words, by crossing $s_+$)
\begin{align*}
    A^{\circ}_{\rightarrow}(\epsilon,\delta,\eta) & \delequal{}\Big\{ \sigma(\eta) < R(\epsilon,\eta),\ \sigma(\eta) = T^\eta_{l^*}(\delta),\ X^\eta_{T^\eta_{l^*}} >s_+ +\epsilon\Big\}
    \\
    &\ \ \ \ \ \ \ \ \ \ \ \ \ \ \ \ \ \ \ \ \ \ \ \cap \Big\{T^\eta_j(\delta) - T^\eta_{j-1}(\delta) \leq \frac{\bar{\epsilon}}{ 2M }\ceil{1/\eta}\ \forall j =2,3,\cdots,l^*   \Big\}.
\end{align*}

First, define event
\begin{align*}
    A^\circ_3(\delta,\eta) = \Big\{ W^\eta_j(\delta) \geq 2b\ \forall j =1,\cdots,l^*,\ T^\eta_j(\delta) - T^\eta_{j-1}(\delta) \leq \frac{\bar{\epsilon}}{2M}\ceil{1/\eta}\ \forall j =2,\cdots,l^*  \Big\},
\end{align*}
and observe some facts on event $A^\circ_3(\delta,\eta) \cap \Big(  A^\times_1(\epsilon,\delta,\eta) \cup A^\times_2(\widetilde{\epsilon},\delta,\eta) \Big)^c$.
\begin{itemize}
    \item $|X^\eta_k| \leq 3\epsilon \forall n < T^\eta_1(\delta)$; (due to $A^\times_1$ not occurring)
    \item $X^\eta_{T^\eta_1(\delta)} \in [b - 3\epsilon,b+3\epsilon]$; (due to $W^\eta_1 \geq 2b$ and the effect of gradient clipping at step $T^\eta_1$, as well as the fact that $X^\eta_{ T^\eta_1 - 1 } \in [-3\epsilon,3\epsilon]$ from the previous bullet point)
    \item Due to $|f^\prime(\cdot)|\leq M$ and $\delta < M$, one can see that (for any $n \geq 1$)
    $$\sup_{x \in [-L,L]}|\eta f^\prime(x)| + |\eta Z_n^{\leq \delta,\eta}| \leq 2\eta M;$$
    this provides an upper bound for the change in SGD iterates at each step, and gives us
    \begin{align*}
        X^\eta_n \in [b - 3\epsilon - \bar{\epsilon}, b + 3\epsilon + \bar{\epsilon}] \ \ \forall T^\eta_1(\delta) \leq n < T^\eta_2(\delta)
    \end{align*}
    where we also used $T^\eta_2(\delta) - T^\eta_{1}(\delta) \leq \frac{\bar{\epsilon}}{2M}\ceil{1/\eta}$
    \item Therefore, at the arrival time of the second large jump, we must have $X^\eta_{T^\eta_2(\delta)} \geq 2b - 3\epsilon - \bar{\epsilon}$;
    \item By repeating the argument above inductively, we can show that (for all $j = 1,2,\cdots,l^*$)
    \begin{align*}
        X^\eta_n &\in [(j-1)b - 3\epsilon - (j-1)\bar{\epsilon},(j-1)b + 3\epsilon + (j-1)\bar{\epsilon} ]\ \ \forall T^\eta_{j-1} \leq n < T^\eta_j \\
        X^\eta_{T^\eta_{j}} &\in [jb - 3\epsilon - (j-1)\bar{\epsilon},jb + 3\epsilon + (j-1)\bar{\epsilon} ];
    \end{align*}
    In particular, we know that $X^\eta_n \in \Omega$ for any $n < T^\eta_{l^*}$ (so the exit does not occur before $T^\eta_{l^*}$), and at the arrival of the $l^*-$th large jump, we have (using $3\epsilon < \bar{\epsilon}$)
    \begin{align*}
        X^\eta_{ T^\eta_{l^*}(\delta) } \geq l^*b - l^*\bar{\epsilon}> s_+ + \epsilon.
    \end{align*}
\end{itemize}
In summary, we have shown that
\begin{align*}
    & A^\circ_3(\delta,\eta) \cap \Big(  A^\times_1(\epsilon,\delta,\eta) \cup A^\times_2(\widetilde{\epsilon},\delta,\eta) \Big)^c \subseteq A^\circ_\rightarrow(\epsilon,\delta,\eta).
\end{align*}
To conclude the proof, just notice that (for sufficiently small $\eta$)
\begin{align*}
    & \mathbb{P}\Big( A^\circ_3(\delta,\eta) \cap \Big(  A^\times_1(\epsilon,\delta,\eta) \cup A^\times_2(\widetilde{\epsilon},\delta,\eta) \Big)^c  \Big) \\
    & \geq \mathbb{P}(A^\circ_3(\delta,\eta)) - \mathbb{P}( A^\times_1(\epsilon,\delta,\eta)  ) - \mathbb{P}( A^\times_2(\widetilde{\epsilon},\delta,\eta) ) \\
    & \geq \mathbb{P}(A^\circ_3(\delta,\eta)) - \eta^N \ \ \ \ \ \text{due to \cref{proof atypical 1 bound rare event 1} and \cref{proof atypical 1 bound rare event 2}} \\
    & \geq \Big(\frac{ H(2b/\eta) }{H(\delta/\eta)}\Big)^{ l^* }\Big( \frac{\bar{\epsilon}}{4M}H(\delta/\eta)/\eta \Big)^{l^* - 1} - \eta^N \ \ \ \ \text{due to Lemma \ref{lemmaGeomFront}} \\
    & \geq 2c_* \delta^\alpha (H(1/\eta)/\eta)^{l^*-1} - \eta^N \ \ \ \ \text{for all $\eta$ sufficiently small, due to $H \in \mathcal{RV}_{-\alpha}$} 
    \\
    & \geq c_* \delta^\alpha (H(1/\eta)/\eta)^{l^*-1}.
\end{align*}
\end{proof}

In order to present the main result of this section, we need to take into account the loss landscape outside of the current attraction field $\Omega$. Recall that there are $n_\text{min}$ attraction fields on $f$. For all the attraction fields different from $\Omega$, we call them $(\widetilde{\Omega}_k)_{k = 1}^{n_\text{min} - 1 }$ where, for each $k \in [ n_\text{min} - 1 ]$, the attraction field $\widetilde{\Omega}_k = (s^-_k,s^+_k)$ with the corresponding local minimum located at $\widetilde{m}_k$. Also, recall that $\sigma(\eta)$ is the first time $X^\eta_n$ exits from $\Omega$. Building upon these concepts, we can define a stopping time
\begin{align}
    \tau(\eta,\epsilon) \delequal{} \min\{ n \geq 0:\ X^\eta_n \in \bigcup_{k=1}^{ n_\text{min} - 1 }[\widetilde{m}_k - 2\epsilon,\widetilde{m}_k+2\epsilon] \} \label{def tau first transition time}
\end{align}
as the first time the SGD iterates visit a minimizer in an attraction field that is different from $\Omega$. Besides, let index $J_\sigma(\eta)$ be such that
\begin{align}
    J_\sigma(\eta) = j \iff X^\eta_{ \sigma(\eta) } \in \widetilde{\Omega}_j\ \ \forall j \in [n_\text{min} - 1]. \label{def J sigma}
\end{align}
In other words, it is the label of the attraction field that $X^\eta_n$ escapes to. Lastly, define
\begin{align}
    \lambda(\eta) & \delequal{} H(1/\eta)\Big( H(1/\eta)/\eta \Big)^{l^* - 1}, \label{def lambda rate}
    \\
    \nu^\Omega & \delequal{} \mu\big(E(0)\big), \label{def nu omega}
    \\
    \nu^\Omega_k & \delequal{} \mu\big( E(0) \cap h^{-1}(\widetilde{\Omega}_k) \big)\ \ \forall k \in [n_\text{min} - 1]. \label{def nu omega k}
\end{align}
For definitions of the measure $\mu$, set $E$, and mapping $h$, see \cref{defSetE} and \cref{defMuMeasure}. 

Now we are ready to state Proposition \ref{proposition first exit time Gradient Clipping}, the most important technical tool in this section. In \cref{goal 1 prop first exit} and \cref{goal 2 prop first exit}, we provide upper and lower bounds for the joint distribution of first exit time $\sigma$ and the label $J_\sigma$ indexing the attraction field we escape to; it is worth noticing that the claims hold uniformly for all $u > C$. In \cref{goal 3 prop first exit} and \cref{goal 4 prop first exit}, we provide upper and lower bounds for the joint distribution of when we first visit a different local minimum (which is equal to $\tau$) and which one we visit (indicated by $X^\eta_{ \tau }$). The similarity between \cref{goal 1 prop first exit} \cref{goal 2 prop first exit} and \cref{goal 3 prop first exit}\cref{goal 4 prop first exit} suggests a strong correlation between the behavior of the SGD iterates at time $\sigma(\eta)$ and that of time $\tau(\eta,\epsilon)$, and this is corroborated by \cref{goal 5 prop first exit}: we show that it is almost always the case that $\tau$ is very close to $\sigma$, and on the short time interval $[\sigma(\eta),\tau(\eta,\epsilon)]$ the SGD iterates stay within the same attraction field.

\begin{proposition} \label{proposition first exit time Gradient Clipping}
Given $C>0$ and some $k^\prime \in [n_\text{min} - 1]$, the following claims hold for all $\epsilon > 0$ that is sufficiently small:
\begin{align}
    \limsup_{\eta \downarrow 0}\sup_{u \in (C,\infty) } \sup_{x \in [-2\epsilon,2\epsilon]}&\mathbb{P}_x\Big( \nu^\Omega \lambda(\eta)\sigma(\eta) > u,\ J_\sigma(\eta) = k^\prime \Big) \nonumber
    \\
    \leq & 2C + \exp( -(1-C)^3u )\frac{ \nu^\Omega_{k^\prime} + C }{ \nu^\Omega }, \label{goal 1 prop first exit}
    \\
     \liminf_{\eta \downarrow 0}\inf_{u \in (C,\infty) } \inf_{x \in [-2\epsilon,2\epsilon]}&\mathbb{P}_x\Big( \nu^\Omega \lambda(\eta)\sigma(\eta) > u,\ J_\sigma(\eta) = k^\prime \Big) \nonumber 
    \\
     \geq & - 2C + \exp( -(1+C)^3 u )\frac{ \nu^\Omega_{k^\prime} - C }{ \nu^\Omega }, \label{goal 2 prop first exit}
     \\
     \limsup_{\eta \downarrow 0}\sup_{u \in (C,\infty) } \sup_{x \in [-2\epsilon,2\epsilon]}&\mathbb{P}_x\Big( \nu^\Omega \lambda(\eta)\tau(\eta,\epsilon) > u,\ X^\eta_{\tau(\eta,\epsilon)} \in B(\widetilde{m}_{k^\prime},2\epsilon) \Big) \nonumber 
     \\
     \leq & 4C + \exp( -(1-C)^3u )\frac{ \nu^\Omega_{k^\prime} + C }{ \nu^\Omega },  \label{goal 3 prop first exit}
    \\
     \liminf_{\eta \downarrow 0}\inf_{u \in (C,\infty) } \inf_{x \in [-2\epsilon,2\epsilon]}&\mathbb{P}_x\Big( \nu^\Omega \lambda(\eta)\tau(\eta,\epsilon) > u,\ X^\eta_{\tau(\eta,\epsilon)} \in B(\widetilde{m}_{k^\prime},2\epsilon)\Big) \nonumber 
     \\
     \geq & - 4C + \exp( -(1+C)^3 u )\frac{ \nu^\Omega_{k^\prime} - C }{ \nu^\Omega }, \label{goal 4 prop first exit}
     \\
     \liminf_{\eta \downarrow 0} \inf_{x \in [-2\epsilon,2\epsilon]}&\mathbb{P}_x\Big( \lambda(\eta)\big( \tau(\eta,\epsilon) - \sigma(\eta) \big) < C,\ X^\eta_n \in \widetilde{\Omega}_{J_\sigma(\eta)} \ \ \forall n \in [\sigma(\eta), \tau(\eta,\epsilon)] \Big) \nonumber 
     \\ 
     \geq & 1-C. \label{goal 5 prop first exit}
\end{align}
\end{proposition}

Before presenting the proof to Proposition \ref{proposition first exit time Gradient Clipping}, we make some preparations. First, we introduce stopping times (for all $k \geq 1$)
\begin{align*}
    \tau_k(\epsilon,\delta,\eta) & = \min\{ n > \widetilde{\tau}_{k-1}(\epsilon,\delta,\eta): \eta|Z_n| > \delta \} \\
    \widetilde{\tau}_k(\epsilon,\delta,\eta) & = \min\{ n \geq \tau_k(\epsilon,\delta,\eta): |X^\eta_n| \leq 2\epsilon  \}
\end{align*}
with the convention that $\tau_0(\epsilon,\delta,\eta) = \widetilde{\tau}_0(\epsilon,\delta,\eta) = 0.$ The intuitive interpretation is as follows. For the fixed $\epsilon$ we treat $[-2\epsilon,2\epsilon]$ as a small neighborhood of the local minimum of the attraction field $\Omega$. All the $\widetilde{\tau}_k$ partitioned the entire timeline into different \textit{attempts} of escaping $\Omega$. The interval $[\widetilde{\tau}_{k-1},\widetilde{\tau}_k]$ can be viewed as the $k-$th \textit{attempt}. If for $\sigma(\eta)$, the first exit time defined in \cref{def first exit time sigma}, we have $\sigma(\eta) > \widetilde{\tau}_k$, then we consider the $k-$th attempt of escape as a \textit{failure} because the SGD iterates returned to this small neighborhood of the local minimum again without exiting the attraction field. On the other hand, the stopping times $\tau_{k-1}$ indicate the arrival time of the first large jump during the $k-$th \textit{attempt}. The proviso that $\widetilde{\tau}_k \geq \tau_{k-1}$ can be interpreted, intuitively, as that an \textit{attempt} is considered failed only if, after some significant efforts to exit (for instance, a large jump) has been observed, the SGD iterates still returned to the small neighborhood $[-2\epsilon,2\epsilon]$. Regarding the notations, we add a remark that when there is no ambiguity we will drop the dependency on $\epsilon,\delta,\eta$ and simply write $\tau_k, \widetilde{\tau}_k$.

To facilitate the characterization of events during each \textit{attempt}, we introduce the following definitions. First, for all $k \geq 1$, let
\begin{align*}
    \textbf{j}_k & \delequal{} \#\{ n = \tau_{k - 1}(\epsilon,\delta,\eta),\tau_{k-1}(\epsilon,\delta,\eta) + 1,\cdots, \widetilde{\tau}_k(\epsilon,\delta,\eta)\wedge \sigma(\eta): \eta|Z_n| > \delta \}
\end{align*}
be the number of large jumps during the $k-$th \textit{attempt}. Two implications of this definition: 
\begin{itemize}
    \item First, for any $k$ with $\sigma(\eta) < \widetilde{\tau}_k$, we have $\textbf{j}_k = 0$. Note that this proposition concerns the dynamics of SGD up until $\sigma(\eta)$, the first time the SGD iterates escaped from $\Omega$, so there is no need to consider an attempt that is after $\sigma(\eta)$, and we will not do so in the analysis below;
    \item Besides, the random variable $\textbf{j}_k$ is measurable w.r.t. $\mathcal{F}_{\widetilde{\tau}_k \wedge \sigma(\eta)}$, the stopped $\sigma-$algebra generated by the stopping time $\widetilde{\tau}_k \wedge \sigma(\eta)$.
\end{itemize}

Furthermore, for each $k = 1,2,\cdots$, let
\begin{align*}
    T_{k,1}(\epsilon,\delta,\eta) & = \tau_{k-1}(\epsilon,\delta,\eta)\wedge \sigma(\eta) , \\
    T_{k,j}(\epsilon,\delta,\eta) & = \min\{ n > T_{k,j-1}(\epsilon,\delta,\eta): \eta|Z_n| > \delta \} \wedge \sigma(\eta) \wedge \widetilde{\tau}_k \ \ \forall j \geq 2, \\
    W_{k,j}(\epsilon,\delta,\eta) & = Z_{T_{k,j}(\epsilon,\delta,\eta)}\ \ \forall j \geq 1
\end{align*}
with the convention $T_{k,0}(\epsilon,\delta,\eta) = \widetilde{\tau}_{k-1}(\epsilon,\delta,\eta)$. 
Note that for any $k \geq 1, j \geq 1$, $T_{k,j}$ is a stopping time. Besides, from the definition of $\textbf{j}_k$, one can see that
\begin{align}
    \widetilde{\tau}_{k-1} + 1 \leq T_{k,j} \leq \widetilde{\tau}_k \wedge \sigma(\eta)\ \ \forall j \in [\textbf{j}_k], \label{proof prop first exit implication of j k counting}
\end{align}
and the sequences $\big( T_{k,j} \big)_{j = 1}^{\textbf{j}_k}$ and $\big( W_{k,j} \big)_{j = 1}^{\textbf{j}_k}$
are the arrival times and sizes of \textit{large} jumps during the $k-$th attempt, respectively. Again, when there is no ambiguity we will drop the dependency on $\epsilon,\delta,\eta$ and simply write $T_{k,j}$ and $W_{k,j}$.

In order to prove Proposition \ref{proposition first exit time Gradient Clipping}, we analyze the most likely scenario that the exit from $\Omega$ would happen. Specifically, we will introduce a series of events with superscript $\times$ or $\circ$, where $\times$ indicates that the event is \textit{atypical} or unlikely to happen and $\circ$ means that it is a \textit{typical} event and is likely to be observed before the first exit from the attraction field $\Omega$. Besides, the subscript $k$ indicates that the event in discussion concerns the dynamics of SGD during the $k-$th attempt. Our goal is to show that for some event $\textbf{A}^\times(\epsilon,\delta,\eta)$ its probability becomes sufficiently small as learning rate $\eta$ tends to 0, so the escape from $\Omega$ almost always occurs in the manner described by $(\textbf{A}^\times(\epsilon,\delta,\eta))^c$. In particular, the definition of this \textit{atypical} scenario $\textbf{A}^\times$ involves the union of some \textit{atypical} events $\textbf{A}^\times_k,\textbf{B}^\times_k$ that occur in the $k-$th attempt. In other words, the intuition of $\textbf{A}^\times$ is that something \textit{abnormal} happened during one of the attempts before the final exit. 

Here is one more comment for the general naming convention of these events. Events with label $\textbf{A}$ often describe the ``efforts'' made in an attempt to get out of $\Omega$ (such as large noises), while those with label $\textbf{B}$ concern how the SGD iterates return to $[-2\epsilon,2\epsilon]$ (and how this attempt fails). For instance, $\textbf{A}^\times_k$ discusses the unlikely scenario before $T_{k,l^*}$, the arrival of the $l^*-$th large jump in this attempt, while $\textbf{B}^\times_k$ in general discusses the abnormal cases after $T_{k,l^*}$ and before the return to $[-2\epsilon,2\epsilon]$. On the other hand, $\textbf{A}^\circ_k$ describes a successful escape during $k-$th attempt, while $\textbf{B}^\circ_k$ means that during this attempt the iterates return to without spending too much time.

Now we proceed and provide a formal definition and analysis of the aforementioned series of events. As building blocks, we inspect the process $(X^\eta_n)_{n \geq 1}$ at a even finer granularity, and bound the probability of some events $(\textbf{A}^\times_{k,i})_{i \geq 0},\ (\textbf{B}^\times_{k,i})_{i \geq 1}$ detailing several cases that are \textit{unlikely} to occur during the escape from or return to local minimum in the $k-$th attempt. First, for each $k \geq 1$, define the event
\begin{align}
    & \textbf{A}^\times_{k,0}(\epsilon,\delta,\eta) \delequal{} \Big\{\exists i = 0,1,\cdots,l^*\wedge \textbf{j}_k\ s.t. \nonumber 
    \\
    & \ \ \ \ \ \ \ \ \ \ \ \ \  \max_{j = T_{k,i}+1,\cdots,(T_{k,i+1} - 1)\wedge \widetilde{\tau}_k \wedge \sigma(\eta) } \eta|Z_{ T_{k,i}+1}^{\leq \delta,\eta} + \cdots + Z_j^{\leq \delta,\eta}| > \frac{\epsilon}{3\bar{\rho} + 3\widetilde{\rho} + 3}  \Big\}. \label{def A cross k 0 in prop first exit} 
\end{align}
Intuitively speaking, the event characterizes the atypical scenario where, during the $k-$th attempt, there is some large fluctuations (compared to $\widetilde{\epsilon}$) between any of the first $l^*$ large jumps (or the first $\textbf{j}_k$ large jumps in case that $\textbf{j}_k < l^*$). Similarly, consider event (for all $k \geq 1$)
\begin{align}
    \textbf{A}^\times_{k,1}(\epsilon,\delta,\eta) \delequal{}\Big\{ \sigma(\eta)<\widetilde{\tau}_k,\ \textbf{j}_k < l^* \Big\} \label{def A cross k 1 in prop first exit}
\end{align}
that describes the atypical case where the exit occurs during the $k-$th attempt with less than $l^*$ large jumps. Next, for all $k \geq 1$ we have another atypical event (note that from \cref{proof prop first exit implication of j k counting} we can see that, for any $j \geq 1$, $\textbf{j}_k \geq j$ implies $T_{k,j} \leq \sigma(\eta) \wedge \widetilde{\tau}_k$)
\begin{align}
    \textbf{A}^\times_{k,2} \delequal{}\Big\{\textbf{j}_k \geq l^*,\ \exists j = 2,3,\cdots,l^*\ s.t.\ T_{k,j} - T_{k,j-1} > 2\hat{t}(\epsilon)/\eta \Big\}. \label{def A cross k 2 in prop first exit}
\end{align}
representing the case where we have at least $l^*$ large noises during the $k-$th attempt, but for some of the large noise (from the 2nd to the $l^*$-th), the inter-arrival time is unusually long. Moving on, we consider the following events (defined for all $k \geq 1$)
\begin{align}
    \textbf{A}^\times_{k,3} \delequal{}\Big\{\textbf{j}_k < l^*,\ \widetilde{\tau}_k < \sigma(\eta),\ \widetilde{\tau}_k - T_{k,1} > 2l^*\hat{t}(\epsilon)/\eta \Big\}  \label{def A cross k 3 in prop first exit}
\end{align}
that describes the atypical case where the $k-$th attempt failed but the return to the small neighborhood $[-2\epsilon,2\epsilon]$ took unusually long time.

The following event also concerns the scenario where there are at least $l^*$ large noises during the $k-$th attempt:
\begin{align}
    \textbf{A}^\times_{k,4}\delequal{}\Big\{\textbf{j}_k \geq l^*,\ |X^\eta_{ T_{k,l^*} }| \geq r - \bar{\epsilon},\ \exists j = 1,2,\cdots,l^*\ s.t.\ \eta|W_{k,j}| \leq \bar{\delta} \Big\}; \label{def A cross k 4 in prop first exit}
\end{align}
specifically, it describes the atypical case where, during this attempt, right after the $l^*-$ large noise the SGD iterate is far enough from the local minimum yet some of the large noises are not that \textit{large}.
Lastly, by defining events
\begin{align}
    \textbf{A}^\times_{k,5} \delequal{}\Big\{\textbf{j}_k \geq l^*,\ T_{k,l^*} \leq \sigma(\eta)\wedge \widetilde{\tau}_k,\ X^\eta_{T_{k,l^*}} \in \bigcup_{ j \in [n_\text{min} - 1] } [s_j - \epsilon,s_j + \epsilon]\Big\}, \label{def event A cross k 5}
\end{align}
we analyze an atypical case where the SGD iterates arrive at somewhere too close to the boundaries of $\Omega$ at the arrival time of the $l^*$ large noise during this attempt. 
As an amalgamation of these atypical scenarios, we let
\begin{align}
    \textbf{A}^\times_k(\epsilon,\delta,\eta) \delequal{} \bigcup_{i = 0}^5 \textbf{A}^\times_{k,i}(\epsilon,\delta,\eta). \label{def event A cross union}
\end{align}
Also, we analyze the probability of some events $(\textbf{B}^\times_k)_{k \geq 1}$ that concern the SGD dynamics after the $l^*-$th large noise during the $k-$th attempt. Let us define
\begin{align}
    \textbf{B}^\times_{k,1}(\epsilon,\delta,\eta) &\delequal{} \Big\{\textbf{j}_k \geq l^*,\ X^\eta_{ T_{k,l^*} }\in [s_- + \epsilon,s_+ - \epsilon],\ T_{k,j} - T_{k,j-1} \leq 2\frac{\hat{t}(\epsilon)}{\eta}\ \forall j = 2,3,\cdots,l^*  \Big\} \nonumber
    \\ 
    \textbf{B}^\times_{k,2}(\epsilon,\delta,\eta) & \delequal{} \{ \widetilde{\tau}_k - T_{k,l^*} > \rho(\epsilon)/\eta \} \cup  \{ \sigma(\eta) < \widetilde{\tau}_k \} \nonumber
    \\
    \textbf{B}^\times_k(\epsilon,\delta,\eta) & \delequal{} \textbf{B}^\times_{k,1} \cap \textbf{B}^\times_{k,2} \label{proof prop first exit def B cross k}
\end{align}
where $\rho(\cdot)$ is the function in Lemma \ref{lemma return to local minimum quickly}. From the definition of $\textbf{B}^\times_k$, in particular the inclusion of $\textbf{B}^\times_{k,2}$, one can see that the intuitive interpretation of event $\textbf{B}^\times_k$ is that the SGD iterates \textit{did not return} to local minimum efficiently (or simply escaped from the attraction field) after the $l^*-$th large noise during the $k-$th attempt.
In comparison, the following events will characterize what would typically happen during each attempt:
 \begin{align}
        \textbf{A}^\circ_k(\epsilon,\delta,\eta) &\delequal{} \{\textbf{j}_k \geq l^*,\ \sigma(\eta) = T_{k,l^*},\ X^\eta_{T_{k,l^*}} \notin [s_- -\epsilon, s_+ + \epsilon], \nonumber 
        \\
        & \ \ \ \ \ \ \ \ \ \ \ \ \ \ \ \ \ \ \ T_{k,j} - T_{k,j-1} \leq \frac{2\hat{t}(\epsilon)}{\eta}\ \forall j = 2,3,\cdots,l^* \}, \label{def event A circ k prop first exit}
        \\ 
        \textbf{B}^\circ_k(\epsilon,\delta,\eta) &\delequal{} \{ \sigma(\eta) > \widetilde{\tau}_k,\ \widetilde{\tau}_k - T_{k,1} \leq \frac{2l^*\hat{t}(\epsilon) + \rho(\epsilon)}{\eta} \}. \label{proof def event B circ k}
    \end{align}
Intuitively speaking, $\textbf{A}^\circ_k$ tells us that the exit happened right at $T_{k,l^*}$, the arrival time of the $l^*-$th large noise during the $k-$th attempt, and $\textbf{B}^\circ_k$ tells us that the first exit from $\Omega$ did not occur during the $k-$th attempt, and the SGD iterates returned to local minimum rather efficiently. All the preparations above allow use to define
\begin{align}
    \textbf{A}^\times(\epsilon,\delta,\eta) \delequal{} \bigcup_{k \geq 1}\bigg( \bigcap_{i = 1}^{k-1}\big( \textbf{A}^\times_i \cup \textbf{B}^\times_i \cup \textbf{A}^\circ_i \big)^c    \bigg) \cap \big(\textbf{A}^\times_k \cup \textbf{B}^\times_k \big). \label{def event A cross final}
\end{align}

We need the next lemma in the proof of Proposition \ref{proposition first exit time Gradient Clipping}. As mentioned earlier, the takeaway is that $\textbf{A}^\times$ is indeed \textit{atypical} in the sense that we will almost always observe $(\textbf{A}^\times)^c$.

\begin{lemma} \label{lemma bound on event A cross final}
Given any $C > 0$, the following claim holds for all $\epsilon>0,\delta>0$ sufficiently small:
\begin{align*}
    \limsup_{\eta\downarrow 0} \sup_{|x| \leq 2\epsilon} \mathbb{P}_x(\textbf{A}^\times(\epsilon,\delta,\eta)) < C.
\end{align*}
\end{lemma}

\begin{proof}

We fix some parameters for the proof. First, with out loss of generality we only consider $C \in (0,1)$, and we fix some $N > \alpha l^*$. Next we discuss the valid range of $\epsilon$ for the claim to hold. We only consider $\epsilon > 0$ such that
\begin{align*}
    \epsilon < \frac{\bar{\epsilon}}{ 6(\bar{\rho}+\widetilde{\rho} + 3) }\wedge \frac{\epsilon_0}{3}
\end{align*}
where $\bar{\rho}$ and $\widetilde{\rho}$ are the constants in Corollary \ref{corollary ODE GD gap} and Corollary \ref{corollary sgd gd gap} respectively, and $\epsilon_0$ is the constant in \cref{assumption detailed function f at critical point}.  Moreover, recall function $\Psi$ in Lemma \ref{lemma atypical 3 } and the constant $c_* > 0$ in Lemma \ref{lemma lower bound typical exit}. Due to $\lim_{\epsilon \downarrow 0}\Psi(\epsilon) = 0$, it holds for all $\epsilon$ small enough such that 
    \begin{align}
        \frac{ 3\Psi(\epsilon) }{ c_* } & < C \label{proof prop choose epsilon 3}
    \end{align}
In our proof we only consider $\epsilon$ small enough so the inequality above holds, and the claim in Lemma \ref{lemma return to local minimum quickly} holds. Now we specify the valid range of parameter $\delta$ that will be used below:
\begin{itemize}
    \item For all sufficiently small $\delta > 0$, the claim in Lemma \ref{lemma stuck at local minimum before large jump} will hold for the prescribed $\epsilon$ and with $N_0 = N$;
    \item For all sufficiently small $\delta > 0$, the claims in Lemma \ref{lemma atypical 1 exit before l star jumps}, Corollary \ref{corollary atypical 1}, Corollary \ref{corollary atypical 2} and Lemma \ref{lemma atypical 2 } will hold with the prescribed $\epsilon$ and $N$;
    \item For all sufficiently small $\delta > 0$, the inequalities in Lemma \ref{lemma atypical 3 } and \ref{lemma lower bound typical exit} will hold for the $\epsilon$ we fixed at the beginning.
\end{itemize}
We show that the claim holds for any $\epsilon,\delta$ small enough to satisfy the conditions above.

First, recall that
\begin{align*}
    & \textbf{A}^\times_{k,0}(\epsilon,\delta,\eta) \delequal{} \Big\{\exists i = 0,1,\cdots,l^*\wedge \textbf{j}_k\ s.t. 
    \\
    & \ \ \ \ \ \ \ \ \ \ \ \ \  \max_{j = T_{k,i}+1,\cdots,(T_{k,i+1} - 1)\wedge \widetilde{\tau}_k \wedge \sigma(\eta) } \eta|Z_{ T_{k,i}+1}^{\leq \delta,\eta} + \cdots + Z_j^{\leq \delta,\eta}| > \frac{\epsilon}{3\bar{\rho} + 3\widetilde{\rho} + 3}  \Big\}.
\end{align*}
Due to our choice of $\delta$ stated earlier and Lemma \ref{lemma prob event A}, there exists some $\eta_0 > 0$ such that for all $\eta \in (0,\eta_0)$,
\begin{align}
    \mathbb{P}\big(\textbf{A}^\times_{k,0}(\epsilon,\delta,\eta)  \big) \leq \eta^N \ \ \forall k \geq 1. \label{proof prop first exit A cross k 0 bound}
\end{align}
Similarly, recall that $\textbf{A}^\times_{k,1}(\epsilon,\delta,\eta) \delequal{}\Big\{ \sigma(\eta)<\widetilde{\tau}_k,\ \textbf{j}_k < l^* \Big\}.$
Let us temporarily focus on the first attempt (namely the case $k = 1$). From Lemma \ref{lemma atypical 1 exit before l star jumps} and our choice of $\epsilon$ and $\delta$, we know the existence of some $\eta_1 > 0$ such that
\begin{align}
    \sup_{|x|\leq 2\epsilon}\mathbb{P}_x\big(\textbf{A}^\times_{1,1}(\epsilon,\delta,\eta)  \big) \leq \eta^N \ \ \forall \eta \in (0,\eta_1). \label{proof prop first exit A cross k 1 bound}
\end{align}
Next, for $\textbf{A}^\times_{k,2} \delequal{}\Big\{\textbf{j}_k \geq l^*,\ \exists j = 2,3,\cdots,l^*\ s.t.\ T_{k,j} - T_{k,j-1} > 2\hat{t}(\epsilon)/\eta \Big\},$
from Corollary \ref{corollary atypical 1} and our choice of $\delta$ at the beginning, we have the existence of some $\eta_2 > 0$ such that
\begin{align}
     \sup_{|x|\leq 2\epsilon}\mathbb{P}_x\big(\textbf{A}^\times_{1,2}(\epsilon,\delta,\eta)  \big)\leq \eta^N \ \ \forall \eta \in (0,\eta_2). \label{proof prop first exit A cross k 2 bound}
\end{align}
Moving on, for $\textbf{A}^\times_{k,3} \delequal{}\Big\{\textbf{j}_k < l^*,\ \widetilde{\tau}_k < \sigma(\eta),\ \widetilde{\tau}_k - T_{k,1} > 2l^*\hat{t}(\epsilon)/\eta \Big\}$,
 due to Corollary \ref{corollary atypical 2} and our choice of $\epsilon,\delta$, we have the existence of some $\eta_3 > 0$ such that
\begin{align}
     \sup_{|x|\leq 2\epsilon}\mathbb{P}_x\big(\textbf{A}^\times_{1,3}(\epsilon,\delta,\eta)  \big)\leq \eta^N \ \ \forall \eta \in (0,\eta_3). \label{proof prop first exit A cross k 3 bound}
\end{align}
As for $\textbf{A}^\times_{k,4}\delequal{}\Big\{\textbf{j}_k \geq l^*,\ |X^\eta_{ T_{k,l^*} }| \geq r - \bar{\epsilon},\ \exists j = 1,2,\cdots,l^*\ s.t.\ \eta|W_{k,j}| \leq \bar{\delta} \Big\},$
from Lemma \ref{lemma atypical 2 }, one can see the existence of $\eta_4 > 0$ such that
\begin{align}
    \sup_{|x|\leq 2\epsilon}\mathbb{P}_x\big(\textbf{A}^\times_{1,4}(\epsilon,\delta,\eta)  \big)\leq \eta^N \ \ \forall \eta \in (0,\eta_4). \label{proof prop first exit A cross k 4 bound}
\end{align}
Lastly, for $\textbf{A}^\times_{k,5} \delequal{}\Big\{\textbf{j}_k \geq l^*,\ T_{k,l^*} \leq \sigma(\eta)\wedge \widetilde{\tau}_k,\ X^\eta_{T_{k,l^*}} \in \bigcup_{ j \in [n_\text{min} - 1] } [s_j - \epsilon,s_j + \epsilon]\Big\},$
from Lemma \ref{lemma atypical 3 } we see the existence of $\eta_5 > 0$ such that 
\begin{align}
    \sup_{|x|\leq 2\epsilon}\mathbb{P}_x\big(\textbf{A}^\times_{1,5}(\epsilon,\delta,\eta)  \big)\leq 2\delta^\alpha \Psi(\epsilon) \Big( H(1/\eta)/\eta \Big)^{l^* - 1}\ \ \forall \eta \in (0,\eta_5). \label{proof prop first exit A cross k 5 bound}
\end{align}

Recall that $\textbf{A}^\times_k(\epsilon,\delta,\eta) = \bigcup_{i = 0}^5\textbf{A}^\times_{k,i}(\epsilon,\delta,\eta)$. Also, for definitions of $\textbf{B}^\times_k, \textbf{A}^\circ_k,\textbf{B}^\circ_k$, see \cref{proof prop first exit def B cross k},\cref{def event A circ k prop first exit},\cref{proof def event B circ k}  respectively. Our next goal is to establish bounds regarding the probabilities of these events. First, if we consider the event $\bigcap_{j = 1}^{k}(\textbf{A}^\times_j \cup \textbf{B}^\times_j)^c \cap\textbf{B}^\circ_j$, then the inclusion of the $(\textbf{B}^\circ_j)_{j = 1}^k$ implies that during the first $k$ attempts the SGD iterates have never left the attraction field, so
\begin{align*}
    \bigcap_{j = 1}^{k}(\textbf{A}^\times_j \cup \textbf{B}^\times_j)^c \cap\textbf{B}^\circ_j = \big(\bigcap_{j = 1}^{k}(\textbf{A}^\times_j \cup \textbf{B}^\times_j)^c \cap\textbf{B}^\circ_j\big)\cap \{ \sigma(\eta) > \widetilde{\tau}_k \}.
\end{align*}
Next, note that
\begin{align*}
    & \mathbb{P}_x\big( \textbf{B}^\times_k\ |\ \bigcap_{j = 1}^{k-1}(\textbf{A}^\times_j \cup \textbf{B}^\times_j)^c \cap\textbf{B}^\circ_j  \big) 
    \\
    = &  \mathbb{P}_x\big( \textbf{B}^\times_{k,1}\ |\ \bigcap_{j = 1}^{k-1}(\textbf{A}^\times_j \cup \textbf{B}^\times_j)^c \cap\textbf{B}^\circ_j  \big) \mathbb{P}_x\Big( \textbf{B}^\times_{k,2}\ \Big|\ \big(\bigcap_{j = 1}^{k-1}(\textbf{A}^\times_j \cup \textbf{B}^\times_j)^c \cap\textbf{B}^\circ_j\big)\cap \textbf{B}^\times_{k,1}  \Big) 
    \\
    \leq & \mathbb{P}_x\big( \textbf{j}_k \geq l^*,\ T_{k,j} - T_{k,j-1} \leq \frac{2\hat{t}(\epsilon)}{\eta}\ \forall j = 2,3,\cdots,l^*\ |\ \bigcap_{j = 1}^{k-1}(\textbf{A}^\times_j \cup \textbf{B}^\times_j)^c \cap\textbf{B}^\circ_j  \big)
    \\
    &\ \ \ \ \ \ \ \ \ \ \ \ \cdot \mathbb{P}_x\Big( \textbf{B}^\times_{k,2}\ \Big|\ \big(\bigcap_{j = 1}^{k-1}(\textbf{A}^\times_j \cup \textbf{B}^\times_j)^c \cap\textbf{B}^\circ_j\big)\cap \textbf{B}^\times_{k,1}  \Big).
\end{align*}
From the definition of the events $\textbf{A}^\times_j, \textbf{B}^\times_j, \textbf{B}^\circ_j$, one can see that $\bigcap_{j = 1}^{k-1}(\textbf{A}^\times_j \cup \textbf{B}^\times_j)^c \cap\textbf{B}^\circ_j  \in \mathcal{F}_{ \widetilde{\tau}_{k-1} \wedge \sigma(\eta) }$, and on this event $\bigcap_{j = 1}^{k-1}(\textbf{A}^\times_j \cup \textbf{B}^\times_j)^c \cap\textbf{B}^\circ_j$ we have $\sigma(\eta) > \widetilde{\tau}_{k-1}$. So by applying strong Markov property at stopping time $\widetilde{\tau}_{k-1} \wedge \sigma(\eta)$, we have
\begin{align*}
    & \mathbb{P}_x\big( \textbf{B}^\times_k\ |\ \bigcap_{j = 1}^{k-1}(\textbf{A}^\times_j \cup \textbf{B}^\times_j)^c \cap\textbf{B}^\circ_j  \big) 
    \\
    \leq & \mathbb{P}\Big( T^\eta_j(\delta) - T^\eta_{j-1}(\delta) \leq 2\hat{t}(\epsilon)/\eta\ \forall j \in [l^* - 1] \Big)   \cdot \mathbb{P}_x\Big( \textbf{B}^\times_{k,2}\ \Big|\ \big(\bigcap_{j = 1}^{k-1}(\textbf{A}^\times_j \cup \textbf{B}^\times_j)^c \cap\textbf{B}^\circ_j\big)\cap \textbf{B}^\times_{k,1}  \Big)
    \\
    \leq & 2\Big( H(\delta/\eta)\hat{t}(\epsilon)/\eta \Big)^{l^* - 1}\cdot \mathbb{P}_x\Big( \textbf{B}^\times_{k,2}\ \Big|\ \big(\bigcap_{j = 1}^{k-1}(\textbf{A}^\times_j \cup \textbf{B}^\times_j)^c \cap\textbf{B}^\circ_j\big)\cap \textbf{B}^\times_{k,1}  \Big)
    \\
    &\ \ \text{(for all $\eta$ sufficiently small due to Lemma \ref{lemmaGeomFront})  }
    \\
    \leq & 4\Big( \frac{\hat{t}(\epsilon)}{\delta^\alpha} \Big)^{l^* - 1}\Big( \frac{H(1/\eta)}{\eta} \Big)^{l^* - 1}\cdot \mathbb{P}_x\Big( \textbf{B}^\times_{k,2}\ \Big|\ \big(\bigcap_{j = 1}^{k-1}(\textbf{A}^\times_j \cup \textbf{B}^\times_j)^c \cap\textbf{B}^\circ_j\big)\cap \textbf{B}^\times_{k,1}  \Big)
\end{align*}
for all sufficiently small $\eta$, due to $H \in \mathcal{RV}_{-\alpha}(\eta)$.
 Meanwhile, note that 
 \begin{itemize}
     \item  $\big(\bigcap_{j = 1}^{k-1}(\textbf{A}^\times_j \cup \textbf{B}^\times_j)^c \cap\textbf{B}^\circ_j\big)\cap \textbf{B}^\times_{k,1} \in \mathcal{F}_{ T_{k,l^*}  }$;
     \item on this event $\big(\bigcap_{j = 1}^{k-1}(\textbf{A}^\times_j \cup \textbf{B}^\times_j)^c \cap\textbf{B}^\circ_j\big)\cap \textbf{B}^\times_{k,1}$ we have $\sigma(\eta)\wedge \widetilde{\tau}_k > T_{k,l^*}$ and $X^\eta_{ T_{k,l^*} }\in [s_- + \epsilon,s_+ - \epsilon]$.
 \end{itemize}
Therefore, using Lemma \ref{lemma return to local minimum quickly} and strong Markov property again (at stopping time $T_{k,l^*}$), we know the following inequality holds for all $\eta$ sufficiently small:
    \begin{align*}
        &\sup_{k \geq 1} \sup_{|x| \leq 2\epsilon}\mathbb{P}_x\Big( \textbf{B}^\times_{k,2}\ \Big|\ \big(\bigcap_{j = 1}^{k-1}(\textbf{A}^\times_j \cup \textbf{B}^\times_j)^c \cap\textbf{B}^\circ_j\big)\cap \textbf{B}^\times_{k,1}  \Big)
        \\
        = & \sup_{k \geq 1} \sup_{|x| \leq 2\epsilon}\mathbb{P}_x\Big( \big\{  \sigma(\eta) > \widetilde{\tau}_k,\ \widetilde{\tau}_k - T_{k,l^*} \leq \rho(\epsilon)/\eta \big\}^c\ \Big|\ \big(\bigcap_{j = 1}^{k-1}(\textbf{A}^\times_j \cup \textbf{B}^\times_j)^c \cap\textbf{B}^\circ_j\big)\cap \textbf{B}^\times_{k,1}  \Big)
        \\
        \leq & \Psi(\epsilon)\frac{\delta^\alpha}{4\big( \hat{t}(\epsilon)/\delta^\alpha \big)^{l^* - 1}  }.
    \end{align*}
Therefore, we know the existence of some $\eta_6 > 0$ such that
\begin{align}
    \sup_{|x| \leq 2\epsilon}\mathbb{P}_x\big( \textbf{B}^\times_k\ |\ \bigcap_{j = 1}^{k-1}(\textbf{A}^\times_j \cup \textbf{B}^\times_j)^c \cap\textbf{B}^\circ_j  \big) \leq  \Psi(\epsilon)  \delta^\alpha \Big( \frac{H(1/\eta)}{\eta} \Big)^{l^* - 1}\ \ \forall \eta \in (0,\eta_6),\ \forall k \geq 1. \label{proof prop first exit bound B cross k}
\end{align}

Similarly, we can bound conditional probabilities of the form $\mathbb{P}_x\big( \textbf{A}^\times_k\ |\ \bigcap_{j = 1}^{k-1}(\textbf{A}^\times_j \cup \textbf{B}^\times_j)^c \cap\textbf{B}^\circ_j  \big)$. To be specific, recall that $\textbf{A}^\times_k = \cup_{i = 0}^5\textbf{A}^\times_{k,i}$. By combining \cref{proof prop first exit A cross k 0 bound}-\cref{proof prop first exit A cross k 5 bound} with Markov property, we know the existence of some $\eta_7 > 0$ such that
\begin{align}
     \sup_{|x| \leq 2\epsilon}\mathbb{P}_x\big( \textbf{A}^\times_k\ |\ \bigcap_{j = 1}^{k-1}(\textbf{A}^\times_j \cup \textbf{B}^\times_j)^c \cap\textbf{B}^\circ_j  \big) \leq  5\eta^N + 2\Psi(\epsilon)\delta^\alpha \Big( \frac{H(1/\eta)}{\eta} \Big)^{l^* - 1}\ \ \forall \eta \in (0,\eta_7),\ \forall k \geq 1.  \label{proof prop first exit bound A cross k}
\end{align}

On the other hand, a lower bound can be established for conditional probability involving $\textbf{A}^\circ_k$, the event defined in \cref{def event A circ k prop first exit} describing the exit from $\Omega$ during an attempt with exactly $l^*$ large noises. Using Lemma \ref{lemma lower bound typical exit} and Markov property of $(X^\eta_n)_{n \geq 1}$, one can see the existence of some $\eta_8 > 0$ such that
    \begin{align}
        \inf_{|x| \leq 2\epsilon}\mathbb{P}_x\big(\textbf{A}^\circ_{k}\ |\ \bigcap_{j = 1}^{k-1}(\textbf{A}^\times_j \cup \textbf{B}^\times_j)^c \cap\textbf{B}^\circ_j  \big) \geq c_*\delta^\alpha \Big( \frac{H(1/\eta)}{\eta} \Big)^{l^* - 1}\ \ \forall \eta \in (0,\eta_8).  \label{proof prop first exit bound A circ k}
    \end{align}
In order to apply the bounds \cref{proof prop first exit bound B cross k}-\cref{proof prop first exit bound A circ k}, we make use of the following inclusion relationship:
\begin{align}
    \big(\bigcap_{j = 1}^{k-1}(\textbf{A}^\times_j \cup \textbf{B}^\times_j)^c \cap\textbf{B}^\circ_j\big)\cap ( \textbf{A}^\times_k \cup \textbf{B}^\times_k )^c \subseteq \textbf{A}^\circ_k \cup \textbf{B}^\circ_k. \label{proof prop first exit inclusion A B x o}
\end{align}
To see why this is true, let us consider a decomposition of the event on the L.H.S. of \cref{proof prop first exit inclusion A B x o}. As mentioned above, on event $\bigcap_{j = 1}^{k-1}(\textbf{A}^\times_j \cup \textbf{B}^\times_j)^c \cap\textbf{B}^\circ_j$ we know that $\sigma(\eta) > \widetilde{\tau}_{k-1}$, so the $k-$th attempt occurred and there are only three possibilities on this event:
\begin{itemize}
    \item $\textbf{j}_k < l^*$;
    \item $\textbf{j}_k \geq l^*,\ X^\eta_{T_{k,l^*}} \notin \Omega$;
    \item $\textbf{j}_k \geq l^*,\ X^\eta_{T_{k,l^*}} \in \Omega$.
\end{itemize}
Let us partition the said event accordingly and analyze them one by one.
\begin{itemize}
    \item On $\big(\bigcap_{j = 1}^{k-1}(\textbf{A}^\times_j \cup \textbf{B}^\times_j)^c \cap\textbf{B}^\circ_j\big)\cap ( \textbf{A}^\times_k \cup \textbf{B}^\times_k )^c \cap \{ \textbf{j}_k < l^* \}$, due to the exclusion of $\textbf{A}^\times_k$ (especially $\textbf{A}^\times_{k,1}$ and $\textbf{A}^\times_{k,3}$), we can see that if $\textbf{j}_k < l^*$, then we must have $\sigma(\eta) > \widetilde{\tau}_k$ and $\widetilde{\tau}_k - T_{k,1} \leq 2l^*\hat{t}(\epsilon)/\eta$. Therefore,
    \begin{align*}
        \big(\bigcap_{j = 1}^{k-1}(\textbf{A}^\times_j \cup \textbf{B}^\times_j)^c \cap\textbf{B}^\circ_j\big)\cap ( \textbf{A}^\times_k \cup \textbf{B}^\times_k )^c \cap \{ \textbf{j}_k < l^* \} \subseteq \textbf{B}^\circ_k.
    \end{align*}
    \item On $\big(\bigcap_{j = 1}^{k-1}(\textbf{A}^\times_j \cup \textbf{B}^\times_j)^c \cap\textbf{B}^\circ_j\big)\cap ( \textbf{A}^\times_k \cup \textbf{B}^\times_k )^c \cap \{ \textbf{j}_k \geq l^*,\ X^\eta_{T_{k,l^*}} \notin \Omega \}$, then the exclusion of $\textbf{A}^\times_{k,2}$ implies that $T_{k,j} - T_{k,j-1} \leq 2\hat{t}(\epsilon)/\eta$ for all $j=2,\cdots,l^*$, and the exclusion of $\textbf{A}^\times_{k,5}$ tells us that if $X^\eta_{T_{k,l^*}} \notin \Omega$, then we have $X^\eta_{T_{k,l^*}} \notin [s_- - \epsilon,s_+ +\epsilon]$. In summary,
    \begin{align*}
        \big(\bigcap_{j = 1}^{k-1}(\textbf{A}^\times_j \cup \textbf{B}^\times_j)^c \cap\textbf{B}^\circ_j\big)\cap ( \textbf{A}^\times_k \cup \textbf{B}^\times_k )^c \cap \{ \textbf{j}_k \geq l^*,\ X^\eta_{T_{k,l^*}} \notin \Omega \} \subseteq \textbf{A}^\circ_k.
    \end{align*}
    \item On $\big(\bigcap_{j = 1}^{k-1}(\textbf{A}^\times_j \cup \textbf{B}^\times_j)^c \cap\textbf{B}^\circ_j\big)\cap ( \textbf{A}^\times_k \cup \textbf{B}^\times_k )^c \cap \{ \textbf{j}_k \geq l^*,\ X^\eta_{T_{k,l^*}} \in \Omega \}$, the exclusion of $\textbf{A}^\times_{k,2}$ again implies that $T_{k,j} - T_{k,j-1} \leq 2\hat{t}(\epsilon)/\eta$ for all $j=2,\cdots,l^*$, hence $T_{k,l^*} - T_{k,1} \leq 2l^*\hat{t}(\epsilon)/\eta$. Similarly, the exclusion of $\textbf{A}^\times_{k,5}$ tells us that if $X^\eta_{T_{k,l^*}} \in \Omega$, then we have $X^\eta_{T_{k,l^*}} \in [s_- + \epsilon,s_+ -\epsilon]$. Now since $\textbf{B}^\times_k$ did not occur (see the definition in \cref{proof prop first exit def B cross k}), we must have $\sigma(\eta) > \widetilde{\tau}_k$ and $\widetilde{\tau}_k - T_{k,l^*} \leq \rho(\epsilon)/\eta$, hence $\widetilde{\tau}_k - T_{k,1} \leq \frac{2l^*\hat{t}(\epsilon) + \rho(\epsilon)}{\eta}$. Therefore,
    \begin{align*}
        \big(\bigcap_{j = 1}^{k-1}(\textbf{A}^\times_j \cup \textbf{B}^\times_j)^c \cap\textbf{B}^\circ_j\big)\cap ( \textbf{A}^\times_k \cup \textbf{B}^\times_k )^c \cap \{ \textbf{j}_k \geq l^*,\ X^\eta_{T_{k,l^*}} \in \Omega \} \subseteq \textbf{B}^\circ_k.
    \end{align*}
\end{itemize}
Collecting results above, we have \cref{proof prop first exit inclusion A B x o}. Now we discuss some of its implications. First, from \cref{proof prop first exit inclusion A B x o} we can immediately get that 
\begin{align}
    \big(\bigcap_{j = 1}^{k-1}(\textbf{A}^\times_j \cup \textbf{B}^\times_j)^c \cap\textbf{B}^\circ_j\big)\cap ( \textbf{A}^\times_k \cup \textbf{B}^\times_k )^c = \big(\bigcap_{j = 1}^{k-1}(\textbf{A}^\times_j \cup \textbf{B}^\times_j)^c \cap\textbf{B}^\circ_j\big)\cap ( \textbf{A}^\times_k \cup \textbf{B}^\times_k )^c \cap (\textbf{A}^\circ_k \cup \textbf{B}^\circ_k). \label{proof prop first exit inclusion 2 prior}
\end{align}
Next, recall the definitions of $\textbf{A}^\circ_k$ in \cref{def event A circ k prop first exit} and $\textbf{B}^\circ_k$ in \cref{proof def event B circ k}, and one can see that $\textbf{A}^\circ_k$ and $\textbf{B}^\circ_k$ are mutually exclusive, since the former implies that the first exit occurs during the $k-$th attempt while the latter implies that this attempt fails. This fact and \cref{proof prop first exit inclusion 2 prior} allow us to conclude that 
\begin{align}
    \bigcap_{i = 1}^{k}\big( \textbf{A}^\times_i \cup \textbf{B}^\times_i \cup \textbf{A}^\circ_i \big)^c  = \bigcap_{i = 1}^{k} \big( \textbf{A}^\times_i \cup \textbf{B}^\times_i \big)^c \cap \textbf{B}^\circ_i = \Big(\bigcap_{i = 1}^{k-1} \big( \textbf{A}^\times_i \cup \textbf{B}^\times_i \big)^c \cap \textbf{B}^\circ_i\Big)\cap ( \textbf{A}^\times_k \cup \textbf{B}^\times_k \cup \textbf{A}^\circ_k )^c.  \label{proof prop first exit inclusion 2}
\end{align}

Now we use the results obtained so far to bound the probability of
\begin{align*}
    \textbf{A}^\times(\epsilon,\delta,\eta) \delequal{} \bigcup_{k \geq 1}\bigg( \bigcap_{i = 1}^{k-1}\big( \textbf{A}^\times_i \cup \textbf{B}^\times_i \cup \textbf{A}^\circ_i \big)^c    \bigg) \cap \big(\textbf{A}^\times_k \cup \textbf{B}^\times_k \big).
\end{align*}
Using \cref{proof prop first exit inclusion 2}, we can see that (for any $x \in [-2\epsilon,2\epsilon])$
\begin{align*}
    & \mathbb{P}_x(\textbf{A}^\times(\epsilon,\delta,\eta)) \\
    = & \sum_{k \geq 1}\mathbb{P}_x\bigg( \bigg( \bigcap_{i = 1}^{k-1}\big( \textbf{A}^\times_i \cup \textbf{B}^\times_i \cup \textbf{A}^\circ_i \big)^c    \bigg) \cap \big(\textbf{A}^\times_k \cup \textbf{B}^\times_k \big) \bigg) \\
    = & \sum_{k \geq 1}\mathbb{P}_x\bigg( \bigg( \bigcap_{i = 1}^{k-1}\big( \textbf{A}^\times_i \cup \textbf{B}^\times_i \big)^c \cap \textbf{B}^\circ_i    \bigg) \cap \big(\textbf{A}^\times_k \cup \textbf{B}^\times_k \big) \bigg) \\
    = & \sum_{k \geq 1}\mathbb{P}_x\Big( \textbf{A}^\times_k \cup \textbf{B}^\times_k\ |\ \bigcap_{i = 1}^{k-1} \big( \textbf{A}^\times_i \cup \textbf{B}^\times_i \big)^c \cap \textbf{B}^\circ_i  \Big)
    \\
    & \ \ \ \ \ \ \ \ \ \ \ \ \ \ \ \ \cdot \prod_{j = 1}^{k - 1}\mathbb{P}_x\Big( \bigcap_{i = 1}^{j}\big( \textbf{A}^\times_i \cup \textbf{B}^\times_i \big)^c \cap \textbf{B}^\circ_i\ \Big|\ \bigcap_{i = 1}^{j-1}\big( \textbf{A}^\times_i \cup \textbf{B}^\times_i \big)^c \cap \textbf{B}^\circ_i     \Big) \\
    = & \sum_{k \geq 1}\mathbb{P}_x\Big( \textbf{A}^\times_k \cup \textbf{B}^\times_k\ |\ \bigcap_{i = 1}^{k-1} \big( \textbf{A}^\times_i \cup \textbf{B}^\times_i \big)^c \cap \textbf{B}^\circ_i  \Big)
    \\ & \ \ \ \ \ \ \ \ \ \ \ \ \ \ \ \ \cdot \prod_{j = 1}^{k - 1}\mathbb{P}_x\Big( \big(\bigcap_{i = 1}^{j-1}\big( \textbf{A}^\times_i \cup \textbf{B}^\times_i \big)^c \cap \textbf{B}^\circ_i \big) \cap (\textbf{A}^\times_j \cup \textbf{B}^\times_j \cup \textbf{A}^\circ_j)^c \ \Big|\ \bigcap_{i = 1}^{j-1}\big( \textbf{A}^\times_i \cup \textbf{B}^\times_i \big)^c \cap \textbf{B}^\circ_i     \Big) \\
    = & \sum_{k \geq 1}\mathbb{P}_x\Big( \textbf{A}^\times_k \cup \textbf{B}^\times_k\ |\ \bigcap_{i = 1}^{k-1} \big( \textbf{A}^\times_i \cup \textbf{B}^\times_i \big)^c \cap \textbf{B}^\circ_i  \Big)
    \\
    & \ \ \ \ \ \ \ \ \ \ \ \ \ \ \ \ \cdot \prod_{j = 1}^{k - 1}\mathbb{P}_x\Big( \big( \textbf{A}^\times_j \cup \textbf{B}^\times_j \cup \textbf{A}^\circ_j \big)^c \ \Big|\ \bigcap_{i = 1}^{j-1}\big( \textbf{A}^\times_i \cup \textbf{B}^\times_i \big)^c \cap \textbf{B}^\circ_i     \Big) \\
     \leq & \sum_{k \geq 1}\mathbb{P}_x\Big( \textbf{A}^\times_k \cup \textbf{B}^\times_k\ |\ \bigcap_{i = 1}^{k-1} \big( \textbf{A}^\times_i \cup \textbf{B}^\times_i \big)^c \cap \textbf{B}^\circ_i  \Big) \cdot \prod_{j = 1}^{k - 1}
     \Bigg(1 - \mathbb{P}_x\Big(  \textbf{A}^\circ_j \ \Big|\ \bigcap_{i = 1}^{j-1}\big( \textbf{A}^\times_i \cup \textbf{B}^\times_i \big)^c \cap \textbf{B}^\circ_i     \Big) \Bigg).
\end{align*}
This allows us to apply \cref{proof prop first exit bound B cross k}-\cref{proof prop first exit bound A circ k} and conclude that (here we only consider $\eta < \min\{\eta_i: i \in [8]\}$ ),
\begin{align}
    & \sup_{|x| \leq 2\epsilon}\mathbb{P}_x(\textbf{A}^\times(\epsilon,\delta,\eta)) \nonumber \\
    \leq & \sum_{k \geq 1} \Big( 5\eta^N + 2\Psi(\epsilon)\delta^\alpha \Big( \frac{H(1/\eta)}{\eta} \Big)^{l^* - 1} \Big) \cdot\Big( 1 -c_*\delta^\alpha \Big( \frac{H(1/\eta)}{\eta} \Big)^{l^* - 1} \Big)^{k-1} \nonumber \\
    = & \frac{5\eta^N + 2\Psi(\epsilon)\delta^\alpha \Big( \frac{H(1/\eta)}{\eta} \Big)^{l^* - 1}}{c_*\delta^\alpha \Big( \frac{H(1/\eta)}{\eta} \Big)^{l^* - 1}} \nonumber \\
    \leq & \frac{ 2\Psi(\epsilon) + 5\eta^\alpha }{ c_*  }\ \ \ \text{for sufficiently small $\eta$, due to $H\in \mathcal{RV}_{-\alpha}(\eta)$ and our choice of $N > \alpha l^*$  } \nonumber \\ 
    \leq & \frac{3\Psi(\epsilon)}{c_*} < C \ \ \ \text{for all $\eta$ small enough such that $5\eta^\alpha < \Psi(\epsilon)$}. \nonumber 
\end{align}
The last inequality follows from our choice of $\epsilon$ in \cref{proof prop choose epsilon 3}. This concludes the proof.
\end{proof}

Having established Lemma \ref{lemma bound on event A cross final}, we return to Proposition \ref{proposition first exit time Gradient Clipping} and give a proof. Recall that, aside from the attraction field $\Omega = (s_-,s_+)$, there are $n_\text{min} - 1$ other attraction fields $\widetilde{\Omega}_k = (s^-_k,s^+_k)$ (for each $k \in [n_\text{min} - 1]$). Besides, the function $\lambda(\cdot)$ and constants $\nu^\Omega,\nu^\Omega_k$ are defined in \cref{def lambda rate}-\cref{def nu omega k}.

\begin{proof}[Proof of Proposition \ref{proposition first exit time Gradient Clipping}]

We fix some parameters for the proof. First, with out loss of generality we only need to consider $C \in (0,1)$. Next we discuss the valid range of $\epsilon$ for the claim to hold. We only consider $\epsilon > 0$ such that
\begin{align*}
    \epsilon < \frac{\bar{\epsilon}}{ 6(\bar{\rho}+\widetilde{\rho} + 3) }\wedge \frac{\epsilon_0}{3}
\end{align*}
where $\bar{\rho}$ and $\widetilde{\rho}$ are the constants in Corollary \ref{corollary ODE GD gap} and Corollary \ref{corollary sgd gd gap} respectively, and $\epsilon_0$ is the constant in \cref{assumption detailed function f at critical point}. Due to continuity of measure $\mu$, it holds for all $\epsilon$ small enough such that (let $\hat{\epsilon} = 3(\bar{\rho} + \widetilde{\rho}+3)\epsilon$)
    \begin{align}
        \frac{ \mu(E(0))  }{ \mu\big(E\big( \hat{\epsilon} \big)\big)  } & < 1/(1-C),\label{proof prop choose epsilon 1} \\
        \frac{ \mu(E(0))  }{ \mu\big(E\big( -\hat{\epsilon} \big)\big)  } & > 1/(1+C),\label{proof prop choose epsilon 2} \\
        \frac{ \mu\Big( h^{-1}\big( (s_- - 2\hat{\epsilon},s_- + 2\hat{\epsilon}) \cup (s_+ - 2\hat{\epsilon},s_+ + 2\hat{\epsilon}) \big) \Big)  }{ \mu(E(-\hat{\epsilon})) } & \leq C \label{proof prop choose epsilon 4}
        \\
      \frac{ \mu\Big( E(\hat{\epsilon})\cap \big( s^-_{k^\prime} - \hat{\epsilon},s^+_{k^\prime} + \hat{\epsilon}) \Big)   }{ \mu(E(\hat{\epsilon}))   } & \leq \frac{ \nu^\Omega_{k^\prime} + C }{ \nu^\Omega } \label{proof prop choose epsilon 5}
      \\
       \frac{ \mu\Big( E(-\hat{\epsilon})\cap \big( s^-_{k^\prime} +2 \hat{\epsilon},s^+_{k^\prime} -2 \hat{\epsilon}) \Big)   }{ \mu(E(-\hat{\epsilon}))   } & \geq \frac{ \nu^\Omega_{k^\prime} - C }{ \nu^\Omega } \label{proof prop choose epsilon 6}
    \end{align}
In our proof we only consider $\epsilon$ small enough so the inequality above holds, and the claims in Lemma \ref{lemma return to local minimum quickly} hold. Moreover, we only consider $\epsilon$ and $\delta$ small enough so that Lemma \ref{lemma bound on event A cross final} hold and we have
\begin{align}
    \lim_{\eta \downarrow 0}\sup_{|x| \leq 2\epsilon}\mathbb{P}_x(\textbf{A}^\times(\epsilon,\delta,\eta)) < C. \label{proof prop first exit ineq 3}
\end{align}
We show that the desired claims hold for all $\epsilon,\delta$ sufficiently small that satisfy conditions above.

First, in order to show \cref{goal 5 prop first exit}, we define event
\begin{align*}
\widetilde{\textbf{A}}^\times(\epsilon,\delta,\eta) & \delequal{} \big(\textbf{A}^\times(\epsilon,\delta,\eta)\big)^c\cap 
\\
&\Big\{ \lambda(\eta)\big( \tau(\eta,\epsilon) - \sigma(\eta) \big) \geq C \text{ or }\exists n = \sigma(\eta) + 1, \cdots,\tau(\eta,\epsilon)\text{ such that } X^\eta_n \notin \widetilde{\Omega}_{J_\sigma(\eta)} \Big\}.
\end{align*}
Since $\lambda \in \RV_{ -1 - l^*(\alpha - 1) }$ and $\alpha > 1$, for the $\epsilon$ we fixed at the beginning of this proof, $\rho(\epsilon)$ is a fixed constant as well (the function $\rho$ is defined in Lemma \ref{lemma return to local minimum quickly}) and we have $\lim_{\eta \downarrow 0} \lambda(\eta)\rho(\epsilon)/\eta = 0$. Next, the occurrence of $\big( \textbf{A}^\times(\epsilon,\delta,\eta) \big)^c$ (in particular, the exclusion of all the $\textbf{A}^\times_{k,5}$ defined in \cref{def event A cross k 5}), we know that $X^\eta_{\sigma(\eta)} \notin [s^-_{J_\sigma}-\epsilon,s^-_{J_\sigma}+\epsilon]\cup[s^+_{J_\sigma}-\epsilon,s^+_{J_\sigma}+\epsilon]$ (recall that for any $k \in [n_\text{min} - 1]$, we have $\widetilde{\Omega}_j = (s^-_j,s^+_j)$; for definition of $J_\sigma$ see \cref{def J sigma}). Meanwhile, for all $\eta$ sufficiently small, we have $\epsilon/\lambda(\eta) > \rho(\epsilon)/\eta$. Therefore, using Lemma \ref{lemma return to local minimum quickly} we can see that (for all $\eta$ sufficiently small)
\begin{align}
    \sup_{|x|\leq 2\epsilon }\mathbb{P}_x\Big( \widetilde{\textbf{A}}^\times(\epsilon,\delta,\eta)\ |\ \big(\textbf{A}^\times(\epsilon,\delta,\eta)\big)^c \Big) \leq C. \label{proof prop first exit ineq 3 - 2}
\end{align}
Lastly, observe that
\begin{align*}
    & \mathbb{P}\Big(\Big\{ \lambda(\eta)\big( \tau(\eta,\epsilon) - \sigma(\eta) \big) \geq C \text{ or }\exists n = \sigma(\eta) + 1, \cdots,\tau(\eta,\epsilon)\text{ such that } X^\eta_n \notin \widetilde{\Omega}_{J_\sigma(\eta)} \Big\}\Big)
    \\
    \leq & \mathbb{P}_x\Big( (\textbf{A}^\times)^c \cap \Big\{ \lambda(\eta)\big( \tau(\eta,\epsilon) - \sigma(\eta) \big) \geq C
    \\
    &\ \ \ \ \text{ or }\exists n = \sigma(\eta) + 1, \cdots,\tau(\eta,\epsilon)\text{ such that } X^\eta_n \notin \widetilde{\Omega}_{J_\sigma(\eta)} \Big\} \Big) + \mathbb{P}_x(\textbf{A}^\times)
\end{align*}
so by combining \cref{proof prop first exit ineq 3} with \cref{proof prop first exit ineq 3 - 2}, we can obtain \cref{goal 5 prop first exit}.

Moving on, we discuss the upper bounds \cref{goal 1 prop first exit} and \cref{goal 3 prop first exit}. Recall that the fixed constant $k^\prime \in [n_\text{min} - 1]$ is prescribed in the description of this proposition. Let us observe some facts on event $(\textbf{A}^\times(\epsilon,\eta,\delta))^c \cap \{ J_\sigma(\eta) = k^\prime \}$:  If we let $J(\epsilon,\delta,\eta)\delequal{}\sup\{ k \geq 0: \widetilde{\tau}_{k} < \sigma(\eta)  \}$ be the number of attempts it took to escape, and 
$$J^\uparrow(\epsilon,\delta,\eta) \delequal{} \min\{ k \geq 1: T_{k,1}\ \text{ has $\Big( 3(\bar{\rho} + \widetilde{\rho} + 3 )\epsilon,\delta,\eta\Big)-$overflow} \},$$ 
then for all $\eta$ sufficiently small, we must have $J \leq J^\uparrow$ on event $(\textbf{A}^\times(\epsilon,\eta,\delta))^c\cap \{ J_\sigma(\eta) = k^\prime \}$. To see this via a proof of contradiction, let us assume that, for some arbitrary positive integer $j$, there exists some sample path on $(\textbf{A}^\times)^c\cap \{ J_\sigma(\eta) = k^\prime \}$ such that $J^\uparrow = j < J$. Then from the definition of $(\textbf{A}^\times)^c$, in particular the exclusion of event $\textbf{A}^\times_{j,0}$ (see the definition in \cref{def A cross k 0 in prop first exit}), for all sufficiently small $\eta$, we are able to apply Corollary \ref{corollary sgd gd gap} and \ref{corollary ODE GD gap} and conclude that $X^\eta_{T_{j,l^*}} \notin \Omega$: indeed, using Corollary \ref{corollary sgd gd gap} and \ref{corollary ODE GD gap} we can show that the distance between $X^\eta_{T_{j,l^*}}$ and the perturbed ODE 
$$\widetilde{\textbf{x}}^\eta\Big(T_{j,l^*} - T_{j,1},0;\big( 0, T_{j,2} - T_{j,1},\cdots, T_{j,l^*} - T_{j,1}\big), \big( \eta W_{j,1},\cdots,\eta W_{j,l^*} \big)  \Big)$$
is strictly less than $3(\bar{\rho} + \widetilde{\rho} + 3 )\epsilon$; on the other hand, the definition of $\Big( 3(\bar{\rho} + \widetilde{\rho} + 3 )\epsilon,\delta,\eta\Big)-$overflow implies that
\begin{align*}
    & \widetilde{\textbf{x}}^\eta\Big(T_{j,l^*} - T_{j,1},0;\big( 0, T_{j,2} - T_{j,1},\cdots, T_{j,l^*} - T_{j,1}\big), \big( \eta W_{j,1},\cdots,\eta W_{j,l^*} \big)  \Big) \\
    & \ \ \ \ \ \ \notin [s_- - 3(\bar{\rho} + \widetilde{\rho} + 3 )\epsilon, s_+ + 3(\bar{\rho} + \widetilde{\rho} + 3 )\epsilon].
\end{align*}
Therefore, we must have $X^\eta_{T_{j,l^*}} \notin \Omega$, which contradicts our assumption $j = J^\uparrow < J$. In summary, we have shown that, on $(\textbf{A}^\times)^c \cap \{ J_\sigma = k^\prime \}$, we have $J^\uparrow(\epsilon,\delta,\eta) \geq J(\epsilon,\delta,\eta)$.
Similarly, if we consider
$$J^\downarrow(\epsilon,\delta,\eta) \delequal{} \min\{ k \geq 1: T_{k,1}\ \text{ has $\big( -3(\bar{\rho} + \widetilde{\rho} + 3 )\epsilon,\delta,\eta\big)-$overflow} \},$$
then by the same argument above we can show that $J^\downarrow(\epsilon,\delta,\eta) \leq J(\epsilon,\delta,\eta)$. Now consider the following decomposition of events.
\begin{itemize}
    \item On $\{ J^\downarrow < J^\uparrow \}$, we know that for the first $k$ such that $T_{k,1}$ has $\big( -3(\bar{\rho} + \widetilde{\rho} + 3 )\epsilon,\delta,\eta\big)-$overflow, it does not have $\big( 3(\bar{\rho} + \widetilde{\rho} + 3 )\epsilon,\delta,\eta\big)-$overflow. Now we analyze the probability that $Z_0$ does not have $(\hat{\epsilon},\delta,\eta)-$overflow conditioning on that it does have $(-\hat{\epsilon},\delta,\eta)-$overflow (recall that we let $\hat{\epsilon} = 3(\bar{\rho} + \widetilde{\rho} + 3)$). Using Lemma \ref{lemmaOverflowProb} and the bound \cref{proof prop choose epsilon 4}, we know that for all $\eta$ sufficiently small,
    \begin{align}
        & \sup_{|x|\leq 2\epsilon}\mathbb{P}_x\Big( (\textbf{A}^\times)^c\cap\{ J^\downarrow < J^\uparrow \} \Big) \nonumber
        \\
        & \leq \frac{ \mu\Big( h^{-1}\big( (s_- - 2\hat{\epsilon},s_- + 2\hat{\epsilon}) \cup (s_+ - 2\hat{\epsilon},s_+ + 2\hat{\epsilon}) \big) \Big)  }{ \mu(E(-\hat{\epsilon})) } \leq C. \label{proof prop first exit J down < J up}
    \end{align}
    
    \item On $(\textbf{A}^\times)^c \cap \{ J_\sigma = k^\prime \}\cap\{J^\uparrow = J^\downarrow\}$, due to $J^\uparrow = J^\downarrow = J$ we know that $T_{J(\epsilon,\delta,\eta),1}$ is the first among all $T_{k,1}$ to have $( \hat{\epsilon},\delta,\eta)-$overflow. Moreover, due to $\{ J_\sigma = k^\prime \}$ and using Corollary \ref{corollary sgd gd gap} and \ref{corollary ODE GD gap} again as we did above, we know that the overflow endpoint of $T_{J(\epsilon,\delta,\eta),1}$ is in $(s^-_{k^\prime} - \hat{\epsilon},s^{+}_{k^\prime} + \hat{\epsilon})$ (recall that $\widetilde{\Omega}_{k^\prime} = (s^-_{k^\prime},s^+_{k^\prime})$). In summary, for any $n \geq 0$
    \begin{align*}
        & (\textbf{A}^\times)^c \cap \{ J_\sigma = k^\prime \}\cap\{J^\uparrow = J^\downarrow > n\}
        \\
        &\subseteq (\textbf{A}^\times)^c  \cap \{J^\uparrow > n\}\cap \Big\{T_{J^{\uparrow},1} \text{ has overflow endpoint in $(s^-_{k^\prime} - \hat{\epsilon},s^{+}_{k^\prime} + \hat{\epsilon})$} \Big\}
    \end{align*}
    so using Lemma \ref{lemmaOverflowProb}, we obtain that (for all $\eta$ sufficiently small)
    \begin{align}
        & \sup_{|x|\leq 2\epsilon}\mathbb{P}_x\Big( (\textbf{A}^\times)^c \cap \{ J_\sigma = k^\prime \}\cap\{J^\uparrow = J^\downarrow > n\} \Big) \nonumber
        \\
        \leq &\sup_{|x|\leq 2\epsilon}\mathbb{P}_x\Big( (\textbf{A}^\times)^c \cap\{J^\uparrow > n\} \Big)\cdot \frac{ p\big( \hat{\epsilon},\delta,\eta; (s^-_{k^\prime} - \hat{\epsilon},s^{+}_{k^\prime} + \hat{\epsilon})  \big)  }{ p( \hat{\epsilon},\delta,\eta  ) } \nonumber
        \\
        \leq &\sup_{|x|\leq 2\epsilon}\mathbb{P}_x\Big( (\textbf{A}^\times)^c \cap\{J^\uparrow > n\} \Big)\cdot \frac{ \nu^\Omega_{k^\prime} + C  }{ \nu^\Omega }. \label{proof prop first exit final ineq 1}
    \end{align}
    uniformly for any $n = 0,1,2,\cdots$ due to \cref{proof prop choose epsilon 5}.
    
    \item On the other hand, on $(\textbf{A}^\times)^c$, if $T_{J^\downarrow,1}$ has overflow endpoint in $(s^-_{k^\prime} + 2\hat{\epsilon},s^+_{k^\prime} - 2\hat{\epsilon})$, then from Definition \ref{definitionOverflow} we know that $T_{J^\downarrow,1}$ also has $(\hat{\epsilon},\delta,\eta)-$overflow, hence $J^\downarrow = J^\uparrow = J$. Moreover, using Corollary \ref{corollary sgd gd gap} and \ref{corollary ODE GD gap} again, we know that $X^{\eta}_{T_{J^\downarrow,l^*}} \in (s^-_{k^\prime},s^+_{k^\prime})$ so $J_\sigma = k^\prime$. In summary, for any $n \geq 0$,
    \begin{align*}
        & (\textbf{A}^\times)^c \cap \{ J_\sigma = k^\prime \}\cap\{J^\uparrow = J^\downarrow > n\}
        \\
        &\supseteq (\textbf{A}^\times)^c  \cap \{J^\downarrow > n\}\cap \Big\{T_{J^{\downarrow},1} \text{ has overflow endpoint in $(s^-_{k^\prime} + 2\hat{\epsilon},s^{+}_{k^\prime} -2 \hat{\epsilon})$} \Big\}
    \end{align*}
    so using Lemma \ref{lemmaOverflowProb}, we obtain that (for all $\eta$ sufficiently small)
    \begin{align}
        & \inf_{|x|\leq 2\epsilon}\mathbb{P}_x\Big( (\textbf{A}^\times)^c \cap \{ J_\sigma = k^\prime \}\cap\{J^\uparrow = J^\downarrow > n\} \Big) \nonumber
        \\
        \geq &\inf_{|x|\leq 2\epsilon}\mathbb{P}_x\Big( (\textbf{A}^\times)^c \cap\{J^\downarrow > n\} \Big)\cdot \frac{ p\big( -\hat{\epsilon},\delta,\eta; (s^-_{k^\prime} +2 \hat{\epsilon},s^{+}_{k^\prime} -2 \hat{\epsilon})  \big)  }{ p( -\hat{\epsilon},\delta,\eta  ) } \nonumber
        \\
        \geq &\inf_{|x|\leq 2\epsilon}\mathbb{P}_x\Big( (\textbf{A}^\times)^c \cap\{J^\downarrow > n\} \Big)\cdot \frac{ \nu^\Omega_{k^\prime} - C  }{ \nu^\Omega }. \label{proof prop first exit final ineq 2}
    \end{align}
    uniformly for any $n = 0,1,2,\cdots$ due to \cref{proof prop choose epsilon 6}.
\end{itemize}

Besides, the following claim holds on event $(\textbf{A}^\times)^c$.
\begin{itemize}
    \item From \cref{proof prop first exit inclusion 2}, the definition of $\textbf{B}^\circ_k$ as well as the definition of event $\textbf{A}^\circ_k$ (see \cref{def event A circ k prop first exit}), one can see that for any $j = 1,2,\cdots,J$, we have $$\widetilde{\tau}_j\wedge \sigma(\eta) - T_{j,1} \leq \frac{2l^*\hat{t}(\epsilon) + \rho(\epsilon)}{\eta}.$$
    \item Now if we turn to the interval $( \widetilde{\tau}_{j-1},T_{j,1}]$ (the time between the start of the $j-$th attempt and the arrival of the first large noise during this attempt) for each $j = 1,2,\cdots,J$, and the following sequence constructed by concatenating these intervals
    \begin{align*}
        \textbf{S}(\epsilon,\delta,\eta) \delequal{} &\big( 1,2,\cdots,T_{1,1},\widetilde{\tau}_{1}+1,\widetilde{\tau}_{1}+2,\cdots,T_{2,1},\cdots, \\
        & \widetilde{\tau}_k + 1,\widetilde{\tau}_k + 2,\cdots,T_{k+1,1},\widetilde{\tau}_{k+1} + 1,\widetilde{\tau}_{k+2} + 1,\cdots \big),
    \end{align*}
    then the discussion above have shown that, for
    \begin{align*}
        \min\{ n \in \textbf{S}(\epsilon,\delta,\eta):\ Z_n \ \text{ has $\Big( 3(\bar{\rho} + \widetilde{\rho} + 3 )\epsilon,\delta,\eta\Big)-$overflow} \} \geq T_{J,1}.
    \end{align*}
    Meanwhile, from the definition of overflow we know that the probability that $Z_1$ has $\Big( 3(\bar{\rho} + \widetilde{\rho} + 3 )\epsilon,\delta,\eta\Big)-$overflow is equal to
    \begin{align*}
        H(\delta/\eta)p\Big( 3(\bar{\rho} + \widetilde{\rho} + 3 )\epsilon,\delta,\eta\Big).
    \end{align*}
    \item Therefore, if, within the duration of each attempt, we split the attempt into two parts at the arrival time of the first large jump $(T_{k,1})_{k \geq 1}$ at each attempt, and define (here the subscript \textit{before} or \textit{after} indicates that we are counting the steps before or after the first large jump in an attempt)
    \begin{align*}
        \textbf{S}_\text{before}(\epsilon,\delta,\eta) & \delequal{} \{ n \in \textbf{S}(\epsilon,\delta,\eta):\ n \leq \sigma(\eta) \},\  I_{\text{before}}(\epsilon,\delta,\eta) \delequal{} \#\textbf{S}_\text{before}(\epsilon,\delta,\eta), 
        \\
         \textbf{S}_\text{after}(\epsilon,\delta,\eta) & \delequal{} \{ n \notin \textbf{S}(\epsilon,\delta,\eta):\ n \leq \sigma(\eta) \},\  I_{\text{after}}(\epsilon,\delta,\eta) \delequal{} \#\textbf{S}_\text{after}(\epsilon,\delta,\eta),
    \end{align*}
    then we have $\sigma(\eta) = I_{\text{before}} + I_{\text{after}}$. Moreover, the discussion above implies that
    \begin{align*}
        I_{\text{after}} & \leq J\big( 2l^*\hat{t}(\epsilon) + \rho(\epsilon) \big)/\eta \\
        I_{\text{before}} & \leq \min\{ n \in \textbf{S}(\epsilon,\delta,\eta):\ Z_n \ \text{ has $\Big( 3(\bar{\rho} + \widetilde{\rho} + 3 )\epsilon,\delta,\eta\Big)-$overflow} \}
    \end{align*}
    and on event $(\textbf{A}^\times)^c$.
\end{itemize}

Define geometric random variables with the following success rates
\begin{align*}
    U_1(\epsilon,\delta,\eta) & \sim \text{Geom}\Big( p\big(3(\bar{\rho} + \widetilde{\rho} + 3)\epsilon,\delta,\eta\big) \Big) \\
    U_2(\epsilon,\delta,\eta) & \sim \text{Geom}\Big( H(\delta/\eta)p\big(3(\bar{\rho} + \widetilde{\rho} + 3)\epsilon,\delta,\eta\big)\Big).
\end{align*}
Using results above to bound $I_{\text{before}}$ and $I_{\text{after}}$ separately on event $(\textbf{A}^\times)^c$, we can show that (for all $\eta$ sufficiently small and any $u > 0$)
\begin{align}
    & \sup_{ x \in [-2\epsilon,2\epsilon] }\mathbb{P}_x\Big( v^\Omega \lambda(\eta)\sigma(\eta) > u, J_\sigma(\eta) = k^\prime \Big) \nonumber
    \\
    \leq & \sup_{|x|\leq 2\epsilon}\mathbb{P}_x(\textbf{A}^\times(\epsilon,\delta,\eta)) + \sup_{ x \in [-2\epsilon,2\epsilon] }\mathbb{P}_x\Big( \{v^\Omega \lambda(\eta)\sigma(\eta) > u, J_\sigma(\eta) = k^\prime\} \cap \big( \textbf{A}^\times(\epsilon,\delta,\eta) \big)^c \Big)  \nonumber
    \\
    \leq & C + \sup_{ x \in [-2\epsilon,2\epsilon] }\mathbb{P}_x\Big( \{v^\Omega \lambda(\eta)\sigma(\eta) > u, J_\sigma(\eta) = k^\prime\} \cap \big( \textbf{A}^\times(\epsilon,\delta,\eta) \big)^c \Big) \ \ \ \text{due to \cref{proof prop first exit ineq 3}} \nonumber
    \\
    \leq & C + \sup_{ x \in [-2\epsilon,2\epsilon] }\mathbb{P}_x\Big( \{v^\Omega \lambda(\eta)I_{\text{before}}(\epsilon,\delta,\eta) > (1-C)u, J_\sigma(\eta) = k^\prime\} \cap \big( \textbf{A}^\times(\epsilon,\delta,\eta) \big)^c \Big) \nonumber
    \\
    &\ \ \ \  + \sup_{ x \in [-2\epsilon,2\epsilon] }\mathbb{P}_x\Big( \{v^\Omega \lambda(\eta)I_{\text{after}}(\epsilon,\delta,\eta) > Cu\} \cap \big( \textbf{A}^\times(\epsilon,\delta,\eta) \big)^c \Big) \nonumber
    \\
    \leq & C + \sup_{ x \in [-2\epsilon,2\epsilon] }\mathbb{P}_x\Big( \{v^\Omega \lambda(\eta)I_{\text{before}}(\epsilon,\delta,\eta) > (1-C)u, J_\sigma(\eta) = k^\prime\} \cap \big( \textbf{A}^\times(\epsilon,\delta,\eta) \big)^c \Big)\nonumber
    \\
    + & \mathbb{P}\Big( v^\Omega\lambda(\eta)\frac{ \rho(\epsilon) + 2l^*\hat{t}(\epsilon ) }{ \eta}\cdot U_1(\epsilon,\delta,\eta) > Cu \Big)\nonumber
    \\
    \leq & C \nonumber \\
    + & \sup_{ x \in [-2\epsilon,2\epsilon] }\mathbb{P}_x\Big( \{v^\Omega \lambda(\eta)I_{\text{before}}(\epsilon,\delta,\eta) > (1-C)u, J_\sigma(\eta) = k^\prime\} \cap \big( \textbf{A}^\times(\epsilon,\delta,\eta) \big)^c \cap \{J^\downarrow = J^\uparrow\} \Big)\nonumber
    \\
    + & \sup_{ x \in [-2\epsilon,2\epsilon] }\mathbb{P}_x\big( (\textbf{A}^\times)^c\cap\{ J^\downarrow < J^\uparrow \}  \big) + \mathbb{P}\Big( v^\Omega\lambda(\eta)\frac{ \rho(\epsilon) + 2l^*\hat{t}(\epsilon ) }{ \eta}\cdot U_1(\epsilon,\delta,\eta) > Cu \Big) \nonumber
    \\
    \leq & 2C + \mathbb{P}\Big( v^\Omega\lambda(\eta)U_2(\epsilon,\delta,\eta) > (1-C)u \Big)\frac{ \nu^\Omega_{k^\prime} + C }{\nu^\Omega} \nonumber 
    \\
    + & \mathbb{P}\Big( v^\Omega\lambda(\eta)\frac{ \rho(\epsilon) + 2l^*\hat{t}(\epsilon ) }{ \eta}\cdot U_1(\epsilon,\delta,\eta) > Cu \Big) \label{proof prop first exit final ineq 3}
\end{align}
where the last inequality follows from \cref{proof prop first exit J down < J up} and \cref{proof prop first exit final ineq 1}.
Now let us analyze the probability terms on the last row of the display above. For the first term, let $a(\eta) = H(\delta/\eta)p\Big(3(\bar{\rho} + \widetilde{\rho} + 3)\epsilon,\delta,\eta\Big)$. Due to Lemma \ref{lemmaOverflowProb}, we have (recall that $\nu^\Omega = \mu(E(0))$)
$$\lim_{\eta \downarrow 0} \frac{a(\eta)}{\lambda(\eta) \mu\big(E\big( 3(\bar{\rho} + \widetilde{\rho} + 3)\epsilon  \big)\big) } =  1.$$
Combining this with \cref{proof prop choose epsilon 1}, one can see that for all $\eta$ sufficiently small,
\begin{align*}
    \mathbb{P}\Big( v^\Omega\lambda(\eta)U_2(\epsilon,\delta,\eta) > (1-C)u \Big) \leq \mathbb{P}\Big( a(\eta)\text{Geom}\big( a(\eta) \big) > (1-C)^2u \Big)\ \forall u > 0.
\end{align*}
Next, let $b(\eta,u) = \mathbb{P}\Big( a(\eta)\text{Geom}\big( a(\eta) \big) > (1-C)^2u \Big) = \mathbb{P}\Big( \text{Geom}\big( a(\eta) \big) > \frac{(1-C)^2u}{a(\eta)} \Big)$. For $g(y) = \log(1-y)$, we know the existence of some $y_0 > 0$ such that for all $y \in(0, y_0)$, we have $\log(1-y) \leq -(1-C)y$. So one can see that for all $\eta$ sufficiently small,
\begin{align}
    \log b(u,\eta) & \leq \frac{(1-C)^2u}{a(\eta)}\log(1 - a(\eta)) \leq -(1-C)^3u \nonumber
    \\
   \Rightarrow b(u,\eta) & \leq \exp\big( -(1-C)^{3}u \big) \label{proof prop first exit final ineq 4}
\end{align}
uniformly for all $u > 0$.

For the second probability term, if we only consider $u \geq C$, then
\begin{align*}
    \mathbb{P}\Big( v^\Omega\lambda(\eta)\frac{ \rho(\epsilon) + 2l^*\hat{t}(\epsilon ) }{ \eta}\cdot U_1(\epsilon,\delta,\eta) > Cu \Big) \leq \mathbb{P}\Big( v^\Omega\lambda(\eta)\frac{ \rho(\epsilon) + 2l^*\hat{t}(\epsilon ) }{ \eta}\cdot U_1(\epsilon,\delta,\eta) > C^2 \Big).
\end{align*}
Using $H \in \mathcal{RV}_{-\alpha}(\eta)$ with $\alpha > 1$, we get
\begin{align*}
  p\Big(3(\bar{\rho} + \widetilde{\rho} + 3)\epsilon,\delta,\eta\Big)U_1(\epsilon,\delta,\eta) \xrightarrow{d} \text{Exp}(1)\ \ \ \text{as }\eta \downarrow 0
\end{align*}
due to the nature of the Geometric random variable $U_1$. Besides, due to $H \in \mathcal{RV}_{-\alpha}(\eta)$ with $\alpha > 1$ and Lemma \ref{lemmaOverflowProb}, it is easy to show that
\begin{align*}
    \lim_{\eta \downarrow 0}\frac{ \lambda(\eta)\frac{ \rho(\epsilon) + 2l^*\hat{t}(\epsilon ) }{ \eta} }{ p\Big(3(\bar{\rho} + \widetilde{\rho} + 3)\epsilon,\delta,\eta\Big) } = 0.
\end{align*}
Combining these results with Slutsky's theorem, we now obtain
\begin{align*}
    \mu(E(0))\lambda(\eta)\frac{ \rho(\epsilon) + 2l^*\hat{t}(\epsilon ) }{ \eta}\cdot U_1(\epsilon,\delta,\eta) \xrightarrow{d} 0\ \ \ \text{as }\eta \downarrow 0.
\end{align*}
Therefore,
\begin{align}
    \limsup_{\eta \downarrow 0}\sup_{u \geq C}\mathbb{P}\Big( \mu(E(0))\lambda(\eta)\frac{ \rho(\epsilon) + 2l^*\hat{t}(\epsilon ) }{ \eta}\cdot U_1(\delta,\eta) > Cu \Big) = 0. \label{proof prop first exit final ineq 5}
\end{align}
Plugging \cref{proof prop first exit final ineq 4} and \cref{proof prop first exit final ineq 5} back into \cref{proof prop first exit final ineq 3}, we can establish the upper bound in \cref{goal 1 prop first exit}. To show \cref{goal 3 prop first exit}, note that for event
\begin{align*}
    E(\epsilon,\eta) = \{\nu^\Omega \lambda(\eta)\tau(\eta,\epsilon) > u,\ X^\eta_{\tau(\eta,\epsilon)} \in B(\widetilde{m}_{k^\prime},2\epsilon)\},
\end{align*}
we have (for definitions of $\tau$, see \cref{def tau first transition time})
\begin{align*}
    E(\epsilon,\eta) 
  \supseteq & \{v^\Omega \lambda(\eta)\sigma(\eta) > u, J_\sigma(\eta) = k^\prime\} \cap \{X^\eta_n \in \widetilde{\Omega}_{J_\sigma(\eta)} \ \ \forall n \in [\sigma(\eta), \tau(\eta,\epsilon)]\},
  \\
  E(\epsilon,\eta) 
  \cap & \{v^\Omega \lambda(\eta)\sigma(\eta) > u, J_\sigma(\eta) = j\} \cap \{X^\eta_n \in \widetilde{\Omega}_{J_\sigma(\eta)} \ \ \forall n \in [\sigma(\eta), \tau(\eta,\epsilon)]\} = \emptyset \ \forall j \neq k^\prime.
\end{align*}
Therefore, for all $\eta$ sufficiently small,
\begin{align*}
    & \sup_{|x|\leq 2\epsilon}\mathbb{P}_x(E(\epsilon,\eta))
    \\
    \leq & \sup_{|x|\leq 2\epsilon}\mathbb{P}_x(\textbf{A}^\times) + \sup_{|x|\leq 2\epsilon}\mathbb{P}_x( (\textbf{A}^\times)^c \cap \{X^\eta_n \notin \widetilde{\Omega}_{J_\sigma(\eta)} \ \text{for some } n \in [\sigma(\eta), \tau(\eta,\epsilon)]\}\Big)
    \\
    + & \sup_{|x|\leq 2\epsilon}\mathbb{P}_x\Big( (\textbf{A}^\times)^c \cap \{v^\Omega \lambda(\eta)\sigma(\eta) > u, J_\sigma(\eta) = k^\prime\} \cap \{X^\eta_n \in \widetilde{\Omega}_{J_\sigma(\eta)} \ \ \forall n \in [\sigma(\eta), \tau(\eta,\epsilon)]\}  \Big)
    \\
    \leq & 4C + \exp\big( -(1-C)^3u \big)\frac{\nu^\Omega_{k^\prime} + C }{\nu^\Omega } 
\end{align*}
uniformly for all $u \geq C$, due to \cref{proof prop first exit ineq 3}, \cref{goal 5 prop first exit} and \cref{proof prop first exit final ineq 3}.

The lower bound can be shown by an almost identical approach. In particular, analogous to \cref{proof prop first exit final ineq 3}, we can show that (for any $u > 0$)
\begin{align}
    & \inf_{ x \in [-2\epsilon,2\epsilon] }\mathbb{P}_x\Big( v^\Omega \lambda(\eta)\sigma(\eta) > u, J_\sigma(\eta) = k^\prime \Big) \nonumber
    \\
    \geq & \inf_{ x \in [-2\epsilon,2\epsilon] }\mathbb{P}_x\Big( \{v^\Omega \lambda(\eta)\sigma(\eta) > u, J_\sigma(\eta) = k^\prime\} \cap \big( \textbf{A}^\times(\epsilon,\delta,\eta) \big)^c \Big)  \nonumber
    \\
    \geq & \inf_{ x \in [-2\epsilon,2\epsilon] }\mathbb{P}_x\Big( \{v^\Omega \lambda(\eta)I_{\text{before}}(\epsilon,\delta,\eta) > u, J_\sigma(\eta) = k^\prime\} \cap \big( \textbf{A}^\times(\epsilon,\delta,\eta) \big)^c \Big) \nonumber
    \\
    \geq & \mathbb{P}\Big( v^\Omega\lambda(\eta)U_2^\prime(\epsilon,\delta,\eta) > (1-C)u \Big)\frac{ \nu^\Omega_{k^\prime} + C }{\nu^\Omega} -2C\nonumber
\end{align}
due to $\mathbb{P}(E\symbol{92}F) \geq \mathbb{P}(E) - \mathbb{P}(F)$ and \cref{proof prop first exit ineq 3}\cref{proof prop first exit J down < J up}\cref{proof prop first exit final ineq 2}
where
$$U^\prime_2(\epsilon,\delta,\eta)\sim \text{Geom}\Big( H(\delta/\eta)p\big(-3(\bar{\rho} + \widetilde{\rho} + 3)\epsilon,\delta,\eta\big) \Big).$$
Using the similar argument leading to \cref{proof prop first exit final ineq 4}, we are able to show \cref{goal 2 prop first exit}, \cref{goal 4 prop first exit} and conclude the proof.
\end{proof}

Recall that $\sigma_i(\eta) = \min\{ n \geq 0: X_n \notin \Omega_i \}$ and that value of constants $q_i, q_{i,j}$ are specified via \cref{defOdeJumpClipping1}-\cref{def q i q ij}.
Define
\begin{align}
    \tau^\text{min}_i(\eta,\epsilon) & \delequal{} \min\{ n \geq \sigma_i(\eta): X^\eta_n \in \bigcup_j [m_j - 2\epsilon,m_j+2\epsilon]   \},
    \\
    J_i(\eta) & = j \iff X^\eta_{\sigma_i(\eta)} \in \Omega_j\ \forall j \in [n_\text{min}].
\end{align}

The following result is simply a restatement of Proposition \ref{proposition first exit time Gradient Clipping} under the new system of notations. Despite the reiteration, we still state it here because this is the version that will be used to prove Lemma \ref{lemma first exit key lemma 2}, which is the key tool for establishing Theorem \ref{theorem first exit time}, as well as many other results in Section \ref{section C appendix}.

\begin{proposition}\label{prop first exit and return time i}
Given $C > 0$ and $i,j\in[n_\text{min}]$ such that $i\neq j$, the following claims hold for all $\epsilon>0$ that are sufficiently small:
\begin{align*}
    \limsup_{\eta \downarrow 0}\sup_{u \in (C,\infty)} \sup_{x \in (m_i - 2\epsilon,m_i + 2\epsilon)}&\mathbb{P}_x\Big( q_i\lambda_i(\eta)\sigma_i(\eta) > u,\ X^\eta_{\sigma_i(\eta)} \in \Omega_j \Big) 
    \\
    \leq & C + \exp\big( -(1-C)u \big)\frac{ q_{i,j} + C }{q_i},
    \\
    \liminf_{\eta \downarrow 0}\inf_{u \in (C,\infty)} \inf_{x \in (m_i - 2\epsilon,m_i + 2\epsilon)}&\mathbb{P}_x\Big( q_i\lambda_i(\eta)\sigma_i(\eta) > u,\ X^\eta_{\sigma_i(\eta)} \in \Omega_j \Big)
    \\
    \geq & -C + \exp\big( -(1+C)u \big)\frac{ q_{i,j} - C }{q_i},
    \\
     \limsup_{\eta \downarrow 0}\sup_{u \in (C,\infty)} \sup_{x \in (m_i - 2\epsilon,m_i + 2\epsilon)}&\mathbb{P}_x\Big( q_i\lambda_i(\eta)\tau^\text{min}_i(\eta,\epsilon) > u,\ X^\eta_{\tau^\text{min}_i(\eta,\epsilon)} \in \Omega_j \Big)
     \\
     \leq & C + \exp\big( -(1-C)u \big)\frac{ q_{i,j} + C }{q_i},
     \\
     \liminf_{\eta \downarrow 0}\inf_{u \in (C,\infty)} \inf_{x \in (m_i - 2\epsilon,m_i + 2\epsilon)}&\mathbb{P}_x\Big( q_i\lambda_i(\eta)\tau^\text{min}_i(\eta,\epsilon) > u,\ X^\eta_{\tau^\text{min}_i(\eta,\epsilon)} \in \Omega_j \Big)
    \\
    \geq & -C + \exp\big( -(1+C)u \big)\frac{ q_{i,j} - C }{q_i},
     \\
      \liminf_{\eta \downarrow 0} \inf_{x \in (m_i - 2\epsilon,m_i + 2\epsilon)}&\mathbb{P}_x\Big( q_i\lambda_i(\eta)\big(\tau^\text{min}_i(\eta,\epsilon) - \sigma_i(\eta)\big) < C,
      \\
      & X^\eta_n \in \Omega_{ J_i(\eta) }\ \forall n \in [\sigma_i(\eta),\tau^\text{min}_i(\eta,\epsilon)] \Big)
      \geq 1-C.
\end{align*}
\end{proposition}

Concluding this section, we apply Proposition \ref{prop first exit and return time i} and prove Lemma \ref{lemma first exit key lemma 2}.
\begin{proof}[Proof of Lemma \ref{lemma first exit key lemma 2}]
Fix some $C \in (0,1)$, $u > 0$, and some $k,l \in [n_\text{min}]$ with $k \neq l$. Let $q_i, q_{i,j}$ be the constants defined in \cref{def q i q ij}.

Fix some $C_0 \in \big(0,\frac{C}{n_\text{min}}\wedge \frac{q_k}{ n_\text{min} }C\big)$. Using Proposition \ref{prop first exit and return time i}, we know that for all $\epsilon$ sufficiently small, we have
\begin{align*}
    & \limsup_{\eta \downarrow 0}\sup_{x \in (m_k - 2\epsilon,m_k + 2\epsilon)}\mathbb{P}_x\Big( q_k\lambda_k(\eta)\sigma_k(\eta) > u,\ X^\eta_{\sigma_k(\eta)} \in \Omega_j \Big)
    \\
    & \leq C_0 + \exp\big( -(1-C)u \big)\frac{ q_{k,j} + C_0 }{q_k}\ \forall j \in [n_\text{min}].
\end{align*}
Summing up the inequality above over all $j \in [n_\text{min}]$, we can obtain \cref{goal 1 key lemma first exit 2}. The lower bound \cref{goal 2 key lemma first exit 2} can be established using an identical approach.

In order to show \cref{goal 4 key lemma first exit 2}, note that we can find $C_1 \in(0,u)$ sufficiently small so that
\begin{align*}
    - C_1 + \exp\big( -(1+C_1)\cdot 2C_1 \big)\frac{ q_{k,l} - C_1 }{q_k} \geq \frac{ q_{k,l} - C }{q_k}.
\end{align*}
Fix such $C_1$. From Proposition \ref{prop first exit and return time i}, we also know that for all $\epsilon$ small enough, we have
\begin{align*}
    & \liminf_{\eta \downarrow 0} \inf_{x \in (m_k - 2\epsilon,m_k + 2\epsilon)} \mathbb{P}_x\Big( q_k\lambda_k(\eta)\sigma_k(\eta) > u, X^\eta_{\sigma_k(\eta)} \in \Omega_l \Big) 
    \\
    & \geq - C_1 + \exp\big( -(1+C_1)\cdot 2C_1 \big)\frac{ q_{k,l} - C_1 }{q_k}.
\end{align*}
Then using $ \mathbb{P}_x\Big(X^\eta_{\sigma_k(\eta)} \in \Omega_l \Big) \geq \mathbb{P}\Big( q_k\lambda_k(\eta)\sigma_k(\eta) > u, X^\eta_{\sigma_k(\eta)} \in \Omega_l \Big)$ we conclude the proof for \cref{goal 4 key lemma first exit 2}.

Moving on, we show \cref{goal 3 key lemma first exit 2} in the following way. Note that we can find $C_2 \in (0,u)$ small enough so that
\begin{align}
    2C_2 + \frac{q_{k,l} + C_2 }{q_k} < \frac{q_{k,l} + C}{q_k}. \label{proof key lemma first exit 2 ineq 1}
\end{align}
Fix such $C_2$. Since \cref{goal 2 key lemma first exit 2} has been established already, we can find some $u_2 > 0$ such that for all $\epsilon$ small enough,
\begin{align}
    \limsup_{ \eta \downarrow 0 }\sup_{x \in (m_k-2\epsilon,m_k+2\epsilon)} \mathbb{P}_x\Big( q_k \lambda_k(\eta) \sigma_k(\eta) \leq u_2 \Big) & < C_2 \label{proof key lemma first exit 2 ineq 2}
\end{align}
Fix such $u_2$. Meanwhile, fix some $C_3 \in (0, C_2 \wedge u_2)$. From Proposition \ref{prop first exit and return time i} we know that for all $\epsilon$ sufficiently small,
\begin{align}
    & \limsup_{ \eta \downarrow 0 }\sup_{x \in (m_k-2\epsilon,m_k+2\epsilon)} \mathbb{P}_x\Big(q_k \lambda_k(\eta) \sigma_k(\eta) > u_2,\ X^\eta_{ \sigma_k(\eta) } \in \Omega_l  \Big)
    \\
    & \leq C_3 + \exp\big(-(1-C_3)u_2\big)\frac{ q_{k,l} + C_3 }{q_k} \nonumber
    \\
    & \leq C_2 + \frac{q_{k,l} + C_2 }{q_k}. \label{proof key lemma first exit 2 ineq 3}
\end{align}
Lastly, observe the following decomposition of events (for any $x \in \Omega_k$)
\begin{align*}
    \mathbb{P}_x\Big( X^\eta_{ \sigma_k(\eta) }\in \Omega_l \Big) & \leq \mathbb{P}_x\Big( q_k \lambda_k(\eta) \sigma_k(\eta) \leq u_2 \Big) + \mathbb{P}_x\Big(q_k \lambda_k(\eta) \sigma_k(\eta) > u_2,\ X^\eta_{ \sigma_k(\eta) } \in \Omega_l  \Big).
\end{align*}
Combining this bound with \cref{proof key lemma first exit 2 ineq 1}-\cref{proof key lemma first exit 2 ineq 3}, we complete the proof.
\end{proof}

\counterwithin{equation}{section}
\section{Proofs for Section \ref{subset:scaling-limit}} \label{section C appendix}

In this section, we show that gradient clipping scheme effectively partitions the entire optimization landscape of $f$ into different regions based on the radius $r_i$ and minimum jump number $l^*_i$ for each attraction field $\Omega_i$. 
Furthermore, when staying in each region, the behavior of SGD iterates closely resembles a Markov chain that \emph{only} visits \emph{wider} attraction fields in this region. 
We exclude the trivial case where $n_\text{min} = 1$ and there is only one attraction field. 

This structure is as follows. First we present some key lemmas that can be used to prove the Theorem \ref{thm main paper eliminiate sharp minima from trajectory}-\ref{corollary irreducible case} in the main paper. Then we devote the rest of the section to establish those lemmas. In order to prove Theorem \ref{thm main paper eliminiate sharp minima from trajectory}, we will make use of the following lemma, where we show that the type of claim in Theorem \ref{thm main paper eliminiate sharp minima from trajectory} is indeed valid if we look at a much shorter time interval. Then when we move onto the proof of Theorem \ref{thm main paper eliminiate sharp minima from trajectory}, it suffices to partition the entire horizon into pieces of these short time intervals, on each of which we analyze the dynamics of SGD respectively.
\begin{lemma} \label{lemma eliminiate sharp minima from trajectory}
Assume the graph $\mathcal{G}$ is \textbf{irreducible}, and let $\epsilon > 0,\delta > 0$ be any positive real numbers. For the following random variables (indexed by $\eta$)
\begin{align}
    V^\text{small}(\eta,\epsilon,t)\delequal{}\frac{1}{\floor{t / \lambda^\text{large}(\eta) }}\int_0^{\floor{t / \lambda^\text{large}(\eta)  } }\mathbbm{1}\Big\{ X^\eta_{ \floor{u} } \in \bigcup_{j: m_j \notin M^\text{large}}\Omega_j \Big\}du,
\end{align}
the following claim holds for any sufficiently small $t>0$:
\begin{align*}
    \limsup_{\eta \downarrow 0} \sup_{x \in [-L,L]} \mathbb{P}_x\Big( V^\text{small}(\eta,\epsilon,t) > \epsilon \Big) \leq 5\delta
\end{align*}
\end{lemma}

\begin{proof}[Proof of Theorem \ref{thm main paper eliminiate sharp minima from trajectory}]
It suffices to show that for any $t > 0, \kappa > 1 + (\alpha-1)l^\text{large}, \epsilon > 0, \delta \in (0,\epsilon)$, we have
\begin{align*}
   \limsup_{\eta \downarrow 0} \mathbb{P}_x\Big( V^*(\eta,t,\kappa) > 3\epsilon \Big) < \delta
\end{align*}
for
\begin{align*}
    V^*(\eta,t,\kappa) \delequal  \frac{1}{\floor{t / \eta^\kappa }}\int_0^{\floor{t / \eta^\kappa } }\mathbbm{1}\Big\{ X^\eta_{ \floor{u} }(x) \in \bigcup_{j: m_j \notin M^\text{large}}\Omega_j \Big\}du.
\end{align*}
Let us fix some $\epsilon > 0,\delta \in (0,\epsilon)$.  First, let
\begin{align*}
    N(\eta) = \ceil{ \frac{ \floor{ t/\eta^\kappa}}{ \floor{t/\lambda^\text{large}(\eta)} } }.
\end{align*}
The regularly varying nature of $H$ implies that $\lambda^\text{large}(\eta) \in \mathcal{RV}_{1 + l^\text{large}(\alpha - 1) }(\eta)$. Since $\kappa > 1 + l^\text{large}(\alpha - 1)$, we know that $\lim_{\eta \downarrow 0}N(\eta) = \infty$. Next, due to Lemma \ref{lemma eliminiate sharp minima from trajectory}, we can find $t_0 > 0$ and $\bar{\eta} > 0$ such that for any $\eta \in (0,\bar{\eta})$
\begin{align}
   \sup_{y \in [-L,L]} \mathbb{P}_y(V^\text{small}(\eta,\epsilon,t_0) > \epsilon ) < \delta. \label{proof thm elimination ineq 1}
\end{align}
For any $k \geq 1$, define
\begin{align}
    V_k(\eta)\delequal{}\frac{1}{\floor{t_0 / \lambda^\text{large}(\eta) }}\int_{(k-1)\floor{t_0 / \lambda^\text{large}(\eta)  } }^{k\floor{t_0 / \lambda^\text{large}(\eta)  } }\mathbbm{1}\Big\{ X^\eta_{ \floor{u} } \in \bigcup_{j: m_j \notin M^\text{large}}\Omega_j \Big\}du.
\end{align}
 It is clear from its definition that $V_k$ stands for the proportion of time that the SGD iterates are outside of \textit{large} attraction fields on the interval $[(k-1)\floor{\frac{t_0}{\lambda^\text{large}(\eta)}},k\floor{\frac{t_0}{\lambda^\text{large}(\eta)}}]$. 
From \cref{proof thm elimination ineq 1} and Markov property, one can see that for any $\eta \in (0,\bar{\eta})$
\begin{align*}
    \sup_{x \in [-L,L]}\mathbb{P}_x( V_k(\eta) > \epsilon \ |\ X^\eta_0,\cdots,X^\eta_{(k-1)\floor{t_0 / \lambda^\text{large}(\eta) }} ) \leq \delta
\end{align*}
uniformly for all $k \geq 1$. Now define $K(\eta) \delequal{} \#\{ n = 1,2,\cdots,N(\eta):\ V_k(\eta) > \epsilon  \}$. By a simple stochastic dominance argument, we have
\begin{align*}
    \sup_{x \in [-L,L]}\mathbb{P}_x(K(\eta) \geq j) \leq \mathbb{P}( \text{Binomial}(N(\eta),\delta) \geq j)\ \ \forall j = 1,2,\cdots.
\end{align*}
Meanwhile, strong law of large numbers implies the existence of some $\bar{\eta}_1 > 0$ such that $\mathbb{P}( \frac{\text{Binomial}(N(\eta),\delta)}{N(\eta)} > 2\delta) < \delta $ for all $\eta \in (0,\bar{\eta}_1)$, thus
\begin{align*}
    \sup_{|x| \leq L}\mathbb{P}_x(K(\eta)/N(\eta) > 2\delta) \leq \delta \ \ \forall \eta \in (0,\bar{\eta}_1 \wedge \bar{\eta}).
\end{align*}

Lastly, from the definition of $K(\eta)$ and $N(\eta)$, we know that for all the $N(\eta)$ intervals $[(k-1)\floor{\frac{t_0}{\lambda^\text{large}(\eta)}},k\floor{\frac{t_0}{\lambda^\text{large}(\eta)}}]$ with $k \in [N(\eta)]$, only on $K(\eta)$ of them did the SGD iterates spent more then $\epsilon$ proportion of time outside of the \textit{large} attraction fields, hence
\begin{align*}
    V^*(\eta,t,\kappa) \leq \epsilon + \frac{K(\eta)}{N(\eta)}.
\end{align*}
In summary, we now have
\begin{align*}
    \mathbb{P}_x(V^*(\eta,t,\kappa) > 3\epsilon) < \delta
\end{align*}
for all $\eta \in (0, \bar{\eta}_1\wedge \bar{\eta}).$ This concludes the proof.
\end{proof}

When introducing Theorem \ref{corollary irreducible case} in the main paper, we stated that the results of eliminating sharp minima can be extended to the more general reducible case. Here we present the corresponding theoretical results in Theorem~\ref{thm main paper MC GP absorbing case} and \ref{thm main paper MC GP transient case} below.
The main message can be summarized as follows: 
(a) SGD with truncated and heavy-tailed noise naturally partitions the entire training landscape into different regions; 
(b) In each region, the dynamics of $X^\eta_n$ for small $\eta$ closely resemble that of a continuous-time Markov chain that \emph{only} visits local minima; 
(3) In particular, any sharp minima within each region is almost completely avoided by SGD.


When the typical transition graph (see Definition \ref{def main paper typical transition graph} in the main paper) is not irreducible, there will be multiple communication classes on the graph. Suppose that there are $K$ communication classes $G_1,\cdots,G_K$. From now on, we zoom in on a specific communication class $G\in\{G_1,\cdots,G_K\}$. 
For this communication class $G$, define $l^*_G \delequal{} \max\{ l^*_i:\ i = 1,2,\cdots,n_\text{min};\ m_i \in G\}.$
 For each local minimum $m_i \in G$, we call its attraction field $\Omega_i$ a \emph{large} attraction field if $l^*_i = l^*_G$, and a \emph{small} attraction field if $l^*_i < l^*_G$. We have thus classified all $m_i$ in $G$ into two groups: the ones in large attraction fields $m^\text{large}_{1},\cdots,m^\text{large}_{i_G}$ and the ones in small attraction fields $m^\text{small}_{1},\cdots,m^\text{small}_{i^\prime_G}$. Also, define a scaling function $\lambda_G$ associated with $G$ as $\lambda_G(\eta) \delequal{} H(1/\eta)\Big( \frac{ H(1/\eta) }{ \eta } \Big)^{l^*_G - 1}.$
\begin{theorem} \label{thm main paper MC GP absorbing case} Under Assumptions \ref{assumption: function f in R 1}-\ref{assumption clipping threshold b}, if $G$ is \textbf{absorbing}, then there exists a continuous-time Markov chain $Y$ on $\{m^\text{large}_{1},\cdots,m^\text{large}_{i_G}\}$ such that for any $x \in \Omega_i,|x| \leq L$ (where $i\in\{1,2,\cdots,n_\text{min}\}$) with $m_i \in G$, and
\begin{align}
    X^\eta_{\floor{ t/\lambda_G(\eta) }}(x) \rightarrow Y_t(\pi_G(m_i))\ \ \ \text{as }\eta \downarrow 0 \nonumber
\end{align}
in the sense of finite-dimensional distributions, where $\pi_G$ is a random mapping satisfying (1) $\pi_G(m) \equiv m$ if $m \in \{m^\text{large}_{1},\cdots,m^\text{large}_{i_G}\}$; (2) $\pi_G(m)$ is a random variable that only takes value in $\{m^\text{large}_{1},\cdots,m^\text{large}_{i_G}\}$ if $m \in \{m^\text{small}_{1},\cdots,m^\text{small}_{i^\prime_G}\}$.
\end{theorem}
We stress that Theorem \ref{corollary irreducible case} in the main paper follows immediately from Theorem \ref{thm main paper MC GP absorbing case} above.

Next, to state the corresponding result for a \textbf{transient} communication class $G$, we introduce a couple of extra definitions.
We consider a version of $X^\eta_n$ that is killed when $X^\eta_n$ leaves $G$. Define stopping time
\begin{align}
    \tau_G(\eta) \delequal{} \min\{ n \geq 0: X^\eta_n \notin \bigcup_{i:m_i \in G}\Omega_i \} \label{def tau G X exits G}
\end{align}
as the first time the SGD iterates leave all attraction fields in $G$, and we use a cemetery state $\boldsymbol{\dagger}$ to construct the following process $X^{\dagger,\eta}_n$ as a version of $X^\eta_n$ with killing at $\tau_G$:
\begin{align}
    X^{\dagger,\eta}_n = \begin{cases} 
X^\eta_n & \text{if }\ n < \tau_G(\eta),\\
\boldsymbol{\dagger} & \text{if }\ n \geq \tau_G(\eta).
\end{cases}
    \label{def process X killed}
\end{align}

\begin{theorem} \label{thm main paper MC GP transient case}
Under Assumptions \ref{assumption: function f in R 1}-\ref{assumption clipping threshold b}, if $G$ is \textbf{transient}, then there exists a continuous-time Markov chain $Y$ \textbf{with killing} that has state space $\{m^\text{large}_{1},\cdots,m^\text{large}_{i_G},\boldsymbol{\dagger}\}$ (we say the Markov chain $Y$ is killed when it enters the absorbing cemetery state $\boldsymbol{\dagger}$) such that for any $x \in \Omega_i, |x| \leq L$ (where $i\in\{1,2,\cdots,n_\text{min}\}$) with $m_i \in G$, and
\begin{align}
    X_{\floor{ t/\lambda_G(\eta) }}^{\dagger,\eta}(x) \rightarrow Y_t(\pi_G(m_i))\ \ \ \text{as }\eta \downarrow 0 \nonumber
\end{align}
in the sense of finite-dimensional distributions, where $\pi_G$ is a random mapping satisfying (1) $\pi_G(m) \equiv m$ if $m \in \{m^\text{large}_{1},\cdots,m^\text{large}_{i_G}\}$; (2) $\pi_G(m)$ is a random variable that only takes value in $\{m^\text{large}_{1},\cdots,m^\text{large}_{i_G},\boldsymbol{\dagger}\}$ if $m \in \{m^\text{small}_{1},\cdots,m^\text{small}_{i^\prime_G}\}$.
\end{theorem}


To show Theorem \ref{thm main paper MC GP absorbing case} and \ref{thm main paper MC GP transient case}, we introduce the following concepts. First, we consider the case where the SGD iterates $X^\eta_n$ is initialized on the communication class $G$ and $G$ is absorbing. For some $\Delta > 0, \eta > 0$, define (let $B(u,v) \delequal{} [u-v,u+v]$)
\begin{align}
    \sigma^G_0(\eta,\Delta) & \delequal{} \min\{ n \geq 0: X^\eta_n \in \bigcup_{i:\ m_i \in G}B(m_i,2\Delta) \}\label{def tau G sigma G I G 1}
    \\
    \tau^G_0(\eta,\Delta) & \delequal{} \min\{ n \geq \sigma^G_0(\eta,\Delta):\ X^\eta_n \in \bigcup_{i: m_i \notin G^\text{small} }B(m_i,2\Delta)  \} \label{def tau G sigma G I G 2}
    \\
    I^G_0(\eta,\Delta) & = j \iff X^\eta_{ \tau^G_0 } \in B(m_j,2\Delta),\ \ \ \widetilde{I}^G_0(\eta,\Delta) = j \iff X^\eta_{ \sigma^G_0 } \in B(m_j,2\Delta) \label{def tau G sigma G I G 3}
    \\
    \sigma^G_k(\eta,\Delta) & \delequal{} \min\{ n > \tau^G_{k-1}(\eta,\Delta):\ X^\eta_n \in \bigcup_{i: m_i \in G,\ i \neq I^G_{k-1}}B(m_i,2\Delta) \}\ \forall k \geq 1 \label{def tau G sigma G I G 4}
    \\
    \tau^G_k(\eta,\Delta) & \delequal{} \min\{ n \geq \sigma^G_{k-1}(\eta,\Delta):\ X^\eta_n \in \bigcup_{i: m_i \notin G^\text{small}}B(m_i,2\Delta) \}\ \forall k \geq 1 \label{def tau G sigma G I G 5}
    \\
    I^G_k(\eta,\Delta) & = j \iff X^\eta_{ \tau^G_k } \in B(m_j,2\Delta),\ \ \ \widetilde{I}^G_k(\eta,\Delta) = j \iff X^\eta_{ \sigma^G_k } \in B(m_j,2\Delta) \forall k \geq 1. \label{def tau G sigma G I G 6}
\end{align}
Intuitively speaking, at each $\tau^G_k$ the SGD iterates visits a minimizer that is not in a \textit{small} attraction field on $G$, and we use $I^G_k$ to mark the label of that large attraction field. Stopping time $\sigma^G_k$ is the first time that SGD visits a minimizer that is \textit{different} from the one visited at $\tau^G_k$, and $\tau^G_{k+1}$ is the first time that a minimizer not in a small attraction field of $G$ is visited again since $\sigma^G_k$ (and including $\sigma^G_k$). It is worth mentioning that, under this definition, we could have $I^G_k = I^G_{k+1}$ for any $k \geq 0$. Meanwhile, define the following process that only keeps track of the updates on the labels $(I^G_k)_{k \geq 0}$ instead of the information of the entire trajectory of $(X^\eta_n)_{n \geq 0}$:
\begin{align}
    \hat{X}^{\eta,\Delta}_n = \begin{cases}\phantom{-} m_{I^G_k} & \text{if }\exists k \geq 0\ \text{such that }\tau^G_k \leq n < \tau^G_{k+1} \\ \ \  0 & \text{otherwise}  \end{cases} \label{def X hat marker}
\end{align}
In other words, when $n < \tau^G_0$ we simply let $\hat{X}^{\eta,\Delta}_n = 0$, otherwise it is equal to the latest ``marker'' for the last visited wide minimum up until step $n$. This marker process $\hat{X}$ jumps between the different minimizers of the large attractions in $G$. In particular, if for some $n$ we have $X^\eta_n \in B(m_j,2\Delta)$ for some $j$ with $m_j \in G^\text{large}$, then we must have $\hat{X}^{\eta,\Delta}_n = m_j$, which implies that, in this case, $\hat{X}^{\eta,\Delta}_n$ indeed indicates the location of $X^\eta_n$. 

 Note that results in Theorem \ref{thm main paper MC GP absorbing case} and \ref{thm main paper MC GP transient case} concern a \textit{scaled} version of $X^\eta$. Here we also define the corresponding \textit{scaled} version of the processes
\begin{align}
    X^{*,\eta}_t & \delequal{} X^\eta_{ \floor{ t / \lambda_G(\eta) } } \label{def X hat star marker scaled 1}
    \\
    \hat{X}^{*,\eta,\Delta}_t & \delequal{} \hat{X}^{\eta,\Delta}_{ \floor{ t / \lambda_G(\eta) } }, \label{def X hat star marker scaled 2}
\end{align}
a mapping $\textbf{T}^*(n,\eta) \delequal{} n\lambda_G(\eta)$ that translates a step $n$ to the corresponding timestamp for the scaled processes, and the following series of scaled stopping times
\begin{align}
    \tau^*_k(\eta,\Delta) = \textbf{T}^*\big( \tau^G_k(\eta,\Delta), \eta\big),\ \sigma^*_k(\eta,\Delta) = \textbf{T}^*\big( \sigma^G_k(\eta,\Delta), \eta\big). \label{def tau star sigma star scaled}
\end{align}

Before presenting the proof of Theorem \ref{thm main paper MC GP absorbing case} and \ref{thm main paper MC GP transient case}, we make several preparations. First, our proof is inspired by ideas in \cite{pavlyukevich2005metastable} and here we provide a briefing. At any time $t>0$, if we can show that $X^{*,\eta}_t$ is almost always in set $\bigcup_{i: m_i \in G^\text{large}}B(m_i,2\Delta)$ (so the SGD iterates is almost always close to a minimizer in a large attraction field), then the marker process $\hat{X}^{*,\eta,\Delta}_t$ is a pretty accurate indicator of the location of $X^{*,\eta}_t$, so it suffices to show that the marker process $\hat{X}^{*,\eta,\Delta}_t$ converges to a continuous-time Markov chain $Y$. 

Second, we construct the limiting process $Y$ and the random mapping $\pi_G$ before utilizing them in Theorem \ref{thm main paper MC GP absorbing case} and \ref{thm main paper MC GP transient case}. As an important building block for this purpose, we start by  considering the following discrete time Markov chain (DTMC) on the entire graph $\mathcal{G} = (V,E)$. Let $\textbf{P}^{DTMC}$ be a transition matrix with
$ \textbf{P}^{DTMC}(m_i,m_j) = \mu_i(E_{i,j})/\mu_i(E_i)$ for all $j \neq i$, and $Y^{DTMC} = (Y^{DTMC}_j)_{j \geq 0}$ be the DTMC induced by the said transition matrix. Let
\begin{align}
    T^{DTMC}_G \delequal{}\min\{ j \geq 0:\ Y^{DTMC}_j \notin G^\text{small} \} \label{def T G DTMC}
\end{align}
be the first time this DTMC visits a large attraction field on the communication class $G$, or escapes from $G$. Lastly, define (for any $j$ such that $m_j \notin G^\text{small}$)
\begin{align}
    p_{i,j} \delequal{} \mathbb{P}\big( Y^{DTMC}_{T^{DTMC}_G}(m_i) = m_j \big) \label{def p i j}
\end{align}
as the probability that the first large attraction field on $G$ visited by $Y^{DTMC}$ is $m_j$ when initialized at $m_i$. 

We add a comment regarding the stopping times $T^{DTMC}_G$ and probabilities $p_{i,j}$ defined above. In the case that $G$ is absorbing, we have $Y^{DTMC}_j(m_i) \in G$ for all $j \geq 0$ if $m_i \in G$. Therefore, in this case, given any $i$ with $m_i \in G$, we must have
\begin{align*}
    T^{DTMC}_G =\min\{ j \geq 0:\ Y^{DTMC}_j(m_i) \in G^\text{large} \},\ \ \ \sum_{j:\ m_j \in G^\text{large}}p_{i,j} = 1.
\end{align*}
On the contrary, when $G$ is transient we may have $\sum_{j:\ m_j \in G^\text{large}}p_{i,j} < 1$ and $\sum_{j:\ m_j \notin G}p_{i,j} > 0$. Lastly, whether $G$ is absorbing or transient, we always have $p_{i,j} = \mathbbm{1}\{i = j\}$ if $m_i \in G^\text{large}$.

Next, consider the following definition of (continuous-time) jump processes.
\begin{definition} \label{def jump process}
A continuous-time process $Y_t$ on $\mathbb{R}$ is a $\Big((U_j)_{j \geq 0},(V_j)_{j \geq 0}\Big)$ \textbf{jump process} if
$$Y_t = \begin{cases}\phantom{-}0 & \text{if }t < U_0 \\ \sum_{j \geq 0}V_j\mathbbm{1}_{ [U_0 + U_1 + \cdots + U_j,\ U_0 + U_1 + \cdots +U_{j+1}) }(t) & \text{otherwise} \end{cases},$$
where $(U_j)_{j \geq 0}$ is a sequence of non-negative random variables such that $U_j > 0\ \forall j \geq 1$ almost surely, and $(V_j)_{j \geq 0}$ is a sequence of random variables in $\mathbb{R}$.
\end{definition}
Obviously, the definition above implies that $Y_t = V_j$ for any $t \in [U_j,U_{j+1})$. 

Now we are ready to construct the limiting continuous-time Markov chain $Y$. To begin with, we address the case where $G$ is absorbing. For any $m^\prime \in G^\text{large}$, let $Y(m^\prime)$ be a $\big( (S_k)_{k \geq 0},(W_k)_{k \geq 0} \big)$-jump process where $S_0 = 0, W_0 = m^\prime$ and (for all $k \geq 0$ and $i,j$ with $m_i\in G^\text{large},m_j \notin G^\text{small}$)
\begin{align}
    & \mathbb{P}\Big( W_{k+1} = m_j,\ S_{k+1} > t   \ \Big|\ W_k = m_i,\ (W_l)_{l = 0}^{k-1},\ (S_l)_{l = 0}^{k} \Big) \label{def limiting MC Y 1}
    \\
    = & \mathbb{P}\Big( W_{k+1} = m_j,\ S_{k+1} > t   \ \Big|\ W_k = m_i \Big) = \exp(-q_i t)\frac{ q_{i,j}  }{ q_i }\ \forall t > 0 \label{def limiting MC Y 2}
\end{align}
where
\begin{align}
    q_i & = \mu_i(E_i) \label{def limiting MC Y 3}
    \\
    q_{i,j} & = \mathbbm{1}\{i \neq j\}\mu_i(E_{i,j}) + \sum_{ k:\ m_k \in G^\text{small}  }\mu_i(E_{i,k})p_{k,j} \label{def limiting MC Y 4}
\end{align}
and $p_{k,j}$ is defined in \cref{def p i j}. In other words, conditioning on $W_k = m_i$, the time until next jump $S_{k+1}$ and the jump location $W_{k+1}$ are independent, where $S_{k+1}$ is Exp$(q_i)$ and $W_{k+1} = m_j$ with probability $q_{i,j}/q_i$. First, it is easy to see that $Y$ is a continuous-time Markov chain. Second, under this definition $Y$ is allowed to have some \textit{dummy} jumps where $W_k = W_{k+1}$: in this case the process $Y_t$ does not move to a different minimizer after the $k+1$-th jump, and by inspecting the path of $Y$ we cannot tell that this dummy jump has occurred. As a result, that generator $Q$ of this Markov chain admits the form (for all $i\neq j$ with $m_i,m_j \in G^\text{large}$)
\begin{align*}
    Q_{i,i} = -\sum_{k: k \neq i,\ m_k \in G^\text{large}}q_{i,k},\ Q_{i,j} = q_{i,j}.
\end{align*}
Moreover, define the following random function $\pi_G(\cdot)$ such that for any $m_i \in G$,
\begin{align}
    \pi_G(m_i) = \begin{cases}\phantom{-} m_j & \text{with probability $q_{i,j}/q_i$ if }m_i \in G^\text{small}\\ \phantom{-} m_i& \text{if }m_i \in G^\text{large} \end{cases} \label{def mapping pi G}
\end{align}
By $Y(\pi_G(m_i))$ we refer to the version of the Markov chain $Y$ where we randomly initialize $W_0 = \pi_G(m_i)$. The following lemma is the key tool for proving Theorem \ref{thm main paper MC GP absorbing case}.

\begin{lemma} \label{lemma key converge to markov chain and always at wide basin}
Assume that the communication class $G$ is absorbing. Given any $m_i \in G$, $x \in \Omega_i$, finitely many real numbers $(t_l)_{l = 1}^{k^\prime}$ such that $0<t_1<t_2<\cdots<t_{k^\prime}$, and a sequence of strictly positive real numbers $(\eta_n)_{n \geq 1}$ with $\lim_{n \rightarrow 0}\eta_n = 0$, there exists a sequence of strictly positive real numbers $(\Delta_n)_{n \geq 1}$ with $\lim_n \Delta_n = 0$ such that
\begin{itemize}
    \item As $n$ tends to $\infty$,
    \begin{align}
        \big(\hat{X}^{*,\eta_n,\Delta_n }_{t_1}(x),\cdots,\hat{X}^{*,\eta_n,\Delta_n }_{t_{k^\prime}}(x)\big) \Rightarrow \big( Y_{t_1}(\pi_G(m_i)),\cdots,Y_{t_{k^\prime} }(\pi_G(m_i))  \big) \label{lemma key converge to markov chain and always at wide basin goal 1}
    \end{align}
    
    \item For all $k \in [k^\prime]$,
    \begin{align}
        \lim_{n \rightarrow \infty}\mathbb{P}_x\Big( X^{*,\eta_n}_{t_k} \notin \bigcup_{j:\ m_j \in G^\text{large}}B(m_j,\Delta_n) \Big) = 0. \label{lemma key converge to markov chain and always at wide basin goal 2}
    \end{align}
\end{itemize}
\end{lemma}

Now we address the case where $G$ is transient, let $\boldsymbol{\dagger}$ be a real number such that $\boldsymbol{\dagger} \notin [-L,L]$, and we use $\dagger$ as the cemetery state since the processes $X^\eta_n$ or $X^{*,\eta}_t$ are restricted on $[-L,L]$. Recall the definition of $\tau_G$ defined in \cref{def tau G X exits G}. Analogous to the process $X^\dagger$ in \cref{def process X killed}, we can also define
\begin{align}
    X^{\dagger,*,\eta}_t = \begin{cases} 
X^{*,\eta}_t & \text{if }\ t < \textbf{T}^*(\tau_G(\eta),\eta) \\
\boldsymbol{\dagger} & \text{otherwise}
\end{cases}
,\ \ 
\hat{X}^{\dagger,*,\eta,\Delta}_t = \begin{cases} 
\hat{X}^{*,\eta,\Delta}_t & \text{if }\ t < \textbf{T}^*(\tau_G(\eta),\eta) \\
\boldsymbol{\dagger} & \text{otherwise},
\end{cases}
.\label{def process X hat X killed}
\end{align}
Next, analogous to $\tau_G$, consider the stopping time
\begin{align*}
    \tau^Y_G \delequal{} \min\{ t > 0: Y_t \notin G\}.
\end{align*}
When $G$ is transient, due to the construction of $Y$ we know that $\tau^Y_G < \infty$ almost surely. The introduction of $\tau^Y_G$ allows us to define
\begin{align}
    Y^\dagger_t = \begin{cases} 
Y_t & \text{if }\ t < \tau^Y_G \\
\boldsymbol{\dagger} & \text{otherwise.}
\end{cases} \label{def process Y killed}
\end{align}

The following Lemma will be used to prove Theorem \ref{thm main paper MC GP transient case}.
\begin{lemma} \label{lemma key converge to markov chain and always at wide basin transient case}
Assume that the communication class $G$ is transient. Given any $m_i \in G$, $x \in \Omega_i$, finitely many real numbers $(t_l)_{l = 1}^{k^\prime}$ such that $0<t_1<t_2<\cdots<t_{k^\prime}$, and a sequence of strictly positive real numbers $(\eta_n)_{n \geq 1}$ with $\lim_{n \rightarrow 0}\eta_n = 0$, there exists a sequence of strictly positive real numbers $(\Delta_n)_{n \geq 1}$ with $\lim_n \Delta_n = 0$ such that
\begin{itemize}
    \item As $n$ tends to $\infty$,
    \begin{align}
        \big(\hat{X}^{\dagger,*,\eta_n,\Delta_n }_{t_1}(x),\cdots,\hat{X}^{\dagger,*,\eta_n,\Delta_n }_{t_{k^\prime}}(x)\big) \Rightarrow \big( Y^\dagger_{t_1}(\pi_G(m_i)),\cdots,Y^\dagger_{t_{k^\prime} }(\pi_G(m_i))  \big) \label{lemma key converge to markov chain and always at wide basin transient case goal 1}
    \end{align}
    
    \item For all $k \in [k^\prime]$,
    \begin{align}
        \lim_{n \rightarrow \infty}\mathbb{P}_x\Big( X^{\dagger,*,\eta_n}_{t_k} \notin \bigcup_{j:\ m_j \in G^\text{large}}B(m_j,\Delta_n) \text{ and }X^{\dagger,*,\eta_n}_{t_k} \neq \boldsymbol{\dagger} \Big) = 0. \label{lemma key converge to markov chain and always at wide basin transient case goal 2}
    \end{align}
\end{itemize}
\end{lemma}

\begin{proof}[Proof of Theorem \ref{thm main paper MC GP absorbing case} and \ref{thm main paper MC GP transient case}.]
We first address the case where $G$ is absorbing. Arbitrarily choose some $\Delta > 0$, a sequence of strictly positive real numbers $(\eta_n)_{n \geq 1}$ with $\lim_n \eta_n = 0$, a positive integer $k^\prime$, a series of real numbers $(t_j)_{j = 1}^{k^\prime}$ with $0 < t_1 < \cdots < t_{k^\prime}$, and a sequence $(w_j)_{j = 1}^{k^\prime}$ with $w_j \in G^\text{large}$ for all $j \in [k^\prime]$. It suffices to show that
\begin{align*}
    \lim_n \mathbb{P}_x\Big( X^{*,\eta_n}_{t_k} \in B(w_k,\Delta)\ \forall k \in [k^\prime] \Big) = \mathbb{P}\Big( Y_{t_k}(\pi_G(m_i)) = w_k\ \forall k \in [k^\prime] \Big).
\end{align*}
Using Lemma \ref{lemma key converge to markov chain and always at wide basin}, we can find a sequence of strictly positive real numbers $(\Delta_n)_{n \geq 1}$ with $\lim_n \Delta_n = 0$ such that \cref{lemma key converge to markov chain and always at wide basin goal 1} and \cref{lemma key converge to markov chain and always at wide basin goal 2} hold. From the weak convergence in \cref{lemma key converge to markov chain and always at wide basin goal 1}, we only need to show
\begin{align*}
    \lim_n \mathbb{P}_x\Big( X^{*,\eta_n}_{t_k} \notin B(\hat{X}^{*,\eta_n,\Delta_n}_{t_k} ,\Delta) \Big) = 0 \ \forall k \in [k^\prime].
\end{align*}
For all $n$ large enough, we have $2\Delta_n < \Delta$. For such large $n$, observe that
\begin{align*}
    & \mathbb{P}_x\Big( X^{*,\eta_n}_{t_k} \notin B(\hat{X}^{*,\eta_n,\Delta_n}_{t_k} ,\Delta) \Big)
    \\
    \leq & \mathbb{P}_x\Big( X^{*,\eta_n}_{t_k} \notin \bigcup_{j: m_j \in G^\text{large}}B(m_j,2\Delta_n) \Big)
    \\
    &\text{(due to definition of the marker process $\hat{X}$, see \cref{def tau G sigma G I G 1}-\cref{def tau G sigma G I G 6} and \cref{def X hat marker}-\cref{def X hat star marker scaled 2})}
    \\
    \leq & \mathbb{P}_x\Big( X^{*,\eta_n}_{t_k} \notin \bigcup_{j: m_j \in G^\text{large}}B(m_j,\Delta_n) \Big)
\end{align*}
and by applying \cref{lemma key converge to markov chain and always at wide basin goal 2} we conclude the proof for Theorem \ref{thm main paper MC GP absorbing case}. 

The proof of Theorem \ref{thm main paper MC GP transient case} is almost identical, with the only modification being that we apply Lemma \ref{lemma key converge to markov chain and always at wide basin transient case} instead of Lemma \ref{lemma key converge to markov chain and always at wide basin}. In doing so, we are able to find a sequence of $(\Delta_n)_{n \geq 1}$ with $\lim_n \Delta_n = 0$ such that \cref{lemma key converge to markov chain and always at wide basin transient case goal 1} and \cref{lemma key converge to markov chain and always at wide basin transient case goal 2} hold. Given the weak convergence claim in \cref{lemma key converge to markov chain and always at wide basin transient case goal 1}, it suffices to show that 
\begin{align*}
    \lim_n \mathbb{P}_x\Big( X^{\dagger*,\eta_n}_{t_k} \notin B(\hat{X}^{\dagger,*,\eta_n,\Delta_n}_{t_k} ,\Delta) \Big) = 0 \ \forall k \in [k^\prime].
\end{align*}
If $X^{\dagger,*,\eta_n}_{t_k} = \boldsymbol{\dagger}$, we must have $\hat{X}^{\dagger,*,\eta_n,\Delta_n}_{t_k} = \boldsymbol{\dagger}$ as well. Therefore, for all $n$ large enough so that $3\Delta_n < \Delta$,
\begin{align*}
    & \mathbb{P}_x\Big( X^{\dagger*,\eta_n}_{t_k} \notin B(\hat{X}^{\dagger,*,\eta_n,\Delta_n}_{t_k} ,\Delta) \Big)
    \\
    = & \mathbb{P}_x\Big( X^{\dagger*,\eta_n}_{t_k} \notin B(\hat{X}^{\dagger,*,\eta_n,\Delta_n}_{t_k} ,\Delta),\ X^{\dagger*,\eta_n}_{t_k}\neq \boldsymbol{\dagger} \Big)
    \\
    \leq & \mathbb{P}_x\Big( X^{\dagger*,\eta_n}_{t_k} \notin \bigcup_{j:\ m_j \in G^\text{large}}B(m_j,2\Delta_n),\ X^{\dagger*,\eta_n}_{t_k}\neq \boldsymbol{\dagger} \Big)
    \\
    \leq & \mathbb{P}_x\Big( X^{\dagger*,\eta_n}_{t_k} \notin \bigcup_{j:\ m_j \in G^\text{large}}B(m_j,\Delta_n),\ X^{\dagger*,\eta_n}_{t_k}\neq \boldsymbol{\dagger} \Big).
\end{align*}
Apply \cref{lemma key converge to markov chain and always at wide basin transient case goal 2} and we conclude the proof.
\end{proof}

\subsection{Proof of Lemma \ref{lemma eliminiate sharp minima from trajectory}}

First, we introduce another dichotomy for \textit{small} and \textit{large} noises. For any $\widetilde{\gamma} > 0$ and any learning rate $\eta > 0$, we say that a noise $Z_n$ is small if 
\begin{align*}
    \eta|Z_n| > \eta^{\widetilde{\gamma}}
\end{align*}
and we say $Z_n$ is large otherwise. For this new classification of small and large noises, we introduce the following notations and definitions:
\begin{align}
    Z^{\leq,\widetilde{\gamma},\eta}_n & = Z_n\mathbbm{1}\{\eta|Z_n| \leq \eta^{\widetilde{\gamma}}\}, \\
     Z^{>,\widetilde{\gamma},\eta}_n & = Z_n\mathbbm{1}\{\eta|Z_n| > \eta^{\widetilde{\gamma}}\}, \\
     \widetilde{T}^{\eta}_1(\widetilde{\gamma}) & \delequal{} \min\{ n \geq 1:\ \eta|Z_n| > \eta^{\widetilde{\gamma}}   \}. \label{def tilde T arrival time of large jump rescaled}
\end{align}
Similar to Lemma \ref{lemma prob event A}, the following result is a direct application of Lemma \ref{lemma bernstein bound small deviation for small jumps}, and shows that it is rather unlikely to observe large perturbation that are caused only by \textit{small} noises. Specifically, since $\alpha > 1$ we can always find
\begin{align*}
    \widetilde{\gamma} & \in (0, (1-\frac{1}{\alpha}\wedge\frac{1}{2}) \\
    \beta & \in \big(1, (2-2\widetilde{\gamma})\wedge \alpha(1-\widetilde{\gamma})\big).
\end{align*}
Now in Lemma \ref{lemma bernstein bound small deviation for small jumps}, if we let $\Delta = \widetilde{\gamma}, \widetilde{\Delta} = \Delta/2$ and $\epsilon = \delta = 1$ (in other words, $u(\eta) = 1/\eta^{1 - \widetilde{\gamma}},\ v(\eta) = \eta^{\widetilde{\gamma}/2}$), then for any positive integer $j$ the condition \cref{proof lemma bernstein parameter 5} is satisfied, allowing us to draw the following conclusion immediately as a corollary from Lemma \ref{lemma bernstein bound small deviation for small jumps}.

\begin{lemma} \label{lemma rare event A prob rescaled version}
Given $N > 0$ and 
\begin{align*}
    \widetilde{\gamma} \in (0, (1-\frac{1}{\alpha})\wedge\frac{1}{2} ),\ \beta \in \big(1,(2 - 2\widetilde{\gamma})\wedge ( \alpha - \alpha\widetilde{\gamma} ) \big),
\end{align*}
we have (as $\eta \downarrow 0)$
\begin{align*}
    \mathbb{P}\Big( \max_{j = 1,2,\cdots,\ceil{ (1/\eta)^\beta }} \eta|  Z^{\leq,\widetilde{\gamma},\eta}_1 + \cdots +  Z^{\leq,\widetilde{\gamma},\eta}_j | > \eta^{\widetilde{\gamma}/2}  \Big) = o(\eta^N).
\end{align*}
\end{lemma}

The flavor of the next lemma is similar to that of Lemma \ref{lemma return to local minimum quickly}. Specifically, we show that, with high probability, the SGD iterates would quickly return to the local minimum as long as they start from somewhere that are not too close the boundary of an attraction field (namely, the points $s_1,s_2,\cdots,s_{n_\text{min}}$). To this end, we consider a refinement of function $\hat{t}(\cdot)$ defined in \cref{def function hat t}. For any $i = 1,2,\cdots,n_\text{min}$, any $x \in \Omega_i$ and any $\eta > 0, \gamma \in (0,1)$, we can define the return time to $\eta^\gamma-$neighborhood for the ODE $\textbf{x}^\eta$ as
\begin{align*}
\hat{t}^{(i)}_\gamma(x,\eta) \delequal{} \min\{ t \geq 0: |\textbf{x}^\eta(t,x) - m_i| \leq \eta^\gamma \}.
\end{align*}

Given the bound in \cref{ineq prior to function hat t} (which is stated for a specific attraction field) and the fact that there only exists finitely many attraction fields, we know the existence of some $c_2 < \infty$ such that for any $i = 1,2,\cdots,n_\text{min}$, any $\eta > 0$, any $\gamma \in (0,1)$ and any $x \in \Omega_i$ such that $|x - s_i|\vee|x - s_{i-1}| > \eta^\gamma$, we have
\begin{align*}
    \hat{t}^{(i)}_\gamma(x,\eta)\leq c_2\gamma \log(1/\eta)/\eta
\end{align*}
and define function $t^\uparrow$ as
\begin{align*}
    t^\uparrow(\eta,\gamma) \delequal{} c_2\gamma \log(1/\eta).
\end{align*}
Lastly, define the following stopping time for any $i = 1,2,\cdots,n_\text{min}$, any $x \in \Omega_i$ and any $\Delta > 0$
\begin{align*}
    T^{(i)}_\text{return}(\eta,\Delta) \delequal{}\min\{ n \geq 0: X^\eta_n(x) \in B(m_i,2\Delta) \}
\end{align*}
where we adopt the notation $B(u,v) \delequal{} [u - v, u+v]$ for the $v-$neighborhood around point $u$.
\begin{lemma} \label{lemma return to local minimum quickly rescaled version}
Given 
\begin{align*}
     \widetilde{\gamma} \in (0, (1-\frac{1}{\alpha})\wedge(\frac{1}{2}) ),\ \gamma \in (0,\frac{\widetilde{\gamma}}{16 M c_2}\wedge\frac{\widetilde{\gamma}}{4} ),
\end{align*}
and any $i = 1,2,\cdots,n_\text{min}$, any $\Delta > 0$, we have
\begin{align*}
    & \liminf_{\eta \downarrow 0} \inf_{ x \in \Omega_i: |x - s_{i-1}|\vee |x - s_i| \geq 2\eta^\gamma }\mathbb{P}_x\Big( T^{(i)}_\text{return}(\eta,\Delta) \leq \frac{2c_2\gamma \log(1/\eta)}{\eta},\ X^\eta_n \in \Omega_i\ \forall n \leq T^{(i)}_\text{return}(\eta,\Delta)  \Big)
    \\
    & = 1.
\end{align*}
\end{lemma}

\begin{proof}
Throughout this proof, we only consider $\eta$ small enough such that
\begin{align}
    2c_2\gamma \log(1/\eta)/\eta < \ceil{ (1/\eta)^\beta },\ \ \eta M \leq \eta^{2\widetilde{\gamma}},\ \ 2\eta^\gamma < \Delta/2,\ \ 2\eta^{\widetilde{\gamma}/4} < \eta^\gamma. \label{proof choice of eta in lemma for efficient return rescaled}
\end{align}
The condition above holds for all $\eta > 0$ sufficiently small because $\beta > 1$, $2\widetilde{\gamma} < 1$, and $\gamma < \widetilde{\gamma}/4$. Also, fix some $\beta \in \big(1,(2 - 2\widetilde{\gamma})\wedge ( \alpha - \alpha\widetilde{\gamma} ) \big)$

Define the following events
\begin{align*}
    A^\times_1( \eta )& \delequal{}\Big\{\max_{j = 1,2,\cdots,\ceil{ (1/\eta)^\beta }} \eta|  Z^{\leq,\widetilde{\gamma},\eta}_1 + \cdots +  Z^{\leq,\widetilde{\gamma},\eta}_j | > \eta^{\widetilde{\gamma}/2} \Big\} \\
    A^\times_2(\eta) & \delequal{}\{ \widetilde{T}^{\eta}_1(\widetilde{\gamma}) \leq \ceil{ (1/\eta)^\beta }  \}\ \ \ \text{(see \cref{def tilde T arrival time of large jump rescaled} for definition of the stopping time involved)}
\end{align*}
and fix some $N > 0$. From Lemma \ref{lemma rare event A prob rescaled version}, we see that (for all sufficiently small $\eta$)
$$ \mathbb{P}(A^\times_1(\eta)) \leq \eta^N.$$
Besides, using Lemma \ref{lemmaGeomFront} together with the fact that $\beta < \alpha(1 - \widetilde{\gamma})$, we know the existence of some constant $\theta > 0$ such that
\begin{align*}
   \mathbb{P}(A^\times_2(\eta)) \leq \eta^\theta
\end{align*}
for all sufficiently small $\eta$. 

Now we focus on the behavior of the SGD iterates on event $\Big( A^\times_1(\eta) \cap A^\times_2(\eta) \Big)^c$. Let us arbitrarily choose some $x \in \Omega_i$ such that $|x - s_i|\vee|x - s_{i-1}|>2\eta^\gamma$. First, from Lemma \ref{lemma Ode Gd Gap} and \cref{proof choice of eta in lemma for efficient return rescaled}, we know that
\begin{align}
    |\textbf{x}^\eta_t(x) - \textbf{y}^\eta_{\floor{t}}(x)| & \leq 2\eta M\exp\big( 2Mc_2\gamma \log(1/\eta) \big) \nonumber
    \\
    & \leq 2\eta^{ 2\widetilde{\gamma} - 2Mc_2\gamma } \leq 2\eta^{\widetilde{\gamma}}\leq \eta^\gamma\ \ \forall t \leq  2c_2\gamma \log(1/\eta)/\eta \label{proof ineq 1 in lemma for efficient return rescaled}
\end{align}
Next, from the definition of the function $t^\uparrow(\cdot)$ and \cref{proof ineq 1 in lemma for efficient return rescaled}, we know that for
\begin{align}
    T_{\text{GD,return}}(x;\eta,\Delta) \delequal{} \min\{n \geq 0:\textbf{y}^\eta_n(x) \in B(m_i,\frac{\Delta}{2} + \eta^\gamma) \}, \label{def return time for GD in lemma for efficient return rescaled}
\end{align}
we have
\begin{align}
    T_{\text{GD,return}}(x;\eta,\Delta) & \leq 2c_2\gamma \log(1/\eta)/\eta \label{proof ineq 2 in lemma for efficient return rescaled} \\ 
   \textbf{y}^\eta_n > s_{i-1} + \eta^\gamma,\ \textbf{y}^\eta_n & < s_{i} - \eta^\gamma\ \ \forall n \leq 2c_2\gamma \log(1/\eta)/\eta. \label{proof ineq 3 in lemma for efficient return rescaled}
\end{align}
Furthermore, on event $\Big( A^\times_1(\eta) \cap A^\times_2(\eta) \Big)^c$, due to Lemma \ref{lemmaBasicGronwall} and \cref{proof choice of eta in lemma for efficient return rescaled}, we have that
\begin{align*}
    | X^\eta_n(x) - \textbf{y}^\eta_n(x) | \leq \eta^{\widetilde{\gamma}/2}\exp\big( 2Mc_2 \log(1/\eta) \big) = \eta^{ \frac{\widetilde{\gamma}}{2} - 2Mc_2\gamma } \leq \eta^{\widetilde{\gamma}/4} < \eta^\gamma\ \ \forall n \leq  2c_2\gamma \log(1/\eta)/\eta.
\end{align*}
Combining this with \cref{proof ineq 2 in lemma for efficient return rescaled}, \cref{proof ineq 3 in lemma for efficient return rescaled}, we can conclude that (recall that due to \cref{proof choice of eta in lemma for efficient return rescaled} we have $2\eta^\gamma < \Delta/2$)
\begin{align*}
    T^{(i)}_\text{return}(\eta,\Delta) & \leq 2c_2\gamma \log(1/\eta)/\eta \\
    X^\eta_n \in \Omega_i\ \ & \forall n \leq 2c_2\gamma \log(1/\eta)/\eta
\end{align*}
on event $\Big( A^\times_1(\eta) \cap A^\times_2(\eta) \Big)^c$. Therefore,
\begin{align*}
    & \liminf_{\eta \downarrow 0} \inf_{ x \in \Omega_i: |x - s_{i-1}|\vee |x - s_i| \geq 2\eta^\gamma }\mathbb{P}_x\Big( T^{(i)}_\text{return}(\eta,\Delta) \leq \frac{2c_2\gamma \log(1/\eta)}{\eta},\ X^\eta_n \in \Omega_i\ \forall n \leq T^{(i)}_\text{return}(\eta,\Delta)  \Big) \\
    & \geq \liminf_{\eta \downarrow 0} \mathbb{P}\Big( \Big( A^\times_1(\eta) \cap A^\times_2(\eta) \Big)^c \Big) \geq \liminf_{\eta \downarrow 0} 1 - \eta^N - \eta^\theta = 1.
\end{align*}
This concludes the proof.
\end{proof}

The takeaway of the next lemma is that, almost always, the SGD iterates will quickly escape from the neighborhood of any $s_i$, the boundaries of each attraction fields.

\begin{lemma} \label{lemma exit smaller neighborhood of s}
Given any $\gamma \in (0,1), t > 0$, we have
\begin{align*}
    \liminf_{\eta \downarrow 0} \inf_{x \in [-L,L]}\mathbb{P}_x\Big( \min\big\{ n \geq 0: X^\eta_n \notin \bigcup_{i}B(s_i,2\eta^\gamma) \big\} \leq \frac{t}{H(1/\eta)} \Big) = 1.
\end{align*}
\end{lemma}

\begin{proof}
We only consider $\eta$ small enough so that 
\begin{align*}
    \min_{ i = 2,3,\cdots,n_\text{min} - 1 }|s_i - s_{i-1}| & > 3\eta^{\frac{1+\gamma}{2}}, \\
    \eta M & < \eta^\gamma.
\end{align*}
Also, the claim is trivial if $x \notin \cup_j B(s_j,2\eta^\gamma)$, so without loss of generality we only consider the case where there is some $j \in [n_\text{min}]$ and $x \in [-L,L], x \in B(s_j,2\eta^\gamma)$.
Let us define stopping times
\begin{align}
    T^\gamma &\delequal{}\min\{ n \geq 1: \eta|Z_n| > 5\eta^\gamma \}; \label{proof exit smaller neighborhood of s choose eta}\\
    T^\gamma_\text{escape}&\delequal{}\min\{ n \geq 0: X^\eta_n \notin \cup_j B(s_j,2\eta^\gamma) \},
\end{align}
and the following two events
\begin{align*}
    A^\times_1(\eta) & \delequal{}\{ T^\gamma > \frac{t}{H(1/\eta)} \}, \\
    A^\times_2(\eta) & \delequal{}\{ \eta|Z_{T^\gamma}| > \eta^{ \frac{1 + \gamma}{2} } \}.
\end{align*}
First, using Lemma \ref{lemmaGeomDistTail} and the regularly varying nature of $H(\cdot) = \mathbb{P}(|Z_1| > \cdot)$, we know the existence of some $\theta > 0$ such that
\begin{align*}
\mathbb{P}(A^\times_1(\eta)) \leq \exp( -1/\eta^\theta )
\end{align*}
for all $\eta>0$ sufficiently small. Next, by definition of $T^\gamma$, one can see that (for any $\eta \in (0,1)$)
\begin{align*}
    \mathbb{P}(A^\times_2(\eta)) & = \frac{H(1/\eta^{\frac{1-\gamma}{2}})}{ H(5/\eta^{1 - \gamma}) }.
\end{align*}
Again, due to $H \in \mathcal{RV}_{-\alpha}$ and $1 - \gamma > 0$, we know the existence of some $\theta_1 > 0$ such that
\begin{align*}
    \mathbb{P}(A^\times_2(\eta)) < \eta^{\theta_1}
\end{align*}
for all $\eta>0$ sufficiently small. To conclude the proof, we only need to note the following fact on event $\big(A^\times_1(\eta) \cup A^\times_2(\eta)\big)^c$. There are only two possibilities on this event: $T^\gamma_\text{escape} \leq T^\gamma -1$, or $T^\gamma_\text{escape} \geq T^\gamma $. Now we analyze the two cases respectively.
\begin{itemize}
    \item On $\big(A^\times_1(\eta) \cup A^\times_2(\eta)\big)^c \cap \{T^\gamma_\text{escape} \leq T^\gamma -1\}$, we must have $T^\gamma_\text{escape} < T^\gamma \leq t/H(1/\eta)$.
    \item On $\big(A^\times_1(\eta) \cup A^\times_2(\eta)\big)^c \cap \{T^\gamma_\text{escape} \geq T^\gamma\}$, we know that at $n = T^\gamma - 1$, there exists an integer $j \in \{1,2,\cdots,n_\text{min} - 1\}$ such that $X^\eta_n \in B(s_j,2\eta^\gamma)$. Now since $\eta M < \eta^\gamma$ and $\eta|Z_{T^\gamma}| > 5\eta^\gamma$, we must have
    \begin{align*}
        |X^\eta_{T^\gamma} - X^\eta_{T^\gamma - 1}| > 4\eta^\gamma \Rightarrow X^\eta_{T^\gamma} \notin B(s_j,2\eta^\gamma).
    \end{align*}
    On the other hand, the exclusion of event $A^\times_2(\eta)$ tells us that $|X^\eta_{T^\gamma} - X^\eta_{T^\gamma - 1}| < 2\eta^{\frac{1+\gamma}{2}}$. Due to \cref{proof exit smaller neighborhood of s choose eta}, we then have $X^\eta_n \notin \cup_i B(s_i,2\eta^\gamma)$.
\end{itemize}
In summary, $\big(A^\times_1(\eta) \cup A^\times_2(\eta)\big)^c \subseteq \{ T^\gamma_\text{return} \leq t/H(1/\eta) \}$ and this conclude the proof.
\end{proof}

In the next lemma, we analyze the number of transitions needed to visit a certain local minimizer in the loss landscape. In general, we focus on a communication class $G$ and, for now, assume it is absorbing. Next, we introduce the following concepts to record the transitions between different local minimum. To be specific, for any $\eta > 0$ and any $\Delta > 0$ small enough so that $B(m_j,\Delta)\cap \Omega_j^c = \emptyset$ for all $j$, define
\begin{align}
    T_0(\eta,\Delta) & = \min\{ n \geq 0:\ X_n^\eta \in \cup_{j}B(m_j,2\Delta) \}; \label{def MC GP T 0} \\
    I_0(\eta,\Delta) & = j \text{ iff }X^\eta_{T_0(\eta,\Delta)} \in B(m_j,2\Delta); \label{def MC GP I 0}\\
    T_k(\eta,\Delta) & = \min\{ n > T_{k-1}(\eta,\Delta):\ X_n^\eta \in \cup_{j \neq I_{k-1}(\eta,\Delta)}B(m_j,2\Delta) \}\ \ \forall k \geq 1 \label{def MC GP T k} \\
    I_k(\eta,\Delta) & = j \text{ iff }X^\eta_{T_k(\eta,\Delta)} \in B(m_j,2\Delta)\ \ \forall k \geq 1. \label{def MC GP I k}
\end{align}
As mentioned earlier, the next goal is to analyze the transitions between attraction fields it takes to visit $m_j$ when starting from $m_i$ when $m_i,m_j \in G$. Define
\begin{align*}
    K_i(\eta,\Delta) & \delequal{} \min\{k \geq 0:\ I_k(\eta,\Delta) = i\}.
\end{align*}

\begin{lemma} \label{lemma geom bound on transition count} 
Assume that $G$ is an absorbing communication class on the graph $\mathcal{G}$. Then there exists some constant $p > 0$ such that for any $i$ with $m_i \in G$, any $\epsilon > 0$, and any $\Delta > 0$,
\begin{align*}
    &\sup_{j:\ m_j \in G;\ x \in B(m_j,2\Delta)}\mathbb{P}_x( K_i(\eta,\Delta) > u\cdot n_\text{min} ) \leq \mathbb{P}(\text{Geom}(p) \geq u) + \epsilon\ \ \forall u = 1,2,\cdots, 
    \\
     &\sup_{j:\ m_j \in G;\ x \in B(m_j,2\Delta)}\mathbb{P}_x\Big( \exists k \in [K_i(\eta,\Delta)]\ s.t.\ m_{I_k(\eta,\Delta)} \notin G \Big) \leq \epsilon
\end{align*}
hold for all $\eta > 0$ sufficiently small.
\end{lemma}

\begin{proof} The claim is trivial if, for the initial condition, we have $x \in B(m_i,2\Delta)$. Next, let us observe the following facts.
\begin{itemize}
    \item Define (recall the definitions of measure $\mu_i$ and sets $E_i,E_{i,j}$ in \cref{defMuMeasure i}\cref{def set E_i}\cref{def set E_ij})
    \begin{align*}
        J(j) & \delequal{} \arg\min_{ \widetilde{j}: \mu_i(E_{j,\widetilde{j}}) > 0  }|i - \widetilde{j}|\ \ \forall j \neq i \\
        p^* & \delequal{} \min_{j:\ j \neq i,\ m_j \in G}\frac{ \mu_j(E_{j,J(j)})  }{ \mu_j(E_j) }.
    \end{align*}
    \item From the definition of $J(j)$ and the fact that there are only finitely many attraction fields we can see that $p^* > 0$. Moreover, $G$ being a communication class implies that $$|J(j) - i| < |j - i|\ \ \ \forall j \neq i,\ m_j \in G.$$ 
    Indeed, if $i < j$, then since $G$ is a communication class and there are some $m_i \in G$ with $i < j$, we will at least have $\mu_j(E_{j,j-1}) > 0$, so $|J(j) - i| \leq |i - j| - 1$; the case that $i > j$ can be approached analogously.
    
    \item Now from the definition of $J(j)$ and Proposition \ref{prop first exit and return time i}, together with the previous bullet point, we know that for all $\eta$ sufficiently small,
    \begin{align*}
        \inf_{x \in [-L,L]}\mathbb{P}_x( |I_{k+1} - i| \leq |I_k - i| - 1,\ m_{I_{k+1}} \in G \ |\ K_i(\eta,\Delta) > k,\ m_{I_{k}} \in G) \geq p^*/2
    \end{align*}
    uniformly for all $k \geq 0$. 
    
    \item Meanwhile, since $p^* > 0$, we are able to fix some $\delta > 0$ small enough such that
    \begin{align*}
      \frac{ n_\text{min}\delta }{ ({p^*}/{2})^{n_\text{min}} } < \epsilon. 
    \end{align*}
    
    \item On the other hand, for any $\widetilde{j}$ with $m_{\widetilde{j}} \notin G$, by definition of the typical transition graph we must have $\mu_j(E_{j,\widetilde{j}}) = 0$ for any $j$ with $m_j \in G$. Then due to Proposition \ref{prop first exit and return time i} again, one can see that for all $\eta >0$ that is sufficiently small,
    \begin{align*}
        \sup_{x\in [-L,L] }\mathbb{P}_x(  m_{I_{k+1}} \notin G \ |\ K_i(\eta,\Delta) > k,\ m_{I_{k}} \in G) < \delta
    \end{align*}
    uniformly for all $k \geq 0$.
    \item Repeat this argument for $n_\text{min}$ times, and we can see that for all $\eta$ sufficiently small
    \begin{align*}
        \inf_{x \in [-L,L]}\mathbb{P}_x( K_i(\eta,\Delta) \leq k + n_\text{min} \ |\ K_i(\eta,\Delta) > k,\ m_{I_k} \in G) & \geq \Big(\frac{p^*}{2}\Big)^{ n_\text{min} }
        \\
        \sup_{x \in [-L,L]}\mathbb{P}_x(\exists l \in [n_\text{min} ]\ \text{s.t.}\ m_{I_{k + l}} \notin G \ |\ K_i(\eta,\Delta) > k,\ m_{I_k} \in G  ) &\leq  n_\text{min}\delta
    \end{align*}
    uniformly for all $k \geq 1$.
    \item Lastly, to apply the bounds established above, we will make use of the following expression of several probabilities. For any $j$ with $m_j \in G$ and $x \in B(m_j,2\Delta)$ and any $u = 1,2,\cdots$,
    \begin{align*}
        & \mathbb{P}_x( \exists l \in [un_\text{min}]\ \text{s.t. }m_{I_l} \notin G,\ K_i(\eta,\Delta) > u\cdot n_\text{min}  ) 
        \\
        = & \sum_{v = 0}^{u - 1}\mathbb{P}_x\Big( \exists k \in [n_\text{min}]\ \text{s.t. }m_{I_{k+vn_\text{min}}} \notin G\ \Big|\ K_i(\eta,\Delta) > v\cdot n_\text{min},\ m_{I_l} \in G\ \forall l \leq vn_\text{min}  \Big)
        \\
        & \ \ \ \cdot \prod_{w = 0}^{v-1}\mathbb{P}_x\Big( K_i(\eta,\Delta) > (w+1)n_\text{min},\ m_{I_{k + wn_\text{min}}} \in G\ \forall k \in [n_\text{min}]\ \Big| \\
        &\ \ \ \ \ \ \ \ \ \ \ \ \ \ \ \ \ \ \ \ \ \ \ \ \ \ \ \ \ \ \ \ \ \ \ \ \ \ \ \ \ \ \ \ \ \ \ \ \ \ \ \ K_i(\eta,\Delta) > w\cdot n_\text{min},\ m_{I_l} \in G\ \forall l \leq wn_\text{min}  \Big)
        \\
        \\
        & \mathbb{P}_x( m_{I_l} \in G\ \forall l \in [un_\text{min}],\ K_i(\eta,\Delta) >  u\cdot n_\text{min}  )
        \\
        = & \prod_{v = 0}^{u - 1}\mathbb{P}\Big( K_i(\eta,\Delta) > (v+1)n_\text{min},\ m_{I_{k + vn_\text{min}}}\in G\ \forall k \in [n_\text{min}]\ \Big| 
        \\
        &\ \ \ \ \ \ \ \ \ \ \ \ \ \ \ \ \ \ \ \ \ \ \ \ \ \ \ \ \ \ \ \ \ \ \ \ \ \ \ \ \ \ \ \ \ \ \ \ \ \ \ \ K_i(\eta,\Delta) > vn_\text{min},\ m_{I_{k}}\in G\ \forall k \in [vn_\text{min}]\Big)
    \end{align*}
\end{itemize}
In summary, now we can see that (for sufficiently small $\eta$)
\begin{align*}
    & \sup_{j:\ m_j \in G;\ x \in B(m_j,2\Delta)}\mathbb{P}_x\Big( \exists k \in [K_i(\eta,\Delta)]\ s.t.\ m_{I_k(\eta,\Delta)} \notin G \Big)
    \\
    \leq & \sum_{u = 0}^\infty \sup_{j:\ m_j \in G;\ x \in B(m_j,2\Delta)}\mathbb{P}_x\Big(\exists v \in [n_\text{min}] \text{ such that }m_{I_{v + un_\text{min}}} \notin G, 
    \\
    & \ \ \ \ \ \ \ \ \ \ \ \ \ \ \ \ \ \ \ \ \ \ \ \ \ \ \ \ \ \ \ \ \ \ \ \ \ \ \ \ \ \ \ \ \ \ \ \ \ \ \ \ \ K_i(\eta,\Delta) > un_\text{min}, m_{I_{k}} \in G\ \forall k \in [un_\text{min}]\Big)
    \\
    \leq & \sum_{u = 0}^\infty \sup_{j:\ m_j \in G;\ x \in B(m_j,2\Delta)}\mathbb{P}_x\Big(\exists v \in [n_\text{min}] \text{ such that }m_{I_{v + un_\text{min}}} \notin G\ \Big|
    \\
    & \ \ \ \ \ \ \ \ \ \ \ \ \ \ \ \ \ \ \ \ \ \ \ \ \ \ \ \ \ \ \ \ \ \ \ \ \ \ \ \ \ \ \ \ \ \ \ \ \ \ \ \ \ K_i(\eta,\Delta) > un_\text{min}, m_{I_{k}} \in G\ \forall k \in [un_\text{min}]\Big)
    \\
    \cdot\prod_{v = 0}^{u - 1} & \sup_{j:\ m_j \in G;\ x \in B(m_j,2\Delta)}\mathbb{P}_x\Big( K_i(\eta,\Delta) > (v+1)n_\text{min}, m_{I_{k + vn_\text{min}}} \in G\ \forall k \in [n_\text{min}] \Big|
    \\
    & \ \ \ \ \ \ \ \ \ \ \ \ \ \ \ \ \ \ \ \ \ \ \ \ \ \ \ \ \ \ \ \ \ \ \ \ \ \ \ \ \ \ \ \ \ \ \ \ \ \ \ \ K_i(\eta,\Delta) > vn_\text{min}, m_{I_{k}} \in G\ \forall k \in [vn_\text{min}]\Big)
    \\
    \leq & \sum_{u \geq 0}n_\text{min}\delta\Big(1 - (\frac{p^*}{2})^{n_\text{min}}\Big)^{u-1}=  \frac{ n_\text{min}\delta }{ ({p^*}/{2})^{n_\text{min}} } \leq \epsilon
\end{align*}
and
\begin{align*}
   &  \sup_{j: m_j \in G,\ x \in B(m_j,2\Delta)}\mathbb{P}_x( K_i(\eta,\Delta) > u\cdot n_\text{min}  )
   \\
   \leq & \sup_{j: m_j \in G,\ x \in B(m_j,2\Delta)}\mathbb{P}_x( \exists l \in [un_\text{min}]\ \text{s.t. }m_{I_l} \notin G,\ K_i(\eta,\Delta) > u\cdot n_\text{min}  ) 
   \\
   + & \sup_{j: m_j \in G,\ x \in B(m_j,2\Delta)}\mathbb{P}_x( m_{I_l} \in G\ \forall l \in [un_\text{min}],\ K_i(\eta,\Delta) > u\cdot n_\text{min}  )
   \\
   \leq & \sum_{v = 1}^u n_\text{min}\delta\Big(1  - (\frac{p^*}{2})^{n_\text{min}}\Big)^{v - 1} + \Big(1 - (\frac{p^*}{2})^{n_\text{min}}\Big)^{u}
   \\
   \leq & \frac{ n_\text{min}\delta }{ ({p^*}/{2})^{n_\text{min}} }+ \Big(1 -  (\frac{p^*}{2})^{n_\text{min}}\Big)^{u}
   \\
   \leq & \epsilon + \Big(1 - (\frac{p^*}{2})^{n_\text{min}}\Big)^{u}
\end{align*}
uniformly for all $u = 1,2,\cdots$. To conclude the proof, it suffices to set $p = (\frac{p^*}{2})^{n_\text{min}}$.
\end{proof}

The proof above can be easily adapted to the case when the communication class $G$ is transient. Define 
\begin{align*}
    K^G_i(\eta,\Delta) & \delequal{} \min\{k \geq 0:\ I_k(\eta,\Delta) = i\ \text{or }m_{I_k(\eta,\Delta)} \notin G\}.
\end{align*}

\begin{lemma} \label{lemma geom bound on transition count transient case} 
Assume that $G$ is a transient communication class on the graph $\mathcal{G}$. Then there exists some constant $p > 0$ such that for any $i$ with $m_i \in G$ and any $\Delta \in (0,\bar{\epsilon}/3)$,
\begin{align}
    \sup_{j:\ m_j \in G;\ x \in B(m_j,2\Delta)}\mathbb{P}_x( K^G_i(\eta,\Delta) > u\cdot n_\text{min} ) & \leq \mathbb{P}(\text{Geom}(p) \geq u) \ \ \forall u = 1,2,\cdots \label{goal geom bound on transition count 1 transient case}
\end{align}
hold for all $\eta > 0$ sufficiently small.
\end{lemma}

\begin{proof} The structure of this proof is analogous to that of Lemma \ref{lemma geom bound on transition count}. Again, the claim is trivial if, for the initial condition, we have $x \in B(m_i,2\Delta)$. Next, let us observe the following facts.
\begin{itemize}
    \item Define (recall the definitions of measure $\mu_i$ and sets $E_i,E_{i,j}$ in \cref{defMuMeasure i}\cref{def set E_i}\cref{def set E_ij})
    \begin{align*}
        J(j) & \delequal{} \arg\min_{ \widetilde{j}: \mu_i(E_{j,\widetilde{j}}) > 0  }|i - \widetilde{j}|\ \ \forall j \neq i \\
        p^* & \delequal{} \min_{j:\ j \neq i,\ m_j \in G}\frac{ \mu_j(E_{j,J(j)})  }{ \mu_j(E_j) }.
    \end{align*}
    \item From the definition of $J(j)$ and the fact that there are only finitely many attraction fields we can see that $p^* > 0$. Moreover, $G$ being a communication class implies that $$|J(j) - i| < |j - i|\ \ \ \forall j \neq i,\ m_j \in G.$$ 
    Indeed, if $i < j$, then since $G$ is a communication class and there are some $m_i \in G$ with $i < j$, we will at least have $\mu_j(E_{j,j-1}) > 0$, so $|J(j) - i| \leq |i - j| - 1$; the case that $i > j$ can be approached analogously.
    
    \item Now from the definition of $J(j)$ and Proposition \ref{prop first exit and return time i}, together with the previous bullet point, we know that for all $\eta$ sufficiently small,
    \begin{align*}
        \inf_{x \in [-L,L]}\mathbb{P}_x( |I_{k+1} - i| \leq |I_k - i| - 1,\ m_{I_{k+1}} \in G \ |\ K^G_i(\eta,\Delta) > k \geq p^*/2
    \end{align*}
    uniformly for all $k \geq 0$. 
    
    \item Repeat this argument for $n_\text{min}$ times, and we can see that for all $\eta$ sufficiently small
    \begin{align*}
        \inf_{x \in [-L,L]}\mathbb{P}_x( K^G_i(\eta,\Delta) \leq k + n_\text{min} \ |\ K^G_i(\eta,\Delta) > k) & \geq \Big(\frac{p^*}{2}\Big)^{ n_\text{min} }
    \end{align*}
    uniformly for all $k \geq 1$.
    \item Lastly, for any $j \neq i$ with $m_j \in G$ and $x \in B(m_j,2\Delta)$ and any $u = 1,2,\cdots$,
    \begin{align*}
        & \mathbb{P}_x( K^G_i(\eta,\Delta) > u\cdot n_\text{min}  ) 
        \\
        = & \prod_{v = 0}^{u-1}\mathbb{P}_x\Big( K^G_i(\eta,\Delta) > (v+1)n_\text{min}\ \Big|\ K^G_i(\eta,\Delta) > v\cdot n_\text{min}\Big)
        \\
        = & \prod_{v = 0}^{u-1}\bigg( 1 - \mathbb{P}_x\Big( K^G_i(\eta,\Delta) \leq (v+1)n_\text{min}\ \Big|\ K^G_i(\eta,\Delta) > v\cdot n_\text{min}\Big) \bigg)
    \end{align*}
\end{itemize}
In summary, now we can see that (for suffciently small $\eta$)
\begin{align*}
   &  \sup_{j: m_j \in G,\ x \in B(m_j,2\Delta)}\mathbb{P}_x( K^G_i(\eta,\Delta) \geq u\cdot n_\text{min}  )
   \leq  \Big(1 -  (\frac{p^*}{2})^{n_\text{min}}\Big)^{u}
\end{align*}
uniformly for all $u = 1,2,\cdots$. To conclude the proof, it suffices to set $p = (\frac{p^*}{2})^{n_\text{min}}$.
\end{proof}

We are now ready to prove Lemma \ref{lemma eliminiate sharp minima from trajectory}, which, as demonstrated earlier, is the key tool in proof of Theorem \ref{thm main paper eliminiate sharp minima from trajectory}.

\begin{proof}[Proof of Lemma \ref{lemma eliminiate sharp minima from trajectory}]
The claim is trivial if $l^\text{large} = 1$, so we focus on the case where $l^\text{large} \geq 2$. Fix some
\begin{align*}
    \widetilde{\gamma} \in (0, (1-\frac{1}{\alpha})\wedge(\frac{1}{2}) ),\ \beta \in \big(1,(2 - 2\widetilde{\gamma})\wedge ( \alpha - \alpha\widetilde{\gamma} ) \big),\ \gamma \in (0,\frac{\widetilde{\gamma}}{16 M c_2}\wedge\frac{\widetilde{\gamma}}{4} ).
\end{align*}
Let $q^* = \max_j \mu_j(E_j(0))$. We show that for any $t \in (0,\frac{\delta}{4q^*})$ the claim is true.

Now we only consider $\Delta \in (0,\bar{\epsilon}/3)$ and $\eta$ small enough so that $\eta M \leq \eta^\gamma$ and $\eta^\gamma < \Delta$. Consider the following stopping times
\begin{align*}
    T^\gamma_\text{escape} &\delequal{}\min\{ n \geq 0: X^\eta_n \notin \cup_j B(s_j,2\eta^\gamma) \}; \\
    T^\gamma_\text{return}&\delequal{}\min\{ n \geq 0: X^\eta_n \in \cup_j B(m_j,2\eta^\gamma) \}.
\end{align*}

First, from Lemma \ref{lemma exit smaller neighborhood of s}, we know that 
\begin{align*}
    \sup_{x \in [-L,L]}\mathbb{P}_x( T^\gamma_\text{escape} > 1/H(1/\eta) ) < \delta/2
\end{align*}
for all $\eta$ sufficiently small. Besides, by combining Lemma \ref{lemma return to local minimum quickly rescaled version} with Markov property (applied at $T^\gamma_\text{escape}$), we have 
\begin{align*}
    \sup_{x \in [-L,L]}\mathbb{P}_x\Big(\ T^\gamma_\text{return} - T^\gamma_\text{escape} > 2c_2\gamma \log(1/\eta)/\eta\ \Big|\ T^\gamma_\text{escape} \leq \frac{1}{H(1/\eta)} \Big) < \delta /2
\end{align*}
for all $\eta$ sufficiently small. Therefore, for all $\eta$ sufficiently small, 
\begin{align}
    \sup_{x \in [-L,L]}\mathbb{P}\Big( T^\gamma_\text{return} > \frac{1}{ H(1/\eta) } + 2c_2\gamma\frac{\log(1/\eta)}{\eta} \Big) < \delta. \label{proof lemma trajectory ratio ineq 0}
\end{align}

Let $J$ be the unique index such that $X^\eta_{ T^\gamma_\text{return} } \in \Omega_J$. Our next goal is to show that, almost always, the SGD iterates will visit the local minimum at some \textit{large} attraction fields. Therefore, without loss of generality, we can assume that $m_J \notin M^\text{large}$, and define
\begin{align*}
    T^\gamma_\text{large} \delequal{} \min\{ n \geq T^\gamma_\text{return}:\ X^\eta_n \in \bigcup_{i: m_i \in M^\text{large}}B(m_i,2\Delta) \}
\end{align*}
and introduce the following definitions:
\begin{align*}
    \tau_0 & \delequal{} T^\gamma_\text{return},\ J_0 \delequal{} J \\
    \tau_k & \delequal{} \min\{ n > \tau_{k-1}:\ X^\eta_n \in \bigcup_{ j \neq J_{k-1} }B(m_j,2\Delta)   \} \\
    J_k & = j \Leftrightarrow X^\eta_{\tau_k } \in \Omega_j\ \ \forall k \geq 1 \\
    K & \delequal{} \min\{ k \geq 0:\ m_{J_k} \in M^\text{large} \}.
\end{align*}
In other words, the sequence of stopping times $(\tau_k)_{k \geq 1}$ is the time that, starting from $T^\gamma_\text{return}$, the SGD iterates visited a local minimum that is different from the one visited at $\tau_{k-1}$, and $(J_k)_{k \geq 0}$ records the label of the visited local minima. The random variable $K$ is the number of transitions required to visit a minimizer in a  \textit{large} attraction field. From Lemma \ref{lemma geom bound on transition count}, we know the existence of some $p^* > 0$ such that (for all $\eta$ sufficiently small)
\begin{align*}
    \sup_{x \in [-L,L]}\mathbb{P}_x( K \geq u\cdot n_\text{min}  ) \leq \mathbb{P}\Big( \text{Geom}( p^* )  \geq u\Big) + \frac{\delta}{2}\ \ \forall u = 1,2,3,\cdots.
\end{align*}
where $\text{Geom}(a)$ is a Geometric random variable with success rate $a \in (0,1)$. Therefore, one can find integer $N(\delta)$ such that (for all sufficiently small $\eta$)
\begin{align}
    \sup_{x \in [-L,L]}\mathbb{P}( K \geq N(\delta) ) \leq \delta. \label{proof lemma trajectory ratio ineq 1}
\end{align}
Next, given results in Proposition \ref{prop first exit and return time i} and the fact that there are only finitely many attraction fields, one can find a real number $u(\delta)$ such that (for all sufficiently small $\eta$)
\begin{align}
    \sup_{x \in [-L,L]}\mathbb{P}_x( \tau_k - \tau_{k-1} \leq \frac{ u(\delta) }{ \lambda_{J_{k-1}}(\eta) } ) \leq \delta/N(\delta)\label{proof lemma trajectory ratio ineq 2}
\end{align}
uniformly for all $k = 1,2,\cdots,N(\delta)$. From \cref{proof lemma trajectory ratio ineq 0}, \cref{proof lemma trajectory ratio ineq 1},\cref{proof lemma trajectory ratio ineq 2}, we now have
\begin{align}
    & \sup_{x \in [-L,L]}\mathbb{P}_x\Big(X^\eta_n \notin \bigcup_{j: m_j \in M^\text{large}}B(m_j,2\Delta)\ \forall n \leq N(\delta)u(\delta)\frac{H(1/\eta)/\eta}{ \lambda^\text{large}(\eta) } + \frac{1}{ H(1/\eta) } + 2c_2\gamma\frac{\log(1/\eta)}{\eta} \Big) \nonumber
    \\
    & \leq 3\delta \label{proof eliminate sharp minima ineq 1}
\end{align}
for any sufficiently small $\eta$. To conclude the proof  we just observe the following facts. First, due to $H \in \mathcal{RV}_{-\alpha}$ and $l^\text{large} \geq 2$, we have
\begin{align*}
    \lim_{\eta \downarrow 0}H(1/\eta)/\eta = 0,\ \ \lim_{\eta \downarrow 0}\frac{ \lambda^\text{large}(\eta) }{H(1/\eta)} = 0,\ \ \lim_{\eta \downarrow 0}\frac{\log(1/\eta)}{\eta}\lambda^\text{large}(\eta) = 0.
\end{align*}
Therefore, for sufficiently small $\eta$, we will have (note that $\epsilon,\delta$ are fixed constants in this proof, so $N(\delta),u(\delta)$ are also fixed)
\begin{align}
    \frac{ N(\delta)u(\delta)\frac{H(1/\eta)/\eta}{ \lambda^\text{large}(\eta) } + \frac{1}{ H(1/\eta) } + 2c_2\gamma\frac{\log(1/\eta)}{\eta}  }{  \floor{ t/\lambda^\text{large}(\eta) } } \leq \epsilon. \label{proof eliminate sharp minima ineq 2}
\end{align}
Second, recall that we fixed some $t \in (0, \frac{\delta}{4q^*})$ where $q^* = \max_j \mu_j(E_j)$. Also, choose some $C > 0$ small enough so that
\begin{align*}
    C < \delta/2,\ 2(1+C)^2 < 4.
\end{align*}
From Proposition \ref{proposition first exit time Gradient Clipping} and the fact that there are only finitely many attraction fields, there exists some $\bar{\eta}_0 > 0$ such that for any $\eta \in (0,\bar{\eta}_0)$ and any $\Delta > 0$ sufficiently small,
\begin{align*}
    & \sup_{i: m_i \in M^\text{large}}\sup_{ x \in [m_i-2\Delta,m_i+2\Delta] } \mathbb{P}_x\Big( \sigma_i(\eta) \leq \frac{t}{ \lambda^\text{large}(\eta) } \Big)
    \\
    \leq & \sup_{i: m_i \in M^\text{large}}\sup_{ x \in [m_i-2\Delta,m_i+2\Delta] } \mathbb{P}_x\Big( \mu_i(E_i)\lambda^\text{large}(\eta)\sigma_i(\eta) \leq q^* t \Big) \\
    \leq & C + 2(1 + C)^2q^* t \leq 2\delta.
\end{align*}
Combine this bound with Markov property (applied at $\tau_K$), and we obtain that
\begin{align*}
    \sup_{x \in [-L,L]}\mathbb{P}_x\Big( \exists n \in \big[ \floor{t/\lambda^\text{large}(\eta)} \big]\ s.t.\ X^\eta_{ n + \tau_K } \notin \bigcup_{ i: m_i \in M^\text{large} }\Omega_i \Big) \leq 2\delta
\end{align*}
for all $\eta$ sufficiently small. Together with \cref{proof eliminate sharp minima ineq 1}\cref{proof eliminate sharp minima ineq 2}, we have shown that
\begin{align*}
    \sup_{x \in [-L,L]} \mathbb{P}_x\Big( V^\text{small}(\eta,\epsilon,t) > \epsilon \Big) \leq 5\delta
\end{align*}
holds for all $\eta$ sufficiently small.
\end{proof}

\subsection{Proof of Lemma \ref{lemma key converge to markov chain and always at wide basin}, \ref{lemma key converge to markov chain and always at wide basin transient case}}

We shall return to the discussion about the dynamics of SGD iterates on a communication class $G$. Recall that $$G^\text{large} = \{m^\text{large}_{1},\cdots,m^\text{large}_{i_G}\},\ G^\text{small} = \{m^\text{small}_{1},\cdots,m^\text{small}_{i_G^\prime}\}.$$
If $X^\eta_n$ is initialized at some sharp minimum on $G$, then we are interested in the behavior of $X^\eta_n$ at the first visit to some large attraction fields on $G$. Define
\begin{align}
    T_G(\eta,\Delta) \delequal{}\min\{n \geq 0: X^\eta_n \in \bigcup_{ i:\ m_i \in G^\text{large} }B(m_i,2\Delta) \text{ or }X^\eta_n \notin \cup_{i:\ m_i \in G}\Omega_i \}. \label{def T G}
\end{align}
Not only is this definition of $T_G$ analogous to the one for $T^{DTMC}_G$ in \cref{def T G DTMC}, but, as illustrated in the next lemma, $T_G$ also behaves similarly as $T_G$ on a communication class $G$ in the following sense: the probabilities $p_{i,j}$ defined in \cref{def p i j} govern the dynamics regarding which large attraction field on $G$ is the first one to be visited. Besides, $T_G$ is usually rather small, meaning that the SGD iterates would efficiently arrive at a large attraction field on $G$ or simply escape from $G$.

\begin{lemma} \label{lemma sharp basin process absorbing}
Given any $\theta \in (0,(\alpha - 1)/2 )$, $\epsilon \in (0,1)$, $i,j \in [n_\text{min}]$ such that $m_i \in G^\text{small},\ m_j \notin G^\text{large}$, the following claims hold for all $\Delta>0$ that is sufficiently small:
\begin{align*}
    \limsup_{\eta\downarrow 0}\sup_{x \in B(m_i,2\Delta)} \mathbb{P}_x\Big( T_G(\eta,\Delta) \leq \frac{\eta^\theta}{\lambda_G(\eta)},\ X^\eta_{T_G} \in B(m_j,2\Delta) \Big) &  \leq p_{i,j} + 5\epsilon, \\
    \liminf_{\eta\downarrow 0}\inf_{x \in B(m_i,2\Delta)} \mathbb{P}_x\Big( T_G(\eta,\Delta) \leq \frac{\eta^\theta}{\lambda_G(\eta)},\ X^\eta_{T_G} \in B(m_j,2\Delta) \Big) &  \geq p_{i,j}- 5\epsilon,
    \\
     \limsup_{\eta\downarrow 0}\sup_{x \in B(m_i,2\Delta)} \mathbb{P}_x\Big( T_G(\eta,\Delta) > \frac{\eta^\theta}{\lambda_G(\eta)}\Big) & \leq 2\epsilon.
\end{align*}
\end{lemma}

\begin{proof}
For $G^\text{small} \neq \emptyset$ to hold (and the discussion to be meaningful), we must have $l^*_G \geq 2$. Throughout the proof, we assume this is the case. Besides, we require that $\Delta \in (0,\bar{\epsilon}/3)$ so we have
\begin{align*}
    B(m_i,3\Delta) \cap \Omega_i^c  = \emptyset \ \forall i \in [n_\text{min}]
\end{align*}
and the $3\Delta$-neighborhood of each local minimum will not intersect with each other. In this proof we will only consider $\Delta$ in this range.

From Lemma \ref{lemma geom bound on transition count} (if $G$ is absorbing) or Lemma \ref{lemma geom bound on transition count transient case} (if $G$ is transient), we know the existence of some integer $N(\epsilon)$ such that for (see the definition of $I_k$ in \cref{def MC GP T 0}-\cref{def MC GP I k})
\begin{align*}
    N_G(\eta,\Delta) \delequal{} \min\{k \geq 0:\ m_{I_k(\eta,\Delta)} \in G^\text{large}\ \text{or }m_{I_k(\eta,\Delta)} \notin G \},
\end{align*}
we have
\begin{align*}
    \sup_{ x \in B(m_i,2\Delta) }\mathbb{P}_x\Big( N_G(\eta,\Delta) > N(\epsilon)\Big) < \epsilon
\end{align*}
for all $\eta$ sufficiently small. Fix such $N(\epsilon)$. Next, from Proposition \ref{prop first exit and return time i}, we can find $u(\epsilon) \in (0,\infty)$ and $\bar{\Delta} \in (0,\bar{\epsilon}/3)$ such that for all $\Delta \in (0,\bar{\Delta})$, we have
\begin{align*}
    \sup_{x \in B(m_i,2\Delta)}\mathbb{P}_x\Big(T_{k}(\eta,\Delta) - T_{k-1}(\eta,\Delta) > u(\epsilon)/\Lambda\big(I_{k-1}(\eta,\Delta),\eta\big)  \Big) \leq \epsilon/N(\epsilon)\ \ \forall k \in [N(\epsilon)]
\end{align*}
for all $\eta$ sufficiently small. Fix such $u(\epsilon)$ and $\bar{\Delta}$. Now note that on the event
\begin{align*}
    A & \delequal{}\Big\{ N_G \leq N(\epsilon)\Big\} \cap \Big\{T_{k}(\eta,\Delta) - T_{k-1}(\eta,\Delta) \leq u(\epsilon)/\Lambda\big(I_{k-1}(\eta,\Delta),\eta\big)\ \forall k \in [N(\epsilon)]\Big\},
\end{align*}
due to the choice of $\theta \in (0,(\alpha - 1)/2)$ and $H \in \RV_{-\alpha}$, we have (when $\eta \in (0,1)$)
\begin{align*}
     & T_{k}(\eta,\Delta) - T_{k-1}(\eta,\Delta) \leq \frac{\eta^{2\theta}}{ \lambda_G(\eta) }\ \forall k < N_G(\eta,\Delta)
     \\
     \Rightarrow & T_G(\eta,\Delta) = T_{N_G(\eta,\Delta)}(\eta,\Delta) \leq N(\epsilon)u(\epsilon)\frac{\eta^{2\theta}}{ \lambda_G(\eta) }.
\end{align*}
For any $\eta$ sufficiently small, we will have $N(\epsilon)u(\epsilon)\frac{\eta^{2\theta}}{ \lambda_G(\eta) } < \frac{\eta^{\theta}}{ \lambda_G(\eta) }$. In summary, we have established that for all $\Delta \in (0,\bar{\Delta)}$,
\begin{align}
    \limsup_{\eta \downarrow 0}\sup_{x \in B(m_i,2\Delta)}\mathbb{P}_x\Big(T_G > \frac{\eta^\theta}{ \lambda_G(\eta) }\Big) \leq \limsup_{\eta \downarrow 0}\sup_{x \in B(m_i,2\Delta)}\mathbb{P}_x(A^c) < 2\epsilon. \label{proof sharp minima process ineq 1}
\end{align}
 
Next, let
\begin{align*}
    \textbf{S}(\epsilon) \delequal{}\Big\{ \big(m^\prime_1,\cdots,m^\prime_{N(\epsilon)}\big) \in \{ m_1,\cdots,m_{n_\text{min}} \}^{ N(\epsilon) }:\ \exists k \in [N(\epsilon)]\ \text{s.t. }m^\prime_k = m_j \Big\}.
\end{align*}
We can see that $\textbf{S}(\epsilon)$ contains all the possible transition path for $Y^{DTMC}$ where the state $m_j$ is visited within the first $N(\epsilon)$ steps. Obviously, $|\textbf{S}(\epsilon)| < \infty$.
Let $\epsilon_1 = \epsilon/|\textbf{S}(\epsilon)|$. If we are able to show the existence of some $\bar{\Delta}_1 > 0$ such that for all $\Delta \in (0,\bar{\Delta}_1)$, the following claim holds for any $(m^\prime_k)_{k = 1}^{N(\epsilon)} \in \textbf{S}(\epsilon)$:
\begin{align}
    \limsup_{\eta \downarrow 0}\sup_{ x \in B(m_i,2\Delta) }\Big| \mathbb{P}_x\Big( m_{I_k} = m^\prime_k\ \forall k \in [N(\epsilon)]  \Big) - \mathbb{P}\Big( Y^{DTMC}_k(m_i) = m^\prime_k\ \forall k \in [N(\epsilon)] \Big) \Big| < \epsilon_1, \label{proof sharp basin process goal 1}
\end{align}
then we must have (for all $\Delta \in (0,\bar{\Delta} \wedge \bar{\Delta}_1)$)
\begin{align*}
    & \limsup_{\eta \downarrow 0}\sup_{x \in B(m_i,2\Delta)}\Big| \mathbb{P}_x\Big( X^\eta_{T_G} \in B(m_j,2\Delta) \Big) - p_{i,j} \Big|
    \\
     = & \limsup_{\eta \downarrow 0}\sup_{x \in B(m_i,2\Delta)}\Big| \mathbb{P}_x\Big( X^\eta_{T_G} \in B(m_j,2\Delta),\ T_G \leq N(\epsilon) \Big) + \mathbb{P}_x\Big( X^\eta_{T_G} \in B(m_j,2\Delta),\ T_G > N(\epsilon) \Big)
     \\
     &\ \ \ \ \ \ - \mathbb{P}\Big( Y^{DTMC}_{T^{DTMC}_G}(m_i) = m_j,\ T^{DTMC}_G \leq N(\epsilon) \Big) - \mathbb{P}\Big( Y^{DTMC}_{T^{DTMC}_G}(m_i) = m_j,\ T^{DTMC}_G > N(\epsilon) \Big) \Big|
    \\
    \leq & \limsup_{\eta \downarrow 0}\sup_{x \in B(m_i,2\Delta)}\Big| \mathbb{P}_x\Big( X^\eta_{T_G} \in B(m_j,2\Delta),\ T_G \leq N(\epsilon) \Big)
    \\
    & - \mathbb{P}\Big( Y^{DTMC}_{T^{DTMC}_G}(m_i) = m_j,\ T^{DTMC}_G \leq N(\epsilon) \Big)  \Big|
    \\
    & + \limsup_{\eta \downarrow 0}\sup_{x \in B(m_i,2\Delta)}\mathbb{P}_x( T_G > N(\epsilon) ) + \mathbb{P}( {T^{DTMC}_G}(m_i) > N(\epsilon) )
    \\
    \leq & |\textbf{S}(\epsilon)|\epsilon_1 + \limsup_{\eta \downarrow 0}\sup_{x \in B(m_i,2\Delta)}\mathbb{P}_x( T_G > N(\epsilon) ) + \mathbb{P}( {T^{DTMC}_G}(m_i) > N(\epsilon) )
    \\
    \leq & 3\epsilon.
\end{align*}
To show that \cref{proof sharp basin process goal 1} is true, we fix some $(m^\prime_k)_{k = 1}^{N(\epsilon)} \in \textbf{S}(\epsilon)$ and let $(\textbf{k}^\prime(k))_{k = 1}^{N(\epsilon)}$ be the sequence with $m_{ \textbf{k}^\prime(k) } = m^\prime_k$ for each $k \in [N(\epsilon)]$. From the definition of $Y^{DTMC}$ we have (let $\textbf{k}^\prime(0) = i$)
\begin{align*}
    & \mathbb{P}\Big( Y^{DTMC}_k(m_i) = m^\prime_k\ \forall k \in [N(\epsilon)] \Big) = \prod_{k = 0}^{N(\epsilon)-1}\frac{ \mu_{\textbf{k}^\prime(k)}(E_{\textbf{k}^\prime(k),\textbf{k}^\prime(k+1)})   }{ \mu_{\textbf{k}^\prime(k)}(E_{\textbf{k}^\prime(k)}) }.
\end{align*}
On the other hand, using Proposition \ref{prop first exit and return time i}, we know that for any arbitrarily chosen $\epsilon^\prime \in (0,1)$, we have
\begin{align*}
   &  \limsup_{\eta \downarrow 0}\sup_{ x \in B(m_i,2\Delta) }\mathbb{P}_x\Big( m_{I_k(\eta,\Delta)} = m^\prime_k\ \forall k \in [N(\epsilon)]  \Big) \leq \prod_{k = 0}^{N(\epsilon)-1}\frac{ \mu_{\textbf{k}^\prime(k)}(E_{\textbf{k}^\prime(k),\textbf{k}^\prime(k+1)})   }{ \mu_{\textbf{k}^\prime(k)}(E_{\textbf{k}^\prime(k)}) }\cdot(1 + \epsilon^\prime),
   \\
   &  \liminf_{\eta \downarrow 0}\inf_{ x \in B(m_i,2\Delta) }\mathbb{P}_x\Big( m_{I_k(\eta,\Delta)} = m^\prime_k\ \forall k \in [N(\epsilon)]  \Big) \geq \prod_{k = 0}^{N(\epsilon)-1}\frac{ \mu_{\textbf{k}^\prime(k)}(E_{\textbf{k}^\prime(k),\textbf{k}^\prime(k+1)})   }{ \mu_{\textbf{k}^\prime(k)}(E_{\textbf{k}^\prime(k)}) }\cdot(1 - \epsilon^\prime),
\end{align*}
for all $\Delta > 0$ sufficiently small. The arbitrariness of $\epsilon^\prime > 0$, together with $|\textbf{S}(\epsilon)| < \infty$, allows us to see the existence of some $\bar{\Delta}_1 > 0$ such that with $\Delta \in (0,\bar{\Delta}_1)$, \cref{proof sharp basin process goal 1} holds for any $(m^\prime_k)_{k = 1}^{N(\epsilon)} \in \textbf{S}(\epsilon)$. To conclude the proof, observe that
\begin{align*}
     & \limsup_{\eta \downarrow 0}\sup_{x \in B(m_i,2\Delta)}\Big| \mathbb{P}_x\Big( X^\eta_{T_G} \in B(m_j,2\Delta) \Big) - \mathbb{P}_x\Big( T_G(\eta,\Delta) \leq \frac{\eta^\theta}{\lambda_G(\eta)},\ X^\eta_{T_G} \in B(m_j,2\Delta) \Big) \Big|
     \\
     \leq &  \limsup_{\eta \downarrow 0}\sup_{x \in B(m_i,2\Delta)} \mathbb{P}_x\Big(T_G > \frac{\eta^\theta}{ \lambda_G(\eta) }\Big)  < \epsilon
\end{align*}
due to \cref{proof sharp minima process ineq 1}.
\end{proof}

Recall that continuous-time process $X^{*,\eta}$ is the \textit{scaled} version of $X^\eta$ defined in \cref{def X hat star marker scaled 1}, and the mapping  $\textbf{T}^*(n,\eta) \delequal{} n\lambda_G(\eta)$ returns the timestamp $t$ for $X^{*,\eta}_t$ corresponding to the unscaled step $n$ on the time horizon of $X^\eta_n$. 
As an \textit{inverse} mapping of $\textbf{T}^*$, we define the mapping $\textbf{N}^*(t,\eta) = \floor{t/\lambda_G(\eta)}$ that maps the scaled timestamp $t$ back to the step number $n$ for the unscaled process $X^\eta$.

In the next lemma, we show that, provided that $X^{*,\eta}$ stays on a communication class $G$ before some time $t$, the scaled process $X^{*,\eta}_t$ is almost always in the largest attraction fields of a communication class $G$. 

\begin{lemma} \label{lemma dynamic scaled X stays mostly at local minima}
Let $G$ be a communication class on the graph $\mathcal{G}$. Given any $ \epsilon_1 > 0 $, $t > 0$ and any $x \in \Omega_i$ with $m_i \in G$, the following claim holds for all $\Delta > 0$ small enough:
\begin{align*}
    \limsup_{\eta \downarrow 0}\mathbb{P}_x\Big( \Big\{X^{*,\eta}_t \notin \bigcup_{j: m_j \in G^\text{large}} B(m_j,3\Delta)\Big\} \cap \Big\{ X^{*,\eta}_s \in \bigcup_{k:\ m_k \in G}\Omega_k\ \forall s \in [0,t] \Big\} \Big) \leq 2\epsilon_1.
\end{align*}
\end{lemma}

\begin{proof}

Let $\Delta \in (0,\bar{\epsilon}/3)$ for the constant $\bar{\epsilon}$ in \cref{assumption multiple jump epsilon 0 constant 1}\cref{assumption multiple jump epsilon 0 constant 2}, so we are certain that each $B(m_i,2\Delta)$ lies entirely in $\Omega_i$ and would not intersect with each other since
\begin{align*}
    B(m_i,3\Delta) \cap \Omega_i^c = \emptyset \ \ \forall i \in [n_\text{min}].
\end{align*}

Besides, with $\epsilon = \Delta / 3$,  we know the existence of some $\delta > 0$ such that claims in Lemma \ref{lemma stuck at local minimum before large jump} would hold of the chosen $\epsilon,\delta$. Fix such $\delta$ for the entirety of this proof. Lastly, fix some
\begin{align*}
    \widetilde{\gamma} \in (0, (1-\frac{1}{\alpha})\wedge(\frac{1}{2}) ),\ \beta \in \big(1,(2 - 2\widetilde{\gamma})\wedge ( \alpha - \alpha\widetilde{\gamma} ) \big),\ \gamma \in (0,\frac{\widetilde{\gamma}}{16 M c_2}\wedge\frac{\widetilde{\gamma}}{4} ).
\end{align*}

The blueprint of this proof is as follows. We will define a sequence of stopping times $(N_j)_{j = 1}^6$ such that the corresponding scaled timestamps $\textbf{T}^*_j = \textbf{T}^*(N_j,\eta)$ gradually approach $t$. By analyzing the behavior of $X^{*,\eta}$ on a time interval $[t-\Delta_t,t]$ that is very close to $t$ (in particular, on the aforementioned stopping times $\textbf{T}^*_j$), we are able to establish the properties of a series of events $A_1 \supseteq A_2 \supseteq A_3$. Moreover, we will show that $A_3 \subseteq \{ X^{*,\eta,\Delta}_t \in \bigcup_{i:\ m_i \in G^\text{large}} B(m_i,3\Delta) \}$, so the properties about events $A,A_2,A_3$ can be used to bound the probability of the target event.

Arbitrarily choose some $\Delta_t \in (0,t)$. To proceed, let $N_0 \delequal{} \textbf{N}^*(t - \Delta_t,\eta)$ be the stopping time corresponding to timestamp $t - \Delta_t$ for the scaled process. Using Lemma \ref{lemma exit smaller neighborhood of s}, we know that for stopping time $ N_1 \delequal{}\min\{n \geq N_0:\ X^\eta_n \notin \cup_{j}B(s_j,2\eta^\gamma) \}$, we have
\begin{align*}
    \liminf_{\eta \downarrow 0}\inf_{x \in [-L,L]}\mathbb{P}_x( N_1 - N_0 < \frac{\Delta_t/4}{H(1/\eta)} ) = 1.
\end{align*}
Next, let $N_2 \delequal{}\min\{ n \geq N_1:\ X^\eta_n \in \cup_j B(m_j,2\Delta) \}$.
From Lemma \ref{lemma return to local minimum quickly rescaled version} and $H \in \RV_{-\alpha}$ (so that $\log(1/\eta)/\eta = o(H(1/\eta))$, we have
\begin{align*}
    \liminf_{\eta \downarrow 0}\inf_{x \in [-L,L]}\mathbb{P}_x(N_2 - N_1 < \frac{\Delta_t/4}{H(1/\eta)}) = 1.
\end{align*}
Collecting results above, we have
\begin{align}
     \liminf_{\eta \downarrow 0}\inf_{x \in [-L,L]}\mathbb{P}_x(N_2 - N_0 < \frac{\Delta_t/2}{H(1/\eta)}) = 1. \label{proof most likely at large basin ineq 1}
\end{align}

Now note the following fact on the event $\{ N_2 - N_0 < \frac{\Delta_t/2}{H(1/\eta)} \}$. The definition of the mapping $\textbf{T}^*$ implies that, for any pair of positive integers $n_1 \leq n_2$, we have $\textbf{T}^*(n_2,\eta) - \textbf{T}^*(n_1,\eta) = (n_2 - n_1)\lambda_G(\eta) \leq (n_2 - n_1)\cdot H(1/\eta)$. Therefore, on $\{ N_2 - N_0 \leq \frac{\Delta_t/2}{H(1/\eta)} \}$ we have
\begin{align*}
    \textbf{T}^*(N_2,\eta) - \textbf{T}^*(N_0,\eta) < \Delta_t/2 \Rightarrow \textbf{T}^*(N_2,\eta) < t - \frac{\Delta_t}{2}.
\end{align*}
Besides, let $\textbf{T}^*_2 = \textbf{T}^*(N_2,\eta)$. Now we can see that for event
\begin{align*}
    A_0 \delequal{} \Big\{ X^{*,\eta}_s \in \bigcup_{k:\ m_k \in G}\Omega_k\ \forall s \in [0,t]  \Big\} \cap \Big\{ N_2 - N_0 < \frac{\Delta_t / 2}{H(1/\eta)} \Big\},
\end{align*}
we have
\begin{align*}
    A_0 \subseteq A_1 \delequal{} \Big\{ \textbf{T}^*_2 < t - \frac{\Delta_t}{2} \Big\} \cap \Big\{ X^{*,\eta}_s \in \bigcup_{k:\ m_k \in G}\Omega_k\ \forall s \in [0,\textbf{T}^*_2]  \Big\}.
\end{align*}
Meanwhile, from \cref{proof most likely at large basin ineq 1} we obtain that
\begin{align}
    \limsup_{\eta \downarrow 0}\sup_{x \in [-L,L]}\mathbb{P}_x\Big(A_1^c \cap \{ X^{*,\eta}_s \in \bigcup_{k:\ m_k \in G}\Omega_k\ \forall s\in[0,t] \} \Big) = 0. \label{proof most likely at large basin ineq 2}
\end{align}

Moving on, we consider the following stopping times
\begin{align*}
    N_3 & \delequal{} \min\{ n \geq N_2:\ X^\eta_n \in \bigcup_{j:\ m_j \in G^\text{large}}B(m_j,2\Delta)\text{ or }X^\eta_n \notin \bigcup_{j:\ m_j \in G}\Omega_j \},
    \\
    \textbf{T}^*_3 & \delequal{} \textbf{T}^*(N_3,\eta).
\end{align*}
Using Lemma \ref{lemma sharp basin process absorbing}, we have
\begin{align}
    \limsup_{\eta \downarrow 0}\mathbb{P}_x\Big(N_3 - N_2 > \frac{\Delta_t/4}{ \lambda_G(\eta) }  \ \Big|\ A_1  \Big) \leq \epsilon_1. \label{proof most likely at large basin ineq 3}
\end{align}
Meanwhile, on event $A_1 \cap \Big\{ N_3 - N_2 \leq \frac{\Delta_t/4}{ \lambda_G(\eta) }  \Big\} \cap \Big\{ X^{*,\eta}_s \in \bigcup_{k:\ m_k \in G}\Omega_k\ \forall s \in [0,t]  \Big\},$
we have $\textbf{T}^*_3 - \textbf{T}^*_2 \leq \Delta_t/4$, hence $\textbf{T}^*_3 \in [t - \Delta_t, t - \Delta_t/4]$. In summary,
\begin{align*}
   & A_1 \cap \Big\{ N_3 - N_2 \leq \frac{\Delta_t/4}{ \lambda_G(\eta) }  \Big\} \cap \Big\{ X^{*,\eta}_s \in \bigcup_{k:\ m_k \in G}\Omega_k\ \forall s \in [0,t]  \Big\}
   \\
   & \subseteq \Big\{ \textbf{T}^*_3 \in [t - \Delta_t, t - \Delta_t/4] \Big\} \cap \Big\{ X^{*,\eta}_s \in \bigcup_{k:\ m_k \in G}\Omega_k\ \forall s \in [0,\textbf{T}^*_3]  \Big\}
\end{align*}
Moreover, on event $\Big\{ X^{*,\eta}_s \in \bigcup_{k:\ m_k \in G}\Omega_k\ \forall s \in [0,t]  \Big\}$, if we let $J_3$ be label of the local minimum visited at $\textbf{T}^*_3$ such that $J_3 = j \iff X^{*,\eta}_{T^*_3} \in B(m_j,2\Delta)$, then we must have $m_{J_3} \in G^\text{large}$. Meanwhile, consider the following stopping times
\begin{align*}
    \textbf{T}^{\sigma} \delequal{} \min\{ s >\textbf{T}^*_3:\ X^{*,\eta}_s \notin \Omega_{J_3}  \}.
\end{align*}
From Proposition \ref{prop first exit and return time i}, we know that
\begin{align}
    & \limsup_{\eta \downarrow 0}\mathbb{P}_x\Big( \textbf{T}^{\sigma} - \textbf{T}^*_3 \leq \Delta_t \ \Big|\ \Big\{ \textbf{T}^*_3 \in [t - \Delta_t, t - \Delta_t/4] \Big\} \cap \Big\{ X^{*,\eta}_s \in \bigcup_{k:\ m_k \in G}\Omega_k\ \forall s \in [0,\textbf{T}^*_3]  \Big\}  \Big) \nonumber
    \\
    \leq & \epsilon_1 + 1 - \exp\big( -(1+\epsilon_1) q^*\Delta_t \big) \label{proof most likely at large basin ineq 4}
\end{align}
where $q^* = \max_j \mu_j(E_j)$. Now we define the event
\begin{align*}
    A_2 \delequal{} A_1 \cap \Big\{ N_3 - N_2 \leq \frac{\Delta_t/4}{ \lambda_G(\eta) }  \Big\} \cap \Big\{ \textbf{T}^\sigma - \textbf{T}^*_3 > \Delta_t \Big\} \cap \Big\{ X^{*,\eta}_s \in \bigcup_{k:\ m_k \in G}\Omega_k\ \forall s \in [0,\textbf{T}^*_3]  \Big\}.
\end{align*}
Using \cref{proof most likely at large basin ineq 2}-\cref{proof most likely at large basin ineq 4}, we get
\begin{align}
    \limsup_{\eta \downarrow 0}\sup_{x \in [-L,L]}\mathbb{P}_x\Big(A_2^c \cap \{ X^{*,\eta}_s \in \bigcup_{k:\ m_k \in G}\Omega_k\ \forall s\in[0,t] \} \Big) \leq 2\epsilon_1 + 1 - \exp\big( -(1+\epsilon_1) q^*\Delta_t \big). \label{proof most likely at large basin ineq 5}
\end{align}
Furthermore, on event $A_2$, due to $\textbf{T}^*_3 \in [t - \Delta_t, t - \Delta_t/4]$ as established above, we must have
\begin{align*}
    X^{*,\eta}_s \in \Omega_{J_3}\ \forall s \in [\textbf{T}^*_3,t].
\end{align*}
Now let us focus on a timestamp $\textbf{T}^*_4 = t - \frac{\Delta_t/8}{ H(1/\eta) }\lambda_G(\eta)$ and $N_4 = \textbf{N}^*(\textbf{T}^*_4,\eta)$ . Obviously, $\textbf{T}^*_4 > \textbf{T}^*_3$ on event $A_2$. Next, define
\begin{align*}
    N_5 & \delequal{}\min\{ n \geq N_4:\ X^\eta_n \in \bigcup_{j}B(m_j,2\Delta) \}
    \\
    \textbf{T}^*_5 & \delequal{} \textbf{T}^*(N_5,\eta).
\end{align*}
Using Lemma \ref{lemma return to local minimum quickly rescaled version} and \ref{lemma exit smaller neighborhood of s} again as we did above when obtaining \cref{proof most likely at large basin ineq 1}, we can show that
\begin{align}
    \limsup_{\eta \downarrow 0}\mathbb{P}_x\Big( N_5 - N_4 > \frac{\Delta_t/16}{H(1/\eta)} \Big) = 0. \label{proof most likely at large basin ineq 6}
\end{align}
On the other hand, on event $A_2 \cap \Big\{ N_5 - N_4 \leq \frac{\Delta_t/16}{H(1/\eta)} \Big\}$ we must have
\begin{itemize}
    \item $ \textbf{T}^*_5 - \textbf{T}^*_4 \leq \frac{\Delta_t/16}{H(1/\eta)}\lambda_G(\eta)$, so  $ \textbf{T}^*_5\in [ t - \frac{\Delta_t/8}{H(1/\eta)}\lambda_G(\eta),t - \frac{\Delta_t/16}{H(1/\eta)}\lambda_G(\eta)]$;
    
    \item $X^{*,\eta}_{\textbf{T}^*_5} \in \Omega_{J_3}$, due to $\textbf{T}^\sigma - \textbf{T}^*_3 > \Delta_t$.
\end{itemize}
This implies that for event
$$\widetilde{A} \delequal{} \Big\{\textbf{T}^*_5\in [ t - \frac{\Delta_t/8}{H(1/\eta)}\lambda_G(\eta),t - \frac{\Delta_t/16}{H(1/\eta)}\lambda_G(\eta)],\ X^{*,\eta}_{\textbf{T}^*_5} \in \Omega_{J_3}  \Big\} \cap \{ X^{*,\eta}_s \in \bigcup_{k:\ m_k \in G}\Omega_j\ \forall s \in [0,\textbf{T}^*_5] \},$$
we have $A_2 \cap \Big\{ N_5 - N_4 \leq \frac{\Delta_t/16}{H(1/\eta)}\Big\} \subseteq \widetilde{A}$. Lastly, observe that
\begin{itemize}
    \item From Lemma \ref{lemmaGeomFront}, we know that for $ N_6(\delta) \delequal{} \min\{n > N_5: \eta|Z_n| > \delta\}$
    we have 
    \begin{align*}
        \limsup_{\eta \downarrow 0}\mathbb{P}\big( N_6(\delta) - N_5 \leq \Delta_t/H(1/\eta) \big) \leq \Delta_t/\delta^\alpha;
    \end{align*}
    \item As stated at the beginning of the proof, our choice of $\delta$ allows us to apply Lemma \ref{lemma stuck at local minimum before large jump} and show that
    \begin{align*}
        \limsup_{\eta \downarrow 0}\sup_{x \in [-L,L]}\mathbb{P}_x\big( \exists n = N_5,\cdots, N_6 - 1\ \text{s.t. }X^\eta_n \notin B(m_{J_3},3\Delta)\ |\ \widetilde{A} \big) = 0;
    \end{align*}
    \item Combining the two bullet points above, we get
    \begin{align}
        \limsup_{\eta \downarrow 0}\mathbb{P}_x\Big( \exists s \in [\textbf{T}^*_5,t] \text{ such that }  X^{*,\eta}_s \notin B(m_{J_3},3\Delta)\ \Big|\ \widetilde{A} \Big) \leq \Delta_t/\delta^\alpha. \label{proof most likely at large basin ineq 7}
    \end{align}
    On the other hand,
    \begin{align*}
        \widetilde{A} \cap \{X^{*,\eta}_s \in B(m_{J_3},3\Delta)\ \forall s\in[\textbf{T}^*_5,t] \} \subseteq \{ X^{*,\eta}_t \in \bigcup_{k:\ m_k \in G^\text{large}}B(m_k,3\Delta) \}.
    \end{align*}
\end{itemize}

In summary, for event
\begin{align*}
    A_3 \delequal{} A_2 \cap \Big\{ N_5 - N_4 \leq \frac{\Delta_t/16}{H(1/\eta)}\Big\} \cap \Big\{X^{*,\eta}_s \in B(m_{J_3},3\Delta)\ \forall s\in[\textbf{T}^*_5,t] \Big\}, 
\end{align*}
we have $A_3 \subseteq \{ X^{*,\eta}_t \in \bigcup_{k:\ m_k \in G^\text{large}}B(m_k,3\Delta) \}$. Besides, due to \cref{proof most likely at large basin ineq 5}\cref{proof most likely at large basin ineq 6}\cref{proof most likely at large basin ineq 7}, we get
\begin{align*}
    & \limsup_{\eta \downarrow 0}\sup_{x \in [-L,L]}\mathbb{P}_x\Big(A_3^c \cap \{ X^{*,\eta}_s \in \bigcup_{k:\ m_k \in G}\Omega_k\ \forall s\in[0,t] \}\Big)  
    \\
    \leq & 2\epsilon_1 + 1 - \exp\big( -(1+\epsilon_1) q^*\Delta_t \big) + \frac{\Delta_t}{\delta^\alpha}.
\end{align*}
Remember that $\delta,\epsilon_1,q^*$ are fixed constants while $\Delta_t$ can be made arbitrarily small, so by driving $\Delta_t$ to 0 we can conclude the proof.
\end{proof}

Recall the definition of jump processes in Definition \ref{def jump process}. Central to the proof of Lemma \ref{lemma key converge to markov chain and always at wide basin}, the next result provides a set of sufficient conditions for the convergence of a sequence of such jump processes in the sense of finite dimensional distributions.

\begin{lemma} \label{lemma weak convergence of jump process}
For a sequence of processes $(Y^n)_{n \geq 1}$ that, for each $n \geq 1$, $Y^n$ is a $\Big( (U^n_j)_{j \geq 0},(V^n_j)_{j \geq 0} \Big)$ jump process, and a $\Big( (U_j)_{j \geq 0},(V_j)_{j \geq 0} \Big)$ jump process $Y$, if
\begin{itemize}
    \item $U_0 \equiv 0$; 
    \item $(U^n_0,V^n_0,U^n_1,V^n_1,U^n_2,V^n_2,\cdots)$ converges in distribution to $(0,V_0,U_1,V_1,U_2,V_2,\cdots)$ as $n\rightarrow \infty$;
    \item For any $x > 0$ and any $n \geq 1$,
    \begin{align*}
        \mathbb{P}(U_1 + \cdots + U_n = x) = 0;
    \end{align*}
    \item For any $x > 0$,
    \begin{align*}
        \lim_{n \rightarrow \infty}\mathbb{P}(U_1 + U_2 + \cdots U_n > x) = 1,
    \end{align*}
\end{itemize}
then the finite dimensional distribution of $Y^n$ converges to that of $Y$ in the following sense: for any $k \in \mathbb{N}$ and any $0 < t_1 < t_2 < \cdots < t_k < \infty$, the random element $(Y^n_{t_1},\cdots,Y^n_{t_k})$ converges in distribution to $(Y_{t_1},\cdots,Y_{t_k})$ as $n \rightarrow \infty$.
\end{lemma}

\begin{proof}
Fix some $k \in \mathbb{N}$ and $0 < t_1 < t_2 < \cdots < t_k < \infty$. For notational simplicity, let $t = t_k$. Let $(\mathbb{D},\textbf{d})$ be the metric space where $\mathbb{D} = \mathbb{D}_{[0,t]}$, the space of all càdlàg functions in $\mathbb{R}$ on the time interval $[0,t]$, and $\textbf{d}$ is the Skorokhod metric defined as
\begin{align*}
    \textbf{d}(\zeta_1,\zeta_2) \delequal{} \inf_{\lambda \in \Lambda} \norm{ \zeta_1 - \zeta_2 \circ \lambda } \vee \norm{ \lambda - I }
\end{align*}
where $\Lambda$ is the set of all nondecreasing homeomorphism from $[0,t]$ onto itself, and $I(s) = s$ is the identity mapping. Also, we arbitrarily choose some $\epsilon \in (0,1)$ and some open set $A \subseteq \mathbb{R}^k$.

From the assumption, we can find integer $J(\epsilon)$ such that $\mathbb{P}\Big( \sum_{j = 1}^{J(\epsilon)} U_j \leq t \Big) < \epsilon$, as well as an integer $N(\epsilon)$ such that, for all $n \geq N(\epsilon)$, we have $\mathbb{P}\Big( \sum_{j = 1}^{J(\epsilon)} U^n_j \leq t \Big) + \mathbb{P}\Big( U^n_0 \geq t_1 \Big) < \epsilon$. We fix such $J(\epsilon),N(\epsilon)$ (we may abuse the notation slightly and simply write $J,N$ when there is no ambiguity).

Using Skorokhod's representation theorem, we can construct a probability space $(\boldsymbol{\Omega},\mathcal{F},\mathbb{Q})$ that supports random variables $(\widetilde{U}^n_0,\widetilde{V}^n_0,\cdots,\widetilde{U}^n_J,\widetilde{V}^n_J)_{n \geq 1}$ and $(\widetilde{U}_0,\widetilde{V}_0,\cdots,\widetilde{U}_J,\widetilde{V}_J)$ and satisfies the following conditions:
\begin{itemize}
    \item $\mathcal{L}(U^n_0,V^n_0,\cdots,U^n_J,V^n_J) = \mathcal{L}(\widetilde{U}^n_0,\widetilde{V}^n_0,\cdots,\widetilde{U}^n_J,\widetilde{V}^n_J)$ for all $n \geq 1$;
    \item $\mathcal{L}(U_0,V_0,\cdots,U_J,V_J) = \mathcal{L}(\widetilde{U}_0,\widetilde{V}_0,\cdots,\widetilde{U}_J,\widetilde{V}_J)$;
    \item $U^n_j \xrightarrow{a.s.} U_j$ and $V^n_j \xrightarrow{a.s.} V_j$ as $n \rightarrow \infty$ for all $j \in [J]$.
\end{itemize}
Therefore, on $(\boldsymbol{\Omega},\mathcal{F},\mathbb{Q})$ we can define the following random elements (taking values in the space of càdlàg functions):
\begin{align*}
    Y_s^{n,\downarrow J} & = \begin{cases}\phantom{-} \widetilde{V}^n_0 & \text{if }s < \widetilde{U}^n_0\\ \sum_{j = 0}^J \widetilde{V}^n_j\mathbbm{1}_{ [\widetilde{U}^n_0 + \widetilde{U}^n_1 + \cdots + \widetilde{U}^n_j,\ \widetilde{U}^n_0 + \widetilde{U}^n_1 + \cdots +\widetilde{U}^n_{j+1}) }(s) & \text{otherwise} \end{cases},
    \\
    Y^{\downarrow J}_s & = \sum_{j = 0}^J \widetilde{V}_j\mathbbm{1}_{ [ \widetilde{U}_1 + \cdots + \widetilde{U}_j,\ \widetilde{U}_1 + \cdots +\widetilde{U}_{j+1}) }(s)\ \ \forall s \geq 0.
\end{align*}
Note that (1) for the first jump time of $Y^{\downarrow J}$ we have $\widetilde{U}_0 \equiv 0$, hence $Y^{\downarrow J}_0 =\widetilde{V}_0$; (2) when defining $Y^{n,\downarrow J}$ we set its value on $[0,\widetilde{U}^n_0)$ to be $\widetilde{V}^n_0$ instead of 0.

Since $U^n_j \xrightarrow{a.s.} U_j$ and $V^n_j \xrightarrow{a.s.} V_j$ as $n \rightarrow \infty$ for all $j \in [J]$, we must have
\begin{align*}
    \lim_{n}\textbf{d}(Y_s^{n,\downarrow J} ,Y_s^{\downarrow J}) = 0
\end{align*}
almost surely, which further implies that $Y_s^{n,\downarrow J} \Rightarrow Y_s^{\downarrow J}$ as $n \rightarrow \infty$ on $(\mathbb{D},\textbf{d})$. Now from our assumption that, for the jump times $U_1+\cdots + U_j$, we have $\mathbb{P}(U_1+\cdots + U_j = x) = 0\ \forall x > 0, j \geq 1$, as well as (13.3) in \cite{billingsley2013convergence}, we then obtain
\begin{align}
    (Y^{n,\downarrow J}_{t_1},\cdots,Y^{n,\downarrow J}_{t_k})\Rightarrow (Y^{\downarrow J}_{t_1},\cdots,Y^{\downarrow J}_{t_k}) \label{proof weak convergence of jump process 1}
\end{align}
as $n \rightarrow\infty$. Recall that $A$ is the open set we arbitrarily chose at the beginning of the proof, and $\epsilon > 0$ is also chosen arbitrarily. Now we observe the following facts.
\begin{itemize}
    \item Using \cref{proof weak convergence of jump process 1}, we can see that
    \begin{align*}
        \liminf_{n}\mathbb{Q}\big( (Y^{n,\downarrow J}_{t_1},\cdots,Y^{n,\downarrow J}_{t_k}) \in A  \big) \geq \mathbb{Q}\big( Y^{\downarrow J}_{t_1},\cdots,Y^{\downarrow J}_{t_k}) \in A \big).
    \end{align*}
    \item The choice of $N(\epsilon)$ and $J(\epsilon)$ above implies that
    \begin{align*}
        & \big|\mathbb{Q}\big( (Y^{\downarrow J(\epsilon)}_{t_1},\cdots,Y^{\downarrow J(\epsilon)}_{t_k}) \in A \big) - \mathbb{P}\big( (Y_{t_1},\cdots,Y_{t_k}) \in A \big)\big| \leq \mathbb{P}\Big( \sum_{j = 1}^{J(\epsilon)} U_j \leq t \Big) < \epsilon,
        \\
        & \big|\mathbb{Q}\big( (Y^{n,\downarrow J(\epsilon)}_{t_1},\cdots,Y^{n,\downarrow J(\epsilon)}_{t_k}) \in A \big) - \mathbb{P}\big( (Y^n_{t_1},\cdots,Y^n_{t_k}) \in A \big)\big|
        \\
        \leq & \mathbb{P}\Big( \sum_{j = 1}^{J(\epsilon)} U^n_j \leq t \Big) +  \mathbb{P}\Big(  U^n_0 \geq t_1 \Big)< \epsilon\ \ \forall n \geq N(\epsilon).
    \end{align*}
\end{itemize}
Collecting the two results above, we have established that
\begin{align*}
    \liminf_n \mathbb{P}\big( (Y^n_{t_1},\cdots,Y^n_{t_k}) \in A  \big) \geq \mathbb{P}\big( (Y_{t_1},\cdots,Y_{t_k}) \in A  \big) - 2\epsilon.
\end{align*}
From Portmanteau theorem, together with arbitrariness of $\epsilon > 0$ and open set $A$,  we can now conclude that $(Y^n_{t_1},\cdots,Y^n_{t_k})$ converges in distribution to $(Y_{t_1},\cdots,Y_{t_k})$.
\end{proof}

The following lemma concerns the scaled version of the marker process $\hat{X}^{*,\eta,\Delta}$ defined in \cref{def X hat star marker scaled 1}-\cref{def X hat star marker scaled 2}. Obviously, it is a jump process that complies with Definition \ref{def jump process}. When there is no ambiguity about the sequences $(\eta_n)_{n \geq 1}$ and $(\Delta_n)_{n \geq 1}$, let $\hat{X}^{(n)}_t \delequal{} \hat{X}^{*,\eta_n,\Delta_n}_t$. From \cref{def tau G sigma G I G 1}-\cref{def tau G sigma G I G 6} and \cref{def tau star sigma star scaled}, we know that for any $n \geq 1$, $\hat{X}^{(n)}$ is a $\Big(\big( \tau^*_k(\eta_n,\Delta_n) - \tau^*_{k-1}(\eta_n,\Delta_n)  \big)_{k \geq 0},\ \big( m_{I_k(\eta_n,\Delta_n)} \big)_{k \geq 0} \Big)$-jump process (with the convention that $\tau^*_{-1} = 0$). Also, for clarity of the exposition, we let (for all $n \geq 1, k \geq 0$)
\begin{align*}
    \widetilde{S}^{(n)}_k & = \sigma^*_k(\eta_n,\Delta_n) - \tau^*_{k-1}(\eta_n,\Delta_n),
    \\
    S^{(n)}_k & = \tau^*_k(\eta_n,\Delta_n) - \tau^*_{k-1}(\eta_n,\Delta_n),
    \\
     \widetilde{W}^{(n)}_k & = m_{ \widetilde{I}^G_k(\eta_n,\Delta_n) },
    \\
    W^{(n)}_k & = m_{ I^G_k(\eta_n,\Delta_n) }.
\end{align*}
Lastly, remember that $Y$ is the continuous-time Markov chain defined in \cref{def limiting MC Y 1}-\cref{def limiting MC Y 4} and $\pi_G(\cdot)$ is the random mapping defined in \cref{def mapping pi G} that is used to initialize $Y$. Besides, $Y$ is a $\big( (S_k)_{k \geq 0}, (W_k)_{k \geq 0} \big)$ jump process under Definition \ref{def jump process}, with $S_0 = 0$ and $W_0 = \pi_G(m_i)$ (here $x \in \Omega_i$ and $X^\eta_0 = x$, so $i$ is the index of the attraction field where the SGD iterate is initialized). The following result states that, given a sequence of learning rates $(\eta_n)_{ n\geq 1}$ that tend to $0$, we are able to find a sequence of $(\Delta_n)_{n\geq 1}$ to parametrize $\hat{X}^{(n)} = \hat{X}^{*,\eta_n,\Delta_n }, {X}^{(n)} = {X}^{*,\eta_n,\Delta_n }$ so that they have several useful properties, one of which is that the jump times and locations of $\hat{X}^{(n)}$ converges in distribuiton to those of $Y(\pi_G(m_i))$.
 
\begin{lemma} \label{lemma preparation weak convergence to markov chain absorbing case}
Assume the communication class $G$ is absorbing. Given any $m_i \in G$, $x \in \Omega_i$, finitely many real numbers $(t_l)_{l = 1}^{k^\prime}$ such that $0<t_1<t_2<\cdots<t_{k^\prime}$, and a sequence of strictly positive real numbers $(\eta_n)_{n \geq 1}$ with $\lim_{n \rightarrow 0}\eta_n = 0$, there exists a sequence of strictly positive real numbers $(\Delta_n)_{n \geq 1}$ with $\lim_n \Delta_n = 0$ such that
\begin{itemize}
    \item Under $\mathbb{P}_x$ (so $X^\eta_0 = x$), as $n$ tends to $\infty$,
    \begin{align}
    (S^{(n)}_0,W^{(n)}_0,S^{(n)}_1,W^{(n)}_1,S^{(n)}_2,W^{(n)}_2,\cdots ) \Rightarrow  (S_0,W_0,S_1,W_1,S_2,W_2,\cdots ) \label{lemma mark process converge to markov chain absorbing case goal 1}
\end{align}

    \item (Recall the definition of $T_k,I_k$ in \cref{def MC GP T 0}-\cref{def MC GP I k}) Given any $\epsilon > 0$, the following claim holds for all $n$ sufficiently large:
    \begin{align}
        \sup_{k \geq 0}\mathbb{P}_x\Big(\exists j \in [T_k(\eta_n,\Delta_n),T_k(\eta_n,\Delta_n)]\ s.t.\ X^\eta_j \notin \bigcup_{l:\ m_l \in G}\Omega_l \ |\ m_{I_{k}(\eta_n,\Delta_n)} \in G \Big) < \epsilon; \label{lemma mark process converge to markov chain absorbing case goal 2}
    \end{align}
    
    \item Given any $\epsilon > 0$, the following claim holds for all $n$ sufficiently large,
    \begin{align}
        & \sup_{k \geq 0}\mathbb{P}_k\Big( m_{I_k(\eta_n,\Delta_n) + v } \notin G^\text{large}\ \forall v\in [ un_\text{min} ]\ \Big|\  m_{I_k(\eta_n,\Delta_n)  } \in G  \Big) \nonumber 
        \\
        \leq & \mathbb{P}(Geom(p^*) \geq u) + \epsilon\ \forall u = 1,2,\cdots; \label{lemma mark process converge to markov chain absorbing case goal 3}
    \end{align}
    
    \item For any $l \in [k^\prime]$,
    \begin{align}
        \lim_{n \rightarrow \infty}\mathbb{P}_x\Big( X^{*,\eta_n}_{t_l} \notin \bigcup_{j:\ m_j \in G^\text{large}}B(m_j,\Delta_n),\ X^{*,\eta_n}_{s} \in \bigcup_{j:\ m_j \in G}\Omega_j\ \forall s \in [0,t_{k^\prime}] \Big) = 0 \label{lemma mark process converge to markov chain absorbing case goal 4}
    \end{align}
\end{itemize}
where $p^* > 0$ is a constant that does not vary with our choices of $\eta_n$ or $\Delta_n$.
\end{lemma}

\begin{proof}
Let
\begin{align*}
    \nu_j & \delequal{} q_j = \mu_j(E_j)
    \\
    \nu_{j,k} & \delequal{} \mu_{j}(E_{j,k})
\end{align*}
so from the definition of $q_{j,k}$ we have $q_{j,k} =\mathbbm{1}\{j \neq k\}\nu_{j,k} + \sum_{l:\ m_l \in G^\text{small}}\nu_{j,l}p_{l,k}$.

In order to specify our choice of $(\Delta_n)_{n \geq 1}$, we consider a construction of sequences $(\bar{\boldsymbol{\Delta}}(j))_{j \geq 0},(\bar{ \boldsymbol{\eta} }(j))_{j \geq 0}$ as follows. Fix some $\theta \in (0,\alpha - 1)/2)$. Let $\bar{\boldsymbol{\Delta}}(0) = \bar{\boldsymbol{\eta}}(0) = 1$. One can see the existence of some $(\bar{\boldsymbol{\Delta}}(j))_{j \geq 1},(\bar{ \boldsymbol{\eta} }(j))_{j \geq 1}$ such that
\begin{itemize}
    \item $\bar{\boldsymbol{\Delta}}(j) \in \big(0,\bar{\boldsymbol{\Delta}}(j-1)/2\big],\ \bar{\boldsymbol{\eta}}(j) \in \big(0,\bar{\boldsymbol{\eta}}(j-1)/2\big]$ for all $j \geq 1$;
    
    \item (Due to Lemma \ref{lemma return to local minimum quickly rescaled version}) for any $j \geq 1$, $\eta \in (0,\bar{\boldsymbol{\eta}}(j)]$, (remember that $x$ and $i$ are the fixed constants prescribed in the description of the lemma)
    \begin{align*}
        \mathbb{P}_x\Big( \sigma^*_0\big(\eta, \bar{\boldsymbol{\Delta}}(j) \big) < \eta^\theta,\ \widetilde{I}^G_0\big(\eta, \bar{\boldsymbol{\Delta}}(j)\big) = i \Big) > 1 - \frac{1}{2^j}.
    \end{align*}
    For definitions of $\sigma^G_k, \tau^G_k, I^G_k, \widetilde{I}^G_k$, see \cref{def tau G sigma G I G 1}-\cref{def tau G sigma G I G 6}.
    
    \item (Due to Lemma \ref{lemma sharp basin process absorbing}) for any $j \geq 1$, $\eta \in (0,\bar{\boldsymbol{\eta}}(j)]$,
    \begin{align*}
        & \bigg|\mathbb{P}_x\Big( \tau^*_k\big(\eta, \bar{\boldsymbol{\Delta}}(j)\big) - \sigma^*_k\big(\eta, \bar{\boldsymbol{\Delta}}(j)\big) < \eta^\theta,\ I^G_k\big(\eta, \bar{\boldsymbol{\Delta}}(j)\big) = i_2    \ \Big|\ \widetilde{I}^G_k\big(\eta, \bar{\boldsymbol{\Delta}}(j)\big) = i_1 \Big) - p_{i_1,i_2}\bigg|
        \\
        & < 1/2^j
    \end{align*}
    uniformly for all $k \geq 0$ and all $m_{i_1} \in G^\text{small}, m_{i_2}\in G^\text{large}$. Also, by definition of $\sigma^*$ and $\tau^*$, we must have
    \begin{align*}
        \mathbb{P}_x\Big( \tau^*_k\big(\eta, \bar{\boldsymbol{\Delta}}(j)\big) - \sigma^*_k\big(\eta, \bar{\boldsymbol{\Delta}}(j)\big) = 0,\ I^G_k\big(\eta, \bar{\boldsymbol{\Delta}}(j)\big) = i_1    \ \Big|\ \widetilde{I}^G_k\big(\eta, \bar{\boldsymbol{\Delta}}(j)\big) = i_1 \Big) = 1
    \end{align*}
    for all $k \geq 0$ and $m_{i_1} \in G^\text{large}$.
    \item (Due to Proposition \ref{prop first exit and return time i}) for any $j \geq 1$, $\eta \in (0,\bar{\boldsymbol{\eta}}(j)]$,
    \begin{align*}
        & -\frac{1}{2^j} + \exp\big( -(1+\frac{1}{2^j})q_{i_1}u \big)\frac{ \nu_{i_1,i_2} - \frac{1}{2^j} }{q_{i_1} }
        \\
        \leq & \mathbb{P}_x\Big( \sigma^*_{k+1}\big(\eta, \bar{\boldsymbol{\Delta}}(j)\big) - \tau^*_k\big(\eta, \bar{\boldsymbol{\Delta}}(j)\big) > u,\ \widetilde{I}^G_{k+1} = i_2  \ \Big|\ I^G_k\big(\eta, \bar{\boldsymbol{\Delta}}(j)\big) = i_1  \Big)
        \\
        \leq & \frac{1}{2^j} + \exp\big( -(1-\frac{1}{2^j})q_{i_1}u \big)\frac{ \nu_{i_1,i_2} + \frac{1}{2^j} }{q_{i_1} }
    \end{align*}
    uniformly for all $k \geq 1$, all $u > 1/2^j$, and all $m_{i_1} \in G^\text{large}, m_{i_2} \in G$.
    
    \item (Due to $G$ being absorbing and, again, Proposition \ref{prop first exit and return time i}) for any $j \geq 1$, for any $j \geq 1$, $\eta \in (0,\bar{\boldsymbol{\eta}}(j)]$, (Recall the definition of $T_k,I_k$ in \cref{def MC GP T 0}-\cref{def MC GP I k})
    \begin{align}
        \mathbb{P}_x\Big(I_{k+1}\big(\eta, \bar{\boldsymbol{\Delta}}(j)\big) = i_2  \ |\ I_k\big(\eta, \bar{\boldsymbol{\Delta}}(j)\big) = i_1 \Big) < \frac{1}{2^j} \label{proof mark process converge to markov chain absorbing case goal 2}
    \end{align}
    uniformly for all $k \geq 0$, $m_{i_1} \in G, m_{i_2} \notin G$.
    
    \item (Due to Lemma \ref{lemma geom bound on transition count}) There exists some $p^* > 0$ such that for any $j \geq 1$, for any $j \geq 1$, $\eta \in (0,\bar{\boldsymbol{\eta}}(j)]$,
    \begin{align}
       &  \mathbb{P}_x\Big(m_{ v + I_k\big(\eta, \bar{\boldsymbol{\Delta}}(j)\big)  } \notin G^\text{large}\ \forall v \in [un_\text{min}] \ \Big|\ I_k\big(\eta, \bar{\boldsymbol{\Delta}}(j)\big) = i_1 \Big) \nonumber 
       \\
       \leq & \mathbb{P}(Geom(p^*) \geq u) + 1/2^j \label{proof mark process converge to markov chain absorbing case goal 3}
    \end{align}
    uniformly for all $k \geq 0, u \geq 1$ and $m_{i_1} \in G$.
    
    \item (Due to Lemma \ref{lemma dynamic scaled X stays mostly at local minima}) for any $j \geq 1$, for any $j \geq 1$, $\eta \in (0,\bar{\boldsymbol{\eta}}(j)]$,
    \begin{align}
        \mathbb{P}_x\Big( X^{*,\eta_n}_{t_k} \notin \bigcup_{j:\ m_j \in G^\text{large}}B(m_j,\bar{\boldsymbol{\Delta}}(j)),\ X^{*,\eta_n}_{s} \in \bigcup_{j:\ m_j \in G}\Omega_j\ \forall s \in [0,t_{k^\prime}] \Big) < 1/2^j \label{proof mark process converge to markov chain absorbing case goal 4}
    \end{align}
    uniformly for all $k \in [k^\prime]$.
\end{itemize}

Fix such $(\bar{\boldsymbol{\Delta}}(j))_{j \geq 0},(\bar{ \boldsymbol{\eta} }(j))_{j \geq 0}$. Define a function $\textbf{J}(\cdot): \mathbb{N} \mapsto \mathbb{N}$ as
\begin{align*}
    \textbf{J}(n) = 0 \vee \max\{ j \geq 0:\ \bar{\boldsymbol{\eta}}(j) \geq \eta_n  \}
\end{align*}
with the convention that $\max \emptyset = -\infty$. Lastly, let
\begin{align*}
    \Delta_n = \bar{\boldsymbol{\Delta}}(\textbf{J}(n))\ \forall n \geq 1.
\end{align*}
Note that, due to $\lim_n \eta_n = 0$, we have $\lim_n \textbf{J}(n) = \infty$, hence $\lim_n \Delta_n = 0$. Besides, the definition of $\textbf{J}(\cdot)$ tells us that in case that $\textbf{J}(n) \geq 1$ (which will hold for all $n$ sufficiently large), the claims above holds with $\eta = \eta_n$ and $j = \textbf{J}(n)$. In particular, by combining $\lim_n \textbf{J}(n) = \infty$ with \cref{proof mark process converge to markov chain absorbing case goal 2}\cref{proof mark process converge to markov chain absorbing case goal 3}\cref{proof mark process converge to markov chain absorbing case goal 4} respectively, we have \cref{lemma mark process converge to markov chain absorbing case goal 2}\cref{lemma mark process converge to markov chain absorbing case goal 3}\cref{lemma mark process converge to markov chain absorbing case goal 4}.

Now it remains to prove \cref{lemma mark process converge to markov chain absorbing case goal 1}. To this end, it suffices to show that, for any positive integer $K$, we have $(S^{(n)}_0,W^{(n)}_0,\cdots,S^{(n)}_K,W^{(n)}_K )$ converges in distribution $(S_0,W_0,\cdots,S_K,W_K )$ as $n$ tends to infinity. In particular, note that $S_0 = 0, W_0 = \pi_G(m_i)$, so $W_0 = m_j$ with probability $p_{i,j}$ if $m_i \in G^\text{small}$, and $W_0 \equiv m_i$ if $m_i \in G^\text{large}$.

For clarity of the exposition, we restate some important claims above under the new notational system with $\widetilde{S}^{(n)}_k,\widetilde{W}^{(n)}_k,S^{(n)}_k,W^{(n)}_k$ we introduced right above this lemma. Given any $\epsilon > 0$, the following claims hold for all $n$ sufficiently large:
\begin{itemize}
    \item First of all,
    \begin{align}
    \mathbb{P}_x\Big( \widetilde{S}^{(n)}_0 < \eta_n^\theta,\ \widetilde{W}^{(n)}_0 = m_i \Big) > 1 - \epsilon.     \label{proof mark process convergence ineq 0 }
    \end{align}
    
    \item For all $k \geq 0$ and all $m_{i_1} \in G^\text{small},m_{i_2} \in G^\text{large}$,
    \begin{align}
        \bigg|\mathbb{P}_x\Big( S^{(n)}_k - \widetilde{S}^{(n)}_k < \eta_n^\theta,\ {W}^{(n)}_k = m_{i_2}    \ \Big|\ \widetilde{W}^{(n)}_k = m_{i_1} \Big) - p_{i_1,i_2}\bigg| < \epsilon. \label{proof mark process convergence ineq 1 }
    \end{align}
    \item For all $k \geq 0$ and all $m_{i_1} \in G^\text{large}$,
    \begin{align}
        \mathbb{P}_x\Big( S^{(n)}_k - \widetilde{S}^{(n)}_k = 0,\ {W}^{(n)}_k = m_{i_1}    \ \Big|\ \widetilde{W}^{(n)}_k = m_{i_1} \Big) = 1. \label{proof mark process convergence ineq 2 }
    \end{align}
    \item For all $k \geq 0$, all $m_{i_1} \in G^\text{large}, m_{i_2} \in G$ and all $u > \epsilon$,
    \begin{align}
        & - \epsilon + \exp\big( -(1+\epsilon)q_{i_1}u \big)\frac{ \nu_{i_1,i_2} - \epsilon }{q_{i_1} } \nonumber
        \\
        \leq & \mathbb{P}_x\Big( \widetilde{S}^{(n)}_{k+1} - S^{(n)}_k > u,\ \widetilde{W}^{(n)}_{k+1} = m_{i_2} \ \Big|\ W^{(n)}_k = m_{i_1}  \Big)  \nonumber
        \\
        \leq & \mathbb{P}_x\Big( \widetilde{S}^{(n)}_{k+1} - S^{(n)}_k > u - \eta_n^\theta,\ \widetilde{W}^{(n)}_{k+1} = m_{i_2} \ \Big|\ W^{(n)}_k = m_{i_1}  \Big)  \nonumber 
        \\
        \leq & \epsilon + \exp\big( -(1-\epsilon)q_{i_1}u \big)\frac{ \nu_{i_1,i_2} + \epsilon }{q_{i_1} } \label{proof mark process convergence ineq 3 }
    \end{align}
    \item Here is one implication of \cref{proof mark process convergence ineq 1 }. Since $|G|\leq n_\text{min}$, we have
    \begin{align}
       \mathbb{P}_x\Big( S^{(n)}_k - \widetilde{S}^{(n)}_k \geq \eta_n^\theta    \ \Big|\ \widetilde{W}^{(n)}_k = m_{i_1} \Big) < n_\text{min}\cdot\epsilon  \label{proof mark process convergence ineq 4 }
    \end{align}
    for all $k \geq 0$ and $m_{i_1} \in G^\text{small}$.
    \item Note that for any $m_{i_1},m_{i_2} \in G^\text{large}$ and any $k \geq 0$
    \begin{align*}
        & \mathbb{P}_x\Big( S^{(n)}_{k+1} - S^{(n)}_k > u,\ W^{(n)}_{k+1} = m_{i_2} \ \Big|\  W^{(n)}_k = m_{i_1}\Big)
        \\
        = & \mathbbm{1}\{i_2 \neq i_1\}\mathbb{P}_x\Big( \widetilde{S}^{(n)}_{k+1} - S^{(n)}_k > u,\ \widetilde{W}^{(n)}_{k+1} = m_{i_2} \ \Big|\  W^{(n)}_k = m_{i_1}\Big)
        \\
        + & \sum_{ i_3:\ m_{i_3} \in G^\text{small} }\int_{ s > 0 }\mathbb{P}_x\Big( S^{(n)}_{k+1} - \widetilde{S}^{(n)}_{k+1}  \geq (u-s)\vee 0 ,\ \widetilde{W}^{(n)}_{k+1} = m_{i_2} \ \Big|\  \widetilde{W}^{(n)}_k = m_{i_3}\Big)
        \\
        &\ \ \ \ \ \ \ \ \ \ \ \ \cdot \mathbb{P}_x\Big( \widetilde{S}^{(n)}_{k+1} - S^{(n)}_k = ds ,\ \widetilde{W}^{(n)}_{k+1} = m_{i_3} \ \Big|\  W^{(n)}_k = m_{i_1}\Big).
    \end{align*}
    Fix some $i_3$ with $m_{i_3} \in G^\text{small}$. On the one hand, due to \cref{proof mark process convergence ineq 4 },
    \begin{align*}
        & \int_{ s \in (0,u - \eta^\theta_n] }\mathbb{P}_x\Big( S^{(n)}_{k+1} - \widetilde{S}^{(n)}_{k+1}  \geq u - s ,\ \widetilde{W}^{(n)}_{k+1} = m_{i_2} \ \Big|\  \widetilde{W}^{(n)}_k = m_{i_3}\Big)
        \\
        & \ \ \ \ \ \ \ \ \ \ \cdot \mathbb{P}_x\Big( \widetilde{S}^{(n)}_{k+1} - S^{(n)}_k = ds ,\ \widetilde{W}^{(n)}_{k+1} = m_{i_3} \ \Big|\  W^{(n)}_k = m_{i_1}\Big)
        \\
        \leq & n_\text{min}\epsilon.
    \end{align*}
    On the other hand, by considering the integral on $(u - \eta_n^\theta,\infty)$, we get
    \begin{align*}
        & \int_{ s \in (u - \eta^\theta_n,\infty) }\mathbb{P}_x\Big( S^{(n)}_{k+1} - \widetilde{S}^{(n)}_{k+1}  \geq (u - s)\vee 0 ,\ \widetilde{W}^{(n)}_{k+1} = m_{i_2} \ \Big|\  \widetilde{W}^{(n)}_k = m_{i_3}\Big)
        \\
        & \ \ \ \ \ \ \ \ \ \ \cdot \mathbb{P}_x\Big( \widetilde{S}^{(n)}_{k+1} - S^{(n)}_k = ds ,\ \widetilde{W}^{(n)}_{k+1} = m_{i_3} \ \Big|\  W^{(n)}_k = m_{i_1}\Big)
        \\
        \geq & \int_{ s \in (u,\infty) }\mathbb{P}_x\Big( \widetilde{W}^{(n)}_{k+1} = m_{i_2} \ \Big|\  \widetilde{W}^{(n)}_k = m_{i_3}\Big)
        \\
        & \ \ \ \ \ \ \ \ \ \ \ \ \cdot \mathbb{P}_x\Big( \widetilde{S}^{(n)}_{k+1} - S^{(n)}_k = ds ,\ \widetilde{W}^{(n)}_{k+1} = m_{i_3} \ \Big|\  W^{(n)}_k = m_{i_1}\Big)
        \\
        \geq & (p_{i_3,i_2} - \epsilon)\Big( - \epsilon + \exp\big( -(1+\epsilon)q_{i_1}u \big)\frac{ \nu_{i_1,i_3} - \epsilon }{q_{i_1} } \Big)
    \end{align*}
    due to \cref{proof mark process convergence ineq 1 } and \cref{proof mark process convergence ineq 3 }. Meanwhile,
    \begin{align*}
        & \int_{ s \in (u - \eta^\theta_n,\infty) }\mathbb{P}_x\Big( S^{(n)}_{k+1} - \widetilde{S}^{(n)}_{k+1}  \geq (u - s)\vee 0 ,\ \widetilde{W}^{(n)}_{k+1} = m_{i_2} \ \Big|\  \widetilde{W}^{(n)}_k = m_{i_3}\Big)
        \\ 
        & \ \ \ \ \ \ \ \ \ \ \cdot \mathbb{P}_x\Big( \widetilde{S}^{(n)}_{k+1} - S^{(n)}_k = ds ,\ \widetilde{W}^{(n)}_{k+1} = m_{i_3} \ \Big|\  W^{(n)}_k = m_{i_1}\Big)
        \\
        \leq & (n_\text{min}\epsilon + p_{i_3,i_2} + \epsilon)\Big(  \epsilon + \exp\big( -(1-\epsilon)q_{i_1}u \big)\frac{ \nu_{i_1,i_3} + \epsilon }{q_{i_1} }  \Big)
    \end{align*}
    due to \cref{proof mark process convergence ineq 1 }, \cref{proof mark process convergence ineq 3 } and \cref{proof mark process convergence ineq 4 }.
    
    \item Therefore, for any $m_{i_1},m_{i_2} \in G^\text{large}$ and any $k \geq 0$,
    \begin{align}
        & \mathbb{P}_x\Big( S^{(n)}_{k+1} - S^{(n)}_k > u,\ W^{(n)}_{k+1} = m_{i_2} \ \Big|\  W^{(n)}_k = m_{i_1}\Big) \nonumber
        \\
        \leq & g(\epsilon) + \exp\big(-(1-\epsilon)q_{i_1} u\big)\frac{ \mathbbm{1}\{i_2 \neq i_1\}\nu_{i_1,i_2} + \sum_{i_3:\ m_{i_3} \in G^\text{small} }\nu_{i_1,i_3}p_{i_3,i_2} }{ q_{i_1}} \nonumber
        \\
        \leq & g(\epsilon) + \exp\big(-(1-\epsilon)q_{i_1} u\big)\frac{ q_{i_1,i_2} }{ q_{i_1}} \label{proof mark process convergence ineq 5}
    \end{align}
    and
    \begin{align}
        & \mathbb{P}_x\Big( S^{(n)}_{k+1} - S^{(n)}_k > u,\ W^{(n)}_{k+1} = m_{i_2} \ \Big|\  W^{(n)}_k = m_{i_1}\Big) \nonumber
        \\
        \geq & -g(\epsilon) + \exp\big(-(1+\epsilon)q_{i_1} u\big)\frac{ \mathbbm{1}\{i_2 \neq i_1\}\nu_{i_1,i_2} + \sum_{i_3:\ m_{i_3} \in G^\text{small} }\nu_{i_1,i_3}p_{i_3,i_2} }{ q_{i_1}} \nonumber
        \\
        \geq & -g(\epsilon) + \exp\big(-(1+\epsilon)q_{i_1} u\big)\frac{ q_{i_1,i_2} }{ q_{i_1}} \label{proof mark process convergence ineq 6 }
    \end{align}
    where $q^* = \max_i q_i$ and
    \begin{align*}
        g(\epsilon)\delequal{} & \epsilon + \frac{\epsilon}{q^*} + n_\text{min}(1+\epsilon)\epsilon + \epsilon\frac{1+\epsilon}{q^*}n_\text{min}
        \\
        & + n_\text{min}(\epsilon+\frac{\epsilon}{q^*}) + n_\text{min}(n_\text{min}\epsilon + \epsilon + 1)(1 + \frac{1}{q^*})\epsilon.
    \end{align*}
    Note that $\lim_{\epsilon \downarrow 0}g(\epsilon) = 0$.
\end{itemize}

Now we apply the bounds in \cref{proof mark process convergence ineq 0 }\cref{proof mark process convergence ineq 5}\cref{proof mark process convergence ineq 6 } to establish the weak convergence claim regarding $(S^{(n)}_0,W^{(n)}_0,\cdots,S^{(n)}_K,W^{(n)}_K )$. Fix some positive integer $K$, some strictly positive real numbers $(s_k)_{k = 0}^K$, a sequence $( w_k )_{k = 0}^K \in \big(G^\text{large}\big)^{K+1}$ with $w_k = m_{i_k}$ for each $k$, and some $\epsilon > 0$ such that $\epsilon < \min_{k = 0,1,\cdots,K}\{s_k\}$. On the one hand, the definition of the CTMC $Y$ implies that
\begin{align*}
    & \mathbb{P}\Big( S_0 < t_0, W_0 = w_0;\ S_k > s_k\ \text{and }W_k = w_k\ \forall k \in [K] \Big)
    \\
    = & \mathbb{P}( \pi_G(m_i) = w_0 )\prod_{k = 1}^K\Big(S_k > s_k\ \text{and }W_k = w_k\ \Big|\  W_{k-1} = w_{k-1}  \Big)
    \\
    = &\Big( \mathbbm{1}\{ m_i \in G^\text{large},\ i_0 = i \} + \mathbbm{1}\{ m_i \in G^\text{small}\}p_{i_0,i_1}  \Big)\cdot\prod_{k = 1}^K\exp( -q_{i_{k-1}}s_k)\frac{ q_{i_{k-1},i_k} }{ q_{i_{k-1}} }.
\end{align*}
On the other hand, using \cref{proof mark process convergence ineq 0 }\cref{proof mark process convergence ineq 5}\cref{proof mark process convergence ineq 6 }, we know that for all $n$ sufficiently large,
\begin{align*}
    & \mathbb{P}_x\Big( S^{(n)}_0 < s_0, W^{(n)}_0 = w_0;\ S^{(n)}_k > s_k\ \text{and }W^{(n)}_k = w_k\ \forall k \in [K] \Big)
    \\
    \geq & (1-\epsilon)\Big( \mathbbm{1}\{ m_i \in G^\text{large},\ i_0 = i \} + \mathbbm{1}\{ m_i \in G^\text{small}\}(p_{i_0,i_1} - \epsilon)  \Big)
    \\
    &\ \ \ \ \ \ \cdot\prod_{k = 1}^K\Big( -g(\epsilon) + \exp( -(1+\epsilon)q_{i_{k-1}}s_k)\frac{ q_{i_{k-1},i_k} }{ q_{i_{k-1}} } \Big)
\end{align*}
and 
\begin{align*}
    & \mathbb{P}_x\Big( S^{(n)}_0 < s_0, W^{(n)}_0 = w_0;\ S^{(n)}_k > s_k\ \text{and }W^{(n)}_k = w_k\ \forall k \in [K] \Big)
    \\
    \leq & \Big( \mathbbm{1}\{ m_i \in G^\text{large},\ i_0 = i \} + \mathbbm{1}\{ m_i \in G^\text{small}\}(p_{i_0,i_1} + \epsilon)  \Big)
    \\
    & \ \ \ \ \ \ \ \ \ \ \ \cdot\prod_{k = 1}^K\Big( g(\epsilon) + \exp( -(1 - \epsilon)q_{i_{k-1}}s_k)\frac{ q_{i_{k-1},i_k} }{ q_{i_{k-1}} } \Big).
\end{align*}
Since $\epsilon > 0$ can be arbitrarily small, we now obtain
\begin{align*}
    & \lim_{n \rightarrow \infty}\mathbb{P}_x\Big( S^{(n)}_0 < t_0, W^{(n)}_0 = w_0;\ S^{(n)}_k > s_k\ \text{and }W^{(n)}_k = m_k\ \forall k \in [K] \Big)
    \\
    = & \mathbb{P}\Big( S_0 < s_0, W_0 = w_0;\ S_k > s_k\ \text{and }W_k = w_k\ \forall k \in [K] \Big),
\end{align*}
and the arbitrariness of the integer $K$, the strictly positive real numbers $(s_k)_{k = 0}^K$, and the sequence $( w_k )_{k = 0}^K \in \big(G^\text{large}\big)^{K+1}$ allows us to conclude the proof.
\end{proof}

To extend the result above to the case where the communication class $G$ is transient, we revisit the definition of the $Y^\dagger$ in \cref{def process Y killed}. Let $\bar{G} = G^\text{large}\cup\{ \boldsymbol{\dagger} \}$ and let $m_0 = \boldsymbol{\dagger}$ (remember that all the local minimizers of $f$ on $[-L,L]$ are $m_1,\cdots,m_\text{min}$). Meanwhile, using $q_i$ and $q_{i,j}$ in \cref{def limiting MC Y 3}\cref{def limiting MC Y 4}, we can define
\begin{align*}
    q^\dagger_{i,j} = 
    \begin{cases}
    q_{i,j} & \text{if }i \geq 1,\ j \geq 1
    \\
    \mathbbm{1}\{j = 0\} & \text{if }i = 0
    \\
    \sum_{ j \in [n_\text{min}], m_j \notin G }q_{i,j}& \text{if }i \geq 1,\ j = 0.
    \end{cases}
\end{align*}
and $q^\dagger_0 = 1,\ q^\dagger_i = q_i\ \forall i \geq 1.$
Next, fix some $i$ with $m_i \in G$ and $x \in \Omega_i$. Define a sequence of random variables $ (S^\dagger_k)_{k \geq 0}, (W^\dagger_k)_{k \geq 0}$ such that $S^\dagger_k = 0$ and $W_0 = \pi_G(m_i), W^\dagger_0 = \boldsymbol{\dagger}\mathbbm{1}\{ W_0 \notin G^\text{large} \} + W_0\mathbbm{1}\{ W_0 \in G^\text{large} \}$ (see the definition of random mapping $\pi_G$  in \cref{def mapping pi G}) and (for all $k \geq 0$ and $i,j$ with $m_{j},m_l \in \bar{G}$)
\begin{align}
    & \mathbb{P}\Big( W^\dagger_{k+1} = m_l,\ S^\dagger_{k+1} > t   \ \Big|\ W^\dagger_k = m_j,\ (W^\dagger_l)_{l = 0}^{k-1},\ (S^\dagger_l)_{l = 0}^{k} \Big)
    \\
    = & \mathbb{P}\Big( W^\dagger_{k+1} = m_l,\ S^\dagger_{k+1} > t   \ \Big|\ W^\dagger_k = m_j \Big) = \exp(-q^\dagger_j t)\frac{ q^\dagger_{j,l}  }{ q^\dagger_j }\ \forall t > 0 
\end{align}
Then it is easy to see that $Y^\dagger(\pi_G(m_i))$ defined in \cref{def process Y killed}  is a $\big( (S^\dagger_k)_{k \geq 0}, (W^\dagger_k)_{k \geq 0}  \big)$ jump process. In particular, from at any state that is not $\boldsymbol{\dagger}$ (namely, any $m_j$ with $m_j \in G^\text{large}$), the probability that $Y^\dagger$ moves to $\boldsymbol{\dagger}$ in the next transition is equal to the chance that, starting from the same state, $Y$ moves to a state that is not in $G$. Once entering $m_0 = \boldsymbol{\dagger}$, the process $Y^\dagger$ will only make dummy jumps (with interarrival times being iid Exp(1)): indeed, we have $q^\dagger_0 = q^\dagger_{0,0} = 1$ and $q^\dagger_{0,j} = 0$ for any $j \geq 1$, implying that, given $W^\dagger_k = m_0 = \boldsymbol{\dagger}$, we must have $W^\dagger_{k+1} = m_0 = \boldsymbol{\dagger}$. These dummy jumps ensure that $Y^\dagger$ is stuck at the cemetery state $\boldsymbol{\dagger}$ after visiting it.

Similarly, we can characterize the jump times and locations of the jump process $\hat{X}^{ \dagger, *,\eta,\Delta }$ (for the definition, see \cref{def process X hat X killed}). When there is no ambiguity about the sequences $(\eta_n)_{n \geq 1},(\Delta_n)_{n \geq 1}$, let $\hat{X}^{\dagger,(n)} = \hat{X}^{\dagger,*,\eta_n,\Delta_n}$ and $X^{\dagger,(n)} = X^{*,\eta_n,\Delta_n}$. Also, recall that $\tau_G$ defined in \cref{def tau G X exits G} is the step $n$ when $X^\eta_n$ exits the communication class $G$. Now let $(E_k)_{k \geq 0}$ be a sequence of iid Exp(1) random variables that is also independent of the noises $(Z_k)_{k \geq 1}$ (so they are independent from the SGD iterates $X^\eta_n$). For all $n \geq 1, k \geq 0$, define (see \cref{def tau G sigma G I G 1}-\cref{def tau G sigma G I G 6} and \cref{def tau star sigma star scaled} for definitions of the quantities involved)
\begin{align*}
    \widetilde{S}^{\dagger,(n)}_k & = 
    \begin{cases}
        \sigma^*_k(\eta_n,\Delta_n)\wedge \textbf{T}^*(\tau_G(\eta_n),\eta_n) - \tau^*_{k-1}(\eta_n,\Delta_n) & \text{if }\tau^*_{k-1}(\eta_n,\Delta_n)< \textbf{T}^*(\tau_G(\eta_n),\eta_n)
        \\
        0 & \text{otherwise}
    \end{cases}
    \\
    S^{\dagger,(n)}_k & = 
    \begin{cases}
        \tau^*_k(\eta_n,\Delta_n)\wedge \textbf{T}^*(\tau_G(\eta_n),\eta_n)  - \tau^*_{k-1}(\eta_n,\Delta_n)& \text{if }\tau^*_{k-1}(\eta_n,\Delta_n)< \textbf{T}^*(\tau_G(\eta_n),\eta_n)
        \\
        E_k & \text{otherwise}
    \end{cases}
    \\
     \widetilde{W}^{\dagger,(n)}_k & = 
     \begin{cases}
     m_{ \widetilde{I}^G_k(\eta_n,\Delta_n) } & \text{if }\sigma^*_{k}(\eta_n,\Delta_n)< \textbf{T}^*(\tau_G(\eta_n),\eta_n)
     \\
     \boldsymbol{\dagger} & \text{otherwise}
     \end{cases}
     \\
    W^{\dagger,(n)}_k & =
    \begin{cases}
     m_{ I^G_k(\eta_n,\Delta_n) } & \text{if }\tau^*_{k}(\eta_n,\Delta_n)< \textbf{T}^*(\tau_G(\eta_n),\eta_n)
     \\
     \boldsymbol{\dagger} & \text{otherwise}
     \end{cases}
\end{align*}
with the convention that $\tau^*_{-1} = 0$. Note that $\textbf{T}^*(\tau_G(\eta_n),\eta_n)$ is the scaled timestamp for $X^{(n)} = X^{*,\eta}$ corresponding to $\tau_G(\eta_n)$, hence $\textbf{T}^*(\tau_G(\eta_n),\eta_n) = \min\{t \geq 0: X^{(n)} \notin \bigcup_{j:\ m_j \in G}\Omega_j\}$. One can see that $\hat{X}^{\dagger,(n)}$ is a $\big( (S^{\dagger,(n)}_k)_{k \geq 0},(W^{\dagger,(n)}_k)_{k \geq 0} \big)$ jump process. The next lemma is similar to Lemma \ref{lemma preparation weak convergence to markov chain absorbing case} and discusses the convergence of the jump times and locations of $\hat{X}^{\dagger,(n)}$ on a communication class $G$ in the transient case.

\begin{lemma} \label{lemma preparation weak convergence to markov chain trainsient case}
Assume that the communication class $G$ is transient. Given any $m_i \in G$, $x \in \Omega_i$, finitely many real numbers $(t_l)_{l = 1}^{k^\prime}$ such that $0<t_1<t_2<\cdots<t_{k^\prime}$, and a sequence of strictly positive real numbers $(\eta_n)_{n \geq 1}$ with $\lim_{n \rightarrow 0}\eta_n = 0$, there exists a sequence of strictly positive real numbers $(\Delta_n)_{n \geq 1}$ with $\lim_n \Delta_n = 0$ such that
\begin{itemize}
    \item Under $\mathbb{P}_x$ (so $X^\eta_0 = x$), as $n$ tends to $\infty$,
    \begin{align}
    (S^{\dagger,(n)}_0,W^{\dagger,(n)}_0,S^{\dagger,(n)}_1,W^{\dagger,(n)}_1,S^{\dagger,(n)}_2,W^{\dagger,(n)}_2,\cdots ) \Rightarrow  (S^\dagger_0,W^\dagger_0,S^\dagger_1,W^\dagger_1,S^\dagger_2,W^\dagger_2,\cdots ) \label{lemma mark process converge to markov chain transient case goal 1}
\end{align}

\item For any $l \in [k^\prime]$,
    \begin{align}
        \lim_{n \rightarrow \infty}\mathbb{P}_x\Big( X^{\dagger,(n)}_{t_l} \notin \bigcup_{j:\ m_j \in G^\text{large}}B(m_j,\Delta_n),\ X^{\dagger,(n)}_{s} \in \bigcup_{j:\ m_j \in G}\Omega_j\ \forall s \in [0,t_{l}] \Big) = 0 \label{lemma mark process converge to markov chain transient case goal 2}
    \end{align}
\end{itemize}

\end{lemma}

\begin{proof}
Let
\begin{align*}
    \nu_{j,k} & \delequal{} \mu_{j}(E_{j,k})\ \forall j,k \geq 1,\ j \neq k
    \\
    p_{j,\dagger} & \delequal{} \sum_{\widetilde{j}:\ m_{\widetilde{j}} \notin G }p_{j,\widetilde{j}}\ \ \ \forall m_j \in G^\text{small}
    \\
    q_{j,\dagger} & \delequal{} \sum_{k: m_k \notin G}\nu_{j,k} + \sum_{k:\ m_k \in G^\text{small}}\nu_{j,k}p_{k,\dagger}\ \ \ \forall m_j \in G^\text{large}.
\end{align*}

In order to specify our choice of $(\Delta_n)_{n \geq 1}$, we consider a construction of sequences $(\bar{\boldsymbol{\Delta}}(j))_{j \geq 0},(\bar{ \boldsymbol{\eta} }(j))_{j \geq 0}$ as follows. Fix some $\theta \in (0,\alpha - 1)/2)$. Let $\bar{\boldsymbol{\Delta}}(0) = \bar{\boldsymbol{\eta}}(0) = 1$. One can see the existence of some $(\bar{\boldsymbol{\Delta}}(j))_{j \geq 1},(\bar{ \boldsymbol{\eta} }(j))_{j \geq 1}$ such that
\begin{itemize}
    \item $\bar{\boldsymbol{\Delta}}(j) \in \big(0,\bar{\boldsymbol{\Delta}}(j-1)/2\big],\ \bar{\boldsymbol{\eta}}(j) \in \big(0,\bar{\boldsymbol{\eta}}(j-1)/2\big]$ for all $j \geq 1$;
    
    \item (Due to Lemma \ref{lemma return to local minimum quickly rescaled version}) for any $j \geq 1$, $\eta \in (0,\bar{\boldsymbol{\eta}}(j)]$, (remember that $x$ and $i$ are the fixed constants prescribed in the description of the lemma)
    \begin{align*}
        \mathbb{P}_x\Big( \sigma^*_0\big(\eta, \bar{\boldsymbol{\Delta}}(j) \big) < \eta^\theta,\ \widetilde{I}^G_0\big(\eta, \bar{\boldsymbol{\Delta}}(j)\big) = i \Big) > 1 - \frac{1}{2^j}.
    \end{align*}
    For definitions of $\sigma^G_k, \tau^G_k, I^G_k, \widetilde{I}^G_k$, see \cref{def tau G sigma G I G 1}-\cref{def tau G sigma G I G 6}.
    
    \item (Due to Lemma \ref{lemma sharp basin process absorbing}) for any $j \geq 1$, $\eta \in (0,\bar{\boldsymbol{\eta}}(j)]$,
    \begin{align*}
        & \bigg|\mathbb{P}_x\Big( \tau^*_k\big(\eta, \bar{\boldsymbol{\Delta}}(j)\big) - \sigma^*_k\big(\eta, \bar{\boldsymbol{\Delta}}(j)\big) < \eta^\theta,\ I^G_k\big(\eta, \bar{\boldsymbol{\Delta}}(j)\big) = i_2    \ \Big|\ \widetilde{I}^G_k\big(\eta, \bar{\boldsymbol{\Delta}}(j)\big) = i_1 \Big) - p_{i_1,i_2}\bigg|
        \\
        & < 1/2^j
    \end{align*}
    uniformly for all $k \geq 0$ and all $m_{i_1} \in G^\text{small}, m_{i_2}\notin G^\text{small}$. Also, by definition of $\sigma^*$ and $\tau^*$, we must have
    \begin{align*}
        \mathbb{P}_x\Big( \tau^*_k\big(\eta, \bar{\boldsymbol{\Delta}}(j)\big) - \sigma^*_k\big(\eta, \bar{\boldsymbol{\Delta}}(j)\big) = 0,\ I^G_k\big(\eta, \bar{\boldsymbol{\Delta}}(j)\big) = i_1    \ \Big|\ \widetilde{I}^G_k\big(\eta, \bar{\boldsymbol{\Delta}}(j)\big) = i_1 \Big) = 1
    \end{align*}
    for all $k \geq 0$ and $m_{i_1} \in G^\text{large}$.
    \item (Due to Proposition \ref{prop first exit and return time i}) for any $j \geq 1$, $\eta \in (0,\bar{\boldsymbol{\eta}}(j)]$,
    \begin{align*}
        & -\frac{1}{2^j} + \exp\big( -(1+\frac{1}{2^j})q_{i_1}u \big)\frac{ \nu_{i_1,i_2} - \frac{1}{2^j} }{q_{i_1} }
        \\
        \leq & \mathbb{P}_x\Big( \sigma^*_{k+1}\big(\eta, \bar{\boldsymbol{\Delta}}(j)\big) - \tau^*_k\big(\eta, \bar{\boldsymbol{\Delta}}(j)\big) > u,\ \widetilde{I}^G_{k+1} = i_2  \ \Big|\ I^G_k\big(\eta, \bar{\boldsymbol{\Delta}}(j)\big) = i_1  \Big)
        \\
        \leq & \frac{1}{2^j} + \exp\big( -(1-\frac{1}{2^j})q_{i_1}u \big)\frac{ \nu_{i_1,i_2} + \frac{1}{2^j} }{q_{i_1} }
    \end{align*}
    uniformly for all $k \geq 1$, all $u > 1/2^j$, and all $m_{i_1} \in G^\text{large}, m_{i_2} \in G$.

    \item (Due to Lemma \ref{lemma dynamic scaled X stays mostly at local minima}) for any $j \geq 1$, for any $j \geq 1$, $\eta \in (0,\bar{\boldsymbol{\eta}}(j)]$,
    \begin{align}
        \mathbb{P}_x\Big( X^{(n)}_{t_k} \notin \bigcup_{j:\ m_j \in G^\text{large}}B(m_j,\bar{\boldsymbol{\Delta}}(j)),\ X^{(n)}_{s} \in \bigcup_{j:\ m_j \in G}\Omega_j\ \forall s \in [0,t_{k}] \Big) < 1/2^j \label{proof mark process converge to markov chain transient case goal 2}
    \end{align}
    uniformly for all $k \in [k^\prime]$.
\end{itemize}

Fix such $(\bar{\boldsymbol{\Delta}}(j))_{j \geq 0},(\bar{ \boldsymbol{\eta} }(j))_{j \geq 0}$. Define a function $\textbf{J}(\cdot): \mathbb{N} \mapsto \mathbb{N}$ as
\begin{align*}
    \textbf{J}(n) = 0 \vee \max\{ j \geq 0:\ \bar{\boldsymbol{\eta}}(j) \geq \eta_n  \}
\end{align*}
with the convention that $\max \emptyset = -\infty$. Lastly, let
\begin{align*}
    \Delta_n = \bar{\boldsymbol{\Delta}}(\textbf{J}(n))\ \forall n \geq 1.
\end{align*}
Note that, due to $\lim_n \eta_n = 0$, we have $\lim_n \textbf{J}(n) = \infty$, hence $\lim_n \Delta_n = 0$. Besides, since $X^{\dagger,(n)}_t = X^{(n)}_t$ given $X^{(n)}_s \in \bigcup_{j:\ m_j \in G}\Omega_j$ for all $s\in [0,t]$, by combining $\lim_n \textbf{J}(n) = \infty$ with \cref{proof mark process converge to markov chain transient case goal 2} we obtain \cref{lemma mark process converge to markov chain transient case goal 2}.

Now it remains to prove \cref{lemma mark process converge to markov chain transient case goal 1}. To this end, it suffices to show that, for any positive integer $K$, we have $(S^{\dagger,(n)}_0,W^{\dagger,(n)}_0,\cdots,S^{\dagger,(n)}_K,W^{\dagger,(n)}_K )$ converges in distribution $(S^\dagger_0,W^\dagger_0,\cdots,S^\dagger_K,W^\dagger_K )$ as $n$ tends to infinity. In particular, due to introduction of the dummy jumps, we know that for any $k$ with $\tau^*_k(\eta_n,\Delta_n) \geq \tau_G(\eta_n)$ (in other words, $\hat{X}^{\dagger,(n)}$ has reached state $\boldsymbol{\dagger}$ within the first $k$ jumps) we have $S^{\dagger,(n)}_{k+1} \sim \text{Exp}(1)$ and $W^{\dagger,(n)}_{k+1} \equiv \boldsymbol{\dagger}$. Similarly, for any $k$ with $S^{\dagger}_0 + \cdots + S^{\dagger}_k \leq \tau^Y_G$, we have $S^\dagger_{k+1} \sim \text{Exp}(1)$ and $W^\dagger_{k+1} \equiv \boldsymbol{\dagger}$. Therefore, it suffices to show that, for any positive integer $K$, any series of strictly positive real numbers $(s_k)_{k = 0}^K$, any sequence $( w_k )_{k = 0}^K \in \big(\bar{G}\big)^{K+1}$ such that $w_j \neq \boldsymbol{\dagger}$ for any $j < K$, indices $i_k$ such that $w_k = m_{i_k}$ for each $k$, we have
\begin{align}
    & \lim_{n \rightarrow \infty}\mathbb{P}_x\Big( S^{\dagger,(n)}_0 < t_0, W^{\dagger,(n)}_0 = w_0;\ S^{\dagger,(n)}_k > s_k\ \text{and }W^{\dagger,(n)}_k = w_k\ \forall k \in [K] \Big) \nonumber
    \\
    = & \mathbb{P}\Big( S^\dagger_0 < s_0, W^\dagger_0 = w_0;\ S^\dagger_k > s_k\ \text{and }W^\dagger_k = w_k\ \forall k \in [K] \Big) \label{proof mark process converge to markov chain transient case goal 3}
\end{align}

Fix some $(s_k)_{k = 0}^K$, $( w_k )_{k = 0}^K \in \big(\bar{G}\big)^{K+1}$, and indices $(i_k)_{k =1}^K$ satisfying the conditions above. Besides, arbitrarily choose some $\epsilon > 0$ so that $\epsilon < \min_{k = 0,\cdots,K}s_k$. To proceed, we start by translating the inequalities established above under the new system of notations.
\begin{itemize}
    \item First of all, for all $n$ sufficiently large, (remember that $x$ and $i$ are prescribed constants in the description of this lemma)
    \begin{align}
    \mathbb{P}_x\Big( \widetilde{S}^{\dagger,(n)}_0 < \eta_n^\theta,\ \widetilde{W}^{\dagger,(n)}_0 = m_i \Big) > 1 - \epsilon.     \label{proof mark process convergence transient case ineq 0 }
    \end{align}
    
    \item For all $k \geq 0$ and all $m_{i_1} \in G^\text{small},m_{i_2} \in G^\text{large}$, it holds for all $n$ sufficiently large that
    \begin{align}
        \bigg|\mathbb{P}_x\Big( S^{\dagger,(n)}_k - \widetilde{S}^{\dagger,(n)}_k < \eta_n^\theta,\ {W}^{\dagger,(n)}_k = m_{i_2}    \ \Big|\ \widetilde{W}^{\dagger,(n)}_k = m_{i_1} \Big) - p_{i_1,i_2}\bigg| < \epsilon. \label{proof mark process convergence transient case ineq 1 }
    \end{align}
    
    \item On the other hand, for all $k \geq 0$ and all $m_{i_1} \in G^\text{small}$, it holds for all $n$ sufficiently large that
    \begin{align}
        & \bigg|\mathbb{P}_x\Big( S^{\dagger,(n)}_k - \widetilde{S}^{\dagger,(n)}_k < \eta_n^\theta,\ {W}^{\dagger,(n)}_k = \boldsymbol{\dagger}    \ \Big|\ \widetilde{W}^{\dagger,(n)}_k = m_{i_1} \Big) - \sum_{ i_2:\ m_{i_2}\notin G } p_{i_1,i_2}\bigg| \nonumber
        \\
        = & \bigg|\mathbb{P}_x\Big( S^{\dagger,(n)}_k - \widetilde{S}^{\dagger,(n)}_k < \eta_n^\theta,\ {W}^{\dagger,(n)}_k = \boldsymbol{\dagger}    \ \Big|\ \widetilde{W}^{\dagger,(n)}_k = m_{i_1} \Big) - p_{i_1,\dagger}\bigg|  < \epsilon. \label{proof mark process convergence transient case ineq 2 }
    \end{align}
    
    \item For all $k \geq 0$ and all $m_{i_1} \in G^\text{large}$, it holds for all $n$ that
    \begin{align}
        \mathbb{P}_x\Big( S^{\dagger,(n)}_k - \widetilde{S}^{\dagger,(n)}_k = 0,\ {W}^{\dagger,(n)}_k = m_{i_1}    \ \Big|\ \widetilde{W}^{\dagger,(n)}_k = m_{i_1} \Big) = 1. \label{proof mark process convergence transient case ineq 3 }
    \end{align}
    
    \item For all $k \geq 0$, all $m_{i_1} \in G^\text{large}, m_{i_2} \in G$ and all $u > \epsilon$, the following claim holds for all $n$ sufficiently large:
    \begin{align}
        & - \epsilon + \exp\big( -(1+\epsilon)q_{i_1}u \big)\frac{ \nu_{i_1,i_2} - \epsilon }{q_{i_1} } \nonumber
        \\
        \leq & \mathbb{P}_x\Big( \widetilde{S}^{\dagger,(n)}_{k+1} - S^{\dagger,(n)}_k > u,\ \widetilde{W}^{\dagger,(n)}_{k+1} = m_{i_2} \ \Big|\ W^{\dagger,(n)}_k = m_{i_1}  \Big)  \nonumber
        \\
        \leq & \mathbb{P}_x\Big( \widetilde{S}^{\dagger,(n)}_{k+1} - S^{\dagger,(n)}_k > u - \eta_n^\theta,\ \widetilde{W}^{\dagger,(n)}_{k+1} = m_{i_2} \ \Big|\ W^{\dagger,(n)}_k = m_{i_1}  \Big)  \nonumber 
        \\
        \leq & \epsilon + \exp\big( -(1-\epsilon)q_{i_1}u \big)\frac{ \nu_{i_1,i_2} + \epsilon }{q_{i_1} } \label{proof mark process convergence transient case ineq 4 }
    \end{align}
    
    \item On the other hand, for all $k \geq 0$, all $m_{i_1} \in G^\text{large}$, the following claim holds for all $n$ sufficiently large:
    \begin{align}
   & - \epsilon + \exp\big( -(1+\epsilon)q_{i_1}u \big)\frac{ -\epsilon+ \sum_{i_2:\ m_{i_2}\notin G} \nu_{i_1,i_2}}{q_{i_1} } \nonumber
        \\
        \leq & \mathbb{P}_x\Big( \widetilde{S}^{\dagger,(n)}_{k+1} - S^{\dagger,(n)}_k > u,\ \widetilde{W}^{\dagger,(n)}_{k+1} = \boldsymbol{\dagger} \ \Big|\ W^{\dagger,(n)}_k = m_{i_1}  \Big)  \nonumber
        \\
        \leq & \mathbb{P}_x\Big( \widetilde{S}^{\dagger,(n)}_{k+1} - S^{\dagger,(n)}_k > u - \eta_n^\theta,\ \widetilde{W}^{\dagger,(n)}_{k+1} = \boldsymbol{\dagger}\ \Big|\ W^{\dagger,(n)}_k = m_{i_1}  \Big)  \nonumber 
        \\
        \leq & \epsilon + \exp\big( -(1-\epsilon)q_{i_1}u \big)\frac{ \epsilon+ \sum_{i_2:\ m_{i_2}\notin G} \nu_{i_1,i_2} }{q_{i_1} } \label{proof mark process convergence transient case ineq 5 }
    \end{align}
    
    \item Here is one implication of \cref{proof mark process convergence transient case ineq 1 }\cref{proof mark process convergence transient case ineq 2 }. Since $|G|\leq n_\text{min}$, we have (when $n$ is sufficiently large)
    \begin{align}
       \mathbb{P}_x\Big( S^{(n)}_k - \widetilde{S}^{(n)}_k \geq \eta_n^\theta    \ \Big|\ \widetilde{W}^{(n)}_k = m_{i_1} \Big) < n_\text{min}\cdot\epsilon  \label{proof mark process convergence transient case ineq 6 }
    \end{align}
    for all $k \geq 0$ and $m_{i_1} \in G^\text{small}$.
    
    \item Note that for any $m_{i_1},m_{i_2} \in G^\text{large}$ and any $k \geq 0$
    \begin{align*}
        & \mathbb{P}_x\Big( S^{\dagger,(n)}_{k+1} - S^{\dagger,(n)}_k > u,\ W^{\dagger,(n)}_{k+1} = m_{i_2} \ \Big|\  W^{\dagger,(n)}_k = m_{i_1}\Big)
        \\
        = & \mathbbm{1}\{i_2 \neq i_1\}\mathbb{P}_x\Big( \widetilde{S}^{\dagger,(n)}_{k+1} - S^{\dagger,(n)}_k > u,\ \widetilde{W}^{\dagger,(n)}_{k+1} = m_{i_2} \ \Big|\  W^{\dagger,(n)}_k = m_{i_1}\Big)
        \\
        + & \sum_{ i_3:\ m_{i_3} \in G^\text{small} }\int_{ s > 0 }\mathbb{P}_x\Big( S^{\dagger,(n)}_{k+1} - \widetilde{S}^{\dagger,(n)}_{k+1}  \geq (u-s)\vee 0 ,\ \widetilde{W}^{\dagger,(n)}_{k+1} = m_{i_2} \ \Big|\  \widetilde{W}^{\dagger,(n)}_k = m_{i_3}\Big)
        \\
        &\ \ \ \ \ \ \ \ \ \ \ \ \cdot \mathbb{P}_x\Big( \widetilde{S}^{\dagger,(n)}_{k+1} - S^{\dagger,(n)}_k = ds ,\ \widetilde{W}^{\dagger,(n)}_{k+1} = m_{i_3} \ \Big|\  W^{\dagger,(n)}_k = m_{i_1}\Big).
    \end{align*}
    Fix some $i_3$ with $m_{i_3} \in G^\text{small}$. Due to \cref{proof mark process convergence transient case ineq 6 }
    \begin{align*}
        & \int_{ s \in (0,u - \eta^\theta_n] }\mathbb{P}_x\Big( S^{\dagger,(n)}_{k+1} - \widetilde{S}^{\dagger,(n)}_{k+1}  \geq u - s ,\ \widetilde{W}^{\dagger,(n)}_{k+1} = m_{i_2} \ \Big|\  \widetilde{W}^{\dagger,(n)}_k = m_{i_3}\Big)
        \\
        & \ \ \ \ \ \ \ \ \ \ \cdot \mathbb{P}_x\Big( \widetilde{S}^{\dagger,(n)}_{k+1} - S^{\dagger,(n)}_k = ds ,\ \widetilde{W}^{\dagger,(n)}_{k+1} = m_{i_3} \ \Big|\  W^{\dagger,(n)}_k = m_{i_1}\Big)
        \\
        \leq & n_\text{min}\epsilon.
    \end{align*}
    Meanwhile, by considering the integral on $(u - \eta_n^\theta,\infty)$, we get
    \begin{align*}
        & \int_{ s \in (u - \eta^\theta_n,\infty) }\mathbb{P}_x\Big( S^{\dagger,(n)}_{k+1} - \widetilde{S}^{\dagger,(n)}_{k+1}  \geq (u - s)\vee 0 ,\ \widetilde{W}^{\dagger,(n)}_{k+1} = m_{i_2} \ \Big|\  \widetilde{W}^{\dagger,(n)}_k = m_{i_3}\Big)
        \\
        & \ \ \ \ \ \ \ \ \ \ \cdot \mathbb{P}_x\Big( \widetilde{S}^{\dagger,(n)}_{k+1} - S^{\dagger,(n)}_k = ds ,\ \widetilde{W}^{\dagger,(n)}_{k+1} = m_{i_3} \ \Big|\  W^{\dagger,(n)}_k = m_{i_1}\Big)
        \\
        \geq & \int_{ s \in (u,\infty) }\mathbb{P}_x\Big( \widetilde{W}^{\dagger,(n)}_{k+1} = m_{i_2} \ \Big|\  \widetilde{W}^{\dagger,(n)}_k = m_{i_3}\Big)
        \\
        & \ \ \ \ \ \ \ \ \ \ \cdot \mathbb{P}_x\Big( \widetilde{S}^{\dagger,(n)}_{k+1} - S^{\dagger,(n)}_k = ds ,\ \widetilde{W}^{\dagger,(n)}_{k+1} = m_{i_3} \ \Big|\  W^{\dagger,(n)}_k = m_{i_1}\Big)
        \\
        \geq & (p_{i_3,i_2} - \epsilon)\Big( - \epsilon + \exp\big( -(1+\epsilon)q_{i_1}u \big)\frac{ \nu_{i_1,i_3} - \epsilon }{q_{i_1} } \Big)
    \end{align*}
    due to \cref{proof mark process convergence transient case ineq 1 } and \cref{proof mark process convergence transient case ineq 4 }. As for the upper bound,
    \begin{align*}
        & \int_{ s \in (u - \eta^\theta_n,\infty) }\mathbb{P}_x\Big( S^{\dagger,(n)}_{k+1} - \widetilde{S}^{\dagger,(n)}_{k+1}  \geq (u - s)\vee 0 ,\ \widetilde{W}^{\dagger,(n)}_{k+1} = m_{i_2} \ \Big|\  \widetilde{W}^{\dagger,(n)}_k = m_{i_3}\Big)
        \\ 
        & \ \ \ \ \ \ \ \ \ \ \cdot \mathbb{P}_x\Big( \widetilde{S}^{\dagger,(n)}_{k+1} - S^{\dagger,(n)}_k = ds ,\ \widetilde{W}^{\dagger,(n)}_{k+1} = m_{i_3} \ \Big|\  W^{\dagger,(n)}_k = m_{i_1}\Big)
        \\
        \leq & (n_\text{min}\epsilon + p_{i_3,i_2} + \epsilon)\Big(  \epsilon + \exp\big( -(1-\epsilon)q_{i_1}u \big)\frac{ \nu_{i_1,i_3} + \epsilon }{q_{i_1} }  \Big)
    \end{align*}
    due to  \cref{proof mark process convergence transient case ineq 1 }, \cref{proof mark process convergence transient case ineq 4 } and \cref{proof mark process convergence transient case ineq 6 }.
    
    \item Therefore, for any $m_{i_1},m_{i_2} \in G^\text{large}$ and any $k \geq 0$,
    \begin{align}
        & \mathbb{P}_x\Big( S^{\dagger,(n)}_{k+1} - S^{\dagger,(n)}_k > u,\ W^{\dagger,(n)}_{k+1} = m_{i_2} \ \Big|\  W^{\dagger,(n)}_k = m_{i_1}\Big) \nonumber
        \\
        \leq & g(\epsilon) + \exp\big(-(1-\epsilon)q_{i_1} u\big)\frac{ \mathbbm{1}\{i_2 \neq i_1\}\nu_{i_1,i_2} + \sum_{i_3:\ m_{i_3} \in G^\text{small} }\nu_{i_1,i_3}p_{i_3,i_2} }{ q_{i_1}} \nonumber
        \\
        \leq & g(\epsilon) + \exp\big(-(1-\epsilon)q_{i_1} u\big)\frac{ q_{i_1,i_2} }{ q_{i_1}} \label{proof mark process convergence transient case ineq 7}
    \end{align}
    and
    \begin{align}
        & \mathbb{P}_x\Big( S^{\dagger,(n)}_{k+1} - S^{\dagger,(n)}_k > u,\ W^{\dagger,(n)}_{k+1} = m_{i_2} \ \Big|\  W^{\dagger,(n)}_k = m_{i_1}\Big) \nonumber
        \\
        \geq & -g(\epsilon) + \exp\big(-(1+\epsilon)q_{i_1} u\big)\frac{ \mathbbm{1}\{i_2 \neq i_1\}\nu_{i_1,i_2} + \sum_{i_3:\ m_{i_3} \in G^\text{small} }\nu_{i_1,i_3}p_{i_3,i_2} }{ q_{i_1}} \nonumber
        \\
        \geq & -g(\epsilon) + \exp\big(-(1+\epsilon)q_{i_1} u\big)\frac{ q_{i_1,i_2} }{ q_{i_1}} \label{proof mark process convergence transient case ineq 8}
    \end{align}
    where $q^* = \max_i q_i$ and
    \begin{align*}
        g(\epsilon)\delequal{} & 2\epsilon + \frac{\epsilon}{q^*} + n_\text{min}(1+\epsilon)\epsilon + \epsilon\frac{1+\epsilon}{q^*}n_\text{min}
        \\
        & + n_\text{min}(\epsilon+\frac{\epsilon}{q^*}) + n_\text{min}(n_\text{min}\epsilon + \epsilon + 1)(1 + \frac{1}{q^*})\epsilon.
    \end{align*}
    Note that $\lim_{\epsilon \downarrow 0}g(\epsilon) = 0$.
    
    \item On the other hand, for the case where the marker process $\hat{X}^{\dagger,(n)}$ jumps to the cemetery state $\boldsymbol\dagger$ from some $m_{i_1} \in G^\text{large}$, note that
    \begin{align*}
        & \mathbb{P}_x\Big( S^{\dagger,(n)}_{k+1} - S^{\dagger,(n)}_k > u,\ W^{\dagger,(n)}_{k+1} = \boldsymbol{\dagger} \ \Big|\  W^{\dagger,(n)}_k = m_{i_1}\Big)
        \\
        = & \mathbb{P}_x\Big( \widetilde{S}^{\dagger,(n)}_{k+1} - S^{\dagger,(n)}_k > u,\ \widetilde{W}^{\dagger,(n)}_{k+1} = \boldsymbol{\dagger} \ \Big|\  W^{\dagger,(n)}_k = m_{i_1}\Big)
        \\
        + & \sum_{ i_3:\ m_{i_3} \in G^\text{small} }\int_{ s > 0 }\mathbb{P}_x\Big( S^{\dagger,(n)}_{k+1} - \widetilde{S}^{\dagger,(n)}_{k+1}  \geq (u-s)\vee 0 ,\ \widetilde{W}^{\dagger,(n)}_{k+1} = \boldsymbol{\dagger}\ \Big|\  \widetilde{W}^{\dagger,(n)}_k = m_{i_3}\Big)
        \\
        &\ \ \ \ \ \ \ \ \ \ \ \ \cdot \mathbb{P}_x\Big( \widetilde{S}^{\dagger,(n)}_{k+1} - S^{\dagger,(n)}_k = ds ,\ \widetilde{W}^{\dagger,(n)}_{k+1} = m_{i_3} \ \Big|\  W^{\dagger,(n)}_k = m_{i_1}\Big).
    \end{align*}
    Arguing similarly as we did above by considering the integral on $[0,u-\eta^\theta_n]$ and $(u - \eta^\theta_n,\infty)$ separately, and using \cref{proof mark process convergence transient case ineq 2 } and \cref{proof mark process convergence transient case ineq 5 }, we then get (for all $n$ sufficiently large)
    \begin{align}
        & \mathbb{P}_x\Big( S^{\dagger,(n)}_{k+1} - S^{\dagger,(n)}_k > u,\ W^{\dagger,(n)}_{k+1} = \boldsymbol{\dagger} \ \Big|\  W^{\dagger,(n)}_k = m_{i_1}\Big)\nonumber
        \\
        \leq & g(\epsilon) + \exp\big(-(1-\epsilon)q_{i_1} u\big)\frac{\sum_{i_2: m_{i_2}\notin G}\nu_{i_1,i_2} + \sum_{i_2: m_{i_2}\in G^\text{small}}\nu_{i_1,i_2}p_{i_2,\dagger} }{ q_{i_1}}  \nonumber
        \\
        = & g(\epsilon) + \exp\big(-(1-\epsilon)q_{i_1} u\big)\frac{q_{i_1,\dagger} }{ q_{i_1}} \label{proof mark process convergence transient case ineq 9}
    \end{align}
    and
    \begin{align}
        & \mathbb{P}_x\Big( S^{\dagger,(n)}_{k+1} - S^{\dagger,(n)}_k > u,\ W^{\dagger,(n)}_{k+1} = \boldsymbol{\dagger} \ \Big|\  W^{\dagger,(n)}_k = m_{i_1}\Big) \nonumber
        \\
        \geq & -g(\epsilon) + \exp\big(-(1+\epsilon)q_{i_1} u\big)\frac{\sum_{i_2: m_{i_2}\notin G}\nu_{i_1,i_2} + \sum_{i_2: m_{i_2}\in G^\text{small}}\nu_{i_1,i_2}p_{i_2,\dagger} }{ q_{i_1}} \nonumber
        \\
        = & -g(\epsilon) +\exp\big(-(1+\epsilon)q_{i_1} u\big)\frac{q_{i_1,\dagger} }{ q_{i_1}} \label{proof mark process convergence transient case ineq 10}
    \end{align}
    
\end{itemize}
For simplicity of presentation, we also let $q_{j,0} = q_{j,\dagger}$ and $p_{j,0} = p_{j,\dagger}$. First of all, remember that we have fixed some series of strictly positive real numbers $(s_k)_{k = 0}^K$, some sequence $( w_k )_{k = 0}^K \in \big(\bar{G}\big)^{K+1}$ such that $w_j \neq \boldsymbol{\dagger}$ for any $j < K$, and indices $i_k$ such that $w_k = m_{i_k}$ for each $k$. The definition of the continuous-time Markov chain $Y^\dagger$ implies that
\begin{align*}
    & \mathbb{P}\Big( S_0 < t_0, W_0 = w_0;\ S_k > s_k\ \text{and }W_k = w_k\ \forall k \in [K] \Big)
    \\
    = & \mathbb{P}( \pi_G(m_i) = w_0 )\prod_{k = 1}^K\Big(S_k > s_k\ \text{and }W_k = m_k\ \Big|\  W_{k-1} = w_{k-1}  \Big)
    \\
    = &\Big( \mathbbm{1}\{ m_i \in G^\text{large},\ i_0 = i \} + \mathbbm{1}\{ m_i \in G^\text{small}\}p_{i_0,i_1}  \Big)\cdot\prod_{k = 1}^K\exp( -q_{i_{k-1}}s_k)\frac{ q_{i_{k-1},i_k} }{ q_{i_{k-1}} }.
\end{align*}
On the other hand, using \cref{proof mark process convergence transient case ineq 7}-\cref{proof mark process convergence transient case ineq 10}, we know that for all $n$ sufficiently large,
\begin{align*}
    & \mathbb{P}_x\Big( S^{(n)}_0 < s_0, W^{(n)}_0 = w_0;\ S^{(n)}_k > s_k\ \text{and }W^{(n)}_k = w_k\ \forall k \in [K] \Big)
    \\
    \geq & (1-\epsilon)\Big( \mathbbm{1}\{ m_i \in G^\text{large},\ i_0 = i \} + \mathbbm{1}\{ m_i \in G^\text{small}\}(p_{i_0,i_1} - \epsilon)  \Big)
    \\
    &\ \ \ \ \ \ \ \ \ \ \cdot\prod_{k = 1}^K\Big( -g(\epsilon) + \exp( -(1+\epsilon)q_{i_{k-1}}s_k)\frac{ q_{i_{k-1},i_k} }{ q_{i_{k-1}} } \Big)
\end{align*}
and 
\begin{align*}
    & \mathbb{P}_x\Big( S^{(n)}_0 < s_0, W^{(n)}_0 = w_0;\ S^{(n)}_k > s_k\ \text{and }W^{(n)}_k = m_k\ \forall k \in [K] \Big)
    \\
    \leq & \Big( \mathbbm{1}\{ m_i \in G^\text{large},\ i_0 = i \} + \mathbbm{1}\{ m_i \in G^\text{small}\}(p_{i_0,i_1} + \epsilon)  \Big)
    \\
    &\ \ \ \ \ \ \ \ \ \ \cdot\prod_{k = 1}^K\Big( g(\epsilon) + \exp( -(1 - \epsilon)q_{i_{k-1}}s_k)\frac{ q_{i_{k-1},i_k} }{ q_{i_{k-1}} } \Big).
\end{align*}
The arbitrariness of $\epsilon > 0$ then allows us to establish \cref{proof mark process converge to markov chain transient case goal 3} and conclude the proof.
\end{proof}

Now we are ready to prove Lemma \ref{lemma key converge to markov chain and always at wide basin} and Lemma \ref{lemma key converge to markov chain and always at wide basin transient case}.

\begin{proof}[Proof of Lemma \ref{lemma key converge to markov chain and always at wide basin}]
From Lemma \ref{lemma preparation weak convergence to markov chain absorbing case}, one can see the existence of some $(\Delta_n)_{n \geq 1}$ with $\lim_{n}\Delta_n = 0$ such that \cref{lemma mark process converge to markov chain absorbing case goal 1}-\cref{lemma mark process converge to markov chain absorbing case goal 4} hold. For simplicity of notations, we let $\hat{X}^{(n)} = \hat{X}^{*,\eta_n,\Delta_n }, X^{(n)} = X^{*,\eta_n}$, and let $\bar{t} = t_{k^\prime}$ 

Combine \cref{lemma mark process converge to markov chain absorbing case goal 1} with Lemma \ref{lemma weak convergence of jump process}, and we immediately get \cref{lemma key converge to markov chain and always at wide basin goal 1}. In order to prove \cref{lemma key converge to markov chain and always at wide basin goal 2}, it suffices to show that for any $\epsilon > 0$,
\begin{align*}
   \limsup_n \mathbb{P}_x\Big( X^{(n)}_{t_k} \notin \bigcup_{j:\ m_j \in G^\text{large}}B(m_j,\Delta_n) \Big) \leq 4\epsilon\ \forall k \in [k^\prime].
\end{align*}
Fix $\epsilon > 0$, and observe following bound by decomposing the events
\begin{align*}
    & \mathbb{P}_x\Big( X^{(n)}_{t_k} \notin \bigcup_{j:\ m_j \in G^\text{large}}B(m_j,\Delta_n) \Big)
    \\
    \leq & \mathbb{P}_x\Big( X^{(n)}_{t_k} \notin \bigcup_{j:\ m_j \in G^\text{large}}B(m_j,\Delta_n),\  X^{(n)}_t \in \bigcup_{ j:\ m_j \in G }\Omega_j\ \forall t \in [0,\bar{t}] \Big)
    \\
    + & \mathbb{P}_x\Big( \exists t \in [0,\bar{t}]\ \text{such that }  X^{(n)}_t \notin \bigcup_{ j:\ m_j \in G }\Omega_j \Big)
\end{align*}
Therefore, given \cref{lemma mark process converge to markov chain absorbing case goal 4}, it suffices to prove
\begin{align}
    \limsup_n \mathbb{P}_x\Big( \exists t \in [0,\bar{t}]\ \text{such that }  X^{(n)}_t \notin \bigcup_{ j:\ m_j \in G }\Omega_j \Big) \leq 3\epsilon. \label{proof key convergence to MC goal}
\end{align}
Let 
\begin{align*}
    T^{(n)}_0 & \delequal{} \min\{ t \geq 0:\ X^{(n)}_t \in \bigcup_{ j: m_j \in G }B(m_j,2\Delta_n)  \}
    \\
    I^{(n)}_0 & = j \iff X^{(n)}_{T^{(n)}_0 } \in B(m_j,2\Delta_n)
    \\
    T^{(n)}_k & \delequal{} \min\{ t  > T^{(n)}_{k-1}:\ X^{(n)}_t \in \bigcup_{ j: m_j \in G,\ j \neq I^{(n)}_{k-1} }B(m_j,2\Delta_n)  \}
    \\
    I^{(n)}_k & = j \iff X^{(n)}_{T^{(n)}_k } \in B(m_j,2\Delta_n).
\end{align*}
Building upon this definition, we define the following stopping times and marks that only records the hitting time to minimizer in \textit{large} attraction fields in $G$ (with convention $\textbf{k}^{(n),\text{large}}(-1) = -1, T^{(n),\text{large}}_{-1} = 0,T^{(n)}_{-1} = 0$)
\begin{align*}
    \textbf{k}^{(n),\text{large}}(k) & \delequal{} \min\{l > \textbf{k}^{(n),\text{large}}({k-1}):\ m_{ I^{(n)}_l } \in G^\text{large}  \} 
    \\
    T^{(n),\text{large}}_k & \delequal{} T^{(n)}_{ \textbf{k}^{(n),\text{large}}(k) },\ \ \ I^{(n),\text{large}}_k \delequal{} I^{(n)}_{ \textbf{k}^{(n),\text{large}}(k) }.
\end{align*}
Now by defining
\begin{align*}
    J^{(n)}(t) & \delequal{} \#\{ k \geq 0:\ T^{(n)}_k \leq t  \},
    \\
    J^{(n)}_\text{large}(t) & \delequal{} \#\{ k \geq 0:\ T^{(n),\text{large}}_k \leq t\},
    \\
    J^{(n)}(s,t) & \delequal{}\#\{ k \geq 0:\ T^{(n)}_k \in [s,t] \},
\end{align*}
we use $J^{(n)}(t)$ to count the numbers of visits to local minima on $G$, and $J^{(n)}_\text{large}(t)$ for the number of visits to minimizers in the \textit{large} attraction fields on $G$. $J^{(n)}(s,t)$ counts the indices $k$ such that at $T^{(n)}_k$ a minimizer on $G$ is visited and regarding the hitting time we have $T^{(n)}_k \in [s,t]$.

First of all, the weak convergence result in \cref{lemma mark process converge to markov chain absorbing case goal 1} implies the existence of some positive integer $N(\epsilon)$ such that
\begin{align*}
    \limsup_n \mathbb{P}_x( J^{(n)}_\text{large}(\bar{t}) > N(\epsilon) ) < \epsilon.
\end{align*}
Fix such $N(\epsilon)$. Next, from \cref{lemma mark process converge to markov chain absorbing case goal 3}, we know the existence of some integer $K(\epsilon)$ such that
\begin{align*}
    \limsup_n\sup_{k \geq 0} \mathbb{P}_x\Big( J^{(n)}\big( T^{(n),\text{large}}_{k-1}, T^{(n),\text{large}}_{k}  \big) > K(\epsilon) \Big) \leq \epsilon/N(\epsilon).
\end{align*}
Fix such $K(\epsilon)$ as well. From the results above, we know that for event
\begin{align*}
    A_1(n) \delequal{} \{ J^{(n)}_\text{large}(\bar{t}) \leq N(\epsilon) \} \cap \Big\{ J^{(n)}\big( T^{(n),\text{large}}_{k-1}, T^{(n),\text{large}}_{k}  \big) \leq K(\epsilon)\ \forall k \in [N(\epsilon)] \Big\},
\end{align*}
we have $\limsup_n \mathbb{P}_x\Big( \big(A_1(n)\big)^c \Big) \leq 2\epsilon$. On the other hand, on event $A_1(n)$, we must have
\begin{align*}
    J^{(n)}(\bar{t}) \leq N(\epsilon)K(\epsilon).
\end{align*}
Meanwhile, it follows immediately from \cref{lemma mark process converge to markov chain absorbing case goal 2} that
\begin{align*}
    \limsup_n\sup_{k \geq 0} \mathbb{P}\Big( \exists t \in [T^{(n)}_{k - 1},T^{(n)}_{k}]\ \text{such that }X^{(n)}_t \notin \bigcup_{j:\ m_j \in G}\Omega_j \Big) < \frac{\epsilon}{ N(\epsilon)K(\epsilon) },
\end{align*}
hence for event
\begin{align*}
    A_2(n) \delequal{} \Big\{ X^{(n)}_t \in \bigcup_{j:\ m_j \in G}\Omega_j\ \forall t \in [0,T^{(n)}_{ N(\epsilon)K(\epsilon) }]   \Big\},
\end{align*}
we must have $\limsup_n \mathbb{P}_x\Big( \big(A_2(n)\big)^c \Big) \leq \epsilon$. To conclude the proof, note that
\begin{align*}
    A_1(n) \cap A_2(n) & \subseteq \{ J^{(n)}(\bar{t}) \leq N(\epsilon)K(\epsilon) \} \cap \Big\{ X^{(n)}_t \in \bigcup_{j:\ m_j \in G}\Omega_j\ \forall t \in [0,T^{(n)}_{ N(\epsilon)K(\epsilon) }]   \Big\}
    \\
    & =  \{ T^{(n)}_{ N(\epsilon)K(\epsilon) } \geq \bar{t} \} \cap \Big\{ X^{(n)}_t \in \bigcup_{j:\ m_j \in G}\Omega_j\ \forall t \in [0,T^{(n)}_{ N(\epsilon)K(\epsilon) }]   \Big\}
    \\
    & \subseteq \Big\{ X^{(n)}_t \in \bigcup_{j:\ m_j \in G}\Omega_j\ \forall t \in [0,\bar{t}]   \Big\}
\end{align*}
so we have established \cref{proof key convergence to MC goal}.
\end{proof}

\begin{proof}[Proof of Lemma \ref{lemma key converge to markov chain and always at wide basin transient case}]
From Lemma \ref{lemma preparation weak convergence to markov chain trainsient case}, one can see the existence of some $(\Delta_n)_{n \geq 1}$ with $\lim_{n}\Delta_n = 0$ such that \cref{lemma mark process converge to markov chain transient case goal 1} and \cref{lemma mark process converge to markov chain transient case goal 2} hold. For simplicity of notations, we let $\hat{X}^{\dagger,(n)} = \hat{X}^{\dagger,*,\eta_n,\Delta_n }, X^{\dagger,(n)} = X^{\dagger,*,\eta_n}$, and let $\bar{t} = t_{k^\prime}$ 

Combine \cref{lemma mark process converge to markov chain transient case goal 1} with Lemma \ref{lemma weak convergence of jump process}, and we immediately get \cref{lemma key converge to markov chain and always at wide basin transient case goal 1}. In order to prove \cref{lemma key converge to markov chain and always at wide basin transient case goal 2},
note that
\begin{align*}
    & \Big\{ X^{\dagger,(n)}_{t_k} \notin \bigcup_{j:\ m_j \in G^\text{large}}B(m_j,\Delta_n)\text{ and }X^{\dagger,(n)}_{t_k} \neq \boldsymbol{\dagger}  \Big\}
    \\
    = & \Big\{ X^{\dagger,(n)}_{t_k} \notin \bigcup_{j:\ m_j \in G^\text{large}}B(m_j,\Delta_n)\text{ and }X^{\dagger,(n)}_{s} \in \bigcup_{j: m_j \in G}\Omega_j\ \forall s \in [0,t_k] \Big\}
\end{align*}
so the conclusion of the proof follows directly from \cref{lemma mark process converge to markov chain transient case goal 2}.
\end{proof}

\section{First Exit Time of Truncated Heavy-tailed SGD in $\mathbb{R}^d$} \label{sec: appendix Rd result}

\subsection{Main Result}

The object of interests is the following truncated heavy-tailed SGD iterates
\begin{align*}
    X^{\eta}_{k+1}(x) &= X^{\eta}_k(x) + \varphi_b\big( -\eta \nabla f(X^{\eta}_k(x)) + \eta Z_{k+1} \big)\ \forall k \geq 0 
\end{align*}
where the initial condition is prescribed by $X_0(x) = x$, $f$ is a real-valued function on $\mathbb{R}^d$, $\eta > 0$ is the learning rate, $(Z_k)_{k \geq 1}$ is the sequence of heavy-tailed noises, the standard gradient clipping operator is $\varphi_b(v) = \min\{1, \frac{b}{\norm{v}}\}\cdot v$ with $\norm{\cdot}$ being $L_2$ norm. 
To ease notations, we also use $\mathbb{P}_x$ to denote the conditional law on $\{X_0^\eta = x\}$.

Specifically, we are interested in the first exit time of $X^\eta_n$ from a domain $\mathcal{G}$, i.e. the stopping time
$$ \sigma(\eta)  = \min\{ n \geq 0: X^\eta_n \notin \mathcal{G} \}.$$
We work with following assumptions.

\begin{assumption} \label{assumption 1 domain G}
The region $\mathcal{G}$ is connected, bounded and open, and $\textbf{0} \in \mathcal{G}$.
\end{assumption}

\begin{assumption} \label{assumption 2 vector field g}
The function $f$ is smooth, i.e. $f \in C^2(\mathbb{R}^d)$.
\end{assumption}

\begin{assumption} \label{assumption 3 boundary of domain G}
The boundary set $\partial \mathcal{G}$ is a $(n-1)$-dimensional manifold of class $C^2$ such that the vector field $n(\cdot)$ of the outer normals on $\partial \mathcal{G}$ exists with
\begin{align}
    \nabla f(v)^Tn(v) \geq c_0\ \ \  \forall v \in \partial \mathcal{G} \label{ineq for assumption on the outer normals}
\end{align}
for some constant $c_0 > 0$.
\end{assumption}

\begin{assumption} \label{assumption 4 jacobian of g at the origin}
For $\nabla^2 f(\bm{0})$, the Hessian of $f(\cdot)$ at point $\bm{0}$, all the eigenvalues are strictly positive.
\end{assumption}

For any $x \in \bar{\mathcal{G}}$, let $\bm{x}_t(x)$ be the ODE flow with $\bm{x}_0(x) = x$ solving
\begin{align*}
    \bm{\dot{x}}_t(x) & = -\nabla f\big( \bm{x}_t(x) \big)\ \ \forall t \geq 0, 
\end{align*}
 
\begin{assumption} \label{assumption 5 zero is an attractor}
W.L.O.G., the origin $\bm{0}$ is an attractor of the domain, i.e. $\nabla f(\bm{0}) = \bm{0}$ and $\textbf{0}$ is asymptotically stable in $\mathcal{G}$ in the sense that
\begin{align*}
    \lim_{t \rightarrow \infty} \textbf{x}_t(x) = \textbf{0}\ \ \ \forall x \in \bar{\mathcal{G}}.
\end{align*}
\end{assumption}

We have the following assumption regarding the heavy-tailed noises $(Z_n)_{n \geq 1}$. For any $x \in \mathbb{R}^d, x \neq 0$, define the following polar transformation
\begin{align}
    \textbf{T}(x) = (\|x\|, x/\|x\|) \label{def_PolarTransformation}
\end{align} 
with $\textbf{T}_r(x) = \norm{x}, \textbf{T}_\theta(x) = x/\norm{x}$. Also, let $\mathbb{O} = \{\bm{0}\}$

\begin{assumption} \label{assumption 6 regular variation of the Rd noise}
$\mathbb{E}Z_1 = \textbf{0}$.  
Besides, there exist a positive integer $m \geq 1$, a sequence $1 < \alpha_1 < \alpha_2 < \cdots < \alpha_m < \infty$, a sequence of slowly-varying functions $(l_1,\cdots,l_m)$, a sequence of probability measures $(S_1,\cdots,S_m)$ on the unit sphere $\mathbb{S}^{d-1}$ with support $F_j \delequal{} \text{supp}(S_j)$ as closed sets on $\mathbb{S}^{d-1}$ such that 
\begin{itemize}
    \item $F_i \cap F_j = \emptyset$ for any $i \neq j$;
    \item For any $j \in [m]$, define the cone $E_j \delequal{} \textbf{T}_\theta^{-1}(F_j) \cup \mathbb{O}$ in $\mathbb{R}^d$ and measure $\mathbb{P}^{(j)}(\cdot) = \mathbb{P}(Z_1 \in \cdot \cap E_j)$, we have
    \begin{align}
        t^{\alpha_j}\cdot l_j(t)\cdot \mathbb{P}^{(j)}\circ \textbf{T}^{-1}\big(t\cdot dr \times d\theta\big) \rightarrow \nu_{\alpha_j}(dr) \times S_j(d\theta)\ \ \ \text{as }t \rightarrow \infty
    \end{align}
    in the sense of $\mathbb{M}( E_j \symbol{92} \mathbb{O} )$. Here $\nu_\alpha$ is a Borel measure defined on $(0,\infty)$ satisfying $\nu_\alpha[t,\infty) = t^{-\alpha}\ \ \forall t > 0$;
    \item For measure $\mathbb{P}^{(0)}(\cdot) \delequal{} \mathbb{P}\big(Z_1 \in \cdot\ \symbol{92}(\cup_{j = 1}^m E_j)\big)$ and any $\alpha > 0$, 
    \begin{align}
        t^{\alpha}\cdot \mathbb{P}^{(0)}\circ \textbf{T}^{-1}\big(t\cdot dr \times d\theta\big) \rightarrow 0\ \ \ \text{as }t \rightarrow \infty
    \end{align}
    in the sense of $\mathbb{M}( \mathbb{R}^d \symbol{92} \mathbb{O} )$;
    \item For any $j \in [m]$, the measure $S_j$ is absolutely continuous w.r.t. spherical measure on $\mathbb{S}^{d-1}$.
\end{itemize}
\end{assumption}
A function $l:\mathbb{R}_+ \mapsto \mathbb{R}_+$ is slowly varying (at $\infty$) if $\lim_{t\rightarrow \infty}l(tx)/l(x) = 1$ holds for any $x>0$. For details on $\mathbb{M}-$convergence and regular variation in general metric spaces, see \cite{lindskog2014regularly}. Here we state one implication of the assumption. For any $j = 0,1,\cdots,m$, let
\begin{align}
    H_j(x) & \delequal{} \mathbb{P}_j\Big( \{\bm{y} \in \mathbb{R}^d:\ \|\bm{y}\| \geq x\} \Big). \label{def H function Rd}
\end{align}
The assumption above immediately implies that, for any $j = 1,\cdots,m$, $H_j(\cdot)$ is regularly varying (at $\infty$) with index $-\alpha_j$ (denoted as $H_j \in RV_{-\alpha_j}$), namely $\lim_{t \rightarrow \infty}H_j(tx)/H_j(t) = x^{-\alpha_j}\ \ \forall x > 0.$ 
Meanwhile, for any $\alpha > 0$, $H_0(x) = o(1/x^\alpha)$ as $x \rightarrow \infty$. 

The asymptotically behavior of the first exit time hinges on the following geometric characterization of the domain $\mathcal{G}$. 
For any integer $k \geq 1$, 
a sequence of strictly positive real numbers $(t_2,\cdots,t_k)$ with $t_1 = 0$, 
a sequence of non-zero vectors $(w_1,\cdots,w_k)$
and some $\eta > 0$,
let $\bm{t}^{(k)} = (t_1,\cdots,t_k), \bm{w}^{(k)} = (w_1,\cdots,w_k)$, and define the ODE path with $k$ jumps (clipped at size $b$) by
$\bm{w}^{(k)} = (w_1,\cdots,w_k)$, we define perturbed ODE path $\widetilde{\bm{x}}^{\eta}$ as 
\begin{align*}
    \widetilde{\bm{x}}^{\eta}(0;\bm{t}^{(k)},\bm{w}^{(k)}) & = \varphi_b(\eta w_1); \\
    \frac{d \widetilde{\bm{x}}^{\eta}(t,x;\bm{t}^{(k)},\bm{w}^{(k)}) }{ dt } & = -\eta \nabla  f\big( \widetilde{\bm{x}}^{\eta}(t,x;\bm{t}^{(k)},\bm{w}^{(k)}) \big)\ \ \ \forall t \notin \{t_1,t_1 + t_2,\cdots,\sum_{j = 1}^k t_{j}\}  \\
    \widetilde{\bm{x}}^{\eta}(t,x;\bm{t}^{(k)},\bm{w}^{(k)}) & =  \widetilde{\bm{x}}^{\eta}(t-,x;\bm{t}^{(k)},\bm{w}^{(k)}) + \varphi_b(\eta w_j) \ \ \ \ \text{if }t =\sum_{i = 1}^jt_i \text{ for some }j
\end{align*}
with the convention that $g(t-) \delequal{} \lim_{s \uparrow t}g(s)$ for any function $g$. Also, when $\eta = 1$ we simply write $\widetilde{\bm{x}}$.
Now we can assign a \textit{cost} to each jump $w_j$ based on the direction using the following function:
\begin{align}
J(w) = \begin{cases} 
\alpha_j - 1 & \text{if }\ w \neq 0, w \in E_j,\\
\infty & \text{otherwise. }
\end{cases} \label{def jump cost function J}
\end{align}
Given any set of perturbations described by $\bm{t}^{(k)} \in \{0\} \times \mathbb{R}^{k-1}_+, \bm{w}^{(k)} \in \big( \mathbb{R}^d \symbol{92} \mathbb{O} \big)^k$, we identify the \textit{destination} of the flow as
\begin{align*}
    h(k; \bm{t}^{(k)}, \bm{w}^{(k)} ) = \widetilde{\bm{x}}( \sum_{j = 1}^k t_k ;\bm{t}^{(k)},\bm{w}^{(k)}).
\end{align*}
Besides, define function $\mathcal{I}( \bm{w}^{(k)} ) = (i_1,\cdots,i_m)$ if $\#\{i \in [k]:\ w_i \in E_j \} = i_j$ for all $j \in [m]$. This allows us to define the following configuration sets
\begin{align}
    \mathcal{A}( i_1,\cdots,i_m ) \delequal{} \big\{\bm{w}^{(k)} \in \big( \mathbb{R}^d \symbol{92} \mathbb{O} \big)^k:\ \mathcal{I}( \bm{w}^{(k)} ) = ( i_1,\cdots,i_m ),\ k = \sum_{j = 1}^m i_j   \big\} \label{def configuration set A}
\end{align}
for any $(i_1,\cdots,i_m)\in \mathbb{N}^m$, i.e., some set of jumps $\bm{w}^{(k)}$ is said to have configuration $( i_1,\cdots,i_m )$ or belong to the configuration set $\mathcal{A}( i_1,\cdots,i_m )$ if the number of jumps in cone $E_j$ is equal to $i_j$.
We can also define the cost for each configuration as
\begin{align}
    \mathcal{J}(i_1,\cdots,i_m) = \sum_{j = 1}^m(\alpha_j - 1) i_j.
\end{align}
Now we can characterize the minimum \textit{cost} to exit $\mathcal{G}$:
\begin{align}
    J_\mathcal{G} \delequal{}& \min\{ \sum_{j = 1}^k J(w_j):\ \exists k \in \mathbb{N}, \bm{t}^{(k)} \in \{0\} \times \mathbb{R}^{k-1}_+, \bm{w}^{(k)} \in \big( \mathbb{R}^d \symbol{92} \mathbb{O} \big)^k \nonumber 
    \\
    & \ \ \ \ \ \ \ \ \ \ \ \ \ \ \ \ \ \ \ \ \ \text{ s.t. } \widetilde{\bm{x}}( \sum_{j = 1}^k t_k, \bm{0} ;\bm{t}^{(k)},\bm{w}^{(k)}) \notin \mathcal{G} \}. \label{def minimum cost J_G}
\end{align}
From the boundedness of $\mathcal{G}$, one can see that $0 < J_\mathcal{G} < \infty$ regardless of the actual value of the clipping threshold $b > 0$. 
We need the following technical assumption regarding the configurations of jumps that can trigger the exit with the minimum cost.

\begin{assumption} \label{assumption unique configuration for minimum exit cost}
There exists only one array $(i_1,\cdots,i_m) \in \mathbb{N}^m$ such that $\sum_{j = 1}^m i_j(\alpha_j-1) = J_\mathcal{G}$ and (for $k = \sum_{j = 1}^m i_m$)
\begin{align*}
\exists \bm{t}^{(k)} \in \{0\} \times \mathbb{R}^{k-1}_+, \bm{w}^{(k)} \in \mathcal{A}( i_1,\cdots,i_m )\text{ s.t. } \widetilde{\bm{x}}( \sum_{j = 1}^k t_k,\bm{0} ;\bm{t}^{(k)},\bm{w}^{(k)}) \notin \mathcal{G}.
\end{align*}
\end{assumption}

We use $\bm{i}^* = (i^*_1,\cdots,i^*_m)$ to denote the unique configuration in Assumption \ref{assumption unique configuration for minimum exit cost} and $k^* = \sum_{j = 1}^m i^*_j$. The implication is that, for a set of jumps $\bm{t}^{(k)},\bm{w}^{(k)}$ with \emph{any other configuration}, one of the following must happen: (i) this set of jumps has a cost strictly higher than $J_G$; (ii) this set of jumps cannot send the ODE flow out of $\mathcal{G}$.

Meanwhile, we introduce the following concept as the \textit{coverage} of a certain configuration set.
\begin{align}
    & \mathcal{G}( i_1,\cdots,i_m ) \delequal{} \big\{\widetilde{\bm{x}}( s, \bm{0} ;\bm{t}^{(k)},\bm{w}^{(k)}): \nonumber 
    \\
    &\ \ \ \ \  k = \sum_{j = 1}^m i_j,\ \bm{t}^{(k)} \in \{0\} \times \mathbb{R}^{k-1}_+, \bm{w}^{(k)} \in \mathcal{A}( i_1,\cdots,i_m ),\ s\in[0,\sum_{j=1}^k t_j] \big\}.
\end{align}
It is easy to see that, for any configuration $(i_1,\cdots,i_m)$ with cost $\mathcal{J}(i_1,\cdots,i_m) \leq J_\mathcal{G}$, the coverage $\mathcal{G}( i_1,\cdots,i_m )$ is a closed set, and the following technical assumption holds for (Lebesgue) almost every $b>0$. Here the distance $\bm{d}(A,B) = \inf_{x \in A, y \in B}\norm{x-y}$ for any $A,B \subseteq \mathbb{R}^d$, and $A^\circ$ is the interior of the set $A$.
\begin{assumption} \label{assumption bounded away}
For any $(i_1,\cdots,i_m) \in \mathbb{N}^m$, one of the following must occur:
\begin{itemize}
    \item $\bm{d}\big(\mathcal{G}( i_1,\cdots,i_m ),\mathcal{G}^c) > 0$;
    \item $\big(\mathcal{G}( i_1,\cdots,i_m ) \cap \mathcal{G}^c\big)^\circ \neq \emptyset$
\end{itemize}
\end{assumption}

Note that with $\bm{i}^*$ and $k^*$ defined above, we can define the following mapping $h$ from $\bm{r} = (r_1,\cdots,r_{k^*}), \bm{\theta} = (\theta_1,\cdots,\theta_{k^*}) \in (\mathbb{S}^{d-1})^{k^*}, \bm{t} = (t_1,\cdots,t_{k^*}) \in \{0\}\times\mathbb{R}^{k^*}_+$ such that
\begin{align*}
    h(\bm{r},\bm{\theta},\bm{t}) = \widetilde{\bm{x}}(\sum_{i = 1}^{k^*}t_i, \bm{0};\ \bm{t},\bm{w}) 
\end{align*}
where $\bm{w} = (w_i)_{i = 1}^{k^*}$ with $w_i = r_i\theta_i$.
Also, we introduce the concept of \emph{type} for configuration $\bm{i}^*$. Specifically, define
\begin{align*}
\bm{j}(i_1,\cdots,i_m) \delequal \big\{ (j_1,\cdots,j_{k^*}) \in \{ 0,1,2,\cdots,m \}^{k^*}:\ \#\{n:\ j_n = k\} = i_k\ \forall k \in [m]  \big\} 
\end{align*}
and for any $\bm{j} \in \bm{j}(\bm{i}^*)$, we say that $\bm{w} = (w_1,\cdots,w_{k^*}) \in \mathcal{A}(\bm{i}^*)$ has type $\bm{j}$ if
\begin{align*}
    w_i \in E_{j_i}\ \forall i \in [k^*].
\end{align*}
In other words, based on the direction of each jump in $\bm{w} \in \mathcal{A}(\bm{i}^*)$ we group them into different types. Note that $| \bm{j}(\bm{i}^*) |<\infty$.
Now for any type $\bm{j} \in \bm{j}(\bm{i}^*)$, define a (Borel) measure on $\mathbb{R}^{k^*}_+ \times ( \mathbb{S}^{d-1} )^{k^*} \times \mathbb{R}^{k^* - 1}_+$ as
\begin{align*}
    \mu_{\bm{j}}\delequal (\prod_{i = 1}^{k^*}\nu_{\alpha_{\bm{j}_i}}) \times ( \prod_{i=1}^{k^*} S_{ \bm{j}_i } ) \times \bm{m}_\text{Leb}^{k^*-1} 
\end{align*}
where, for any $\alpha > 0$, the measure $\nu_\alpha$ is the Borel measure on $(0,\infty)$ with $\nu_\alpha(x,\infty) = 1/x^{1+\alpha}$.
Lastly, define measure $\mu$ as
$\mu = \sum_{ \bm{j} \in \bm{j}(\bm{i}^*) }\mu_{\bm{j}}.$
As will be established in Lemma \ref{lemma boundary set zero mass}, the following technical assumption is also a very moderate one since it holds for (Lebesgue) almost every $b > 0$ under the current setting.
\begin{assumption} \label{assumption boundary set with zero mass}
The set $h^{-1}\big( \partial \mathcal{G}\big)$ has zero mass under the measure $\mu$.
\end{assumption}

Having specified the problem setting, we are now ready to present Theorem \ref{thm main result Rd}, the main result of this section.
The implication of the theorem is clear:
under proper scaling, the first exit time $\sigma(\eta)$ converges in distribution to an Exponential random variable.
Moreover, the scaling $\lambda(\eta)$ is roughly of order $\eta^{1 + J_\mathcal{G}}$, implying that the first exit time $\sigma(\eta)$ is roughly of order $1/\eta^{1 + \mathcal{G}}$ as the learning rate $\eta$ approaches $0$.

\begin{theorem} \label{thm main result Rd}
Let Assumptions \ref{assumption 1 domain G}-\ref{assumption boundary set with zero mass} hold.
There exists a function $\lambda(\eta)$ that is regularly varying (as $\eta \downarrow 0$) with index $(1 + J_\mathcal{G})$ such that,
for any $x \in \mathcal{G}$ and any $t > 0$,
\begin{align*}
    \lim_{\eta \downarrow 0}\mathbb{P}_x( \sigma(\eta)\lambda(\eta) > t ) = \exp(-qt)
\end{align*}
where the constant $q = \mu\big( h^{-1}(\mathcal{G}^c) \big)$.
\end{theorem}

The proof is provided in Section \ref{sec: proof Rd result}.
As a concluding remark,
we stress that the order of the first exit time is dictated by $J_\mathcal{G}$, the minimum cost for exit we introduced above.
Bearing obvious similarity to the first exit time analysis of SDE driven by heavy-tailed L\'evy processes in \cite{imkeller2010first},
our result can be viewed as a natural generalization when gradient clipping is applied and heavy-tailed noises may not always align with a finite number of lines.
In \cite{imkeller2010first}, the unclipped setting implies that the escape from the domain $\mathcal{G}$ can always be achieved with one big jump.
However, a big perturbation (in noise) along different directions may correspond to different heavy-tailed indices, i.e., induce different costs given our definition of cost function $J(w)$ in \cref{def jump cost function J}.
In Theorem 1 of \cite{imkeller2010first}, we see that the order of the first exit time is determined by $\alpha_1$, the smallest heavy-tailed index.
The underlying reason is that the exit is almost always trigger by a single big jump with the smallest cost.
Similarly, in our setting where multiple jumps are required for escape due to the clipping mechanism,
we see that the order of the first exit time is not dictated by the number of jumps or the accumulated distances of the jumps, but the smallest possible accumulated \emph{cost} defined as the summation of $J(w_i)$ where $w_i$'s are the jumps the lead to escape from $\mathcal{G}$.
As detailed in the proof, this is because such jumps with the smallest costs $J_\mathcal{G}$ dictates the most likely way for exiting $\mathcal{G}$.

\subsection{A Special Case: Uniform Heavy Tail Index Along All Directions}

As stated above,
Theorem \ref{thm main result Rd} deals with general case where the noise distribution is allowed to have different heavy-tailed indices along different directions (see Assumption \ref{assumption 6 regular variation of the Rd noise}).
As a special case, it is worth noticing that if a single heavy-tailed index $\alpha$ can be used to describe the tail behavior of noises along any direction (this can be easily guaranteed if the heavy-tailed noise is manually injected into a light-tailed setting),
then the minimum cost $J_\mathcal{G}$ will be equal to $l^*(\alpha-1)$ where $l^*$ is the minimum number of jumps required for escape.
The readers can see that this is a natural extension of our $\mathbb{R}^1$ results,
implying that the same strong preference for ``wide'' minima still hold in $\mathbb{R}^d$ under truncated heavy-tailed SGD.
To be specific, we work with the following assumption about the noise distribution.
\begin{assumption} \label{assumption uniform alpha regular variation of the Rd noise}
$\mathbb{E}Z_1 = \textbf{0}$.  
Besides, 
there exists some $\alpha > 1$,
a slowly-varying functions $l$,
a probability measures $S$ on the unit sphere $\mathbb{S}^{d-1}$ with support $\text{supp}(S_j) = \mathbb{S}^{d-1}$ such that 
\begin{itemize}
    \item For the measure $\nu(\cdot) = \mathbb{P}(Z_1 \in \cdot \cap E_j)$, we have
    \begin{align}
        t^{\alpha}\cdot l(t)\cdot \nu\circ \textbf{T}^{-1}\big(t\cdot dr \times d\theta\big) \rightarrow \nu_{\alpha}(dr) \times S(d\theta)\ \ \ \text{as }t \rightarrow \infty
    \end{align}
    in the sense of $\mathbb{M}( \mathbb{R}^d \symbol{92} \mathbb{O} )$. Here $\nu_\alpha$ is a Borel measure defined on $(0,\infty)$ satisfying $\nu_\alpha[t,\infty) = t^{-\alpha}\ \ \forall t > 0$;
    \item The measure $S$ is absolutely continuous w.r.t. spherical measure on $\mathbb{S}^{d-1}$.
\end{itemize}
\end{assumption}
Compared to Assumption \ref{assumption 6 regular variation of the Rd noise},
one can easily see that Assumption \ref{assumption uniform alpha regular variation of the Rd noise} is a stronger version with the specific proviso that a single index $\alpha$ can describe the tail of the noise distribution along any direction in $\mathbb{R}^d$.
The first exit time results now admit a simplified form.
In particular, the cost of jumps will degenerate to the \emph{count} of jumps in the sense that
$\sum_{i = 1}^k J(w_i) = k(\alpha-1) $
for any $k \geq 1$ and any $(w_1,\cdots,w_k) \in (\mathbb{R}^d\symbol{92}\mathbb{O})^k$.
Moreover, we now have $J_\mathcal{G} = l^*_\mathcal{G}\cdot (\alpha-1)$ where
\begin{align*}
    l^*_\mathcal{G} \delequal{} \min\{ k \in \mathbb{N}:\ \exists \bm{t}^{(k)} \in \{0\} \times \mathbb{R}^{k-1}_+, \bm{w}^{(k)} \in \big( \mathbb{R}^d \symbol{92} \mathbb{O} \big)^k\text{ s.t. } \widetilde{\bm{x}}( \sum_{j = 1}^k t_k, \bm{0} ;\bm{t}^{(k)},\bm{w}^{(k)}) \notin \mathcal{G} \} 
\end{align*}
and the measure $\mu$ can now be expressed as
$  \mu =  (\prod_{i = 1}^{l^*_\mathcal{G}}\nu_{\alpha}) \times ( \prod_{i=1}^{l^*_\mathcal{G}} S ) \times \bm{m}_\text{Leb}^{l^*_\mathcal{G} - 1}$.
Therefore, the following theorem is merely a restatement of Theorem \ref{thm main result Rd} in this simplified setting.
Still, we present the result to highlight the role of the minimum jump number $J_\mathcal{G}$ in the first exit time.
\begin{theorem} \label{thm main result Rd simplified 1}
Let Assumptions \ref{assumption 1 domain G}-\ref{assumption 5 zero is an attractor}, \ref{assumption uniform alpha regular variation of the Rd noise} and \ref{assumption unique configuration for minimum exit cost}-\ref{assumption boundary set with zero mass} hold.
There exists a function $\lambda(\eta)$ that is regularly varying (as $\eta \downarrow 0$) with index $1 + l^*_\mathcal{G}(\alpha-1)$ such that,
for any $x \in \mathcal{G}$ and any $t > 0$,
\begin{align*}
    \lim_{\eta \downarrow 0}\mathbb{P}_x( \sigma(\eta)\lambda(\eta) > t ) = \exp(-qt)
\end{align*}
where the constant $q = \mu\big( h^{-1}(\mathcal{G}^c) \big)$.
\end{theorem}
It is clear that first exit time since the first exit time is roughly of order $1/\eta^{1 + (\alpha-1) l^*_\mathcal{G}}$. Therefore, even in the general $\mathbb{R}^d$ case, the quantity $l^*_\mathcal{G}$, i.e. the minimum count of jumps to escape from a domain $\mathcal{G}$, induces a hierarchy of first exit time as the first exit time from the domain with largest $l^*_\mathcal{G}$ (requiring most number of jumps to escape) will dominate the first exit time from other regions.

Lastly, if we work under the standard setting such as the ones in \cite{zhou2020towards} where local strong convexity of $f$ in the domain $\mathcal{G}$ is assumed,
the results can be further simplified and the minimum jump number $l^*_\mathcal{G}$ will be directly tied to the \emph{width} of the each domain.
Specifically, 
the quantity $r_\mathcal{G} \delequal \sup_{y \in \mathcal{G}}\norm{y}$ can be interpreted the effective width or radius of $\mathcal{G}$,
and we work with the following assumption.
\begin{assumption}\label{assumption convexity of f in G}
The function $f$ is strongly convex on the closed ball $\{y \in \mathbb{R}^d: \norm{y} \leq r_\mathcal{G} \}$.
\end{assumption}
Two consequences follow immediately from this assumption.
First of all, there exists some constant $c > 0$ such that
$$ \norm{\bm{x}_t(x)} \leq \norm{x}e^{-ct}\ \ \forall x \in \{y:\ \norm{y} < r_\mathcal{G}\}.$$
As a result, we must have $l^*_\mathcal{G} \geq r_\mathcal{G}/b.$
Next, as long as $r_\mathcal{G}/b$ is not an integer (which holds for Lebesgue almost every $b > 0$), for $k = \ceil{r_\mathcal{G}/b}$ we can find $t_2 > 0,\cdots,t_{k} > 0$ and $w_1 \neq \bm{0}, \cdots, w_k \neq \bm{0}$ where $\widetilde{\bm{x}}\big(s,\bm{0};(0,t_2,\cdots,t_k),(w_1,\cdots,w_k)\big) \notin \mathcal{G}$ for some $s > 0$.
In fact, it is worth noticing that Assumption \ref{assumption bounded away} now degenerates to the condition that $r_\mathcal{G}/b$ is not an integer,
and we now know that for any such $b > 0$, we have $l^*_\mathcal{G} \leq \ceil{r_\mathcal{G}/b}.$
In summary, we have established the following result indicating that the strong preference for wider minima under truncated heavy-tailed SGD still persists in $\mathbb{R}^d$ given proper convexity assumption on $f$.
\begin{theorem} \label{thm main result Rd simplified 2}
Let Assumptions \ref{assumption 1 domain G}-\ref{assumption 5 zero is an attractor}, \ref{assumption uniform alpha regular variation of the Rd noise}-\ref{assumption convexity of f in G} and \ref{assumption unique configuration for minimum exit cost}-\ref{assumption boundary set with zero mass} hold.
There exists a function $\lambda(\eta)$ that is regularly varying (as $\eta \downarrow 0$) with index $1 + l^*(\alpha-1)$ such that,
for any $x \in \mathcal{G}$ and any $t > 0$,
\begin{align*}
    \lim_{\eta \downarrow 0}\mathbb{P}_x( \sigma(\eta)\lambda(\eta) > t ) = \exp(-qt)
\end{align*}
where the constant $q = \mu\big( h^{-1}(\mathcal{G}^c) \big)$ and $l^* = \ceil{ r_\mathcal{G}/b }$.
\end{theorem}

\subsection{Relaxing the Technical Assumptions}

In order to achieve the tightest characterization of the first exit time, some assumptions introduced above are slightly stronger than the ones in \cite{imkeller2010first}.
For instance, in Assumption \ref{assumption 3 boundary of domain G} we require $\partial \mathcal{G}$ to be of class $C^2$ while the assumption (A3) in  \cite{imkeller2010first} only requires it to be a $C^1$ manifold.
Besides, Assumption \ref{assumption 6 regular variation of the Rd noise} requires that $S_j$, the limiting distribution of the directions of each heavy-tailed component in noises, to be absolutely continuous w.r.t. the spherical measure.
As will be stressed in Section \ref{subsec: proof of lemma zero mass on boundary set}, these conditions are only imposed to prove Lemma \ref{lemma boundary set zero mass},
thus ensuring that Assumption \ref{assumption boundary set with zero mass} holds for almost every $b > 0$.
Briefly speaking, all the efforts to guarantee Assumption \ref{assumption boundary set with zero mass} allows us to conclude that the law of scaled first exit time $\lambda(\eta)\sigma(\eta)$ converges exactly to that of $Exp(q)$ with $q = \mu\big(h^{-1}( \mathcal{G}^c )\big)$.

Fortunately, the discussion below will show that, even when we relax all the extra technical assumptions above (hence removing Assumption \ref{assumption boundary set with zero mass}),
the order of the first exit time is still dictated by the minimum cost for escape $J_\mathcal{G}$,
and a similar result about the first exit time can be obtained
where the distribution of the scaled first exit time $\lambda(\eta)\sigma(\eta)$ will be asymptotically bounded by two Exponential RVs.

Specifically, we reiterate that the first half of Assumption \ref{assumption 3 boundary of domain G}, i.e. $\partial \mathcal{G}$ is a differential manifold, will only be applied to prove Lemma \ref{lemma boundary set zero mass}.
The second half of the proof, namely the lower bound in \cref{ineq for assumption on the outer normals},
only serves the ensure that for any $x \in \mathcal{G}$, we have
$ \bm{x}_t(x) \in \mathcal{G}\ \ \forall t \geq 0 $
so that the gradient flow starting in $\mathcal{G}$ will never leave this domain.
Now let us focus on the open sets $\mathcal{G}^\epsilon$ and $\mathcal{G}_\epsilon$ for some small $\epsilon > 0$ where
\begin{align*}
    \mathcal{G}^\epsilon \delequal \{y\in\mathbb{R}^d:\ \bm{d}(y,\mathcal{G}) < \epsilon \},\ \ \ \mathcal{G}_\epsilon \delequal \{ y \in \mathcal{G}:\ \bm{d}(y,\mathcal{G}^c)>\epsilon \}.
\end{align*}
Note that $\cap_\epsilon \mathcal{G}^\epsilon = \overline{\mathcal{G}},\ \cup_\epsilon \mathcal{G}_\epsilon = \mathcal{G}$ and for any positive reals $\epsilon_1, \epsilon_2$ such that $\epsilon_1 \neq \epsilon_2$ and the two reals are small enough for $G_{\epsilon_1}$ and $G_{\epsilon_2}$ to be non-empty,
$$ \partial\mathcal{G}^{\epsilon_1} \cap \partial \mathcal{G}^{\epsilon_2} = \emptyset,\ \ \ \ \partial \mathcal{G}_{\epsilon_1} \cap \partial \mathcal{G}_{\epsilon_2} = \emptyset.$$
Given \cref{ineq for assumption on the outer normals}, as well as the smoothness of the vector field $\nabla f$ (see Assumption \ref{assumption 2 vector field g}) and the fact that $\mathcal{G}$ is connected, bounded and open and contains a unique attractor $\bm{0}$ (see Assumption \ref{assumption 1 domain G} and \ref{assumption 5 zero is an attractor}),
it is easy to see that for all $\epsilon > 0$ sufficiently small, we will have $ \bm{x}_t(x) \in \mathcal{G}^\epsilon \ \ \forall t \geq 0 $ for any $x \in \mathcal{G}^\epsilon $ and $ \bm{x}_t(x) \in \mathcal{G}_\epsilon \ \ \forall t \geq 0 $ for any $x \in \mathcal{G}_\epsilon$.

Meanwhile, in our proof below we will establish \cref{property bar delta bar t}
(which, again, does not require the $C^2$ class assumption on $\partial \mathcal{G}$ or the absolute continuity assumption on the measures $S_j$).
Note that \cref{property bar delta bar t} implies the existence of $\Delta > 0$ such that $\mu\big( h^{-1}( \mathcal{G}_{\Delta} ) \big) < \infty$.
Therefore, for all but only countably many $\epsilon \in (0,\Delta)$, we have $\mu\big( h^{-1}( \partial \mathcal{G}^\epsilon ) \big) =  \mu\big( h^{-1}( \partial \mathcal{G}_\epsilon ) \big) = 0$.
In other words, we can find a sequence of $\epsilon_n$ with $\lim_n\epsilon_n = 0$ such that 
$$\mu\big( h^{-1}( \partial \mathcal{G}^{\epsilon_n} ) \big) =  \mu\big( h^{-1}( \partial \mathcal{G}_{\epsilon_n} ) \big) = 0\ \ \forall n \geq 1$$.
Also, for a fixed $b > 0$ that ensures Assumption \ref{assumption unique configuration for minimum exit cost} and \ref{assumption bounded away}, one can see that similar conditions will also hold for sets $\mathcal{G}^\epsilon$ and $\mathbb{G}_\epsilon$ as long as $\epsilon > 0$ is sufficiently small.
In summary, if we consider first exit times
\begin{align*}
    \sigma^n(\eta) \delequal \min\{k \geq 0:\ X^\eta_k \notin \mathcal{G}^{\epsilon_n}\},\ \ \ \sigma_n(\eta) \delequal \min\{k \geq 0:\ X^\eta_k \notin \mathcal{G}_{\epsilon_n}\}
\end{align*}
with the obvious bounds that $\sigma_n(\eta) \leq \sigma(\eta) \leq \sigma^n(\eta)$ for all $n \geq 1$,
then it suffices to apply Theorem \ref{thm main result Rd} directly onto $\sigma_n(\eta), \sigma^n(\eta)$.
Note that we did not attempt to establish that $\partial \mathcal{G}^\epsilon$ or $\partial \mathcal{G}_\epsilon$ are differential manifolds.
However, we reiterate that this is not needed since the $C^2$ manifold condition in Assumption \ref{assumption 3 boundary of domain G} only serves to ensure that Assumption \ref{assumption boundary set with zero mass} holds for almost surely every $b > 0$, which we sidestep by picking a proper sequence $\epsilon_n$ to explicitly satisfy the condition. 
More formally speaking, consider the following relax ted assumptions.
\begin{assumption} \label{assumption relaxed boundary of domain G}
The boundary set $\partial \mathcal{G}$ is a $(n-1)$-dimensional manifold of class $C^1$ such that the vector field $n(\cdot)$ of the outer normals on $\partial \mathcal{G}$ exists with
\begin{align*}
    \nabla f(v)^Tn(v) \geq c_0\ \ \  \forall v \in \partial \mathcal{G} 
\end{align*}
for some constant $c_0 > 0$.
\end{assumption}
\begin{assumption} \label{assumption relaxed regular variation of the Rd noise}
$\mathbb{E}Z_1 = \textbf{0}$.  
Besides, there exist a positive integer $m \geq 1$, a sequence $0 < \alpha_1 < \alpha_2 < \cdots < \alpha_m < \infty$, a sequence of slowly-varying functions $(l_1,\cdots,l_m)$, a sequence of probability measures $(S_1,\cdots,S_m)$ on the unit sphere $\mathbb{S}^{d-1}$ with support $F_j \delequal{} \text{supp}(S_j)$ as closed sets on $\mathbb{S}^{d-1}$ such that 
\begin{itemize}
    \item $F_i \cap F_j = \emptyset$ for any $i \neq j$;
    \item For any $j \in [m]$, define the cone $E_j \delequal{} \textbf{T}_\theta^{-1}(F_j) \cup \mathbb{O}$ in $\mathbb{R}^d$ and measure $\mathbb{P}^{(j)}(\cdot) = \mathbb{P}(Z_1 \in \cdot \cap E_j)$, we have
    \begin{align*}
        t^{\alpha_j}\cdot l_j(t)\cdot \mathbb{P}^{(j)}\circ \textbf{T}^{-1}\big(t\cdot dr \times d\theta\big) \rightarrow \nu_{\alpha_j}(dr) \times S_j(d\theta)\ \ \ \text{as }t \rightarrow \infty
    \end{align*}
    in the sense of $\mathbb{M}( E_j \symbol{92} \mathbb{O} )$. Here $\nu_\alpha$ is a Borel measure defined on $(0,\infty)$ satisfying $\nu_\alpha[t,\infty) = t^{-\alpha}\ \ \forall t > 0$;
    \item For measure $\mathbb{P}^{(0)}(\cdot) \delequal{} \mathbb{P}\big(Z_1 \in \cdot\ \symbol{92}(\cup_{j = 1}^m E_j)\big)$ and any $\alpha > 0$, 
    \begin{align*}
        t^{\alpha}\cdot \mathbb{P}^{(0)}\circ \textbf{T}^{-1}\big(t\cdot dr \times d\theta\big) \rightarrow 0\ \ \ \text{as }t \rightarrow \infty
    \end{align*}
    in the sense of $\mathbb{M}( \mathbb{R}^d \symbol{92} \mathbb{O} )$.
\end{itemize}
\end{assumption}
We say Assumptions \ref{assumption 1 domain G}-\ref{assumption 2 vector field g}, \ref{assumption relaxed boundary of domain G}, \ref{assumption 4 jacobian of g at the origin}-\ref{assumption 5 zero is an attractor}, \ref{assumption relaxed regular variation of the Rd noise}, \ref{assumption unique configuration for minimum exit cost}-\ref{assumption bounded away} are the set of \emph{relaxed assumptions}.
The discussion above implies that the following result is an immediate consequence from Theorem \ref{thm main result Rd}.

\begin{theorem} \label{thm main result Rd realxed}
Let the relaxed assumptions hold.
There exists a function $\lambda(\eta)$ that is regularly varying (as $\eta \downarrow 0$) with index $(1 + J_\mathcal{G})$ such that,
for any $x \in \mathcal{G}$ and any $t > 0$,
\begin{align*}
    \exp(-qt) \leq \liminf_{\eta \downarrow 0}\mathbb{P}_x( \sigma(\eta)\lambda(\eta) > t ) \leq \limsup_{\eta \downarrow 0}\mathbb{P}_x( \sigma(\eta)\lambda(\eta) > t ) \leq \exp(-q^\uparrow t)
\end{align*}
with $q = \mu\big( h^{-1}(\mathcal{G}^c) \big), q^\uparrow = \mu\Big( h^{-1}\big((\overline{\mathcal{G}})^c\big) \Big)$.
\end{theorem}

\section{Proof of Theorem \ref{thm main result Rd}}\label{sec: proof Rd result}

This section is devoted to establishing Theorem \ref{thm main result Rd}.
For clarity of the exposition, we break down the proof into several major steps, each contained in a subsection below.

\subsection{Picking constants $\bar{t},\bar{\epsilon},\bar{\delta}$ }
The goal of this subsection is to fix several important constants that characterize the typical sizes and inter-arrival times of large perturbations in SGD trajectory that can cause the escape from domain $\mathcal{G}$.
We will abuse the notations slightly when referencing certain constants, and quantities such as $c_0, c_1, c_2$ may not be equal to the ones used in assumptions above.

First, from Assumption \ref{assumption 6 regular variation of the Rd noise} and the definition of $J_\mathcal{G}$, we can define
\begin{align}
    l^* = \ceil{J_\mathcal{G}/(\alpha_1-1)} + 1\ \ \label{def l star Rd}
\end{align}
and we must have $\infty > l^* > k^*$. Meanwhile, from Assumption \ref{assumption bounded away}, we can find $\bar{\epsilon} > 0$ such that
\begin{align}
    \bm{d}\big(\mathcal{G}( i_1,\cdots,i_m ),\mathcal{G}^c) > 100l^*\bar{\epsilon} \label{ineq insufficient cost, bounded away from Gc}
\end{align}
for any $(i_1,\dots,i_m) \in \mathbb{N}^m$ with $\sum_{j = 1}^m i_j(\alpha_j-1) < J_\mathcal{G}$, as well as
\begin{align*}
    \exists x \in \mathcal{G}(\bm{i}^*) \text{ such that }\bm{d}(x,\mathcal{G}) > 100l^*\bar{\epsilon}.
\end{align*}
Furthermore, Assumption \ref{assumption 5 zero is an attractor} implies that $\nabla f(v) \neq \textbf{0}$ for any $v \in \overline{\mathcal{G}}\symbol{92}\{\textbf{0}\}$. From Assumptions \ref{assumption 3 boundary of domain G}, \ref{assumption 4 jacobian of g at the origin} and \ref{assumption 5 zero is an attractor} (and by picking a smaller $\bar{\epsilon} > 0$ if needed), one can see the existence of some $c_0 > 0, \bar{\epsilon} > 0$ such that 
\begin{align}
    \norm{ \bm{x}_t(x) } & \leq e^{ -c_0 t } \norm{ x }\ \ \forall t \geq 0,\ x \in \overline{B(\textbf{0},\bar{\epsilon})}, \label{ineq ODE geometric convergence to origin} \\
    \norm{ \nabla f(x) } & \geq c_0 \ \ \forall x \in \overline{\mathcal{G}}\symbol{92}B(\textbf{0},\bar{\epsilon}),\label{ineq LB on gradient size} \\
    \bm{d}\big(\bm{x}_t(x), \mathcal{G}^c\big) & \geq \bm{d}(x,\mathcal{G}^c) + c_0 t\ \ \ \text{if}\ \bm{x}_s(x) \in \mathcal{G}\text{ and } \bm{d}\big( \bm{x}_s(x), \mathcal{G}^c \big) \leq \bar{\epsilon} \ \forall s\in[0,t].  \label{ineq ODE going inward at the boundary}
\end{align}
Here $B(x,r) = \{y \in \mathbb{R}^d: \norm{x-y} < r \}$ is the open ball centered at $x$ with radius $r>0$. One immediate consequence from \cref{ineq ODE going inward at the boundary} is that
\begin{align}
    \bm{d}\big(\bm{x}_t(x), \mathcal{G}^c\big) \geq \min\{ \bm{d}(x,\mathcal{G}^c),\bar{\epsilon} \}\ \ \ \forall x \in \mathcal{G}, t \geq 0. \label{ineq LB on distance to boundary set}
\end{align}
Besides, the boundedness of domain $\mathcal{G}$ and smoothness of $f(\cdot)$ imply the existence of some $M > 0$ such that
\begin{align}
    \overline{ \mathcal{G} } & \subseteq B(\textbf{0}, M), \\
    \norm{ \nabla f(x) } & \leq M\ \ \forall x \in \overline{ \mathcal{G} }, \label{ineq bound gradient size} \\
    \norm{ \nabla^2 f(x) } & \leq M\ \ \forall x \in \overline{ \mathcal{G} } \label{ineq bound Hessian size}
\end{align}
where for the matrix norm we use spectral norm.

Recall the definitions of the configuration sets $\mathcal{A}(i_1,\cdots,i_m)$, the unique configuration $\bm{i}^*$ and count of jumps $k^*$ to trigger the exit with minimum cost (see Assumption \ref{assumption unique configuration for minimum exit cost} and the remark underneath). The following three results provide control on the sizes and inter-arrival times of typical jumps that can trigger the escape from $\mathcal{G}$. The proofs involve analyses of the deterministic dynamical system $\widetilde{\bm{x}}_t$ and will be provided in Section \ref{section: proof of technical lemmas}.
In particular, we introduce another concept as \emph{type} of jumps $(w_1,\cdots,w_k)$ that is similar {configuration} $(i_1,\cdots,i_m)$ and configuration sets $\mathcal{A}(i_1,\cdots,i_m)$.
For any $k \in \mathbb{N}$ and any $(w_1,\cdots,w_k) \in (\mathbb{R}^d\symbol{92}\mathbb{O})^k$, 
we say that the $(w_1,\cdots,w_k)$ is of type-$\bm{j}$ for some $\bm{j} = (j_1,\cdots,j_k) \in \{1,2,\cdots,m\}^k$ iff
\begin{align}
    w_i\in E_{j_i}\ \forall i = 1,\cdots,k. \label{def type of jumps w}
\end{align}
In other words, if $\bm{w}^1$ and $\bm{w}^2$ are of the same type $\bm{j}$, then they have the same cardinality $|\bm{w}^1| = |\bm{w}^2| = |\bm{j}| = k$ for some $k \in \mathbb{N}$;
moreover, for any $i \in [k]$, the jumps $w^1_i$ and $w^2_i$ are both in the cone $E_{j_i}$ so both jumps point at directions in cone $E_{j_i}$ and have the same cost $J(w^1_i) = J(w^2_i) = \alpha_{j_i} - 1$ for some $j_i \in [m]$.
Define the set
\begin{align*}
    \mathcal{A}^\text{type}_{\bm{j}} = \{ \bm{w} \in  (\mathbb{R}^d\symbol{92}\mathbb{O})^{ | \bm{j} | }:\ \bm{w}\text{ is of type-}\bm{j} \}.
\end{align*}
The accumulated cost for any $\bm{w} \in \mathcal{A}^\text{type}_{\bm{j}}$ is defined as $\mathcal{J}_\text{type}(\bm{j}) = \sum_{i = 1}^{ |\bm{j}| }(\alpha_{j_i} - 1)$. Lastly, define the skip-one accumulated cost for any type $\bm{j}$ as
\begin{align*}
    \mathcal{J}_\text{type}^\downarrow(\bm{j}) = \max\{\sum_{i = 1,\cdots,|\bm{j}|, i \neq k}(\alpha_{j_i}-1):\ k = 1,2,\cdots,|\bm{j}| \},
\end{align*}
i.e., $\mathcal{J}^\downarrow_\text{type}(\bm{j})$ is the highest possible accumulated cost if we remove one element in $\bm{w}$ for $\bm{w} \in \mathcal{A}^\text{type}_{\bm{j}}$. 

\begin{lemma} \label{lemma bound on interarrival time and jump sizes for typical large jumps}

There exist some $\bar{t} \in (0,\infty), \bar{\delta} \in (0,\infty), \epsilon_0 > 0$ such that the following claim holds.
Let $k^\prime \in \mathbb{N}$ and $\bm{j} = (j_1,\cdots,j_{k^\prime}) \in \{1,2,\cdots,m\}^{k^\prime}$ be such that $\mathcal{J}^\downarrow_\text{type}(\bm{j}) < J_\mathcal{G}$.
For any $\bm{t}^{ (k^\prime) } = (t_1,\cdots,t_{k^\prime})\in \{0\} \times \mathbb{R}^{k^\prime - 1}_+,\bm{w}^{(k^\prime)} = (w_1,\cdots,w_{k^\prime}) \in  \mathcal{A}^\text{type}_{\bm{j}}$,
the following set of conditions
\begin{align}
    t_j & < \Bar{t}\ \ \forall j = 2,3,\cdots,k^\prime \label{def bar t Rd} \\
    \norm{w_j} & > \Bar{\delta} \ \ \forall j =1,2,\cdots,k^\prime. \label{def bar delta Rd}
\end{align}
is the necessary condition for $\inf_{s \geq 0} \bm{d}\big(\widetilde{\bm{x}}( s, \bm{0};\ \bm{t}^{(k)},\bm{w}^{(k)}), \mathcal{G}^c\big) \leq \epsilon_0$.
\end{lemma}

\begin{lemma} \label{lemma bound on interarrival time and jump sizes for typical large jumps generalized}
Let $\bar{t},\bar{\delta},\epsilon_0$ be the positive constants prescribed in Lemma \ref{lemma bound on interarrival time and jump sizes for typical large jumps}.
There exists some $\epsilon_1 > 0$ such that the following claim holds.

Let $k^\prime \in \mathbb{N}$ and $\bm{j} = (j_1,\cdots,j_{k^\prime}) \in \{1,2,\cdots,m\}^{k^\prime}$ be such that $\mathcal{J}_\text{type}^\downarrow(\bm{j}) < J_\mathcal{G}$.
For any $x \in B(\bm{0},\epsilon_1)$,
any $\bm{t}^{ (k^\prime) } = (t_1,\cdots,t_{k^\prime})\in \{0\} \times \mathbb{R}^{k^\prime - 1}_+,\bm{w}^{(k^\prime)} = (w_1,\cdots,w_{k^\prime}) \in  \mathcal{A}^\text{type}_{\bm{j}}$,
the following set of conditions
\begin{align}
    t_j & < 2\Bar{t}\ \ \forall j = 2,3,\cdots,k^\prime \\
    \norm{w_j} & > \Bar{\delta}/2 \ \ \forall j =1,2,\cdots,k^\prime.
\end{align}
is the necessary condition for $\inf_{s \geq 0} \bm{d}\big(\widetilde{\bm{x}}( s, x;\ \bm{t}^{(k)},\bm{w}^{(k)}), \mathcal{G}^c\big) \leq \epsilon_0/2$.
\end{lemma}

\begin{lemma} \label{lemma bound on small perturbation at origin}
There exist some $\epsilon_0 > 0, \delta_0 > 0$ such that 
\begin{align*}
    \sup_{s \geq 0 }\bm{d}\big( \widetilde{\bm{x}}(s, x;\ \bm{t}^{(k)},\bm{w}^{(k)}), \mathcal{G}^c\big) > \epsilon_0
\end{align*}
for all $x \in B(\bm{0},\delta_0)$, all $(i_1,\cdots,i_m) \in \mathbb{N}^m$ with $\mathcal{J}(i_1,\cdots,i_m) < J_\mathcal{G}$ and $k = \sum_{j=1}^m i_j$, and all $\bm{w}^{(k)} \in \mathcal{A}(i_1,\cdots,i_m),\ \bm{t}^{(k)}\in \{0\}\times \mathbb{R}^{k-1}_+$.
\end{lemma}

Similar to the ODE path $\bm{x}$, we defined the following paths with rate $\eta > 0$:
\begin{align}
    \dot{(\bm{x}^\eta_t)}(x) & = - \eta\nabla f\big( \bm{x}^\eta_t(x) \big)\ \ \forall t \geq 0, \label{def ODE with rate 1}  \\
    \bm{x}^\eta_0(x) & = x. \label{def ODE with rate 2} 
\end{align}
In other words, ODE flow $\bm{x}_t$ is equivalent to $\bm{x}^\eta_t$ with rate $\eta = 1$. The next result gives us an upper bound for the return time to a neighborhood of the local minimum as for the gradient flow.

\begin{lemma} \label{lemma bound on return time Rd}
For $\tau_{ODE}^\eta(x, \epsilon) \delequal{} \min\{t \geq 0:\ \bm{x}^\eta_t(x) \in \overline{B(\bm{0}, \epsilon)} \}$, there exists $c_1 \in (0,\infty)$ such that for any $\epsilon \in (0,\bar{\epsilon})$,
\begin{align*}
    \tau_{ODE}^\eta(x, \epsilon) \leq \frac{ c_1 + c_1\log(1/\epsilon) }{\eta}\ \ \ \forall x \in \overline{\mathcal{G}}.
\end{align*}
\end{lemma}

\begin{proof}
Due to \cref{ineq LB on gradient size}, we know that $T_0 \delequal \frac{ \sup_{x \in \bar{\mathcal{G}}}f(x) - \inf_{x \in \bar{\mathcal{G}}}f(x)   }{ c_0 } < \infty$ and
$$  \sup_{x \in \bar{\mathcal{G}}}\tau^{1}_{ODE}(x,\bar{\epsilon}) \leq T_0.$$
Furthermore, from \cref{ineq ODE geometric convergence to origin} one can see that
\begin{align*}
    \sup_{x \in \bar{\mathcal{G}}}\tau^{1}_{ODE}(x,\epsilon) - \tau^{1}_{ODE}(x,\bar{\epsilon}) \leq \frac{ \log(1/\epsilon) - \log(1/\bar{\epsilon}) }{c_0}.
\end{align*}
Now combining these two bounds with a $1/\eta$ time scaling, we have
\begin{align*}
    \sup_{x \in \bar{\mathcal{G}}}\tau^{\eta}_{ODE}(x,\epsilon)\leq \frac{ c_1 + c_1\log(1/\epsilon) }{\eta}
\end{align*}
for $c_1 = \max\{ T_0 + \frac{ \log(\bar{\epsilon}) }{c_0},\ \frac{1}{c_0} \}$.
\end{proof}

With Lemma \ref{lemma bound on interarrival time and jump sizes for typical large jumps}-\ref{lemma bound on return time Rd}, we are able to pick some constants to facilitate the analysis below.
By decreasing $\bar{\epsilon}$ and increasing $M$ if needed, we can assume the existence of some $\bar{t} < \infty, \bar{\delta} < \bar{\epsilon}$ such that
\begin{itemize}
\item (Due to Lemma \ref{lemma bound on interarrival time and jump sizes for typical large jumps} and \ref{lemma bound on interarrival time and jump sizes for typical large jumps generalized}) 
For any $k \in \mathbb{N}$ and $\bm{j} \in \{1,2,\cdots,m\}^k$ such that $\mathcal{J}_\text{type}^\downarrow(\bm{j}) < J_\mathcal{G}$,
any $x \in B(\bm{0},100 l^* \bar{\epsilon})$,
any $\bm{t}^{(k)} = (t_1,\cdots,t_{k}) \in \{0\} \times \mathbb{R}^{k - 1}_+, \bm{w}^{(k)} = (w_1,\cdots,w_{k}) \in \mathcal{A}^\text{type}_{ \bm{j} }$ with
\begin{align*}
   \bm{d}\big(\widetilde{\bm{x}}( \sum_{j = 1}^{k} t_j, x;\ \bm{t}^{(k)},\bm{w}^{(k)}), \mathcal{G}^c\big) < 100l^*\bar{\epsilon},
\end{align*}
it holds that
\begin{align}
    t_j < \Bar{t}\ \ \forall j = 2,3,\cdots,k\text{ and } \norm{w_j} & > \Bar{\delta} \ \ \forall j =1,2,\cdots,k. \label{property bar delta bar t}
\end{align}
    \item For any $(i_1,\dots,i_m) \in \mathbb{N}^m$ with $\mathcal{J}(i_1,\dots,i_m) < J_\mathcal{G}$,
    \begin{align}
    \bm{d}\big(\mathcal{G}( i_1,\cdots,i_m ),\mathcal{G}^c) > 100l^*\bar{\epsilon}; \label{property bounded away with small cost}
\end{align}
    \item There exists some $x \in \mathcal{G}(\bm{i}^*)$ such that
    \begin{align}
        \bm{d}(x,\mathcal{G}) > 100l^*\bar{\epsilon};\label{property existence of escape path}
    \end{align}
    \item (Due to Lemma \ref{lemma bound on small perturbation at origin}) The constant $\bar{\epsilon} > 0$ is sufficiently small such that
    \begin{align}
    \sup_{s \geq 0 }\bm{d}\big( \widetilde{\bm{x}}(s, x;\ \bm{t}^{(k)},\bm{w}^{(k)}), \mathcal{G}^c\big) > 100l^*\bar{\epsilon} \label{implication of lemma bound on small perturbation at origin}
\end{align}
for all $x \in B(\bm{0},100l^*\bar{\epsilon})$, all $(i_1,\cdots,i_m) \in \mathbb{N}^m$ with $\mathcal{J}(i_1,\cdots,i_m) < J_\mathcal{G}$ and $k = \sum_{j=1}^m i_j$, and all $\bm{w}^{(k)} \in \mathcal{A}(i_1,\cdots,i_m),\ \bm{t}^{(k)}\in \{0\}\times \mathbb{R}^{k-1}_+$;
    \item (Due to Lemma \ref{lemma bound on return time Rd}) The constant $\bar{t}<\infty$ is large enough so that
    \begin{align*}
        \sup_{x \in \bar{\mathcal{G}}}\tau^\eta_\text{ODE}(x,\bar{\epsilon}) < \bar{t}/\eta.
    \end{align*}
\end{itemize}

\subsection{Bounding the distance between ODE flow and SGD iterates}

Moving on, we show that, without large jumps, the SGD iterates $X^\eta_n(x)$ are unlikely to show significant deviation from the deterministic gradient descent process $\textbf{y}^\eta_n(x)$ defined as
\begin{align}
    \textbf{y}^\eta_0(x) & = x, \label{def GD process Rd 1} \\
    \textbf{y}^\eta_n(x) & = \textbf{y}^\eta_{n-1}(x) - \eta \nabla f\Big( \textbf{y}^\eta_{n-1}(x) \Big).\label{def GD process Rd 2}
\end{align}
We are ready to state the first lemma, where we bound the distance between the gradient descent iterates $\textbf{y}^\eta_n(y)$ and the ODE $\textbf{x}^\eta(t,x)$ when the initial conditions $x,y$ are close enough.

\begin{lemma} \label{lemma Ode Gd Gap Rd}
Given $t > 0$ and $x,y \in \mathcal{G}$ such that the line segment between $\textbf{x}^\eta(k,x)$ and $\textbf{y}^\eta_k(y)$ lies in $\mathcal{G}$ for any $k = 1,2,\cdots,\floor{t}$,
\begin{align*}
    \sup_{s \in [0,t]}\norm{\textbf{x}^\eta(s,x) - \textbf{y}^\eta_{\floor{s}}(y)} \leq (2\eta M + \norm{x-y})\exp(\eta M t)
\end{align*}
where $M \in (0,\infty)$ is the constant in \cref{ineq bound gradient size}\cref{ineq bound Hessian size}.
\end{lemma}

 \begin{proof}

Define a continuous-time process $\textbf{y}^\eta(s;y) \delequal{} \textbf{y}^\eta_{\floor{s}}(y)$, and note that
\begin{align*}
    \textbf{x}^\eta(s,x) & = \textbf{x}^\eta(\floor{s},x) - \eta\int_{\floor{s}}^s \nabla f(\textbf{x}^\eta(u,x) )du \\
    \textbf{x}^\eta(\floor{s},x) & = x - \eta\int_0^{\floor{s} }\nabla f(\textbf{x}^\eta(u,x) )du \\
    \textbf{y}^\eta_{\floor{s}}(y) = \textbf{y}^\eta(\floor{s} ,y) & = y - \eta \int_0^{ \floor{s} }\nabla f( \textbf{y}^\eta(u,y) )du.
\end{align*}

Therefore, if we define function
$$b(u) = \textbf{x}^\eta(u,x) - \textbf{y}^\eta(u,y),$$
from the fact $\norm{\nabla f(\cdot)}\leq M$, one can see that $\norm{b(u)} \leq \eta M + \norm{x - y}$ for any $u \in [0,1)$ and $\norm{b(1)} \leq 2\eta M + \norm{x - y}$. In case that $s > 1$, from the display above and the fact $\norm{\nabla^2 f(\cdot)}\leq M$, we now have
\begin{align*}
    \norm{\textbf{y}^\eta_{\floor{s}}(x) - \textbf{x}^\eta(s,x)} & \leq \norm{b({\floor{s}})} + \eta M; \\
    \norm{b({\floor{s}})} & \leq \eta M \int_1^{\floor{s}}\norm{b(u)}du.
\end{align*}
From Gronwall's inequality (see Theorem 68, Chapter V of \cite{protter2005stochastic}, where we let function $\alpha(u)$ be $\alpha(u) = \norm{b(u+1)}$), we have
$$\norm{\textbf{y}^\eta_{\floor{s}}(x) - \textbf{x}^\eta(s,x)}\leq (2\eta M + \norm{x-y})\exp(\eta M t).$$
This concludes the proof.  
\end{proof}

Now we consider an extension of the previous Lemma in the following sense: 
when both perturbed by at most $l^*$ similar perturbations, the ODE flow and gradient descent process should still stay close enough. 
Analogous to the definition of the perturbed ODE $\widetilde{\bm{x}}^\eta$, we can construct a process $\widetilde{\bm{y}}^\eta$ as a perturbed gradient descent process as follows. 
For some integer $1 \leq l \leq l^*$, 
a sequence of strictly positive integers $(t_2,\cdots,t_l)$ (with convention $t_1 = 0$ and let $\bm{t}^{(l)} = (t_1,\cdots,t_l)$),
a sequence of non-zero vectors $\widetilde{\bm{w}}^{(l)}=(\widetilde{w}_1,\cdots,\widetilde{w}_{l})$
and $y \in \mathbb{R}^d$, define the perturbed gradient descent iterates with gradient clipping at $b$ as
\begin{align}
     & \widetilde{ \bm{y} }^{\eta}_{n}(y;\bm{t}^{(l)},\widetilde{\bm{w}}^{(l)})
     \nonumber 
     \\ 
     = & \widetilde{ \bm{y} }^{\eta}_{n-1}(y;\bm{t}^{(l)},\widetilde{\bm{w}}^{(l)}) + \varphi_b\big(- \eta \nabla f( \widetilde{ \bm{y} }^{\eta}_{n-1}(y;\bm{t}^{(l)},\widetilde{\bm{w}}^{(l)})) + \sum_{j = 2}^{l}\mathbbm{1}\{ n = t_1 + \cdots + t_j \}\widetilde{w}_j\big)\label{def perturbed GD Rd}
\end{align}
with initial condition $\widetilde{ \bm{y} }^{\eta}_{0}(y;\bm{t}^{(l)},\widetilde{\bm{w}}^{(l)}) =y + \varphi_b(\widetilde{w}_1) $.

\begin{corollary} \label{corollary ODE GD gap Rd}
For any fixed $\epsilon > 0, t > 0$, the following claim holds for all $\eta>0$ sufficiently small: given any $x,y \in \mathcal{G}$,
some integer $1 \leq l \leq l^*$,
a sequence of strictly positive integers $\bm{t}^{(l)} = (t_j)_{j = 2}^{l}$ (with convention $t_1 = 0$), 
and any two sequences of vectors $\bm{w}^{(l)} = (w_j)_{j = 1}^{l}, \widetilde{\bm{w}}^{(l)} = (\widetilde{w}_j)_{j \geq 1}^{l}$ such that (let $t_\text{end} = \sum_{j=1}^l t_j$)
\begin{itemize}
    \item $|x-y| < \epsilon$;
    \item $t_j \leq 2t/\eta$ for all $j \in [l]$; 
    \item $|w_j - \widetilde{w}_{j}| < \epsilon$ for all $j \in [l]$;
    \item The line segment between $ \widetilde{\bm{x}}^\eta(k,x;\bm{t}^{(l)},\bm{w}^{(l)})$ and $\widetilde{ \bm{y} }^{\eta}_{k}(y;\bm{t}^{(l)},\widetilde{\bm{w}}^{(l)})$ lies in $\mathcal{G}$ for any $k \leq t_\text{end} - 1$,
\end{itemize}
then we have 
\begin{align*}
    \sup_{s \in [0, t_\text{end} ]}\norm{\widetilde{\bm{x}}^\eta(s,x;\bm{t}^{(l)},\bm{w}^{(l)}) - \widetilde{ \bm{y} }^{\eta}_{\floor{s}}(y;\bm{t}^{(l)},\widetilde{\bm{w}}^{(l)})} \leq \bar{\rho}(t)\epsilon
\end{align*}
where the constant $\bar{\rho} = (2 \exp(2 M t) + 2)^{l^*}$. In particular, $\bar{\rho} = \bar{\rho}(\bar{t})$ is a constant where $\bar{t}$ is the constant in \cref{def bar t Rd}.
\end{corollary}

 \begin{proof}Throughout this proof, fix some $\eta \in (0,\epsilon/2M)$. We will show that for any $\eta$ in the range the claim would hold.

First, on interval $[0,t_2)$, from Lemma \ref{lemma Ode Gd Gap Rd}, one can see that (since $2M\eta < \epsilon$)
\begin{align*}
    \sup_{s \in [0, t_2) }\norm{\widetilde{\bm{x}}^\eta(s,x;\bm{s}^{(l)},\bm{w}^{(l)}) - \widetilde{ \bm{y} }^{\eta}_{\floor{s}}(y;\bm{t}^{(l)},\widetilde{\bm{w}}^{(l)})} \leq 2 \exp(2M t)\cdot \epsilon.
\end{align*}
The at $t = t_2$, by considering the difference between $w_2$ and $\widetilde{w}_2$, and the possible change due to one more gradient descent step (which is bounded by $\eta M < \epsilon$), we have
\begin{align*}
    \sup_{s \in [0, t_2] }\norm{\widetilde{\bm{x}}^\eta(s,x;\bm{s}^{(l)},\bm{w}^{(l)}) - \widetilde{ \bm{y} }^{\eta}_{\floor{s}}(y;\bm{t}^{(l)},\widetilde{\bm{w}}^{(l)})} \leq (2 \exp(2 M t) + 2)\cdot \epsilon.
\end{align*}
Now we proceed inductively. For any $j = 2,3,\cdots, l - 1$, assume that
\begin{align*}
    \sup_{s \in [0, t_j] }\norm{\widetilde{\bm{x}}^\eta(s,x;\bm{s}^{(l)},\bm{w}^{(l)}) - \widetilde{ \bm{y} }^{\eta}_{\floor{s}}(y;\bm{t}^{(l)},\widetilde{\bm{w}}^{(l)})} \leq (2 \exp(2M t) + 2)^{j - 1}\cdot \epsilon.
\end{align*}
Then by focusing on interval $[t_j, t_{j+1}]$ and using Lemma \ref{lemma Ode Gd Gap Rd} again, one can show that
\begin{align*}
    \sup_{t \in [t_j, t_{j+1}] } \norm{\widetilde{\bm{x}}^\eta(s,x;\bm{s}^{(l)},\bm{w}^{(l)}) - \widetilde{ \bm{y} }^{\eta}_{\floor{s}}(y;\bm{t}^{(l)},\widetilde{\bm{w}}^{(l)})} & \leq 2\epsilon + \big( (2\exp(2M t) + 2)^{j - 1} + 1 \big)\exp(2M t) \epsilon \\
    & \leq (2\exp(2M t) + 2)^{j}\cdot \epsilon.
\end{align*}
This concludes the proof. 
\end{proof}

Using an almost identical approach, one can show the following results that bounds the distance between two ODE flows with similar initial values and jumps.
\begin{corollary} \label{corollary ODE ODE gap Rd}
There exists some constant $\rho^* \in (0,\infty)$
such that the following claim holds:
For any fixed $\epsilon > 0$, 
any $x,y \in \mathcal{G}$,
any integer $1 \leq l \leq l^*$,
a sequence of strictly positive integers $\bm{t}^{(l)} = (t_j)_{j = 2}^{l}$ (with convention $t_1 = 0$), 
and any two sequences of vectors $\bm{w}^{(l)} = (w_j)_{j = 1}^{l}, \widetilde{\bm{w}}^{(l)} = (\widetilde{w}_j)_{j \geq 1}^{l}$ such that (let $t_\text{end} = \sum_{j=1}^l t_j$)
\begin{itemize}
    \item $|x-y| < \epsilon$;
    \item $t_j \leq 2\bar{t}/\eta$ for all $j \in [l]$; 
    \item $|w_j - \widetilde{w}_{j}| < \epsilon$ for all $j \in [l]$;
    \item The line segment between $ \widetilde{\bm{x}}^\eta(s,x;\bm{t}^{(l)},\bm{w}^{(l)})$ and $\widetilde{\bm{x}}^\eta(s,y;\bm{t}^{(l)},\widetilde{\bm{w}}^{(l)})$ lies in $\mathcal{G}$ for any $s < t_\text{end}$,
\end{itemize}
then we have 
\begin{align*}
    \sup_{s \in [0, t_\text{end} ]}\norm{\widetilde{\bm{x}}^\eta(s,x;\bm{t}^{(l)},\bm{w}^{(l)}) - \widetilde{\bm{x}}^\eta(s,y;\bm{t}^{(l)},\widetilde{\bm{w}}^{(l)})} \leq \rho^*\epsilon.
\end{align*}
\end{corollary}

In the next few results, we show that a similar type of control can be obtained on the distance between gradient descent iterates $\widetilde{\bm{y}}^\eta_n$ and the SGD iterates $X^\eta_n$.
Specifically, our first goal is to show that before any \textit{large} jump, it is unlikely that the (deterministic) gradient descent process $\bm{y}^\eta_n$ would deviate too far from $X^\eta_n$. 
To facilitate the analysis below, we introduce some additional notations. 
First, we will group the noises $Z_n$ based on a threshold level $\delta > 0$: let us define
\begin{align}
    Z_{n}^{\leq \delta,\eta} \delequal{} Z_n\mathbbm{1}\{\eta\norm{Z_n}\leq \delta\}, \label{defSmallJump_GradientClipping_Rd} \\
     Z_{n}^{> \delta,\eta} \delequal{} Z_n\mathbbm{1}\{\eta\norm{Z_n}> \delta\}.\label{defLargeJump_GradientClipping_Rd}
\end{align}
The former are viewed as \textit{small noises} while the latter will be referred to as \textit{large noises} or \textit{large jumps}. 
Furthermore, for any $j \geq 1$, define the $j$\textsuperscript{th} arrival time, size, and direction of large jumps as 
\begin{align}
    T^\eta_j(\delta) & \delequal{} \min\{ n > T^\eta_{j-1}(\delta):\ \eta\norm{Z_n} > \delta  \},\quad T_0^\eta(\delta) = 0, \label{defArrivalTime large jump_Rd} \\
    W^\eta_j(\delta) & \delequal{} Z_{T^\eta_j(\delta)}, \label{defSize large jumpRd} \\
    \Theta^\eta_j(\delta) & \delequal{} W^\eta_j(\delta)/\norm{ W^\eta_j(\delta) }. \label{defDirection large jump Rd}
\end{align}
The following event
\begin{align}
    A(n,\eta,\epsilon,\delta) & = \Big\{ \max_{k = 1,2,\cdots, n \wedge (T^\eta_1(\delta) - 1)  }\eta\norm{Z_1 + \cdots + Z_k} \leq \epsilon \Big\} \label{def event A small noise large deviation Rd}
\end{align}
describes a scenario where, before the first large jump, not much perturbation has been caused by noises.

\begin{lemma} \label{lemma SGD GD gap Rd}
For any $\epsilon > 0, \eta > 0$, 
any $x,y \in \mathcal{G}$ and positive integer $n$ such that  $\norm{x-y}<\frac{  \epsilon}{2\exp(\eta M n)  }$, on event
\begin{align*}
& A(n,\eta,\frac{  \epsilon}{2\exp(\eta M n)},\delta  ) \bigcap \\
& \ \ \ \ \ \ \ \ \Big\{ \text{the line segment between } \textbf{y}^\eta_j(y)\text{ and }X^\eta_j(x)\text{ lies in }\mathcal{G} \ \ \forall j = 0,1,\cdots,n \wedge (T^\eta_1(\delta) - 1 ) \Big\},    
\end{align*}
we have 
$$\norm{\bm{y}^\eta_m(y) - X^\eta_m(x)} < \epsilon\ \ \forall m = 1,2,\cdots,n\wedge(T^\eta_1(\delta) - 1 ) .$$
\end{lemma}

 \begin{proof} 

On the said event, we are able to apply Lemma \ref{lemmaBasicGronwall} and obtain that
\begin{align*}
    \norm{\bm{y}^\eta_j(y) - X^\eta_j(x)} \leq (\norm{x - y} + \frac{\epsilon}{2\exp(\eta M n)} )\exp(\eta M j) < \epsilon \ \ \ \forall j = 1,2,\cdots,n \wedge (T^\eta_1(\delta) - 1 )
\end{align*}
and conclude the proof.
\end{proof}

Similar to the extension from Lemma \ref{lemma Ode Gd Gap Rd} to Corollary \ref{corollary ODE GD gap Rd}, we can extend Lemma \ref{lemma SGD GD gap Rd} to show that, if we consider the a gradient descent process that is only perturbed by large noises, then it should stay pretty close to the SGD iterates $X^\eta_n$. To be specific, let
\begin{align*}
    Y^\eta_0(x) & = x \\ 
     Y^\eta_n(x) & =Y^\eta_{n - 1}(x) - \varphi_b\big( - \eta\nabla f\big( Y^\eta_{n - 1}(x) \big) + \sum_{j \geq 1}\mathbbm{1}\{n = T^\eta_j(\delta)\}\eta Z_{n}\big). 
\end{align*}
be a gradient descent process (with gradient clipping at threshold $b$) that is only perturbed by large noises. The next corollary can be shown by an approach identical to the one for Corollary \ref{corollary ODE GD gap Rd} (namely, inductively repeating Lemma \ref{lemma SGD GD gap Rd} at arrival time of each jump) so we omit the details here.

\begin{corollary} \label{corollary sgd gd gap Rd}
There exists some function $\widetilde{\rho}: (0,\infty) \mapsto (0,\infty)$ such that the following claim hold for all $\epsilon > 0, t > 0$ and all sufficiently small $\eta > 0$: For any $x,y \in \mathcal{G})$ with $\norm{x-y} < \epsilon$ and any $l \in [l^*]$, on event $ A_0(\epsilon,\eta,\delta,l,t) \cap B_0(\epsilon,\eta,\delta,l,t) \cap C_0(\epsilon,\eta,\delta,l,t)$, we have
\begin{align*}
    \norm{Y^\eta_n(y) - X^\eta_n(x)} < \widetilde{\rho}(t)\epsilon\ \ \forall n =1,2,\cdots,T^\eta_{l}(\delta)
\end{align*}
where
\begin{align*}
    A_0(\epsilon,\eta,\delta,l,t) & \delequal{} \Big\{\forall i = 1,\cdots,l,\ \max_{j = T^\eta_{i-1}(\delta) + 1, \cdots, T^\eta_{i}(\delta) - 1} \eta\norm{Z_{ T^\eta_{i-1}(\delta) + 1} + \cdots + Z_j} \leq \frac{ \epsilon }{2\exp(2tM) }  \Big\}; \\
    B_0(\epsilon,\eta,\delta,l,t) & \delequal{} \Big\{\forall j = 2,\cdots,l, T^\eta_j(\delta) - T^\eta_{j - 1}(\delta) \leq 2t/\eta  \Big\}, \\
    C_0(\epsilon,\eta,\delta,l,t) & \delequal{}\{ \text{the line segment between } Y^\eta_j(y)\text{ and }X^\eta_j(x)\text{ lies in }\mathcal{G} \ \ \forall j = 0,1,\cdots, T^\eta_l(\delta) - 1 \}.
\end{align*}
In particular, $\widetilde{\rho} = \widetilde{\rho}(\bar{t})$ is a constant 
\end{corollary}

The next result shows that the type of events $A(n,\eta,\epsilon,\delta)$ defined in \cref{def event A small noise large deviation Rd} is indeed very likely to occur, especially for small $\epsilon$. For clarity of the presentation, we introduce the following definitions that are slightly more general than the \textit{small} and \textit{large} jumps defined in \cref{defSmallJump_GradientClipping_Rd}\cref{defLargeJump_GradientClipping_Rd} (for any $c > 0$)
\begin{align*}
    Z^{\leq c}_n & \delequal{} Z_n\mathbbm{1}\{ \norm{Z_n} \leq c \}, \\
    Z^{> c}_n & \delequal{} Z_n\mathbbm{1}\{ \norm{Z_n} > c \}.
\end{align*}

\begin{lemma} \label{lemma prob event A Rd}
Given any $t > 0$, $\epsilon > 0$, and $N > 0$, the following holds for any sufficiently small $\delta > 0$:
\begin{align*}
    \mathbb{P}\Big( \max_{ n = 1,2,\cdots,\ceil{ t/\eta } } \eta\norm{ Z^{\leq\delta/\eta}_1 + \cdots + Z^{\leq\delta/\eta}_n } > \epsilon  \Big) = o(\eta^N)
\end{align*}
as $\eta \downarrow 0$.
\end{lemma}

\begin{proof}
For any $C\in(0,\infty), \delta > 0, \beta \in (0,1)$, consider the following decomposition of the small noise $Z^{\leq\delta/\eta}_n$. Here we adopt the convention that $E_0 = \mathbb{R}^d\symbol{92}( \cup_{j = 1}^m E_j )$. (For definitions and properties of the closed cones $E_1,\cdots,E_m$, see Assumption \ref{assumption 6 regular variation of the Rd noise} and the remark underneath.)
\begin{align*}
    Z^{(j)}_n & \delequal{} Z_n\mathbbm{1}\{ Z_n \in E_j \}, \\
    Z^{\leq C,(j)}_n & \delequal{} Z^{(j)}_n\mathbbm{1}\{ \norm{Z^{(j)}_n} \leq C \}, \\
    Z^{\downarrow,(j)}_{C,\delta,\eta}(n) & \delequal{} Z^{(j)}_n\mathbbm{1}\{ \norm{Z^{(j)}_n} \in (C,\delta/\eta] \}
\end{align*}
Based on this decomposition, we define the following iid random variable sequences
\begin{align*}
    \widetilde{z}^{(j)}(n) & \delequal{} Z^{\leq C,(j)}_n - \mathbb{E}Z^{\leq C,(j)}_n, \\
    \widetilde{Z}^{(j)}(n) & \delequal{} \norm{Z^{\downarrow,(j)}_{C,\delta,\eta}(n)}, \\
    \widetilde{Z}^{(j)}_\beta(n) & \delequal{} \norm{Z^{\leq 1/\eta^\beta,(j)}_n}.
\end{align*}
Meanwhile, define the projection operators $\textbf{P}_i:\mathbb{R}^d \mapsto \mathbb{R}$ with $\textbf{P}_i(v_1,\cdots,v_d) = v_i$. First, recall that $\mathbb{E}Z_1 = \bm{0}$. So for all $C$ sufficiently large, we have
\begin{align*}
    \sup_{y \geq C}\norm{\mathbb{E}Z^{\leq y}_1 } = \sup_{y \geq C}\norm{\mathbb{E}Z^{> y}_1 } < \frac{\epsilon}{2t}.
\end{align*}
Moreover, due to $H_j \in RV_{-(\alpha_j + 1)}$ for all $j = 1,2,\cdots,m$ and $H_0(x) = o(1/x^\alpha)$ for any $\alpha > 0$ (for definition of functions $H_j$, see \cref{def H function Rd}),  we know that $\mathbb{E}\norm{Z^{(j)}_n} < \infty$ for any $j = 0,1,\cdots,m$, implying that for all $C$ large enough, we have (for any $j = 0,1,\cdots,m$)
\begin{align}
    \sup_{y,z \geq C}\norm{ \mathbb{E}Z^{\leq y,(j)}_1 - \mathbb{E}Z^{\leq z,(j)}_1 } < \frac{1}{2t}\cdot\frac{\epsilon}{ 24\sqrt{m+1} }. \label{proof prob event A choice of C Rd}
\end{align}
Fix such $C$. Besides, recall that in Assumption \ref{assumption 6 regular variation of the Rd noise} we have $1<\alpha_1<\alpha_2<\cdot<\alpha_m$. Therefore, we are able fix some $\beta \in (0,1)$ and $\delta > 0$ so that
\begin{align}
    \beta\alpha_1 & > 1, \label{proof prob event A choose param 1 Rd} \\
    (N + 1)\delta & < \frac{\epsilon}{12\sqrt{m+1}}, \label{proof prob event A choose param 2 Rd} \\
    \big( \frac{N}{ \beta\alpha_1 - 1 } + 1 \big)\delta & < \frac{\epsilon}{12\sqrt{m+1}}. \label{proof prob event A choose param 3 Rd}
\end{align}
With the fixed $C$ and $\delta$, note that eventually $\delta/\eta > C$ as $\eta \downarrow 0$. Therefore, for all $\eta$ sufficiently small,
\begin{align}
    & \mathbb{P}\Big( \max_{ n = 1,2,\cdots,\ceil{ t/\eta } } \eta\norm{ Z^{\leq\delta/\eta}_1 + \cdots + Z^{\leq\delta/\eta}_n } > \epsilon  \Big) \nonumber \\
    \leq & \mathbb{P}\Big( \max_{ n = 1,2,\cdots,\ceil{ t/\eta } } \eta\norm{ \sum_{k = 1}^n Z^{\leq\delta/\eta}_k - \mathbb{E}Z^{\leq\delta/\eta}_k } > \frac{\epsilon}{2}  \Big) \nonumber \\
    = & \mathbb{P}\Big( \max_{ n = 1,2,\cdots,\ceil{ t/\eta } } \eta\norm{ \sum_{j = 0}^m\sum_{k = 1}^n Z^{\leq C,(j)}_k - \mathbb{E}Z^{\leq C,(j)}_k + Z^{\downarrow,(j)}_{C,\delta,\eta}(k) + \mathbb{E}Z^{\leq C,(j)}_k - \mathbb{E}Z^{\leq\delta/\eta,(j)}_k } > \frac{\epsilon}{2}  \Big) \nonumber \\
    \leq & \sum_{j = 0}^{m}\mathbb{P}\Big( \max_{ n = 1,2,\cdots,\ceil{ t/\eta } } \eta\norm{\sum_{k = 1}^n Z^{\leq C,(j)}_k - \mathbb{E}Z^{\leq C,(j)}_k + Z^{\downarrow,(j)}_{C,\delta,\eta}(k) + \mathbb{E}Z^{\leq C,(j)}_k - \mathbb{E}Z^{\leq\delta/\eta,(j)}_k }
    \nonumber 
    \\
    & \ \ \ \ \ \ \ \ \ \ \ \ \ \ \ \ \ \ \ \ \ \ \ \ \ > \frac{\epsilon}{2\sqrt{m+1}}  \Big) \nonumber \\
    \leq & \sum_{j = 0}^{m}\mathbb{P}\Big( \max_{ n = 1,2,\cdots,\ceil{ t/\eta } } \eta\norm{ \sum_{k = 1}^n \widetilde{z}^{(j)}(k)  } > \frac{\epsilon}{6\sqrt{m+1}}   \Big) \label{proof prob event A goal 1 Rd} \\
    + & \sum_{j = 0}^{m}\mathbb{P}\Big( \eta\sum_{k = 1}^{ \ceil{t/\eta} } \widetilde{Z}^{(j)}(k)   > \frac{\epsilon}{6\sqrt{m+1}}   \Big).\label{proof prob event A goal 2 Rd}
\end{align}
We bound the two terms \cref{proof prob event A goal 1 Rd} and \cref{proof prob event A goal 2 Rd} respectively. First, observe that
\begin{align*}
    \cref{proof prob event A goal 1 Rd} & \leq \sum_{j = 0}^{m}\sum_{i = 1}^d \mathbb{P}\Big( \max_{ n = 1,2,\cdots,\ceil{ t/\eta } } \eta\big|\sum_{k = 1}^n \textbf{P}_i\big(\widetilde{z}^{(j)}(k)\big)  \big| > \frac{\epsilon}{6\sqrt{(m+1)d}}   \Big).
\end{align*}
Let $\widetilde{z}^{(j)}_i(k) = \textbf{P}_i\big(\widetilde{z}^{(j)}(k)\big)$ and let $\widetilde{z}^{(j)}_i$ be an iid copy. By definition, $|\widetilde{z}^{(j)}_i| \leq 2C$. Then from Hoeffding's inequality,
\begin{align*}
    \cref{proof prob event A goal 1 Rd} & \leq 2(m+1)d\cdot\exp\Big( - \frac{ 1 }{2 }\cdot \big(\frac{\epsilon}{6\sqrt{d(m+1)}}\big)^2 \cdot\frac{ 1/\eta^2 }{ 2C\ceil{t/\eta} \ } \Big) = o(\eta^N).
\end{align*}
As for term \cref{proof prob event A goal 2 Rd}, w.l.o.g. let us fix some $j = 0,1,\cdots,m$ and bound
\begin{align}
& \mathbb{P}\Big( \eta\sum_{n = 1}^{ \ceil{t/\eta} } \widetilde{Z}^{(j)}(n)   > \frac{\epsilon}{6\sqrt{m+1}}   \Big)\label{proof prob event A goal 3 Rd} \\
= &\sum_{i=0}^{k-1}\underbrace{\P\Big(\eta\sum_{n = 1}^{ \ceil{t/\eta} } \widetilde{Z}^{(j)}(n)   > \frac{\epsilon}{6\sqrt{m+1}},\ I = i\Big)}_{\triangleq \text{(I)}} + \underbrace{\P\Big(\eta\sum_{n = 1}^{ \ceil{t/\eta} } \widetilde{Z}^{(j)}(n)   > \frac{\epsilon}{6\sqrt{m+1}},\ I \geq k\Big)}_{\triangleq \text{(II)}} \nonumber
\end{align}
where $I \delequal\#\{n = 1,2,\cdots,\ceil{t/\eta}:\ \widetilde{Z}^{(j)}(n) > 1/\eta^\beta \}$
and $k$ is some positive integer such that
$$ k\delta < \frac{\epsilon}{12\sqrt{m+1}},\ \ k(\beta\alpha_1 - 1) > N.$$
Note that we can find such $k$ due to our choice of constants in \cref{proof prob event A choose param 1 Rd}-\cref{proof prob event A choose param 3 Rd}.

For (I) and any $i = 0,1,\cdots,k-1$, observe that (the non-negative RVs $\widetilde{Z}_\beta(n)$ and $\widetilde{Z}(n)$ are defined at the beginning of the proof)
\begin{align}
\text{(I)}
&\leq 
{\ceil{t/\eta} \choose i} \cdot\P\Big(\eta\sum_{n = 1}^{ \ceil{t/\eta} - i } \widetilde{Z}^{(j)}(n)   > \frac{\epsilon}{6\sqrt{m+1}} - i\eta\cdot\frac{\delta}{\eta},\ \widetilde{Z}^{(j)}(n) \leq 1/\eta^\beta\ \forall n \in [ \ceil{t/\eta} - i ] \Big)
\nonumber
\\&\leq 
\ceil{t/\eta}^i\cdot\P\Big(\sum_{n = 1}^{ \ceil{t/\eta} - i } \widetilde{Z}^{(j)}(n)   > \frac{\epsilon}{12\sqrt{m+1}}\cdot\frac{1}{\eta},\ \widetilde{Z}^{(j)}(n) \leq 1/\eta^\beta\ \forall n \in [ \ceil{t/\eta} - i ]\Big)
\nonumber
\\& \leq 
\ceil{t/\eta}^i\cdot\P\Big(\sum_{n = 1}^{ \ceil{t/\eta}} \widetilde{Z}^{(j)}_\beta(n)   > \frac{\epsilon}{12\sqrt{m+1}}\cdot\frac{1}{\eta}\Big).
\nonumber
\end{align}
Let $\widetilde{W}^{(j)}(n) \delequal \widetilde{Z}^{(j)}_\beta(n) - \E \widetilde{Z}^{(j)}_\beta(n)$ and let $\widetilde{W}^{(j)}$ be an iid copy. 
Due to \cref{proof prob event A choice of C Rd}, for all $\eta$ sufficiently small we will have that $\ceil{t/\eta}|\E\widetilde{Z}^{(j)}_\beta(1)| < \frac{\epsilon}{24\sqrt{m+1}}\cdot \frac{1}{\eta}$. Therefore, for all $\eta$ sufficiently small,
\begin{align}
    \text{(I)}
&\leq 
\ceil{t/\eta}^i\cdot\P\Big(|\sum_{n = 1}^{ \ceil{t/\eta}} \widetilde{W}^{(j)}(n)|   > \frac{\epsilon}{24\sqrt{m+1}}\cdot\frac{1}{\eta}\Big).
\label{proof lemma bernstein prior to bernstein bound Rd}
\end{align}
Observe that $|\widetilde{W}^{(j)}| \leq 2/\eta^\beta$ and 
\begin{align*}
    \E(\widetilde{W}^{(j)})^2 = var{\widetilde{Z}^{(j)}_\beta} \leq \E(\widetilde{Z}^{(j)}_\beta)^2 & = \int_0^{ \infty }2y\mathbb{P}( \widetilde{Z}^{(j)}_\beta > y )dy = \int_0^{ 1/\eta^\beta }2yH_j(y)dy.
\end{align*}
Note that there are only two possibilities:
\begin{itemize}
    \item If $j \neq 0$ and $\alpha_j \leq 2$, then from Karamata's theorem, we have $ \E(\widetilde{Z}^{(j)}_\beta)^2 \in RV_{ -\beta(2 - \alpha_j)  }(\eta)$ and note that $\beta(2 - \alpha_j) < 1$;
    \item Otherwise, one can find some $\Delta > 0$ such that $H_j(x) = o(1/x^{2 + \Delta})$, implying that  $\E(\widetilde{Z}^{(j)}_\beta)^2 \leq \int_0^{ \infty }2yH_j(y)dy < \infty$ regardless of the actual value of $\eta > 0$.
\end{itemize}
In any case, using Bernstein's inequality, we can conclude that
\begin{align*}
    \cref{proof lemma bernstein prior to bernstein bound Rd}
    & \leq
    \ceil{t/\eta}^i \cdot 2\exp\big(  - \frac{ \frac{1}{2}\big(\frac{\epsilon}{12\sqrt{m+1}}\big)^2/\eta^2  }{ \ceil{t/\eta} \E(\widetilde{W}^{(j)})^2 + \frac{1}{3}\big(\frac{\epsilon}{12\sqrt{m+1}}\big)/\eta^{1 - \beta} } \big) = o(\eta^N).
\end{align*}
On the other hand,
\begin{align*}
\text{(II)}&
\leq 
\P( I \geq k) 
\leq 
{\ceil{t/\eta}\choose k}\cdot \P\Big(\widetilde{Z}^{(j)}(n) > 1/\eta^\beta \ \forall n=1,\ldots,k\Big)
\leq 
\ceil{t/\eta}^k\cdot \Big(H_j(1/\eta^\beta)\Big)^k,
\end{align*}
which is upper bounded by a function regularly varying w.r.t.\ $\eta$ with index $k\big(\beta\alpha_1 - 1\big) > N$.
Therefore, $\text{(II)} = o(\eta^N)$ and we conclude the proof.
\end{proof}

\subsection{Analyzing different scenarios before the first exit or return}

Using results and arguments above, we are able to illustrate the typical behavior of the SGD iterates $X^\eta_n$ in the following two scenarios. First, we show that, when starting from anywhere in $\mathcal{G}$, the SGD iterates $X^\eta_n$
 will most likely return to the neighborhood of the $\bm{0}$ within a short period of time without exiting $\Omega$. 
\begin{lemma} \label{lemma return to local minimum quickly Rd}
For any $\epsilon \in (0,\bar{\epsilon})$, the following claim holds:
\begin{align*}
   \lim_{\eta \downarrow 0} \inf_{ x : \bm{d}(x,\mathcal{G}^c)>\epsilon } \mathbb{P}_x\Big( X^\eta_n \in \mathcal{G} \ \forall n \leq T_\text{return}(\eta,\epsilon) \text{and }T_\mathrm{return}(\eta,\epsilon) \leq \rho(\epsilon)/\eta  \Big) = 1
\end{align*}
where the stopping time involved is defined as
\begin{align*}
    T_\text{return}(\eta,\epsilon) \delequal{} \min\{ n \geq 0: \norm{X^\eta_n} < 2\epsilon \}
\end{align*}
and the function $\rho(\epsilon) = c_1 + c_1\log(2/\epsilon)$ with $c_1 \in (0,\infty)$ being the constant specified in Lemma \ref{lemma bound on return time Rd}.
\end{lemma} 

\begin{proof}
Let $t = \rho(\epsilon) =  c_1 + c_1\log(2/\epsilon)$ and $n = \ceil{t/\eta}$ 
and arbitrarily choose some $x$ with $\bm{d}(x,\mathcal{G}^c) > \epsilon.$
Also, fix some $N > 0$. 
Due to Lemma \ref{lemma prob event A Rd}, we are able to pick some $\delta > 0$ such that (note that $\eta n \leq 2t$ eventually as $\eta \downarrow 0$)
\begin{align*}
    \mathbb{P}\Big( \max_{ n = 1,2,\cdots,\ceil{ t/\eta } } \eta\norm{ Z^{\leq\delta/\eta}_1 + \cdots + Z^{\leq\delta/\eta}_n } > \frac{ \epsilon }{ 2\exp(2\eta n M) } \Big) = o(\eta^N)
\end{align*}
First, from Lemma \ref{lemma Ode Gd Gap Rd} we know that for any $\eta$ small enough such that
\begin{align*}
    2\eta M \exp(2tM) < \epsilon/2,
\end{align*}
we must have 
\begin{align*}
    \norm{\bm{x}^\eta(s,x) - \bm{y}^\eta_{\floor{s}}(x)} \leq \epsilon/2 \ \ \ \forall s \in [0,n].
\end{align*}
Next, using Lemma \ref{lemma SGD GD gap Rd}, we see that on event
\begin{align*}
    A(n,\eta,\frac{\epsilon}{2\exp(\eta n M) }, \delta) \bigcap \{ T^\eta_1(\delta) > \ceil{t/\eta} \},
\end{align*}
we must have
\begin{align*}
    \norm{ \bm{y}^\eta_k(x) - X^\eta_k(x)  } < \epsilon\ \ \ \forall k = 0,1,\cdots,\ceil{t/\eta}.
\end{align*}
Combining these facts with the Lemma \ref{lemma bound on return time Rd}, we have shown that, for all $\eta \in (0, \frac{\epsilon}{4M\exp(2tM)})$, on event $A(n,\eta,\frac{\epsilon}{2\exp(\eta n M) }, \delta) \bigcap \{ T^\eta_1(\delta) > \ceil{t/\eta} \}$ we must have $X^\eta_n \in \mathcal{G} \ \forall n \leq T_\text{return}(\eta,\epsilon) \text{ and }T_\mathrm{return}(\eta,\epsilon) \leq \rho(\epsilon)/\eta$. 
In other words,
\begin{align*}
    & \lim_{\eta \downarrow 0}\sup_{ x : \bm{d}(x,\mathcal{G}^c)>\epsilon } \mathbb{P}_x\Big( X^\eta_n \notin \mathcal{G} \ \text{for some } n \leq T_\text{return}(\eta,\epsilon) \text{ OR }T_\mathrm{return}(\eta,\epsilon) > \rho(\epsilon)/\eta  \Big)
    \\
    \leq & \lim_{\eta \downarrow 0}\mathbb{P}\big( A(n,\eta,\frac{\epsilon}{2\exp(\eta n M) }, \delta)^c \big) + \lim_{\eta \downarrow 0}\mathbb{P}( T^\eta_1(\delta) \leq \ceil{t/\eta} ).
\end{align*}
However, our choice of $\delta$ at the beginning of the proof implies that $\lim_{\eta \downarrow 0}\mathbb{P}\big( A(n,\eta,\frac{\epsilon}{2\exp(\eta n M) }, \delta)^c \big) = 0$.
Meanwhile, from Assumption \ref{assumption 6 regular variation of the Rd noise}, we see that
\begin{align*}
    \mathbb{P}( T^\eta_1(\delta) \leq \ceil{t/\eta} ) = &\sum_{k = 1}^{ \ceil{t/\eta} }\big(1 - \sum_{j =0}^m H_j(\delta/\eta) \big)^{k-1} \cdot( \sum_{j =0}^m H_j(\delta/\eta) )
    \\
    \leq & (m+1)H_1(\delta/\eta)\sum_{k = 1}^{ \ceil{t/\eta} }\big(1 -  H_1(\delta/\eta) \big)^{k-1} 
\end{align*}
for all $\eta$ sufficiently small (due to regularly varying nature of $H_j$ and $1<\alpha_1<\cdots<\alpha_m$). In particular, since $H_1 \in RV_{-\alpha_1}$, for any fixed $\Delta \in (0,\alpha_1 - 1)$, we have (for all $\eta$ sufficiently small),
\begin{align*}
    \mathbb{P}( T^\eta_1(\delta) \leq \ceil{t/\eta} )
    \leq & (m+1)\eta^{\alpha_1 - \Delta}\sum_{k = 1}^{ \ceil{t/\eta} }( 1 - \eta^{\alpha_1 + \Delta})^{k-1}  = \frac{m+1}{\eta^{2\Delta}}\mathbb{P}( U(\eta) \leq \ceil{t/\eta} )
\end{align*}
where $U(\eta) \distequal{} Geom( \eta^{ \alpha_1 + \Delta}  )$.
Lemma \ref{lemmaGeomFront} then tells us that, for all $\eta$ sufficiently small, $ \mathbb{P}( U(\eta) \leq \ceil{t/\eta} ) \leq \frac{2}{t}\eta^{\alpha_1 - 1 + \Delta} $, thus implying
\begin{align*}
    \limsup_{\eta \downarrow 0}\mathbb{P}( T^\eta_1(\delta) \leq \ceil{t/\eta} ) \leq \limsup_{\eta \downarrow 0}\frac{2(m+1)}{t}\eta^{\alpha_1 - 1 - \Delta} = 0
\end{align*}
 and concluding the proof.
\end{proof}

In the next result, we show that, once entering a $\epsilon-$small neighborhood of the local minimum, the SGD iterates will most likely stay there until the next large jump.

\begin{lemma} \label{lemma stuck at local minimum before large jump Rd}
Given $N_0 > 0$, the following claim holds for all $\epsilon \in (0,\bar{\epsilon}/3)$ and all $\delta>0$ that is sufficiently small:
\begin{align*}
    \sup_{x \in {B(\bm{0},2\epsilon)}}\mathbb{P}\Big( \exists n < T^\eta_1(\delta)\ s.t.\ \norm{X^\eta_n(x)} > 3\epsilon \Big) = o(\eta^{N_0})
\end{align*}
as $\eta \downarrow 0$.
\end{lemma}

\begin{proof}
Let $t = c_1 + c_1\log(1/\epsilon)$ where $c_1 \in (0,\infty)$ is the constant specified in Lemma \ref{lemma bound on return time Rd}, and $n = \ceil{t/\eta}$.
Fix some $\beta > 1 + \alpha_1$ and $N > \beta - 1 + N_0$.
Due to Lemma \ref{lemma prob event A Rd}, the following claim holds for all $\delta$ that is sufficiently small:
\begin{align*}
    \mathbb{P}\Big( \max_{ k = 1,2,\cdots,\ceil{ t/\eta } } \eta\norm{ Z^{\leq\delta/\eta}_1 + \cdots + Z^{\leq\delta/\eta}_k } > \frac{ \epsilon/2 }{ 2\exp(2\eta n M) } \Big) = o(\eta^N)
\end{align*}
We fix one of such $\delta$.
Therefore, for event
\begin{align*}
    \widetilde{A} \delequal &\Big\{ \exists j \leq \ceil{ \frac{ 1/\eta^\beta }{ \ceil{t/\eta} } }\text{ s.t. }
    \\
    &\ \ \ \ \ \ \ \ \ \  \max_{ k = 1,2,\cdots,\ceil{ t/\eta } } \eta\norm{ Z^{\leq\delta/\eta}_{1 + (j-1)\ceil{ t/\eta }} + Z^{\leq\delta/\eta}_{2 + (j-1)\ceil{ t/\eta }} + \cdots + Z^{\leq\delta/\eta}_k } > \frac{ \epsilon/2 }{ 2\exp(2\eta n M) } \Big\},
\end{align*}
the iid nature of noises $(Z_n)_{n \geq 1}$ implies that
\begin{align*}
    \P(\widetilde{A}) \leq & (1 + \frac{1}{t\eta^{\beta - 1}}) \mathbb{P}\Big( \max_{ k = 1,2,\cdots,\ceil{ t/\eta } } \eta\norm{ Z^{\leq\delta/\eta}_1 + \cdots + Z^{\leq\delta/\eta}_k } > \frac{ \epsilon/2 }{ 2\exp(2\eta n M) } \Big)
    \\
    = & o( \eta^{ N - \beta + 1} )\ \ \text{as }\eta \downarrow 0.
\end{align*}
Meanwhile,
Note that $T^\eta_1(\delta) \leq T^{\eta,(1)}_1(\delta) \delequal \min\{n \geq 0:\ \eta|Z_n| > \delta,\ Z_n \in E_1\}$.
Therefore,
$ \mathbb{P}( T^\eta_1(\delta) > 1/\eta^\beta ) \leq \P(T^{\eta,(1)}_1(\delta) > 1/\eta^\beta)$.
Specifically, note that $T^\eta_1(\delta) \distequal{} Geom( H_1(\delta/\eta) )$ and $H_1 \in RV_{-\alpha_1}$.
Lemma \ref{lemmaGeomDistTail} then tells us the existence of some $c > 0$ such that
\begin{align*}
    \mathbb{P}( T^\eta_1(\delta) > 1/\eta^\beta ) = o\big( \exp(-c/\eta) \big)
\end{align*}
as $\eta \downarrow 0$. In summary, we have established that
\begin{align*}
    \mathbb{P}( \widetilde{A}\cup\{  T^\eta_1(\delta) > 1/\eta^\beta \} ) = o(\eta^{N_0}).
\end{align*}
 
Now let us focus on the complementary event $(\widetilde{A})^c \cap \{  T^\eta_1(\delta) \leq 1/\eta^\beta \}$. On this event, there must be some $j \leq \ceil{ \frac{ 1/\eta^\beta }{ \ceil{t/\eta} }}$ such that $1 + (j-1)\ceil{t/\eta} \leq T^\eta_1(\delta) \leq j\ceil{t/\eta}$. 

Let us arbitrarily choose some $x \in B(\bm{0}, 2\epsilon)$. If $j = 1$, then due to \cref{ineq ODE geometric convergence to origin}, we know that $\bm{x}^\eta(s,x) \in B(\bm{0}, 2\epsilon)$ for all $s \geq 0$. Also, from Lemma \ref{lemma Ode Gd Gap Rd} we know that for any $\eta$ small enough such that $2\eta M \exp(2tM) < \epsilon/2,$
we must have 
\begin{align*}
    \norm{\bm{x}^\eta(s,x) - \bm{y}^\eta_{\floor{s}}(x)} \leq \epsilon/2 \ \ \ \forall s \in [0,\ceil{t/\eta}].
\end{align*}
Besides, on event $(\widetilde{A})^c \cup \{  T^\eta_1(\delta) \leq 1/\eta^\beta \}\cap\{j = 1\}$, due to Lemma \ref{lemma SGD GD gap Rd},
we must have
\begin{align*}
    \norm{ \bm{y}^\eta_k(x) - X^\eta_k(x)  } < \epsilon/2\ \ \ \forall k = 0,1,\cdots,T^\eta_1(\delta)
\end{align*}
 with $T^\eta_1(\delta) < \ceil{t/\eta}.$ In conclusion, in the case that $j = 1$, we must have $\norm{X^\eta_n(x)} < 3\epsilon$ for all $n < T^\eta_1(\delta)$. 
 
Now consider the case with $j \geq 2$. Similarly, due to Lemma \ref{lemma Ode Gd Gap Rd} and \ref{lemma SGD GD gap Rd}, we now have that
\begin{align*}
    \norm{\bm{x}^\eta(s,x) - \bm{y}^\eta_{\floor{s}}(x)} & \leq \epsilon/2 \ \ \ \forall s \in [0,\ceil{t/\eta}], \\
    \norm{ \bm{y}^\eta_k(x) - X^\eta_k(x)  } & < \epsilon/2\ \ \ \forall k = 0,1,\cdots,\ceil{t/\eta}, \\
    \norm{\bm{x}^\eta(s,x)} & < 2\epsilon\ \ \forall s \in [0,\ceil{t/\eta}].
\end{align*}
In particular, due to our choice of $t$ at the beginning of the proof and Lemma \ref{lemma bound on return time Rd}, one can see that $\norm{\bm{x}^\eta(\ceil{t/\eta},x)} < \epsilon.$ Collecting all the results, we now know that
\begin{align*}
    \norm{ X^\eta_k(x) } < 3\epsilon\ \forall k \leq \ceil{t/\eta},\ \ \norm{ X^\eta_{ \ceil{t/\eta} }(x) } < 2\epsilon.
\end{align*}
Now it suffices to apply the repeated apply the same arguments above. For instance, if $j = 2$, then by letting $x_1 = X^\eta_{ \ceil{t/\eta} }(x)$ and bounding the gap between trajectory of $X^\eta_{ \floor{s} + \ceil{t/\eta}  }(x)$ and $\bm{x}^\eta(s,x_1)$ using Lemma \ref{lemma Ode Gd Gap Rd} and \ref{lemma SGD GD gap Rd}, one can show that $\norm{X^\eta_n(x)} < 3\epsilon$ for all $n<T^\eta_1(\delta)$ (since in the case of $j=2$, we have $T^\eta_1(\delta)<2\ceil{t/\eta}$). Otherwise, we have $j \geq 3$ and the same arguments above can be applied to show that 
\begin{align*}
    \norm{ X^\eta_k(x) } < 3\epsilon\ \forall k \leq 2\ceil{t/\eta},\ \ \norm{ X^\eta_{ \ceil{t/\eta} }(x) } < 2\epsilon, ,\ \ \norm{ X^\eta_{ 2\ceil{t/\eta} }(x) } < 2\epsilon.
\end{align*}
In summary, by proceeding inductively, we establish that for all $\eta \in (0, \frac{\epsilon}{4M\exp(2tM)})$,
\begin{align*}
    \sup_{ x\in B(\bm{0}, 2\epsilon) }\mathbb{P}\Big( \exists n < T^\eta_1(\delta)\ s.t.\ \norm{X^\eta_n(x)} > 3\epsilon \Big) \leq \mathbb{P}( \widetilde{A}\cup\{  T^\eta_1(\delta) > 1/\eta^\beta \} )
\end{align*}
and this concludes the proof.
\end{proof}



We introduce a few concepts that will be crucial to the analysis below. Recall the definition of ODE with jumps $\widetilde{\bm{x}}$ 
(note that we will drop the notational dependency on learning rate $\eta$ when we choose $\eta = 1$). 
Recall the definitions of $\bm{i}^*,k^*$ (see Assumption \ref{assumption unique configuration for minimum exit cost} and \ref{assumption bounded away} and the remarks underneath). We define the following mapping $h$ from $\bm{r} = (r_1,\cdots,r_{k^*}), \bm{\theta} = (\theta_1,\cdots,\theta_{k^*}) \in (\mathbb{S}^{d-1})^{k^*}, \bm{t} = (t_1,\cdots,t_{k^*}) \in \{0\}\times\mathbb{R}^{k^*}_+$ such that
\begin{align}
    h(\bm{r},\bm{\theta},\bm{t}) = \widetilde{\bm{x}}(\sum_{i = 1}^{k^*}t_i, \bm{0};\ \bm{t},\bm{w}) \label{def mapping h Rd}
\end{align}
where $\bm{w} = (w_i)_{i = 1}^{k^*}$ with $w_i = r_i\theta_i$. 
From the continuity of the ODE flow, we see that $h$ is a continuous mapping.
Also, we introduce the concept of \emph{type} for configuration $\bm{i}^*$. Specifically, define
\begin{align}
\bm{j}(i_1,\cdots,i_m) \delequal \big\{ (j_1,\cdots,j_{k^*}) \in \{ 0,1,2,\cdots,m \}^{k^*}:\ \#\{n:\ j_n = k\} = i_k\ \forall k \in [m]  \big\} \label{def j i set}
\end{align}
and for any $\bm{j} \in \bm{j}(\bm{i}^*)$, we say that $\bm{w} = (w_1,\cdots,w_{k^*}) \in \mathcal{A}(\bm{i}^*)$ has type $\bm{j}$ if
\begin{align*}
    w_i \in E_{j_i}\ \forall i \in [k^*].
\end{align*}
In other words, based on the direction of each jump in $\bm{w} \in \mathcal{A}(\bm{i}^*)$ we group them into different types. Note that $| \bm{j}(\bm{i}^*) |<\infty$.
Now for any type $\bm{j} \in \bm{j}(\bm{i}^*)$, define a (Borel) measure on $\mathbb{R}^{k^*}_+ \times ( \mathbb{S}^{d-1} )^{k^*} \times \mathbb{R}^{k^* - 1}_+$ as
\begin{align}
    \mu_{\bm{j}}\delequal (\prod_{i = 1}^{k^*}\nu_{\alpha_{\bm{j}_i}}) \times ( \prod_{i=1}^{k^*} S_{ \bm{j}_i } ) \times \textbf{Leb}^{k^*-1} \label{def measure mu j Rd}
\end{align}
where, for any $\alpha > 0$, the measure $\nu_\alpha$ is the Borel measure on $(0,\infty)$ with $\nu_\alpha(x,\infty) = 1/x^{\alpha}$.
We use the concepts above to characterize behavior of large noises in $(Z_n)_{n \geq 1}.$

\begin{definition}\label{definitionTypeJ Rd}
For any $\eta,\delta > 0, \bm{j} \in \bm{j}(\bm{i}^*)$ and any Borel set $A \subseteq \mathbb{R}^d$, we say noise $Z_n$ is of \textbf{type-$(A,\delta,\eta,\bm{j})$} iff
\begin{itemize}
    \item $\eta\norm{Z_n} > \delta$;
    \item For $\widetilde{t}_1 = n$ and $\widetilde{t}_i \delequal \min\{ k > \widetilde{t}_{i-1}:\ \eta\norm{Z_k} > \delta  \}$ for all $i = 2,3,\cdots,k^*$, it holds that $\widetilde{t}_{i} - \widetilde{t}_{i-1} < 2\bar{t}/\eta$ for all $i = 2,3,\cdots,k^*$;
    \item For $t_i \delequal \widetilde{t}_{i} - \widetilde{t}_{i-1}$ (with $t_1 = 0$), $w_i = \eta Z_{\widetilde{t}_i}$, and
    \begin{align*}
        \bm{r} & \delequal (\norm{w_1},\norm{w_2},\cdots,\norm{w_{k^*}}), \\
        \bm{\theta} & = ( w_1/\norm{w_1},\cdots,w_{k^*}/\norm{w_{k^*}}  ), \\
        \bm{t} & = (t_1,\cdots,t_{k^*}),
    \end{align*}
it holds that $h(\bm{r},\bm{\theta},\bm{t}) \in A$. 
\end{itemize}
More generally, we say noise $Z_n$ is of \textbf{type-$(A,\delta,\eta)$} iff there exists $\bm{j} \in \bm{j}(\bm{i}^*)$ such that $Z_n$ is of \textbf{type-$(A,\delta,\eta,\bm{j})$}.
\end{definition}
Due to the iid nature of $(Z_j)_{j \geq 1}$, let us consider an iid sequence $(V_j)_{j \geq 0}$ where the sequence has the same law of $Z_1$. 
Note that for any fixed $n \geq 1$, the probability that $Z_n$ is of type-$(A,\delta,\eta,\bm{j})$ is equal to the probability that $V_0$ is of type-$(A,\delta,\eta,\bm{j})$. 
More specifically, 
let $H(x) = \sum_{j = 0}^m H_j(x) = \mathbb{P}( \norm{Z_1} > x )$.
Due to Assumption \ref{assumption 6 regular variation of the Rd noise}, $H\in RV_{-\alpha_1}$.
Besides, $\mathbb{P}(\eta\norm{V_0} > \delta) = H(\delta/\eta)$. 
Now we focus on conditional probability admitting the following form:
\begin{align*}
    p(A,\delta,\eta, \bm{j}) = \mathbb{P}\Big(V_0\ \text{is of type-$(A,\delta,\eta,\bm{j})$}\ \Big|\ \eta\norm{V_0} > \delta \Big). 
\end{align*}

\begin{lemma} \label{lemmaTypeJProb Rd}
For any $\delta > 0$, any $\bm{j} \in \bm{j}(\bm{i}^*)$, and any Borel set $A \subseteq \mathbb{R}^d$ such that
\begin{align}
    h( r_1,\cdots,r_{k^*}, \theta_1,\cdots,\theta_{k^*}, t_2,\cdots,t_{k^*}  ) \in A \ \Longrightarrow \ t_i < 2\bar{t}\ \forall i = 2,\cdots,k^* \text{ and }r_i > \delta\ \forall i \in [k^*], \label{lemmaTypeJ Rd bounded away assumption}
\end{align}
it holds that
\begin{align*}
    \mu_{\bm{j}}\big(h^{-1}(A^\circ)\big) 
    &\leq \liminf_\eta \frac{ p(A,\delta,\eta,\bm{j})  }{ \delta^{ \alpha_1}\cdot \frac{ \prod_{k = 1}^m \big( H_k(1/\eta) \big)^{i^*_k}  }{ H(1/\eta)\eta^{k^* -1}  }  } 
    \\
    &\leq \limsup_\eta \frac{ p(A,\delta,\eta,\bm{j})  }{ \delta^{\alpha_1}\cdot \frac{ \prod_{k = 1}^m \big( H_k(1/\eta) \big)^{i^*_k}  }{ H(1/\eta)\eta^{k^* -1}  }  } \leq \mu_{\bm{j}}\big(h^{-1}(\bar{A})\big).
\end{align*}
\end{lemma}

 \begin{proof}
Let us start by fixing some notations. Let $T_1 = 0$, and define stopping times $T_j = \min\{ n > T_{j - 1}: \eta\norm{V_n} > \delta \}$ and inter-arrival times $T^\prime_j = T_j - T_{j - 1}$ for any $j \geq 2$ (with $T^\prime_1 = 0$), 
and large jump $W_j = V_{T_j}$ for any $j \geq 0$. 
Note that: first, the pair $(T^\prime_i,W_i)$ is independent of $(T^\prime_j,W_j)$ for any $i \neq j$; besides, $W_j$ and $T^\prime_j$ are independent for all $j \geq 1$.

Define the following sequence (of random elements) $\bm{w} = (w_1,\cdots,w_{l^*})$ and $\bm{t} = (t_1,\cdots,t_{l^*})$ by
\begin{align*}
    w_i = \eta W_i,\ \ t_i = \eta T^\prime_i.
\end{align*}
and $\bm{r} = (\textbf{T}_r(w_i))_{i = 1}^{k^*}, \bm{\theta} = (\textbf{T}_\theta(w_i))_{i = 1}^{k^*}$ where operators $\textbf{T}_r(x) = \norm{x}, \textbf{T}_\theta(x) = x/\norm{x}$ constitute the polar coordinate transform.
By definition of {type-$(A,\delta,\eta,\bm{j})$}, we know that
\begin{itemize}
    \item $T^\prime_i \leq 2\bar{t}/\eta$ for any $i = 2,\cdots,k^*$;
    \item $\eta\norm{W_i}>{\delta}$ for any $i=1,2,\cdots,k^*$;
    \item $w_i \in E_{ \bm{j}_i }$ (i.e. $\theta_i \in F_i$) for all $i \in [k^*]$; 
    \item $h(\bm{r},\bm{\theta},\bm{t}) \in A$.
\end{itemize}
Therefore, (let $R_i = \textbf{T}_r(W_i), \Theta_i = \textbf{T}_\theta(W_i)$)
\begin{align}
    & p(A,\delta,\eta,\bm{j}) \nonumber \\
    = &\Big( \mathbb{P}(T^\prime_2 \leq 2\bar{t}/\eta ) \Big)^{k^* - 1} \cdot \int \mathbbm{1}\Big\{ ( \bm{r}, \bm{\theta}, \bm{t} ) \in h^{-1}(A) \Big\} \nonumber \\
    &\ \ \ \ \ \ \ \ \ \ \ \ \ \ \ \cdot \big(\prod_{i = 1}^{k^*}\mathbb{P}( W_i \in E_i )\cdot\mathbb{P}( \eta R_i = d r_i,\ \Theta_i = d\theta_i  )\big)\cdot\big(\prod_{i = 2}^{k^*}\mathbb{P}( \eta T^\prime_i = d t_i\ |\ \eta T^\prime_i \leq 2\bar{t} ) \big) \nonumber \\
    = & \Big( \mathbb{P}(T^\prime_2 \leq 2\bar{t}/\eta ) \Big)^{k^* - 1}\cdot \big[\prod_{i = 1}^{k^*}\mathbb{P}( W_i \in E_i )\big]\cdot \mathbb{Q}_{\eta,\delta,\bm{j}}\big(h^{-1}(A) \big)\label{proofCalculatePDeltaEpsilon Rd}
\end{align}
where $\mathbb{Q}_{\eta,\delta,\bm{j}}$ is the probability measure on $\mathbb{R}^{k^*}_+\times (\mathbb{S}^{d-1})^{k^*} \times \mathbb{R}^{k^*-1}_+$ induced by a sequence of random elements
$(\eta R_1,\cdots,\eta R_{k^*}, \Theta_1,\cdots,\Theta_{k^*}, \eta T^{\uparrow}_2(\eta),\cdots, \eta T^{\uparrow}_{k^*}(\eta))$ such that
\begin{itemize}
    \item For any $i = 1,\cdots,k^*$, we have $R_i = \textbf{T}_r(W^\uparrow_i(\eta)), \Theta_i = \textbf{T}_\theta(W^\uparrow_i(\eta))$ where the distribution of $W^{\uparrow}_i(\eta)$ follows the law $\mathbb{P}\Big( V_0 \in \cdot\ \Big|\ \eta\norm{V_0} > \delta,\ V_0 \in E_{ \bm{j}_i } \Big)$; Besides, $(W^\uparrow_i(\eta))_{i = 1}^{k^*}$ is an independent sequence;
    \item For any $i = 2,\cdots,l^*$, the distribution of $T^{\uparrow}_i(\eta)$ follows from $\mathbb{P}\Big( \eta T^\delta_1(\eta) \in \cdot\ \Big|\ \eta T^\delta_1(\eta) \leq 2\bar{t} \Big)$ (for the definition of stopping times $T^\delta_k(\eta)$, see \cref{defArrivalTime large jump_Rd}); Besides, $(T^\uparrow_i(\eta))_{i = 1}^{k^*}$ is an independent sequence, and it is independent of $(W^\uparrow_i(\eta))_{i = 1}^{k^*}$;
    \item $ \mathbb{Q}_{\eta,\delta,\bm{j}}(\cdot) = \mathbb{P}\Big(  (\eta R_1,\cdots,\eta R_{k^*}, \Theta_1,\cdots,\Theta_{k^*}, \eta T^{\uparrow}_2(\eta),\cdots, \eta T^{\uparrow}_{k^*}(\eta))   \in \cdot  \Big)$.
\end{itemize}
As for weak convergence of $(\eta R_1,\cdots,\eta R_{k^*}, \Theta_1,\cdots,\Theta_{k^*}, \eta T^{\uparrow}_2(\eta),\cdots, \eta T^{\uparrow}_{k^*}(\eta))$, observe that
\begin{itemize}
    \item For any $i \in [k^*]$, from Assumption \ref{assumption 6 regular variation of the Rd noise} one can see the regularly varying nature of distribution of $V_0$ on the cone $E_{\bm{j}_i}$ (hence for $W^\uparrow_i(\eta)$ as well), 
    thus yielding that $(\eta R_i,\Theta_i)\Rightarrow (R^*_i,\Theta^*_i)$ as $\eta \downarrow 0$
    where $R^*_i$ and $\Theta^*_i$ are independent,
    the law of $\Theta^*_i$ is $S_{\bm{j}_i}$,
    and the law of $R^*_i$ is the Pareto distribution with 
    $$\mathbb{P}(R^*_i > x) = \frac{\delta^{ \alpha_{\bm{j}_i} }}{( x \vee \delta )^{  \alpha_{\bm{j}_i} }};$$
    \item For any $x \in [0,2\bar{t}]$, since $\lim_{\eta \downarrow 0}\floor{x/\eta}H(\delta/\eta) = 0$, it is easy to show that
    $$\lim_{\eta \downarrow 0}\frac{1 - (1 - H(\delta/\eta))^{\floor{x/\eta}} }{ \floor{x/\eta}H(\delta/\eta)  } = 1;$$
    therefore, we have (for any $x \in (0,2\bar{t}]$)
    \begin{align*}
        \mathbb{P}(\eta T_1^\delta(\eta) \leq x\ |\ \eta T_1^\delta(\eta) \leq 2\bar{t}) & = \frac{ 1 - (1 - H(\delta/\eta))^{\floor{x/\eta}}  }{1 - (1 - H(\delta/\eta))^{\floor{2\bar{t}/\eta}}  } \rightarrow \frac{x}{2\bar{t}}
    \end{align*}
    as $\eta\downarrow 0$, which implies that $T^{\uparrow}_i(\eta)$ converges weakly to a uniform RV on $[0,2\bar{t}]$.
\end{itemize}
Together with the assumption on set $A$ in \cref{lemmaTypeJ Rd bounded away assumption},
we can now see that, if we denote the weak limit of measure $\mathbb{Q}_{\eta,\delta, \bm{j}}$ as $\mu_{\delta,\bm{j}}$,
then (the measure $\mu_{\bm{j}}$ is defined in \cref{def measure mu j Rd})
\begin{align*}
    \mu_{\delta,\bm{j}}\big(h^{-1}(A)\big) = \frac{ \prod_{i = 1}^{k^*}\delta^{ \alpha_{\bm{j}_i}}  }{ (2\bar{t})^{k^*-1} }\mu_{\bm{j}}\big(h^{-1}(A)\big).
\end{align*}
From the continuity of mapping $h$, one can see that $h^{-1}(\bar{A})$ is a closed set and $h^{-1}(A^\circ)$ is an open set. Using Portmanteau theorem, we now have
\begin{align}
    \frac{ \prod_{i = 1}^{k^*}\delta^{ \alpha_{\bm{j}_i}}  }{ (2\bar{t})^{k^*-1} }\mu_{\bm{j}}\big(h^{-1}(A^\circ)\big) 
    & \leq \liminf_\eta \mathbb{Q}_{\eta,\delta, \bm{j}}\big(h^{-1}(A)\big) \leq \limsup_\eta \mathbb{Q}_{\eta,\delta, \bm{j}}\big(h^{-1}(A)\big)
    \\
    & \leq \frac{ \prod_{i = 1}^{k^*}\delta^{ \alpha_{\bm{j}_i}}  }{ (2\bar{t})^{k^*-1} }\mu_{\bm{j}}\big(h^{-1}(\bar{A})\big). \label{proof Type J Rd prob bound term 1}
\end{align}

Moving on, we analyze the limit of the other terms in \cref{proofCalculatePDeltaEpsilon Rd}.
For any $i \in [k^*]$, note that
\begin{align*}
    \mathbb{P}( W_i \in E_i ) = \frac{ H_{\bm{j}_i}(\delta/\eta) }{ H(\delta/\eta) }.
\end{align*}
Now due to the regularly varying nature of $H_j$ and $H$,
\begin{align}
    \lim_{\eta} \frac{ \prod_{i = 1}^{k^*}\mathbb{P}( W_i \in E_i )  }{ \prod_{i = 1}^{k^*}H_{\bm{j}_i}(1/\eta)/H(1/\eta)  } = \frac{ \delta^{ k^*\alpha_1 }  }{ \prod_{i = 1}^{k^*}\delta^{\alpha_{\bm{j}_i}}  }. \label{proof Type J Rd prob bound term 2}
\end{align}
On the other hand, from Lemma \ref{lemmaGeomFront} and the regularly varying nature of function $H$, for any fixed $\kappa > 1$, we have
\begin{align*}
    \limsup_{\eta \downarrow 0}\Big( \frac{\delta^{\alpha_1}}{2\bar{t}}\cdot \frac{\mathbb{P}(T^\prime_2 \leq 2\bar{t}/\eta )}{H(1/\eta)/\eta}  \Big)^{k^* - 1} & \leq \kappa^{k^* - 1}\limsup_{\eta \downarrow 0}\Big( \frac{\delta^{\alpha_1}}{2\bar{t}}\cdot \frac{2\bar{t}H(\delta/\eta)/\eta}{H(1/\eta)/\eta} \Big)^{k^* - 1} = \kappa^{k^* - 1}, 
    \\
    \liminf_{\eta \downarrow 0}\Big( \frac{\delta^{\alpha_1}}{2\bar{t}}\cdot \frac{\mathbb{P}(T^\prime_2 \leq 2\bar{t}/\eta )}{H(1/\eta)/\eta}  \Big)^{k^* - 1} & \geq (1/\kappa)^{k^* - 1}\liminf_{\eta \downarrow 0}\Big( \frac{\delta^{\alpha_1}}{2\bar{t}}\cdot \frac{2\bar{t}H(\delta/\eta)/\eta}{H(1/\eta)/\eta} \Big)^{k^* - 1}
    \\
    & = (1/\kappa)^{k^* - 1}.
\end{align*}
Due to arbitrariness of $\kappa > 1$, we yield that
\begin{align}
    \lim_\eta \frac{  \big( \mathbb{P}(T^\prime_1 \leq 2\bar{t}/\eta ) \big)^{k^* - 1}  }{ \Big( \frac{H(1/\eta)}{\eta}\cdot \frac{2\bar{t}}{ \delta^{\alpha_1} } \Big)^{k^* - 1}   } = 1. \label{proof Type J Rd bound term 3}
\end{align}
Collecting all the limits in \cref{proof Type J Rd prob bound term 1}-\cref{proof Type J Rd bound term 3} and plugging them into \cref{proofCalculatePDeltaEpsilon Rd}, we now have established that
\begin{align*}
    \mu_{\bm{j}}\big(h^{-1}(A^\circ)\big) \leq \liminf_\eta \frac{ p(A,\delta,\eta,\bm{j})  }{ \delta^{\alpha_1}\cdot \frac{ \prod_{i = 1}^{k^*} H_{\bm{j}_i}(1/\eta)  }{ H(1/\eta)\eta^{k^* -1}  }  } \leq \limsup_\eta \frac{ p(A,\delta,\eta,\bm{j})  }{ \delta^{\alpha_1}\cdot \frac{ \prod_{i = 1}^{k^*} H_{\bm{j}_i}(1/\eta)  }{ H(1/\eta)\eta^{k^* -1}  }  } \leq \mu_{\bm{j}}\big(h^{-1}(\bar{A})\big).
\end{align*}
To conclude the proof, recall that for any type $\bm{j} \in \bm{j}(\bm{i}^*)$, we have $\#\{k = 1,2,\cdots,k^*:\ \bm{j}_k = j \} = i^*_j$ (i.e. for any $j \in [m]$, the number of elements in $ (\bm{j}_1,\cdots,\bm{j}_{k^*})$ that are equal to $j$ is exactly $i^*_j$), thus $ \prod_{i = 1}^{k^*} H_{\bm{j}_i}(1/\eta) = \prod_{k = 1}^m \big( H_k(1/\eta) \big)^{i^*_k}$.
\end{proof}
More generally, for
\begin{align}
    p(A,\delta,\eta) = \mathbb{P}\Big(V_0\ \text{is of type-$(A,\delta,\eta)$}\ \Big|\ \eta\norm{V_0} > \delta \Big), \label{def Type conditional probability generalized}
\end{align}
we know that $p(A,\delta,\eta) = \sum_{ \bm{j} \in \bm{j}(\bm{i}^*) } p(A,\delta,\eta,\bm{j})$.
Also, define measure $\mu$ as
\begin{align}
    \mu = \sum_{ \bm{j} \in \bm{j}(\bm{i}^*) }\mu_{\bm{j}}. \label{def measure mu Rd}
\end{align}
The next result follows immediately from Lemma \ref{lemmaTypeJProb Rd} and the fact that $|\bm{j}(\bm{i}^*)|<\infty$. Recall that $J_\mathcal{G}$ is defined in \cref{def minimum cost J_G}.
\begin{corollary} \label{CorTypeProbRd}
There exists a function $\widetilde{\lambda}(\eta) \in RV_{1 + J_\mathcal{G} - \alpha_1 }(\eta)$ such that, given any $\delta > 0$ and any Borel set $A \subseteq \mathbb{R}^d$ with
\begin{align}
    h( r_1,\cdots,r_{k^*}, \theta_1,\cdots,\theta_{k^*}, t_2,\cdots,t_{k^*}  ) \in A \ \Longrightarrow \ t_i < 2\bar{t}\ \forall i = 2,\cdots,k^* \text{ and }r_i > \delta\ \forall i \in [k^*], 
\end{align}
it holds that
\begin{align*}
    \mu\big(h^{-1}(A^\circ)\big) \leq \liminf_\eta \frac{ p(A,\delta,\eta)  }{ \delta^{ \alpha_1}\cdot \widetilde{\lambda}(\eta)  } \leq \limsup_\eta \frac{ p(A,\delta,\eta)  }{ \delta^{\alpha_1}\cdot \widetilde{\lambda}(\eta)  } \leq \mu\big(h^{-1}(\bar{A})\big).
\end{align*}
In particular, the regularly varying function $\widetilde{\lambda}(\cdot)$ admits the form
\begin{align}
    \widetilde{\lambda}(\eta) = \frac{ \prod_{j= 1}^m \big( H_j(1/\eta) \big)^{i^*_j}  }{ H(1/\eta)\eta^{k^* -1}  } \label{def scaling function tilde lambda}
\end{align}
\end{corollary}

Consider the following stopping times
\begin{align*}
    \sigma(\eta) & = \min\{n \geq 0: X^\eta_n \notin \mathcal{G}\}; 
    \\
    R(\epsilon,\delta,\eta) & = \min\{n \geq T^\eta_1(\delta): \norm{X^\eta_n} \leq 3\epsilon \}. 
\end{align*}
Here $\sigma$ indicate the first time that the iterates exit domain $\mathcal{G}$, 
while $R$ denotes the time the SGD iterates return to a small neighborhood of the attractor $\bm{0}$ 
after the first large jump. 
In the next few results, we study the probability of the different scenarios regarding the first exit time $\sigma(\eta)$ and first return time $R(\epsilon,\delta,\eta)$.
To this end, let
\begin{align*}
    K^{\eta}(\delta) \delequal \max\{ k \geq 0:\ T^\eta_k(\delta) \leq \sigma(\eta) \wedge R(\epsilon,\delta,\eta) \}
\end{align*}
be the count of large jumps before the first exit or first return.
Furthermore, we introduce the following concept as the \emph{accumulated cost} of large jumps before the first exit or return.
Let $\mathcal{J}^{\eta,\delta}_0 = 0$ and 
\begin{align*}
    \mathcal{J}^{\eta,\delta}_k \delequal \mathcal{J}^{\eta,\delta}_{k-1} + J( W^\eta_{k}(\delta) )\ \ \ \forall k \leq K^{\eta}(\delta)
\end{align*}
where the cost function $J(\cdot)$ is defined in \cref{def jump cost function J}. 
In other words, $\mathcal{J}^{\eta,\delta}_k$ indicates the total cost of the first $k$ large jumps, if there are at least $k$ large jumps, before the first exit or first return.
Similarly, we can also define a step-wise accumulated cost as
\begin{align*}
    \mathcal{J}^{\downarrow}_{\eta,\delta}(n) = \max\{ \mathcal{J}^{\eta,\delta}_k:\ T^\eta_k(\delta) \leq n \}\ \forall n \leq \sigma(\eta) \wedge R(\epsilon,\delta,\eta).
\end{align*}
In other words, $\mathcal{J}^{\downarrow}_{\eta,\delta}(n)$ evaluates the total cost of large jumps up until step $n$.
As a preparation for our analyses below, we first discuss the following technical tools. 
For any $\epsilon > 0$, let 
$\hat{t}(\epsilon) = c_1 + c_1 \log(1/\epsilon)$ where $c_1$ is the constant in Lemma \ref{lemma bound on return time Rd}.
Besides, define
\begin{align}
    \widetilde{\epsilon}(\epsilon) \delequal \frac{\epsilon}{ 4\exp(2\hat{t}M) } \wedge \frac{\epsilon}{\bar{\rho}( \hat{t}(\epsilon ))} \wedge \frac{\epsilon}{\widetilde{\rho}( \hat{t}(\epsilon ))} \label{def function tilde epsilon Rd}
\end{align}
where functions $\bar{\rho}(\cdot), \widetilde{\rho}(\cdot)$ are defined in Corollary \ref{corollary ODE GD gap Rd} and \ref{corollary sgd gd gap Rd} respectively.
Furthermore, define an event $A^\times = A^\times_1(\epsilon,\delta,\eta) \cup A^\times_2(\epsilon,\delta,\eta)$ where
\begin{align}
     A^\times_1(\epsilon,\delta,\eta) & \delequal \Big\{ \exists n < T^\eta_1(\delta)\ s.t.\ \norm{X^\eta_n(x)} > 3\epsilon \Big\}, \label{def event A cross 1 for A cross Rd} \\
     A^\times_2(\epsilon,\delta,\eta) & \delequal{} \Big\{ \exists j = 2,\cdots,l^*\ s.t.\nonumber 
     \\
     &\ \ \ \ \ \ \ \ \ \ \max_{k = 1,2,\cdots,(T^\eta_j(\delta) -T^\eta_{j-1}(\delta) - 1) \wedge (2\hat{t}(\epsilon)/\eta) } \eta\norm{ Z_{ T^\eta_{j-1}(\delta) + 1 } + \cdots + Z_{ T^\eta_{j-1}(\delta) + j } } > \widetilde{\epsilon}(\epsilon) \Big\} \label{def event A cross 2 for A cross Rd}
\end{align}
where the positive integer $l^*$ is defined in \cref{def l star Rd}.
By definition of $l^*$, we must have $K^\eta(\delta) < l^*$.
Lastly, define events (for any $K \in [l^*]$)
\begin{align*}
    B^\times_1(K) & \delequal (A^\times)^c \cap \{\mathcal{J}^{\eta,\delta }_K < J_\mathcal{G},\ K \leq K^\eta(\delta) \}
    \\
    & \cap \{\forall j \in [K] ,\ [T^\eta_{j+1}(\delta)\wedge \sigma(\eta) \wedge R(\epsilon,\delta,\eta)] -  T^\eta_{j}(\delta) \leq \frac{2\hat{t}(\epsilon)}{\eta}  \},
    \\
    B^\times_2(K) & \delequal (A^\times)^c \cap \{\mathcal{J}^{\eta,\delta}_{ K } \leq J_\mathcal{G},\ K \leq K^\eta(\delta) \}
    \\
    & \cap \{\exists j \in [K]\text{ s.t. } [T^\eta_{j+1}(\delta)\wedge \sigma(\eta) \wedge R(\epsilon,\delta,\eta)] -  T^\eta_{j}(\delta) > \frac{2\hat{t}(\epsilon)}{\eta}  \}.
\end{align*}
Recall that the minimum cost for exit $J_\mathcal{G}$ is defined in \cref{def minimum cost J_G}. Note that, in the definition of $B^\times_1(K)$ and $B^\times_2(K)$ above, $T^\eta_{j+1}(\delta)\wedge \sigma(\eta) \wedge R(\epsilon,\delta,\eta) < T^\eta_{j+1}(\delta)$ only if $j = K = K^\eta(\delta)$.

\begin{lemma} \label{lemma event B 1 cross Rd}
For any $K \in [l^*]$, any $\epsilon \in (0,\bar{\epsilon}/3)$ and any $\eta > 0$ sufficiently small, the following claim holds on event $B^\times_1(K)$:
\begin{align*}
\sup_{ s \in [0, T^\eta_{K+1}(\delta) \wedge \sigma(\eta) \wedge R(\epsilon,\delta,\eta) - T^\eta_1(\delta)  ] }\norm{  X^\eta_{ \floor{s} + T^\eta_1(\delta)  }(x) - \widetilde{\bm{x}}^\eta\big(s,X^{(1)}(x);\ \bm{T},\bm{W}  \big)  } < 2\epsilon
\end{align*}
where
$\bm{T} = \big(0,T^\eta_2(\delta) - T^\eta_1(\delta), T^\eta_3(\delta) - T^\eta_2(\delta), \cdots, T^\eta_{ K + 1 }(\delta) - T^\eta_{ K  }(\delta)\big)$,
$\bm{W} = \big( W^\eta_1(\delta), W^\eta_2(\delta),\cdots,W^\eta_{ K + 1 }(\delta) \big)$,
and $X^{(1)}(x) = X^\eta_{ T^\eta_1(\delta) - 1 }(x)$.
\end{lemma}

\begin{proof}

We focus on the distances between the following three objects: $X^\eta_n(x)$, 
\begin{align*}
    \widetilde{Y}^\eta_n(x) = 
    \begin{cases}
     X^\eta_n(x) & \text{ if }n < T^\eta_1(\delta); \\
     \widetilde{Y}^\eta_{n-1}(x) + \varphi_b\big( - \eta\nabla f\big( \widetilde{Y}^\eta_{n - 1}(x) \big) + \sum_{j \geq 1}\mathbbm{1}\{n = T^\eta_j(\delta)\}\eta Z_{n}\big) & \text{ otherwise;}
    \end{cases}
\end{align*}
and
\begin{align*}
    \bm{z}(s) = \widetilde{\bm{x}}^\eta\big(s,X^{(1)}(x);\ \bm{T},\bm{W}  \big)\ \forall s \geq 0.
\end{align*}
First of all, for $X^{(1)} = X^\eta_{ T^\eta_1(\delta) - 1 }(x) = \widetilde{Y}^\eta_{ T^\eta_1(\delta) - 1 }(x)$, by definition of event $(A^\times)^c$ we have $\norm{X^{(1)}} < 3\epsilon < \bar{\epsilon}$.
Then due to \cref{implication of lemma bound on small perturbation at origin} and the definition of event $B^\times_1(K)$ (in particular, the fact that $\mathcal{J}^{\eta,\delta}_K < J_\mathcal{G}$), we know that
\begin{align*}
    \bm{d}(\bm{z}(s), \mathcal{G}^c) > 100l^*\bar{\epsilon}\ \ \forall s \in \big[0, T^\eta_{K+1}(\delta) \wedge \sigma(\eta) \wedge R(\epsilon,\delta,\eta) - T^\eta_1(\delta) \big).
\end{align*}
In light of the definition of $B^\times_1(K)$ (i.e. the upper bound on $T^\eta_{j+1}(\delta) - T^\eta_{j}(\delta)$), by applying Corollary \ref{corollary ODE GD gap Rd}, we can show that
\begin{align*}
    \sup_{ s \in [0, T^\eta_{K+1}(\delta) \wedge \sigma(\eta) \wedge R(\epsilon,\delta,\eta) - T^\eta_1(\delta)] }\norm{ \bm{z}(s) - \widetilde{Y}^\eta_{ \floor{s} + T^\eta_1(\delta) }(x)  } < \epsilon.
\end{align*}
In particular, by applying Corollary \ref{corollary ODE GD gap Rd} inductively for all $ n \leq T^\eta_{K+1}(\delta) \wedge \sigma(\eta) \wedge R(\epsilon,\delta,\eta) - T^\eta_1(\delta)$, we know that the line segment between $\bm{z}(s)$ and $\widetilde{Y}^\eta_{ \floor{s} + T^\eta_1(\delta) }(x)$ is in $\mathcal{G}$ for all $s < T^\eta_{K+1}(\delta) \wedge \sigma(\eta) \wedge R(\epsilon,\delta,\eta) - T^\eta_1(\delta)$ so Corollary \ref{corollary ODE GD gap Rd} can be further applied to $n + 1$ and establish the inequality above.

Similarly, due to our choice of $\hat{\epsilon}(\epsilon)$ in \cref{def function tilde epsilon Rd} and the upper bound on $T^\eta_{j+1}(\delta) - T^\eta_{j}(\delta)$ in definition of event $B^\times_1(K)$, by applying Corollary \ref{corollary sgd gd gap Rd} inductively, we know that for all $\eta$ sufficiently small,
\begin{align*}
    \sup_{ j = 0,1,2,\cdots, T^\eta_{K+1}(\delta) \wedge \sigma(\eta) \wedge R(\epsilon,\delta,\eta) - T^\eta_1(\delta) }\norm{ \widetilde{Y}^\eta_{ \floor{s} + T^\eta_1(\delta) }(x) - X^\eta_{ \floor{s} + T^\eta_1(\delta) }(x)   } < \epsilon
\end{align*}
holds on event $B^\times_1(K)$ and this concludes the proof.
\end{proof}

\begin{lemma}\label{lemma event B 2 cross Rd}
For any $K \in [l^*]$, any $\epsilon > 0$ and any $\eta > 0$ sufficiently small, 
$$B^\times_2(K) = \emptyset.$$
\end{lemma}

\begin{proof}
The definition of event $B^\times_2(K)$ ensures that we can define some $j^*$ as the smallest integer in $[K^\eta(\delta)]$ such that for $T^* \delequal T^\eta_{j^*}(\delta) + \ceil{2\hat{t}(\epsilon)/\eta}$, we have
\begin{align}
    T^\eta_{j^*}(\delta) < T^* < T^\eta_{j^* + 1}(\delta) \wedge \sigma(\eta) \wedge R(\epsilon,\delta,\eta). \label{proof def T star Rd}
\end{align}

Analogous to the proof of the previous lemma, we focus on the pair-wise distances between the following three objects: $X^\eta_n(x)$, 
\begin{align*}
    \widetilde{Y}^\eta_n(x) = 
    \begin{cases}
     X^\eta_n(x) & \text{ if }n < T^\eta_1(\delta); \\
     \widetilde{Y}^\eta_{n-1}(x) + \varphi_b\big( - \eta\nabla f\big( \widetilde{Y}^\eta_{n - 1}(x) \big) + \sum_{j \geq 1}\mathbbm{1}\{n = T^\eta_j(\delta)\}\eta Z_{n}\big) & \text{ otherwise;}
    \end{cases}
\end{align*}
and
\begin{align*}
    \bm{z}(s) = \widetilde{\bm{x}}^\eta\big(s,X^{(1)};\ \bm{T},\bm{W}  \big)\ \forall s \geq 0.
\end{align*}
with
$\bm{T} = \big(0,T^\eta_2(\delta) - T^\eta_1(\delta), T^\eta_3(\delta) - T^\eta_2(\delta), \cdots, T^\eta_{ j^* + 1 }(\delta) - T^\eta_{ j^* }(\delta) \big)$,
$\bm{W} = \big( W^\eta_1(\delta), W^\eta_2(\delta),\cdots,W^\eta_{ j^* + 1 }(\delta)\big)$,
and $X^{(1)} = X^\eta_{ T^\eta_1(\delta) }$.
Again, using Corollary \ref{corollary ODE GD gap Rd} and \ref{corollary sgd gd gap Rd}, one can see that for all $\eta > 0$ sufficiently small, we must have
\begin{align*}
    \sup_{ s \in [0, T^* - T^\eta_1(\delta)  ] }\norm{  X^\eta_{ \floor{s} + T^\eta_1(\delta)  }(x) - \bm{z}(s)  } < 2\epsilon
\end{align*}
on event $B^\times_2(K)$. 
However, \cref{implication of lemma bound on small perturbation at origin} and $\norm{ X^{(1)}} < 3\epsilon < \bar{\epsilon}$ implies that
\begin{align*}
\sup_{ s \in [0, T^\eta_{j^*}(\delta) - T^\eta_{1}(\delta)]  }\bm{d}\big( \bm{z}(s),\ \mathcal{G}^c\big) & > 100l^*\bar{\epsilon}.
\end{align*}
In the meantime, Lemma \ref{lemma bound on return time Rd} and our choice of $\widetilde{\epsilon}(\epsilon)$ in \cref{def function tilde epsilon Rd} implies that
\begin{align*}
    \norm{\bm{z}( T^* - T^\eta_1(\delta) )} < \epsilon \text{ and } \norm{ X^\eta_{T^*}(x) } < 3\epsilon,
\end{align*}
thus dictating that $R(\epsilon,\delta,\eta) < T^*$ and contradicting \cref{proof def T star Rd}. This concludes the proof.
\end{proof}

Now we are ready to apply the tools above and analyze some atypical scenarios regrading the first exit and first return time.
In the next result we show that, when starting from the local minimum, it is very unlikely to escape if the total cost of large jumps is less than $J_\mathcal{G}$ .
\begin{lemma} \label{lemma atypical 1 exit before J_G Rd}
For any fixed $\epsilon \in (0,\bar{\epsilon}/3), N>0$, 
the following claim holds for any sufficiently small $\delta > 0$:
\begin{align*}
    \sup_{x:\ \norm{x} \leq 2\epsilon }\mathbb{P}_x\Big( \sigma(\eta) < R(\epsilon,\delta,\eta),\ \mathcal{J}^{\downarrow}_{\eta,\delta}( \sigma(\eta) ) < J_\mathcal{G} \Big) = o(\eta^N)\ \ \text{as }\eta \downarrow 0.
\end{align*}
\end{lemma}

 \begin{proof}
Since the constant $\epsilon>0$ is fixed, one can see that $\hat{t} = c_1 + c_1 \log(1/\epsilon)$ is also fixed where $c_1$ is the constant in Lemma \ref{lemma bound on return time Rd}.
Besides, we can fix some $\widetilde{\epsilon} > 0$ such that
\begin{align*}
    \widetilde{\epsilon} < \frac{\epsilon}{ 4\exp(2\hat{t}M) },\ \ \bar{\rho}(\hat{t})\widetilde{\epsilon} < \epsilon,\ \  \widetilde{\rho}(\hat{t})\widetilde{\epsilon} < \epsilon
\end{align*}
where functions $\bar{\rho}(\cdot), \widetilde{\rho}(\cdot)$ are defined in Corollary \ref{corollary ODE GD gap Rd} and \ref{corollary sgd gd gap Rd} respectively.

Let us define an event $A^\times = A^\times_1\cup A^\times_2$ where
\begin{align*}
     A^\times_1 & \delequal \Big\{ \exists n < T^\eta_1(\delta)\ s.t.\ \norm{X^\eta_n(x)} > 3\epsilon \Big\}, \\
     A^\times_2 & \delequal{} \Big\{ \exists j = 2,\cdots,l^*\ s.t. \\
     &\ \ \ \ \ \ \max_{k = 1,2,\cdots,(T^\eta_j(\delta) -T^\eta_{j-1}(\delta) - 1) \wedge (2\hat{t}/\eta) } \eta\norm{ Z_{ T^\eta_{j-1}(\delta) + 1 } + \cdots + Z_{ T^\eta_{j-1}(\delta) + j } } > \widetilde{\epsilon} \Big\}
\end{align*}
where the positive integer $l^*$ is defined in \cref{def l star Rd}.
By definition of $l^*$, we must have $K^\eta(\delta) < l^*$ on event $\{\mathcal{J}^{\downarrow}_{\eta,\delta}( \sigma(\eta) ) < J_\mathcal{G}\}$.

Now let us analyze the dynamics of $X^\eta_n$ on event $(A^\times)^c \cap \{\mathcal{J}^{\downarrow}_{\eta,\delta}( \sigma(\eta) ) < J_\mathcal{G} \}$.
In particular, we decompose it into two events
\begin{align*}
    B_1 \delequal & (A^\times)^c \cap \{\mathcal{J}^{\downarrow}_{\eta,\delta}( \sigma(\eta) ) < J_\mathcal{G} \}
    \\
    & \cap \{\forall j = 1,2,\cdots,K^\eta(\delta) ,\ [T^\eta_{j+1}(\delta)\wedge \sigma(\eta) \wedge R(\epsilon,\delta,\eta)] -  T^\eta_{j}(\delta) \leq 2\hat{t}/\eta  \},
    \\
    B_2 \delequal & (A^\times)^c \cap \{\mathcal{J}^{\downarrow}_{\eta,\delta}( \sigma(\eta) ) < J_\mathcal{G} \} 
    \\
    & \cap \{\exists j = 1,2,\cdots,K^\eta(\delta)\text{ s.t. } [T^\eta_{j+1}(\delta)\wedge \sigma(\eta) \wedge R(\epsilon,\delta,\eta)] -  T^\eta_{j}(\delta) > 2\hat{t}/\eta  \}.
\end{align*}

Using Lemma \ref{lemma event B 1 cross Rd} and the fact that $B_1 = \big(\bigcup_{K = 1}^{l^*}B^\times_1(K) \cap \{ K^\eta(\delta) = K \}\big) \cap \{\mathcal{J}^{\eta,\delta}_{ K^{\eta}(\delta) } < J_\mathcal{G}  \}$, one can see that for any $\eta > 0$ sufficiently small, we must have
\begin{align*}
    \bm{d}( X^\eta_{ \sigma(\eta)\wedge R(\epsilon,\delta,\eta)  }(x), \mathcal{G}^c ) > 100l^*\bar{\epsilon} - 2\epsilon > 100l^*\bar{\epsilon} - \bar{\epsilon} > 0.
\end{align*}
on event $B_1$.
In other words, on event $B_1$ we must have $R(\epsilon,\delta,\eta) < \sigma(\eta)$. 

On the other hand, Lemma \ref{lemma event B 2 cross Rd} the fact that $B_2 \subseteq \bigcup_{K = 1}^{l^*}B^\times_2(K) \cap \{ K^\eta(\delta) = K \}$ implies that $B_2  = \emptyset$ whenever $\eta > 0$ is sufficiently small.

In summary, we have established that
$$\sup_{x:\ \norm{x} \leq 2\epsilon }\mathbb{P}_x\Big( \sigma(\eta) < R(\epsilon,\eta),\ \mathcal{J}^{\downarrow}_{\eta,\delta}( \sigma(\eta) ) < J_\mathcal{G} \Big) \leq \sup_{x:\ \norm{x} \leq 2\epsilon } \P(A^\times).$$
Lastly, from Lemma \ref{lemma prob event A Rd} and \ref{lemma stuck at local minimum before large jump Rd},
one can see that for any $\delta>0$ that is sufficiently small,
\begin{align*}
    \sup_{x:\ \norm{x} \leq 2\epsilon}\mathbb{P}(A^\times) = o(\eta^N),
\end{align*}
and this concludes the proof.
\end{proof}

Using an almost identical approach, we can establish the next two results and conclude that it is also rather unlikely to have scenarios where
\begin{itemize}
    \item The accumulated cost of large jumps exceeds $J_\mathcal{G}$ before the first return or exit,
    \item Or the first return occurs before the first exit,
\end{itemize}
yet it takes rather long for the said event to occur.

\begin{lemma} \label{corollary atypical 1 Rd}
Given $\epsilon \in (0,\bar{\epsilon}/3), N > 0$, the following claim holds for any sufficiently small $\delta > 0$:
\begin{align*}
    & \sup_{ x \in \overline{B(\bm{0},2\epsilon)}}\mathbb{P}_x\Big( \exists K \in \mathbb{N} \text{ s.t. }\mathcal{J}^{\eta,\delta}_K \geq J_\mathcal{G} \text{ and } T^\eta_j(\delta) - T^\eta_{j-1}(\delta) > 2\hat{t}(\epsilon)/\eta\text{ for some }j = 2,3,\cdots,K  \Big) 
    \\
    & = o(\eta^N)
\end{align*}
as $\eta \downarrow 0$ where $\hat{t}(\epsilon) = c_1 + c_1\log(1/\epsilon)$.
\end{lemma}

\begin{proof}
Given the fixed $\epsilon>0$, define $\hat{t} = c_1 + c_1 \log(1/\epsilon)$. Also, fix $\widetilde{\epsilon} > 0$ as the largest possible value such that
\begin{align*}
    \widetilde{\epsilon} \leq \frac{\epsilon}{ 4\exp(2\hat{t}M) },\ \ \bar{\rho}(\hat{t})\widetilde{\epsilon} \leq \epsilon,\ \  \widetilde{\rho}(\hat{t})\widetilde{\epsilon} \leq \epsilon
\end{align*}
where functions $\bar{\rho}(\cdot), \widetilde{\rho}(\cdot)$ are defined in Corollary \ref{corollary ODE GD gap Rd} and \ref{corollary sgd gd gap Rd} respectively.

Let us define an event $A^\times = A^\times_1\cup A^\times_2$ where
\begin{align*}
     A^\times_1 & \delequal \Big\{ \exists n < T^\eta_1(\delta)\ s.t.\ \norm{X^\eta_n(x)} > 3\epsilon \Big\}, \\
     A^\times_2 & \delequal{} \Big\{ \exists j = 2,\cdots,l^*\ s.t. \\
     & \ \max_{k = 1,2,\cdots,(T^\eta_j(\delta) -T^\eta_{j-1}(\delta) - 1) \wedge (2\hat{t}/\eta) } \eta\norm{ Z_{ T^\eta_{j-1}(\delta) + 1 } + \cdots + Z_{ T^\eta_{j-1}(\delta) + j } } > \widetilde{\epsilon} \Big\}
\end{align*}
where the positive integer $l^*$ is defined in \cref{def l star Rd}.
Furthermore, 
define event
\begin{align*}
    B \delequal  (A^\times)^c \cap \Big\{ \exists K \in \mathbb{N} \text{ s.t. }\mathcal{J}^{\eta,\delta}_K \geq J_\mathcal{G} \text{ and } T^\eta_j(\delta) - T^\eta_{j-1}(\delta) > \frac{2\hat{t}}{\eta}\text{ for some }j = 2,3,\cdots,K  \Big\}.
\end{align*}
Note that on event $B$, the index
\begin{align*}
    K^* = \max\{ k \geq 1: \mathcal{J}^{\eta,\delta}_j < J_\mathcal{G}\ \forall j = 1,2,\cdots,k-1 \}
\end{align*}
is well defined with $1 \leq  K^* < l^*$ (due to the definition of $l^*$).
As a consequence, we must have $ T^\eta_{K^*}(\delta) \leq \sigma(\eta) \wedge R(\epsilon,\delta,\eta)$.

Now based on the exact value of $K^*$, we can decompose the event $B$ into $B = \cup_{K = 0}^{l^*}B(K)$ where
\begin{align*}
    B(K) \delequal (A^\times)^c \cap \{K^* - 1 = K\} \cap \{ \mathcal{J}^{\eta,\delta}_K <  J_\mathcal{G} \}\cap \{ T^\eta_{j+1}(\delta) - T^\eta_{j}(\delta) > \frac{2\hat{t}}{\eta}\ \text{for some } j \in [K]  \}.
\end{align*}
By applying Lemma \ref{lemma event B 2 cross Rd}, we know that for all $\eta > 0$ sufficiently small, $B(K) = \emptyset$. Therefore, for all sufficiently small $\eta$,
\begin{align*}
    & \sup_{ x \in \overline{B(\bm{0},2\epsilon)}}\mathbb{P}_x\Big( \exists K \in \mathbb{N} \text{ s.t. }\mathcal{J}^{\eta,\delta}_K \geq J_\mathcal{G} \text{ and } T^\eta_j(\delta) - T^\eta_{j-1}(\delta) > 2\hat{t}(\epsilon)/\eta\text{ for some }j = 2,3,\cdots,K  \Big)
    \\
    \leq & \sup_{ x \in \overline{B(\bm{0},2\epsilon)}}\mathbb{P}_x(A^\times).
\end{align*}
Lastly, from Lemma \ref{lemma prob event A Rd} and \ref{lemma stuck at local minimum before large jump Rd},
one can see that for any $\delta>0$ that is sufficiently small,
\begin{align*}
    \sup_{x:\ \norm{x} \leq 2\epsilon}\mathbb{P}(A^\times) = o(\eta^N),
\end{align*}
and this concludes the proof.
\end{proof}

\begin{lemma} \label{corollary atypical 2 Rd}
Given $\epsilon \in (0,\bar{\epsilon}/3), N > 0$, the following claim holds for any sufficiently small $\delta > 0$:
\begin{align*}
    \sup_{ x \in \overline{B(\bm{0},2\epsilon)}}& \mathbb{P}_x\Big( R(\epsilon,\delta,\eta) < \sigma(\eta),\ J^{\eta,\delta}_{K^\eta(\delta)} < J_\mathcal{G},
    \\
    & T^\eta_{j+1}(\delta) \wedge R(\epsilon,\delta,\eta) - T^\eta_j(\delta) > \frac{2\hat{t}(\epsilon)}{\eta}\text{ for some }j \in [K^\eta(\delta)] \Big)
    = o(\eta^N)\ \ \text{ as }\eta \downarrow 0
\end{align*}
where $\hat{t}(\epsilon) = c_1 + c_1\log(1/\epsilon)$.
\end{lemma}

\begin{proof}
Given the fixed $\epsilon>0$, define $\hat{t} = c_1 + c_1 \log(1/\epsilon)$. Also, fix $\widetilde{\epsilon} > 0$ as the largest possible value such that
\begin{align*}
    \widetilde{\epsilon} \leq \frac{\epsilon}{ 4\exp(2\hat{t}M) },\ \ \bar{\rho}(\hat{t})\widetilde{\epsilon} \leq \epsilon,\ \  \widetilde{\rho}(\hat{t})\widetilde{\epsilon} \leq \epsilon
\end{align*}
where functions $\bar{\rho}(\cdot), \widetilde{\rho}(\cdot)$ are defined in Corollary \ref{corollary ODE GD gap Rd} and \ref{corollary sgd gd gap Rd} respectively.

Let us define an event $A^\times = A^\times_1\cup A^\times_2$ where
\begin{align*}
     A^\times_1 & \delequal \Big\{ \exists n < T^\eta_1(\delta)\ s.t.\ \norm{X^\eta_n(x)} > 3\epsilon \Big\}, \\
     A^\times_2 & \delequal{} \Big\{ \exists j = 2,\cdots,l^*\ s.t. 
     \\
     & \ \max_{k = 1,2,\cdots,(T^\eta_j(\delta) -T^\eta_{j-1}(\delta) - 1) \wedge (2\hat{t}/\eta) } \eta\norm{ Z_{ T^\eta_{j-1}(\delta) + 1 } + \cdots + Z_{ T^\eta_{j-1}(\delta) + j } } > \widetilde{\epsilon} \Big\}
\end{align*}
where the positive integer $l^*$ is defined in \cref{def l star Rd}.
By definition of $l^*$, we must have $K^\eta(\delta) < l^*$ on event $\{J^{\eta,\delta}_{K^\eta(\delta)} < J_\mathcal{G}\}$.
Furthermore, 
define event
\begin{align*}
    B \delequal &  (A^\times)^c \cap \Big\{ R(\epsilon,\delta,\eta) < \sigma(\eta),\ J^{\eta,\delta}_{K^\eta(\delta)} < J_\mathcal{G},
    \\
    & T^\eta_{j+1}(\delta) \wedge R(\epsilon,\delta,\eta) - T^\eta_j(\delta) > \frac{2\hat{t}(\epsilon)}{\eta}\text{ for some }j \in [K^\eta(\delta)]  \Big\}.
\end{align*}
Now observe that $B \subseteq \bigcup_{K = 1}^{l^*}B(K)$ where
\begin{align*}
    B(K) \delequal & (A^\times)^c \cap \{ K^\eta(\delta) = K \} \cap \{\mathcal{J}^{\eta,\delta}_K < J_\mathcal{G} \}
    \\
    & \cap \{\exists j \in [K^\eta(\delta)]\text{ s.t. } [T^\eta_{j+1}(\delta)\wedge \sigma(\eta) \wedge R(\epsilon,\delta,\eta)] -  T^\eta_{j}(\delta) > \frac{2\hat{t}}{\eta}  \}.
\end{align*}
By applying Lemma \ref{lemma event B 2 cross Rd}, we know that for all $\eta > 0$ sufficiently small, $B(K) = \emptyset$. Therefore, for all sufficiently small $\eta$,
\begin{align*}
    & \sup_{ x \in \overline{B(\bm{0},2\epsilon)}}\mathbb{P}_x\Big( R(\epsilon,\delta,\eta) < \sigma(\eta),\ J^{\eta,\delta}_{K^\eta(\delta)} < J_\mathcal{G},
    \\
    & T^\eta_{j+1}(\delta) \wedge R(\epsilon,\delta,\eta) - T^\eta_j(\delta) > \frac{2\hat{t}(\epsilon)}{\eta}\text{ for some }j \in [K^\eta(\delta)]  \Big)
    \leq \sup_{ x \in \overline{B(\bm{0},2\epsilon)}}\mathbb{P}_x(A^\times).
\end{align*}
Lastly, from Lemma \ref{lemma prob event A Rd} and \ref{lemma stuck at local minimum before large jump Rd},
one can see that for any $\delta>0$ that is sufficiently small,
\begin{align*}
    \sup_{x:\ \norm{x} \leq 2\epsilon}\mathbb{P}(A^\times) = o(\eta^N),
\end{align*}
and this concludes the proof.
\end{proof}

Recall that in Lemma \ref{lemma atypical 1 exit before J_G Rd}, we have shown that it is rather unlikely to have the first exit with accumulated cost of large jumps less than $J_\mathcal{G}$.
In the next result, we show that even with large jumps of accumulated cost $J_\mathcal{G}$,
if some jumps are still not large enough,
or the inter-arrival times are too long,
then it is still very unlikely for the SGD iterates to even get close the the boundary.
Define 
\begin{align*}
    K_{J_\mathcal{G}}^{\eta,\delta} \delequal \min\{ k \in [K^\eta(\delta)]:\ \mathcal{J}^{\eta,\delta}_k \geq J_\mathcal{G} \}
\end{align*}
and let $K_{J_\mathcal{G}}^{\eta,\delta} = \infty$ if $\mathcal{J}^{\eta,\delta}_{ K^\eta(\delta) } < J_\mathcal{G}$.
Also, if $K_{J_\mathcal{G}}^{\eta,\delta} < \infty$, let
\begin{align*}
    T^{ \geq J_\mathcal{G} }(\eta,\delta) \delequal T^\eta_{ K_{J_\mathcal{G}}^{\eta,\delta} }(\delta)
\end{align*}
In other words, $T^{ \geq J_\mathcal{G} }(\eta,\delta)$ is the first time that, prior to first exit or return, the accumulated cost of large jumps has reached $J_\mathcal{G}$.
Regarding the $o(\eta^{1 + J_\mathcal{G} - \alpha_1 + \Delta} )$ term in the next Lemma, we note that the function $\widetilde{\lambda}$ defined in \cref{def scaling function tilde lambda} is regularly varying (as $\eta \downarrow 0$) with index $1 + J_\mathcal{G} - \alpha_1 $. 
Therefore, $\eta^{1 + J_\mathcal{G} - \alpha_1 + \Delta} = o( \widetilde{\lambda}(\eta) )$ as $\eta$ approaches $0$.

\begin{lemma} \label{lemma atypical 2 Rd}
There exists some $\Delta > 0$ such that the following claim holds for all $\epsilon \in (0,\bar{\epsilon}/3)$ and all $\delta>0$ that is sufficiently small:
\begin{align*}
    \sup_{ x:\ \norm{x} \leq 2\epsilon }\mathbb{P}_x\big(B^\times_{2,\text{(I)}}( \epsilon,\delta,\eta )\cup B^\times_{2,\text{(II)}}( \epsilon,\delta,\eta ) \big) = o(\eta^{1 + J_\mathcal{G} - \alpha_1 + \Delta} )
\end{align*}
as $\eta \downarrow 0$ where $\hat{t}(\epsilon) = c_1 + c_1\log(1/\epsilon)$ and
\begin{align*}
    & B^\times_{2,\text{(I)}}( \epsilon,\delta,\eta ) =
    \{ K_{J_\mathcal{G}}^{\eta,\delta}  < \infty \} \cap \{ \mathcal{J}^{\eta,\delta}_{K_{J_\mathcal{G}}^{\eta,\delta}} > J_\mathcal{G} \}
    ,
    \\
    & B^\times_{2,\text{(II)}}( \epsilon,\delta,\eta ) 
    = 
    \{ K_{J_\mathcal{G}}^{\eta,\delta}  < \infty \} \cap \{ \mathcal{J}^{\eta,\delta}_{K_{J_\mathcal{G}}^{\eta,\delta}} = J_\mathcal{G} \}
    \\
    & \cap \Big\{ \exists j = 2,3,\cdots,K_{J_\mathcal{G}}^{\eta,\delta} \ s.t.
 T^\eta_{j}(\delta) - T^\eta_{j - 1}(\delta) > 2\bar{t}/\eta 
   \text{ or }\exists j = 1,2,\cdots,K_{J_\mathcal{G}}^{\eta,\delta} \ s.t.\ \eta\norm{W^\eta_j(\delta)}\leq \bar{\delta} \Big\}
   \\
   &\ \ \ \ \ \cap \big\{ \min_{j = 0,1,\cdots,T^{ \geq J_\mathcal{G} }(\eta,\delta)  }\bm{d}\big( X^\eta_j , \mathcal{G}^c \big) \leq \bar{\epsilon} \big\}.
\end{align*}
\end{lemma}

 \begin{proof}
 
From Assumption \ref{assumption bounded away}, we know the existence of some $\Delta > 0$ satisfying the following condition:
For any $k \in \mathbb{N}$ and any $\bm{j} \in \{1,\cdots,m\}^k$ such that $\sum_{ i = 1  }^{k-1}(\alpha_{j_i}-1) < J_\mathcal{G} < \sum_{ i = 1  }^{k}(\alpha_{j_i} - 1)$ (note that there are only finitely many possible choices for such $\bm{j}$), we have
\begin{align}
    J_\mathcal{G} + \Delta < \sum_{ i = 1  }^{k}(\alpha_{j_i} - 1) \label{proof lemma atypical 2 Rd choose delta}
\end{align}
Fix such $\Delta > 0$.
Meanwhile, given the fixed $\epsilon>0$, we can define $\hat{t} = c_1 + c_1 \log(1/\epsilon)$ and choose $\widetilde{\epsilon} > 0$ as the largest possible value such that
\begin{align*}
    \widetilde{\epsilon} \leq \frac{\epsilon}{ 4\exp(2\hat{t}M) },\ \ \bar{\rho}(\hat{t})\widetilde{\epsilon} \leq \epsilon,\ \  \widetilde{\rho}(\hat{t})\widetilde{\epsilon} \leq \epsilon
\end{align*}
where functions $\bar{\rho}(\cdot), \widetilde{\rho}(\cdot)$ are defined in Corollary \ref{corollary ODE GD gap Rd} and \ref{corollary sgd gd gap Rd} respectively.
We stress that $\bar{t}$ is a fixed constant while $\hat{t}$ depends on the value of $\epsilon$, and for any sufficiently small $\epsilon$ we will have $\hat{t} > \bar{t}$.

Let us define an event $A^\times = A^\times_1\cup A^\times_2$ where
\begin{align*}
     A^\times_1 & \delequal \Big\{ \exists n < T^\eta_1(\delta)\ s.t.\ \norm{X^\eta_n(x)} > 3\epsilon \Big\}, \\
     A^\times_2 & \delequal{} \Big\{ \exists j = 2,\cdots,l^*\ s.t.
     \\
     &\ \max_{k = 1,2,\cdots,(T^\eta_j(\delta) -T^\eta_{j-1}(\delta) - 1) \wedge (2\hat{t}/\eta) } \eta\norm{ Z_{ T^\eta_{j-1}(\delta) + 1 } + \cdots + Z_{ T^\eta_{j-1}(\delta) + j } } > \widetilde{\epsilon} \Big\}
\end{align*}
where the positive integer $l^*$ is defined in \cref{def l star Rd}.
By definition of $l^*$, we must have $K^{\eta,\delta}_{J_\mathcal{G}} < l^*$ on event $\{K^{\eta,\delta}_{J_\mathcal{G}} < \infty\}$.

First, we analyze the following event
\begin{align*}
    B \delequal  (A^\times)^c \cap \{K^{\eta,\delta}_{J_\mathcal{G}} < \infty \} \cap \Big\{T^\eta_{j}(\delta) - T^\eta_{j - 1}(\delta) > \frac{2\hat{t}}{\eta}\text{ for some }j =2,3,\cdots,K^{\eta,\delta}_{J_\mathcal{G}}  \Big\}.
\end{align*}
In particular, note that $B = \bigcup_{K = 0}^{l^* - 1}B(K)$ where
\begin{align*}
    B(K) \delequal & (A^\times)^c \cap \{K^{\eta,\delta}_{J_\mathcal{G}} = K + 1 \} \cap \{ \mathcal{J}^{\eta,\delta}_K < J_\mathcal{G}  \} 
    \\
    & \ \ \cap \Big\{T^\eta_{j+1}(\delta) - T^\eta_{j}(\delta) > \frac{2\hat{t}}{\eta}\text{ for some }j =1,2,3,\cdots,K \Big\}.
\end{align*}
When $K = 0 < 1$, the event $B(0) = \emptyset$ by definition (due to event $\Big\{T^\eta_{j+1}(\delta) - T^\eta_{j}(\delta) > \frac{2\hat{t}}{\eta}\text{ for some }j =1,2,3,\cdots,K \Big\}$).
For $K = 1,\cdots,l^*$, Lemma \ref{lemma event B 2 cross Rd} implies that $B(K) = \emptyset$ for all $\eta$ sufficiently small.
In summary, $B = \emptyset$ for all $\eta$ sufficiently small.

Now we focus on the following two events
\begin{align*}
    C & \delequal (A^\times)^c \cap \{K^{\eta,\delta}_{J_\mathcal{G}} < \infty \} \cap \Big\{T^\eta_{j}(\delta) - T^\eta_{j - 1}(\delta) \leq \frac{2\hat{t}}{\eta}\text{ for all }j =2,3,\cdots,K^{\eta,\delta}_{J_\mathcal{G}}  \Big\} \cap \{ \mathcal{J}^{\eta,\delta}_{ K^\eta(\delta) } > J_\mathcal{G} \},
    \\
    D & \delequal (A^\times)^c \cap \{K^{\eta,\delta}_{J_\mathcal{G}} < \infty \} \cap \Big\{T^\eta_{j}(\delta) - T^\eta_{j - 1}(\delta) \leq \frac{2\hat{t}}{\eta}\text{ for all }j =2,3,\cdots,K^{\eta,\delta}_{J_\mathcal{G}}  \Big\} \cap \{ \mathcal{J}^{\eta,\delta}_{ K^\eta(\delta) } = J_\mathcal{G} \}.
\end{align*}
On the one hand,
event $C$ can be decomposed as follows.
Let 
$$\mathbb{J}^\uparrow \delequal \{ \bm{j} \in \{1,2,\cdots,m\}^k:\ k  = |\bm{j}|,\ \sum_{ i = 1  }^{k-1}(\alpha_{j_i} - 1) < J_\mathcal{G} < \sum_{ i = 1  }^{k}(\alpha_{j_i} - 1) \}$$
be the set that contains all the types $\bm{j}$ such that the accumulated cost reach $J_\mathcal{G}$ if and only if the last element is kept.
One can see that there are only finitely many elements in $\mathbb{J}^\uparrow$.
Now we have $C = \bigcup_{\bm{j} \in  \mathbb{J}^\uparrow }C(\bm{j})$ where (see the definition in \cref{def type of jumps w})
\begin{align*}
    C(\bm{j}) &  \delequal (A^\times)^c \cap \{K^{\eta,\delta}_{J_\mathcal{G}} < \infty \} \cap \Big\{T^\eta_{j}(\delta) - T^\eta_{j - 1}(\delta) \leq \frac{2\hat{t}}{\eta}\text{ for all }j =2,3,\cdots,K^{\eta,\delta}_{J_\mathcal{G}}  \Big\} 
    \\ 
     &\ \ \ \ \ \ \ \ \ \ \cap \big\{ \big(W^{\eta}_1(\delta),W^{\eta}_2(\delta),\cdots,W^{\eta}_{ K^{\eta,\delta}_{J_\mathcal{G}} }(\delta) \big)\text{ is of type-}\bm{j} \big\}.
\end{align*}
For any $\bm{j} \in  \mathbb{J}^\uparrow$, observe that (let $k_{ \bm{j} } = |\bm{j}|$)
\begin{align*}
    & \sup_{ \norm{x} \leq 2\epsilon }\P_x(C(\bm{j}))
    \\
    \leq & \P\Big( W^\eta_i(\delta) \in E_{ \bm{j}_i }\ \forall i \in [K^{\eta,\delta}_{J_\mathcal{G}} ]\text{ and }T^\eta_{i}(\delta) - T^\eta_{i - 1}(\delta) \leq \frac{2\hat{t}}{\eta}\text{ for all }i =2,3,\cdots,K^{\eta,\delta}_{J_\mathcal{G}}    \Big) \\
    = & \big(\prod_{ i = 1 }^{ k_{\bm{j}} }\frac{ H_{\bm{j}_i}(\delta/\eta)  }{H(\delta/\eta)}\big)\cdot \big( \P( T^\eta_1(\delta) \leq 2\hat{t}/\eta ) \big)^{ k_{\bm{j}} - 1  }.
\end{align*}
The last equality is due to independence of $W^\eta_i(\delta)$ and $T^\eta_i(\delta) - T^\eta_{i-1}(\delta)$.
The regularly varying natures of functions $H_j$ and $H$ imply that
\begin{align*}
    \lim_{\eta \downarrow 0}\frac{\prod_{ i = 1 }^{ k_{\bm{j}} }\frac{ H_{\bm{j}_i}(\delta/\eta)  }{H(\delta/\eta)}}{ \prod_{ i = 1 }^{ k_{\bm{j}} }\frac{ H_{\bm{j}_i}(1/\eta)  }{H(1/\eta)}  } = \frac{ \delta^{ k_{\bm{j}}\alpha_1 }  }{ \delta^{ \sum_{i = 1}^{k_{\bm{j}}}\alpha_{ \bm{j}_i }  }  }.
\end{align*}
Meanwhile, from Lemma \ref{lemmaGeomFront} and the regularly varying nature of function $H_j$ and $H$ (in particular, repeating the same calculations that leads to \cref{proof Type J Rd bound term 3}), we have
\begin{align*}
    \lim_{\eta \downarrow 0}\frac{ \big( \P( T^\eta_1(\delta) \leq 2\hat{t}/\eta ) \big)^{ k_{\bm{j}} - 1  }  }{ \big( \frac{H(1/\eta)}{\eta}\cdot \frac{2\hat{t}}{\delta^{ \alpha_1}} \big)^{k_{\bm{j}} - 1} } = 1.
\end{align*}
Therefore, we have yielded that
\begin{align*}
    \limsup_{\eta \downarrow 0} \frac{ \sup_{ \norm{x} \leq 2\epsilon }\P_x(C(\bm{j})) }{ \widetilde{\lambda}_{\bm{j}}(\eta)  } \leq \frac{ \delta^{ k_{\bm{j}}\alpha_1 }  }{ \delta^{ \sum_{i = 1}^{k_{\bm{j}}}\alpha_{ \bm{j}_i }  }  }\cdot { \big( \frac{2\hat{t}}{\delta^{ \alpha_1}} \big)^{k_{\bm{j}} - 1} } < \infty
\end{align*}
where
\begin{align*}
    \widetilde{\lambda}_{\bm{j}}(\eta) \delequal \frac{\prod_{i = 1}^{k_{\bm{j}}}H_{\bm{j}_i}(1/\eta) }{\eta^{k_{\bm{j}} - 1}H(1/\eta) } = o(\eta^{1 + J_\mathcal{G} - \alpha_1 + \Delta})
\end{align*}
due to \cref{proof lemma atypical 2 Rd choose delta}. Therefore, one can see that as $\eta \downarrow 0$,
\begin{align*}
    \sup_{ \norm{x} \leq 2\epsilon }\P_x(C) = o(\eta^{1 + J_\mathcal{G} - \alpha_1 + \Delta}).
\end{align*}
Furthermore, the discussion above have shown that for all $\eta$ sufficiently small,
\begin{align*}
    \sup_{ x:\ \norm{x} \leq 2\epsilon }\mathbb{P}_x\big(B^\times_{2,\text{(I)}}( \epsilon,\delta,\eta )\big) \leq o(\eta^{1 + J_\mathcal{G} - \alpha_1 + \Delta}) + \sup_{ x:\ \norm{x} \leq 2\epsilon }\mathbb{P}_x( A^\times).
\end{align*}
From Lemma \ref{lemma prob event A Rd} and \ref{lemma stuck at local minimum before large jump Rd},
one can see that for all $N > 1 + J_\mathcal{G} - \alpha_1 + \Delta$ and all $\delta>0$ that is sufficiently small,
$\sup_{x:\ \norm{x} \leq 2\epsilon}\mathbb{P}(A^\times) = o(\eta^N).$
Fix such $\delta > 0$.
In conclusion, we now have that
$\sup_{ x:\ \norm{x} \leq 2\epsilon }\mathbb{P}_x\big(B^\times_{2,\text{(I)}}( \epsilon,\delta,\eta )\big) = o(\eta^{1 + J_\mathcal{G} - \alpha_1 + \Delta})$.

On the other hand, on event $D$, Assumption \ref{assumption unique configuration for minimum exit cost} shows that (see the definition in \cref{def j i set})
\begin{align*}
    & \{K^{\eta,\delta}_{J_\mathcal{G}} < \infty \}  \cap \{ \mathcal{J}^{\eta,\delta}_{ K^\eta(\delta) } = J_\mathcal{G} \}
    \\
    & =  \{K^{\eta,\delta}_{J_\mathcal{G}} < \infty \}  \cap \big\{ \big(W^{\eta}_1(\delta),W^{\eta}_2(\delta),\cdots,W^{\eta}_{ K^{\eta,\delta}_{J_\mathcal{G}} }(\delta) \big)\text{ is of type-}\bm{j}\text{ for some }\bm{j} \in \bm{j}(\bm{i}^*) \big\}.
\end{align*}
Therefore, 
\begin{align*}
    & B^\times_{2,\text{(II)}}( \epsilon,\delta,\eta ) \cap D
    \\
    = & 
    (A^\times)^c 
    \cap \{K^{\eta,\delta}_{J_\mathcal{G}} < \infty \}  \cap \big\{ \big(W^{\eta}_1(\delta),W^{\eta}_2(\delta),\cdots,W^{\eta}_{ K^{\eta,\delta}_{J_\mathcal{G}} }(\delta) \big)\text{ is of type-}\bm{j}\text{ for some }\bm{j} \in \bm{j}(\bm{i}^*) \big\}
    \\
    &\cap \Big\{ \forall j = 2,3,\cdots,K_{J_\mathcal{G}}^{\eta,\delta},\ \ T^\eta_{j}(\delta) - T^\eta_{j - 1}(\delta) \leq 2\hat{t}/\eta \Big\}
    \cap 
    \big\{ \min_{j = 0,1,\cdots,T^{ \geq J_\mathcal{G} }(\eta,\delta)  }\bm{d}\big( X^\eta_j , \mathcal{G}^c \big) \leq \bar{\epsilon} \big\}
    \\
    &\cap \Big\{ \exists j = 2,3,\cdots,K_{J_\mathcal{G}}^{\eta,\delta} \ s.t.\ T^\eta_{j}(\delta) - T^\eta_{j - 1}(\delta) > 2\bar{t}/\eta\text{ OR }\exists j = 1,2,\cdots,K_{J_\mathcal{G}}^{\eta,\delta} \ s.t.\ \eta\norm{W^\eta_j(\delta)}\leq \bar{\delta} \Big\}.
\end{align*}
In particular, using Lemma \ref{lemma event B 1 cross Rd}, one can see that when $\eta$ is sufficiently small, on event $B^\times_{2,\text{(II)}}( \epsilon,\delta,\eta ) \cap D$ we have $ \norm{X^\eta_k  } < 3\epsilon $ for all $k < T^\eta_1(\delta)$ and
\begin{align*}
    \sup_{ s \in [0,T^{\geq J_\mathcal{G}}(\eta,\delta) - T^\eta_1(\delta) ] }\norm{ X^\eta_{ \floor{s} } - \bm{z}^\eta(s)  } < 2\epsilon.
\end{align*}
Here 
\begin{align*}
    \bm{z}^\eta(s) = \widetilde{\bm{x}}^\eta\big(s,X^{(1)};\ \bm{T},\bm{W}  \big)\ \forall s \geq 0
\end{align*}
with
$\bm{T} = \big(0,T^\eta_2(\delta) - T^\eta_1(\delta), T^\eta_3(\delta) - T^\eta_2(\delta), \cdots, T^\eta_{ K_{J_\mathcal{G}}^{\eta,\delta} + 1 }(\delta) - T^\eta_{ K_{J_\mathcal{G}}^{\eta,\delta} }(\delta)\big)$,
$\bm{W} = \big( W^\eta_1(\delta), W^\eta_2(\delta),\cdots,W^\eta_{K_{J_\mathcal{G}}^{\eta,\delta} }(\delta) \big)$,
and $X^{(1)} = X^\eta_{ T^\eta_1(\delta) - 1 }$.
Besides, \cref{property bar delta bar t} dictates that, on event $B^\times_{2,\text{(II)}}( \epsilon,\delta,\eta ) \cap D$, we must have $\inf_{s \geq 0}\bm{d}(\bm{z}^\eta(s), \mathcal{G}^c) \geq 100l^*\bar{\epsilon},$ hence
\begin{align*}
    \min_{ k \in [ T^{\geq J_\mathcal{G}}(\eta,\delta) ] }\bm{d}(X^\eta_k, \mathcal{G}^c) > 50l^*\bar{\epsilon}.
\end{align*}
However, this clearly contradicts the definition of event $B^\times_{2,\text{(II)}}( \epsilon,\delta,\eta )$. 
In conclusion, we have established that (for all $\eta$ sufficiently small) $B^\times_{2,\text{(II)}}( \epsilon,\delta,\eta ) \cap D = \emptyset$ and
\begin{align*}
    \sup_{ x:\ \norm{x} \leq 2\epsilon }\mathbb{P}_x\big(B^\times_{2,\text{(II)}}( \epsilon,\delta,\eta )\big) \leq \sup_{ x:\ \norm{x} \leq 2\epsilon }\mathbb{P}_x( A^\times).
\end{align*}
However, as per the argument above, our choice of sufficiently small $\delta > 0$ ensures that $\sup_{x:\ \norm{x} \leq 2\epsilon}\mathbb{P}(A^\times) = o(\eta^{1 +  J_\mathcal{G} - \alpha_1 + \Delta})$ and concludes the proof.
\end{proof}

Furthermore, starting from the local minimum $\bm{0}$, it is unlikely that the SGD iterates will be extremely close to $\partial \mathcal{G}$ when the accumulated cost of large jumps reaches $J_\mathcal{G}$. The next lemma provides an upper bound for the probability of the said scenario.
For any set $A \subseteq \mathbb{R}^d$, define its $\epsilon-$enlargement as the following open set
$A^\epsilon \delequal \{x:\ \bm{d}(x,A) < \epsilon \}.$

\begin{lemma} \label{lemma atypical 3 Rd}
For any $\epsilon \in \big( 0, \bar{\epsilon}/(3 + 3\rho^*)\big)$ and any $\delta > 0$ that is sufficiently small,
\begin{align*}
    \limsup_{\eta \downarrow 0} \frac{ \sup_{ x:\ \norm{x} \leq 2\epsilon }\mathbb{P}_x( B^\times_3(\epsilon,\delta,\eta) ) }{ \widetilde{\lambda}(\eta)  } \leq \delta^{\alpha_1}\Psi(\epsilon) 
\end{align*}
where $B^\times_3(\epsilon,\delta,\eta) = \{  K^{\eta,\delta}_{J_\mathcal{G}} < \infty \}\cap \{ \bm{d}(X^\eta_{T^{ \geq J_\mathcal{G} }(\eta,\delta)},\ \partial \mathcal{G}) < \epsilon \}$
and $\Psi(\epsilon) = \mu\Big( h^{-1}\big( (\partial \mathcal{G})^{(3 + 3\rho^*)\epsilon} \big) \Big)$ with $h$ defined in \cref{def mapping h Rd} and $\mu$ defined in \cref{def measure mu Rd}, and $\rho^* \in (0,\infty)$ is the constant provided in Corollary \ref{corollary ODE ODE gap Rd}.
\end{lemma}

\begin{proof}
Again, given the fixed $\epsilon>0$, we can define $\hat{t} = c_1 + c_1 \log(1/\epsilon)$ and choose $\widetilde{\epsilon} > 0$ as the largest possible value such that
\begin{align*}
    \widetilde{\epsilon} \leq \frac{\epsilon}{ 4\exp(2\hat{t}M) },\ \ \bar{\rho}(\hat{t})\widetilde{\epsilon} \leq \epsilon,\ \  \widetilde{\rho}(\hat{t})\widetilde{\epsilon} \leq \epsilon
\end{align*}
where functions $\bar{\rho}(\cdot), \widetilde{\rho}(\cdot)$ are defined in Corollary \ref{corollary ODE GD gap Rd} and \ref{corollary sgd gd gap Rd} respectively.
We stress that $\bar{t}$ is a fixed constant while $\hat{t}$ depends on the value of $\epsilon$, and for any sufficiently small $\epsilon$ we will have $\hat{t} > \bar{t}$.

Let us define an event $A^\times = A^\times_1\cup A^\times_2$ where
\begin{align*}
     A^\times_1 & \delequal \Big\{ \exists n < T^\eta_1(\delta)\ s.t.\ \norm{X^\eta_n(x)} > 3\epsilon \Big\}, \\
     A^\times_2 & \delequal{} \Big\{ \exists j = 2,\cdots,l^*\ s.t. \\
     &\ \max_{k = 1,2,\cdots,(T^\eta_j(\delta) -T^\eta_{j-1}(\delta) - 1) \wedge (2\hat{t}/\eta) } \eta\norm{ Z_{ T^\eta_{j-1}(\delta) + 1 } + \cdots + Z_{ T^\eta_{j-1}(\delta) + j } } > \widetilde{\epsilon} \Big\}
\end{align*}
where the positive integer $l^*$ is defined in \cref{def l star Rd}.
Consider the following decomposition of event $B^\times_3(\epsilon,\delta,\eta) \cap (A^\times)^c$:
\begin{align*}
    C_\text{(I)} & \delequal B^\times_3(\epsilon,\delta,\eta) \cap (A^\times)^c
    \cap 
    \{\mathcal{J}^{\eta,\delta}_{ K^{\eta,\delta}_{J_\mathcal{G}} } > J_\mathcal{G}\}. \\
    C_\text{(II)} & \delequal
    B^\times_3(\epsilon,\delta,\eta) \cap (A^\times)^c
    \cap \{ \mathcal{J}^{\eta,\delta}_{K_{J_\mathcal{G}}^{\eta,\delta}} = J_\mathcal{G} \} \cap \Big\{ \exists j = 2,3,\cdots,K_{J_\mathcal{G}}^{\eta,\delta} \ s.t.\\
    & \ \ \ \ T^\eta_{j}(\delta) - T^\eta_{j - 1}(\delta) > 2\bar{t}/\eta \text{ or }\exists j = 1,2,\cdots,K_{J_\mathcal{G}}^{\eta,\delta} \ s.t.\ \eta\norm{W^\eta_j(\delta)}\leq \bar{\delta} \Big\} 
    \\ 
    C_\text{(III)} & \delequal \big(B^\times_3(\epsilon,\delta,\eta) \cap (A^\times)^c\big) \symbol{92} \big( C_\text{(I)} \cup C_\text{(II)} \big).
\end{align*}
Using Lemma \ref{lemma atypical 2 Rd}, we know that for all $\delta > 0$ sufficiently small,
\begin{align*}
    \limsup_{\eta \downarrow 0}\frac{ \sup_{\norm{x} \leq 2\epsilon} \P_x\big( C_\text{(I)} \cup C_\text{(II)} \big) }{ \widetilde{\lambda}(\eta) } = 0.
\end{align*}
As a result,
\begin{align*}
     \limsup_{\eta \downarrow 0}\frac{ \sup_{\norm{x} \leq 2\epsilon} \P_x(B^\times_3(\epsilon,\delta,\eta) ) }{ \widetilde{\lambda}(\eta)} 
     & \leq 
     \limsup_{\eta \downarrow 0}\frac{ \sup_{\norm{x} \leq 2\epsilon} \P_x(A^\times) }{ \widetilde{\lambda}(\eta)} + \limsup_{\eta \downarrow 0}\frac{ \sup_{\norm{x} \leq 2\epsilon} \P_x( C_\text{(III)})  }{ \widetilde{\lambda}(\eta)}.
\end{align*}
From Lemma \ref{lemma prob event A Rd} and \ref{lemma stuck at local minimum before large jump Rd},
one can see that for all $N > 1 + J_\mathcal{G} - \alpha_1$ and all $\delta>0$ that is sufficiently small, we have
$\sup_{x:\ \norm{x} \leq 2\epsilon}\mathbb{P}(A^\times) = o(\eta^N) = o(\widetilde{\lambda}(\eta)).$

Moving on, we focus on bounding the probability of event $C_\text{(III)}$. 
In particular, note that
\begin{align*}
    & C_\text{(III)} 
    =
    (A^\times)^c \cap \{  K^{\eta,\delta}_{J_\mathcal{G}} < \infty \} \cap  \{ \mathcal{J}^{\eta,\delta}_{K_{J_\mathcal{G}}^{\eta,\delta}} = J_\mathcal{G} \} \cap \{ \bm{d}(X^\eta_{T^{ \geq J_\mathcal{G} }(\eta,\delta)},\ \partial \mathcal{G}) < \epsilon \}
    \\
    &\cap \Big\{ \forall j = 2,3,\cdots,K_{J_\mathcal{G}}^{\eta,\delta},\ T^\eta_{j}(\delta) - T^\eta_{j - 1}(\delta) \leq 2\bar{t}/\eta \text{ and }\forall j = 1,2,\cdots,K_{J_\mathcal{G}}^{\eta,\delta},\ \eta\norm{W^\eta_j(\delta)}\geq \bar{\delta} \Big\}. 
\end{align*}
By applying Lemma \ref{lemma event B 1 cross Rd} on event $C_\text{(III)}$,
one can see that when $\eta$ is sufficiently small, on this event we have $ \norm{X^\eta_k  } < 3\epsilon $ for all $k < T^\eta_1(\delta)$ and
\begin{align*}
    \sup_{ s \in [0,T^{\geq J_\mathcal{G}}(\eta,\delta) - T^\eta_1(\delta) ] }\norm{ X^\eta_{ \floor{s} + T^\eta_1(\delta)3 } - \bm{z}^\eta(s)  } < 2\epsilon.
\end{align*}
Here 
\begin{align*}
    \bm{z}^\eta(s) = \widetilde{\bm{x}}^\eta\big(s,X^{(1)};\ \bm{T},\bm{W}  \big)\ \forall s \geq 0
\end{align*}
with
$\bm{T} = \big(0,T^\eta_2(\delta) - T^\eta_1(\delta), T^\eta_3(\delta) - T^\eta_2(\delta), \cdots, T^\eta_{ K_{J_\mathcal{G}}^{\eta,\delta} }(\delta) - T^\eta_{ K_{J_\mathcal{G} - 1}^{\eta,\delta} }(\delta)\big)$,
$\bm{W} = \big( W^\eta_1(\delta), W^\eta_2(\delta),\cdots,W^\eta_{K_{J_\mathcal{G}}^{\eta,\delta} }(\delta) \big)$,
and $X^{(1)} = X^\eta_{ T^\eta_1(\delta) - 1 }$.
To better control the location of $\bm{z}^\eta(s)$, we also construct
\begin{align*}
    \bm{z}^\eta_0(s) = \widetilde{\bm{x}}^\eta\big(s,\bm{0};\ \bm{T},\bm{W}  \big)\ \forall s \geq 0
\end{align*}
where the only difference is that we substitute the initial value $X^{(1)}$ with $\bm{0}$.
Note that on event $C_\text{(III)}$, we have $\norm{X^{(1)}} < 3\epsilon$.
By applying Corollary \ref{corollary ODE ODE gap Rd} (the condition about line segments contained in $\mathcal{G}$ is verified due to \cref{property bounded away with small cost}),
we have the following bound on event $C_\text{(III)}$:
\begin{align*}
    \sup_{ s \in [0, \sum_{i = 1}^{K^{\eta,\delta}_{J_\mathbb{G}}}\bm{T}_i  ] }\norm{ \bm{z}^\eta(s) - \bm{z}^\eta_0(s) } < 3\rho^*\epsilon.
\end{align*}
Now recall Definition \ref{definitionTypeJ Rd}.
Combining all the bounds we have obtained so far, we see that, when $\eta$ is sufficiently small, on event $C_\text{(III)}$ we have
\begin{align*}
    \bm{d}\Big(\bm{z}^\eta_0\big(\sum_{i = 1}^{K^{\eta,\delta}_{J_\mathbb{G}}}\bm{T}_i\big),\ \partial\mathcal{G}\Big) < (2 + 3\rho^*)\epsilon
\end{align*}
and $\bm{W}$ is of type-$\bm{j}$ for some $\bm{j} \in \bm{j}(\bm{i}^*)$.
It then follows immediately from Corollary \ref{CorTypeProbRd} that
\begin{align*}
    \limsup_{\eta \downarrow 0}\frac{ \sup_{\norm{x} \leq 2\epsilon} \P_x( C_\text{(III)})  }{ \widetilde{\lambda}(\eta)} 
    \leq 
    \limsup_\eta \frac{ p\big( (\partial \mathcal{G})^{ (2 + 3\rho^*)\epsilon }  ,\delta,\eta\big)  }{  \widetilde{\lambda}(\eta)  } 
    & \leq \delta^{\alpha_1}\mu\Big(h^{-1}\big( \overline{ (\partial \mathcal{G})^{ (2 + 3\rho^*)\epsilon }  }  \big)\Big)
    \\
    & \leq \delta^{\alpha_1}\mu\Big(h^{-1}\big( { (\partial \mathcal{G})^{ (3 + 3\rho^*)\epsilon }  }  \big)\Big)
\end{align*}
and this concludes the proof.
\end{proof}

Lastly, we provide a lower bound for the probability of the \emph{most likely} way of escape,
i.e., due to $k^*$ large jumps with accumulated cost $J_\mathcal{G}$ and relatively short inter-arrival time less (when compared to $\bar{t}/\eta$). 
We stress that, in the next result, the value of constant $c_* > 0$ would not vary with the choice of parameters $\epsilon,\delta$.

\begin{lemma}\label{lemma lower bound typical exit Rd}
There exists some $c_* > 0$ such that the following claim holds for
any $\epsilon \in (0,\bar{\epsilon}/(4 + 3\rho^*))$ and any sufficiently small $\delta > 0$:
\begin{align*}
\liminf_{ \eta \downarrow 0 }\frac{ \inf_{\norm{x} \leq 2\epsilon} \mathbb{P}_x( A^\circ(\epsilon,\delta,\eta) ) }{ \widetilde{\lambda}(\eta) } \geq c_*\delta^{\alpha_1}
\end{align*}
where the event is defined as
\begin{align*}
    A^\circ(\epsilon,\delta,\eta) 
    & \delequal{}
    \{ K^{\eta,\delta}_{J_\mathcal{G}} < \infty \}
    \cap \Big\{ \sigma(\eta) = T^{\geq J_\mathcal{G}}(\eta,\delta)\Big\}
    \cap \Big\{T^\eta_j(\delta) - T^\eta_{j-1}(\delta) \leq \frac{2\bar{t}}{\eta}\ \forall j =2,3,\cdots,k^*   \Big\}
\end{align*}
and $\rho^* \in (0,\infty)$ is the constant provided in Corollary \ref{corollary ODE ODE gap Rd}.
\end{lemma}

 \begin{proof}
 First of all, Assumption \ref{assumption unique configuration for minimum exit cost} implies that
 \begin{align*}
     \{ K^{\eta,\delta}_{J_\mathcal{G}} < \infty \}
    \cap \Big\{ \sigma(\eta) = T^{\geq J_\mathcal{G}}(\eta,\delta)\Big\}
    =
    \{ K^{\eta,\delta}_{J_\mathcal{G}} = k^* \}
    \cap \Big\{ \sigma(\eta) = T^\eta_{k^*}(\delta)\Big\}.
 \end{align*}
We also stress that $K^{\eta,\delta}_{J_\mathcal{G}} = k^*$ means that $\sigma(\eta)\wedge R(\epsilon,\delta,\eta) \geq T^\eta_{k^*}(\delta)$ (i.e. the arrival time of the $k^*-$th large jumps and that the accumulated cost of large jumps $W^\eta_1(\delta),\cdots,W^\eta_{k^*}(\delta)$ is exactly $J_\mathcal{G}$.
 
Due to \cref{property existence of escape path} and \cref{property bar delta bar t},
there exists some $\bm{j} \in \bm{j}(\bm{i}^*)$,
some $\widetilde{\bm{\theta}} = ( \widetilde{\theta}_1,\cdots,\widetilde{\theta}_{k^*} )$ with $\widetilde{\theta}_i \in F_i\ \forall i$,
some $\widetilde{ \bm{r} } = (\widetilde{r}_1,\cdots,\widetilde{r}_{k^*})$ such that $\widetilde{r}_i \geq \bar{\delta}\ \forall i$,
some $\widetilde{\bm{t}} = (\widetilde{t}_1,\cdots,\widetilde{t}_{k^*})$ such that $\widetilde{t}_1 = 0$ and $\widetilde{t}_i \in (0,\bar{t})$ for all $i = 2,\cdots,k^*$
such that
\begin{align*}
    \bm{d}\big(h(\widetilde{ \bm{r} },\widetilde{\bm{\theta}},\widetilde{\bm{t}}),\ \mathcal{G} \big) > 100l^*\bar{\epsilon}
\end{align*}
 where $h$ is the mapping defined in \cref{def mapping h Rd}.
 Let $\widetilde{y} \delequal h(\widetilde{ \bm{r} },\widetilde{\bm{\theta}},\widetilde{\bm{t}}).$
 Moreover, the continuity of mapping $h$ implies the existence of some $\Delta > 0$ such that $(\widetilde{ \bm{r} },\widetilde{\bm{\theta}},\widetilde{\bm{t}}) \in \mathcal{X} \subseteq  h^{-1}\big( B(\widetilde{y}, \bar{\epsilon}) \big) $
where the open domain $\mathcal{X}$ is defined as
\begin{align*}
    \mathcal{X} \delequal \{ (\bm{r},\bm{\theta},\bm{t}):\ t_1 = 0\text{ and }\ \norm{ \theta_i - \widetilde{\theta}_i }<\Delta,\ |r_i - \widetilde{r}_i|<\Delta,\ |t_i - \widetilde{t}_i|<\Delta\ \forall i \in [k^*] \}.
\end{align*}
In particular, the parameter $\Delta$ can be chosen small enough with
\begin{align*}
    \Delta < \min_{i = 2,\cdots,k^*}\widetilde{t}_i,\ \Delta < \min_{i = 1,\cdots,k^*}\widetilde{r}_i \wedge 1
\end{align*}
 so that for any $(\bm{r},\bm{\theta},\bm{t}) \in \mathcal{X}$, we have $t_i > 0$ for all $i=2,\cdots,k^*$ and $r_i > 0$ for all $i \in [k^*]$.
 Fix such $\Delta > 0$.
 Meanwhile, from the definition of the measure $\mu(\cdot)$ in \cref{def measure mu j Rd}\cref{def measure mu Rd}
 and the fact that $\widetilde{\theta}_i \in F_{\bm{j}_i}$ where the closed set $F_{\bm{j}_i}$ is the support of the probability measure $S_{\bm{j}_i}$ on the unit sphere $\mathbb{S}^{d-1}$
 (see Assumption \ref{assumption 6 regular variation of the Rd noise}),
 one can see that $\mu(\mathcal{X}) > 0$.
 
Let $\bm{R} = (R_i)_{i = 1}^{k^*}, \bm{\Theta} = (\Theta_i)_{i = 1}^{k^*}, \bm{T} = (T_i)_{i = 1}^{k^*}$ with $R_i = \eta\norm{W^\eta_i(\delta)},\ \Theta_i = W^\eta_i(\delta)/\norm{W^\eta_i(\delta)}$ and $T_i = \eta( T^\eta_i(\delta) - T^\eta_{i-1}(\delta) )$ for all $i \geq 2$ and $T_1 = 0$. Now recall Definition \ref{definitionTypeJ Rd} and consider the following event
 \begin{align*}
     B \delequal \{ (\bm{R},\bm{\Theta},\bm{T}  ) \in \mathcal{X} \} = \{ Z_{ T^\eta_1(\delta) }\text{ is of type-}( h(\mathcal{X}),\eta,\delta ) \}.
 \end{align*}
Corollary \ref{CorTypeProbRd} then gives the bound
 \begin{align*}
     \liminf_{\eta \downarrow 0}\frac{\mathbb{P}(B)}{\widetilde{\lambda}(\eta)} \geq \delta^{\alpha_1}\mu(\mathcal{X}) > 0.
 \end{align*}
 
From now on, we fix some $c_* \in (0,\mu(\mathcal{X}))$.
Given the fixed $\epsilon>0$, we can define $\hat{t} = c_1 + c_1 \log(1/\epsilon)$ and choose $\widetilde{\epsilon} > 0$ as the largest possible value such that
\begin{align*}
    \widetilde{\epsilon} \leq \frac{\epsilon}{ 4\exp(2\hat{t}M) },\ \ \bar{\rho}(\hat{t})\widetilde{\epsilon} \leq \epsilon,\ \  \widetilde{\rho}(\hat{t})\widetilde{\epsilon} \leq \epsilon
\end{align*}
where functions $\bar{\rho}(\cdot), \widetilde{\rho}(\cdot)$ are defined in Corollary \ref{corollary ODE GD gap Rd} and \ref{corollary sgd gd gap Rd} respectively.
We stress that $\bar{t}$ is a fixed constant while $\hat{t}$ depends on the value of $\epsilon$, and for any sufficiently small $\epsilon$ we will have $\hat{t} > \bar{t}$.

Let us define an event $A^\times = A^\times_1\cup A^\times_2$ where
\begin{align*}
     A^\times_1 & \delequal \Big\{ \exists n < T^\eta_1(\delta)\ s.t.\ \norm{X^\eta_n(x)} > 3\epsilon \Big\}, \\
     A^\times_2 & \delequal{} \Big\{ \exists j = 2,\cdots,k^*\ s.t.
     \\
     & \ \max_{k = 1,2,\cdots,(T^\eta_j(\delta) -T^\eta_{j-1}(\delta) - 1) \wedge (2\hat{t}/\eta) } \eta\norm{ Z_{ T^\eta_{j-1}(\delta) + 1 } + \cdots + Z_{ T^\eta_{j-1}(\delta) + j } } > \widetilde{\epsilon} \Big\}
\end{align*}
where the positive integer $l^*$ is defined in \cref{def l star Rd}.

Now we focus on the event $B \symbol{92} A^\times$.
On one hand, Lemma \ref{lemma event B 1 cross Rd} shows that, for all $\eta$ sufficiently small,
on this event we have $ \norm{X^\eta_k  } < 3\epsilon <\bar{\epsilon}$ for all $k < T^\eta_1(\delta)$ and
\begin{align*}
    \sup_{ s \in [0,T^\eta_{k^*}(\delta) - T^\eta_1(\delta) ] }\norm{ X^\eta_{ \floor{s} + T^\eta_1(\delta } - \bm{z}^\eta(s)  } < 2\epsilon < \bar{\epsilon}.
\end{align*}
Here 
\begin{align*}
    \bm{z}^\eta(s) = \widetilde{\bm{x}}^\eta\big(s,X^{(1)};\ \bm{T},\bm{W}  \big)\ \forall s \geq 0
\end{align*}
with
$\bm{T} = \big(0,T^\eta_2(\delta) - T^\eta_1(\delta), T^\eta_3(\delta) - T^\eta_2(\delta), \cdots, T^\eta_{ k^* }(\delta) - T^\eta_{ k^* - 1 }(\delta)\big)$,
$\bm{W} = \big( W^\eta_1(\delta), W^\eta_2(\delta),\cdots,W^\eta_{k^* }(\delta) \big)$,
and $X^{(1)} = X^\eta_{ T^\eta_1(\delta) - 1 }$.
To better control the location of $\bm{z}^\eta(s)$, we also construct
\begin{align*}
    \bm{z}^\eta_0(s) = \widetilde{\bm{x}}^\eta\big(s,\bm{0};\ \bm{T},\bm{W}  \big)\ \forall s \geq 0
\end{align*}
where the only difference is that we substitute the initial value $X^{(1)}$ with $\bm{0}$.

First of all, from \cref{implication of lemma bound on small perturbation at origin}, we see that
\begin{align*}
    \bm{d}(\bm{z}^\eta(s), \mathcal{G}^c) > 100l^*\bar{\epsilon},\ \bm{d}(\bm{z}_0^\eta(s), \mathcal{G}^c) > 100l^*\bar{\epsilon}\ \ \forall s \in [0, \sum_{i = 1}^{k^*}\bm{T}_i). 
\end{align*}
Besides, due to $\mathcal{X} \subseteq  h^{-1}\big( B(\widetilde{y}, \bar{\epsilon}) \big)$, we have that 
$\bm{d}\big(\bm{z}_0^\eta\big(\sum_{i = 1}^{k^*}\bm{T}_i\big), \mathcal{G}\big) > 99l^*\bar{\epsilon}.$
Next, using Corollary \ref{corollary ODE ODE gap Rd}, we have that
$ \sup_{\in [0, \sum_{i = 1}^{k^*}\bm{T}_i]}\norm{ \bm{z}^\eta(s) - \bm{z}^\eta_0(s)  } < 3\rho^*\epsilon < \bar{\epsilon}$,
which further implies that
\begin{align*}
    \bm{d}\big(\bm{z}^\eta\big(\sum_{i = 1}^{k^*}\bm{T}_i\big), \mathcal{G}\big) > 99l^*\bar{\epsilon} - \bar{\epsilon} \geq 98l^*\bar{\epsilon}.
\end{align*}
 Now we have the following facts regarding the distance between SGD iterates $X^\eta_j$ and the domain $\mathcal{G}$.
 First, on event $B \symbol{92} A^\times$, it holds that
 \begin{align*}
     \bm{d}(X^\eta_j, \mathcal{G}^c) > 100l^*\bar{\epsilon} - \bar{\epsilon} \geq 99l^*\bar{\epsilon}\ \ \forall j < T^\eta_{k^*}(\delta),
 \end{align*}
implying that $\sigma(\eta) \geq T^\eta_{k^*}(\delta)$.

Next, if $R(\epsilon,\delta,\eta) < T^\eta_{k^*}(\delta)$, then there exists some $T^\eta_{1}(\delta \leq j < T^\eta_{k^*}(\delta)$ such that $\norm{X^\eta_j} \leq 2 \epsilon$
hence $\norm{ \bm{z}^\eta( j - T^\eta_{1} ) } \leq 2\epsilon + 2\epsilon < \bar{\epsilon} < 100l^*\bar{\epsilon}$.
However, in light of \cref{implication of lemma bound on small perturbation at origin} we know that on event $\{R(\epsilon,\delta,\eta) < T^\eta_{k^*}(\delta)\} \cap (B \symbol{92} A^\times)$,
we have 
\begin{align*}
    \bm{d}\big(\bm{z}^\eta\big(\sum_{i = 1}^{k^*}\bm{T}_i\big), \mathcal{G}^c \big) > 100l^*\bar{\epsilon}
\end{align*}
and yield a contradiction.
Therefore, on event $B \symbol{92} A^\times$ we must have $R(\epsilon,\delta,\eta) \geq T^\eta_{k^*}(\delta)$.

Besides, bounding the gap between $X^\eta_j, \bm{z}^\eta(j)$ and $\bm{z}^\eta_0(j)$ at time $j = T^\eta_{k^*}(\delta)$ using results above, we can show that
\begin{align*}
    \bm{d}(X^\eta_{T^\eta_{k^*}(\delta)},\mathcal{G}) > 98l^*\bar{\epsilon} - \bar{\epsilon} \geq 97l^*\bar{\epsilon}.
\end{align*} 
 Therefore, we must have $\sigma(\eta) = T^\eta_{k^*}(\delta)$.
 In summary, we have shown that, for all $\eta$ sufficiently small, $B \symbol{92} A^\times \subseteq A^\circ(\epsilon,\delta,\eta)$.
 Combining all the bounds above, we have
 \begin{align*}
     \liminf_{ \eta \downarrow 0 }\frac{ \inf_{\norm{x} \leq 2\epsilon} \mathbb{P}_x( A^\circ(\epsilon,\delta,\eta) ) }{ \widetilde{\lambda}(\eta) } 
     & \geq 
     \liminf_{\eta \downarrow 0}\frac{\mathbb{P}(B)}{\widetilde{\lambda}(\eta)}
     -
     \limsup_{\eta \downarrow 0}\frac{ \sup_{ \norm{x} \leq 2\epsilon } \mathbb{P}(A^\times)}{\widetilde{\lambda}(\eta)} \\
     & \geq
     \delta^{\alpha_1}c_*
     -
     \limsup_{\eta \downarrow 0}\frac{ \sup_{ \norm{x} \leq 2\epsilon } \mathbb{P}(A^\times)}{\widetilde{\lambda}(\eta)}.
 \end{align*}
 To conclude the proof, it suffices to invoke Lemma \ref{lemma prob event A Rd} and \ref{lemma stuck at local minimum before large jump Rd}
with some $N > 1 + J_\mathcal{G} - \alpha_1$.
\end{proof}

\subsection{Proof of the main result}

Now we are ready to state the main result and provide upper and lower bounds to the distribution of the scaled first exit time from domain $\mathcal{G}$.
In particular, we define a scaling function
\begin{align}
    \lambda(\eta) \delequal \widetilde{\lambda}(\eta)H(1/\eta) \label{def scaling function lambda Rd}
\end{align}
where the function $\widetilde{\lambda}(\cdot)$ is defined in \cref{def scaling function tilde lambda}.
One can easily see that $\lambda \in RV_{ 1 + J_\mathcal{G}}(\eta)$.
We show that the scaled first exit time $\lambda(\eta)\sigma(\eta)$ converges in distribution to an exponential random variable as $\eta \downarrow 0$,
which implies that, in expectation, the first exit time is roughly a $1/\eta^{1 + J_\mathcal{G}}$ term and its order is dictated by the minimum cost for exit $J_\mathcal{G}$.

\begin{proposition} \label{proposition main result Rd}
Given any $C\in (0,1), u > 0$, the following inequalities hold for all $\epsilon > 0$ sufficiently small,
\begin{align*}
\limsup_{\eta\downarrow 0}\sup_{ \norm{x} \leq 2\epsilon }\P_x( \sigma(\eta)\lambda(\eta) > u ) & \leq 2C + \exp( -(1-C)^2qu ),
\\
\liminf_{\eta\downarrow 0}\inf_{ \norm{x} \leq 2\epsilon }\P_x( \sigma(\eta)\lambda(\eta) > u ) & \geq -C + \exp( -(1+C)qu )
\end{align*}
where $q = \mu\big( h^{-1}(\mathcal{G}^c) \big)$.
\end{proposition}

Before presenting the proof to Proposition \ref{proposition main result Rd}, we make some preparations. 
First, we introduce stopping times (for all $k \geq 1$)
\begin{align*}
    \tau_k(\epsilon,\delta,\eta) & = \min\{ n > \widetilde{\tau}_{k-1}(\epsilon,\delta,\eta): \eta\norm{Z_n} > \delta \} \\
    \widetilde{\tau}_k(\epsilon,\delta,\eta) & = \min\{ n \geq \tau_k(\epsilon,\delta,\eta): \norm{X^\eta_n} \leq 2\epsilon  \}
\end{align*}
with the convention that $\tau_0(\epsilon,\delta,\eta) = \widetilde{\tau}_0(\epsilon,\delta,\eta) = 0.$ 
Intuitively, at each $\widetilde{\tau}_k$ the SGD iterates have just returned to a small neighborhood of the local minimum $\bm{0}$.
Due to Markov property of $X^\eta_n$, the SGD iterates almost \emph{regenerate} at each $\widetilde{\tau}_k$ despite the previous trajectory, and the times $(\widetilde{\tau}_k)_{k \geq 1}$ partitions the entire timeline into different segments that can almost be interpreted as \emph{regeneration cycles},
where $\widetilde{\tau}_{k-1}$ can be understood as the starting point of the $k-$th cycle.
In light of the embedded (informal) regeneration process, one natural approach is to determine the dynamics and probability of making an exit on one (hence every) cycle, and this will be carried out with the help of technical results in the previous section.
It is worth noticing that $\widetilde{\tau}_k$ are defined under the proviso that $\widetilde{\tau}_k \geq \tau_k$ where $\tau_k$ is the first big jump during the $k-$th cycle.
Considering results such as Lemma \ref{lemma stuck at local minimum before large jump Rd},
it is reasonable to expect that the SGD iterates would be trapped at local minimum until a large jump strikes.
Therefore, we define the (informal) regeneration points $\widetilde{\tau}_k$ in such a way that a cycle $[\widetilde{\tau}_{k-1},\widetilde{\tau}_k)$ ends only if we have observed at least one large jump already.
Regarding the notations, we add a remark that when there is no ambiguity we will drop the dependency on $\epsilon,\delta,\eta$ and simply write $\tau_k, \widetilde{\tau}_k$.

Specifically, we are interested in the large jumps and the accumulated cost thereof during each cycle, which definitely entails some systematic bookkeeping.
For all $k \geq 1$, the random variable
\begin{align*}
    \textbf{j}_k & \delequal{} \#\{ n = \tau_{k - 1}(\epsilon,\delta,\eta),\tau_{k-1}(\epsilon,\delta,\eta) + 1,\cdots, \widetilde{\tau}_k(\epsilon,\delta,\eta)\wedge \sigma(\eta): \eta\norm{Z_n} > \delta \}
\end{align*}
can be understood as the count of large jumps during the $k-$th cycle. Here are two remarks on this definition.
\begin{itemize}
    \item First, for any $k$ with $\sigma(\eta) < \widetilde{\tau}_k$, we have $\textbf{j}_k = 0$.
    Recall that the object we study here is the first exit time $\sigma(\eta)$, so there is virtually no need to keep track of the dynamics of the SGD after it leaves the domain $\mathcal{G}$.
    \item The random variable $\textbf{j}_k$ is measurable w.r.t. $\mathcal{F}_{\widetilde{\tau}_k \wedge \sigma(\eta)}$, the stopped $\sigma-$algebra generated by the stopping time $\widetilde{\tau}_k \wedge \sigma(\eta)$, which, intuitively, is stating that one should be able to determine the number of large jumps during the $k-$th cycle when this cycle ends.
\end{itemize}
Furthermore, for each $k = 1,2,\cdots$, let
\begin{align*}
    T_{k,1}(\epsilon,\delta,\eta) & \delequal \tau_{k-1}(\epsilon,\delta,\eta)\wedge \sigma(\eta) , \\
    T_{k,j}(\epsilon,\delta,\eta) & \delequal\min\big\{ n > T_{k,j-1}(\epsilon,\delta,\eta): \eta\norm{Z_n} > \delta \big\} \wedge \sigma(\eta) \wedge \widetilde{\tau}_k \ \ \forall j \geq 2, \\
    W_{k,j}(\epsilon,\delta,\eta) & \delequal Z_{T_{k,j}(\epsilon,\delta,\eta)}\ \ \forall j \geq 1
\end{align*}
with the convention $T_{k,0}(\epsilon,\delta,\eta) = \widetilde{\tau}_{k-1}(\epsilon,\delta,\eta)$.
Note that for any $k \geq 1, j \geq 1$, $T_{k,j}$ is a stopping time.
Besides, for any $k$ with $\widetilde{\tau}_k < \sigma(\eta)$
\begin{align}
    \widetilde{\tau}_{k-1} + 1 \leq T_{k,j} \leq \widetilde{\tau}_k \wedge \sigma(\eta)\ \ \forall j \in [\textbf{j}_k], \label{proof prop first exit implication of j k counting Rd}
\end{align}
and the sequences $\big( T_{k,j} \big)_{j = 1}^{\textbf{j}_k}$ and $\big( W_{k,j} \big)_{j = 1}^{\textbf{j}_k}$
are the arrival times and sizes of \textit{large} jumps during the $k-$th cycle, respectively. 
Again, when there is no ambiguity we will drop the dependency on $\epsilon,\delta,\eta$ and simply write $T_{k,j}$ and $W_{k,j}$.

We are now able to keep track of the accumulated cost of jumps during each cycle.
For any $k \geq 1, i \geq 1$, we define
\begin{align*}
    \mathcal{J}^{\text{cycle}}_k(i;\epsilon,\delta,\eta) \delequal \sum_{l = 1}^{i} J( W_{k,i}(\epsilon,\delta,\eta) ).
\end{align*}
Moreover, we define the following index
\begin{align*}
    \textbf{j}^{\geq J_\mathcal{G}}(k;\epsilon,\delta,\eta) \delequal
    \begin{cases}
     \min\{j = 1,2,\cdots,\textbf{j}_k:\ \mathcal{J}^{\text{cycle}}_k(j;\epsilon,\delta,\eta) \geq J_\mathcal{G} \} & \text{ if }\widetilde{\tau}_k(\epsilon,\delta,\eta) < \sigma(\eta) \\
     \infty & \text{ otherwise}
    \end{cases}
\end{align*}
with the convention that $\min \emptyset = \infty$.
That is to say, if $\widetilde{\tau}_k \geq \sigma(\eta)$ or $\mathcal{J}^{\text{cycle}}_k(\textbf{j}_k) < J_\mathcal{G}$, we define $\textbf{j}^{\geq J_\mathcal{G}}(k)$ as $\infty$;
otherwise $\textbf{j}^{\geq J_\mathcal{G}}(k)$  is the index for the first large jump during the $k-$th cycle that drives the accumulated cost of large jumps to reach $J_\mathcal{G}$.
For clarity of the presentation below, we also define
\begin{align*}
    T^{\geq J_\mathcal{G}}(k;\epsilon,\delta,\eta) & \delequal \mathbbm{1}\{ \textbf{j}^{\geq J_\mathcal{G}}(k) < \infty \}T_{k,\textbf{j}^{\geq J_\mathcal{G}}(k)}  + \mathbbm{1}\{ \textbf{j}^{\geq J_\mathcal{G}}(k) = \infty   \}0, \\
    X^{\geq J_\mathcal{G}}(k;\epsilon,\delta,\eta) & \delequal \mathbbm{1}\{ \textbf{j}^{\geq J_\mathcal{G}}(k) < \infty \}X^\eta_{ T^{\geq J_\mathcal{G}}(k;\epsilon,\delta,\eta)  } + \mathbbm{1}\{ \textbf{j}^{\geq J_\mathcal{G}}(k) = \infty   \}\bm{0}.
\end{align*}
In other words, when $\textbf{j}^{\geq J_\mathcal{G}}(k) < \infty$, the random variable $X^{\geq J_\mathcal{G}}(k;\epsilon,\delta\eta)$ is equal to $X^\eta_{ T_{k,\textbf{j}^{\geq J_\mathcal{G}}(k)}  }$, the location of the SGD iterates right when the accumulated cost on the $k-$th cycle reaches $J_\mathcal{G}$.
Notation wise, we will drop the dependency on $\epsilon,\delta,\eta$ again and simply write $\textbf{j}^{\geq J_\mathcal{G}}(k), X^{\geq J_\mathcal{G}}(k)$ or $\mathcal{J}^{\text{cycle}}_k(j)$ when there is no risk of ambiguity.

As have mentioned above, we want to zoom in on the each cycle and analyze the probability of each possible case.
To be specific, we want to introduce a series of scenarios, formally defined as several events, that exhaust all the possibilities during a cycle.
The events will be denoted in the form of $\textbf{A}^{\circ}_k$ or $\textbf{B}^{\times}_j$ largely following the next few rules.
First, we say an event/scenario is \emph{atypical} if its probability is rather small, and we assign it with superscript $\times$;
otherwise we say the event is \emph{typical} and add a superscript $\circ$.
Besides, the subscript $k$ indicates the cycle that the event concerns.
Lastly, events with label (of type) $\textbf{A}$ usually describe how the SGD iterates try to escape from $\mathcal{G}$, whereas events with label $\textbf{B}$ focuses on how the SGD iterates return to the local minimum. 

Now we proceed and formally define the said series of events.
First, for each $k \geq 1$, define the event
\begin{align}
    \textbf{A}^\times_{k,0}(\epsilon,\delta,\eta) & \delequal{} \Big\{\exists i \in [l^*]\ s.t.\ \max_{j = T_{k,i-1}+1,T_{k,i-1}+2,\cdots,(T_{k,i} - 1)\wedge(T_{k,i-1} + \ceil{ \frac{2\hat{t}(\epsilon)}{\eta}}) \wedge\widetilde{\tau}_k \wedge \sigma(\eta) } \eta\norm{Z_{ T_{k,i}+1} + \cdots + Z_j} > \widetilde{\epsilon}(\epsilon) \Big\} \nonumber
    \\
    & \ \ \ \ \cup \Big\{\norm{X^\eta_j} > 3\epsilon\ \text{ for some }j = \widetilde{\tau}_{k - 1},\widetilde{\tau}_{k - 1} + 1,\cdots,T_{k,1} - 1  \Big\} 
    \label{def A cross k 0 in prop main result Rd} 
\end{align}
with $\hat{t}(\epsilon) = c_1 + c_1\log(1/\epsilon)$ and function $\widetilde{\epsilon}(\cdot)$ defined in \cref{def function tilde epsilon Rd}.
It is similar to the event $A^\times$ defined in \cref{def event A cross 1 for A cross Rd}\cref{def event A cross 2 for A cross Rd} in the previous section, and its probability will be controlled using a similar approach.

Next, consider event (for all $k \geq 1$)
\begin{align}
    \textbf{A}^\times_{k,1}(\epsilon,\delta,\eta) \delequal{}\Big\{ \sigma(\eta)<\widetilde{\tau}_k,\ \mathcal{J}^\text{cycle}_k(\textbf{j}_k) < J_\mathcal{G} \Big\} \label{def A cross k 1 in prop main result Rd}
\end{align}
that describes an atypical case where the exit occurs during the $k-$th cycle yet the accumulated cost of all large jumps is less than $J_\mathcal{G}$. 
Another atypical event is defined as 
\begin{align}
    \textbf{A}^\times_{k,2} \delequal{}\{\bm{j}^{\geq J_\mathcal{G}}(k) < \infty\} \cap \{\exists j = 2,3,\cdots,\bm{j}^{\geq J_\mathcal{G}}(k)\ s.t.\ T_{k,j} - T_{k,j-1} > 2\hat{t}(\epsilon)/\eta \}. \label{def A cross k 2 in prop main result Rd}
\end{align}
This event describes the case where the accumulated cost of large jumps in a cycle has reached $J_\mathcal{G}$ but the inter-arrival time between some large jumps are unusually long.
We stress that, by definition of $\bm{j}^{\geq J_\mathcal{G}}(k)$, 
we have 
$\{\bm{j}^{\geq J_\mathcal{G}}(k) < \infty\} = \{ \widetilde{\tau}_{k-1} < \sigma(\eta) \}\cap\{J^\text{cycle}_k(\bm{j}_k) \geq J_\mathcal{G}\}$. 

Moving on, we consider the following events (defined for all $k \geq 1$)
\begin{align}
    \textbf{A}^\times_{k,3} \delequal{}\{\mathcal{J}^\text{cycle}_k(\textbf{j}_k) < J_\mathcal{G},\ \widetilde{\tau}_k < \sigma(\eta)\}\cap \Big\{\exists j = 2,3,\cdots,\textbf{j}_k\ s.t.\ T_{k,j} - T_{k,j-1} > 2\hat{t}(\epsilon)/\eta  \Big\}  \label{def A cross k 3 in prop main result Rd}
\end{align}
that describes another similar case where the $k-$th cycle ends with return to the local minimum but the inter-arrival time between some large jumps are unusually long.

The next event
\begin{align}
    \textbf{A}^\times_{k,4}\delequal{} & \{\textbf{j}^{\geq J_\mathcal{G}}(k) < \infty,\ \mathcal{J}^\text{cycle}_k(\textbf{j}^{\geq J_\mathcal{G}}(k)) = J_\mathcal{G} \} \cap  \{ \bm{d}( X^{\geq J_\mathcal{G}}(k), \mathcal{G}^c  ) \leq \bar{\epsilon} \} \cap \Big\{\exists j \in [\bm{j}^{\geq J_\mathcal{G}}(k)]\ s.t.\ \eta\norm{W_{k,j}} \leq \bar{\delta} \Big\} \label{def A cross k 4 in prop main result Rd}
\end{align}
describes the case where the accumulated cost on the $k-$th cycle has hit $J_\mathcal{G}$ exactly at some point with SGD iterates being rather close to $\mathcal{G}^c$ yet some large jumps are not large enough when compared to the fixed constant $\bar{\delta}$.

Meanwhile, with events
\begin{align}
    \textbf{A}^\times_{k,5} \delequal{} \{\textbf{j}^{\geq J_\mathcal{G}}(k) < \infty,\ \mathcal{J}^\text{cycle}_k(\textbf{j}^{\geq J_\mathcal{G}}(k)) = J_\mathcal{G} \} \cap \Big\{ \bm{d}( X^{\geq J_\mathcal{G}}(k), \partial \mathcal{G}) < \epsilon \Big\}, \label{def event A cross k 5 in main result Rd}
\end{align}
we analyze the case where the SGD iterates reaches somewhere close the the boundary set $\partial\mathcal{G}$ with large jumps of cost equal to $J_\mathcal{G}$. Lastly, define event
\begin{align}
     \textbf{A}^\times_{k,6} \delequal{} \{\textbf{j}^{\geq J_\mathcal{G}}(k) < \infty,\ \mathcal{J}^\text{cycle}_k(\textbf{j}^{\geq J_\mathcal{G}}(k)) > J_\mathcal{G} \} \label{def event A cross k 6 in main result Rd}
\end{align}
for the case where the accumulated cost of large jumps on the $k-$th cycle exceeds $J_\mathcal{G}$ without hitting it.
As an amalgamation of these atypical scenarios, we let
\begin{align}
    \textbf{A}^\times_k(\epsilon,\delta,\eta) \delequal{} \bigcup_{i = 0}^6 \textbf{A}^\times_{k,i}(\epsilon,\delta,\eta). \label{def event A cross union Rd}
\end{align}
Next, we analyze the probability of some events $(\textbf{B}^\times_k)_{k \geq 1}$ that concern the behavior of SGD iterates during the $k-$th cycle after $T^{\geq J_\mathcal{G}}(k)$.
Let us define
\begin{align}
    \textbf{B}^\times_{k,1}(\epsilon,\delta,\eta) &\delequal{} \{\textbf{j}^{\geq J_\mathcal{G}}(k) < \infty,\ \mathcal{J}^\text{cycle}_k(\textbf{j}^{\geq J_\mathcal{G}}(k)) = J_\mathcal{G} \} \cap \{\bm{d}( X^{\geq J_\mathcal{G}}(k), \mathcal{G}^c  ) \geq \epsilon\} \nonumber \\
    &\ \ \ \ \ \cap \Big\{T_{k,j} - T_{k,j-1} \leq 2\frac{\hat{t}(\epsilon)}{\eta}\ \forall j = 2,3,\cdots,\textbf{j}^{\geq J_\mathcal{G}}(k)  \Big\} \nonumber
    \\ 
    \textbf{B}^\times_{k,2}(\epsilon,\delta,\eta) & \delequal{} \{ \widetilde{\tau}_k - T_{k,\textbf{j}^{\geq J_\mathcal{G}}(k)} > \rho(\epsilon)/\eta \} \cup  \{ \sigma(\eta) < \widetilde{\tau}_k \} \nonumber
    \\
    \textbf{B}^\times_k(\epsilon,\delta,\eta) & \delequal{} \textbf{B}^\times_{k,1} \cap \textbf{B}^\times_{k,2} \label{def prop first exit def B cross k main result Rd}
\end{align}
where $\rho(\cdot)$ is the function in Lemma \ref{lemma return to local minimum quickly Rd}.
From the definition of $\textbf{B}^\times_k$, in particular the inclusion of $\textbf{B}^\times_{k,2}$, one can see that the intuitive interpretation of event $\textbf{B}^\times_k$ is that the SGD iterates \textit{did not return} to local minimum efficiently or even escaped from $\mathcal{G}$ after large jumps with cost equal to $J_\mathcal{G}$.
In comparison, the following events will characterize what would typically happen during each attempt:
 \begin{align}
        \textbf{A}^\circ_k(\epsilon,\delta,\eta) &\delequal{} \{\textbf{j}^{\geq J_\mathcal{G}}(k) < \infty,\ \mathcal{J}^\text{cycle}_k(\textbf{j}^{\geq J_\mathcal{G}}(k)) = J_\mathcal{G} \} \cap \{\sigma(\eta) = T^{\geq J_\mathcal{G}}(k)\} \nonumber \\
        & \ \ \ \ \ \cap \Big\{T_{k,j} - T_{k,j-1} \leq \frac{2\hat{t}(\epsilon)}{\eta}\ \forall j = 2,3,\cdots,\textbf{j}^{\geq J_\mathcal{G}}(k) \Big\}, \label{def event A circ k prop main result Rd}
        \\ 
        \textbf{B}^\circ_k(\epsilon,\delta,\eta) &\delequal{} \{ \sigma(\eta) > \widetilde{\tau}_k,\ \widetilde{\tau}_k - T_{k,1} \leq \frac{2l^*\hat{t}(\epsilon) + \rho(\epsilon)}{\eta} \} \cap \{ \mathcal{J}^\text{cycle}_k(\textbf{j}_k) \leq J_\mathcal{G} \}. \label{def event B circ k Rd}
    \end{align}
Intuitively speaking, $\textbf{A}^\circ_k$ tells us that the exit happened right at $T^{\geq J_\mathcal{G}}(k)$ with accumulated cost of large jumps in the cycle being exactly $J_\mathcal{G}$,
and $\textbf{B}^\circ_k$ requires that that the first exit did not occur during the $k-$th cycle, and the SGD iterates returned to local minimum rather efficiently. 
It is worth noticing that, by definition, $\textbf{j}^{\geq J_\mathcal{G}}(k) < l^*$ if $\textbf{j}^{\geq J_\mathcal{G}}(k)<\infty$.

Our immediate next goal is to analyze the event
\begin{align}
    \textbf{A}^\times(\epsilon,\delta,\eta) \delequal{} \bigcup_{k \geq 1}\bigg( \bigcap_{i = 1}^{k-1}\big( \textbf{A}^\times_i \cup \textbf{B}^\times_i \cup \textbf{A}^\circ_i \big)^c    \bigg) \cap \big(\textbf{A}^\times_k \cup \textbf{B}^\times_k \big). \label{def event A cross final Rd}
\end{align}
In particular, we show that its probability can be made arbitrarily small with proper $\epsilon,\delta,\eta$,
implying that we will almost always observe the event $(\textbf{A}^\times)^c$.

\begin{lemma} \label{lemma bound on event A cross final Rd}
Given any $C > 0$, the following claim holds for all $\epsilon>0,\delta>0$ sufficiently small:
\begin{align*}
    \limsup_{\eta\downarrow 0} \sup_{x:\ \norm{x}\leq 2\epsilon } \mathbb{P}_x(\textbf{A}^\times(\epsilon,\delta,\eta)) \leq C.
\end{align*}
\end{lemma}

\begin{proof}
We start by bounding for the probabilities of all the $\textbf{A}^\times_{k,i}$ events.
Fix some $N > 1 + J_\mathcal{G}$.
For 
\begin{align*}
    \textbf{A}^\times_{k,0}(\epsilon,\delta,\eta) & =\Big\{\exists i \in [l^*]\ s.t.\ \max_{j = T_{k,i-1}+1,T_{k,i-1}+2,\cdots,(T_{k,i} - 1)\wedge(T_{k,i-1} + \ceil{ \frac{2\hat{t}(\epsilon)}{\eta}}) \wedge\widetilde{\tau}_k \wedge \sigma(\eta) } \eta\norm{Z_{ T_{k,i}+1} + \cdots + Z_j} > \widetilde{\epsilon}(\epsilon) \Big\} \nonumber
    \\
    & \ \ \ \ \cup \Big\{\norm{X^\eta_j} > 3\epsilon\ \text{ for some }j = \widetilde{\tau}_{k - 1},\widetilde{\tau}_{k - 1} + 1,\cdots,T_{k,1} - 1  \Big\}, 
\end{align*}
using Lemma \ref{lemma prob event A Rd} and \ref{lemma stuck at local minimum before large jump Rd},
one can show the existence of some $\epsilon_0 \in (0,\bar{\epsilon}/3)$ such that for any $\epsilon \in (0,\epsilon_0)$, there is some $\delta_0(\epsilon) > 0$ with
\begin{align}
    \limsup_{\eta} \frac{ \sup_{\norm{x}\leq 2\epsilon}\P_x( \textbf{A}^\times_{1,0}(\epsilon,\delta,\eta)) }{\eta^N} = 0\ \ \ \forall \delta \in (0,\delta_0(\epsilon)). \label{proof prop first exit A cross k 0 bound Rd}
\end{align}
Next, for $\textbf{A}^\times_{k,1}(\epsilon,\delta,\eta) =\Big\{ \sigma(\eta)<\widetilde{\tau}_k,\ \mathcal{J}^\text{cycle}_k(\textbf{j}_k) < J_\mathcal{G} \Big\},$
it follows from Lemma \ref{lemma atypical 1 exit before J_G Rd} that, for any $\epsilon \in (0,\epsilon_0)$, there exists some $0<\delta_1(\epsilon) \leq \delta_0(\epsilon)$ such that
\begin{align}
    \limsup_{\eta} \frac{ \sup_{\norm{x}\leq 2\epsilon}\P_x( \textbf{A}^\times_{1,1}(\epsilon,\delta,\eta)) }{\eta^N} = 0\ \ \ \forall \delta \in (0,\delta_1(\epsilon)). \label{proof prop first exit A cross k 1 bound Rd}
\end{align}
For event $\textbf{A}^\times_{k,2} =\{\bm{j}^{\geq J_\mathcal{G}}(k) < \infty\} \cap \{\exists j = 2,3,\cdots,\bm{j}^{\geq J_\mathcal{G}}(k)\ s.t.\ T_{k,j} - T_{k,j-1} > 2\hat{t}(\epsilon)/\eta \}$,
Lemma \ref{corollary atypical 1 Rd} implies that, for any $\epsilon \in (0,\epsilon_0)$, there exists some $0<\delta_2(\epsilon) \leq \delta_1(\epsilon)$ such that
\begin{align}
    \limsup_{\eta} \frac{ \sup_{\norm{x}\leq 2\epsilon}\P_x( \textbf{A}^\times_{1,2}(\epsilon,\delta,\eta)) }{\eta^N} = 0\ \ \ \forall \delta \in (0,\delta_2(\epsilon)). \label{proof prop first exit A cross k 2 bound Rd}
\end{align}
For event $\textbf{A}^\times_{k,3} =\{\mathcal{J}^\text{cycle}_k(\textbf{j}_k) < J_\mathcal{G},\ \widetilde{\tau}_k < \sigma(\eta)\}\cap \Big\{\exists j = 2,3,\cdots,\textbf{j}_k\ s.t.\ T_{k,j} - T_{k,j-1} > 2\hat{t}(\epsilon)/\eta  \Big\}$,
thanks to Lemma \ref{corollary atypical 2 Rd}, one can see that, for any $\epsilon \in (0,\epsilon_0)$, there exists some $0<\delta_3(\epsilon) \leq \delta_2(\epsilon)$ such that
\begin{align}
    \limsup_{\eta} \frac{ \sup_{\norm{x}\leq 2\epsilon}\P_x( \textbf{A}^\times_{1,3}(\epsilon,\delta,\eta)) }{\eta^N} = 0\ \ \ \forall \delta \in (0,\delta_3(\epsilon)). \label{proof prop first exit A cross k 3 bound Rd}
\end{align}
For event 
\begin{align*}
    \textbf{A}^\times_{k,4}= \{\textbf{j}^{\geq J_\mathcal{G}}(k) < \infty,\ \mathcal{J}^\text{cycle}_k(\textbf{j}^{\geq J_\mathcal{G}}(k)) = J_\mathcal{G} \} \cap  \{ \bm{d}( X^{\geq J_\mathcal{G}}(k), \mathcal{G}^c  ) \leq \bar{\epsilon} \} \cap \Big\{\exists j \in [\bm{j}^{\geq J_\mathcal{G}}(k)]\ s.t.\ \eta\norm{W_{k,j}} \leq \bar{\delta} \Big\}, 
\end{align*}
using Lemma \ref{lemma atypical 2 Rd}, one can see the existence of some $\Delta > 0$ such that, for all $\epsilon \in (0,\epsilon_0)$, there is some $0<\delta_4(\epsilon) \leq \delta_3(\epsilon)$ such that
\begin{align}
    \limsup_{\eta} \frac{ \sup_{\norm{x}\leq 2\epsilon}\P_x( \textbf{A}^\times_{1,4}(\epsilon,\delta,\eta)) }{\eta^{1 + J_\mathcal{G} - \alpha_1 + \Delta}} = 0\ \ \ \forall \delta \in (0,\delta_4(\epsilon)). \label{proof prop first exit A cross k 4 bound Rd}
\end{align}
As for event $\textbf{A}^\times_{k,5} = \{\textbf{j}^{\geq J_\mathcal{G}}(k) < \infty,\ \mathcal{J}^\text{cycle}_k(\textbf{j}^{\geq J_\mathcal{G}}(k)) = J_\mathcal{G} \} \cap \Big\{ \bm{d}( X^{\geq J_\mathcal{G}}(k), \partial \mathcal{G}) < \epsilon \Big\}$,
the plan is to use Lemma \ref{lemma atypical 3 Rd} to control its probability.
But before that, we recall the definition of function $\Psi(\epsilon) = \mu\Big( h^{-1}\big( (\partial \mathcal{G})^{(3 + 3\rho^*)\epsilon} \big) \Big)$ in Lemma \ref{lemma atypical 3 Rd} and make several observations.
First, note that $\partial \mathcal{G}$ is a closed set, so $\cap_{\epsilon > 0}h^{-1}\big( (\partial \mathcal{G})^{\epsilon}  \big) = h^{-1}(\partial \mathcal{G})$ for the mapping $h$ in \cref{def mapping h Rd}.
Then it follows from Assumption \ref{assumption boundary set with zero mass} that
$\lim_{\epsilon \downarrow 0}\mu\Big(h^{-1}\big( (\partial \mathcal{G})^{\epsilon}  \big) \Big) = 0$,
implying that for $\epsilon > 0$ sufficiently small, we must have
\begin{align*}
    \Psi(\epsilon) < \frac{Cc_*}{4}.
\end{align*}
where $C > 0$ is the fixed constant in the description of the Lemma, and $c_* > 0$ is the constant in Lemma \ref{lemma lower bound typical exit Rd}.
We stress that both of them would not vary with $\epsilon,\delta,\eta$.
Combining this with Lemma \ref{lemma atypical 3 Rd}, we know the existence of some $\epsilon_5 \in (0,\epsilon_0]$ such that, for all $\epsilon \in (0,\epsilon_5)$, one can find some $0 < \delta_5(\epsilon) \leq \delta_4(\epsilon)$ with
\begin{align}
    \limsup_{\eta} \frac{ \sup_{\norm{x}\leq 2\epsilon}\P_x( \textbf{A}^\times_{1,5}(\epsilon,\delta,\eta)) }{ \widetilde{\lambda}(\eta) } \leq \delta^{\alpha}\Psi(\epsilon) < \delta^{\alpha}\frac{Cc_*}{4}\ \ \ \forall \delta \in (0,\delta_5(\epsilon)). \label{proof prop first exit A cross k 5 bound Rd}
\end{align}
Lastly, for event $\textbf{A}^\times_{k,6} = \{\textbf{j}^{\geq J_\mathcal{G}}(k) < \infty,\ \mathcal{J}^\text{cycle}_k(\textbf{j}^{\geq J_\mathcal{G}}(k)) > J_\mathcal{G} \}$,
from Lemma \ref{lemma atypical 2 Rd} again we know that (let $\epsilon_6 = \epsilon_5)$ for all $\epsilon \in (0,\epsilon_6)$, one can find some $0 < \delta_6(\epsilon) \leq \delta_5(\epsilon)$ with
\begin{align}
    \limsup_{\eta} \frac{ \sup_{\norm{x}\leq 2\epsilon}\P_x( \textbf{A}^\times_{1,6}(\epsilon,\delta,\eta)) }{\eta^{1 + J_\mathcal{G} - \alpha_1 + \Delta}} = 0\ \ \ \forall \delta \in (0,\delta_6(\epsilon)). \label{proof prop first exit A cross k 6 bound Rd}
\end{align}

Recall that $\textbf{A}^\times_k(\epsilon,\delta,\eta) = \bigcup_{i = 0}^6\textbf{A}^\times_{k,i}(\epsilon,\delta,\eta)$.
Also, the events $\textbf{B}^\times_k, \textbf{A}^\circ_k,\textbf{B}^\circ_k$ are defined in 
\cref{def prop first exit def B cross k main result Rd}\cref{def event A circ k prop main result Rd}\cref{def event B circ k Rd} respectively.
Before bounding the conditional probability of events $\textbf{B}^\times_k$, we make several observations.
First, if we consider the event $\bigcap_{j = 1}^{k}(\textbf{A}^\times_j \cup \textbf{B}^\times_j)^c \cap\textbf{B}^\circ_j$,
the inclusion of the $(\textbf{B}^\circ_j)_{j = 1}^k$ implies that during the first $k$ cycles the SGD iterates have never left $\mathcal{G}$, so
\begin{align*}
    \bigcap_{j = 1}^{k}(\textbf{A}^\times_j \cup \textbf{B}^\times_j)^c \cap\textbf{B}^\circ_j = \big(\bigcap_{j = 1}^{k}(\textbf{A}^\times_j \cup \textbf{B}^\times_j)^c \cap\textbf{B}^\circ_j\big)\cap \{ \sigma(\eta) > \widetilde{\tau}_k \}.
\end{align*}
Next, note that
\begin{align*}
    & \mathbb{P}_x\big( \textbf{B}^\times_k\ |\ \bigcap_{j = 1}^{k-1}(\textbf{A}^\times_j \cup \textbf{B}^\times_j)^c \cap\textbf{B}^\circ_j  \big) 
    \\
    = &  \underbrace{\mathbb{P}_x\big( \textbf{B}^\times_{k,1}\ |\ \bigcap_{j = 1}^{k-1}(\textbf{A}^\times_j \cup \textbf{B}^\times_j)^c \cap\textbf{B}^\circ_j  \big)}_{ \delequal\text{(I)}} \cdot \underbrace{\mathbb{P}_x\Big( \textbf{B}^\times_{k,2}\ \Big|\ \big(\bigcap_{j = 1}^{k-1}(\textbf{A}^\times_j \cup \textbf{B}^\times_j)^c \cap\textbf{B}^\circ_j\big)\cap \textbf{B}^\times_{k,1}  \Big)}_{ \delequal \text{(II)} }.
\end{align*}
We start by analyzing term (I). From Assumption \ref{assumption unique configuration for minimum exit cost}, one can see that on event $\textbf{B}^\times_{k,1}$, it holds that
$$\textbf{j}^{ \geq J_\mathcal{G}}(k) = k^*,\ \ (W_{k,1},\cdots,W_{k,k^*}) \text{ is of type-}\bm{j}\text{ for some }\bm{j}\in\bm{j}(i^*)$$
and $T_{k,j} - T_{k,j-1} \leq 2\frac{\hat{t}(\epsilon)}{\eta}\ \forall j = 2,3,\cdots,k^*$.
Due to independence of $W_{k,i}$ and $T_{k,i+1} - T_{k,i}$, we have
\begin{align*}
    \sup_{\norm{x} \leq 2}\text{(I)} & \leq \big(\P(T^\eta_1(\delta) \leq \frac{2\hat{t}(\epsilon)}{\eta})\big)^{k^* - 1}\cdot \sum_{ \bm{j}:\ \bm{j} \in  \bm{j}(\bm{i}^*) }\P\Big( (W^\eta_1(\delta),\cdots,W^\eta_{k^*}(\delta))\text{ is of type-}\bm{j} \Big) \\
    & = \big(\P(T^\eta_1(\delta) \leq \frac{2\hat{t}(\epsilon)}{\eta})\big)^{k^* - 1}\cdot \sum_{ \bm{j}:\ \bm{j} \in  \bm{j}(\bm{i}^*) }\prod_{i = 1}^{k^*}\frac{ H_{\bm{j}_i}(\delta/\eta) }{H(\delta/\eta)}.
\end{align*}
Due to the strong Markov property at stopping time $\widetilde{\tau}_{k - 1}$, Lemma \ref{lemmaGeomFront}, and the regularly varying nature of functions $H_i, H$,
\begin{align*}
    \limsup_{\eta \downarrow 0}\frac{\sup_{\norm{x} \leq 2}\text{(I)}}{ \big(2\hat{t}(\epsilon)/\delta^{\alpha_1}\big)^{k^* - 1}\big( H(1/\eta)/\eta \big)^{k^* - 1}\cdot |\bm{j}(\bm{i}^*)|\cdot\frac{\delta^{k^*\alpha_1}}{ \delta^{k^* + J_\mathcal{G}}  }\cdot\Big(\prod_{q = 1}^m \big(H_i(1/\eta)\big)^{\bm{i}^*_q}\Big)\big/\big( H(1/\eta) \big)^{k^*}  } \leq 1.
\end{align*}
In other words, for any $\epsilon,\delta > 0$, there exists some $C(\epsilon,\delta) < \infty$ such that
\begin{align*}
    \limsup_{\eta \downarrow 0}\frac{ \sup_{\norm{x} \leq 2}\text{(I)}}{\delta^{\alpha_1} \widetilde{\lambda}(\eta) } \leq |\bm{j}(\bm{i}^*)|C(\epsilon,\delta).
\end{align*}
On the other hand, due to Lemma \ref{lemma return to local minimum quickly Rd} and the strong Markov property at stopping time $T_{k,k^*}$, for any $\epsilon \in (0,\bar{\epsilon})$ and any $\delta > 0$,
\begin{align*}
    \limsup_{\eta \downarrow 0}\sup_{\norm{x} \leq 2\epsilon}\text{(II)} < \frac{ Cc_*/4 }{|\bm{j}(\bm{i}^*)|C(\epsilon,\delta)}.
\end{align*}
In summary, we have established that for any $\epsilon \in (0,\bar{\epsilon})$ and any $\delta > 0$,
\begin{align}
    \limsup_{\eta\downarrow 0}\frac{ \sup_{ \norm{x}\leq 2\epsilon } \mathbb{P}_x\big( \textbf{B}^\times_k\ |\ \bigcap_{j = 1}^{k-1}(\textbf{A}^\times_j \cup \textbf{B}^\times_j)^c \cap\textbf{B}^\circ_j  \big)}{\delta^{ \alpha_1}\cdot\widetilde{\lambda}(\eta)} < Cc_*/4\ \forall k \geq 1. \label{proof prop first exit bound B cross k Rd}
\end{align}

Similarly, we can bound conditional probabilities of the form $\mathbb{P}_x\big( \textbf{A}^\times_k\ |\ \bigcap_{j = 1}^{k-1}(\textbf{A}^\times_j \cup \textbf{B}^\times_j)^c \cap\textbf{B}^\circ_j  \big)$. 
To be specific, recall that $\textbf{A}^\times_k = \cup_{i = 0}^6\textbf{A}^\times_{k,i}$.
Combining \cref{proof prop first exit A cross k 0 bound Rd}-\cref{proof prop first exit A cross k 6 bound Rd} with strong Markov property, we know the for all $\epsilon \in (0,\epsilon_6)$ and $k \geq 1$,
\begin{align}
     \limsup_{\eta \downarrow 0}\frac{\sup_{\norm{x}\leq 2\epsilon}\mathbb{P}_x\big( \textbf{A}^\times_k\ |\ \bigcap_{j = 1}^{k-1}(\textbf{A}^\times_j \cup \textbf{B}^\times_j)^c \cap\textbf{B}^\circ_j  \big)}{ \delta^{\alpha_1}\cdot \widetilde{\lambda}(\eta) } < Cc_*/4\ \ \forall \delta \in (0,\delta_6(\epsilon)).  \label{proof prop first exit bound A cross k Rd}
\end{align}
On the other hand, a lower bound can be established for conditional probability of event $\textbf{A}^\circ_k$, the event defined in \cref{def event A circ k prop main result Rd}.
Due to Lemma \ref{lemma lower bound typical exit Rd} and the strong Markov property at $\widetilde{\tau}_{k-1}$, one can see the existence of some $\epsilon_7 \in (0,\epsilon_6]$ such that, for any $\epsilon \in (0,\epsilon_7)$, there is some $0 < \delta_7(\epsilon) \leq \delta_6(\epsilon)$ such that
\begin{align}
    \liminf_{\eta\downarrow 0} \frac{  \inf_{\norm{x} \leq 2\epsilon}\mathbb{P}_x\big(\textbf{A}^\circ_{k}\ |\ \bigcap_{j = 1}^{k-1}(\textbf{A}^\times_j \cup \textbf{B}^\times_j)^c \cap\textbf{B}^\circ_j  \big) }{ \delta^{\alpha_1}\widetilde{\lambda}(\eta) } \geq c_*\ \ \ \forall \delta \in (0,\delta_7(\epsilon)),\ k \geq 1. \label{proof prop first exit bound A circ k Rd}
\end{align}

In order to apply the bounds \cref{proof prop first exit bound B cross k Rd}-\cref{proof prop first exit bound A circ k Rd}, we make use of the following inclusion relationship:
\begin{align}
    \big(\bigcap_{j = 1}^{k-1}(\textbf{A}^\times_j \cup \textbf{B}^\times_j)^c \cap\textbf{B}^\circ_j\big)\cap ( \textbf{A}^\times_k \cup \textbf{B}^\times_k )^c \subseteq \textbf{A}^\circ_k \cup \textbf{B}^\circ_k. \label{proof prop first exit inclusion A B x o Rd}
\end{align}
To see why this is true, let us consider a decomposition of the event on the L.H.S. of \cref{proof prop first exit inclusion A B x o Rd}. As mentioned above, on event $\bigcap_{j = 1}^{k-1}(\textbf{A}^\times_j \cup \textbf{B}^\times_j)^c \cap\textbf{B}^\circ_j$ we know that $\sigma(\eta) > \widetilde{\tau}_{k-1}$, 
implying that for the first $k-1$ cycles the first exit does not occur and leaving us with
only three possibilities on this event:
\begin{itemize}
    \item $\textbf{j}^{\geq J_\mathcal{G}}(k) = \infty$ (on $\{ \sigma(\eta) > \widetilde{\tau}_{k-1} \}$, this would happen if and only if $\mathcal{J}^\text{cycle}_k(\textbf{j}_k) < J_\mathcal{G}$);
    \item $\textbf{j}^{\geq J_\mathcal{G}}(k) < \infty, X^{\geq J_\mathcal{G}}(k) \notin \mathcal{G}$;
    \item $\textbf{j}^{\geq J_\mathcal{G}}(k) < \infty, X^{\geq J_\mathcal{G}}(k) \in \mathcal{G}$.
\end{itemize}
Let us partition the said event accordingly and analyze them one by one.
\begin{itemize}
    \item On $\big(\bigcap_{j = 1}^{k-1}(\textbf{A}^\times_j \cup \textbf{B}^\times_j)^c \cap\textbf{B}^\circ_j\big)\cap ( \textbf{A}^\times_k \cup \textbf{B}^\times_k )^c \cap \{ \mathcal{J}^\text{cycle}_k(\textbf{j}_k) < J_\mathcal{G} \}$,
    due to the exclusion of $\textbf{A}^\times_k$ (especially $\textbf{A}^\times_{k,1}$ and $\textbf{A}^\times_{k,3}$), we can see that if $\textbf{j}_k < l^*$, then we must have $\sigma(\eta) > \widetilde{\tau}_k$ and $\widetilde{\tau}_k - T_{k,1} \leq 2l^*\hat{t}(\epsilon)/\eta$. 
    In particular, note that on this event we must have $l^* > \textbf{j}_k$ by definition of $l^*$ in \cref{def l star Rd}.
    Therefore,
    \begin{align*}
        \big(\bigcap_{j = 1}^{k-1}(\textbf{A}^\times_j \cup \textbf{B}^\times_j)^c \cap\textbf{B}^\circ_j\big)\cap ( \textbf{A}^\times_k \cup \textbf{B}^\times_k )^c \cap \{ \textbf{j}^{\geq J_\mathcal{G}}(k) = \infty\} \subseteq \textbf{B}^\circ_k.
    \end{align*}
    \item On $\big(\bigcap_{j = 1}^{k-1}(\textbf{A}^\times_j \cup \textbf{B}^\times_j)^c \cap\textbf{B}^\circ_j\big)\cap ( \textbf{A}^\times_k \cup \textbf{B}^\times_k )^c \cap \{ \textbf{j}^{\geq J_\mathcal{G}}(k) < \infty, X^{\geq J_\mathcal{G}}(k) \notin \mathcal{G} \}$, 
    the exclusion of $\textbf{A}^\times_{k,2}$ implies that $T_{k,j} - T_{k,j-1} \leq 2\hat{t}(\epsilon)/\eta$ for all $j=2,\cdots,\textbf{j}^{\geq J_\mathcal{G}}(k) $,
    and the exclusion of $\textbf{A}^\times_{k,6}$ tells us that $\mathcal{J}^\text{cycle}_k( \textbf{j}^{\geq J_\mathcal{G}}(k) ) = J_\mathcal{G}$.
    Moreover, the exclusion of $\textbf{A}^\times_{k,5}\cup\textbf{A}^\times_{k,6}$ tells us that if $X^{\geq J_\mathcal{G}}(k) \notin \mathcal{G}$, we must have $\bm{d}(X^{\geq J_\mathcal{G}}(k),\mathcal{G}) > \epsilon$. In summary,
    \begin{align*}
        \big(\bigcap_{j = 1}^{k-1}(\textbf{A}^\times_j \cup \textbf{B}^\times_j)^c \cap\textbf{B}^\circ_j\big)\cap ( \textbf{A}^\times_k \cup \textbf{B}^\times_k )^c \cap \{ \textbf{j}^{\geq J_\mathcal{G}}(k) < \infty, X^{\geq J_\mathcal{G}}(k) \notin \mathcal{G} \} \subseteq \textbf{A}^\circ_k.
    \end{align*}
    \item On $\big(\bigcap_{j = 1}^{k-1}(\textbf{A}^\times_j \cup \textbf{B}^\times_j)^c \cap\textbf{B}^\circ_j\big)\cap ( \textbf{A}^\times_k \cup \textbf{B}^\times_k )^c \cap \{ \textbf{j}^{\geq J_\mathcal{G}}(k) < \infty, X^{\geq J_\mathcal{G}}(k) \in \mathcal{G} \}$,
    the same argument in the previous bullet point can be applied to show that
    $T_{k,j} - T_{k,j-1} \leq 2\hat{t}(\epsilon)/\eta$ for all $j=2,\cdots,\textbf{j}^{\geq J_\mathcal{G}}(k) $,
   $\mathcal{J}^\text{cycle}_k( \textbf{j}^{\geq J_\mathcal{G}}(k) ) = J_\mathcal{G}$,
    and $\bm{d}(X^{\geq J_\mathcal{G}}(k),\mathcal{G}^c) > \epsilon$.
 Now since $\textbf{B}^\times_k$ did not occur,
 we must have $\sigma(\eta) > \widetilde{\tau}_k$ and $\widetilde{\tau}_k - T^{\geq J_\mathcal{G}}(k) \leq \rho(\epsilon)/\eta$, hence $\widetilde{\tau}_k - T_{k,1} \leq \frac{2l^*\hat{t}(\epsilon) + \rho(\epsilon)}{\eta}$. Therefore,
    \begin{align*}
        \big(\bigcap_{j = 1}^{k-1}(\textbf{A}^\times_j \cup \textbf{B}^\times_j)^c \cap\textbf{B}^\circ_j\big)\cap ( \textbf{A}^\times_k \cup \textbf{B}^\times_k )^c \cap \{ \textbf{j}^{\geq J_\mathcal{G}}(k) < \infty, X^{\geq J_\mathcal{G}}(k) \in \mathcal{G}\} \subseteq \textbf{B}^\circ_k.
    \end{align*}
\end{itemize}
With \cref{proof prop first exit inclusion A B x o Rd} established, we can immediately get that 
\begin{align}
    \big(\bigcap_{j = 1}^{k-1}(\textbf{A}^\times_j \cup \textbf{B}^\times_j)^c \cap\textbf{B}^\circ_j\big)\cap ( \textbf{A}^\times_k \cup \textbf{B}^\times_k )^c = \big(\bigcap_{j = 1}^{k-1}(\textbf{A}^\times_j \cup \textbf{B}^\times_j)^c \cap\textbf{B}^\circ_j\big)\cap ( \textbf{A}^\times_k \cup \textbf{B}^\times_k )^c \cap (\textbf{A}^\circ_k \cup \textbf{B}^\circ_k). \label{proof prop first exit inclusion 2 prior Rd}
\end{align}
Next, recall the definitions of $\textbf{A}^\circ_k$ in \cref{def event A circ k prop main result Rd} and $\textbf{B}^\circ_k$ in \cref{def event B circ k Rd}, and one can see that the two events $\textbf{A}^\circ_k$ and $\textbf{B}^\circ_k$ are mutually exclusive, since the former implies that the first exit occurs during the $k-$th cycle whereas the latter implies the first exit does not occur in the first $k$ cycles.
This fact and \cref{proof prop first exit inclusion 2 prior Rd} allow us to conclude that 
\begin{align}
    \bigcap_{i = 1}^{k}\big( \textbf{A}^\times_i \cup \textbf{B}^\times_i \cup \textbf{A}^\circ_i \big)^c  = \bigcap_{i = 1}^{k} \big( \textbf{A}^\times_i \cup \textbf{B}^\times_i \big)^c \cap \textbf{B}^\circ_i = \Big(\bigcap_{i = 1}^{k-1} \big( \textbf{A}^\times_i \cup \textbf{B}^\times_i \big)^c \cap \textbf{B}^\circ_i\Big)\cap ( \textbf{A}^\times_k \cup \textbf{B}^\times_k \cup \textbf{A}^\circ_k )^c.  \label{proof prop first exit inclusion 2 Rd}
\end{align}

The last step is to use all the results so far to bound the probability of event
\begin{align*}
    \textbf{A}^\times(\epsilon,\delta,\eta) = \bigcup_{k \geq 1}\bigg( \bigcap_{i = 1}^{k-1}\big( \textbf{A}^\times_i \cup \textbf{B}^\times_i \cup \textbf{A}^\circ_i \big)^c    \bigg) \cap \big(\textbf{A}^\times_k \cup \textbf{B}^\times_k \big).
\end{align*}
Using \cref{proof prop first exit inclusion 2 Rd}, we can see that (for any $x$ with $\norm{x} \leq 2\epsilon$)
\begin{align*}
    & \mathbb{P}_x(\textbf{A}^\times(\epsilon,\delta,\eta)) \\
    = & \sum_{k \geq 1}\mathbb{P}_x\bigg( \bigg( \bigcap_{i = 1}^{k-1}\big( \textbf{A}^\times_i \cup \textbf{B}^\times_i \cup \textbf{A}^\circ_i \big)^c    \bigg) \cap \big(\textbf{A}^\times_k \cup \textbf{B}^\times_k \big) \bigg) \\
    = & \sum_{k \geq 1}\mathbb{P}_x\bigg( \bigg( \bigcap_{i = 1}^{k-1}\big( \textbf{A}^\times_i \cup \textbf{B}^\times_i \big)^c \cap \textbf{B}^\circ_i    \bigg) \cap \big(\textbf{A}^\times_k \cup \textbf{B}^\times_k \big) \bigg) \\
    = & \sum_{k \geq 1}\mathbb{P}_x\Big( \textbf{A}^\times_k \cup \textbf{B}^\times_k\ |\ \bigcap_{i = 1}^{k-1} \big( \textbf{A}^\times_i \cup \textbf{B}^\times_i \big)^c \cap \textbf{B}^\circ_i  \Big) \cdot \prod_{j = 1}^{k - 1}\mathbb{P}_x\Big( \bigcap_{i = 1}^{j}\big( \textbf{A}^\times_i \cup \textbf{B}^\times_i \big)^c \cap \textbf{B}^\circ_i\ \Big|\ \bigcap_{i = 1}^{j-1}\big( \textbf{A}^\times_i \cup \textbf{B}^\times_i \big)^c \cap \textbf{B}^\circ_i     \Big) \\
    = & \sum_{k \geq 1}\mathbb{P}_x\Big( \textbf{A}^\times_k \cup \textbf{B}^\times_k\ |\ \bigcap_{i = 1}^{k-1} \big( \textbf{A}^\times_i \cup \textbf{B}^\times_i \big)^c \cap \textbf{B}^\circ_i  \Big)
    \\ & \ \ \ \ \ \ \ \ \ \ \ \ \ \ \ \ \cdot \prod_{j = 1}^{k - 1}\mathbb{P}_x\Big( \big(\bigcap_{i = 1}^{j-1}\big( \textbf{A}^\times_i \cup \textbf{B}^\times_i \big)^c \cap \textbf{B}^\circ_i \big) \cap (\textbf{A}^\times_j \cup \textbf{B}^\times_j \cup \textbf{A}^\circ_j)^c \ \Big|\ \bigcap_{i = 1}^{j-1}\big( \textbf{A}^\times_i \cup \textbf{B}^\times_i \big)^c \cap \textbf{B}^\circ_i     \Big) \\
    = & \sum_{k \geq 1}\mathbb{P}_x\Big( \textbf{A}^\times_k \cup \textbf{B}^\times_k\ |\ \bigcap_{i = 1}^{k-1} \big( \textbf{A}^\times_i \cup \textbf{B}^\times_i \big)^c \cap \textbf{B}^\circ_i  \Big) \cdot \prod_{j = 1}^{k - 1}\mathbb{P}_x\Big( \big( \textbf{A}^\times_j \cup \textbf{B}^\times_j \cup \textbf{A}^\circ_j \big)^c \ \Big|\ \bigcap_{i = 1}^{j-1}\big( \textbf{A}^\times_i \cup \textbf{B}^\times_i \big)^c \cap \textbf{B}^\circ_i     \Big) \\
     \leq & \sum_{k \geq 1}\mathbb{P}_x\Big( \textbf{A}^\times_k \cup \textbf{B}^\times_k\ |\ \bigcap_{i = 1}^{k-1} \big( \textbf{A}^\times_i \cup \textbf{B}^\times_i \big)^c \cap \textbf{B}^\circ_i  \Big) \cdot \prod_{j = 1}^{k - 1}
     \Bigg(1 - \mathbb{P}_x\Big(  \textbf{A}^\circ_j \ \Big|\ \bigcap_{i = 1}^{j-1}\big( \textbf{A}^\times_i \cup \textbf{B}^\times_i \big)^c \cap \textbf{B}^\circ_i     \Big) \Bigg).
\end{align*}
This allows us to apply \cref{proof prop first exit bound B cross k Rd}-\cref{proof prop first exit bound A circ k Rd} and conclude that for all $\epsilon \in (0,\epsilon_7)$, all $\delta \in (0,\delta_7(\epsilon))$ and all $\eta > 0$
\begin{align}
    \sup_{|x| \leq 2\epsilon}\mathbb{P}_x(\textbf{A}^\times(\epsilon,\delta,\eta))
    \leq & \sum_{k \geq 1} \big( \delta^{\alpha_1}\widetilde{\lambda}(\eta)\cdot \frac{c_*C}{2} \big) \cdot \big( 1 -\delta^{\alpha_1}\widetilde{\lambda}(\eta)\cdot \frac{c_*}{2} \big)^{k-1} \nonumber 
    \\
    = & \frac{\delta^{\alpha_1}\widetilde{\lambda}(\eta)\cdot \frac{c_*C}{2}}{\delta^{\alpha_1}\widetilde{\lambda}(\eta)\cdot \frac{c_*}{2}}  = C.\nonumber 
\end{align}
when $\eta$ is sufficiently small.
\end{proof}

Having established Lemma \ref{lemma bound on event A cross final Rd}, we are ready to provide a proof for Proposition \ref{proposition main result Rd}.

 \begin{proof}[Proof of Proposition \ref{proposition main result Rd}.]
Define sets
\begin{align*}
    E^+(z) \delequal \{y\in\mathbb{R}^d:\ \bm{d}(y,\mathcal{G}) > (3+3\rho^*)z \},\ \ \ E^-(z) \delequal \{y\in\mathbb{R}^d:\ \bm{d}(y,\mathcal{G}^c) < (3+3\rho^*)z \}
\end{align*}
where $\rho^*$ is the fixed constant in Corollary \ref{corollary ODE ODE gap Rd}.
From the continuity of measure, we have
\begin{align*}
    \lim_{\epsilon \downarrow 0}\mu\big( h^{-1}( E^+(\epsilon) ) \big) = \mu\big( h^{-1}( (\overline{\mathcal{G}})^c ) \big), \ \ \ \lim_{\epsilon \downarrow 0}\mu\big( h^{-1}( E^-(\epsilon) ) \big) = \mu\big( h^{-1}(\mathcal{G}^c)\big).
\end{align*}
From Assumption \ref{assumption unique configuration for minimum exit cost} we know $q \delequal \mu\big( h^{-1}(\mathcal{G}^c)\big) > 0$.
From Assumption \ref{assumption boundary set with zero mass}
we also have $\mu\big( h^{-1}(\partial \mathcal{G}) \big) = 0$, hence $\mu\big( h^{-1}( (\overline{\mathcal{G}})^c ) \big) = q$ as well.
Therefore, together with Lemma \ref{lemma bound on event A cross final Rd}, we know the existence of some $\epsilon_1 \in (0, \frac{\bar{\epsilon}}{3 + 3\rho^*})$ such that, for any $\epsilon \in (0,\epsilon_1)$, there is some $0<\delta_1(\epsilon)$ with
\begin{align*}
    \limsup_{\eta \downarrow 0}\sup_{\norm{x} \leq 2\epsilon}\mathbb{P}_x(\textbf{A}^\times(\epsilon,\delta,\eta)) < C. 
\end{align*}
and
\begin{align*}
   (1-C)q < \mu\big( h^{-1}( E^+(\epsilon) ) \big)\leq q \leq \mu\big( h^{-1}( E^-(\epsilon) ) \big) < (1+C)q.
\end{align*}
Henceforth, we fix some $\epsilon$ and $\delta$ in this range and prove the inequalities in the Proposition.

To proceed, let 
$J(\epsilon,\delta,\eta)\delequal{}\sup\{ k \geq 1: \widetilde{\tau}_{k - 1} < \sigma(\eta)  \}$.
We now show that on event $(\textbf{A}^\times(\epsilon,\delta,\eta))^c$, we must have $J \leq J^\uparrow$ where (recall Definition \ref{definitionTypeJ Rd})
\begin{align*}
    J^\uparrow(\epsilon,\delta,\eta) \delequal \min\Big\{k \geq 1:\ Z_{T_{k,1}}\text{ is of type-}\big( E^+(\epsilon) ,\eta,\delta\big)\Big\}
\end{align*}
with $E^+(z) \delequal \{y\in\mathbb{R}^d:\ \bm{d}(y,\mathcal{G}) > (3+3\rho^*)z \}$ where $\rho^*$ is the fixed constant in Corollary \ref{corollary ODE ODE gap Rd}.
To see this via a proof of contradiction, 
let us assume that $J^\uparrow = j < J$ for some integer $j$. Now we make some observations.
\begin{itemize}
    \item First of all, due to $J^\uparrow = j < J$ and \cref{property bar delta bar t}, on event $(\textbf{A}^\times(\epsilon,\delta,\eta))^c \cap \{J^\uparrow = j < J\}$ we have
    \begin{align*}
        T_{j,i} - T_{j,i-1} \leq 2\bar{t}/\eta \ \forall i = 2,3,\cdots,k^*;
    \end{align*}
    \item The definition of event $(\textbf{A}^\times)^c$, especially the exclusion of event $\textbf{A}^\times_{j,0}$ defined in \cref{def A cross k 0 in prop main result Rd} now allows us to apply 
    Lemma \ref{lemma event B 1 cross Rd} 
    and show that
    \begin{align*}
    \sup_{ s \in [0,T_{j,k^*} - T_{j,1} ] }\norm{ X^\eta_{ \floor{s} + T_{j,1} } - \bm{z}^\eta(s)  } < 2\epsilon.
\end{align*}
Here 
\begin{align*}
    \bm{z}^\eta(s) = \widetilde{\bm{x}}^\eta\big(s,X^{(1)};\ \bm{T},\bm{W}  \big)\ \forall s \geq 0
\end{align*}
with
$\bm{T} = \big(0,T_{j,2} - T_{j,1}, T_{j,3} - T_{j,1}, \cdots, T_{j,k^*} - T_{j,k^* - 1}\big)$,
$\bm{W} = \big( W_{j,1},\cdots,W_{j,k^*} \big)$,
and $X^{(1)} = X^\eta_{ -1 + T_{j,1}}$.
Moreover, the previous bullet point allows us to apply Corollary \ref{corollary ODE ODE gap Rd} and conclude that
\begin{align*}
    \sup_{ s \in [0,T_{j,k^*} - T_{j,1} ] }\norm{ \bm{z}^\eta(s) - \bm{z}^\eta_0(s) } \leq 3\rho^*\epsilon
\end{align*}
where $ \bm{z}^\eta(s) = \widetilde{\bm{x}}^\eta\big(s,\bm{0};\ \bm{T},\bm{W}  \big)\ \forall s \geq 0$ and $\rho^*$ is the fixed constant in Corollary \ref{corollary ODE ODE gap Rd}.
As a result,
$ \norm{\bm{z}^\eta_0( T_{j,k^*} - T_{j,1} )  - X^\eta_{T_{j,k^*}} } < (2 + 3\rho^*)\epsilon$.
\item However, due to $J^\uparrow = j$, we have $\bm{d}(\bm{z}^\eta_0( T_{j,k^*} - T_{j,1} ),\mathcal{G})>(3 + 3\rho^*)\epsilon$.
This immediately implies that $X^\eta_{T_{j,k^*}} \notin \mathcal{G}$ and yields the contradiction $J \leq j$.
\end{itemize}
Similarly, one can show that on event $(\textbf{A}^\times)^c$ we must have $J \geq J^\downarrow$ for
\begin{align*}
    J^\downarrow(\epsilon,\delta,\eta) \delequal \min\Big\{k \geq 1:\ Z_{T_{k,1}}\text{ is of type-}\big( E^-(\epsilon) ,\eta,\delta\big)\Big\}
\end{align*}
with $E^-(z) \delequal \{y\in\mathbb{R}^d:\ \bm{d}(y,\mathcal{G}^c) < (3+3\rho^*)z \}$.

Besides, the following claim holds on event $(\textbf{A}^\times)^c \cap \{ J < \infty\}$.
\begin{itemize}
    \item The definition of $(\textbf{A}^\times)^c$ implies that during the $J-$th cycle the event $\textbf{A}^\circ_J$ occurs whereas for any $j < J$ we have $\textbf{B}^\circ_j$.
    Therefore,
    for any $k = 1,2,\cdots,J$, we have 
    $$\widetilde{\tau}_k\wedge \sigma(\eta) - T_{k,1} \leq \frac{2l^*\hat{t}(\epsilon) + \rho(\epsilon)}{\eta}.$$
    In particular, the coefficient $l^*$ on the R.H.S. can be justified as follows:
    by definition of $\textbf{A}^\circ_k$ and $\textbf{B}^\circ_k$,
    we can see that $\mathcal{J}^\text{cycle}_k(\textbf{j}_k) \leq J_\mathcal{G}$ for all $k \leq J$ (with $\textbf{j}_J = k^*$ due to the occurrence of $\textbf{A}^\circ_J$); the definition of $l^*$ in \cref{def l star Rd} then implies that $l^* > \textbf{j}_k$ for all $k \leq J$.
    
    \item Now we consider the following set
    \begin{align*}
        \textbf{S}(\epsilon,\delta,\eta)  \delequal{} & \cup_{k \geq 1}\{ \widetilde{\tau}_{k - 1} + 1, \widetilde{\tau}_{k - 1} + 2, \cdots, T_{k,1} - 1, T_{k,1} \}
    \end{align*}
    the can be understood as the concatenation of all steps between any return time $\widetilde{\tau}_{k-1}$ and the first large jump time $T_{k,1}$ during the $k-$th cycle.
    Our discussion above implies that, on event $(\textbf{A}^\times)^c$, we have
    then the discussion above have shown that, for
    \begin{align*}
        \min\{ n \in \textbf{S}(\epsilon,\delta,\eta):\ Z_n\text{ is of type-}(E^+(\epsilon),\eta,\delta) \} \geq T_{J,1}.
    \end{align*}
    It is worth noticing that the probability of $Z_1$ being of type-$(E^+(\epsilon),\eta,\delta)$ is 
    (see the definition in \cref{def Type conditional probability generalized}) $H(\delta/\eta)p\big( E^{+}(\epsilon),\delta,\eta\big).$
\end{itemize}
Therefore, for the following partition of the timeline
\begin{align*}
        \textbf{S}_\text{before}(\epsilon,\delta,\eta) & \delequal{} \{ n \in \textbf{S}(\epsilon,\delta,\eta):\ n \leq \sigma(\eta) \},\  I_{\text{before}}(\epsilon,\delta,\eta) \delequal{} \#\textbf{S}_\text{before}(\epsilon,\delta,\eta), 
        \\
         \textbf{S}_\text{after}(\epsilon,\delta,\eta) & \delequal{} \{ n \notin \textbf{S}(\epsilon,\delta,\eta):\ n \leq \sigma(\eta) \},\  I_{\text{after}}(\epsilon,\delta,\eta) \delequal{} \#\textbf{S}_\text{after}(\epsilon,\delta,\eta),
    \end{align*}
we have $\sigma(\eta) = I_{\text{before}} + I_{\text{after}}$.
Moreover, on event $(\textbf{A}^\times)^c$, we must have
\begin{align*}
        I_{\text{after}} & \leq J\big( 2l^*\hat{t}(\epsilon) + \rho(\epsilon) \big)/\eta \leq J^\uparrow\big( 2l^*\hat{t}(\epsilon) + \rho(\epsilon) \big)/\eta \\
        I_{\text{before}} & \leq \min\{ n \in \textbf{S}(\epsilon,\delta,\eta):\ Z_n\text{ is of type-}(E^+(\epsilon),\eta,\delta) \} \}.
\end{align*}
Next, define geometric random variables with the following success rates
\begin{align*}
U_1(\epsilon,\delta,\eta) & \sim \text{Geom}\Big( H(\delta/\eta)p\big( E^{+}(\epsilon),\delta,\eta\big)\Big), \\
    U_2(\epsilon,\delta,\eta) & \sim \text{Geom}\Big( p\big( E^{+}(\epsilon),\delta,\eta\big) \Big).
\end{align*}
Given the results above for bounding $I_{\text{before}}$ and $I_{\text{after}}$ on event $(\textbf{A}^\times)^c$,
we have
\begin{align}
    & \sup_{ \norm{x}\leq 2\epsilon }\mathbb{P}_x\Big( \lambda(\eta)\sigma(\eta) > u \Big)
    \nonumber
    \\
    \leq & \sup_{ \norm{x}\leq 2\epsilon }\mathbb{P}_x\big( (\textbf{A}^\times)^c \big)
    +
    \sup_{ \norm{x}\leq 2\epsilon }\mathbb{P}_x\Big( (\textbf{A}^\times)^c \cap \big\{ \lambda(\eta)(I_\text{before} + I_\text{after})> u   \big\} \Big)
    \nonumber \\
    \leq & \sup_{ \norm{x}\leq 2\epsilon }\mathbb{P}_x\big( (\textbf{A}^\times)^c \big)
    +
     \sup_{ \norm{x}\leq 2\epsilon }\mathbb{P}_x\Big( (\textbf{A}^\times)^c \cap \big\{ \lambda(\eta)I_\text{before}> (1-C)u   \big\} \Big) \nonumber
     \\
     &+
     \sup_{ \norm{x}\leq 2\epsilon }\mathbb{P}_x\Big( (\textbf{A}^\times)^c \cap \big\{ \lambda(\eta)I_\text{after}> Cu   \big\} \Big)
     \nonumber \\
     \leq & \sup_{ \norm{x}\leq 2\epsilon }\mathbb{P}_x\big( (\textbf{A}^\times)^c \big)
     +
     \underbrace{\P\big( \lambda(\eta)\cdot U_1 > (1-C)u \big)}_{ \triangleq \text{(I)} }
     +
     \underbrace{ \P\Big( \lambda(\eta)\cdot U_2 \cdot \frac{2l^*\hat{t}(\epsilon) + \rho(\epsilon)}{\eta} > Cu \Big) }_{ \triangleq \text{(II)} }. \nonumber
\end{align}
For term (I), recall that $\lambda(\eta) = H(1/\eta)\widetilde{\lambda}(\eta)$.
Corollary \ref{CorTypeProbRd} and the regularly varying nature of $H(\cdot)$ then imply that
\begin{align*}
    \lim_{\eta \downarrow 0}\frac{ H(\delta/\eta)p\big( E^{+}(\epsilon),\delta,\eta\big) }{ \lambda(\eta)} = \frac{ H(\delta/\eta) \cdot p\big( E^{+}(\epsilon),\delta,\eta\big) }{  \frac{H(1/\eta)}{\delta^{\alpha_1}}\cdot\delta^{\alpha_1}\widetilde{\lambda}(\eta)} = \mu\big( h^{-1}( E^+(\epsilon) ) \big) > (1-C)q.
\end{align*}
From Lemma \ref{lemmaGeomDistTail}, we then have $\limsup_{\eta\downarrow 0}\text{(I)} \leq \exp( - (1-C)^2u ).$
For term (II), let
$C(\epsilon) \delequal 2l^*\hat{t}(\epsilon) + \rho(\epsilon)$
and note that
$$\text{(II)} = \P\Big( U_2 > \frac{C}{C(\epsilon)}\cdot\frac{ \widetilde{\lambda}(\eta) }{\lambda(\eta)/\eta}\cdot \frac{1}{ \widetilde{\lambda}(\eta) }  \Big).$$
Moreover, since
$$\frac{ \widetilde{\lambda}(\eta) }{\lambda(\eta)/\eta} \in RV_{ -\alpha_1 + 1}(\eta),$$
we know that for any $M>0$, we have $\frac{ \widetilde{\lambda}(\eta) }{\lambda(\eta)/\eta} > M$ eventually as $\eta \downarrow 0$.
Therefore, given any $M > 0$, we have $\text{(II)} \leq \P\Big( U_2 > M/\widetilde{\lambda}(\eta)  \Big)$ for any $\eta$ sufficiently small.
Now from Lemma \ref{lemmaGeomDistTail}, $ \limsup_{\eta\downarrow 0}\text{(II)} \leq \exp(-M )$ for any $M > 0$,
and here we fix some $M > 0$ such that $\exp(-M ) < C$.
In summary, we have shown
\begin{align*}
    \limsup_{\eta \downarrow 0}\sup_{ \norm{x}\leq 2\epsilon }\mathbb{P}_x\Big( \lambda(\eta)\sigma(\eta) > u \Big) \leq 2C + \exp(-(1-C)^2qu)
\end{align*}
and established the upper bound.
The lower bound can be shown by an almost identical approach.
In particular, since
\begin{align}
    \inf_{ \norm{x}\leq 2\epsilon }\mathbb{P}_x\Big( \lambda(\eta)\sigma(\eta) > u \Big)
    \geq 
    & \inf_{ \norm{x}\leq 2\epsilon }\mathbb{P}_x\Big( (\textbf{A}^\times)^c \cap \big\{ \lambda(\eta)(I_\text{before} + I_\text{after})> u   \big\} \Big)
    \nonumber \\
    \geq & \inf_{ \norm{x}\leq 2\epsilon }\mathbb{P}_x\Big( (\textbf{A}^\times)^c \cap \big\{ \lambda(\eta)I_\text{before}> u   \big\} \Big)
    \nonumber \\ 
    \geq & \P\big( \lambda(\eta)\cdot U_1^\prime > u \big)  - \sup_{ \norm{x}\leq 2\epsilon }\mathbb{P}_x\Big( (\textbf{A}^\times)^c \Big)
   \nonumber
\end{align}
where $U^\prime_1(\epsilon,\delta,\eta) \sim \text{Geom}\Big( H(\delta/\eta)p\big( E^{-}(\epsilon),\delta,\eta\big)\Big)$.
The same calculation above for term (I) can be used here to provide a lower bound for $\liminf_{\eta \downarrow 0}\P\big( \lambda(\eta)\cdot U_1^\prime > u \big)$ and conclude the proof.
\end{proof}

Now we are ready to prove Theorem \ref{thm main result Rd}, the main theorem of this section.

\begin{proof}[Proof of Theorem \ref{thm main result Rd}.]
Recall that our choice of scaling function $\lambda(\cdot)$ in \cref{def scaling function lambda Rd} is regularly varying (w.r.t. $\eta$) with index $1 + J_\mathcal{G}$.
Fix some $x \in \mathcal{G}$, $t > 0$ and $C \in (0,1)$.
It suffices to show that
\begin{align*}
    \limsup_{\eta \downarrow 0}\mathbb{P}_x( \sigma(\eta)\lambda(\eta) > t ) \leq \exp(-q(1-C)^3t) + 2C,
    \\
    \liminf_{\eta \downarrow 0}\mathbb{P}_x( \sigma(\eta)\lambda(\eta) > t ) \geq \exp(-q(1+C)t) - C.
\end{align*}
First, we are able to pick some $\epsilon \in (0,1)$ sufficiently small such that $\bm{d}(x,\mathcal{G}^c) > \epsilon$
and Proposition \ref{proposition main result Rd} is applicable.
Due to Lemma \ref{lemma return to local minimum quickly Rd},
for event
$$A \delequal \{T_\text{return}(\eta,\epsilon) \leq \rho(\epsilon)/\eta,\ X^\eta_n \in \mathcal{G}\ \forall n \leq T_\text{return}(\eta,\epsilon) \},$$
we have $\lim_{\eta\downarrow 0}\P_x(A) = 1$.
Next, on event $A$, we have $\sigma(\eta) - T_\text{return}(\eta,\epsilon) > 0$ and $\norm{X^\eta_{T_\text{return}(\eta,\epsilon)}} \leq 2\epsilon$.
Moreover, by combining Proposition \ref{proposition main result Rd} with the strong Markov property at stopping time $T_\text{return}(\eta,\epsilon)$, we have
\begin{align*}
\limsup_{\eta\downarrow 0}\P_x\Big( \big(\sigma(\eta) - T_\text{return}(\eta,\epsilon)\big)\lambda(\eta) > (1-C)t\ |\ A \Big) & \leq 2C + \exp( -(1-C)^3qt ),
\\
\liminf_{\eta\downarrow 0}\P_x\Big( \big(\sigma(\eta) - T_\text{return}(\eta,\epsilon)\big)\lambda(\eta) > t\ |\ A \Big) & \geq -C + \exp( -(1+C)qt ).
\end{align*}
Observe that
\begin{align*}
    \P_x(\sigma(\eta)\lambda(\eta) > t) & \leq \P_x( \{\sigma(\eta)\lambda(\eta) > t\} \cap A ) + \P_x(A^c ) \\
    & \leq 
    \P_x\Big( \big(\sigma(\eta) - T_\text{return}(\eta,\epsilon)\big)\lambda(\eta) > (1-C)t\ |\ A \Big)\P_x(A) \\
    & + \P_x\Big( \big\{ T_\text{return}(\eta,\epsilon)\lambda(\eta) > Ct \big\}\cap A \Big) + \P_x(A^c ).
\end{align*}
Besides, on event $A$ we have $T_\text{return}(\eta,\epsilon) \leq O(1/\eta)$ as $\eta \downarrow 0$.
Given that $\lambda(\eta) \in RV_{(1 + J_\mathcal{G})}(\eta)$,
we have
$ T_\text{return}(\eta,\epsilon)\lambda(\eta) \leq Ct$
on event $A$ for all $\eta$ sufficiently small. 
Therefore, by applying the bounds above, we establish that
$ \limsup_{\eta \downarrow 0}\P_x(\sigma(\eta)\lambda(\eta) > t) \leq 2C + \exp( -(1-C)^3qt ).$
Similarly, in order to show the lower bound, observe that
\begin{align*}
    \P_x(\sigma(\eta)\lambda(\eta) > t) & \geq \P_x( \{\sigma(\eta)\lambda(\eta) > t\} \cap A )
    \\
    & \geq \P_x\Big( \big(\sigma(\eta) - T_\text{return}(\eta,\epsilon)\big)\lambda(\eta) > t\ |\ A \Big)\P_x(A).
\end{align*}
Taking $\liminf$ on both sides yields $\liminf_{\eta}\P_x(\sigma(\eta)\lambda(\eta) > t) \geq -C + \exp( -(1+C)qt )$ and concludes the proof.
\end{proof}

\counterwithin{equation}{section}
\section{Proof of Technical Lemmas} \label{section: proof of technical lemmas}

\begin{proof}[Proof of Lemma~\ref{lemmaGeomDistTail}.]
For any $\epsilon > 0$,
\begin{align*}
    \mathbb{P}\Big( U(\epsilon) > \frac{1}{b(\epsilon)} \Big) & = \Big(1 - a(\epsilon)\Big)^{ \floor{1/b(\epsilon)} }.
\end{align*}
By taking logarithm on both sides, we have
\begin{align*}
    \ln \mathbb{P}\Big( U(\epsilon) > \frac{1}{b(\epsilon)} \Big) & = \floor{1/b(\epsilon)}\ln\Big(1 - a(\epsilon)\Big) \\
    & = \frac{\floor{1/b(\epsilon)} }{ 1/b(\epsilon) }\frac{\ln\Big(1 - a(\epsilon)\Big) }{-a(\epsilon) }\frac{-a(\epsilon) }{ b(\epsilon) }.
\end{align*}
Since $\lim_{x \rightarrow 0}\frac{\ln(1 + x)}{x} = 1$, we know that for $\epsilon$ sufficiently small, we will have
\begin{align}
 -c \frac{a(\epsilon)}{b(\epsilon)}   \leq \ln \mathbb{P}\Big( U(\epsilon) > \frac{1}{b(\epsilon)} \Big) \leq -\frac{a(\epsilon)}{c\cdot b(\epsilon)}. \label{proofGeomBound}
\end{align}
$$$$
By taking exponential on both sides, we conclude the proof.
\end{proof}

\begin{proof}[Proof of Lemma~\ref{lemmaGeomFront}.]
To begin with, for any $\epsilon > 0$ we have
\begin{align*}
    \mathbb{P}\Big( U(\epsilon) \leq \frac{1}{b(\epsilon)} \Big) & = 1 -  \mathbb{P}\Big( U(\epsilon) > \frac{1}{b(\epsilon)} \Big).
\end{align*}
Using bound \cref{proofGeomBound}, we know that for $\epsilon$ sufficiently small, $\mathbb{P}(U(\epsilon) > 1/b(\epsilon)) \geq \exp(-c\cdot a(\epsilon)/b(\epsilon))$. 
The upper bound follows from the generic bound $1-\exp(-x) \leq x,\ \forall x\in \mathbb{R}$ with $x=c\cdot a(\epsilon)/b(\epsilon)$.

Now we move onto the lower bound. Again, from bound \cref{proofGeomBound}, we know that for sufficiently small $\epsilon$, we will have
$$\mathbb{P}\Big( U(\epsilon)  \leq \frac{1}{b(\epsilon)} \Big) \geq 1 - \exp(-\frac{1}{\sqrt{c}}\cdot\frac{a(\epsilon)}{b(\epsilon)}). $$
Due to the assumption that $\lim_{\epsilon \downarrow 0}a(\epsilon)/b(\epsilon) = 0$ and the fact that $1 - \exp(-x) \geq \frac{x}{\sqrt{c}}$ for $x>0$ sufficiently close to $0$, we will have (for $\epsilon$ small enough)
$ \mathbb{P}\Big( U(\epsilon)  \leq \frac{1}{b(\epsilon)} \Big) \geq \frac{1}{c}\cdot\frac{a(\epsilon)}{b(\epsilon)}. $
\end{proof}

\begin{proof}[Proof of Lemma~\ref{lemmaBasicGronwall}.]
Let $a_k \delequal{} x_k - \widetilde{x}_k$. Using intermediate value theorem, one can easily see that
\begin{align*}
    a_k & = \eta\sum_{j = 1}^k\Big( \nabla g(\widetilde{x}_{j - 1}) - \nabla g(x_{j-1})  \Big) + \eta(z_1 + \cdots +z_k) + x - \widetilde{x}; \\
    \Rightarrow \norm{a_k} & \leq \eta C (\norm{a_0} + \cdots \norm{a_{k-1}}) + \widetilde{c}.
\end{align*}
The desired bound then follows immediately from Gronwall's inequality.
\end{proof}

\begin{proof}[Proof of Lemma~\ref{lemma bound on interarrival time and jump sizes for typical large jumps}]
Due to the finiteness of all types $\bm{j}^\prime$ with $\mathcal{J}_\text{type}(\bm{j}^\prime) < J_\mathcal{G}$,
it suffices to fix one of such $\bm{j}^\prime$ and prove the existence of constants $\bar{t},\bar{\delta},\epsilon_0$.

Let $\bar{\epsilon}$ be the constants in \cref{ineq ODE geometric convergence to origin}. For any $r \in (0,\bar{\epsilon})$, define the following stopping time
\begin{align*}
    \tau_r(x) \delequal{} \min\{ t \geq 0:\ \bm{x}_t(x) \in \overline{B(\bm{0},r)} \}.
\end{align*}
Due to Assumptions \ref{assumption 4 jacobian of g at the origin} and \ref{assumption 5 zero is an attractor} as well as continuity in Poincar\'e Map (see Theorem 12 in \cite{immler2019flow}), for any fixed $r > 0$ we know that $\tau_r(\cdot)$ is a continuous function on $\overline{\mathcal{G}}$. 
Note that, by definition, we have $\tau_r(x) = 0$ for any $x \in \overline{B(\bm{0},r)}$. 
Due to $\overline{\mathcal{G}}$ being compact and \cref{ineq ODE geometric convergence to origin}, we know the existence of some constant $T_r \in (0,\infty)$ such that $\sup_{x \in \bar{\mathcal{G}}}\tau_r(x) \leq T_r$ and $\bm{x}_t(x) \in \overline{B(\bm{0},r)}$ for any $x \in \overline{\mathcal{G}}, t \geq T_r$. 

Due to Assumption \ref{assumption bounded away} and the finiteness of all types $\bm{j}^\prime$ with $\mathcal{J}_\text{type}(\bm{j}^\prime) < J_\mathcal{G}$, there exists a $\epsilon_0 \in (0,\bar{\epsilon}/2)$ such that 
\begin{align}
    \sup_{s \geq 0}\bm{d}\big( \widetilde{\bm{x}}(s,\bm{0}; \bm{t}, \bm{w} ) ,\mathcal{G}^c \big) > 2\epsilon_0 \label{proof lemma bound on jump size and time key}
\end{align}
for all  $\bm{j}^\prime$ with $\mathcal{J}_\text{type}(\bm{j}^\prime) < J_\mathcal{G}$, $\bm{t}\in \{0\} \times \mathbb{R}^{|\bm{j}^\prime| - 1}_+,\bm{w} \in  \mathcal{A}^\text{type}_{\bm{j}^\prime}$.
We fix such $\epsilon_0 > 0$.
Here is one implication that is worth mentioning.
Recall that $\bm{j}$ is fixed in the description of the Lemma and $k^\prime = |\bm{j}|$.
If
\begin{align*}
    \inf_{s \geq 0}\bm{d}\big( \widetilde{\bm{x}}(s,\bm{0}; \bm{t}^{(k^\prime)}, \bm{w}^{ (k^\prime) } ) ,\mathcal{G}^c \big) \leq \epsilon_0
\end{align*}
for some $\bm{t}^{(k^\prime)}\in \{0\} \times \mathbb{R}^{k^\prime - 1}_+, \bm{w}^{(t^\prime)} \in  \mathcal{A}^\text{type}_{\bm{j}}$, 
then we must have 
\begin{align*}
    \bm{d}\big( \widetilde{\bm{x}}(\sum_{ i = 1 }^{k^\prime}t_i,\bm{0}; \bm{t}^{(k^\prime)}, \bm{w}^{ (k^\prime) } ) ,\mathcal{G}^c \big) \leq \epsilon_0.
\end{align*}

To show the existence of some $\bar{\delta} > 0$ we appeal to a proof by contradiction. 
(For a clean presentation, in $\bm{t}$ and $\bm{w}$ we omit the $(k^\prime)$ term in the superscript since the cardinality is fixed.)
Assume that we can find a sequence $\big( \bm{t}^{n}, \bm{w}^{n} \big)_{n \geq 1}$ such that 
$\min\big\{ \norm{w^{n}_1},\cdots, \norm{w^{n}_{k^\prime}}\big\} \leq 1/n$ 
and 
$\bm{d}\big(\widetilde{\bm{x}}\big( \sum_{i = 1}^{k^\prime} t^{n}_{i}, \bm{0};\ \bm{t}^{n}, \bm{w}^{n}\big), \mathcal{G}^c\big) \leq \epsilon_0$ for any $n$. Due to the truncation operator $\varphi_b(\cdot)$ in the definition of $\widetilde{\bm{x}}_t$, without loss of generality we can replace all jumps $w^{n}_j$ by $\varphi_b\big( w^{n}_j\big)$ to ensure that $w^{n}_j$ is always in a compact set. Define 
\begin{align*}
    \bm{y}^{n}_j = \widetilde{\bm{x}}\big( t^{n}_j, \bm{0};\ \bm{t}^{n}, \bm{w}^{n}  \big)\ \forall j = 1,2,\cdots,k^\prime.
\end{align*}
By picking a subsequence when necessary, we can further assume that
\begin{itemize}
    \item For any $i \geq 1$, $w^{(n)}_j$ converges to some $w^*_i$ and $\bm{y}^{n}_i$ converges to some $\bm{y}^*_i$. In particular, there exists some $I \in [k^\prime]$ such that $w^*_I = \bm{0}$.
    \item Also, since $\mathcal{G}^c$ is a closed set, we have $\bm{d}(\bm{y}^*_{k^\prime},\mathcal{G}^c)\leq \epsilon_0$.
    \item For any $i = 2,3,\cdots,k^\prime$, $t^{n}_i$ either converges to some finite $t^*_i$, or $\lim_n t^{n}_i = \infty$.
    \item Note that
    \begin{align}
        \lim_n t^{n}_i = \infty \Rightarrow \bm{y}^*_i = \bm{0} \label{proof infinite time in ODE causes return to origin}
    \end{align}
    for the following reason: for any $r \in (0,\bar{\epsilon})$, our discussion about $\tau_r(x)$ at the beginning of the proof implies that $\limsup_n \norm{ \bm{y}^{n}_i } \leq r$.
\end{itemize}
Let $\widetilde{I}$ be the largest index $i \in [k^\prime]$ such that $\bm{y}^*_i = \bm{0}$. In case that we cannot find such index, let $\widetilde{I} = 0$. 
We first consider the case where $\widetilde{I} = 0$.
The bullet points above then imply that, for any $i \geq 1$, we must have $\lim_n t^{n}_i = t^*_i < \infty$. 
Now due to boundedness of $t^*_i$ and the continuity of ODE flow, for
\begin{align*}
    \bm{t}^* & \delequal{} (0,t^*_{  2 }, t^*_{ 2 }, \cdots, t^*_{k^\prime} ), \\
    \bm{w}^* & \delequal{} (w^*_{ 1}, w^*_{2}, w^*_{ 3}, \cdots, w^*_{k^\prime} ),
\end{align*}
we have
\begin{align}
    \widetilde{\bm{x}}( \sum_{i = 1}^{k^\prime}t^*_{i}, \bm{0}; \bm{t}^*, \bm{w}^*) = \bm{y}^*_{k^\prime}. \nonumber 
\end{align}
Now recall that $w^*_I = 0$ for some $I \in [k^\prime]$.
By removing the vacuous (size-zero) jumps in $(\bm{t}^*, \bm{w}^*)$, we now know the existence of some $\widetilde{k}^\prime \in \mathbb{N}$, $\widetilde{\bm{w}}^* \in (\mathbb{R}^d\symbol{92}\mathbb{O})^{\widetilde{k}^\prime}, \widetilde{\bm{t}}^* \in \{0\}\times \mathbb{R}_+^{\widetilde{k}^\prime}$ such that $\bm{d}\big(\widetilde{\bm{x}}( \sum_{i = 1}^{\widetilde{k}^\prime}\widetilde{t}^*_{i}, \bm{0}; \widetilde{\bm{t}}^*, \widetilde{\bm{w}}^*), \mathcal{G}^c\big) \leq \epsilon_0$.
Meanwhile, the condition $\mathcal{J}_\text{type}^\downarrow(\bm{j}) < J_\mathcal{G}$ implies that, for the total cost of jumps in $\widetilde{\bm{w}}^*$ after removal of size-zero jumps, $ \sum_{i = 1}^{ \widetilde{k}^\prime }J(\widetilde{\bm{w}}^*_j) < \mathcal{J}_G$.
However, this contradicts \cref{proof lemma bound on jump size and time key}.

Similarly, in the case with $\widetilde{I} \geq 1$, we have that for any $i > \widetilde{I}$, we must have $\lim_n t^{n}_i = t^*_i < \infty$.
From the boundedness of $t^*_i$ and the continuity of ODE flow, for
\begin{align}
    \bm{t}^* & \delequal{} (0,t^*_{ \widetilde{I} + 1 }, t^*_{ \widetilde{I} + 2 }, \cdots, t^*_{k^\prime} ), \label{proof lemma bound on jump size and time, t star} \\
    \bm{w}^* & \delequal{} (w^*_{ \widetilde{I}}, w^*_{ \widetilde{I} + 1}, w^*_{ \widetilde{I} + 2}, \cdots, w^*_{k^\prime} ), \label{proof lemma bound on jump size and time, w star} 
\end{align}
we have
\begin{align}
    \widetilde{\bm{x}}( \sum_{i = \widetilde{I} + 1}^{k^\prime}t^*_{i}, \bm{0}; \bm{t}^*, \bm{w}^*) = \bm{y}^*_{k^\prime}. \label{proof lemma bound on interarrival time and jump size}
\end{align}
In particular, if $\widetilde{I} = 1$, then for the index $I$ with $w^*_I = \bm{0}$, we must have that $I \geq \widetilde{I}$, meaning that there is at least one vacuous jump in $\bm{w}^*$. 
The same argument for the case $\widetilde{I} = 0$ above can lead to the same contradiction with \cref{proof lemma bound on jump size and time key}.
Otherwise, with $\widetilde{I} \geq 2$, we already know that the accumulated cost of all jumps in $w^*$ is strictly less than $J_\mathcal{G}$, yet we still have $\bm{d}\big(\widetilde{\bm{x}}( \sum_{i = \widetilde{I} + 1}^{k^\prime}t^*_{i}, \bm{0}; \bm{t}^*, \bm{w}^*), \mathcal{G}^c  \big) \leq \epsilon_0$. This contradicts \cref{proof lemma bound on jump size and time key} again.

In summary, we have established the existence of the lower bound $\bar{\delta} > 0$ on jump sizes. We fix such $\bar{\delta} > 0$.
The existence of $\bar{t} < \infty$ can be shown by an almost identical argument. In particular, if such $\bar{t}<\infty$ does not exist, then by picking a subsequence if needed we are able to find a converging sequence $\big( \bm{t}^{n}, \bm{w}^{n} \big)_{n \geq 1}$ such that \cref{proof lemma bound on interarrival time and jump size} holds and $\widetilde{I} \geq 1$ due to inter-arrival time blowing up to infinity. 
In particular, since $\bm{y}^*_1 = w^*_1$ and $\norm{ w^*_1} \geq \bar{\delta} > 0$, we must have $\widetilde{I} \geq 2$.
By considering the same $\bm{t}^*,\bm{w}^*$ pair in \eqref{proof lemma bound on jump size and time, t star} and \eqref{proof lemma bound on jump size and time, w star}, one can see that the accumulated cost of jumps in $\bm{w}$ is strictly less than $J_\mathcal{G}$ and yield a contradiction with \cref{proof lemma bound on jump size and time key}.
\end{proof}

\begin{proof}[Proof of Lemma \ref{lemma bound on interarrival time and jump sizes for typical large jumps generalized}.]
Due to the finiteness of all types $\bm{j}^\prime$ with $\mathcal{J}_\text{type}(\bm{j}^\prime) < J_\mathcal{G}$,
it suffices to fix one of such $\bm{j}^\prime$ and prove the existence of the positive constant $\epsilon_1$.

To proceed with a proof by contradiction, the assume that such $\epsilon_1 > 0$ does not exists.
(For a clean presentation, in $\bm{t}$ and $\bm{w}$ we omit the $(k^\prime)$ term in the superscript since the cardinality is fixed.)
As a result, we are able to pick a sequence $(x^n,\bm{t}^n,\bm{w}^n)_{n \geq 1}$ such that one of the following two cases must occur:
\begin{itemize}
    \item $\lim_{n}\norm{x^n} = 0$; For any $n \geq 1$,  $\inf_{s \geq 0} \bm{d}\big(\widetilde{\bm{x}}( s, x^n;\ \bm{t}^{n},\bm{w}^{n}), \mathcal{G}^c\big) \leq \epsilon_0/2$ and there is some $j_n \in [k^\prime]$ such that $\bm{t}^n_{j_n} \geq 2\bar{t}$;
    \item $\lim_{n}\norm{x^n} = 0$; For any $n \geq 1$,  $\inf_{s \geq 0} \bm{d}\big(\widetilde{\bm{x}}( s, x^n;\ \bm{t}^{n},\bm{w}^{n}), \mathcal{G}^c\big) \leq \epsilon_0/2$ and there is some $j_n \in [k^\prime]$ such that $\norm{\bm{w}^n_{j_n}} \leq \bar{\delta}/2$.
\end{itemize}
We detail the analysis for the first case, as the second case can be addressed by an almost identical argument.
First of all, due to the truncation operator $\varphi_b(\cdot)$ in the definition of $\widetilde{\bm{x}}_t$,
without loss of generality we can replace all jumps $w^{n}_j$ by $\varphi_b\big( w^{n}_j\big)$ to ensure that $w^{n}_j$ is always in a compact set.
Define 
\begin{align*}
    \bm{y}^{n}_j = \widetilde{\bm{x}}\big( t^{n}_j, x^n;\ \bm{t}^{n}, \bm{w}^{n}  \big)\ \forall j = 1,2,\cdots,k^\prime.
\end{align*}
By picking a subsequence if necessary, we can further assume that
\begin{itemize}
    \item For any $i \geq 1$, $w^{(n)}_j$ converges to some $w^*_i$ and $\bm{y}^{n}_i$ converges to some $\bm{y}^*_i$.
    \item Also, since $\mathcal{G}^c$ is a closed set, we have $\bm{d}(\bm{y}^*_{k^\prime},\mathcal{G}^c)\leq \epsilon_0/2$.
    \item For any $i = 2,3,\cdots,k^\prime$, $t^{n}_i$ either converges to some finite $t^*_i$, or $\lim_n t^{n}_i = \infty$. In particular, there is some $I \in [k^\prime]$ such that $\lim_n t^{n}_I \geq 2\bar{t}$.
    \item Due to \cref{proof infinite time in ODE causes return to origin}, $\lim_n t^{n}_i = \infty$ would imply $\bm{y}^*_i = \bm{0}$.
\end{itemize}
Let $I_1 = \max\{i \in [k^\prime]:\ \lim_n t^n_i \in [2\bar{t},\infty)\}$ and $I_2 = \max\{i \in [k^\prime]:\ \lim_n t^n_i = \infty\}$.
If either of the two sets above is empty, let the corresponding $I_1$ or $I_2$ be 0.
The discussion above implies that at least one of them must be non-zero, so there are only two possibilities:
(i) $0 \leq I_1 < I_2$;
(ii) $0 \leq I_2 < I_1$.
We consider each scenario respectively.

First, if $0 \leq I_1 < I_2$, then $\lim_{n}y^n_{I_2} = y^*_{I_2} = \bm{0}$, implying that $I_2 < k^\prime$.
Moreover, for any $i = I_2 + 1, I_2 + 2,\cdots,k^\prime$, $\lim_n t^n_i = t^*_i < \infty$.
Using the boundedness of $t^*_i$ and the continuity of ODE flow, for
\begin{align*}
    \bm{t}^* \delequal{} (0,t^*_{ 1 + I_2 }, t^*_{ 2 + I_2 }, \cdots, t^*_{k^\prime} ), 
    \ \ \ \bm{w}^* & \delequal{} (w^*_{ I_2}, w^*_{ 1 + I_2}, w^*_{ 2 + I_2}, \cdots, w^*_{k^\prime} ), 
\end{align*}
we have
$ \widetilde{\bm{x}}( \sum_{i = 1 + I_2}^{k^\prime}t^*_{i}, \bm{0}; \bm{t}^*, \bm{w}^*) = \bm{y}^*_{k^\prime}$ with $\bm{d}(\bm{y}^*_{k^\prime},\mathcal{G}^c) \leq \epsilon_0/2.$
However, this contradicts \eqref{ineq insufficient cost, bounded away from Gc}.

Next, in scenario (ii) with $0 \leq I_2 < I_1$, we have $\lim_{n}t^n_{I_1} = t^*_{I_1} \geq 2\bar{t}$.
Moreover, for any $i = I_2 + 1, I_2 + 2,\cdots,k^\prime$, $\lim_n t^n_i = t^*_i < \infty$.
Using the boundedness of $t^*_i$ and the continuity of ODE flow, for
\begin{align*}
    \bm{t}^* \delequal{} (0,t^*_{ 1 + I_2 }, t^*_{ 2 + I_2 }, \cdots, t^*_{k^\prime} ), 
    \ \ \ \bm{w}^* & \delequal{} (w^*_{ I_2}, w^*_{ 1 + I_2}, w^*_{ 2 + I_2}, \cdots, w^*_{k^\prime} ), 
\end{align*}
we have
$ \widetilde{\bm{x}}( \sum_{i = 1 + I_2}^{k^\prime}t^*_{i}, \bm{0}; \bm{t}^*, \bm{w}^*) = \bm{y}^*_{k^\prime}$ with $\bm{d}(\bm{y}^*_{k^\prime},\mathcal{G}^c) \leq \epsilon_0/2.$
However, $bm{w}^*$ is still of type $\bm{j}$ with $\mathcal{J}^\downarrow_\text{type}(\bm{j}) < J_\mathcal{G}$, and there is some $i$ such that $\bm{t}^*_i \geq 2\bar{t}$.
This would contradict Lemma \ref{lemma bound on interarrival time and jump sizes for typical large jumps}.

In summary, we have established the existence of $\epsilon_1 > 0$ such that $t_j < 2\bar{t}$ is a necessary condition for any $x \in B(\bm{0},\epsilon_1)$
and any $\bm{t} = (t_1,\cdots,t_{k^\prime})\in \{0\} \times \mathbb{R}^{k^\prime - 1}_+,\bm{w} = (w_1,\cdots,w_{k^\prime}) \in  \mathcal{A}^\text{type}_{\bm{j}}$
with $\inf_{s \geq 0} \bm{d}\big(\widetilde{\bm{x}}( s, x;\ \bm{t}^{(k)},\bm{w}^{(k)}), \mathcal{G}^c\big) \leq \epsilon_0/2$.
As mentioned above, the necessity of $\norm{w_j} > \bar{\delta}/2$ can be shown in an almost identical way.
We omit the details here and conclude the proof.
\end{proof}

\begin{proof}[Proof of Lemma~\ref{lemma bound on small perturbation at origin}]
Since there are only finitely many $(i_1,\cdots,i_m)$ with $\mathcal{J}(i_1,\cdots,i_m) < J_\mathcal{G}$, it suffices to fix one of such $(i_1,\cdots,i_m)$ and establish the existence of the required $\epsilon_0, \delta_0$. (Henceforth, let $k = \sum_{j = 1}i_m$.)  

Assumption \ref{assumption bounded away}, together with the bound in \cref{ineq LB on distance to boundary set}, implies the existence of some $\epsilon_1 \in (0,\bar{\epsilon})$ such that
\begin{align}
    \sup_{t \geq 0}\bm{d}\big( \widetilde{\bm{x}}(s,\bm{0};\ \bm{t}^{(k)},\bm{w}^{(k)}  \big) > \epsilon_1 \label{proof bound on small perturbation at origin contradiction}
\end{align}
for any $\bm{w}^{(k)} \in \mathcal{A}(i_1,\cdots,i_m),\ \bm{t}^{(k)}\in \{0\}\times \mathbb{R}^{k-1}_+$. For $\epsilon_0 = \epsilon_1 / 2$, we establish the existence of the prescribed $\delta_0$ via a proof by contradiction.
(Henceforth we drop the notational dependence on $(k)$ when referencing sequences $\bm{t}^{(k)},\bm{w}^{(k)}$ since the cardinality $k$ is fixed.)

Assume the existence of a sequence $(x_n)_{n \geq 1}$ in $\mathbb{R}^d$, 
a sequence of real positives $(s_n)_{n \geq 1}$, 
a sequence $(\bm{t}_n)_{n \geq 1} = (t_{1,n},\cdots,t_{k,n})_{n \geq 1}$ in $\{0\}\times \mathbb{R}^{k-1}_+$,
and a sequence $(\bm{w}_n)_{n \geq 1} = (w_{1,n},\cdots,w_{k,n})_{n \geq 1}$ in $\mathcal{A}(i_1,\cdots,i_m)$ such that 
$\lim_n\norm{x_n} = 0$ and 
\begin{align*}
    \bm{d}\big( \widetilde{\bm{x}}(s_n, x_n;\ \bm{t}_n,\bm{w}_n), \mathcal{G}^c\big) \leq \epsilon_0.
\end{align*}
Due to existence of the clipping operator, all jumps $w_{j,n}$ can be replaced by $\varphi_b(w_{j,n})$ without loss of generality to ensure that all $w_{j,n}$ are in a compact set. Also, without loss of generality, all $s_n$ can be chosen as
\begin{align*}
    s_n = \inf\{s \geq 0:\  \bm{d}\big( \widetilde{\bm{x}}(s, x_n;\ \bm{t}_n,\bm{w}_n), \mathcal{G}^c\big) \leq \epsilon_0 \}.
\end{align*}
From Assumption \cref{ineq LB on distance to boundary set}, one can see that $s_n$ must be equal to $\sum_{i = 1}^{j_n}t_{i,n}$ for some $j_n \in [k]$, i.e., it must be the arrival time of some jump. Moreover, one can easily see that for 
\begin{align*}
    y_n = \widetilde{\bm{x}}(s_n, x_n;\ \bm{t}_n,\bm{w}_n),
\end{align*}
we always have $\bm{d}(y_n,\mathcal{G}) \leq b$ so all $y_n$ are in a compact set as well due to the boundedness of $\mathcal{G}$. Therefore, by picking a subsequence when necessary, we can further assume that
\begin{itemize}
    \item There exists some $j^* \in [k]$ such that $s_n = \sum_{i = 1}^{j^*}t_{i,n}$ for all $n \geq 1$, i.e. $s_n$ is always the arrival time of the $j^*-$th jump;
    \item For any $j \in [j^*]$, there exists some $y^*_j$ such that for $y_{n,j} = \widetilde{\bm{x}}( \sum_{i = 1}^j t_{i,n}, x_n;\ \bm{t}_n, \bm{w}_n)$ we have $\lim_n y_{n,j} = y^*_j$; In particular, for any $j < j^*$ we have $\bm{d}(y^*_j, \mathcal{G}^c) >\epsilon_1$ and $\bm{d}(y^*_{j^*}, \mathcal{G}^c) \leq \epsilon_0$;
    \item For any $j\in[k]$, there exists some $w^*_j$ such that $\lim_n w_{j,n} = w^*_j$; Moreover, note that $y^*_1 = w^*_1$.
    \item In particular, $(w^*_1,\cdots,w^*_{k}) \in \mathcal{A}(i_1,\cdots,i_m)$;
    \item For any $j \in [k]$, either there exists some $t^*_j < \infty$ such that $\lim_n t_{j,n} = t^*_j$, or $\lim_n t_{j,n} = \infty$ (in this case we let $t^*_j = \infty$); in the latter case, due to the same argument in \cref{proof infinite time in ODE causes return to origin}, we must have $\lim_n \widetilde{\bm{x}}( \sum_{i = 1}^j t_{i,n}, x_i;\ \bm{t}_n, \bm{w}_n) = \bm{0}$;
\end{itemize}
Obviously, $y^*_{j^*} \neq \bm{0}$. Let $j_\downarrow \delequal \max\{ j = 0,1,\cdots,j^*:\ y^*_j = \bm{0} \}$ with the convention that $y^*_0 = \bm{0}, y^*_{0,n} = x_n$. We must have $j_\downarrow < j^*$ and 
\begin{itemize}
    \item $\lim_n y^*_{j_\downarrow,n} = y^*_{j_\downarrow} = \bm{0}$
    \item For any $j = j_\downarrow + 1, j_\downarrow + 2, \cdots, j^*$, $\lim_n t_{j,n} = t^*_j < \infty$. 
\end{itemize}
Now using the continuity of the ODE flow, we must have that
\begin{align*}
    \lim_n & \widetilde{\bm{x}}\big( \sum_{j = j_\downarrow + 1}^{j^*} t_{j,n}, y^*_{j_\downarrow,n};\ (t_{j_\downarrow + 1,n},t_{j_\downarrow + 2,n},\cdots,t_{j^*,n}), (w_{j_\downarrow + 1,n},w_{j_\downarrow + 2,n},\cdots,w_{j^*,n})  \big) \\
    = & \widetilde{\bm{x}}\big( \sum_{j = j_\downarrow + 1}^{j^*} t^*_{j}, \bm{0};\ (t^*_{j_\downarrow + 1},t^*_{j_\downarrow + 2},\cdots,t^*_{j^*}), (w^*_{j_\downarrow + 1},w^*_{j_\downarrow + 2},\cdots,w^*_{j^*})  \big) =  y^*_{j^*}.
\end{align*}
However, due to $\bm{d}(y^*_{j^*}, \mathcal{G}^c) \leq \epsilon_0$ and recall our choice of $\epsilon_0 = \epsilon_1 /2$, for all $n$ sufficiently large, we must have 
\begin{align*}
    \bm{d}\Big(\widetilde{\bm{x}}\big( \sum_{j = j_\downarrow + 1}^{j^*} t_{j,n}, y^*_{j_\downarrow,n};\ (t_{j_\downarrow + 1,n},t_{j_\downarrow + 2,n},\cdots,t_{j^*,n}), (w_{j_\downarrow + 1,n},w_{j_\downarrow + 2,n},\cdots,w_{j^*,n})  \big), \mathcal{G}^c \Big) \leq \frac{5}{8}\epsilon_1.
\end{align*}
Meanwhile, using Gronwall's inequality repeatedly and $\sup_{j = j_\downarrow + 1,\cdots,j^*, n \geq 1}t_{j,n} < \infty$, by substituting the initial condition $y^*_{j_\downarrow,n}$ with $\bm{0}$, we have
\begin{align*}
    \lim_n \Big|\Big| &\widetilde{\bm{x}}\big( \sum_{j = j_\downarrow + 1}^{j^*} t_{j,n}, y^*_{j_\downarrow,n};\ (t_{j_\downarrow + 1,n},t_{j_\downarrow + 2,n},\cdots,t_{j^*,n}), (w_{j_\downarrow + 1,n},w_{j_\downarrow + 2,n},\cdots,w_{j^*,n})  \big) \\
    -&\widetilde{\bm{x}}\big( \sum_{j = j_\downarrow + 1}^{j^*} t_{j,n}, \bm{0};\ (t_{j_\downarrow + 1,n},t_{j_\downarrow + 2,n},\cdots,t_{j^*,n}), (w_{j_\downarrow + 1,n},w_{j_\downarrow + 2,n},\cdots,w_{j^*,n})  \big)\Big|\Big| = 0.
\end{align*}
This implies that, for all $n$ sufficiently large,
\begin{align*}
    \bm{d}\Big(\widetilde{\bm{x}}\big( \sum_{j = j_\downarrow + 1}^{j^*} t_{j,n}, \bm{0};\ (t_{j_\downarrow + 1,n},t_{j_\downarrow + 2,n},\cdots,t_{j^*,n}), (w_{j_\downarrow + 1,n},w_{j_\downarrow + 2,n},\cdots,w_{j^*,n})  \big), \mathcal{G}^c \Big) \leq \frac{3}{4}\epsilon_1.
\end{align*}
However, this contradicts \cref{proof bound on small perturbation at origin contradiction}. This implies the existence of the required $\delta_0$ and concludes the proof.
\end{proof}

\subsection{Sufficient conditions for Assumption \ref{assumption boundary set with zero mass}} \label{subsec: proof of lemma zero mass on boundary set}
In this section, we show that under a proper set of regularity conditions on the boundary set $\partial \mathcal{G}$ and the distribution of noises $Z_n$,
Assumption \ref{assumption boundary set with zero mass} will hold for (Lebesgue) almost every $b > 0$.
In particular, we stress that the $C^2$ condition about manifold $\partial \mathcal{G}$ in Assumption \ref{assumption 3 boundary of domain G}, as well as the condition that measures $S_j$ are absolutely continuous w.r.t. the spherical measure $\bm{\sigma}$ on $\mathbb{S}^{d-1}$, are only used to prove that $\mu\big(h^{-1}(\partial \mathcal{G})\big) = 0$ and will only be applied in this section.

The key of our argument is the following geometric observation regarding the intersection of $C^2-$manifold $\partial \mathcal{G}$ and $\partial B$ of some ball $B$.
We stress that, as made evident by the proof, this lemma is essentially based on two assumptions:
(I) As the boundary set of the connected bounded region $\mathcal{G}$, $\partial \mathcal{G}$ is a closed set in $\mathbb{R}^d$;
(II) As a $(d-1)$-dimensional manifold, $\partial \mathcal{G}$ is of class $C^2$.
\begin{lemma} \label{lemma zero mass intersection between boundary set and sphere}
Let $\bm{\sigma}_{x,b}$ be the spherical measure on the sphere of the open ball $B(x,b)$ for any $x \in \mathbb{R}^d, b > 0$.
Under Assumptions \ref{assumption 1 domain G} and \ref{assumption 3 boundary of domain G},
it holds for (Lebesgue) almost every $b>0$ that
\begin{align}
    \bm{\sigma}_{x,b}\big( \partial \mathcal{G} \cap \partial B(x,b) \big) = 0\ \ \ \forall x \in \mathbb{R}^d. \label{claim lemma zero mass intersection between boundary set and sphere}
\end{align}
\end{lemma}
\begin{proof}
The fact that $\partial \mathcal{G}$ is a subset of a separable metric space implies the existence of a countable atlas for this manifold.
Therefore, we can find a sequence $(U_i)_{i \geq 1}$ that are bounded open subsets of $\mathbb{R}^{d-1}$ containing $\bm{0}$,
a sequence $(V_i)_{i \geq 1}$ that are open sets in $\partial \mathcal{G}$ (in the metric space induced by Euclidean distance),
and a sequence of injective $C^2$ mapping $f_i$ with $f_i: \mathbb{R}^{d-1} \mapsto \mathbb{R}^{d}$ and $f_i(U_i) = V_i$ such that $\partial \mathcal{G} = \cup_{i}V_i$.

Next, we zoom in on a specific chart $(U_i,V_i,f_i)$ and observe the following facts.
For any $x \in U_i, y \in V_i$ with $f_i(x) = y$, there exist some orthogonal matrix $Q_y \in \mathbb{R}^{d \times d}$
such that $Q_y n(y) = (0,0,\cdots,0,1)^T$
where the vector field $n(\cdot)$ is the outer normal on $\partial\mathcal{G}$.
Besides, there is some vector $a_y \in \mathbb{R}^d$ such that $Q_y y + a_y = \bm{0}$.
Moreover, there exist an open set on $\widetilde{U}_x$ in $\mathbb{R}^{d-1}$ containing $\bm{0}$,
an open set $y \in V_y \subset V_i$,
an open set $\widetilde{V}_y \delequal \{ Q_y w + a_y:\ w \in V_y \}$,
a $C^2$ function $\widetilde{g}_y:\widetilde{U}_x \mapsto \mathbb{R}$ satisfying $\widetilde{g}_y(\bm{0}) = 0$ and
\begin{align*}
    \widetilde{g}_y(w_1,\cdots,w_{d-1}) = w_d\ \ \forall \bm{w} = (w_1,\cdots,w_d) \in \widetilde{V}_y.
\end{align*}
In other words, for any given $y$ on this chart we simply rotate the chart to ensure that the tangent space at $y$ after rotation is
$\{(x_1,\cdots,x_{d-1},0):\ x_i \in \mathbb{R} \ \forall i \in [d-1]\}$,
and reparametrize the $C^2$ diffeomorphism associated to this chart around $y$ so that the coordinates are simply the projection onto the said tangent space.
In this sense, the (rotated) manifold is also the graph of the $C^2$ mapping $g_y$.
This allows us to define (for any $y \in \partial \mathcal{G}$)
\begin{align*}
    A(y) \delequal \Big(\lambda_1\big( \nabla^2 g_y(\bm{0}) \big),\lambda_2\big( \nabla^2 g_y(\bm{0}) \big), \cdots, \lambda_{d-1}\big( \nabla^2 g_y(\bm{0}) \big)\Big)
\end{align*}
where, for any real symmetric $(d-1)\times(d-1)$ matrix $A$, $\lambda_1(A) \geq \lambda_2(A) \geq \cdots \geq \lambda_{d-1}(A)$ are the ordered eigenvalues of the matrix,
and $\nabla^2 g_y(\cdot)$ is the Hessian of $g_y$.
It is worth noticing that $A(\cdot)$ is a continuous function (on $\partial \mathcal{G}$) due to the manifold being of class $C^2$.

Restricting our discussion on some fixed chart $(U_i,V_i,f_i)$ for now,
for any $b > 0$, let
\begin{align*}
    \mathcal{A}_i(b) \delequal \{ x \in U_i:\ \text{for }y = f_i(x), \exists j \in [d-1] \text{ such that } |\lambda_j\big( \nabla^2 g_y(\bm{0}) \big)| = 1/b\}.
\end{align*}
Note that for any $b>0$, the set $\mathcal{A}_i(b)$ is a closed set (hence Borel measurable) since the continuity of $A(\cdot)$ implies that $(\mathcal{A}_i(b))^c$ is an open set on $U_i$.
Furthermore, since $\bm{m}_\text{Leb}( U_i ) < \infty$, there are at most countably many $b > 0$ such that
$\bm{m}_\text{Leb}( \mathcal{A}_i(b) ) > 0.$
Given the countability of the atlas, we know that
\begin{align*}
    \mathcal{B}^* \delequal \{ b > 0:\ \exists i \in \mathbb{N}\text{ s.t. } \bm{m}_\text{Leb}( \mathcal{A}_i(b) ) > 0\}
\end{align*}
contains at most countably many elements.
In the rest of this proof, we show that \cref{claim lemma zero mass intersection between boundary set and sphere} holds for any $b > 0$ such that $b \notin \mathcal{B}^*$.

Henceforth, we arbitrarily choose some $b > 0$ such that $b \notin \mathcal{B}^*$.
We also arbitrarily choose some $x \in \mathbb{R}^d$ and let
$\textbf{B} = B(x,b)$.
To facilitate the discussion, we introduce a concept that is closely related to the set $\mathcal{A}_i(b)$: let
\begin{align*}
    \mathcal{A}^\mathcal{G}_i(b) \delequal \{ y \in V_i:\ f^{-1}_i(y) \in \mathcal{A}_i(b) \} = \{y \in V_i:\ A(y) = (b,b,\cdots,b) \text{ or } (-b,-b,\cdots,-b) \}
\end{align*}
and let $\mathcal{A}^\mathcal{G}(b) \delequal \cup_{i \geq 1}\mathcal{A}^\mathcal{G}_i(b)$.
Now consider the following decomposition of the sphere $\partial \textbf{B}$:
\begin{align*}
    \textbf{B}_2 & \delequal \partial \textbf{B} \cap \mathcal{A}^\mathcal{G}(b), \\
    \textbf{B}_1 & \delequal \partial \textbf{B} \symbol{92}\textbf{B}_2.
\end{align*}
First, note that for any $y \in \textbf{B}_1$, one of the following four cases has to occur:
\begin{itemize}
    \item $y \notin \partial \mathcal{G}$;
    \item $y \in \partial \mathcal{G}$ and the vector $y - x$ lies in $T_y \partial \mathcal{G}$, the tangent space of the manifold $\partial \mathcal{G}$ at $y$;
    \item $y \in \partial \mathcal{G}$; the vector $y - x$ is not in $T_y \partial \mathcal{G}$ yet it is not orthogonal to $T_y \partial \mathcal{G}$ either, i.e. $y - x$ is not equal to $c\cdot n(y)$ for any $c \in \mathbb{R}$ where $n(y)$ is the outer normal at $y$;
    \item $y \in \partial \mathcal{G}$; the vectors $y - x$ and $n(y)$ are linearly dependent.
\end{itemize}
Our next goal is to show that, in any of these four cases, we can always find a set $y \in O_y$ that is open on the sphere $\partial \textbf{B}$ such that
$\bm{\sigma}_{x,b}( O_y \cap \partial \mathcal{G} ) = 0$.
Note that this is obviously true when $y \notin \partial \mathcal{G}$, since both $\partial \mathcal{G}$ and $\partial \textbf{B}$ are closed sets in $\mathbb{R}^d$.
Now we consider the second case.
If the vector $y - x$ lies in $T_y \partial \mathcal{G}$,
then after applying the affine transformation with orthogonal matrix $Q_y$ (recall that $Q_y n(y) = (0,\cdots,0,1)$),
we have
\begin{align*}
    Q_y(y - x) \in Q_y T_y \partial \mathcal{G} = \{(w_1,\cdots,w_{d-1},0):\ w_i \in \mathbb{R}\ \forall i \in [d-1]\}.
\end{align*}
Since $Q_y y + a_y = \bm{0}$, we now know that after the affine transformation, the center of the ball $\textbf{B}$ moves to
\begin{align*}
    Q_y x + a_y = Q_y y + a_y + Q_y(x-y) \in \{(w_1,\cdots,w_{d-1},0):\ w_i \in \mathbb{R}\ \forall i \in [d-1]\}.
\end{align*}
Without loss of generality, we can assume that $Q_y x + a_y = (b,0,\cdots,0,0)$.
In other words, after the affine transformation, the ball becomes 
$\widetilde{\textbf{B}}_y \delequal Q_y \textbf{B} + a_y = B\big( (b,0,0,\cdots,0),b \big).$
Moreover, for any $\bm{w} = (w_1,\cdots,w_d) \in \partial \widetilde{\textbf{B}}_y$,
we must have $w_1 \geq 0$ and
\begin{align*}
    (w_1 - b)^2 + (w_2)^2 + \cdots + (w_d)^2 = b^2.
\end{align*}
Meanwhile, from the definition of the mapping $\widetilde{g}_y$,
one can see that there is an open set $U_y$ around $y$ such that
for any $w \in U_y$ with $w \in \partial \mathcal{G} \cap \partial \textbf{B}$
(hence $\widetilde{w} = Q_y w + a_y \in \widetilde{V}_y \in \partial \widetilde{B}_y$)
such that
\begin{align*}
    \widetilde{g}_y(\widetilde{w})^2 = b^2 - (\widetilde{w}_1 - b)^2 - \widetilde{w}_2^2 - \cdots - \widetilde{w}_{d-1}^2.
\end{align*}
Let $\widetilde{f}(w_1,\cdots,w_{d-1}) \delequal \widetilde{g}^2_y(w_1,\cdots,w_{d-1}) + w^2_1 - 2bw_1 + w^2_2 + \cdots + w^2_{d-1}$.
Now we now that $\widetilde{f}(\bm{0}) = 0$ and $\widetilde{f}(\widetilde{w}) = 0$ for any $\widetilde{w}\in \widetilde{V}_y \in \partial \widetilde{B}_y$.
Moreover, by definition we have $\nabla \widetilde{g}_y(\bm{0}) = \bm{0}$ so
$ \frac{\partial}{\partial w_1}\widetilde{f}(\bm{0}) = -2b \neq 0$.
Due to implicit function theorem,
we now know the existence of some open set $U^*$ in $\mathbb{R}^{d}$ containing $\bm{0}$,
some $C^2$ function $g^*:\mathbb{R}^{d-2}\mapsto \mathbb{R}$ such that
for any $w \in \widetilde{V}_y \cap \partial \widetilde{\textbf{B}}_y \cap U^*$,
\begin{align*}
    w_1 &  = g^*(w_2,\cdots,w_{d-1}),\ \ \ w_d = \widetilde{g}_y(w_1,\cdots,w_{d-1}).
\end{align*}
Therefore, within some open neighborhood $V^*_y$ of such $y$, the set $V^*_y \cap \partial\textbf{B} \cap \partial \mathcal{G}$ is a submanifold with dimension $d-2$,
so we must have
\begin{align*}
    \bm{\sigma}_{x,b}( V^*_y \cap \partial\textbf{B} \cap \partial \mathcal{G} ) = 0.
\end{align*}

Next, we consider the case where $y \in \partial \mathcal{G} \cap \textbf{B}_1$ but $y - x$ is neither in $T_y \partial \mathcal{G}$ nor orthogonal to $T_y \partial \mathcal{G}$.
Similar to the construction of the affine transformation with $Q_y, a_y$ above,
one can find
an orthogonal matrix $\widetilde{Q}_y$ and a vector $\widetilde{a}_y$ such that
$\widetilde{Q}_y y + a_y = \bm{0}$
and $\widetilde{Q}_y(y - x) = (0,0,\cdots,0,b)$.
Let $\widetilde{\textbf{B}} = B( (0,0,\cdots,-b), b)$.
In other words, this time the rotation we constructed ensures that, after rotation, the vector between $\widetilde{Q}_y y + a_y = \bm{0}$ (on the sphere $\partial \widetilde{\textbf{B}}$) and the center $\widetilde{Q}_y x + \widetilde{a}_y$ of the ball $\widetilde{\textbf{B}}$ is aligned with the $d-$th axis.
Moreover, there is an open set $y \in V^\text{alt}_y \subset \mathbb{R}^d$
and a $C^2$ function $\widetilde{g}^\text{alt}_y: \mathbb{R}^{d-1} \mapsto \mathbb{R}$
with $\nabla \widetilde{g}^\text{alt}_y(\bm{0}) \neq \bm{0}$
such that for any $\widetilde{y} \in \widetilde{V}^\text{alt}_y \delequal \widetilde{Q}_yV^\text{alt}_y + \widetilde{a}_y$,
we have
\begin{align*}
    \widetilde{y}_{d} = \widetilde{g}^\text{alt}_y(\widetilde{y}_1,\cdots,\widetilde{y}_{d-1}).
\end{align*}
By applying a change of coordinates if necessary (which can be achieved by multiplying another orthogonal matrix),
we can assume without loss of generality that
$\nabla \widetilde{g}^\text{alt}_y(\bm{0}) = (0,0,\cdots,c)$
for some $c \neq 0$.
Then for any $\widetilde{w} \in V^\text{alt}_y$ such that $w \in \partial \textbf{B} \cap \partial \mathcal{G}$,
let $w = \widetilde{Q}_y\widetilde{w} + \widetilde{a}_y$ and note that we must have
\begin{align*}
    w^2_1 + w^2_2 + \cdots + w^2_{d-1} +  \big( b + \widetilde{g}^\text{alt}_y(w_1,\cdots,w_{d-1})\big)^2 = b^2.
\end{align*}
In particular, $\widetilde{g}^\text{alt}_y(w_1,\cdots,w_{d-1}) = cw_{d-1} + r(w)$
for some $C^2$ function $r$ with $r(\bm{0}) = 0, \nabla r(\bm{0}) = \bm{0}$.
As a result, for function
$\widetilde{f}(w) \delequal w^2_1 + w^2_2 + \cdots + w^2_{d-1} +  \big( b + \widetilde{g}^\text{alt}_y(w_1,\cdots,w_{d-1})\big)^2.$
we have $\frac{\partial}{\partial w_{d-1}}\widetilde{f}(\bm{0}) = 2bc \neq 0$.
Using implicit function theorem again,
one can see the existence of some open set $\bm{0} \in V^*_y \in \mathbb{R}^d$ and some $C^2$ function $g^*:\mathbb{R}^{d-2} \mapsto \mathbb{R}$
such that for any $\widetilde{w} \in V^*_y \cap \partial \mathcal{G} \cap \partial \textbf{B}$, we have that (for $w \delequal \widetilde{Q}_y\widetilde{w} + \widetilde{a}_y$)
\begin{align*}
    w_{d-1} = g^*(w_1,\cdots,w_{d-2}), \ \ \ w_d = \widetilde{g}^\text{alt}_y(w_1,\cdots,w_{d-1}).
\end{align*}
Again, we have established that within some open neighborhood $V^*_y$ of such $y$, the set $V^*_y \cap \partial\textbf{B} \cap \partial \mathcal{G}$ is a submanifold with dimension $d-2$,
so we must have
\begin{align*}
    \bm{\sigma}_{x,b}( V^*_y \cap \partial\textbf{B} \cap \partial \mathcal{G} ) = 0.
\end{align*}

Lastly, consider the case where $y \in \partial \mathcal{G} \cap \textbf{B}_2$ and the vectors $y - x$ and $n(y)$ are linearly dependent.
In other words, the tangent space $T_y \partial \mathcal{G}$ is also the tangent space $T_y \partial \textbf{B}$.
Since $y \notin \textbf{B}_1$, we know that
$ \widetilde{g}_y(w) = \frac{1}{2}w^T A_y w + r_1(w)$
where $A_y$ is a real symmetric matrix with no eigenvalue equal to $\pm 1/b$
and $r_1$ is a $C^2$ function with $|r_1(w)| = o(\norm{w}^2)$.
On the other hand, for any $\widetilde{w}$ in the open set $V_y \subset \partial \mathcal{G}$,
if we also have $\widetilde{w} \in \partial \textbf{B} \cap \partial \mathcal{G}$, then for $w \delequal Q_y\widetilde{w} + a_y$ we have
\begin{align}
    w_d = \widetilde{g}_y(w_1,\cdots,w_{d-1}) = \sqrt{b^2 - w^2_1 - w^2_2 - \cdots - w^2_{d-1}  } - b. \label{equation lemma zero mass intersection between boundary set and sphere}
\end{align}
Also, note that
$\sqrt{b^2 - w^2_1 - w^2_2 - \cdots - w^2_{d-1}  } - b = -\frac{1}{2}(w_1,\cdots,w_{d-1})\frac{\textbf{I}_{d-1}}{b}(w_1,\cdots,w_{d-1})^T + r_2(w_1,\cdots,w_{d-1})$
where $r_2$ is also a $C^2$ function with $|r_2(w)| = o(\norm{w}^2)$.
Therefore, for any $(w_1,\cdots,w_{d-1})$ satisfying the equation \cref{equation lemma zero mass intersection between boundary set and sphere}, we have
\begin{align*}
    \frac{1}{2}(w_1,\cdots,w_{d-1})\big(A_y -  \frac{\textbf{I}_{d-1}}{b}\big)(w_1,\cdots,w_{d-1})^T = - r_1(w_1,\cdots,w_{d-1}) + r_2(w_1,\cdots,w_{d-1}).
\end{align*}
However, for the real symmetric matrix $A_y -  \frac{\textbf{I}_{d-1}}{b}$, note that none of its eigenvalue is equal to 0, implying the existence of some $\epsilon > 0$ such that
$$\Big| \frac{1}{2}(w_1,\cdots,w_{d-1})\big(A_y -  \frac{\textbf{I}_{d-1}}{b}\big)(w_1,\cdots,w_{d-1})^T\Big| \geq \epsilon(w^2_1 + \cdots + w^2_{d-1}).$$
For this fixed $\epsilon > 0$, we can also find $\delta > 0$ such that
$$|- r_1(w_1,\cdots,w_{d-1}) + r_2(w_1,\cdots,w_{d-1})| \leq  \frac{\epsilon}{2}(w^2_1 + \cdots + w^2_{d-1})$$
for any $w^2_1 + \cdots + w^2_{d-1} < \delta.$
As a result, the only solution to \cref{equation lemma zero mass intersection between boundary set and sphere} with $w^2_1 + \cdots + w^2_{d-1} < \delta$ is $w_1 = w_2 = \cdots, w_{d-1} = 0$.
In summary, we have shown that there exists some set $y \in V^*_y$ that is open in $\mathbb{R}^d$ such that
$ V^*_y \cap \partial\textbf{B} \cap \partial \mathcal{G}= \{y\}.$

Collecting the results we have established so far, we now know that for any $y \in \textbf{B}_1$ (recall that $\textbf{B}_1$ is an open set on $\partial \textbf{B}$),
there is an open set $V^*_y$ containing $y$ and satisfying $\bm{\sigma}_{x,b}( V^*_y \cap \partial\textbf{B} \cap \partial \mathcal{G} ) = 0$.
In particular, given the open cover $\cup_{y \in \textbf{B}_2}V^*_y = \textbf{B}_1$, Lindelöf property then allows us to extract a countable open $\cup_{i \geq 1}V^*_{y_i} = \textbf{B}_1$ cover and conclude that
\begin{align*}
    \bm{\sigma}_{x,b}(\textbf{B}_1 \cap \partial \mathcal{G} ) \leq \sum_{i \geq 1}\bm{\sigma}_{x,b}( V^*_{y_i} \cap \partial\textbf{B} \cap \partial \mathcal{G} ) = 0.
\end{align*}

Moving on, we evaluate $\bm{\sigma}_{x,b}(\textbf{B}_2 \cap \partial \mathcal{G} )$.
For any $y \in \partial \textbf{B}$, one of the four cases has to occur:
\begin{itemize}
    \item $y \notin \mathcal{A}^\mathcal{G}(b)$;
    \item $y \in \mathcal{A}^\mathcal{G}(b)$ and the vector $y - x$ lies in $T_y \partial \mathcal{G}$, the tangent space of the manifold $\partial \mathcal{G}$ at $y$;
    \item $y \in \mathcal{A}^\mathcal{G}(b)$; the vector $y - x$ is not in $T_y \partial \mathcal{G}$ yet it is not orthogonal to $T_y \partial \mathcal{G}$ either, i.e. $y - x$ is not equal to $c\cdot n(y)$ for any $c \in \mathbb{R}$ where $n(y)$ is the outer normal at $y$;
    \item $y \in \mathcal{A}^\mathcal{G}(b)$; the vectors $y - x$ and $n(y)$ are linearly dependent.
\end{itemize}
Again, we show that in any of these four cases, there is some set $y \in V^*_y$ open in $\partial \textbf{B}$ such that
$\bm{\sigma}_{x,b}(\mathcal{A}^\mathcal{G}(b) \cap V^*_y ) = 0.$
In the first case, the fact that $\mathcal{A}^\mathcal{G}(b)$ is closed on $\partial \mathcal{G}$ immediately implies the existence of some $V^*_y$ such that $\mathcal{A}^\mathcal{G}(b) \cap V^*_y = \emptyset$.
For the second and third case, this can be shown using exactly the same implicit function argument above.
For the last case where $y \in \mathcal{A}^\mathcal{G}(b) \cap \partial \textbf{B}$ and the vectors $y - x$ and $n(y)$ are linearly dependent,
from $y \in \mathcal{A}^\mathcal{G}(b) \cap \partial \textbf{B}$
we know that $y \in V_i$ where $(U_i,V_i,f_i)$ is a chart of $\partial\mathcal{G}$
and $V_i$ is open on $\partial\mathcal{G}$.
Moreover, recall the construction of open set $y \in V_y \subset V_i$ at the beginning of the proof.
It is worth noticing that $V_i \cap \mathcal{A}^\mathcal{G}(b) = V_i \cap \mathcal{A}_i^\mathcal{G}(b)$.
Besides, due to the fact that the vectors $y - x$ and $n(y)$ are linearly dependent,
we know that the tangent space $T_y \partial \mathcal{G}$ is also the tangent space $T_y \partial \textbf{B}$.
Therefore, by definition of $\widetilde{g}_y$ and $Q_y,a_y$, we know that
$Q_y y + a_y = \bm{0}$,
and under the affine transformation, the ball becomes $B( (0,0,\cdots,\pm b),b)$.
Without loss of generality, we assume it is $\widetilde{\textbf{B}}\delequal B( (0,0,\cdots,b),b)$.
Moreover, since $b \notin \mathcal{B}^*$, we have that
$\bm{m}_\text{Leb}(\{w \in \widetilde{U}_y:\ x\in \mathcal{A}_i(b) \} ) = 0$
where, as defined at the beginning of the proof, $\widetilde{U}_y$ is the domain of the $C^2$ mapping $\widetilde{g}_y$, $V_y$ is the image of the mapping,
and  $\widetilde{V}_y \delequal \{ Q_y w + a_y:\ w \in V_y \}$ is the image of $V_y$ under the affine transformation.
Therefore, for any $\widetilde{w} \in \mathcal{A}^\mathcal{G}_i(b) \cap \partial \textbf{B}$, let $w = Q_y \widetilde{w} + a_y$ and we must have
\begin{align*}
    w_d = \widetilde{g}_y(w_1,\cdots,w_{d-1}) = -b + \sqrt{b^2 - w^2_1 - w^2_2 - \cdots - w^2_{d-1} },
    \\
    (w_1,\cdots,w_{d-1}) \in \{w \in \widetilde{U}_y:\ x\in \mathcal{A}_i(b) \}.
\end{align*}
Now let $V^*_y \delequal{} \Big\{ Q_y^T(v - a_y) :\ v = \big(w_1,\cdots,w_{d-1},-b + \sqrt{b^2 - w^2_1 - w^2_2 - \cdots - w^2_{d-1} }\big)\text{ for some }w \in \widetilde{U}_x \Big\}$
and note that $y \in V^*_y$ is an open set on $\partial \textbf{B}$.
(Specifically, note that we simply identify an open set on the transformed sphere $\partial \widetilde{B}_y$, and then perform the inverse transformation to move the set back to the original sphere $\partial \textbf{B}$.)
Then it follows immediately from
$\bm{m}_\text{Leb}(\{w \in \widetilde{U}_y:\ x\in \mathcal{A}_i(b) \} ) = 0$
that $\bm{\sigma}_{x,b}\big(\mathcal{A}^\mathcal{G}(b) \cap V^*_y \big) = 0.$

In summary, for any $y \in \partial \textbf{B}$, we can find a set $y\in V^*_y$ open on $\partial \textbf{B}$ such that
$\bm{\sigma}_{x,b}\big(\mathcal{A}^\mathcal{G}(b) \cap V^*_y \big) = 0.$
Lastly, by applying Lindelöf property again, we extract a countable open cover $\cup_{i \geq 1}V^*_{y_i} = \textbf{B}_2$ cover and conclude that
\begin{align*}
    \bm{\sigma}_{x,b}(\partial \textbf{B} \cap \mathcal{A}^\mathcal{G}(b)) \leq \sum_{i \geq 1}\bm{\sigma}_{x,b}\big(\mathcal{A}^\mathcal{G}(b) \cap V^*_{y_i} \big) = 0
\end{align*}
and this concludes the proof.
\end{proof}

As a result of Lemma \ref{lemma zero mass intersection between boundary set and sphere},
the following Lemma is essentially built upon three assumptions/facts:
(I) As a boundary set, $\partial \mathcal{G}$ is a closed set in $\mathbb{R}^d$;
(II) As a $(d-1)$-dimensional manifold, $\partial \mathcal{G}$ is of class $C^2$;
(III) The measures $S_j$ in Assumption \ref{assumption 6 regular variation of the Rd noise} are absolutely continuous w.r.t. the spherical measure $\bm{\sigma}$.

\begin{lemma} \label{lemma boundary set zero mass}
Under Assumptions \ref{assumption 1 domain G}, \ref{assumption 3 boundary of domain G} and \ref{assumption 6 regular variation of the Rd noise},
it holds for (Lebesgue) almost every $b>0$ that
$$\mu\Big(h^{-1}\big( \partial \mathcal{G}\big)\Big) = 0.$$
\end{lemma}

\begin{proof}

Fix some $b > 0$ satisfying the conditions in Lemma \ref{lemma zero mass intersection between boundary set and sphere}.
Recall that $\mu = \sum_{\bm{j} \in \bm{j}(\bm{i}^*)}\mu_{\bm{j}}$ (see \cref{def measure mu j Rd}).
It suffices to show that $\mu_{\bm{j}}\Big( h^{-1}\big( \partial \mathcal{G}\big) \Big) = 0$ for some fixed $\bm{j} \in \bm{j}(\bm{i}^*)$.
In particular, observe that
\begin{align}
    & \mu_{\bm{j}}\Big( h^{-1}\big( \partial \mathcal{G}\big) \Big)
    \nonumber
    \\
    = & \int_{ t_{i+1} > 0,\ \theta_i \in \mathbb{S}^{d-1},\ r_i > 0\ \forall i \in [k^*-1] }
    \nonumber 
    \\ 
    & \ \ \cdot\Big(\int_{ \theta_{k^*} \in \mathbb{S}^{d-1}, r_{k^*} > 0 } \mathbbm{1}\big\{ h^*( r_1,\cdots,r_{k^*-1},\theta_1,\cdots,\theta_{k^*-1},t_2,\cdots,t_{k^*}) + \varphi_b( r_{k^*}\theta_{k^*} ) \in \partial\mathcal{G}  \big\}
    \nonumber
    \\
    & \cdot S_{ \bm{j}_{k^*} }(d\theta_{k^*})\nu_{ \alpha_{\bm{j}_{k^*}} }(dr_{k^*}) \Big) \cdot \prod_{i = 1}^{ k^* - 1 }\nu_{\alpha_{\bm{j}_i}}(d r_i)\times S_{\bm{j}_i}( d \theta_i )\times \bm{m}_\text{Leb}(d t_{i+1})
    \label{expression for mu lemma boundary set zero mass}
\end{align}
where the function $h^*$ is defined as
\begin{align*}
   &  h^*( r_1,\cdots,r_{k^*-1},\theta_1,\cdots,\theta_{k^*-1},t_2,\cdots,t_{k^*})
   \\
   & \delequal \widetilde{\bm{x}}\big(\sum_{i = 2}^{k^*-1}t_i ,\bm{0}; (0,t_2,\cdots,t_{k^*-1}), (r_1\theta_1, r_2\theta_2,\cdots, r_{k^*-1}\theta_{k^*-1}) \big)
\end{align*}

Let $z \delequal h^*( r_1,\cdots,r_{k^*-1},\theta_1,\cdots,\theta_{k^*-1},t_2,\cdots,t_{k^*})$.
Now by separating the two cases based on whether the truncation operator takes effect or not, we have
\begin{align}
    & \int_{ \theta_{k^*} \in \mathbb{S}^{d-1}, r_{k^*} > 0 } \mathbbm{1}\big\{ h^*( r_1,\cdots,r_{k^*-1},\theta_1,\cdots,\theta_{k^*-1},t_2,\cdots,t_{k^*}) + \varphi_b( r_{k^*}\theta_{k^*} ) \in \partial\mathcal{G}  \big\}\nonumber 
    \\
    & \ \ \ \ \ \ \ \ \ \ \cdot S_{ \bm{j}_{k^*} }(d\theta_{k^*})\nu_{ \alpha_{\bm{j}_{k^*}} }(dr_{k^*}) \nonumber
    \\
    = & \int_{ \theta_{k^*} \in \mathbb{S}^{d-1}, r_{k^*} \in (0,b) } \mathbbm{1}\big\{ z + r_{k^*}\theta_{k^*} \in \partial\mathcal{G}  \big\}S_{ \bm{j}_{k^*} }(d\theta_{k^*})\nu_{ \alpha_{\bm{j}_{k^*}} }(dr_{k^*}) \label{term 1 lemma boundary set zero mass}
    \\
    + & \int_{ \theta_{k^*} \in \mathbb{S}^{d-1}} \mathbbm{1}\big\{ z + b\theta_{k^*} \in \partial\mathcal{G}  \big\}S_{ \bm{j}_{k^*} }(d\theta_{k^*})\cdot \int_{r_{k^*} > b}\nu_{ \alpha_{\bm{j}_{k^*}} }(dr_{k^*}).\label{term 2 lemma boundary set zero mass}
\end{align}
For term \cref{term 1 lemma boundary set zero mass},
note that it is equal to
$\int\mathbbm{1}\big\{ \textbf{T}^{-1}(r_{k^*},\theta_{k^*}) \in (-z + \partial\mathcal{G}) \cap \partial B(0,b)  \big\}S_{ \bm{j}_{k^*} }(d\theta_{k^*})\times\nu_{ \alpha_{\bm{j}_{k^*}} }(dr_{k^*})$
where $\textbf{T}^{-1}(r,\theta) = r\theta$ is the inverse of the polar coordinate transform.
Furthermore, since the set $(-z + \partial\mathcal{G}) \cap \partial B(0,b)$ is either empty or is a $(d-1)$-dimensional $C^2$ submanifold (w.r.t. $B(0,b)$ when viewed as a $d-$dimensional manifold).
In other words, it has zero mass under $\bm{m}^{d}_\text{Leb}$.
For the measure $\nu^* \delequal \textbf{T}^{-1} \circ ( S_{\bm{j}_{k^*}}\times \nu_{\alpha_{\bm{j}_{k^*}}} )$,
due to $S_j$ being absolutely continuous w.r.t. $\bm{\sigma}$,
it is easy to see that $\nu^*$ is absolutely continuous w.r.t. $\bm{m}^{d}_\text{Leb}$.
Therefore, we must have 
$$\nu^*( (-z + \partial\mathcal{G}) \cap \partial B(0,b) ) = 0,$$
implying that the integral in $\cref{term 1 lemma boundary set zero mass} = 0$.
On the other hand, for term \cref{term 2 lemma boundary set zero mass},
we know that 
$ \int_{r_{k^*} > b}\nu_{ \alpha_{\bm{j}_{k^*}} }(dr_{k^*}) = 1/b^{1 + \alpha_{\bm{j}_{k^*}}} < \infty$.
Besides,
\begin{align*}
    & \int_{ \theta_{k^*} \in \mathbb{S}^{d-1}} \mathbbm{1}\big\{ z + b\theta_{k^*} \in \partial\mathcal{G}  \big\}S_{ \bm{j}_{k^*} }(d\theta_{k^*})
    = \int_{\mathbb{S}^{d-1}}\mathbbm{1}\{ z + b\theta \in \partial \mathcal{G} \cap \partial B(z,b) \}S_{ \bm{j}_{k^*} }(d\theta).
\end{align*}
Then it follows immediately from Lemma \ref{lemma zero mass intersection between boundary set and sphere} and $S_j$ being absolutely continuous w.r.t. the spherical measure $\bm{\sigma}$ that the integral in term \cref{term 2 lemma boundary set zero mass} is equal to $0$.
In summary, we have shown that
\begin{align*}
    & \int_{ \theta_{k^*} \in \mathbb{S}^{d-1}, r_{k^*} > 0 } \mathbbm{1}\big\{ h^*( r_1,\cdots,r_{k^*-1},\theta_1,\cdots,\theta_{k^*-1},t_2,\cdots,t_{k^*}) + \varphi_b( r_{k^*}\theta_{k^*} ) \in \partial\mathcal{G}  \big\}
    \\
    &\ \ \ \ \ \ \ \ \ \ \ \ \ \cdot S_{ \bm{j}_{k^*} }(d\theta_{k^*})\nu_{ \alpha_{\bm{j}_{k^*}} }(dr_{k^*}) = 0
\end{align*}
for any $( r_1,\cdots,r_{k^*-1},\theta_1,\cdots,\theta_{k^*-1},t_2,\cdots,t_{k^*}) $.
Plug this result back into \cref{expression for mu lemma boundary set zero mass} and we conclude the proof.
\end{proof}

\section{Notations}
Table~\ref{tableRefNotations} lists the notations used in Section~\ref{section proof thm 1}.

\renewcommand{\arraystretch}{1.3}
\newpage
\begin{longtable}{l  p{12cm}}
\caption{Summary of notations frequently used in Section \ref{section proof thm 1}}\label{tableRefNotations}\\

    $[k]$ & $\{1,2,\ldots,k\}$ \\
    $\eta$ & Learning rate (gradient descent step size)\\
    $b$ & Truncation threshold of stochastic gradient\\
    $\epsilon$ & An accuracy parameter; typically used to denote an $\epsilon-$neighborhood of $s_i, m_i$\\
    $\delta$ & A threshold parameter used to define \textit{large} noises\\
    $\bar{\epsilon}$ & A constant defined for \cref{assumption multiple jump epsilon 0 constant 1}-\cref{assumption multiple jump epsilon 0 constant 2}. Since $\bar{\epsilon}<\epsilon_0$, in \cref{assumption detailed function f at critical point} the claim holds for $|x-y|<\bar{\epsilon}$. Note that the value of the constant $\bar{\epsilon}$ does not vary with our choice of $\eta,\epsilon,\delta$. \\
    
    $M$ & Upper bound of $|f'|$ and $|f''|$ \hfill \cref{assumption detailed function f at critical point 3}\\
    $L$ & Radius of training domain \hfill \cref{assumption detailed function f at critical point 3}\\
    $\Omega$ & The open interval $(s_-,s_+)$; a simplified notation for $\Omega_i$\\
    
    $\varphi$, $\varphi_c$ & $\varphi_c(w) \triangleq \varphi(w,c) \delequal{} (w\wedge c)\vee(-c)$
    \hfill truncation operator at level $c > 0$
    \\
    $Z^{\leq \delta,\eta}_n$ & $Z_n\mathbbm{1}\{\eta|Z_n|\leq \delta\}$ \hfill``small'' noise \cref{defSmallJump_GradientClipping}
    \\
    $Z^{> \delta,\eta}_n$ & $Z_n\mathbbm{1}\{\eta|Z_n|>\delta\}$ \hfill``large'' noise  \cref{defLargeJump_GradientClipping}
    \\
    $T^\eta_{j}(\delta)$ & $\min\{ n > T^\eta_{j-1}(\delta):\ \eta|Z_n| > \delta  \}$\hfill arrival time of $j$-th large noise \cref{defArrivalTime large jump}
    \\
    $W^\eta_j(\delta)$ & $Z_{T^\eta_j(\delta)}$ \hfill size of $j$-th large noise \cref{defSize large jump}
    \\
    $X^\eta_n(x)$ & $ X^\eta_{n+1}(x) =  \varphi_L\Big(X^{\eta}_n(x) - \varphi_b\big(\eta (f^\prime(X^{\eta}_n(x)) - Z_{n+1})\big)\Big)$,\quad $X^\eta_0(x) = x$\hfill SGD
    \\
    $\textbf{y}^\eta_n(x)$
    & 
    $\textbf{y}_n^\eta(x) = \textbf{y}_{n-1}^\eta(x) - \eta f^\prime ( \textbf{y}_{n-1}^\eta(x))$,\quad $\textbf{y}_0^\eta(x) = x$\hfill GD
    \\
    \\    
    ${Y}^{\eta }_n(x)$
    & 
    $\textbf{y}^\eta_n(x)$ perturbed by large noises $(\mathbf{T}^\eta(\delta), \mathbf{W}^\eta(\delta))$
    \hfill GD + large jump
    \\    
    $\widetilde{\textbf{y}}^\eta_n(x;\textbf{t},\textbf{w})$
    & 
    $\textbf{y}^\eta_n(x)$ perturbed by noise vector $(\textbf{t},\textbf{w})$
    \hfill perturbed GD
    \\    
    $\textbf{x}^\eta(t,x)$ & $d\textbf{x}^\eta(t;x) = -\eta f^\prime\Big( \textbf{x}^\eta(t;x) \Big)dt$,\quad $\textbf{x}^\eta(0;x) = x$ \hfill ODE
    \\
    $\textbf{x}(t,x)$ &     $\textbf{x}^1(t,x)$ \hfill 
    \\
    $\widetilde{\textbf{x}}^\eta(t,x;\textbf{t},\textbf{w})$ & $\textbf{x}^\eta(t,x)$ perturbed by noise vector $(\textbf{t},\textbf{w})$ \hfill perturbed ODE
    \\    
    $A(n,\eta,\epsilon,\delta)$ & $\displaystyle\Big\{ \max_{k \in[n \wedge (T^\eta_1(\delta) - 1)]  }\eta|Z_1 + \cdots + Z_k| \leq \epsilon \Big\}.$\hfill  \cref{def event A small noise large deviation}   
    \\
    $r$ & $r \delequal{}\min\{ - s_{-}, s_+ \}$. Effective radius of the attraction field $\Omega$. \\
    $l^*$ & $l^* \delequal{} \ceil{ r/b }$. The minimum number of jumps required to escape $\Omega$ when starting from its local minimum $m = 0$. \\
     $h(\textbf{w},\textbf{t})$ & A mapping defined as $ h(\textbf{w},\textbf{t})  = \widetilde{\textbf{x}}(t_{l^*},0;\textbf{t},\textbf{w})$. \\ 
    $\bar{t},\ \bar{\delta}$ & Necessary conditions for $h(\textbf{w},\textbf{t})$ to be outside of $\Omega$ \hfill \cref{def bar t}-\cref{def bar delta}\\
    $\hat{t}(\epsilon)$ & $\hat{t}(\epsilon) \delequal{} c_1\log(1/\epsilon).$ The quantity $\hat{t}(\epsilon)/\eta$ provides an upper bound for the time it takes $\textbf{x}^\eta$ to return to $2\epsilon-$neighborhood of local minimum $m = 0$ when starting from somewhere $\epsilon-$away from $s_-,s_+$. See \cref{ineq prior to function hat t}. \\
    $E(\epsilon)$ & 
    $\big\{(\textbf{w},\textbf{t}) \subseteq \mathbb{R}^{l^*}\times \mathbb{R}_{+}^{l^* - 1}: h(\textbf{w},\textbf{t}) \notin [ (s_- -\epsilon)\vee(-L),(s_+ + \epsilon)\wedge L  ] \big\}$\\
    $ p(\epsilon,\delta,\eta)$ & The probability that, for $\textbf{t} = \big(T^\eta_j(\delta) - 1\big)_{j = 1}^{l^*}$ and $\textbf{w} = \big(\eta W^\eta_j(\delta) \big)_{j = 1}^{l^*}$, we have $(\textbf{w},\textbf{t}) \in E(\epsilon)$ conditioning on $\{T^\eta_1(\delta) = 1\}$. Intuitively speaking, it characterizes the probability that the first $l^*$ \textit{large} noises alone can drive the ODE out of the attraction field. Defined in \cref{def overflow conditional probability}.  \\
    $\nu_\alpha$ & The Borel measure on $\mathbb{R}$ with density $$\nu_\alpha(dx) = \mathbbm{1}\{x > 0\}\frac{\alpha p_+}{x^{\alpha + 1}} + \mathbbm{1}\{x <0\}\frac{\alpha p_-}{|x|^{\alpha + 1}}$$
    where $p_-,p_+$ are constants in Assumption 2 in the main paper.\\
    $\mu$ & The product measure $\mu= (\nu_\alpha)^{l^*}\times(\textbf{Leb}_+)^{l^* - 1}.$ \\
    $\sigma(\eta)$ & $\min\{n \geq 0: X^\eta_n \notin \Omega\}.$\hfill first exit time
    \\
    $H(x)$ & $\P(|Z_1| > x) = x^{-\alpha} L(x)$
    \\
    $T_\text{return}(\epsilon, \eta)$ & $\min\{ n \geq 0: X^\eta_n(x) \in [-2\epsilon,2\epsilon] \}$
    \\

\end{longtable}

\pagebreak

\section{Results about tail distributions of noises in our numerical experiments}

\subsection{QQ plots}

QQ plots below clearly show that the tails in noise distribution are always much lighter than the Pareto distributions with alpha = 2 or even 10. In fact, the tail of noise distributions seem to be between that of lognormal and normal distributions, implying that it is lighter than any power-law distribution.

\begin{figure}[h]
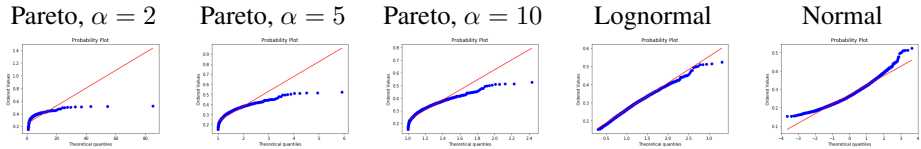

\begin{center}
\begin{tabular}{c c c c c}
Pareto, $\alpha = 2$ & Pareto, $\alpha = 5$ & Pareto, $\alpha = 10$ & Lognormal & Normal \\
\includegraphics[width=0.15\textwidth]{QQplots/LeNet100_pareto2.png}  &
\includegraphics[width=0.15\textwidth]{QQplots/LeNet100_pareto5.png}  &
\includegraphics[width=0.15\textwidth]{QQplots/LeNet100_pareto10.png}  &
\includegraphics[width=0.15\textwidth]{QQplots/LeNet100_lognorm.png}  &
\includegraphics[width=0.15\textwidth]{QQplots/LeNet100_normal.png}
\\
\end{tabular}
\caption{Ablation Study, Corrupted FMNIST \& LeNet: At the beginning}
\end{center}
\end{figure}

\begin{figure}[h]
\begin{center}
\begin{tabular}{c c c c c}
Pareto, $\alpha = 2$ & Pareto, $\alpha = 5$ & Pareto, $\alpha = 10$ & Lognormal & Normal \\
\includegraphics[width=0.15\textwidth]{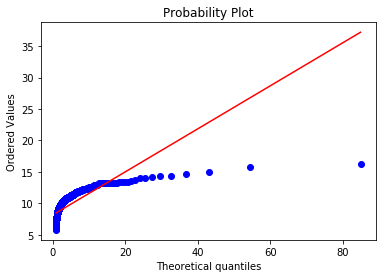}  &
\includegraphics[width=0.15\textwidth]{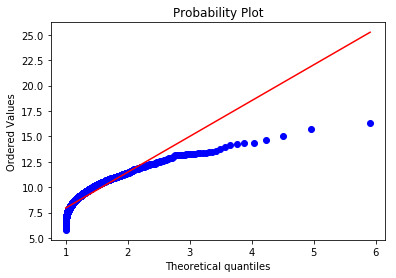}  &
\includegraphics[width=0.15\textwidth]{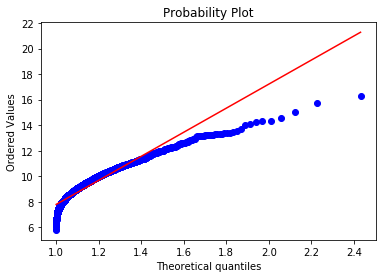}  &
\includegraphics[width=0.15\textwidth]{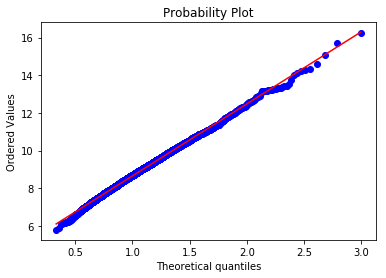}  &
\includegraphics[width=0.15\textwidth]{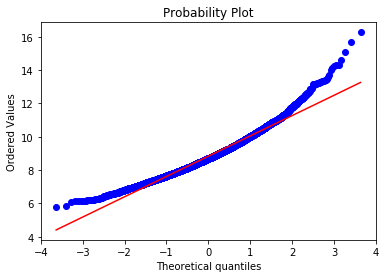}
\\
\end{tabular}
\caption{Ablation Study, Corrupted FMNIST \& LeNet: Half way through the training}
\end{center}
\end{figure}

\begin{figure}[h]
\begin{center}
\begin{tabular}{c c c c c}
Pareto, $\alpha = 2$ & Pareto, $\alpha = 5$ & Pareto, $\alpha = 10$ & Lognormal & Normal \\
\includegraphics[width=0.15\textwidth]{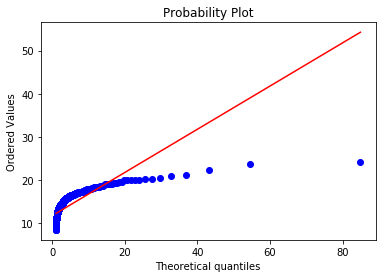}  &
\includegraphics[width=0.15\textwidth]{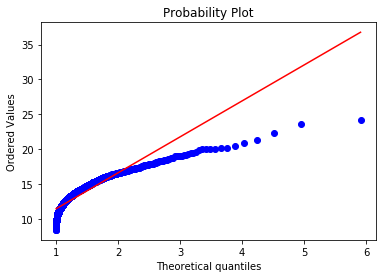}  &
\includegraphics[width=0.15\textwidth]{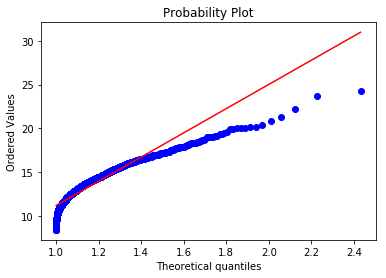}  &
\includegraphics[width=0.15\textwidth]{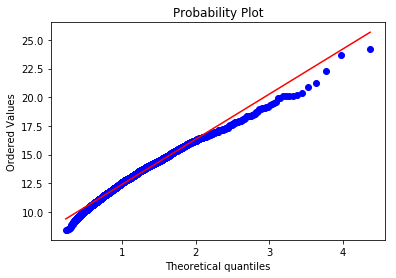}  &
\includegraphics[width=0.15\textwidth]{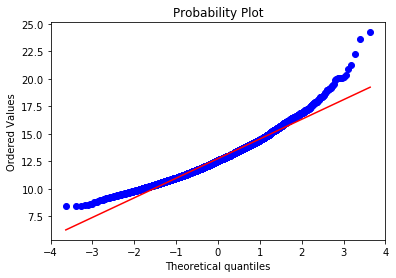}
\\
\end{tabular}
\caption{Ablation Study, Corrupted FMNIST \& LeNet: At the end of training}
\end{center}
\end{figure}


\begin{figure}[!]
\begin{center}
\begin{tabular}{c c c c c}
Pareto, $\alpha = 2$ & Pareto, $\alpha = 5$ & Pareto, $\alpha = 10$ & Lognormal & Normal \\
\includegraphics[width=0.15\textwidth]{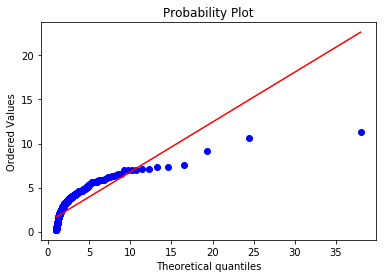}  &
\includegraphics[width=0.15\textwidth]{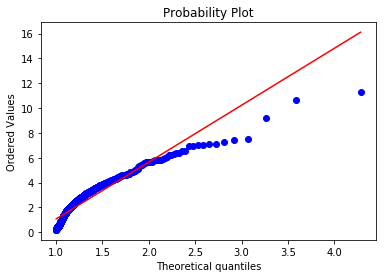}  &
\includegraphics[width=0.15\textwidth]{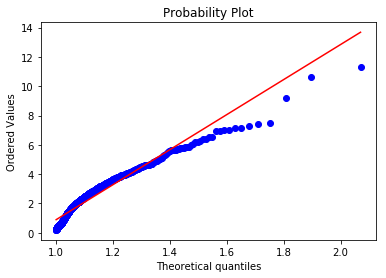}  &
\includegraphics[width=0.15\textwidth]{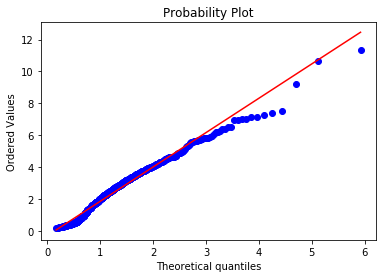}  &
\includegraphics[width=0.15\textwidth]{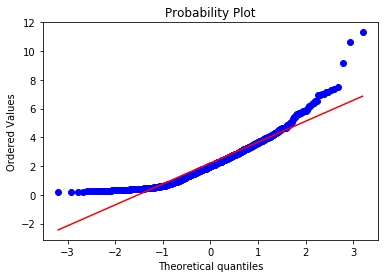}
\\
\end{tabular}
\caption{Ablation Study, SVHN \& VGG11: At the beginning}
\end{center}
\end{figure}

\begin{figure}[!]
\begin{center}
\begin{tabular}{c c c c c}
Pareto, $\alpha = 2$ & Pareto, $\alpha = 5$ & Pareto, $\alpha = 10$ & Lognormal & Normal \\
\includegraphics[width=0.15\textwidth]{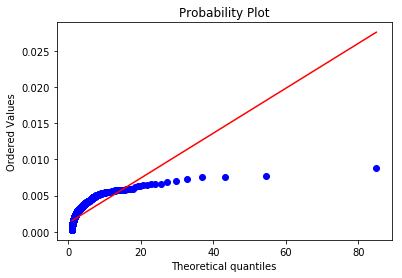}  &
\includegraphics[width=0.15\textwidth]{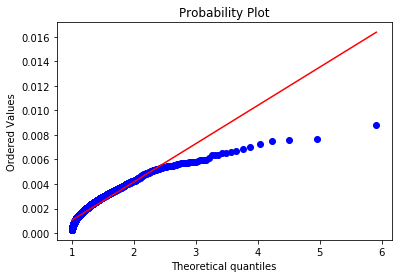}  &
\includegraphics[width=0.15\textwidth]{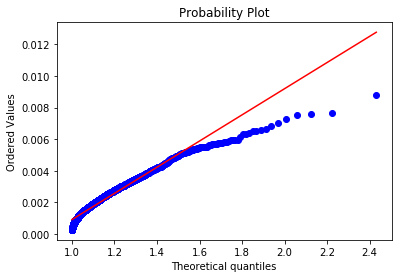}  &
\includegraphics[width=0.15\textwidth]{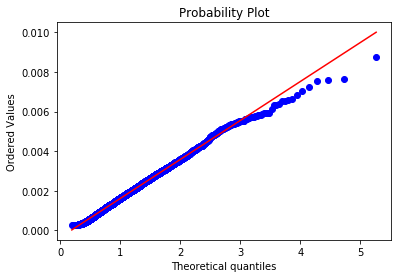}  &
\includegraphics[width=0.15\textwidth]{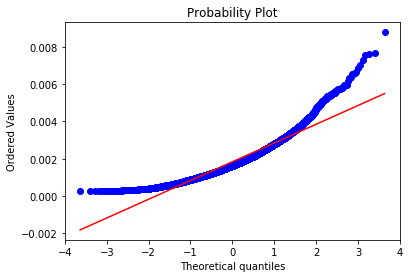}
\\
\end{tabular}
\caption{Ablation Study, SVHN \& VGG11: Half way through the training}
\end{center}
\end{figure}

\begin{figure}[!]
\begin{center}
\begin{tabular}{c c c c c}
Pareto, $\alpha = 2$ & Pareto, $\alpha = 5$ & Pareto, $\alpha = 10$ & Lognormal & Normal \\
\includegraphics[width=0.15\textwidth]{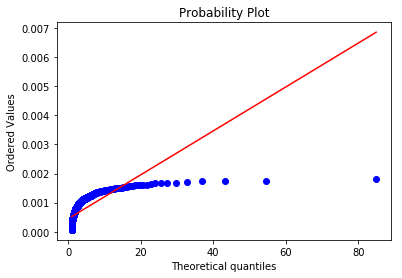}  &
\includegraphics[width=0.15\textwidth]{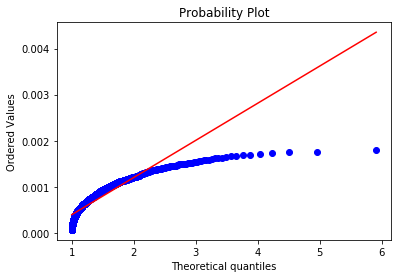}  &
\includegraphics[width=0.15\textwidth]{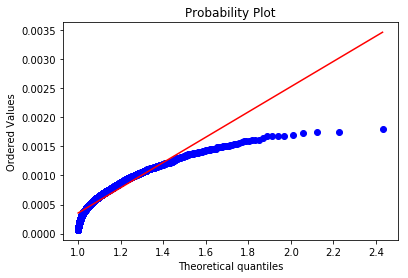}  &
\includegraphics[width=0.15\textwidth]{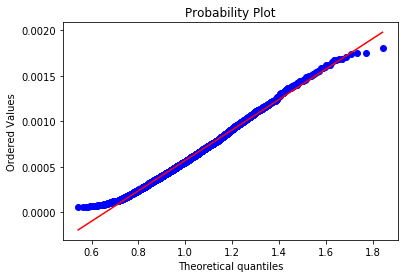}  &
\includegraphics[width=0.15\textwidth]{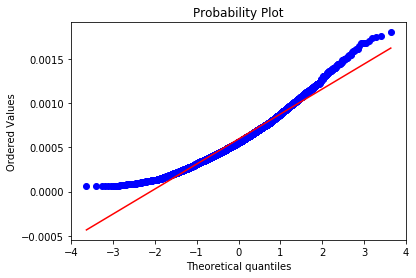}
\\
\end{tabular}
\caption{Ablation Study, SVHN \& VGG11: At the end of training}
\end{center}
\end{figure}


\begin{figure}[!]
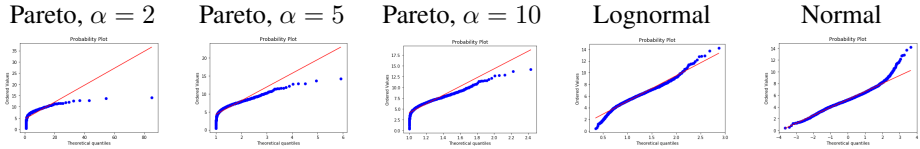

\begin{center}
\begin{tabular}{c c c c c}
Pareto, $\alpha = 2$ & Pareto, $\alpha = 5$ & Pareto, $\alpha = 10$ & Lognormal & Normal \\
\includegraphics[width=0.15\textwidth]{QQplots/CIFAR10_batch100_pareto2.png}  &
\includegraphics[width=0.15\textwidth]{QQplots/CIFAR10_batch100_pareto5.png}  &
\includegraphics[width=0.15\textwidth]{QQplots/CIFAR10_batch100_pareto10.png}  &
\includegraphics[width=0.15\textwidth]{QQplots/CIFAR10_batch100_lognorm.png}  &
\includegraphics[width=0.15\textwidth]{QQplots/CIFAR10_batch100_normal.png}
\\
\end{tabular}
\caption{Ablation Study, CIFAR10 \& VGG11: At the beginning}
\end{center}
\end{figure}

\begin{figure}[!]
\begin{center}
\begin{tabular}{c c c c c}
Pareto, $\alpha = 2$ & Pareto, $\alpha = 5$ & Pareto, $\alpha = 10$ & Lognormal & Normal \\
\includegraphics[width=0.15\textwidth]{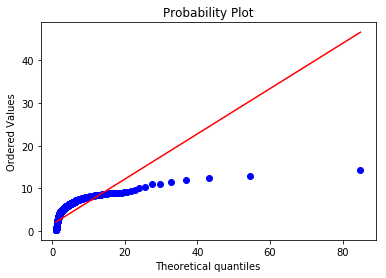}  &
\includegraphics[width=0.15\textwidth]{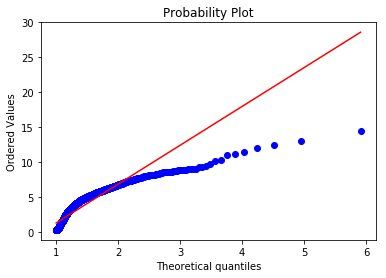}  &
\includegraphics[width=0.15\textwidth]{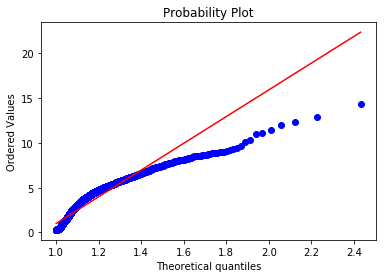}  &
\includegraphics[width=0.15\textwidth]{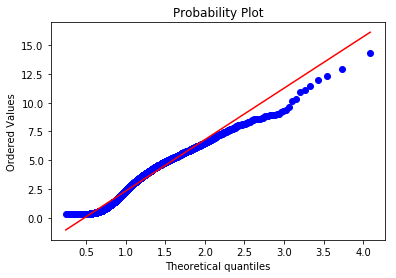}  &
\includegraphics[width=0.15\textwidth]{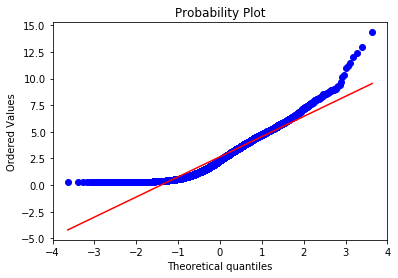}
\\
\end{tabular}
\caption{Ablation Study, CIFAR10 \& VGG11: Half way through the training}
\end{center}
\end{figure}

\begin{figure}[!]
\begin{center}
\begin{tabular}{c c c c c}
Pareto, $\alpha = 2$ & Pareto, $\alpha = 5$ & Pareto, $\alpha = 10$ & Lognormal & Normal \\
\includegraphics[width=0.15\textwidth]{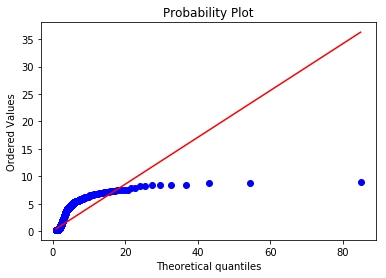}  &
\includegraphics[width=0.15\textwidth]{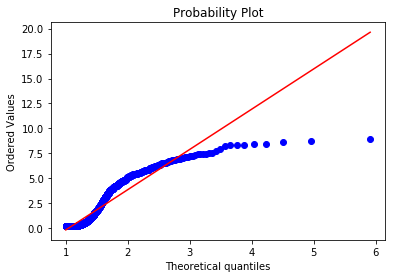}  &
\includegraphics[width=0.15\textwidth]{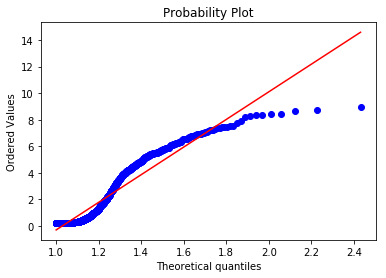}  &
\includegraphics[width=0.15\textwidth]{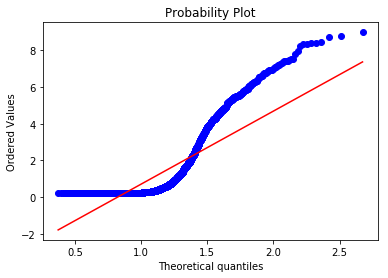}  &
\includegraphics[width=0.15\textwidth]{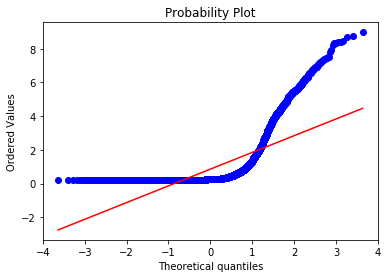}
\\
\end{tabular}
\caption{Ablation Study, CIFAR10 \& VGG11: At the end of training}
\end{center}
\end{figure}


\begin{figure}[!]
\begin{center}
\begin{tabular}{c c c c c}
Pareto, $\alpha = 2$ & Pareto, $\alpha = 5$ & Pareto, $\alpha = 10$ & Lognormal & Normal \\
\includegraphics[width=0.15\textwidth]{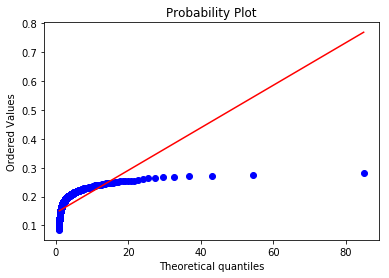}  &
\includegraphics[width=0.15\textwidth]{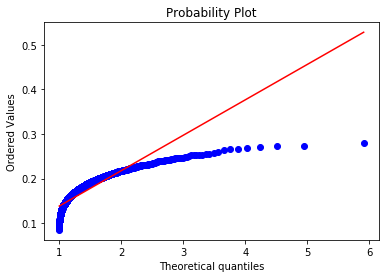}  &
\includegraphics[width=0.15\textwidth]{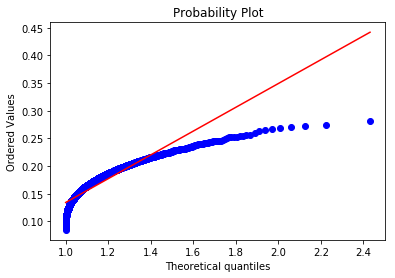}  &
\includegraphics[width=0.15\textwidth]{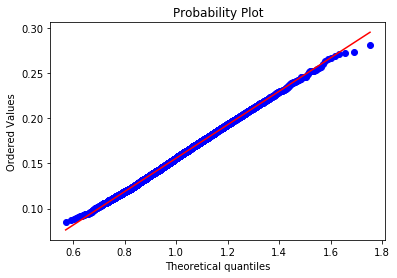}  &
\includegraphics[width=0.15\textwidth]{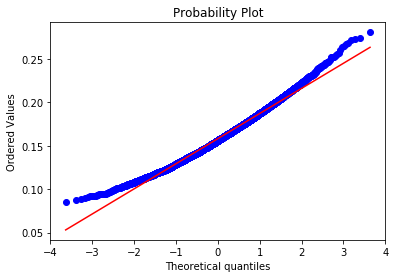}
\\
\end{tabular}
\caption{Data Augmentation, CIFAR10 \& VGG11: At the beginning}
\end{center}
\end{figure}

\begin{figure}[!]
\begin{center}
\begin{tabular}{c c c c c}
Pareto, $\alpha = 2$ & Pareto, $\alpha = 5$ & Pareto, $\alpha = 10$ & Lognormal & Normal \\
\includegraphics[width=0.15\textwidth]{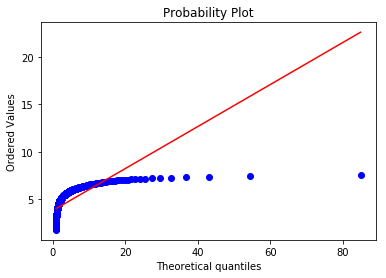}  &
\includegraphics[width=0.15\textwidth]{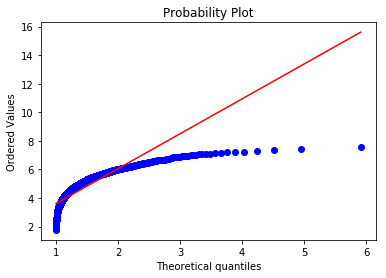}  &
\includegraphics[width=0.15\textwidth]{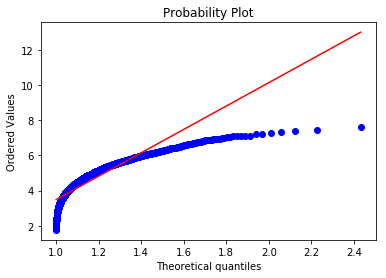}  &
\includegraphics[width=0.15\textwidth]{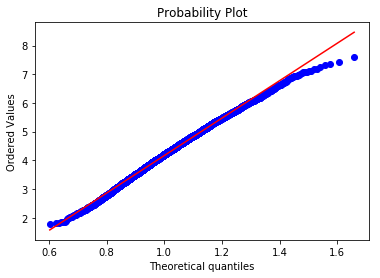}  &
\includegraphics[width=0.15\textwidth]{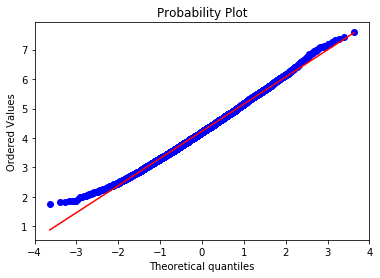}
\\
\end{tabular}
\caption{Data Augmentation, CIFAR10 \& VGG11: Half way through the training}
\end{center}
\end{figure}

\begin{figure}[!]
\begin{center}
\begin{tabular}{c c c c c}
Pareto, $\alpha = 2$ & Pareto, $\alpha = 5$ & Pareto, $\alpha = 10$ & Lognormal & Normal \\
\includegraphics[width=0.15\textwidth]{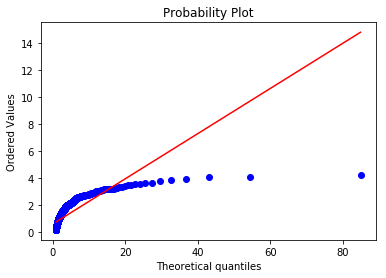}  &
\includegraphics[width=0.15\textwidth]{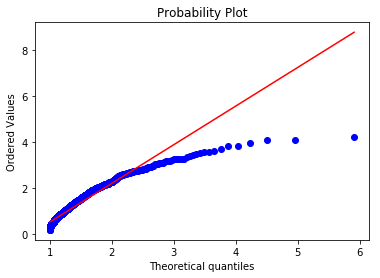}  &
\includegraphics[width=0.15\textwidth]{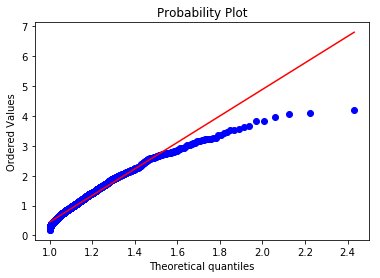}  &
\includegraphics[width=0.15\textwidth]{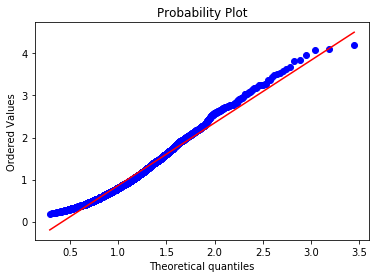}  &
\includegraphics[width=0.15\textwidth]{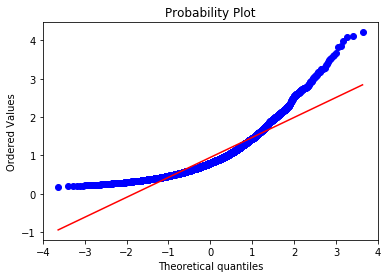}
\\
\end{tabular}
\caption{Data Augmentation, CIFAR10 \& VGG11: At the end of training}
\end{center}
\end{figure}


\begin{figure}[!]
\begin{center}
\begin{tabular}{c c c c c}
Pareto, $\alpha = 2$ & Pareto, $\alpha = 5$ & Pareto, $\alpha = 10$ & Lognormal & Normal \\
\includegraphics[width=0.15\textwidth]{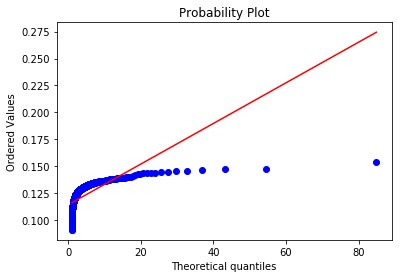}  &
\includegraphics[width=0.15\textwidth]{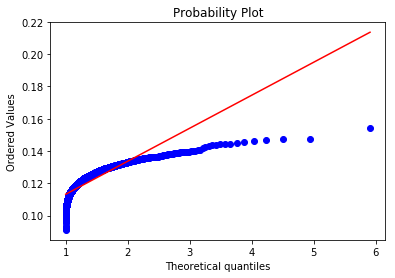}  &
\includegraphics[width=0.15\textwidth]{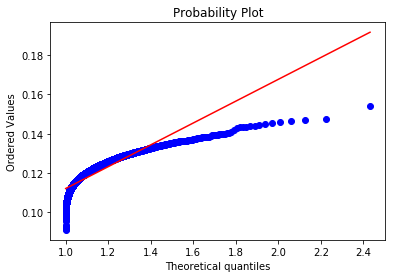}  &
\includegraphics[width=0.15\textwidth]{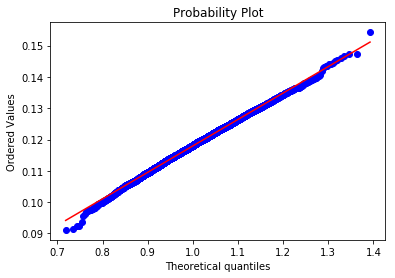}  &
\includegraphics[width=0.15\textwidth]{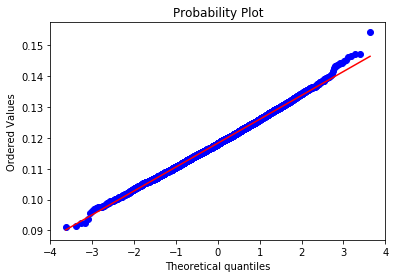}
\\
\end{tabular}
\caption{Data Augmentation, CIFAR100 \& VGG16: At the beginning}
\end{center}
\end{figure}

\begin{figure}[!]
\begin{center}
\begin{tabular}{c c c c c}
Pareto, $\alpha = 2$ & Pareto, $\alpha = 5$ & Pareto, $\alpha = 10$ & Lognormal & Normal \\
\includegraphics[width=0.15\textwidth]{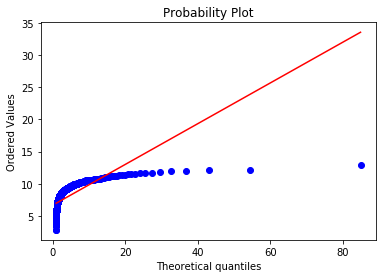}  &
\includegraphics[width=0.15\textwidth]{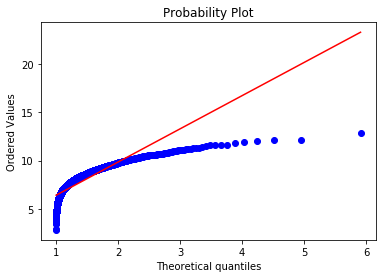}  &
\includegraphics[width=0.15\textwidth]{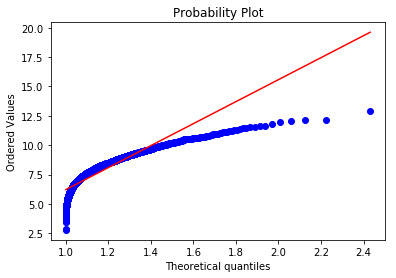}  &
\includegraphics[width=0.15\textwidth]{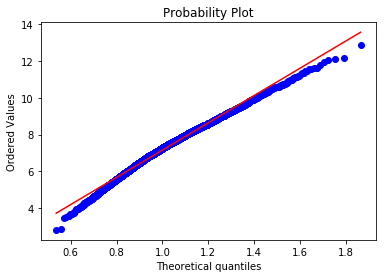}  &
\includegraphics[width=0.15\textwidth]{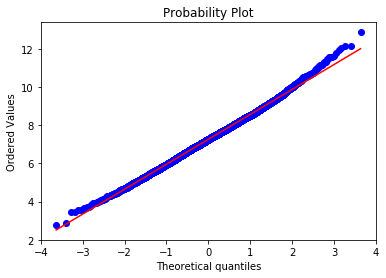}
\\
\end{tabular}
\caption{Data Augmentation, CIFAR100 \& VGG16: Halfway through the training}
\end{center}
\end{figure}

\begin{figure}[!]
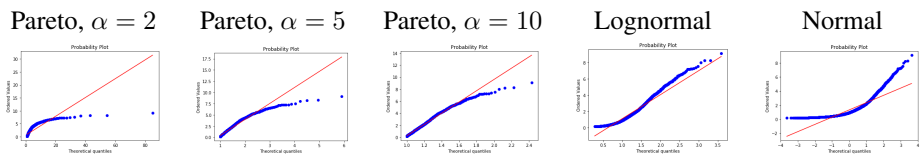

\begin{center}
\begin{tabular}{c c c c c}
Pareto, $\alpha = 2$ & Pareto, $\alpha = 5$ & Pareto, $\alpha = 10$ & Lognormal & Normal \\
\includegraphics[width=0.15\textwidth]{QQplots/dataAug_CIFAR100_batch128_end_pareto2.png}  &
\includegraphics[width=0.15\textwidth]{QQplots/dataAug_CIFAR100_batch128_end_pareto5.png}  &
\includegraphics[width=0.15\textwidth]{QQplots/dataAug_CIFAR100_batch128_end_pareto10.png}  &
\includegraphics[width=0.15\textwidth]{QQplots/dataAug_CIFAR100_batch128_end_lognorm.png}  &
\includegraphics[width=0.15\textwidth]{QQplots/dataAug_CIFAR100_batch128_end_normal.png}
\\
\end{tabular}
\caption{Data Augmentation, CIFAR100 \& VGG16: At the end of training}
\end{center}
\end{figure}


\pagebreak
\subsection{Empirical mean residual life (EMRL) plots}

It is well known that the mean residual life blows up to infinity if and only if the distribution is heavy-tailed (more precisely, long-tailed). However, from the figures below, one can see that none of the EMRL exhibits such a pattern in any case tested in our experiments. Instead, we see clear downward trends, which strongly suggests light tails in all cases tested.

\begin{figure}[h]
\begin{center}
\begin{tabular}{c c c}
Beginning & Middle & End \\
\includegraphics[width=0.25\textwidth]{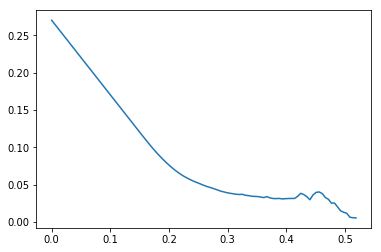}  &
\includegraphics[width=0.25\textwidth]{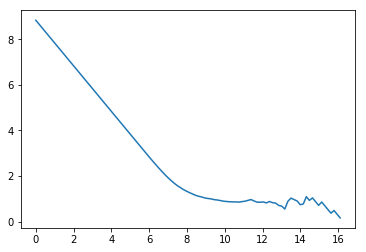}  &
\includegraphics[width=0.25\textwidth]{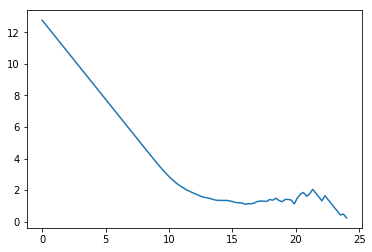}
\\
\end{tabular}
\caption{Plots of empirical mean residual life for noises in FMNIST\&LeNet Task throughout training}
\end{center}
\end{figure}

\begin{figure}[!]
\begin{center}
\begin{tabular}{c c c}
Beginning & Middle & End \\
\includegraphics[width=0.25\textwidth]{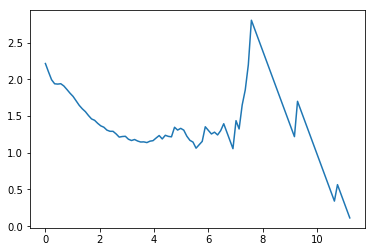}  &
\includegraphics[width=0.25\textwidth]{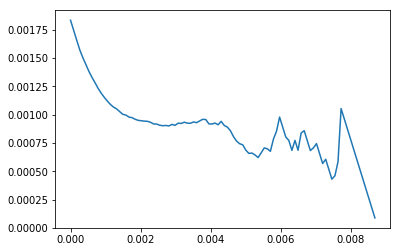}  &
\includegraphics[width=0.25\textwidth]{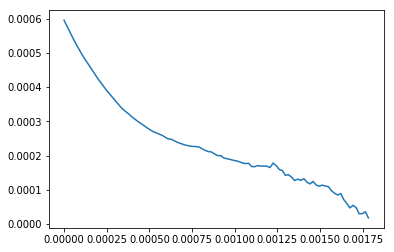}
\\
\end{tabular}
\caption{Plots of empirical mean residual life for noises in SVHN\&VGG 11 Task throughout training}
\end{center}
\end{figure}

\begin{figure}[!]
\begin{center}
\begin{tabular}{c c c}
Beginning & Middle & End \\
\includegraphics[width=0.25\textwidth]{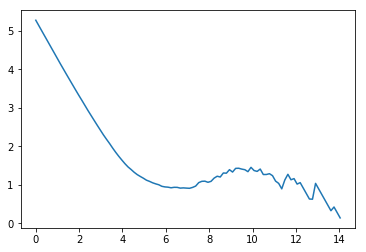}  &
\includegraphics[width=0.25\textwidth]{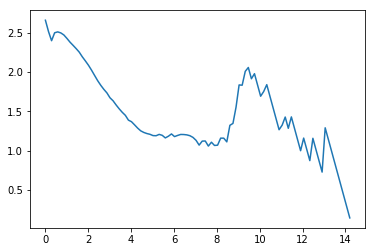}  &
\includegraphics[width=0.25\textwidth]{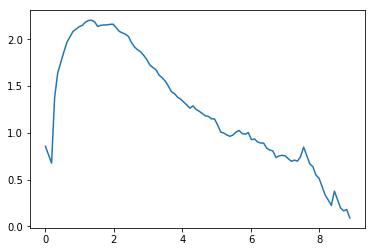}
\\
\end{tabular}
\caption{Plots of empirical mean residual life for noises in CIFAR 10\&VGG 11 Task throughout training}
\end{center}
\end{figure}

\begin{figure}[!]
\begin{center}
\begin{tabular}{c c c}
Beginning & Middle & End \\
\includegraphics[width=0.25\textwidth]{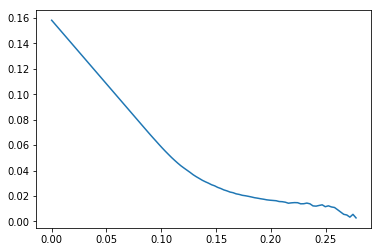}  &
\includegraphics[width=0.25\textwidth]{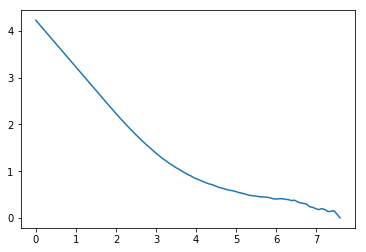}  &
\includegraphics[width=0.25\textwidth]{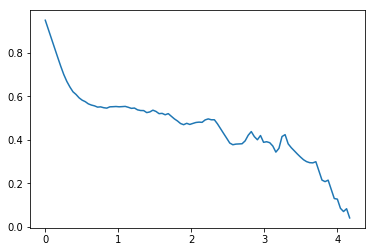}
\\
\end{tabular}
\caption{Plots of empirical mean residual life for noises in dataAug, CIFAR 10\&VGG 11 Task throughout training}
\end{center}
\end{figure}

\begin{figure}[!]
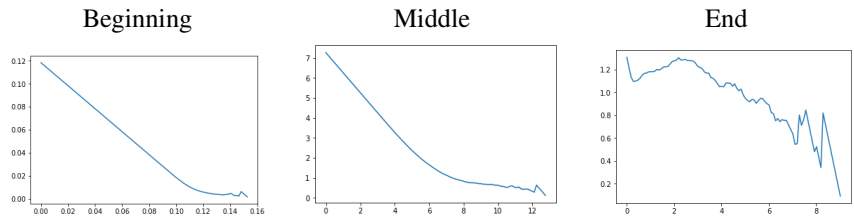

\begin{center}
\begin{tabular}{c c c}
Beginning & Middle & End \\
\includegraphics[width=0.25\textwidth]{EMRL/dataAug_CIFAR100_batch128_EMRL.png}  &
\includegraphics[width=0.25\textwidth]{EMRL/dataAug_CIFAR100_batch128_middle_EMRL.png}  &
\includegraphics[width=0.25\textwidth]{EMRL/dataAug_CIFAR100_batch128_end_EMRL.png}
\\
\end{tabular}
\caption{Plots of empirical mean residual life for noises in dataAug, CIFAR 100\&VGG 11 Task throughout training}
\end{center}
\end{figure}

\subsection{Hill plots}

In the hill plots below, the estimated power-law indices using only the top 1\% of samples (on the left hand side of the dashed lines) stay well above 10 for the most part and almost never drop below 2. This strongly suggests that even if the gradient noises are from a heavy-tailed distribution, it is likely to have a very high power law index (implying relatively lighter tails), and hence, we cannot expect to observe a prominent heavy-tailed behavior from them.

\begin{figure}[h]
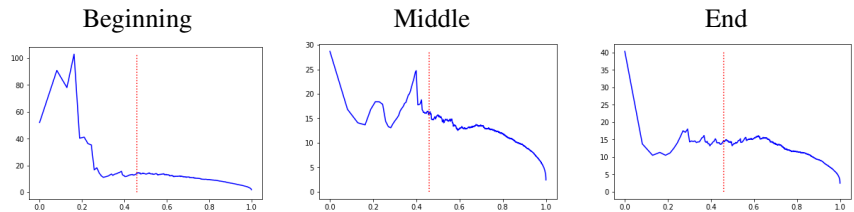

\begin{center}
\begin{tabular}{c c c}
Beginning & Middle & End \\
\includegraphics[width=0.25\textwidth]{Hill Plots/LeNet100_hill.png}  &
\includegraphics[width=0.25\textwidth]{Hill Plots/LeNet100_middle_hill.png}  &
\includegraphics[width=0.25\textwidth]{Hill Plots/LeNet100_end_hill.png}
\\
\end{tabular}
\caption{altHill Plots for noises in FMNIST\&LeNet Task throughout training. Dashed Red Line: Estimation based on the largest $1\%$ data}
\end{center}
\end{figure}

\begin{figure}[h]
\begin{center}
\begin{tabular}{c c c}
Beginning & Middle & End \\
\includegraphics[width=0.25\textwidth]{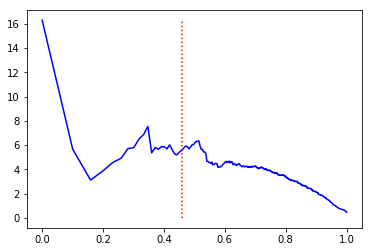}  &
\includegraphics[width=0.25\textwidth]{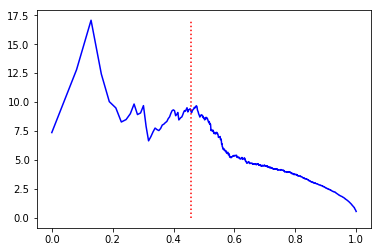}  &
\includegraphics[width=0.25\textwidth]{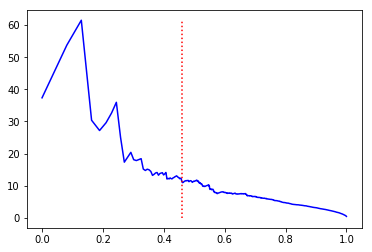}
\\
\end{tabular}
\caption{altHill Plots for noises in SVHN\&VGG 11 Task throughout training. Dashed Red Line: Estimation based on the largest $1\%$ data}
\end{center}
\end{figure}

\begin{figure}[h]
\begin{center}
\begin{tabular}{c c c}
Beginning & Middle & End \\
\includegraphics[width=0.25\textwidth]{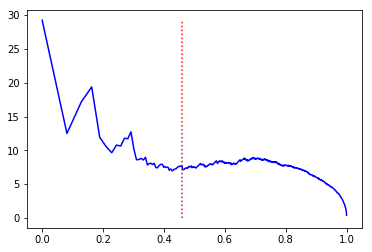}  &
\includegraphics[width=0.25\textwidth]{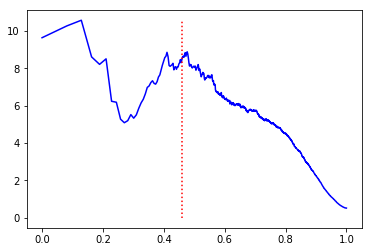}  &
\includegraphics[width=0.25\textwidth]{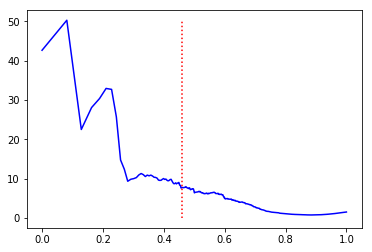}
\\
\end{tabular}
\caption{altHill Plots for noises in CIFAR 10\&VGG 11 Task throughout training. Dashed Red Line: Estimation based on the largest $1\%$ data}
\end{center}
\end{figure}

\begin{figure}[h]
\begin{center}
\begin{tabular}{c c c}
Beginning & Middle & End \\
\includegraphics[width=0.25\textwidth]{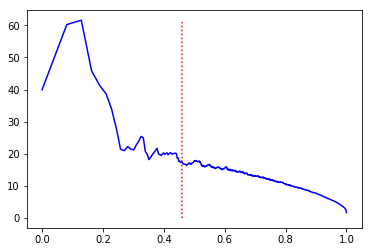}  &
\includegraphics[width=0.25\textwidth]{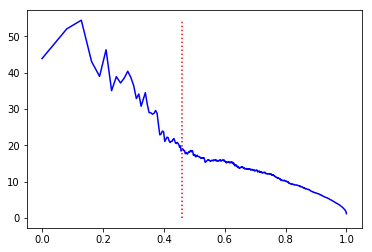}  &
\includegraphics[width=0.25\textwidth]{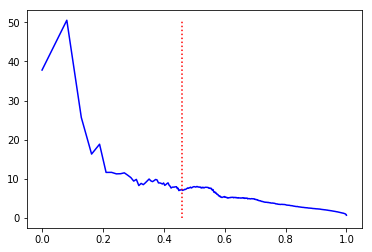}
\\
\end{tabular}
\caption{altHill Plots for noises in Data Augmentation, CIFAR 10\&VGG 11 Task throughout training. Dashed Red Line: Estimation based on the largest $1\%$ data}
\end{center}
\end{figure}

\begin{figure}[!]
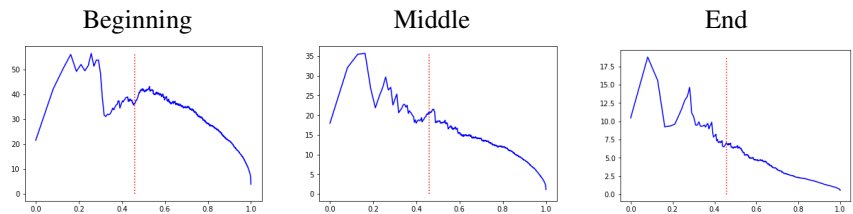

\begin{center}
\begin{tabular}{c c c}
Beginning & Middle & End \\
\includegraphics[width=0.25\textwidth]{Hill Plots/dataAug_CIFAR100_batch128_hill.png}  &
\includegraphics[width=0.25\textwidth]{Hill Plots/dataAug_CIFAR100_batch128_middle_hill.png}  &
\includegraphics[width=0.25\textwidth]{Hill Plots/dataAug_CIFAR100_batch128_end_hill.png}
\\
\end{tabular}
\caption{altHill Plots for noises in Data Augmentation, CIFAR 100\&VGG 16 Task throughout training. Dashed Red Line: Estimation based on the largest $1\%$ data}
\end{center}
\end{figure}


\

\begin{table}
  \caption{Power-law Indices Estimation throughout the Training, using PLFIT. All the estimations are at least 5 for all cases tested in our experiments, and most of the times the estimation is above 10. This means that even under the assumption that the gradient noises were from a heavy-tailed distribution, they should have much lighter tails than any $\alpha$-stable distribution (which requires $\alpha<2$) or the heavy-tailed noises we injected during tail inflation experiments ($\alpha = 1.4$). }
  \centering
  \begin{tabular}{llll}
    \toprule
   Task & Beginning & Middle & End   \\
    \midrule
    FMNIST, LeNet & 14.3 & 14.2 & 16.5 \\
    SVHN, VGG11 & 5.0 & 5.2 & 12.5 \\
    CIFAR10, VGG11 & 9.2 & 6.6 & 7.0 \\
    dataAug, CIFAR10, VGG11 & 16.2 & 16.2 & 8.5 \\
    dataAug, CIFAR100, VGG16 & 35.1 & 14.4 & 5.35 \\
    \bottomrule
  \end{tabular}
\end{table}

\end{document}